\documentclass{article}

\PassOptionsToPackage{numbers, compress}{natbib}

\usepackage[preprint]{neurips_2023}

\usepackage[utf8]{inputenc} 
\usepackage[T1]{fontenc}    
\usepackage{hyperref}       
\usepackage{url}            
\usepackage{booktabs}       
\usepackage{amsfonts}       
\usepackage{nicefrac}       
\usepackage{microtype}      
\usepackage{xcolor}         

\usepackage{amsmath}    
\usepackage{amsthm}
\usepackage{graphicx}
\usepackage{subcaption}
\usepackage[ruled,vlined]{algorithm2e}

\newtheorem{theorem}{Theorem}
\newtheorem{example}{Example}[section]

\newcommand{\ts}{\textsuperscript}

\title{Diffusion Sampling with Momentum for Mitigating Divergence Artifacts}

\author{
  Suttisak ~Wizadwongsa \\
  VISTEC, Thailand \\
  \texttt{suttisak.w\textunderscore s19@vistec.ac.th} \\
  \And
  Worameth Chinchuthakun \\
  Tokyo Institute of Technology, Japan \\
  \texttt{chinchuthakun.w.aa@m.titech.ac.jp} \\
  \And
  Pramook Khungurn \\
  pixiv Inc.\\
  \texttt{pramook@gmail.com} \\
  \And
  Amit Raj \\
  Google \\
  \texttt{amitrajs@google.com} \\
  \And
  Supasorn ~Suwajanakorn \\
  VISTEC, Thailand \\
  \texttt{supasorn.s@vistec.ac.th} \\
}

\usepackage{tikz}
\usepackage{tabularray}
\usepackage{adjustbox}
\usepackage{tabularx}
\usepackage{tabu}
\usepackage[position=bottom]{subcaption}
\usepackage{caption}

\usepackage[toc,page,header]{appendix}
\usepackage{titletoc}
 
\begin{document}

\maketitle

\begin{abstract}
Despite the remarkable success of diffusion models in image generation, slow sampling remains a persistent issue. To accelerate the sampling process, prior studies have reformulated diffusion sampling as an ODE/SDE and introduced higher-order numerical methods. However, these methods often produce \emph{divergence} artifacts, especially with a low number of sampling steps, which limits the achievable acceleration. In this paper, we investigate the potential causes of these artifacts and suggest that the small stability regions of these methods could be the principal cause. To address this issue, we propose two novel techniques. The first technique involves the incorporation of Heavy Ball (HB) momentum, a well-known technique for improving optimization, into existing diffusion numerical methods to expand their stability regions.  We also prove that the resulting methods have first-order convergence. The second technique, called Generalized Heavy Ball (GHVB), constructs a new high-order method that offers a variable trade-off between accuracy and artifact suppression. 
Experimental results show that our techniques are highly effective in reducing artifacts and improving image quality, surpassing state-of-the-art diffusion solvers on both pixel-based and latent-based diffusion models for low-step sampling.
Our research provides novel insights into the design of numerical methods for future diffusion work.
\end{abstract}
\section{Introduction}
Diffusion models \cite{ho2020denoising, song2020denoising} are a type of generative models that has garnered considerable attention due to their remarkable image quality. Unlike Generative Adversarial Networks (GANs) \cite{goodfellow2014generative}, which may suffer from mode collapse and instabilities during training, diffusion models offer reduced sensitivity to hyperparameters \citep{ho2020denoising, rombach2022high} and improve sampling quality \citep{dhariwal2021diffusion}. Additionally, diffusion models have been successfully applied to various image-related tasks, such as text-to-image generation \cite{nichol2021glide}, image-to-image translation \cite{su2022dual}, image composition \citep{sasaki2021unit}, adversarial purification \citep{wang2022guided, wu2022guided}, super-resolution \citep{choi2021ilvr}, and text-to-audio conversion \cite{ghosal2023text}.

One significant drawback of diffusion models, however, is their slow sampling speed.
This is because the sampling process involves a Markov chain that requires a large number of iterations to generate high-quality results. Recent attempts to accelerate the process include improvements to the noise schedule \citep{nichol2021improved, watson2021learning} and network distillation \cite{salimans2022progressive, watson2022learning, song2023consistency}. Fortunately, the sampling process can be represented by ordinary or stochastic differential equations, and numerical methods can be used to reduce the number of iterations required. While DDIM \cite{song2020denoising}, a 1\ts{st}-order method, is the most commonly used approach, it still requires a considerable number of iterations. Higher-order numerical methods, such as DEIS \cite{zhang2022fast}, DPM-Solver \cite{lu2022dpm}, and PLMS \cite{liu2022pseudo}, have been proposed to generate high-quality images in fewer steps. However, these methods begin to produce artifacts (see Figure \ref{fig:highlighted_image}) when the number of steps is decreased beyond a certain value, thereby limiting how much we can reduce the sampling time.

In this study, we investigate the potential causes of these artifacts and found that the narrow stability region of high-order numerical methods can cause solutions to diverge, resulting in divergence artifacts. 
To address this issue and enable low-step, artifact-free sampling, we propose two techniques. 
The first technique involves incorporating Polyak's Heavy Ball (HB) momentum \cite{polyak1987introduction}, a well-known technique for improving optimization, into existing diffusion numerical methods. This approach effectively reduces divergence artifacts, but its accuracy only has first order of convergence. In this context, the accuracy measures how close the approximated, low-step solution is to the solution computed from a very high-step solver (e.g., 1,000-step DDIM). 
The second technique, called Generalized Heavy Ball (GHVB), is a new high-order numerical method that offers a variable trade-off between accuracy and artifact suppression. Both techniques are training-free and incur negligible additional computational costs. Figure \ref{fig:highlighted_image} demonstrates the superiority of both techniques in reducing divergence artifacts compared to previous diffusion sampling methods. Furthermore, our experiments show that our techniques are effective on both pixel-based and latent-based diffusion models.

The paper is structured as follows. 
Section \ref{sec2} covers background and related work on the diffusion sampling process in differential equation forms and stability region. Section \ref{sec3} analyzes visual artifacts in diffusion sampling and establishes a connection to the stability region of the solver. Section \ref{sec4} proposes a technique to apply momentum to existing numerical methods, as well as a technique that generalizes momentum to high-order numerical methods. Section \ref{sec5} presents experiments and ablation studies. Finally, Section \ref{sec6} concludes and discusses the implications and impacts of our work.

\begin{figure}
    \centering
    \begin{tabu} to \textwidth {@{}l@{\hspace{5pt}}c@{\hspace{2pt}}c@{\hspace{2pt}}c@{\hspace{5pt}}c@{\hspace{2pt}}c@{\hspace{2pt}}c@{}}  
        & \multicolumn{3}{c}{\scriptsize DPM-Solver++ (2M) \cite{lu2022dpm}}
        & \multicolumn{3}{c}{\scriptsize PLMS4 \cite{liu2022pseudo}} \\
        
        & \tiny (a) & \tiny (b) &  \tiny (c) &\tiny  (a) & \tiny (b) & \tiny (c) \\
        
        \shortstack[c]{\tiny Without \\ \tiny Momentum} &
        \noindent\parbox[c]{0.14\columnwidth}{\includegraphics[width=0.14\columnwidth]{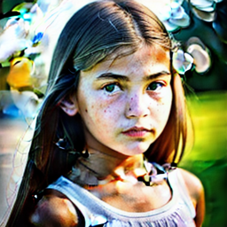}} & 
        \noindent\parbox[c]{0.14\columnwidth}{\includegraphics[width=0.14\columnwidth]{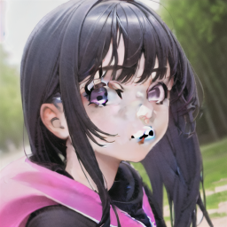}} &
        \noindent\parbox[c]{0.14\columnwidth}{\includegraphics[width=0.14\columnwidth]{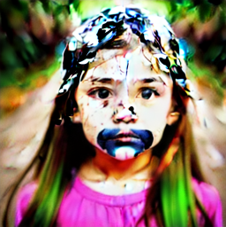}} &
        \noindent\parbox[c]{0.14\columnwidth}{\includegraphics[width=0.14\columnwidth]{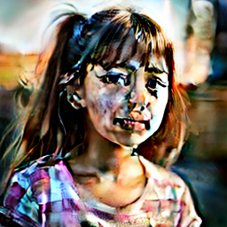}} &
        \noindent\parbox[c]{0.14\columnwidth}{\includegraphics[width=0.14\columnwidth]{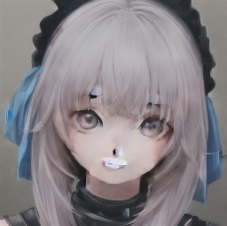}} &
        \noindent\parbox[c]{0.14\columnwidth}{\includegraphics[width=0.14\columnwidth]{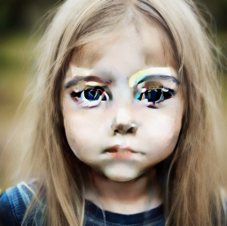}} \\

        \shortstack[c]{\tiny With \tiny HB 0.5 \\ \tiny (Ours)} &
        \noindent\parbox[c]{0.14\columnwidth}{\includegraphics[width=0.14\columnwidth]{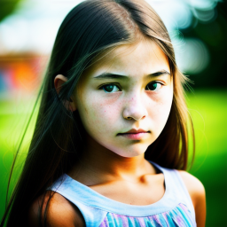}} & 
        \noindent\parbox[c]{0.14\columnwidth}{\includegraphics[width=0.14\columnwidth]{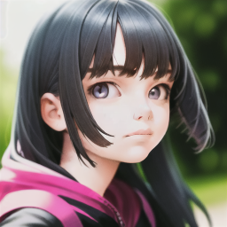}} &
        \noindent\parbox[c]{0.14\columnwidth}{\includegraphics[width=0.14\columnwidth]{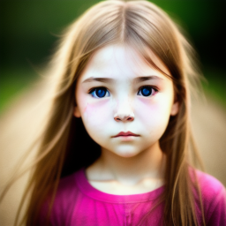}} &
        \noindent\parbox[c]{0.14\columnwidth}{\includegraphics[width=0.14\columnwidth]{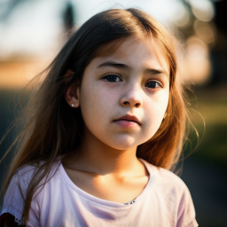}} &
        \noindent\parbox[c]{0.14\columnwidth}{\includegraphics[width=0.14\columnwidth]{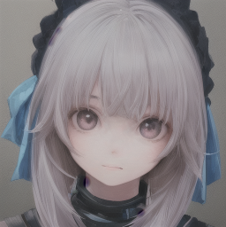}} &
        \noindent\parbox[c]{0.14\columnwidth}{\includegraphics[width=0.14\columnwidth]{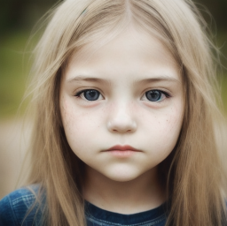}} \\

        \shortstack[c]{\tiny GHVB 1.5 \\ \tiny (Ours)} &
        \noindent\parbox[c]{0.14\columnwidth}{\includegraphics[width=0.14\columnwidth]{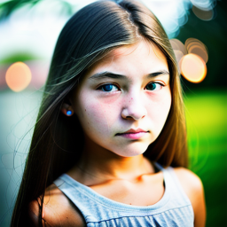}} & 
        \noindent\parbox[c]{0.14\columnwidth}{\includegraphics[width=0.14\columnwidth]{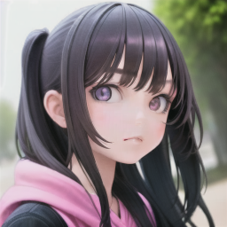}} &
        \noindent\parbox[c]{0.14\columnwidth}{\includegraphics[width=0.14\columnwidth]{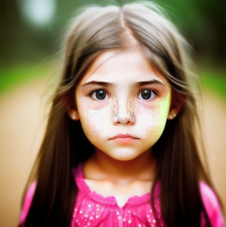}} &
        \noindent\parbox[c]{0.14\columnwidth}{\includegraphics[width=0.14\columnwidth]{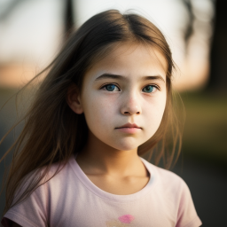}} &
        \noindent\parbox[c]{0.14\columnwidth}{\includegraphics[width=0.14\columnwidth]{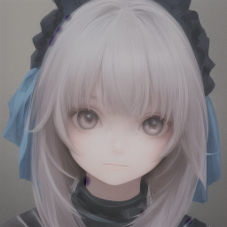}} &
        \noindent\parbox[c]{0.14\columnwidth}{\includegraphics[width=0.14\columnwidth]{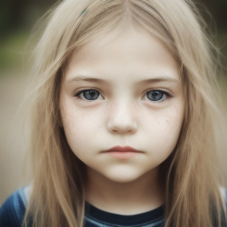}} \\
        
        & \tiny $s = 15$ & \tiny $s = 25$ &  \tiny $s = 20$ &\tiny  $s = 8$ & \tiny $s = 6$ & \tiny $s = 6$ 
    \end{tabu}
    \caption{We demonstrate the occurrence of divergence artifacts in DPM-Solver++\cite{lu2022dpm} and PLMS4\cite{liu2022pseudo} with 15 sampling steps, where $s$ denotes the text-guidance scale. By integrating HB momentum into these methods, we effectively mitigate the artifacts. Additionally, we compare the results with our GHVB 1.5 method. Prompt: "photo of a girl face" (a) Realistic Vision v2.0\cite{Realistic}, (b) Anything Diffusion v4.0\cite{Anything}, (c) Deliberate Diffusion\cite{Deliberate}}
    \label{fig:highlighted_image}
\end{figure}


\section{Background} \label{sec2}
This section first presents the theoretical foundation of diffusion sampling when modeled as an ordinary differential equation (ODE) and related numerical methods. Second, we discuss ODE forms for guided diffusion sampling and prior splitting numerical methods. Third, we cover the concept of stability region, which is our primary analysis tool.


\subsection{Diffusion in ODE Form}
Modeling diffusion sampling as an ODE is commonly based on the non-Markovian sampling of Denoising Diffusion Implicit Model (DDIM) \cite{song2020denoising}.
DDIM is well-known for its simplicity, as it enables deterministic sampling after a random initialization, given by:
\begin{equation} \label{ddim}
x_{t-1} = \sqrt{\frac{\alpha_{t-1}}{\alpha_{t}} }
\left( x_t - \sqrt{1-\alpha_{t}} \epsilon_\theta (x_t,t)\right)
+\sqrt{1-\alpha_{t-1}} \epsilon_\theta (x_t,t).
\end{equation}
Here, $\epsilon_\theta(x_t, t)$ is a neural network that predicts noise, with learnable parameters $\theta$ that take the current state $x_t$ and time $t$ as input. The parameter $\alpha_t$ is a schedule that controls the degree of diffusion at each time step.
Previous research has shown that DDIM \ref{ddim} can be rewritten into an ODE, making it possible to use numerical methods to accelerate the sampling process. Two ODEs have been proposed in the literature:

\begin{minipage}{0.48\textwidth}
\begin{equation} \label{eq:ode}
\frac{d\bar{x}}{d\sigma} = \bar{\epsilon}(\bar{x},\sigma),
\end{equation}
\end{minipage}
\begin{minipage}{0.48\textwidth}
\begin{equation} \label{eq:ode2}
\frac{d\tilde{x}}{d\tilde{\sigma}} = s(\tilde{x},\tilde{\sigma}),
\end{equation}
\end{minipage}

Equation \ref{eq:ode} can be obtained by re-parameterizing $\sigma = \sqrt{1-\alpha_{t}}/ \sqrt{\alpha_{t}}$, \: $\bar{x} = x_t/\sqrt{\alpha_{t}}$, and $\bar{\epsilon}(\bar{x},\sigma) = \epsilon_\theta(x_t, t)$. These transformations are widely used in various diffusion solvers \cite{zhang2022fast, lu2022dpm, liu2022pseudo, zhao2023unipc}. If $\epsilon_\theta(x_t, t)$ is a sum of multiple terms, such as in guided diffusion, we can easily split Equation \ref{eq:ode} and solve each resulting equation separately, as demonstrated in \cite{wizadwongsa2023accelerating}. Another ODE, given by Equation \ref{eq:ode2}, can be derived by defining $\tilde{\sigma} = \sqrt{\alpha_{t}}/\sqrt{1-\alpha_{t}}$ and $\tilde{x} = x_t/\sqrt{1-\alpha_{t}}$, where $s(\tilde{x},\tilde{\sigma}) = (x_t - \sqrt{1-\alpha_{t}} \epsilon_\theta(x_t, t))/\sqrt{\alpha_{t}}$, which is an approximation of the final result. This ODE has the advantage of keeping the differentiation bounded within the pixel value range in pixel-based diffusion.
Recent research on DPM-Solver++ \cite{lu2022dpmpp} has shown that Equation \ref{eq:ode2} outperforms Equation \ref{eq:ode} in many cases.

\subsection{Guided Diffusion Sampling}
Guided diffusion sampling is a widely used technique for conditional sampling, such as text-to-image and class-to-image generation. There are main two approaches for guided sampling: 

\textbf{Classifier guidance} \cite{dhariwal2021diffusion,song2020denoising} uses a pre-trained classifier model $p_\phi(c \mid x_t, t)$ to define the conditional noise prediction model at inference time:
\begin{equation}
 \hat{\epsilon}(x_t, t \mid c)=\epsilon_\theta(x_t, t)-s \nabla \log p_\theta(c \mid x_t, t), 
\end{equation}
 
 where $s>0$ is a ``guidance'' scale. The model can be extended to accept any guidance function, such as CLIP function \cite{radford2021learning} for text-to-image generation \cite{Adam2021disco}. This approach only modifies the sampling equation at inference time and thus can be applied to a trained diffusion model without retraining.

\textbf{Classifier-free guidance}, proposed by Ho et al. \cite{ho2022classifier}, trains a conditional noise model $\epsilon_\theta(x_t, t \mid c)$ to generate data samples with the label $c$:
\begin{equation}
\hat{\epsilon}(x_t, t \mid c)=\epsilon_\theta(x_t, t\mid \phi) + s(\epsilon_\theta(x_t, t \mid c)-\epsilon_\theta(x_t, t \mid
\phi)),
\end{equation}
where $\phi$ is a null label to allow for unconditional sampling. The sampling equations in both approaches can be expressed as a ``guided ODE'' of the form
\begin{align} \label{eq:guided_ode}
    \frac{d\bar{x}}{d \sigma} = \bar{\epsilon}(\bar{x},\sigma) + g(\bar{x},\sigma),
\end{align}
where $g(\bar{x},\sigma)$ represents a guidance function. To accelerate guided diffusion sampling, splitting numerical methods have been proposed, such as Lie-Trotter Splitting (LTSP) \cite{wizadwongsa2023accelerating}. This method divides Equation \ref{eq:guided_ode} into two subproblems, i) $\frac{dy}{d\sigma} = \bar{\epsilon}(y,\sigma)$ and ii) $\frac{dz}{d\sigma}=g(z,\sigma)$, but can only apply high-order numerical methods to the first equation while resorting to the Euler method for the second equation to avoid numerical instability.
Higher-order splitting methods, such as Strang Splitting (STSP) \cite{wizadwongsa2023accelerating}, are also able to mitigate artifacts. However, these methods require solving the second equation twice per step, which is comparable to increasing the total sampling step to avoid artifacts. Both approaches require non-negligible computation.

\subsection{Stability Region} \label{stab_region}
The stability region is a fundamental concept in numerical methods for solving ODEs. It determines the step sizes that enable numerical approximations to converge. To illustrate this concept, let us consider the Euler method, a simple, first-order method for solving ODEs, given by 
\begin{align} \label{euler}
    x_{n+1} = x_n + \delta f(x_n),
\end{align}
 where $x_n$ is the approximate solution and $\delta$ is the step size.
To analyze the stability of the Euler method, we can consider a test equation of the form $x' = \lambda x$, where $\lambda$ is a complex constant. The solution of this test equation can be expressed as 
\begin{align}
x_{n+1} = x_n + \delta \lambda x_n = (1 + \delta \lambda)x_n = (1 + \delta  \lambda)^{n+1} x_0,
\end{align}
where $x_0$ is the initial value.
For the approximate solution to converge to the true solution, it is necessary that $|1 + \delta \lambda| \leq 1$. Hence, the stability region of the Euler method is $S=\{z \in \mathbb{C}  : |1+z|\leq 1 \}$ because if $z = \delta\lambda$ lies outside of $S$, the solution $x_n$ will tend to $\pm \infty$ as $n \rightarrow \infty$.

\begin{minipage}{0.6\textwidth}
In diffusion sampling, another common numerical solver is the Adams-Bashforth (AB) methods, also referred to as Pseudo Linear Multi-Step (PLMS). AB methods encompass the Euler method as its first-order special case (AB1), and the second-order AB2 is given by:
\begin{equation}\label{2nd_Adam} 
x_{n+1} = x_n + \delta \left(\frac{3}{2} f(x_n) - \frac{1}{2} f(x_{n-1})\right).  
\end{equation} 
The stability regions of AB methods of various orders are derived in Appendix \ref{apx:stab_ana}. To visualize these regions, we use the boundary locus technique \cite{lambert1991numerical}, which determines the boundaries of the stability regions as depicted in Figure \ref{fig:stability_ab}.
As the order of the method increases, the stability region decreases in size, and its boundary becomes more restrictive. 
\end{minipage}
\hspace{0.01\textwidth}
\begin{minipage}{0.38\textwidth}
\centering
\includegraphics[width=0.99\textwidth]{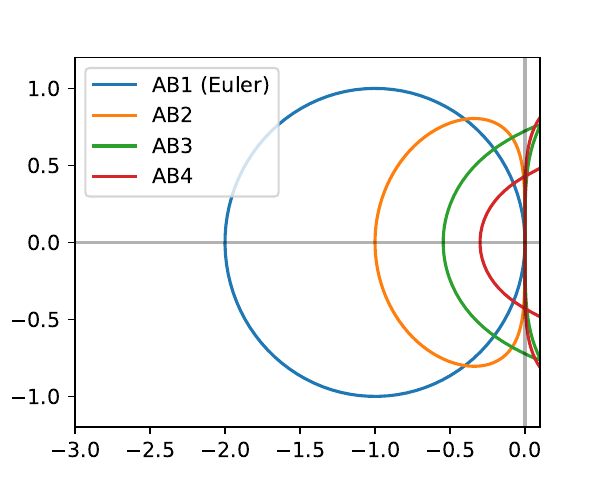}
\captionsetup{type=figure}
\caption{Boundaries of stability regions of the first 4 Adams-Bashforth methods.}
\label{fig:stability_ab}
\end{minipage}

\section{Understanding Artifacts in Diffusion Sampling} \label{sec3}
One unique issue in diffusion sampling is the occurrence of ``divergence’’ artifacts, which are characterized by regions with unrealistic, oversaturated pixels in the output. This problem typically arises due to several factors, including the use of high-order numerical solvers, too few sampling steps, or a high guidance scale. (See discussion in Appendix \ref{apx:artifact_factors}). The current solution is to simply avoid these factors, albeit at the cost of slower sampling speed or less effective guidance. 
This section investigates the source of these artifacts, then we propose solutions that do not sacrifice sampling speed in the next section.


\subsection{Analyzing Diffusion Artifacts}  \label{diffusion_artifacts} 

We analyze the areas where divergence artifacts occur during sampling by examining the magnitudes of the latent variables in those areas. Specifically, we use Stable Diffusion \cite{rombach2022high}, which operates and performs diffusion sampling on a latent space of dimension $64\times64\times4$, to generate images with and without artifacts by varying the number of steps. Then, we visualize each latent variable $z \in \mathbb{R}^4$ in the $64\times64$ spatial grid by subtracting the channel-wise mean and dividing by the channel-wise standard deviation, computed from the COCO dataset \cite{lin2014microsoft}. Figure \ref{fig:artifact_evidence} shows the magnitudes of the normalized latent variables after max pooling for visualization purposes.

We found that artifacts mainly appear in areas where the latent magnitudes are higher than usual. Note that images without artifacts can also have high latent magnitudes in some regions, although this is very rare. Conversely, when artifacts appear, those regions almost always have high magnitudes. In pixel-based diffusion models, the artifacts manifest directly as pixel values near 1 or 0 due to clipping, which can be observed in Figure \ref{fig:img_adm} in Appendix \ref{apx:adm}.

 
 \setlength\tabcolsep{3pt}
\begin{figure}
    \centering
    \begin{tabularx}{\textwidth}{lX}

        \shortstack[l]{\tiny (a) PLMS4 \\ \tiny 250 steps} & \noindent\parbox[c]{\hsize}{
            \includegraphics[width=0.093\columnwidth]{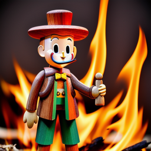}
            \includegraphics[width=0.093\columnwidth]{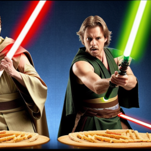}
            \includegraphics[width=0.093\columnwidth]{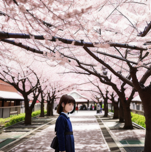}
            \includegraphics[width=0.093\columnwidth]{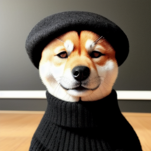}
            \includegraphics[width=0.093\columnwidth]{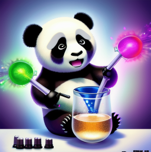}
            \includegraphics[width=0.093\columnwidth]{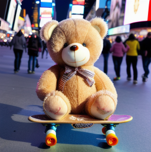}
            \includegraphics[width=0.093\columnwidth]{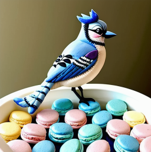}
            \includegraphics[width=0.093\columnwidth]{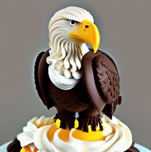}
            \includegraphics[width=0.093\columnwidth]{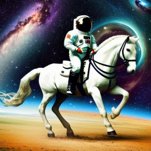}
        } \\ 

        & \noindent\parbox[c]{\hsize}{
            \includegraphics[width=0.093\columnwidth]{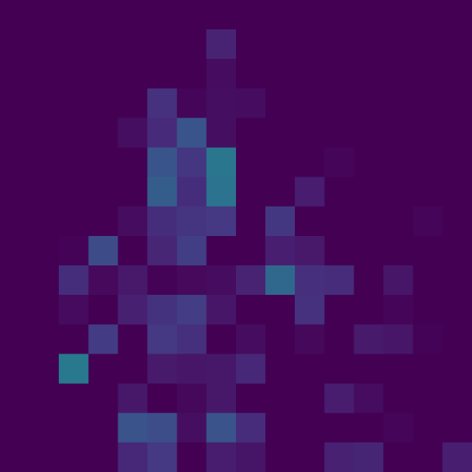}
            \includegraphics[width=0.093\columnwidth]{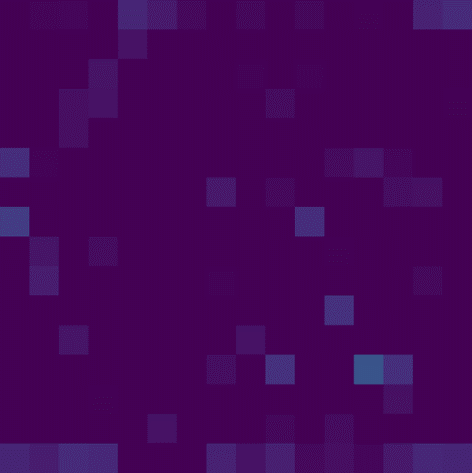}
            \includegraphics[width=0.093\columnwidth]{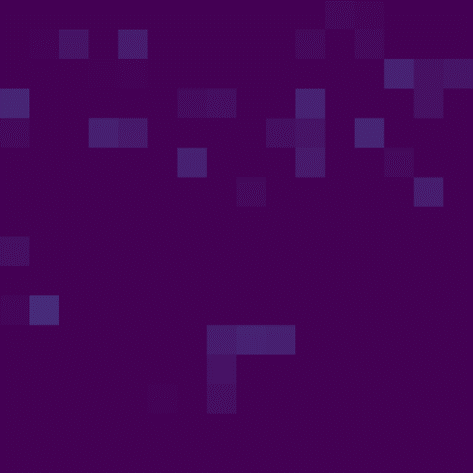}
            \includegraphics[width=0.093\columnwidth]{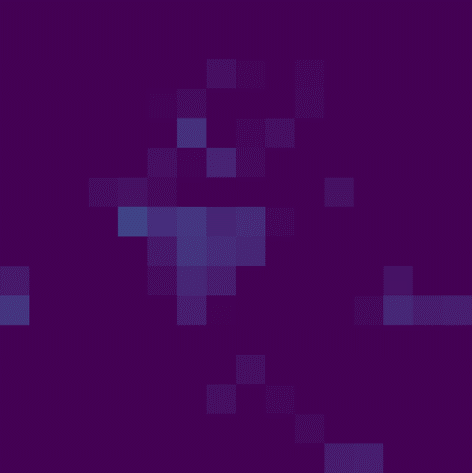}
            \includegraphics[width=0.093\columnwidth]{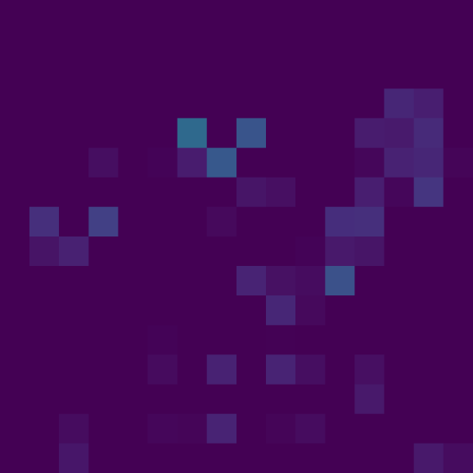}
            \includegraphics[width=0.093\columnwidth]{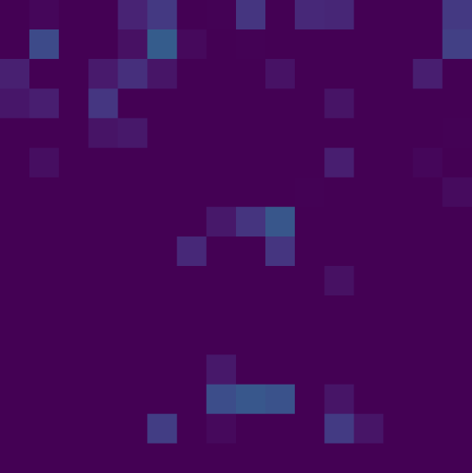}
            \includegraphics[width=0.093\columnwidth]{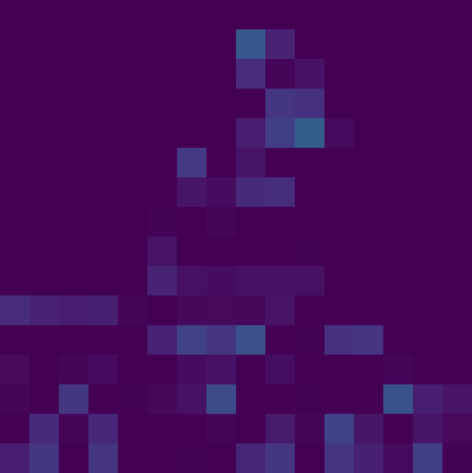}
            \includegraphics[width=0.093\columnwidth]{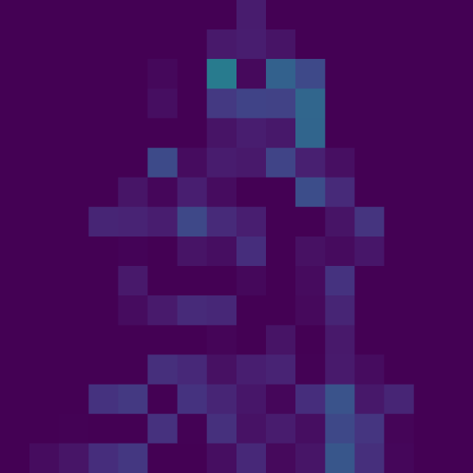}
            \includegraphics[width=0.093\columnwidth]{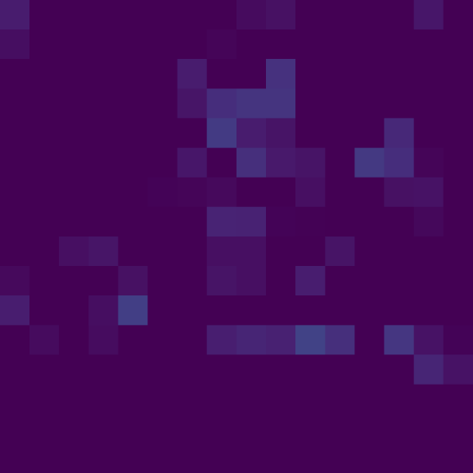}
        } \\[18pt] 

        \shortstack[l]{\tiny (a) PLMS4 \\ \tiny 15 steps} & \noindent\parbox[c]{\hsize}{
            \includegraphics[width=0.093\columnwidth]{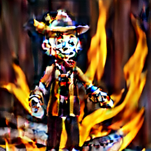}
            \includegraphics[width=0.093\columnwidth]{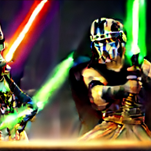}
            \includegraphics[width=0.093\columnwidth]{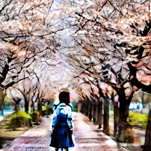}
            \includegraphics[width=0.093\columnwidth]{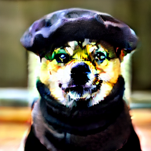}
            \includegraphics[width=0.093\columnwidth]{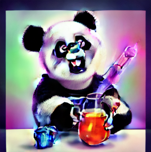}
            \includegraphics[width=0.093\columnwidth]{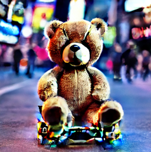}
            \includegraphics[width=0.093\columnwidth]{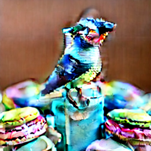}
            \includegraphics[width=0.093\columnwidth]{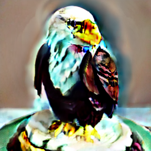}
            \includegraphics[width=0.093\columnwidth]{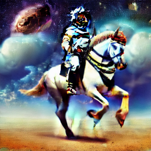}
        } \\ 

        & \noindent\parbox[c]{\hsize}{
            \includegraphics[width=0.093\columnwidth]{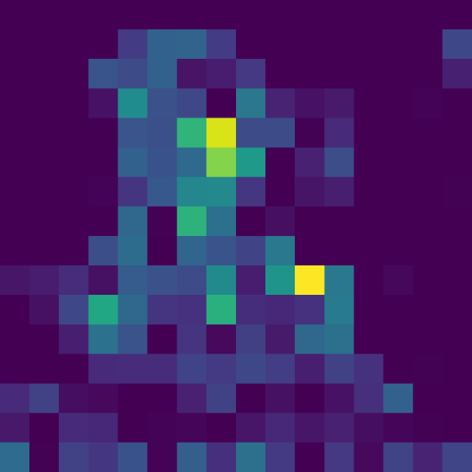}
            \includegraphics[width=0.093\columnwidth]{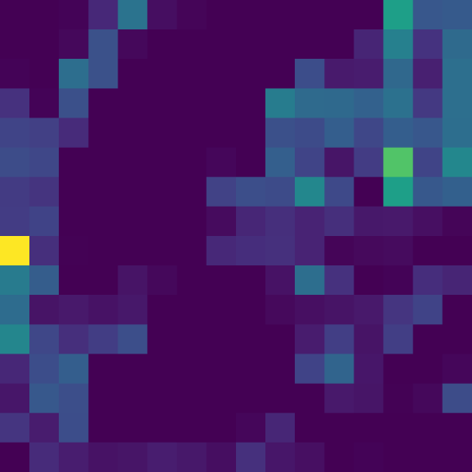}
            \includegraphics[width=0.093\columnwidth]{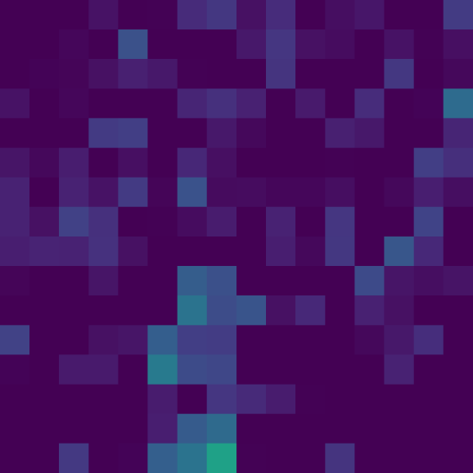}
            \includegraphics[width=0.093\columnwidth]{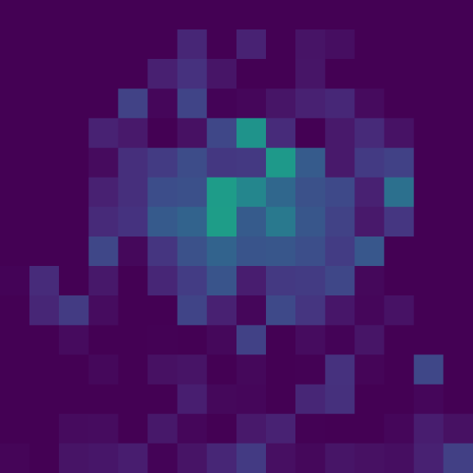}
            \includegraphics[width=0.093\columnwidth]{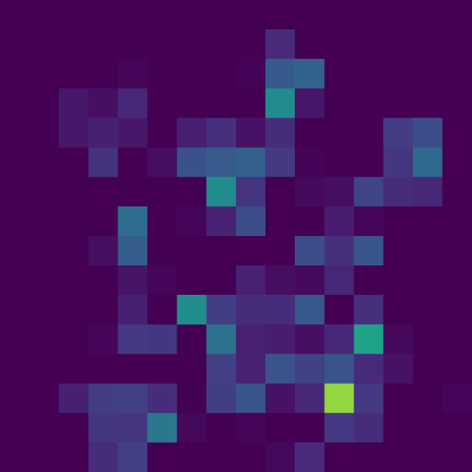}
            \includegraphics[width=0.093\columnwidth]{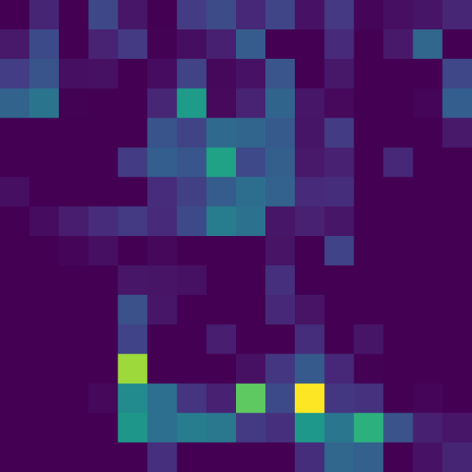}
            \includegraphics[width=0.093\columnwidth]{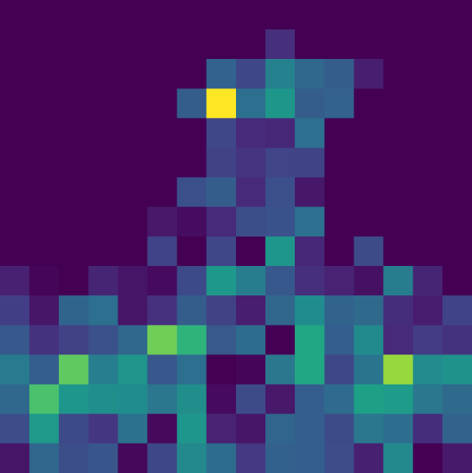}
            \includegraphics[width=0.093\columnwidth]{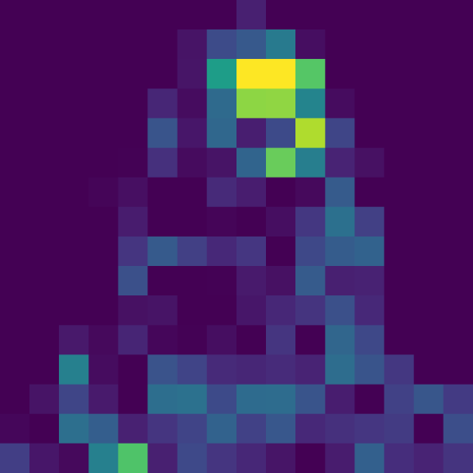}
            \includegraphics[width=0.093\columnwidth]{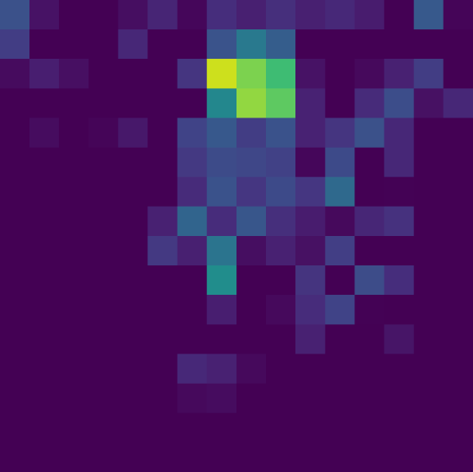}
        }

    \end{tabularx}
    \caption{Comparison of generated images and latent variable magnitudes with and without artifacts, obtained using low and high sampling steps. Latent magnitude maps are max-pooled to 16x16, with brighter colors indicating higher values. These results suggest a relationship between artifacts and large latent magnitudes.
    }
    \label{fig:artifact_evidence}
\end{figure}
 
\subsection{Connection Between ODE Solver and Artifacts} \label{artifact_sampling}
We hypothesize that numerical instability during sampling is the cause of these visual artifacts. To see this mathematically, we analyze the ODE for diffusion sampling in Equation \ref{eq:ode} using the problem reduction technique for stiffness analysis \cite{higham1993stiffness}. Assuming that the effect of $\sigma$ on the function $\bar{\epsilon}$ is negligible, we use Taylor expansion to approximate the RHS of Equation \ref{eq:ode}, which yields
\begin{align} \label{eq:expand_ode}
\frac{d\bar{x}}{d\sigma} = \nabla\bar{\epsilon}(x^*) (\bar{x}-x^*) + \mathcal{O}(\|\bar{x}-x^*\|^2).
\end{align}
Here, $x^*$ denotes the converged solution that should not have any noise left (i.e. $\bar{\epsilon}(x^*) = 0$), and $\nabla\bar{\epsilon}(x^*)$ denotes the Jacobian matrix at $x^*$. As $\bar{x}$ converges to $x^*$, the term $\mathcal{O}(\|\bar{x}-x^*\|^2)$ becomes negligibly small, so we may drop it from the equation.

Let $\lambda$ be an eigenvalue of $\nabla\bar{\epsilon}(x^*)^T$ and $v$ be the corresponding normalized eigenvector such that $\nabla\bar{\epsilon}(x^*)^Tv=\lambda v$. We define $u = v^T(\bar{x} - x^*)$ and obtain $u' = \lambda u$ as our test equation. According to Section\ref{stab_region}, if $\delta \lambda$ falls outside the stability region of a numerical method, the numerical solution to $u$ may diverge, resulting in diffusion sampling results with larger magnitudes that later manifest as divergence artifacts. Therefore, when the stability region is too small, divergence artifacts are more likely to occur. Although some numerical methods have infinite stability regions, those used in diffusion sampling have only finite stability regions, which implies that the solution will always diverge if the step size $\delta$ is sufficiently high. More details about the derivation can be found in Appendix \ref{apx:test_eq} and a 2D toy example illustrating this effect is provided in Appendix \ref{apx:toy}.

One possible solution to mitigate artifacts is to reduce the step size $\delta$, which shifts $\delta \lambda$ closer to the origin of the complex plane. However, this approach increases the number of steps, making the process slower. Instead, we will modify the numerical methods to enlarge their stability regions.

\section{Methodology} \label{sec4}

This section describes two techniques for improving stability region and reducing divergence artifacts. 
Specifically, we first show how to apply Polyak's Heavy Ball Momentum (HB) to diffusion sampling, and secondly, how to generalize HB to higher orders.
Our techniques are designed to be simple to implement and do not require additional training.

\subsection{Polyak's Heavy Ball Momentum for Diffusion Sampling}
Recall that Polyak's Heavy Ball Momentum \cite{polyak1987introduction} is an optimization algorithm that enhances gradient descent ($x_{n+1} = x_n - \beta_n \nabla f(x_n)$). The method takes inspiration from the physical analogy of a heavy ball moving through a field of potential with damping friction. The update rule for Polyak's HB optimization algorithm is given by:
\begin{align}
x_{n+1} = x_n + \alpha_n(x_n-x_{n-1}) - \beta_n \nabla f(x_n),
\end{align}
where $\alpha_n$ and $\beta_n$ are parameters. We can apply HB to the Euler method \ref{euler}, in which case we typically set $\alpha_n = (1-\beta_n)$, to obtain
\begin{align} \label{eqn:euler-with-hb}
x_{n+1} = x_n + (1-\beta_n)(x_n-x_{n-1}) + \delta \beta_n f(x_n),
\end{align}
and we may show that the numerical method above has the same order of convergence as the original Euler method. For simplicity, we assume that $\beta_n = \beta \in (0,1]$, which is a constant known as the {\it damping coefficient.} Then, we can reformulate Equation \ref{eqn:euler-with-hb} as:
\begin{align} \label{eq:1st_momentum}
v_{n+1} = (1-\beta)v_n + \beta f(x_n), \qquad x_{n+1} = x_n + \delta v_{n+1},
\end{align}
Here, we may interpret $x_n$ as the heavy ball's position, and $v_{n+1}$---the exponential moving average of $f(x_n)$---as its velocity. We can see that position is updated with ``displacement = time $\times$ velocity,'' much like in physics.

Consider a high-order method of the form  $x_{n+1} = x_n + \delta \sum^k_{i=0} b_i f(x_{n-i})$. We can apply HB to it as follows:
\begin{align}
v_{n+1} = (1-\beta) v_n + \beta \sum^k_{i=0} b_i f(x_{n-i}), \qquad x_{n+1} = x_n + \delta v_{n+1}.
\end{align}
The resulting numerical method has a larger stability region, as can be seen in 
Figures \ref{fig:2nd_polyak} to \ref{fig:4th_polyak}, in which we show stability boundaries of AB methods after HB is applied to them with varying $\beta$s. (We use HB 0.4 to denote $\beta=0.4$). However, Theorem \ref{thm:hb} in Appendix \ref{apx:conv} shows that as soon as $\beta$ deviates from 1, the theoretical order of convergence drops to 1, leading to a significant decrease in image quality, as illustrated in Figure \ref{fig:motivate_GHVB}. 
In the next subsection, we propose an alternative approach that increases the stability region while maintaining high order of convergence.

\begin{figure}[ht]
  \begin{subfigure}{0.26\textwidth}
    \includegraphics[width=\textwidth]{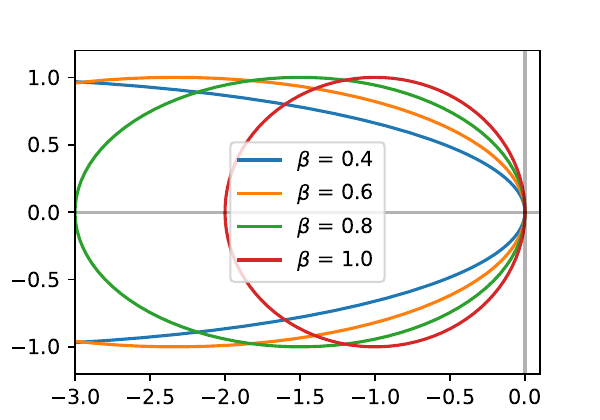}
    \caption{\small 1\ts{st} order \newline \tiny (PLMS1 w/ HB 0.4 - 1)}
    \label{fig:1st_polyak}
  \end{subfigure}
  \hspace{-0.03\textwidth}
  \begin{subfigure}{0.26\textwidth}
    \includegraphics[width=\textwidth]{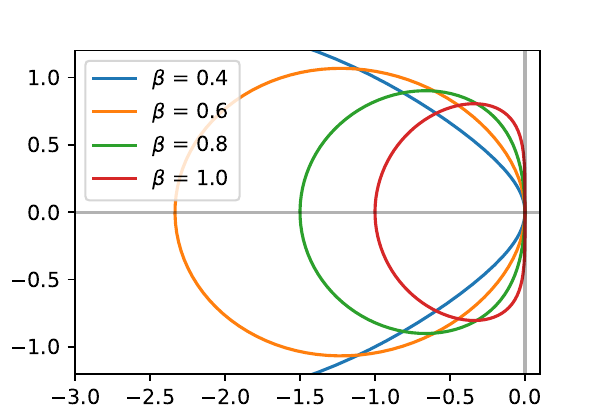}
    \caption{\small 2\ts{nd} order \newline \tiny (PLMS2 w/ HB 0.4 - 1) }
    \label{fig:2nd_polyak}
  \end{subfigure}
  \hspace{-0.03\textwidth}
  \begin{subfigure}{0.26\textwidth}
    \includegraphics[width=\textwidth]{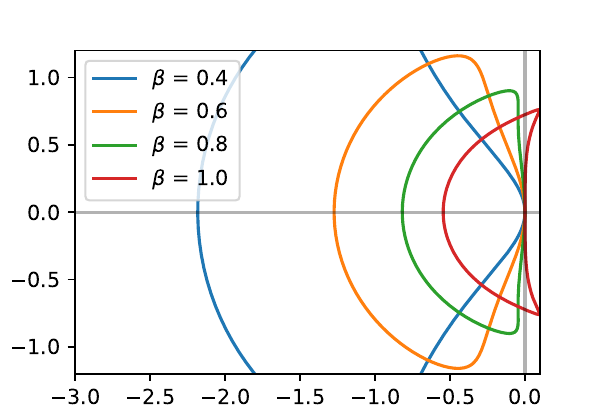}
    \caption{\small 3\ts{rd} order \newline \tiny (PLMS3 w/ HB 0.4 - 1) }
    \label{fig:3rd_polyak}
  \end{subfigure}
  \hspace{-0.03\textwidth}
  \begin{subfigure}{0.26\textwidth}
    \includegraphics[width=\textwidth]{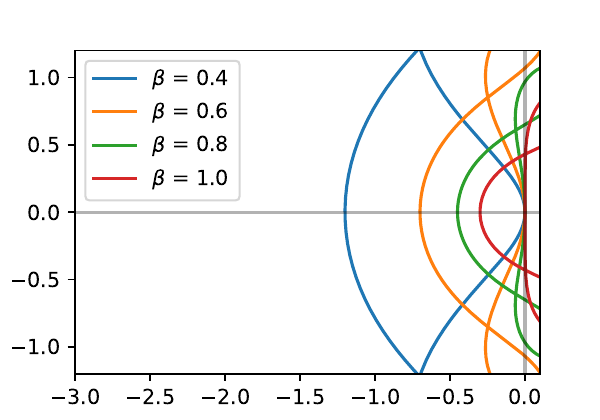}
    \caption{\small 4\ts{th} order \newline \tiny (PLMS4 w/ HB 0.4 - 1) }
    \label{fig:4th_polyak}
  \end{subfigure}
  \caption{Boundaries of stability regions of 1\ts{st}- to 4\ts{th}-order AB methods with HB applied to them with different values of the damping coefficient $\beta$.
  }
  \label{fig:polyak_comparison}
\end{figure}

\tabulinesep=1pt
\begin{figure}[ht]
    \centering
    \begin{tabu} to \textwidth {
        @{}
        l@{\hspace{6pt}}
        c@{\hspace{2pt}}
        c@{\hspace{2pt}}
        c@{\hspace{2pt}}
        c@{}
        c@{\hspace{10pt}} | @{\hspace{5pt}}
        c
        @{}
    }
        & \multicolumn{1}{c}{\shortstack{\scriptsize $\beta = 0.2$}}
        & \multicolumn{1}{c}{\shortstack{\scriptsize $\beta = 0.4$}}
        & \multicolumn{1}{c}{\shortstack{\scriptsize $\beta = 0.6$}}
        & \multicolumn{1}{c}{\shortstack{\scriptsize $\beta = 0.8$}} &
        & \multicolumn{1}{c}{\shortstack{\tiny PLMS4 $(\beta = 1.0)$}} \\

        \shortstack[l]{\scriptsize (a) PLMS4 with HB} &
        \noindent\parbox[c]{0.14\columnwidth}{\includegraphics[width=0.14\columnwidth]{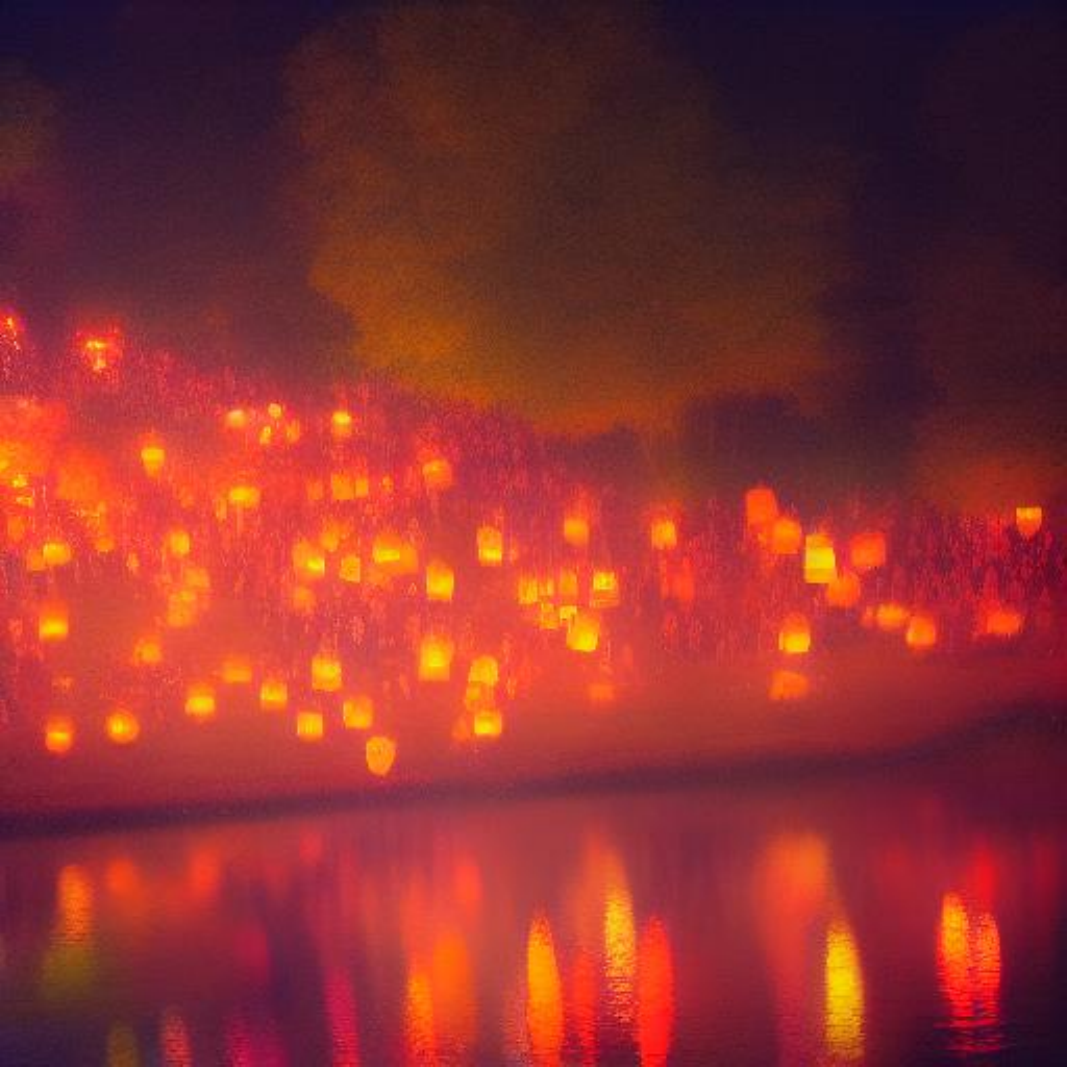}} & 
        \noindent\parbox[c]{0.14\columnwidth}{\includegraphics[width=0.14\columnwidth]{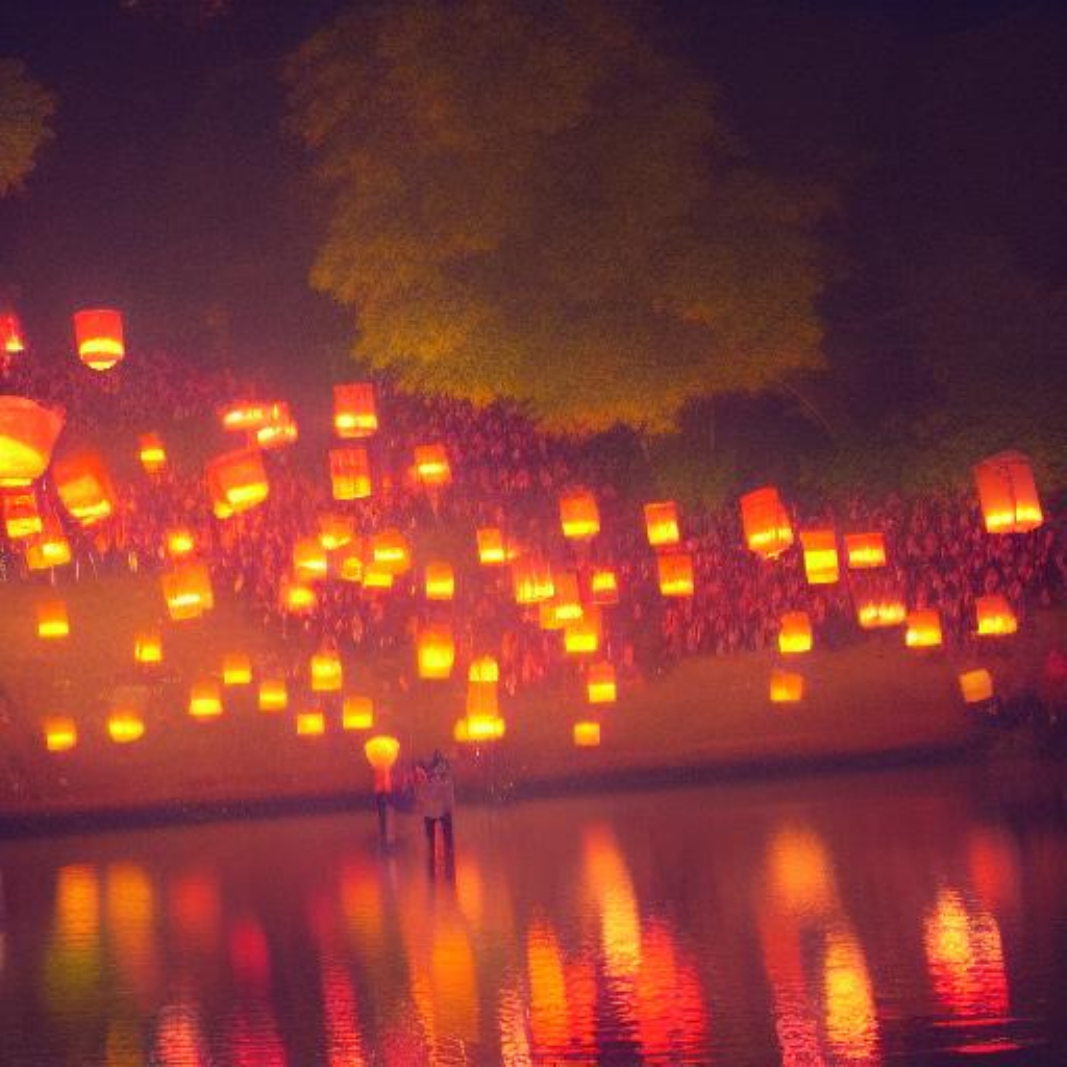}} & 
        \noindent\parbox[c]{0.14\columnwidth}{\includegraphics[width=0.14\columnwidth]{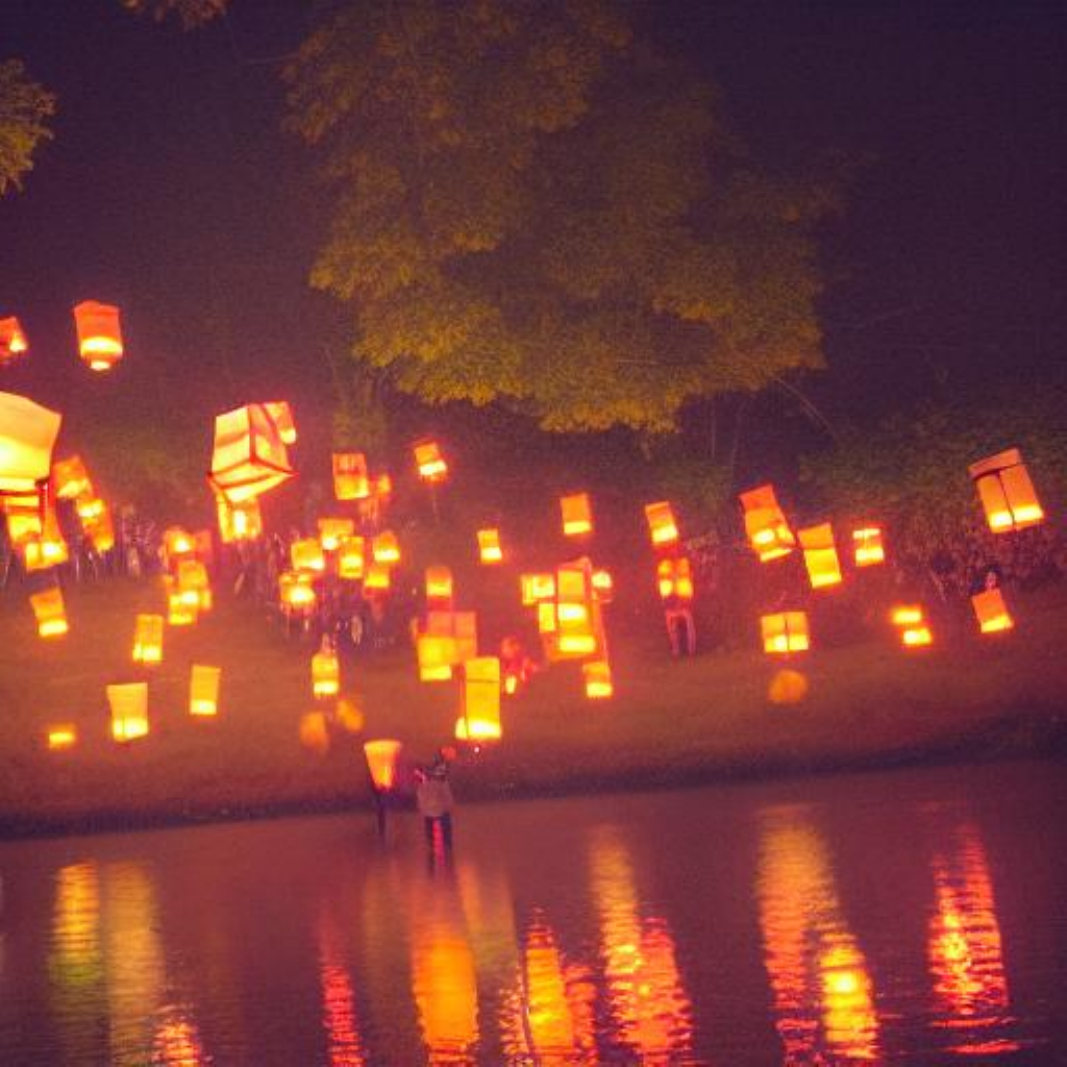}} & 
        \noindent\parbox[c]{0.14\columnwidth}{\includegraphics[width=0.14\columnwidth]{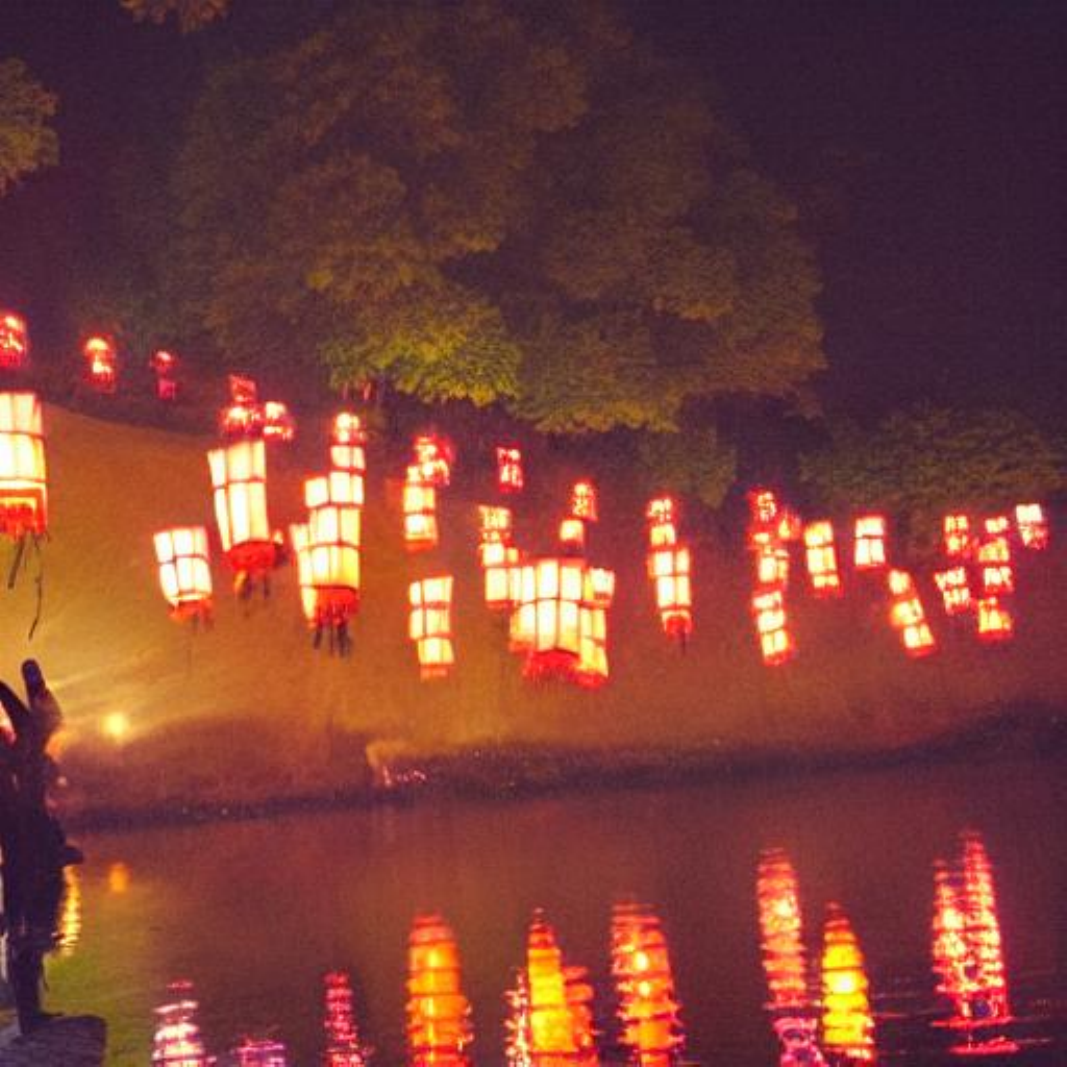}} & &
        \noindent\parbox[c]{0.14\columnwidth}{\includegraphics[width=0.14\columnwidth]{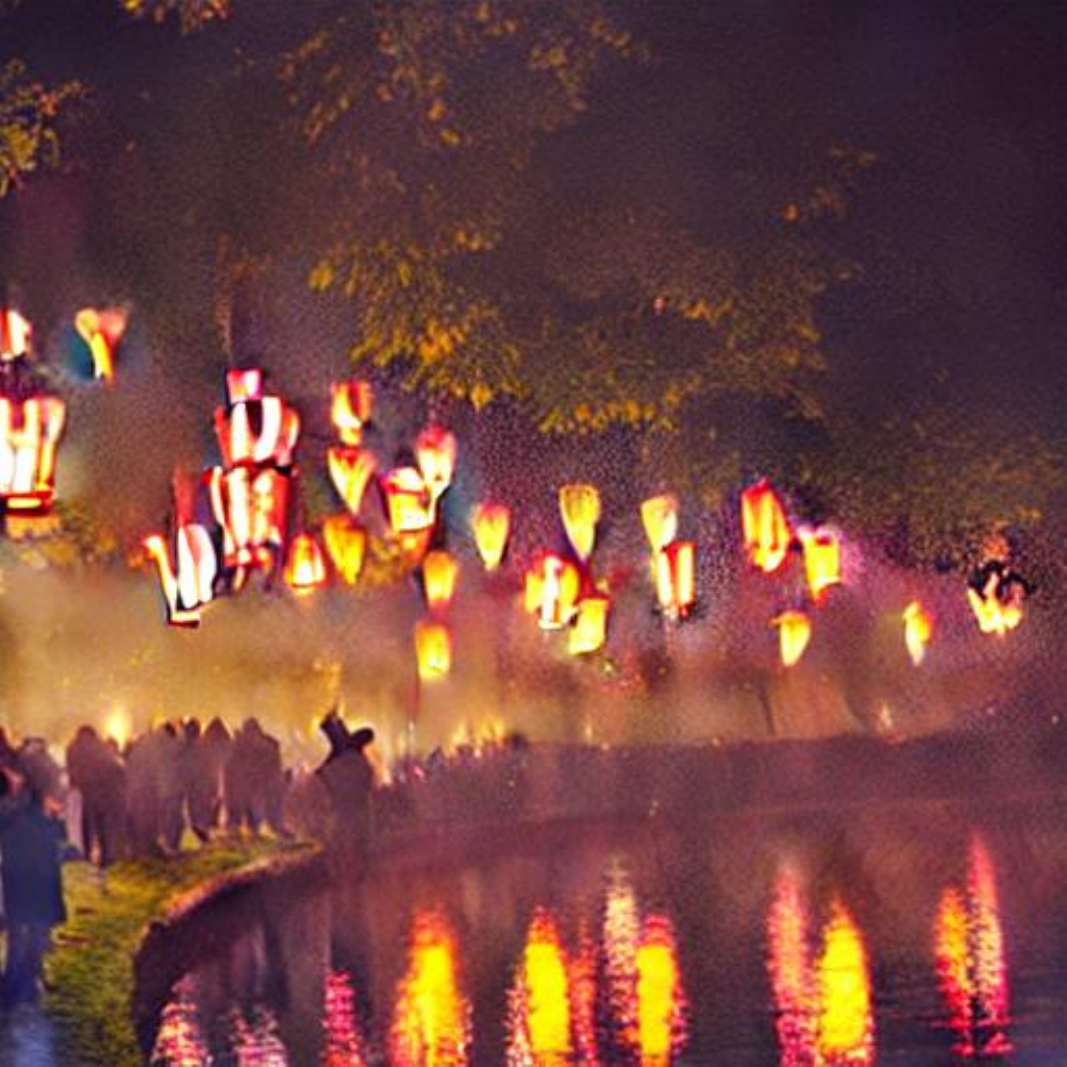}} \\

        & \multicolumn{1}{c}{\shortstack{\scriptsize $\beta = 0.2$}}
        & \multicolumn{1}{c}{\shortstack{\scriptsize $\beta = 0.4$}}
        & \multicolumn{1}{c}{\shortstack{\scriptsize $\beta = 0.6$}}
        & \multicolumn{1}{c}{\shortstack{\scriptsize $\beta = 0.8$}} &
        & \multicolumn{1}{c}{\shortstack{\tiny 1,000-steps DDIM \cite{song2020denoising}, }} \\

        \shortstack[l]{\scriptsize (b) 4\ts{th} order GHVB } &
        \noindent\parbox[c]{0.14\columnwidth}{\includegraphics[width=0.14\columnwidth]{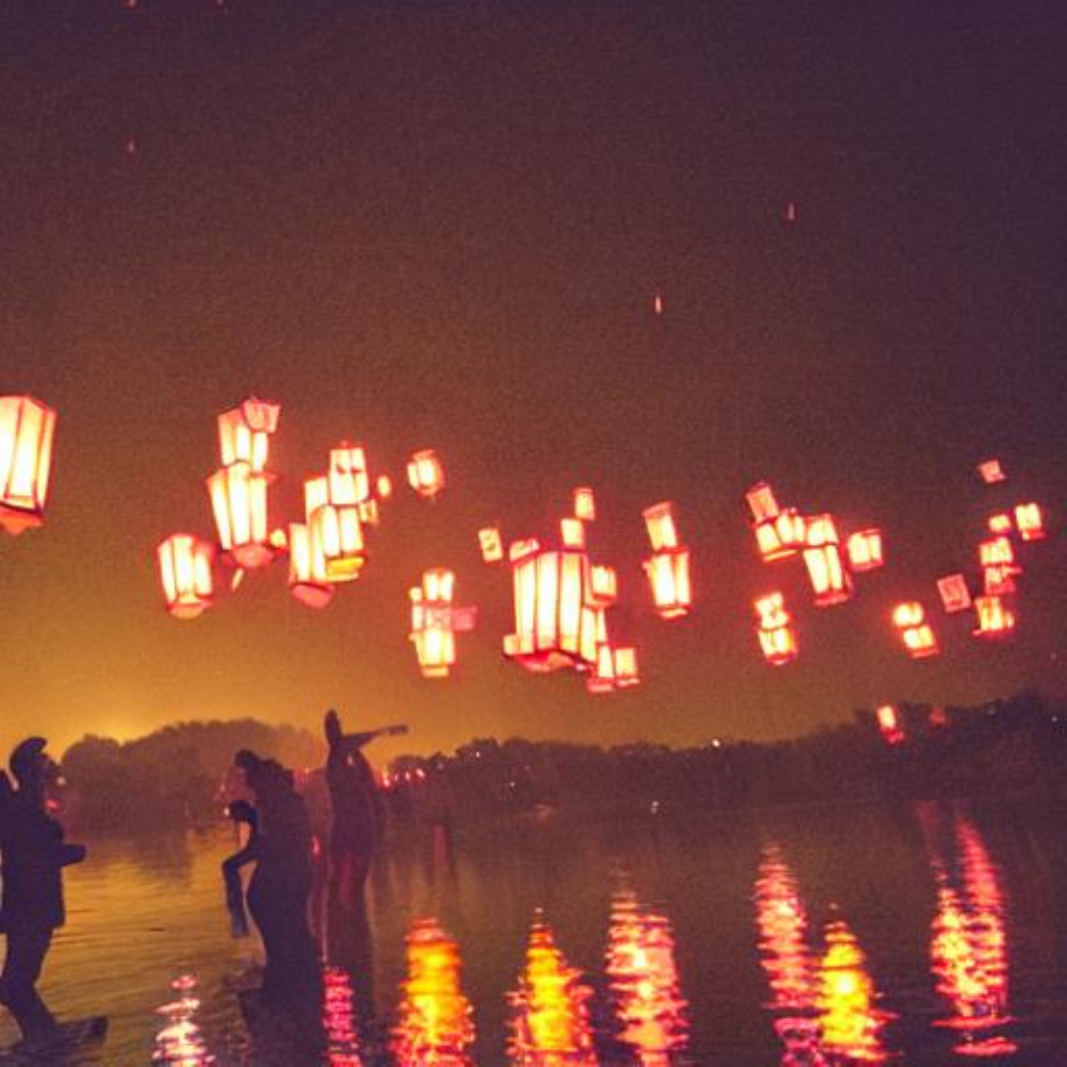}} & 
        \noindent\parbox[c]{0.14\columnwidth}{\includegraphics[width=0.14\columnwidth]{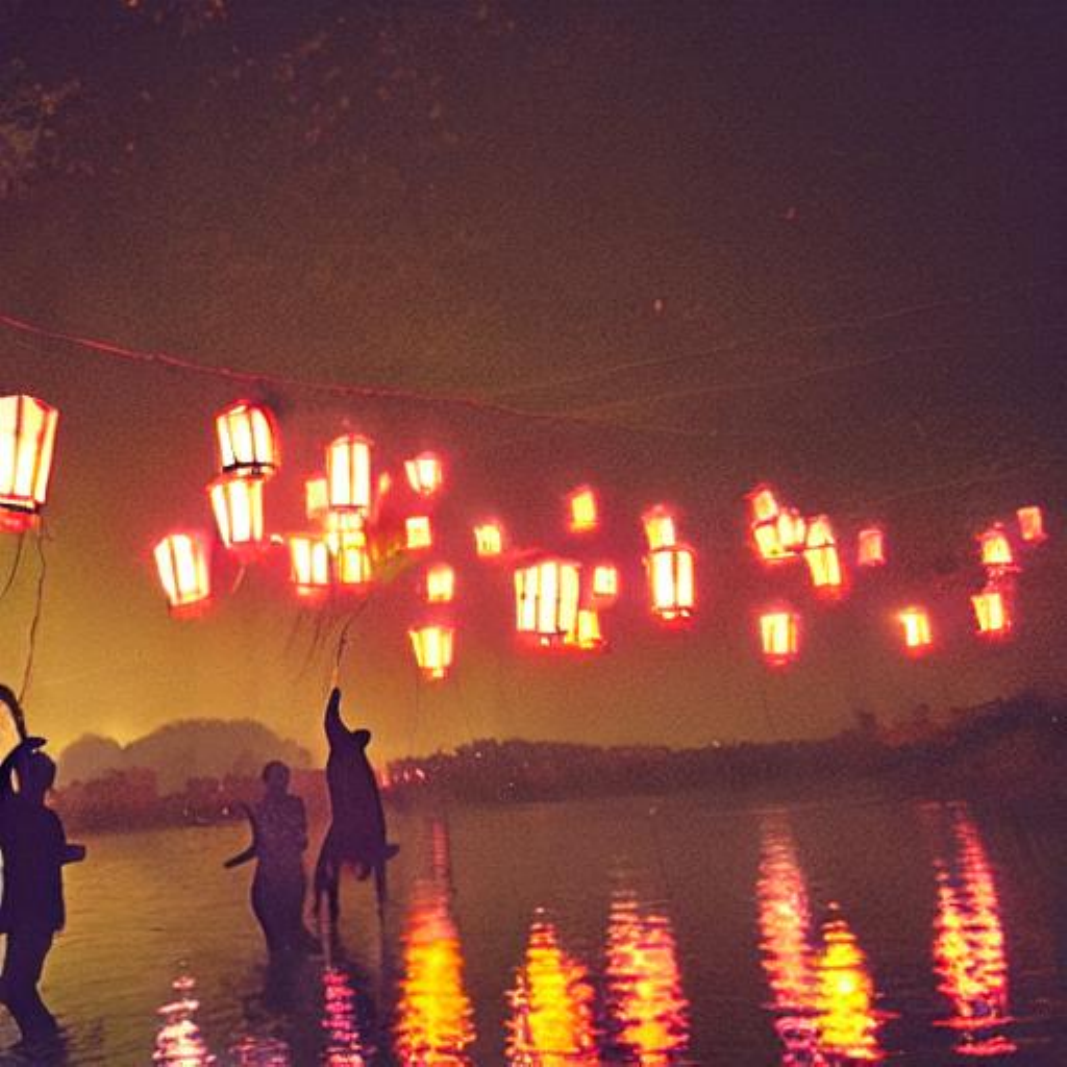}} & 
        \noindent\parbox[c]{0.14\columnwidth}{\includegraphics[width=0.14\columnwidth]{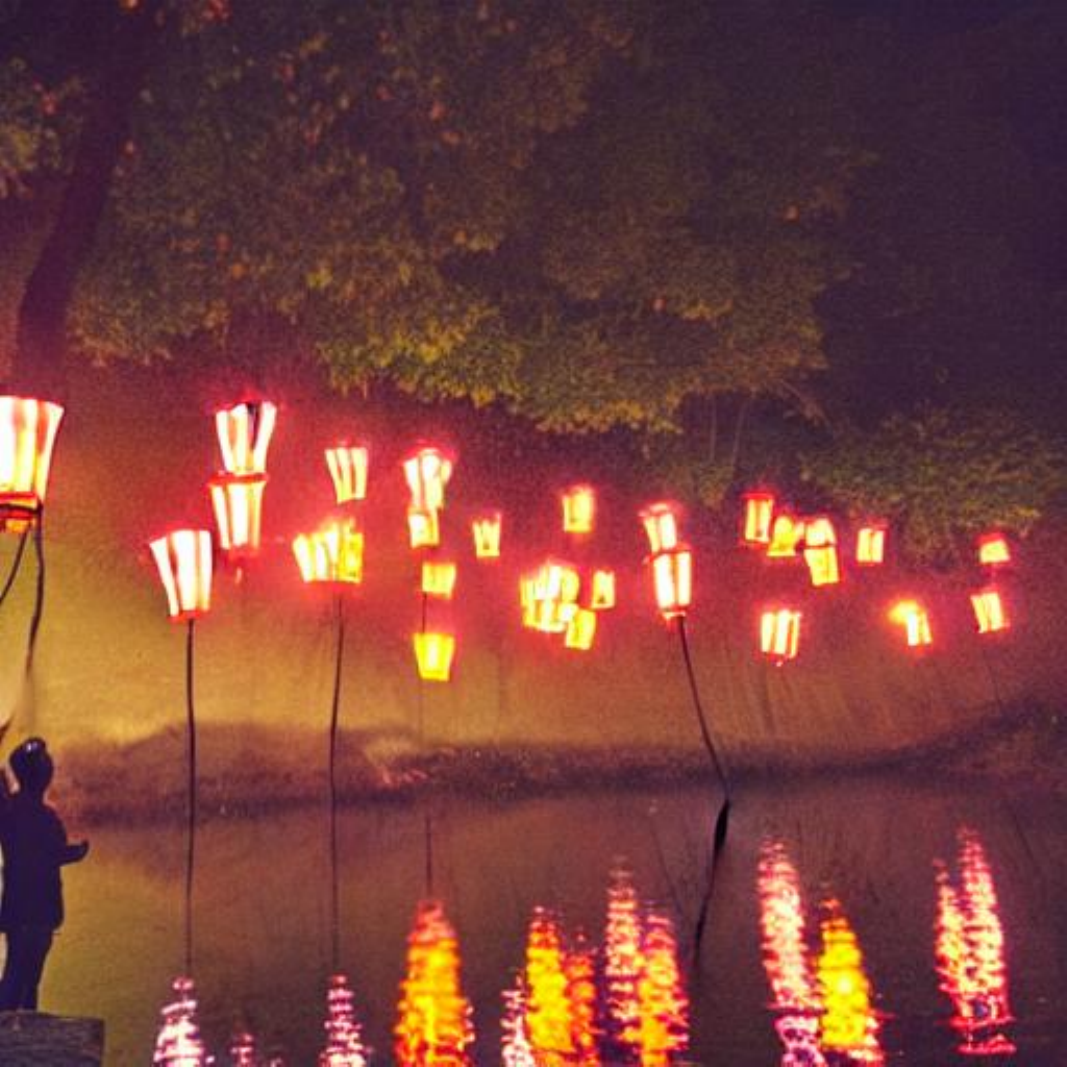}} & 
        \noindent\parbox[c]{0.14\columnwidth}{\includegraphics[width=0.14\columnwidth]{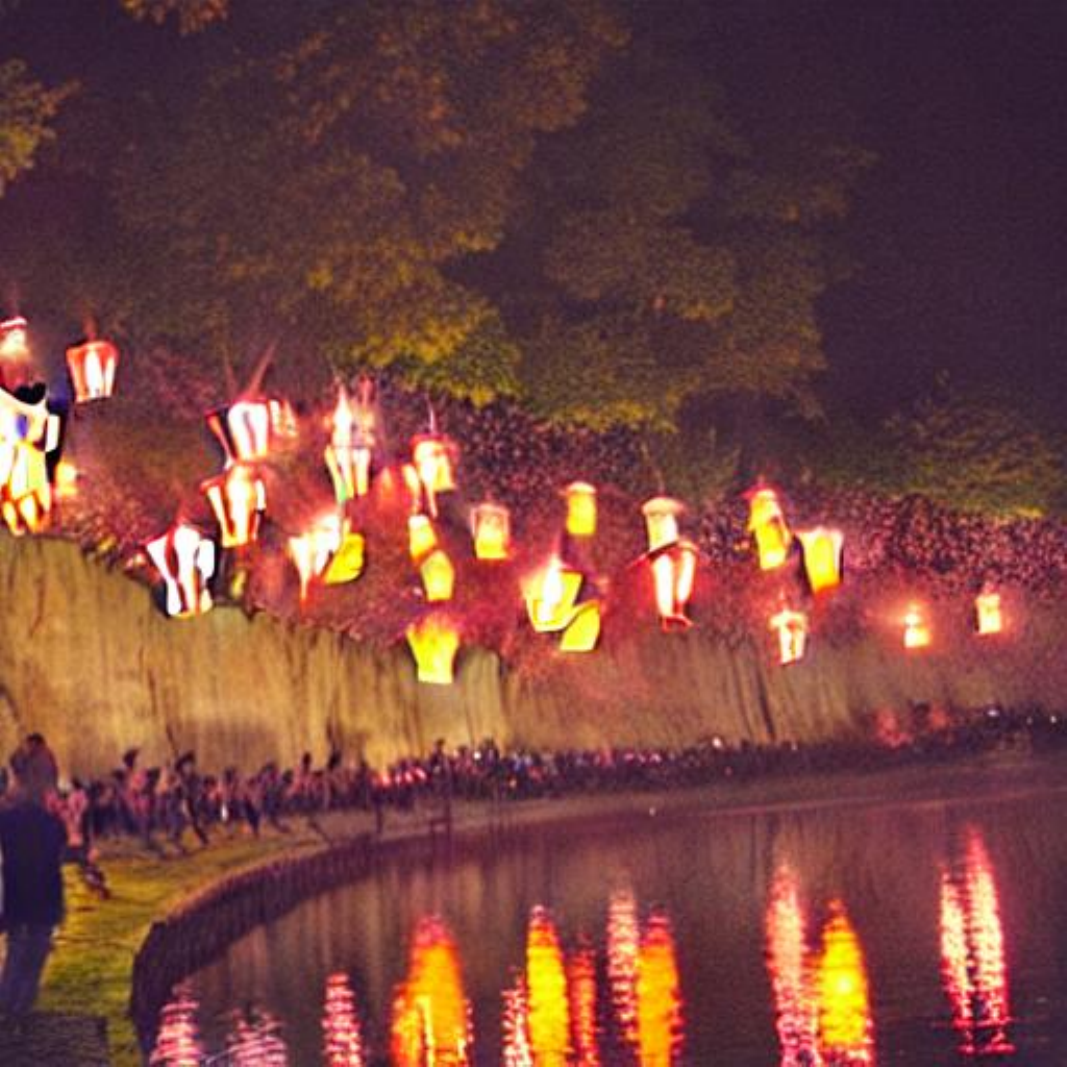}} & &
        \noindent\parbox[c]{0.14\columnwidth}{\includegraphics[width=0.14\columnwidth]{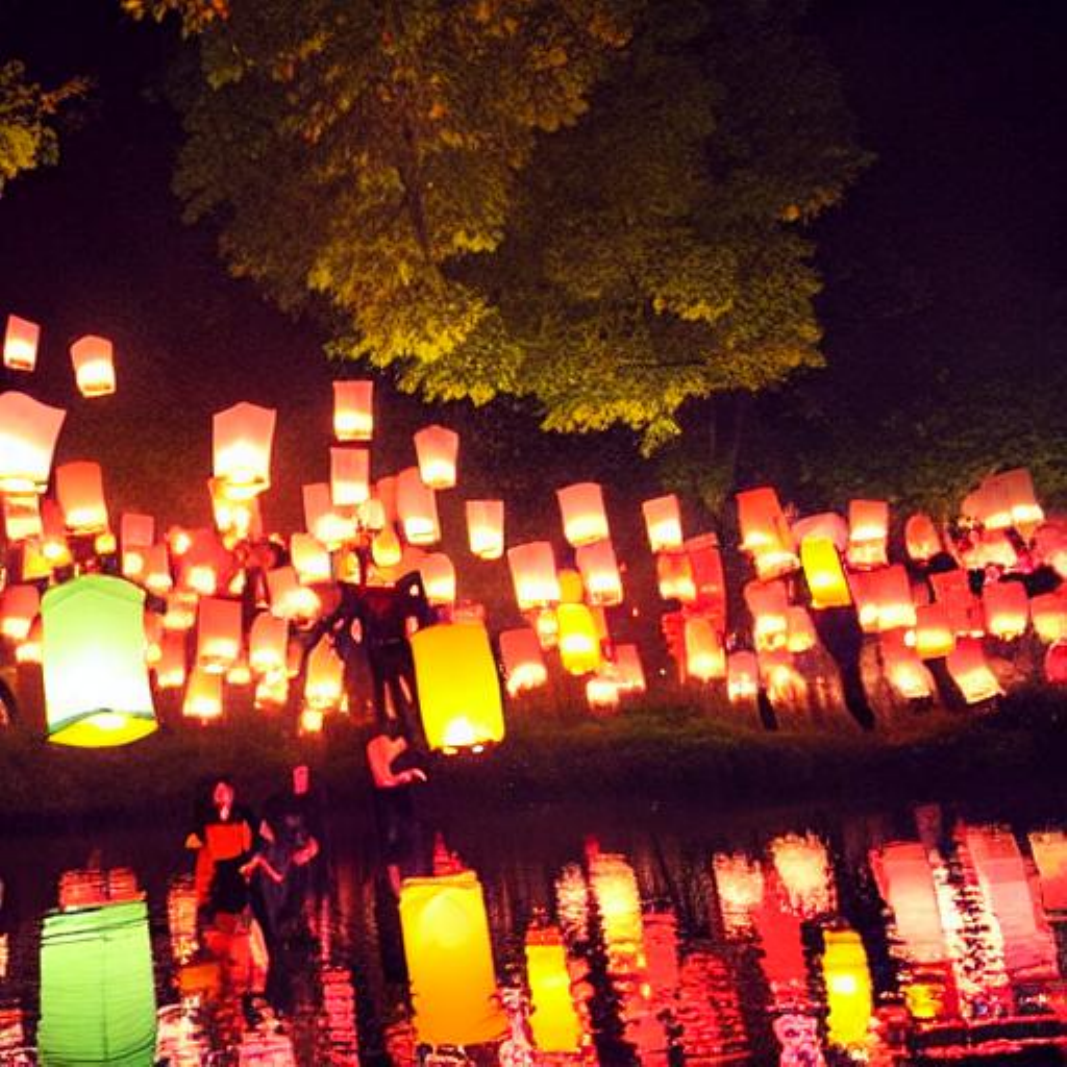}} \\

    \end{tabu}
    \caption{Comparison between the two techniques we propose: (a) HB and (b) GHVB, applied to PLMS4 \cite{liu2022pseudo} with 15 sampling steps. Both are effective at reducing artifacts, but HB's accuracy drops faster than GHVB's as $\beta$ moves away from $1$. Positions of the lanterns in Row (a) deviate more from the ground truth (1000 steps DDIM) than those in Row (b). Moreover, the image at $\beta = 0.2$ in Row (a) becomes blurry as HB yields a numerical method with a lower order of convergence than what GHVB does. Prompt: "A beautiful illustration of people releasing lanterns near a river".}
    \label{fig:motivate_GHVB}
\end{figure}

\subsection{Generalizing Polyak's Heavy Ball to Higher Orders} \label{sec:ghvb}
In this section, we generalize Euler method with HB momentum to achieve high-order convergence in a similar way to how the Adams–Bashforth methods generalize the Euler method. We define the backward difference operator $\Delta$ as $\Delta x_n = x_n - x_{n-1}$. According to \cite{berry2004implementation}, we can express the AB formula as:

\begin{align} \label{eq:adam_bash}
\Delta x_{n+1} = \: \delta \left(1 + \frac{1}{2}\Delta + \frac{5}{12}\Delta^2 + \frac{3}{8}\Delta^3 + \frac{251}{720}\Delta ^4 + \frac{95}{288}\Delta^5 + \ldots \right) f(x_n).
\end{align}
The order convergence is determined by the number of terms on the RHS. For example, the 2\ts{nd}-order AB method can be written as $\Delta x_{n+1} = \delta \left(1 + \frac{1}{2}\Delta \right) f(x_n)$.
The update rule for $v_n$ in Equation \ref{eq:1st_momentum} can be rewritten as $(\beta + (1-\beta)\Delta)v_{n+1} = \beta f(x_n)$. Multiplying both sides of Equation \ref{eq:adam_bash} by $(\beta + (1-\beta) \Delta)$, we have:
\begin{align} \label{eq:adam_bata}
   (\beta + (1-\beta) \Delta)\Delta x_{n+1} =& \: \delta \left(\beta + \frac{2-\beta}{2}\Delta + \frac{6-\beta}{12}\Delta^2 + \frac{10-\beta}{24}\Delta^3 +  \ldots \right) f(x_n).
\end{align}
Next, we can choose the order of convergence by fixing the number of terms on the RHS. To get, say, a 2\ts{nd}-order method, we may choose:
\begin{align}\label{derive_2nd}
   (\beta + (1-\beta) \Delta)\Delta x_{n+1} =& \: \delta \left(\beta + \frac{2-\beta}{2}\Delta \right) f(x_n)= \delta \left(1 + \frac{2-\beta}{2\beta}\Delta \right) \beta f(x_n) \\
   =& \: \delta \left(1 + \frac{2-\beta}{2\beta}\Delta \right) (\beta + (1-\beta) \Delta) v_{n+1}.
\end{align}
Eliminating $(\beta - (1-\beta)\Delta)$ from both sides, we obtain the 2\ts{nd}-order generalized HB method:
\begin{align} \label{2nd_bAdam}
   v_{n+1} = (1-\beta) v_n + \beta f(x_n), \qquad x_{n+1} = x_n + \delta \left(\frac{2+\beta}{2\beta} v_{n+1} + \frac{2-\beta}{2\beta} v_{n} \right).
\end{align}
Algorithm \ref{algo:ghvb} details a complete implementation.
When $\beta = 1$, the formulation in Equation \ref{2nd_bAdam} is equivalent to the AB2 formulation in Equation \ref{2nd_Adam}. As $\beta$ approaches 0, Equation \ref{derive_2nd} converges to the 1\ts{st}-order Euler method \ref{euler}. Thus, this generalization also serves as an interpolating technique between two adjacent-order AB methods, except for the 1\ts{st}-order GHVB, which is equivalent to the Euler method with HB momentum in Equation \ref{eq:1st_momentum}.

We call this new method the Generalized Heavy Ball (GHVB) and associate with it a \emph{momentum number}, whose ceiling indicates the method's order.
For example, GHVB 1.8 refers to the 2\ts{nd}-order GHVB with $\beta = 0.8$. The main difference between HB and GHVB is that HB calculates the moving average after summing high-order coefficients, whereas GHVB calculates it before the summation. 

We analyze the stability region of GHVB using the same approach as before and visualize the region's locus curve in Figure \ref{fig:stability_regions}.
The theoretical order of accuracy of this method is given by Theorem \ref{thm:ghvb} in Appendix \ref{apx:conv}. 
We discuss alternative momentum methods, such as Nesterov’s momentum, which can offer comparable performance but are less simple in Appendix \ref{apx:moment}.

\begin{figure}[ht]
  \begin{subfigure}{0.26\textwidth}
    \includegraphics[width=\textwidth]{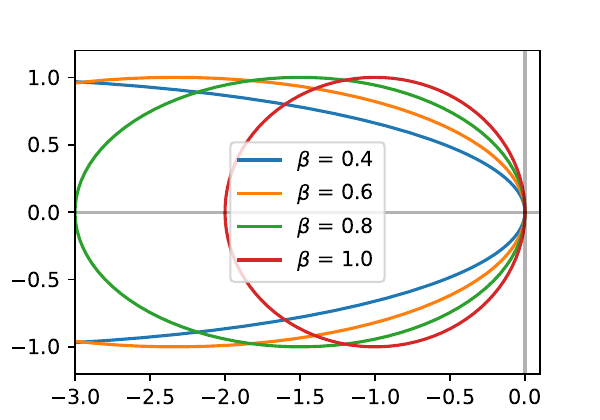}
    \caption{\small 1\ts{st} order \newline \tiny (GHVB 0.4 - 1.0)}
    \label{fig:1st_general}
  \end{subfigure}
  \hspace{-0.03\textwidth}
  \begin{subfigure}{0.26\textwidth}
    \includegraphics[width=\textwidth]{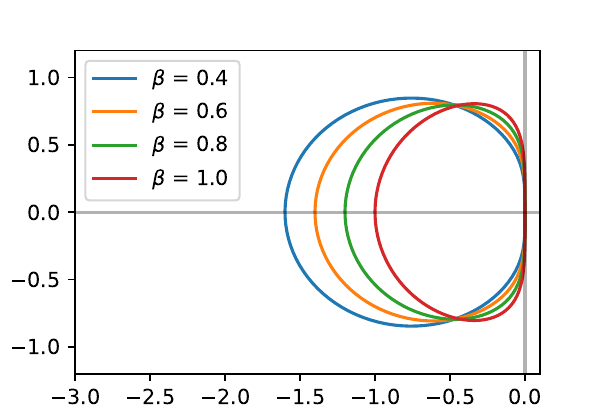}
    \caption{\small 2\ts{nd} order \newline \tiny (GHVB 1.4 - 2.0)}
    \label{fig:2nd_general}
  \end{subfigure}
  \hspace{-0.03\textwidth}
  \begin{subfigure}{0.26\textwidth}
    \includegraphics[width=\textwidth]{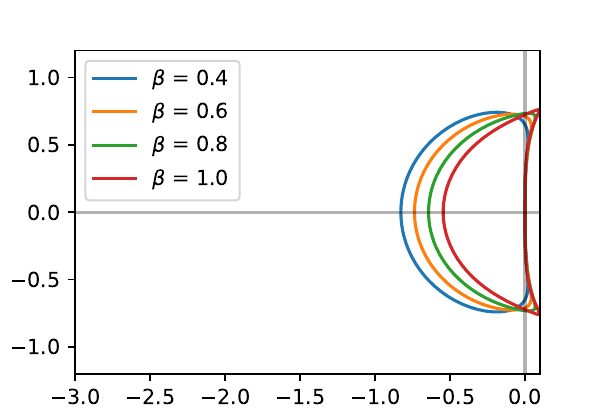}
    \caption{\small 3\ts{rd} order \newline \tiny (GHVB 2.4 - 3.0)}
    \label{fig:3rd_general}
  \end{subfigure}
  \hspace{-0.03\textwidth}
  \begin{subfigure}{0.26\textwidth}
    \includegraphics[width=\textwidth]{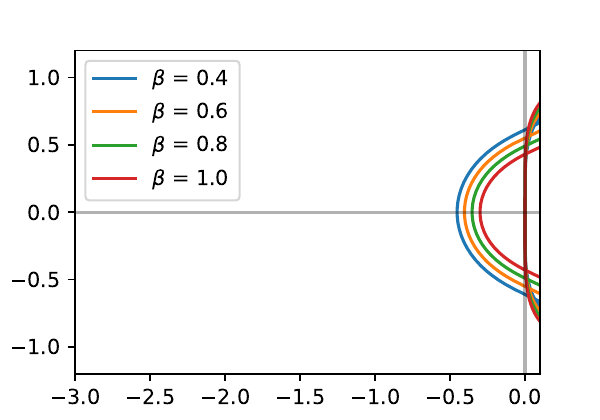}
    \caption{\small 4\ts{th} order\newline \tiny (GHVB 3.4 - 4.0)}
    \label{fig:4th_general}
  \end{subfigure}
  \caption{Boundary of stability regions for 1\ts{st}-to 4\ts{th}- order Generalized Heavy Ball Methods (GHVB).}
  \label{fig:stability_regions}
\end{figure}

\section{Experiments} \label{sec5}

We present a series of experiments to evaluate the effectiveness of our techniques. 
In Section \ref{sec:exp_sd}, we assess the reduction of divergence 
artifacts through qualitative results and quantitative measurements of the latent magnitudes in a text-to-image diffusion model. Besides reducing artifacts, another important goal is to ensure that the overall sampling quality improves and does not degenerate (e.g., becoming color blobs).
We test this with experiments on both pixel-based and latent-based diffusion models trained on ImageNet@256\cite{russakovsky2015imagenet} (Section \ref{sec:exp_adm} and \ref{sec:exp_dit}), which show that our techniques indeed significantly improve image quality, as measured by the standard Fréchet Inception Distance (FID) score.
Lastly, in Section \ref{sec:exp_ghvb}, we present an ablation study of GHVB methods with varying degrees of orders. A similar study on HB methods
can be found in Appendix \ref{apx:hb}.

\subsection{Artifacts Mitigation} \label{sec:exp_sd}
In this experiment, we apply our HB and GHVB techniques to the most popular 2\ts{nd} and 4\ts{th}-order solvers, DPM-Solver++ \cite{lu2022dpm} and PLMS4 \cite{liu2022pseudo}, using 15 sampling steps and various guidance scales on three different text-to-image diffusion models. The qualitative results in Figure \ref{fig:highlighted_image} show our techniques significantly reduce the divergence artifacts and produce realistic results (columns a, c).
More qualitative results are in Figure \ref{fig:highlighted_image2} in Appendix \ref{apx:sample}.

\begin{minipage}{0.62\textwidth}
Quantitatively measuring divergence artifacts can be challenging, as metrics like MSE or LPIPS may only capture the discrepancy between the approximated and the true solutions, which does not necessarily indicate the presence of divergence artifacts. 
In this study, we use the magnitudes of latent variables as introduced in Section \ref{diffusion_artifacts} as a proxy metric to measure artifacts. In particular, we define a magnitude score $v = \sum_{i,j} f(z'_{ij})$ that sums over the latent variables in a max-pooled latent grid, where $f(x) = x$ if $x \geq \tau$ and $0$ otherwise. 
We generate 160 samples from the same set of text prompts and seeds for each method from a fine-tuned Stable Diffusion model called Anything V4 \cite{Anything}.

The results using $\tau=3$ (magnitude considered high when above 3 std.) are shown in Figures \ref{fig:ablation_magnitude1} and \ref{fig:ablation_magnitude2}. We observe that the magnitude score increases as the number of sampling steps decreases and higher-order methods result in higher magnitude scores. Figure \ref{fig:ablation_magnitude1} shows that adding HB momentum to PLMS4 \cite{liu2022pseudo} or DPM-Solver++\cite{lu2022dpmpp} can reduce their magnitude scores, while Figure \ref{fig:ablation_magnitude2} shows that GHVB can also reduce the magnitude scores by reducing the momentum number. Next, we show that our results with less artifacts also have good image quality.
\end{minipage}
\begin{minipage}{0.37\textwidth}
\centering
    \includegraphics[width=0.99\textwidth]{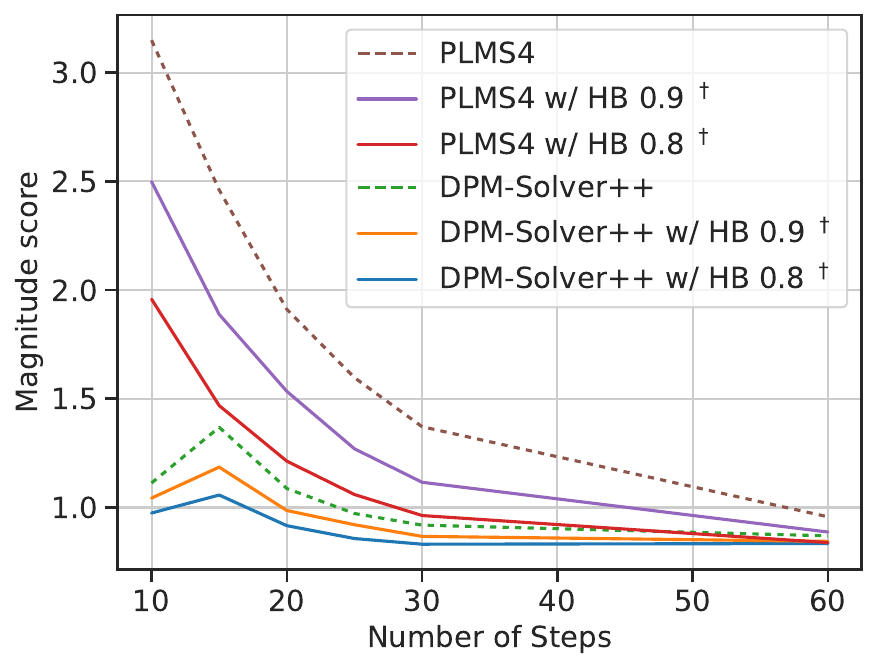}
\vspace{-0.13\textwidth}
\captionsetup{type=figure}
\caption{Average magnitude scores 
}
\label{fig:ablation_magnitude1}

    \includegraphics[width=0.99\textwidth]{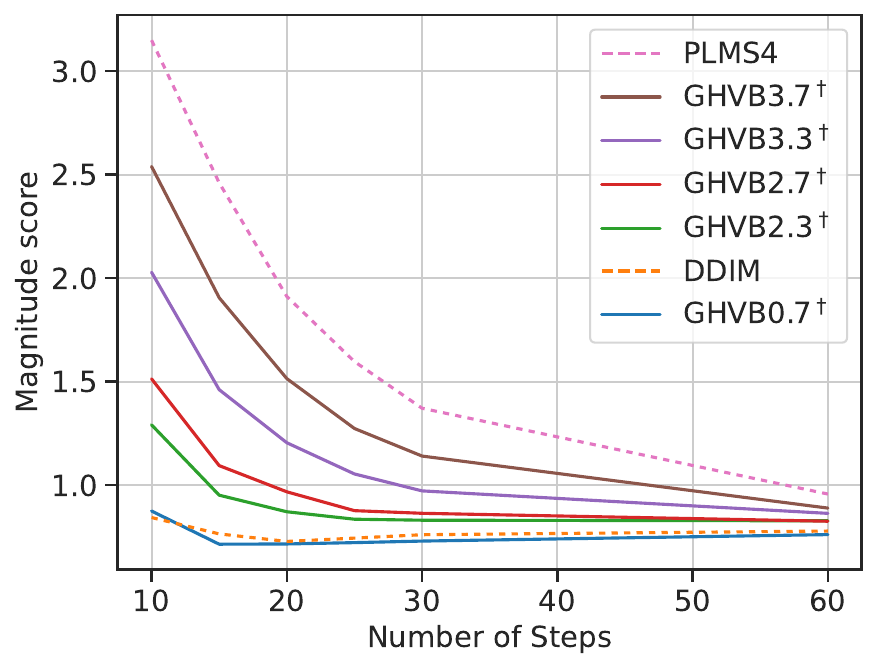}
\vspace{-0.13\textwidth}
\captionsetup{type=figure}
\caption{Average magnitude scores 
}
\label{fig:ablation_magnitude2}

\end{minipage}

\subsection{Experiments on Pixel-based Diffusion Models} \label{sec:exp_adm}
We evaluate our techniques using classifier-guided diffusion sampling with ADM \cite{peebles2022scalable}, an unconditioned pixel-based diffusion model, with their classifier model. Additionally, we compare our methods with two other diffusion sampling methods, namely DPM-Solver++ \cite{lu2022dpmpp} and LTSP \cite{wizadwongsa2023accelerating}, which have demonstrated strong performance in classifier-guided sampling.

\begin{minipage}{0.62\textwidth}
For DPM-Solver++, we use a 2\ts{nd}-order multi-step method and compare the results with and without HB momentum. For LTSP, a split numerical method, we use PLMS4 \cite{liu2022pseudo} to solve the first subproblem (see \cite{wizadwongsa2023accelerating}) and compare different methods for solving the second subproblem, including regular Euler method and Euler method with HB momentum (equivalent to GHVB 0.8).

Our techniques effectively improve FID scores for both DPM-Solver++ and LTSP, as shown in Figure \ref{fig:fid_adm}. Notably, applying our HB momentum to LTSP consistently produces the lowest FID scores. 
This experiment highlights the benefits of using HB momentum, which provides a better choice than Euler method. Table \ref{tab:results1} presents additional results, and Figure \ref{fig:img_adm} provides examples of the generated images.

\end{minipage}
\begin{minipage}{0.38\textwidth}
\centering
\includegraphics[width=0.99\textwidth]{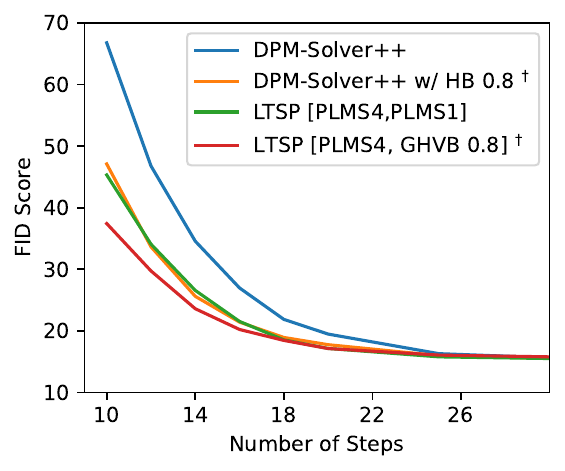}
\vspace{-0.13\textwidth}
\captionsetup{type=figure}
\caption{FID scores on ADM. ($^\dagger$ours)}
\label{fig:fid_adm}
\end{minipage}

\subsection{Experiment on Latent-based Diffusion Models} \label{sec:exp_dit}

We evaluate our techniques using classifier-free guidance diffusion sampling with DiT-XL \cite{peebles2022scalable}, a pre-trained latent-space diffusion model. In this particular setting, 4th-order solvers, such as PLMS4 \cite{liu2022pseudo}, demonstrate superior performance compared to other methods (refer to Appendix \ref{apx:dit}), making it our selected method for comparison.

\begin{minipage}{0.62\textwidth}
In Figure \ref{fig:fid_dit}, a significant gap in FID scores can be observed between 4\ts{th}-order PLMS4 and 1\ts{st}-order DDIM, but this is mostly due to the difference in convergence order rather than divergence artifacts.
Our GHVB 3.8 and 3.9 techniques successfully mitigate numerical divergence and lead to improved FID scores compared to PLMS4, particularly when the number of steps is below 10. Additionally, HB 0.9 also improves FID scores. However, using HB 0.8 with PLMS4 can worsen FID scores compared to using PLMS4 alone, since the method has 1\ts{st}-order convergence, which is the same as DDIM. For high sampling steps, both HB and GHVB achieve comparable performance to PLMS4 without significant degradation of quality. We provide additional results in Table \ref{tab:results}, and example images in Figure \ref{fig:img_dit}.

\end{minipage}
\begin{minipage}{0.37\textwidth}
\centering
\includegraphics[width=0.99\textwidth]{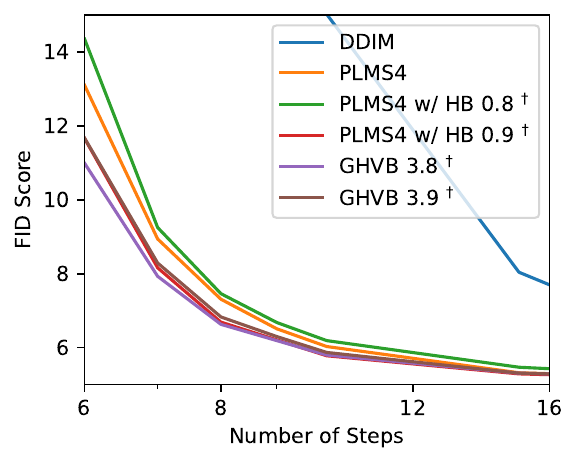}
\vspace{-0.13\textwidth}
\captionsetup{type=figure}
\caption{FID score on DiT. ($^\dagger$ours)}
\label{fig:fid_dit}
\end{minipage}

\subsection{Ablation Study of GHVB} \label{sec:exp_ghvb}

In this section, we conduct an ablation study of the GHVB method. As explained in Section \ref{sec:ghvb}, the damping coefficient $\beta$ of GHVB interpolates between two existing AB methods, DDIM and PLMS2. Our goal here is to analyze the convergence error of GHVB methods. The comparison is done on Stable Diffusion 1.5 with the target results obtained from a 1,000-step PLMS4 method. We measure the mean L2 distance between the sampled results and the target results in the latent space. The results in Figure \ref{fig:ablation_on_beta} suggest that the convergence error of GHVB 1.1 to GHVB 1.9 interpolates between the convergence errors of DDIM and PLMS2 accordingly.

Furthermore, we empirically verify that GHVB does achieve high order of convergence as predicted by Theorem \ref{thm:ghvb}.
We compute the numerical order of convergence using the formula $q \approx \frac{\log(e_{\text{new}} / e_{\text{old}})}{\log(k_{\text{new}} / k_{\text{old}})}$, where $e$ is the error between the sampled and the target latent codes, and $k$ is the number of sampling steps. As shown in Figure \ref{fig:ablation_on_ghvb}, the numerical orders of GHVB 0.5 and GHVB 1.5 approach 0.5 and 1.5, respectively, as the number of steps increases. However, for GHVB 2.5 and GHVB 3.5, the estimated error $e$ may be too small when tested with large numbers of steps, and other sources of error may hinder their convergence. 
Nonetheless, these GHVB methods can achieve high orders of convergence. A detailed analysis of this and other methods is in Appendix \ref{apx:order}.

\begin{minipage}{0.45\textwidth}
\centering

\includegraphics[width=.99\textwidth]{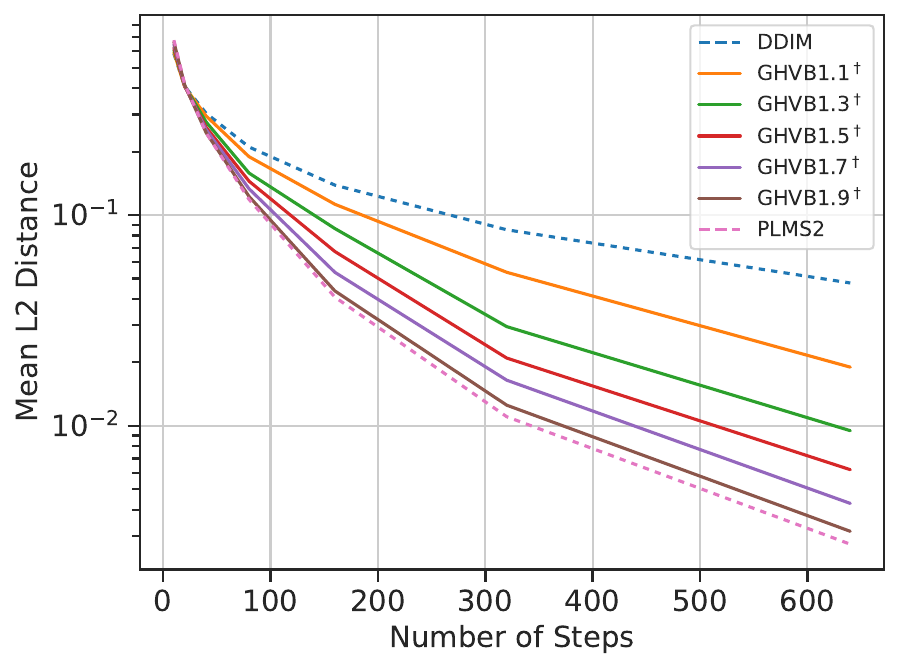}
\captionsetup{type=figure}
\caption{L2 distance in latent space between different sampling methods and the 1,000-step PLMS4 method.}
\label{fig:ablation_on_beta}

\end{minipage}
\hspace{0.01\textwidth}
\begin{minipage}{0.45\textwidth}
\centering

\includegraphics[width=0.99\textwidth]{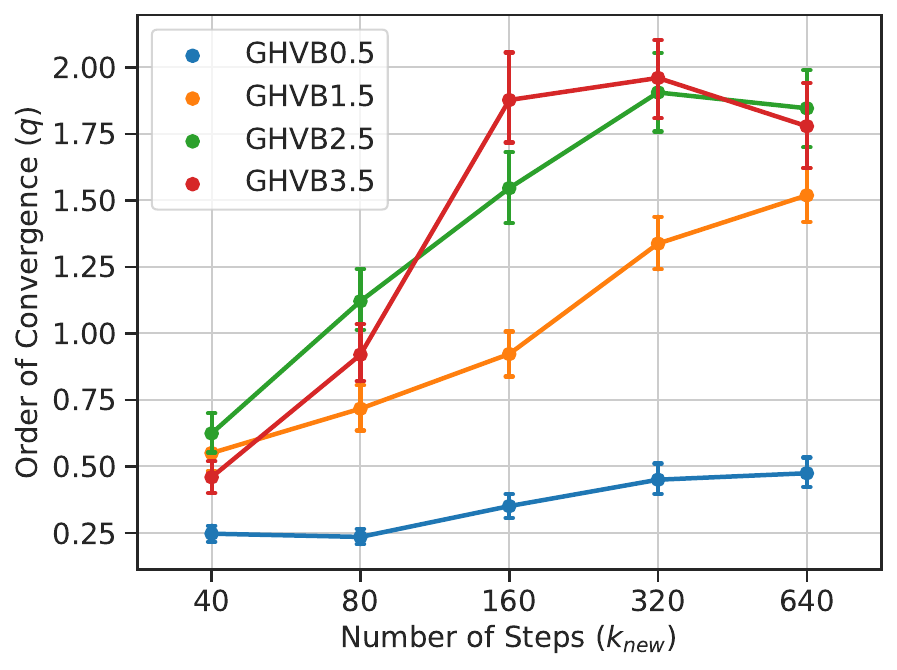}
\captionsetup{type=figure}
\caption{the numerical order of convergence for GHVB.}
\label{fig:ablation_on_ghvb}
\end{minipage}



\section{Discussion}\label{sec6}

The findings of our study highlight an issue when employing high-order methods for sampling diffusion models with a low number of steps. This can result in solution divergence and the emergence of artifacts. To tackle these challenges, we propose two techniques inspired by Polyak's HB momentum, which effectively reduce artifacts while maintaining efficient sampling.

Our work is closely related to several other approaches aimed at improving the sampling speed of diffusion models. One approach involves training separate models that can be sampled faster, which include
model distillation \cite{luhman2021knowledge,salimans2022progressive}, Schrödinger bridge \cite{de2021diffusion}, consistency models \cite{song2023consistency}, and GENIE \cite{dockhorn2022genie}. Another approach focuses on creating better samplers, such as high-order numerical methods, that can be applied to existing diffusion models.
While some samplers were designed for the SDE formulation of diffusion models \cite{tachibana2021ito,dockhorn2021score,song2020score}, most of them deal with the ODE formulation. These include linear multistep methods \cite{liu2022pseudo,lu2022dpmpp,zhang2022fast}, predictor-corrector methods \cite{karras2022elucidating,zhang2022gddim,zhao2023unipc}, and splitting methods \cite{wizadwongsa2023accelerating}. 
Our paper specifically proposes new numerical methods for the ODE formulation, but these methods can also be extended to other sampling approaches involving multiple steps, including the SDE formulation.
These techniques are not mutually exclusive.


\bibliography{reference}

\begin{thebibliography}{10}

\bibitem{ho2020denoising}
Ho, J., A.~Jain, P.~Abbeel.
\newblock (2020), Denoising diffusion probabilistic models.
\newblock In \emph{Proceedings of the 34th International Conference on Neural
  Information Processing Systems}, pages 6840--6851.

\bibitem{song2020denoising}
Song, J., C.~Meng, S.~Ermon.
\newblock (2020), Denoising diffusion implicit models.
\newblock In \emph{International Conference on Learning Representations}.

\bibitem{goodfellow2014generative}
Goodfellow, I., J.~Pouget-Abadie, M.~Mirza, et~al.
\newblock (2014), Generative adversarial nets.
\newblock \emph{Advances in neural information processing systems}, 27.

\bibitem{rombach2022high}
Rombach, R., A.~Blattmann, D.~Lorenz, et~al.
\newblock (2022), High-resolution image synthesis with latent diffusion models.
\newblock In \emph{Proceedings of the IEEE/CVF Conference on Computer Vision
  and Pattern Recognition}, pages 10684--10695.

\bibitem{dhariwal2021diffusion}
Dhariwal, P., A.~Nichol.
\newblock (2021), Diffusion models beat \text{GANs} on image synthesis.
\newblock \emph{Advances in Neural Information Processing Systems},
  34:8780--8794.

\bibitem{nichol2021glide}
Nichol, A.~Q., P.~Dhariwal, A.~Ramesh, et~al.
\newblock (2022), \text{GLIDE}: Towards photorealistic image generation and
  editing with text-guided diffusion models.
\newblock In \emph{International Conference on Machine Learning}, pages
  16784--16804. PMLR.

\bibitem{su2022dual}
Su, X., J.~Song, C.~Meng, et~al.
\newblock (2022), Dual diffusion implicit bridges for image-to-image
  translation.
\newblock In \emph{The Eleventh International Conference on Learning
  Representations}.

\bibitem{sasaki2021unit}
Sasaki, H., C.~G. Willcocks, T.~P. Breckon.
\newblock (2021), \text{Unit-DDPM}: Unpaired image translation with denoising
  diffusion probabilistic models.
\newblock \emph{arXiv preprint arXiv:2104.05358}.

\bibitem{wang2022guided}
Wang, J., Z.~Lyu, D.~Lin, et~al.
\newblock (2022), Guided diffusion model for adversarial purification.
\newblock \emph{arXiv preprint arXiv:2205.14969}.

\bibitem{wu2022guided}
Wu, Q., H.~Ye, Y.~Gu.
\newblock (2022), Guided diffusion model for adversarial purification from
  random noise.
\newblock \emph{arXiv preprint arXiv:2206.10875}.

\bibitem{choi2021ilvr}
Choi, J., S.~Kim, Y.~Jeong, et~al.
\newblock (2021), \text{ILVR}: Conditioning method for denoising diffusion
  probabilistic models.
\newblock In \emph{2021 IEEE/CVF International Conference on Computer Vision
  (ICCV)}, pages 14347--14356. IEEE.

\bibitem{ghosal2023text}
Ghosal, D., N.~Majumder, A.~Mehrish, et~al.
\newblock (2023), Text-to-audio generation using instruction-tuned \text{LLM}
  and latent diffusion model.
\newblock \emph{arXiv preprint arXiv:2304.13731}.

\bibitem{nichol2021improved}
Nichol, A.~Q., P.~Dhariwal.
\newblock (2021), Improved denoising diffusion probabilistic models.
\newblock In \emph{International Conference on Machine Learning}, pages
  8162--8171. PMLR.

\bibitem{watson2021learning}
Watson, D., J.~Ho, M.~Norouzi, et~al.
\newblock (2021), Learning to efficiently sample from diffusion probabilistic
  models.
\newblock \emph{arXiv preprint arXiv:2106.03802}.

\bibitem{salimans2022progressive}
Salimans, T., J.~Ho.
\newblock (2022), Progressive distillation for fast sampling of diffusion
  models.
\newblock In \emph{International Conference on Learning Representations}.

\bibitem{watson2022learning}
Watson, D., W.~Chan, J.~Ho, et~al.
\newblock (2022), Learning fast samplers for diffusion models by
  differentiating through sample quality.
\newblock In \emph{International Conference on Learning Representations}.

\bibitem{song2023consistency}
Song, Y., P.~Dhariwal, M.~Chen, et~al.
\newblock (2023), Consistency models.
\newblock \emph{International Conference on Learning Representations}.

\bibitem{zhang2022fast}
Zhang, Q., Y.~Chen.
\newblock (2022), Fast sampling of diffusion models with exponential
  integrator.
\newblock In \emph{NeurIPS 2022 Workshop on Score-Based Methods}.

\bibitem{lu2022dpm}
Lu, C., Y.~Zhou, F.~Bao, et~al.
\newblock (2022), \text{DPM-Solver}: A fast \text{ODE} solver for diffusion
  probabilistic model sampling in around 10 steps.
\newblock In \emph{Advances in Neural Information Processing Systems}.

\bibitem{liu2022pseudo}
Liu, L., Y.~Ren, Z.~Lin, et~al.
\newblock (2022), Pseudo numerical methods for diffusion models on manifolds.
\newblock In \emph{International Conference on Learning Representations}.

\bibitem{polyak1987introduction}
Polyak, B.~T.
\newblock (1987), Introduction to optimization. optimization software.
\newblock \emph{Inc., Publications Division, New York}, 1:32.

\bibitem{Realistic}
Realistic vision v2.0.
\newblock \url{https://huggingface.co/SG161222/Realistic_Vision_V2.0}, 2023.

\bibitem{Anything}
Anything diffusion v4.0.
\newblock \url{https://huggingface.co/andite/anything-v4.0}, 2023.

\bibitem{Deliberate}
Deliberate diffuson.
\newblock \url{https://huggingface.co/XpucT/Deliberate}, 2023.

\bibitem{zhao2023unipc}
Zhao, W., L.~Bai, Y.~Rao, et~al.
\newblock (2023), \text{UniPC}: A unified predictor-corrector framework for
  fast sampling of diffusion models.
\newblock \emph{arXiv preprint arXiv:2302.04867}.

\bibitem{wizadwongsa2023accelerating}
Wizadwongsa, S., S.~Suwajanakorn.
\newblock (2023), Accelerating guided diffusion sampling with splitting
  numerical methods.
\newblock \emph{International Conference on Learning Representations}.

\bibitem{lu2022dpmpp}
Lu, C., Y.~Zhou, F.~Bao, et~al.
\newblock (2022), \text{DPM-Solver++}: Fast solver for guided sampling of
  diffusion probabilistic models.
\newblock \emph{arXiv preprint arXiv:2211.01095}.

\bibitem{radford2021learning}
Radford, A., J.~W. Kim, C.~Hallacy, et~al.
\newblock (2021), Learning transferable visual models from natural language
  supervision.
\newblock In \emph{International Conference on Machine Learning}, pages
  8748--8763. PMLR.

\bibitem{Adam2021disco}
Letts, A., C.~Scalf, A.~Spirin, et~al.
\newblock Disco diffusion.
\newblock \url{https://github.com/alembics/disco-diffusion}, 2021.

\bibitem{ho2022classifier}
Ho, J., T.~Salimans.
\newblock (2021), Classifier-free diffusion guidance.
\newblock In \emph{NeurIPS 2021 Workshop on Deep Generative Models and
  Downstream Applications}.

\bibitem{lambert1991numerical}
Lambert, J.~D., et~al.
\newblock (1991), \emph{Numerical methods for ordinary differential systems},
  vol. 146.
\newblock Wiley New York.

\bibitem{lin2014microsoft}
Lin, T.-Y., M.~Maire, S.~Belongie, et~al.
\newblock (2014), Microsoft \text{COCO}: Common objects in context.
\newblock In \emph{Computer Vision--ECCV 2014: 13th European Conference,
  Zurich, Switzerland, September 6-12, 2014, Proceedings, Part V 13}, pages
  740--755. Springer.

\bibitem{higham1993stiffness}
Higham, D.~J., L.~N. Trefethen.
\newblock (1993), Stiffness of \text{ODEs}.
\newblock \emph{BIT Numerical Mathematics}, 33:285--303.

\bibitem{berry2004implementation}
Berry, M.~M., L.~M. Healy.
\newblock (2004), Implementation of \text{Gauss-Jackson} integration for orbit
  propagation.
\newblock \emph{The Journal of the Astronautical Sciences}.

\bibitem{russakovsky2015imagenet}
Russakovsky, O., J.~Deng, H.~Su, et~al.
\newblock (2015), \text{ImageNet} large scale visual recognition challenge.
\newblock \emph{International journal of computer vision}, 115(3):211--252.

\bibitem{peebles2022scalable}
Peebles, W., S.~Xie.
\newblock (2022), Scalable diffusion models with transformers.
\newblock \emph{arXiv preprint arXiv:2212.09748}.

\bibitem{luhman2021knowledge}
Luhman, E., T.~Luhman.
\newblock (2021), Knowledge distillation in iterative generative models for
  improved sampling speed.
\newblock \emph{arXiv preprint arXiv:2101.02388}.

\bibitem{de2021diffusion}
De~Bortoli, V., J.~Thornton, J.~Heng, et~al.
\newblock (2021), Diffusion \text{Schr{\"o}dinger} bridge with applications to
  score-based generative modeling.
\newblock \emph{Advances in Neural Information Processing Systems},
  34:17695--17709.

\bibitem{dockhorn2022genie}
Dockhorn, T., A.~Vahdat, K.~Kreis.
\newblock (2022), \text{GENIE}: Higher-order denoising diffusion solvers.
\newblock \emph{arXiv preprint arXiv:2210.05475}.

\bibitem{tachibana2021ito}
Tachibana, H., M.~Go, M.~Inahara, et~al.
\newblock (2021), \text{It{\^o}-Taylor} sampling scheme for denoising diffusion
  probabilistic models using ideal derivatives.
\newblock \emph{arXiv e-prints}, pages arXiv--2112.

\bibitem{dockhorn2021score}
Dockhorn, T., A.~Vahdat, K.~Kreis.
\newblock (2022), Score-based generative modeling with critically-damped
  langevin diffusion.
\newblock \emph{International Conference on Learning Representations}.

\bibitem{song2020score}
Song, Y., J.~Sohl-Dickstein, D.~P. Kingma, et~al.
\newblock (2020), Score-based generative modeling through stochastic
  differential equations.
\newblock In \emph{International Conference on Learning Representations}.

\bibitem{karras2022elucidating}
Karras, T., M.~Aittala, T.~Aila, et~al.
\newblock (2022), Elucidating the design space of diffusion-based generative
  models.
\newblock In \emph{NeurIPS 2022 Workshop on Deep Generative Models and
  Downstream Applications}.

\bibitem{zhang2022gddim}
Zhang, Q., M.~Tao, Y.~Chen.
\newblock (2023), \text{gDDIM}: Generalized denoising diffusion implicit
  models.
\newblock \emph{International Conference on Learning Representations}.

\bibitem{lucas2018aggregated}
Lucas, J., S.~Sun, R.~Zemel, et~al.
\newblock (2018), Aggregated momentum: Stability through passive damping.
\newblock \emph{arXiv preprint arXiv:1804.00325}.

\bibitem{nesterov1983method}
Nesterov, Y.
\newblock (1983), A method for unconstrained convex minimization problem with
  the rate of convergence o(1/k\^{}2).
\newblock In \emph{Doklady an ussr}, vol. 269, pages 543--547.

\bibitem{kynkaanniemi2019improved}
Kynk{\"a}{\"a}nniemi, T., T.~Karras, S.~Laine, et~al.
\newblock (2019), Improved precision and recall metric for assessing generative
  models.
\newblock \emph{Advances in Neural Information Processing Systems}, 32.

\bibitem{zhang2018unreasonable}
Zhang, R., P.~Isola, A.~A. Efros, et~al.
\newblock (2018), The unreasonable effectiveness of deep features as a
  perceptual metric.
\newblock In \emph{Proceedings of the IEEE conference on computer vision and
  pattern recognition}, pages 586--595.

\end{thebibliography}
\bibliographystyle{ref_style}
\newpage
\appendix
\addcontentsline{toc}{section}{Appendix} 
\part{Appendices} 

\startcontents 
\printcontents{}{0}{\noindent\textbf{Appendix contents}\vskip3pt\hrule\vskip5pt} 
\vspace{10pt}

\definecolor{mycolor}{RGB}{220, 220, 220}
\newcommand{\whitebox}[1]{{\setlength{\fboxsep}{2pt}\colorbox{white}{#1}}}
\newcommand{\graybox}[1]{{\setlength{\fboxsep}{2pt}\colorbox{mycolor}{#1}}}

\section{Stability Region of Adam-Bashforth Method} \label{apx:stab_ana}

To investigate the stability of the AB2 method, we apply AB2 to the test equation $x'=\lambda x$, which was also used with the Euler method (Section \ref{stab_region}). We have $x_{n+1} = x_n + \delta \left( \frac{3}{2} \lambda x_n - \frac{1}{2} \lambda x_{n-1} \right)$. To solve this linear recurrence relation, we substitute $x_n = r^n$ into the formula, where $r$ is a complex constant. Simplifying the resulting equation, we obtain the characteristic equation:
\begin{align}
r^2 - \left(1 + \frac{3}{2} \delta \lambda\right) r + \frac{1}{2} \delta \lambda = 0,
\end{align}
which has the solutions
\begin{align}
r_1 &= \frac{1}{2} \left(1 + \frac{3}{2} \delta \lambda + \sqrt{\left(1 + \frac{3}{2} \delta \lambda\right)^2-2\delta\lambda}\right), \\
r_2 &= \frac{1}{2} \left(1 + \frac{3}{2} \delta \lambda - \sqrt{\left(1 + \frac{3}{2} \delta \lambda\right)^2-2\delta\lambda}\right).
\end{align}
The general formulation of $x_n$ can be expressed as
\begin{align}
x_n = a_1 r_1^n + a_2 r_2^n,
\end{align}
where $a_1$ and $a_2$ are constants. The numerical solution $x_n$ tends to 0 as $n$ tends to infinity when both $|r_1|<1$ and $|r_2|<1$, which means the stability region of AB2 is determined by the complex region
\begin{align}
S = \left\{z \in \mathbb{C} : \left|\frac{1}{2} \left(1 + \frac{3}{2} z \pm \sqrt{\left(1 + \frac{3}{2} z\right)^2-2z} \right)\right|\leq 1 \right\}.
\end{align}

Solving for the complex area from the roots of the characteristic equation can pose significant challenges in numerical analysis. One commonly employed graphical technique to visualize the stability region is the boundary locus technique \cite{lambert1991numerical}.

\subsection{The Boundary Locus Technique}
The boundary locus technique \cite{lambert1991numerical} begins by defining the shift operator $E$ such that $Ex_k = x_{k-1}$. Note that $E^2 x_k = Ex_{k-1} = x_{k-2}$. Generally, a numerical method can be represented in the following form:
\begin{align} \label{eq:general_method}
A(E)x_{n} = \delta B(E)f(x_n),
\end{align}
where $A$ and $B$ are polynomials of $E$. For example, in the case of the AB2 method, we have $A(E) = 1 - E$ and $B(E) = \frac{3}{2}E - \frac{1}{2}E^2.$

To determine the stability region of a numerical method, we apply the boundary locus technique to the general form given by Equation \ref{eq:general_method}. The characteristic equation of the method can be obtained by substituting $f(x_n) = \lambda x_n$ (i.e., the test equation) and $x_n = r^n$, 
which yields
\begin{align}
A(r^{-1}) = \delta \lambda B(r^{-1}),
\end{align}
where $r$ is the root of the method's characteristic equation. The stability region of the method is the area in the complex plane where the characteristic root $r$ have modulus less than 1. The boundary of the stability region can be determined by substituting $r$ with a modulus of 1 (which means that $r = e^{i\theta}$ for some real value $\theta$) into the characteristic equation and solving for $z = \delta\lambda$. This yields the locus of points in the complex plane where the characteristic roots of the method are on the boundary of the stability region. Specifically, we can obtain the curve $z = s(\theta) = A(e^{-i\theta})/ B(e^{-i\theta})$, where $\theta \in [-\pi, \pi]$, that represents the boundary of the stability region in the complex plane. By comparing the stability regions of different numerical methods, we can determine which method is more stable and accurate for a given problem. The boundary locus technique provides a powerful tool for analyzing the stability of numerical methods and can help guide the selection of appropriate methods for solving ODE problems.

\begin{example}
(Euler Method) The Euler method, a numerical technique for approximating solutions of ODE, can be expressed as:
\begin{align}
(1-E) x_{n} = \delta E f(x_{n})
\end{align}
The associated polynomials for this method are:
\begin{align}
A(z) = 1 - z, \quad B(z) = z
\end{align}
The stability region of the Euler method corresponds to the locus curve in which the solution remains bounded. This region can be determined by evaluating the complex function:
\begin{align}
s (\theta) = \frac{A(e^{-i \theta})}{B(e^{-i\theta})} = \frac{1-e^{-i \theta} }{e^{-i \theta}} = e^{i\theta} - 1, \quad \theta \in [-\pi, \pi].
\end{align}
The locus curve forms a perfect circle with a radius of 1 and a center at -1.
\end{example}

\begin{example}
(AB Methods) The 2\ts{nd}-order Adams-Bashforth (AB2) method is given by:
\begin{align}
(1-E) x_{n} = \delta \left(\frac{3}{2}E -\frac{1}{2} E^2 \right) f(x_{n}).
\end{align}

The locus curve representing the stability region of this method is given by:
\begin{align}
s (\theta) = \frac{1 -e^{-i\theta}}{\frac{3}{2}e^{-i\theta} - \frac{1}{2} e^{-2i\theta}} = \frac{2(1 -e^{-i\theta})}{3e^{-i\theta} - e^{-2i\theta}} , \quad \theta \in [-\pi, \pi].
\end{align}

Similarly, the stability regions for the AB3 and AB4 methods can be obtained by evaluating the complex functions:
\begin{align}
s (\theta) = \frac{12(1 -e^{-i\theta})}{23e^{-i\theta} - 16 e^{-2i\theta} + 5 e^{-3i\theta}}, \quad \theta \in [-\pi, \pi].
\end{align}
\begin{align}
s(\theta) = \frac{24(1 -e^{-i\theta})}{55e^{-i\theta} - 59 e^{-2i\theta} + 37 e^{-3i\theta} - 9 e^{-4i\theta}}, \quad \theta \in [-\pi, \pi].
\end{align}
The locus curves for the  boundary of stability regions of first four AB methods are visualized in Figure \ref{fig:stability_ab}.
\end{example}

\section{Derivation of Test Equation} \label{apx:test_eq}
This section presents the derivation of the test equation $u'=\lambda u$, which serves as a fundamental tool for analyzing the stability of numerical methods in diffusion sampling, as discussed in the Section \ref{artifact_sampling}.

Starting from the differential equation for $\bar{x}$, we have 
\begin{align}
\frac{d\bar{x}}{d\sigma} = \nabla \bar{\epsilon}(x^*) (\bar{x}-x^*). 
\end{align}

We then define $u=v^T(\bar{x} - x^*)$, where $v$ is a normalized eigenvector of $\nabla \bar{\epsilon}(x^*)^T$ corresponding to the eigenvalue $\lambda$. Taking the derivative of $u$ with respect to $\sigma$ and using the chain rule, we have:

\begin{align}
\frac{du}{d\sigma} =& \: v^T \frac{d}{d\sigma} (\bar{x} - x^*) = v^T [\nabla \bar{\epsilon} (x^*)] (\bar{x} - x^*) \\
=& \: [\nabla \bar{\epsilon} (x^*)^Tv] ^T (\bar{x} - x^*) = (\lambda v)^T (\bar{x} - x^*) \\
=& \: \lambda u
\end{align}

Thus, we obtain the test equation $u'=\lambda u$.

\section{Toy ODE Problem} \label{apx:toy}
This section aims to demonstrate how the solutions yielded by numerical methods can diverge when the stability regions of the methods are too small. Additionally, we illustrate how our momentum-based techniques can enlarge the stability region. The demonstration is conducted on a 2D toy ODE problem given by:

\begin{align} \label{eq:toy_ode}
\frac{dx}{dt} =\begin{bmatrix} 0 & 1 \\ -9 & -10\end{bmatrix}\begin{bmatrix} x_1 \\ x_2\end{bmatrix}, \qquad x(0) = \begin{bmatrix} -1 \\ 0\end{bmatrix}.
\end{align}

The eigenvalues of the $2\times2$ matrix are $-9$ and $-1$, and the exact solution of Equation \ref{eq:toy_ode} is given by:
\begin{align}
x(t) = \frac{1}{8}\begin{bmatrix} 1 \\ -9 \end{bmatrix} e^{-9t} + \frac{9}{8}\begin{bmatrix} -1 \\ 1 \end{bmatrix} e^{-t}.
\end{align}
As $t$ increases, $x(t)$ converges to the origin.

Let us say we want to numerically compute $x(3)$ by integrating the ODE for 26 steps with a numerical method. We divide the time interval $[0,3]$ into 26 equal intervals, resulting in a step size of $\delta = 3/26$. For this particular setting, it turns out that the the 2\ts{nd}-order Adams-Bashforth (AB2) method diverges, but the Euler method converges. To see this, observe that the stability region of the AB2 method only cover the interval $[-1,0]$ of the real line, as depicted in Figures \ref{fig:cur_ABHB} and \ref{fig:cur_GHVB}. So, for the eigenvalue $\lambda = -9$, the product $\delta \lambda = -27/26$ lies just outside the region. Consequently, the numerical solution yielded by AB2 diverges, as indicated by the blue line in Figures \ref{fig:sol_ABHB} and \ref{fig:sol_GHVB}. In contrast, the Euler method's stability region contains both values of $\delta \lambda$, and the numerical solution, represented by the green line in Figures \ref{fig:sol_ABHB} and \ref{fig:sol_GHVB}, appears to be more accurate.


The AB2 method can be made to more accurately compute $x(3)$ by applying to it any of our proposed techniques: Heavy Ball momentum (HB) and GHVB. The stability regions of the modified AB2 methods are given by the red ($\beta = 0.8$) and yellow ($\beta = 0.9$) lines in Figures \ref{fig:cur_ABHB} and \ref{fig:cur_GHVB}. Observe that they contains the points associated with the $\lambda \delta$ values. As a result, the numerical solutions of the modified methods converge to the origin, as demonstrated by the red and yellow lines in Figures \ref{fig:sol_ABHB} and \ref{fig:sol_GHVB}, respectively.

\begin{figure}[ht]
  \begin{subfigure}{0.26\textwidth}
    \includegraphics[width=\textwidth]{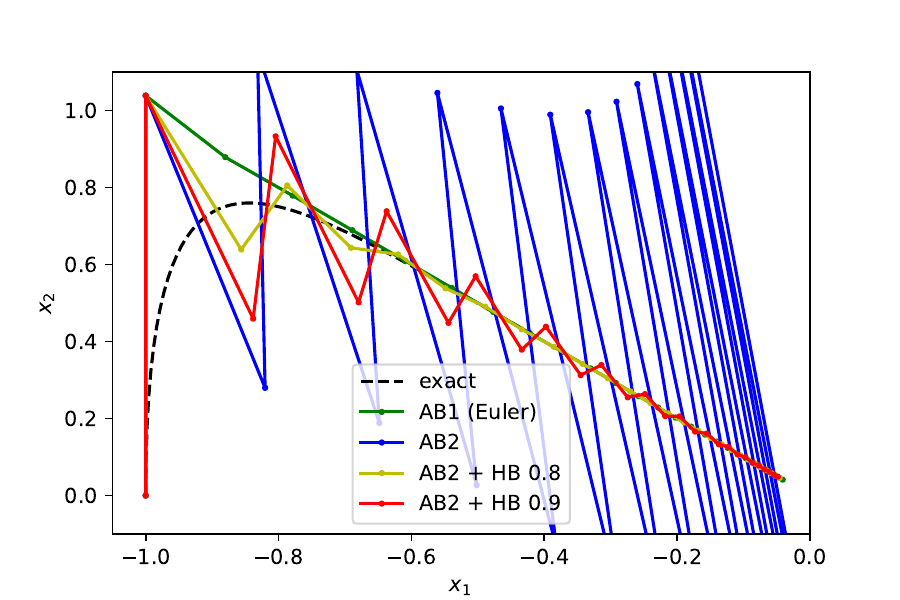}
    \caption{solution}
    \label{fig:sol_ABHB}
  \end{subfigure}
  \hspace{-0.03\textwidth}
  \begin{subfigure}{0.26\textwidth}
    \includegraphics[width=\textwidth]{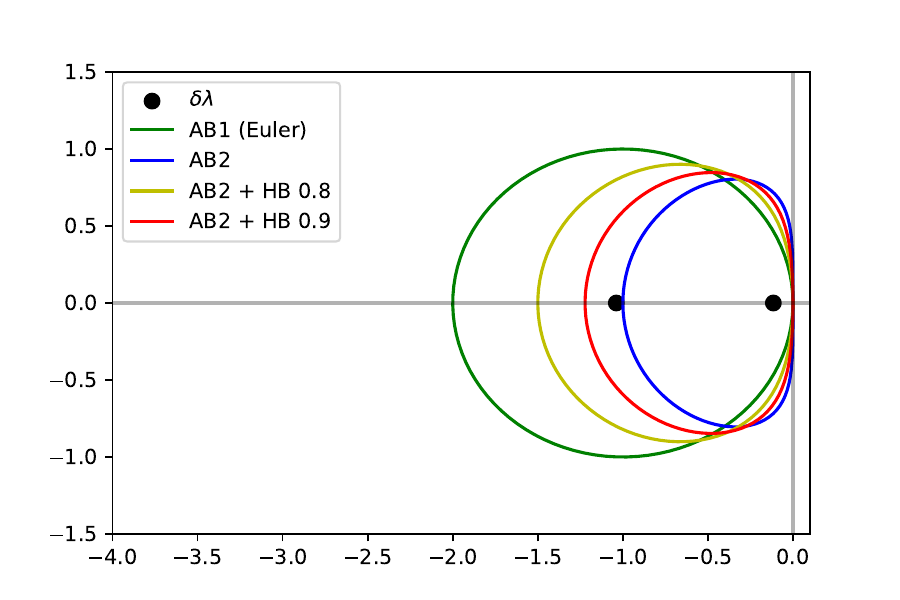}
    \caption{stability region}
    \label{fig:cur_ABHB}
  \end{subfigure}
  \hspace{-0.03\textwidth}
  \begin{subfigure}{0.26\textwidth}
    \includegraphics[width=\textwidth]{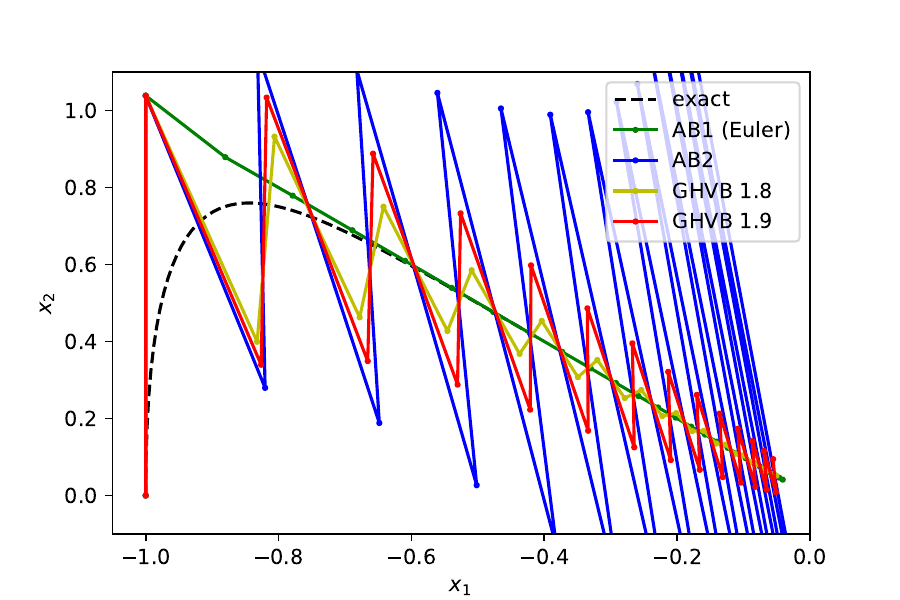}
    \caption{solution}
    \label{fig:sol_GHVB}
  \end{subfigure}
  \hspace{-0.03\textwidth}
  \begin{subfigure}{0.26\textwidth}
    \includegraphics[width=\textwidth]{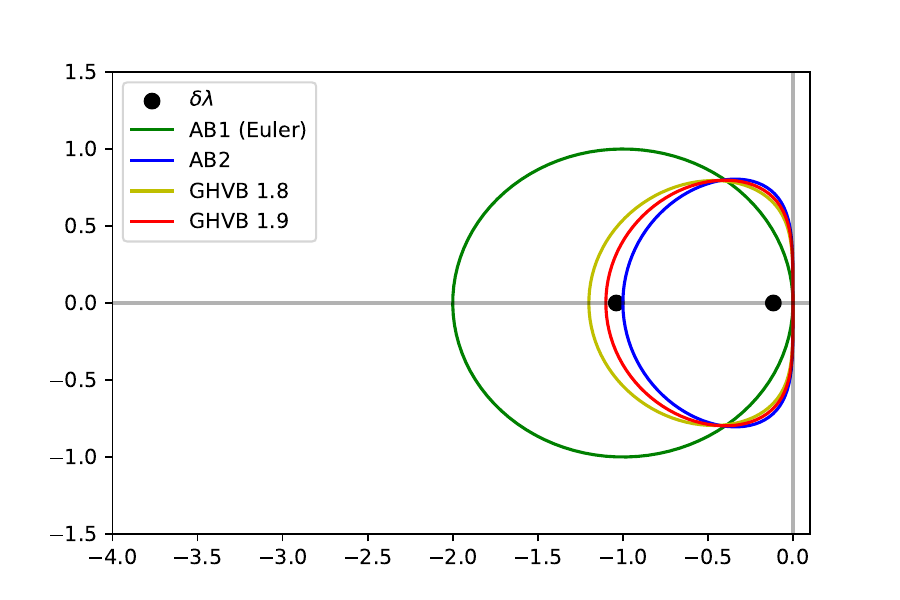}
    \caption{stability region}
    \label{fig:cur_GHVB}
  \end{subfigure}
  \caption{Comparison of solution trajectories and stability regions of various numerical methods when applied to the toy ODE problem. Here, we seek to compute $x(3)$ in 26 steps with the Euler method, the AB2 method, and methods resulting from modifying AB2 with our momentum-based techniques. Subfigure (a) presents the numerical solutions obtained using our modified AB2 method with HB momentum, while subfigure (c) showcases those obtained using our GHVB. The stability regions of the methods are depicted in subfigures (b) and (d) respectively.}
  \label{fig:toy_ex}
\end{figure}

\section{Implementation Details of PLMS with HB and GHVB Methods}

In this section, we present the complete algorithms for the PLMS method with the HB momentum and the GHVB method in Algorithm \ref{algo:plms_hb} and Algorithm \ref{algo:ghvb}, respectively. Additionally, we include the locus curves that represent the boundaries of the stability regions for each method.

For the PLMS method with HB momentum, the locus curves representing the stability regions with parameter $\beta$ are given by:

\textbf{PLMS1 with HB $\beta$:}
\begin{align}
s(\theta) = \frac{(1-e^{-i\theta})(1-(1-\beta)e^{-i\theta})}{\beta e^{-i\theta}}
\end{align}

\textbf{PLMS2 with HB $\beta$:}
\begin{align}
s(\theta) = \frac{2(1-e^{-i\theta})(1-(1-\beta)e^{-i\theta})}{\beta (3e^{-i\theta} - e^{-2i\theta})}
\end{align}

\textbf{PLMS3 with HB $\beta$:}
\begin{align}
s(\theta) = \frac{12(1-e^{-i\theta})(1-(1-\beta)e^{-i\theta})}{\beta (23e^{-i\theta} - 16e^{-2i\theta} + 5e^{-3i\theta})}
\end{align}

\textbf{PLMS4 with HB $\beta$:}
\begin{align}
s(\theta) = \frac{24(1-e^{-i\theta})(1-(1-\beta)e^{-i\theta})}{\beta (55e^{-i\theta} - 59e^{-2i\theta} + 37e^{-3i\theta} - 9e^{-4i\theta})}
\end{align}

The locus curves representing the boundaries of the stability regions for different values of $\beta$ are illustrated in Figure \ref{fig:polyak_comparison}.

For  GHVB method, the locus curve representing the stability region with parameter $\beta$ is given by:

\textbf{1st-order GHVB} (equivalent to PLMS1 with HB):
\begin{align}
s(\theta) = \frac{(1-e^{-i\theta})(1-(1-\beta)e^{-i\theta})}{\beta e^{-i\theta}}
\end{align}

\textbf{2nd-order GHVB}:
\begin{align}
s(\theta) = \frac{2(1-e^{-i\theta})(1-(1-\beta)e^{-i\theta})}{((2+\beta)e^{-i\theta} - (2-\beta)e^{-2i\theta})}
\end{align}

\textbf{3rd-order GHVB}:
\begin{align}
s(\theta) = \frac{12(1-e^{-i\theta})(1-(1-\beta)e^{-i\theta})}{(18+5\beta)e^{-i\theta} - (24-8\beta)e^{-2i\theta} + (6-\beta)e^{-3i\theta}}
\end{align}

\textbf{4th-order GHVB}:
\begin{align}
s(\theta) = \frac{24(1-e^{-i\theta})(1-(1-\beta)e^{-i\theta})}{(46+9\beta)e^{-i\theta} - (78-19\beta)e^{-2i\theta} + (42-5\beta)e^{-3i\theta} - (10-\beta)e^{-4i\theta}}
\end{align}

These locus curves describe the boundaries of the stability regions and are shown in Figure \ref{fig:stability_regions}.

\begin{algorithm}[htbp]
 \caption{PLMS with HB momentum} \label{algo:plms_hb}
\SetAlgoLined        
    \textbf{input:} $\bar{x}_n$ (previous result), $\delta$ (step size), \\
    $\{e_i\}_{i<n}$ (evaluation buffer), $r$ (method order),\\
    $v_n$ (previous velocity) \;
   \quad $e_n = \bar{\epsilon}_\sigma(\bar{x}_n)$ \; \quad $c= \min(r, n)$ \;
   \quad \textbf{if} $c==1$ \textbf{then}  \\ \quad \quad 
        $\hat{e} = e_n $ \;
   \quad \textbf{else if} $c==2$ \textbf{then} \\ \quad \quad 
        $\hat{e} = (3e_n - e_{n-1})/2 $ \;
   \quad \textbf{else if} $c==3$ \textbf{then} \\ \quad\quad 
        $\hat{e} = (23e_n - 16e_{n-1} + 5 e_{n-2})/12  $ \;
   \quad \textbf{else} \\ \quad\quad 
        $\hat{e} = (55e_n - 59e_{n-1} + 37e_{n-3} - 9e_{n-4})/24  $ \;
   \quad $v_{n+1} = (1-\beta)v_n + \beta \hat{e} $\; 
   \quad  \KwResult{ $\bar{x}_n + \delta v_{n+1}$ } 
\end{algorithm}

\begin{algorithm}[htbp]
 \caption{GHVB} \label{algo:ghvb}
\SetAlgoLined
    \textbf{input:} $\bar{x}_n$ (previous result), $\delta$ (step size), $\beta$ (damping parameter) \\
    $\{v_i\}_{i\leq n}$ (evaluation buffer), $r$ (method order),  \;
   \quad $v_{n+1} = (1-\beta)v_n + \beta\bar{\epsilon}_\sigma(\bar{x}_n)$ \; \quad $c= \min(r, n)$ \;
   \quad \textbf{if} $c==1$ \textbf{then}  \\ \quad \quad 
        $\hat{e} = v_{n+1} $ \;
   \quad \textbf{else if} $c==2$ \textbf{then} \\ \quad \quad 
        $\hat{e} = ((2+\beta)v_{n+1} - (2-\beta)v_{n})/2\beta $ \;
   \quad \textbf{else if} $c==3$ \textbf{then} \\ \quad\quad 
        $\hat{e} = ((18+5\beta)v_{n+1} - (24-8\beta)v_{n}$ \\ \quad\quad \:\:\:\:
        $+ (6-\beta)v_{n+1})/12\beta   $ \;
   \quad \textbf{else if} $c==4$  \textbf{then}\\ \quad\quad 
        $\hat{e} = ((46+9\beta)v_{n+1} - (78-19\beta)v_{n}$ \\ \quad\quad \:\:\:\:
        $ + (42-5\beta)v_{n-1} - (10-\beta)v_{n-2})/24\beta  $ \;
   \quad \textbf{else}\\ \quad\quad 
        $\hat{e} = ((1650+251\beta)v_{n+1} - (3420-646\beta)v_{n}$ \\ \quad\quad \:\:\:\:
        $ + (2880-264\beta)v_{n-1} - (1380-106\beta)v_{n-2} $  \\ \quad\quad \:\:\:\:
        $ + (270-19\beta)v_{n-3})/720\beta  $ \;
   \quad  \KwResult{ $\bar{x}_n + \delta \hat{e}$ } 
\end{algorithm}

\section{Variance Momentum Methods} \label{apx:moment}

In 2019, a variant of Polyak's HB momentum called aggregated momentum \cite{lucas2018aggregated} was proposed. Its objective is to enhance stability while also offering convergence advantages. This modification introduces multiple velocities, denoted by $v_n^{(i)}$, each associated with its specific damping coefficient $\beta^{(i)}$.
\begin{align}
v_{n+1}^{(i)} = (1-\beta^{(i)})v_n^{(i)} + \beta^{(i)}f(x_n), \qquad x_{n+1} = x_n + \delta \sum^K_{i=1} w^{(i)} v_{n+1}^{(i)}
\end{align}
Nesterov's momentum \cite{nesterov1983method} is one version of the classic momentum that can also be applied to diffusion sampling processes to improve stability. It can be written as follows:
\begin{align}
y_{n+1} = x_n + \delta \beta f(x_n), \qquad x_{n+1} = y_{n+1} + (1-\beta) (y_{n+1} - y_n)
\end{align}
In fact, Nesterov's momentum can be obtained from aggregated momentum by considering the following:
\begin{align}
v_{n+1}^{(1)} &= (1-\beta) v_n^{(1)} + \beta f(x_n), \qquad v_{n+1}^{(2)} = f(x_n), \nonumber \\
x_{n+1} &= x_n + \delta ((1-\beta) v_{n+1}^{(1)} + \beta v_{n+1}^{(2)}).
\end{align}

The stability regions of Nesterov's momentum when applied to the Euler method and high-order Adams-Bashforth methods are illustrated in Figures \ref{fig:1st_nest} through \ref{fig:4th_nest}. Observe that the stability regions of methods with Nesterov's momentum become larger in a similar manner to those with Polyak's HB momentum. However, the enlargement due to Nesterov's momentum is more pronounced in the vertical direction, while the Polyak's HB momentum's enlargement is more horizontal in nature. (See Figures \ref{fig:cur_ABHB} and \ref{fig:cur_GHVB}) The differences in the shapes of the stability regions suggest one type of momentum is more suitable to certain ODE problems than the other.

\begin{figure}[ht]
  \begin{subfigure}{0.26\textwidth}
    \includegraphics[width=\textwidth]{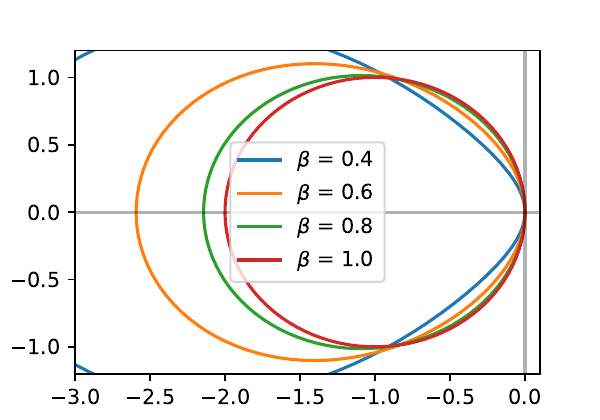}
    \caption{1\ts{st} order}
    \label{fig:1st_nest}
  \end{subfigure}
  \hspace{-0.03\textwidth}
  \begin{subfigure}{0.26\textwidth}
    \includegraphics[width=\textwidth]{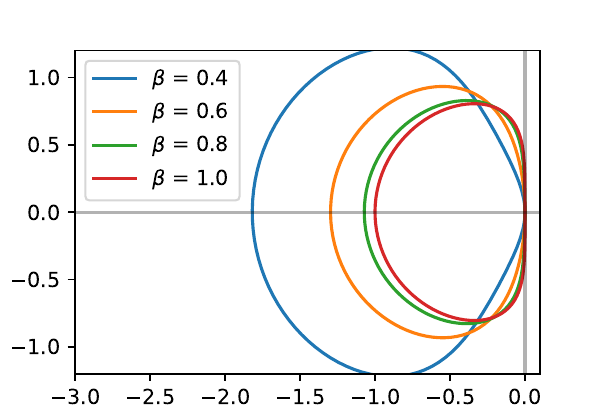}
    \caption{2\ts{nd} order}
    \label{fig:2nd_nest}
  \end{subfigure}
  \hspace{-0.03\textwidth}
  \begin{subfigure}{0.26\textwidth}
    \includegraphics[width=\textwidth]{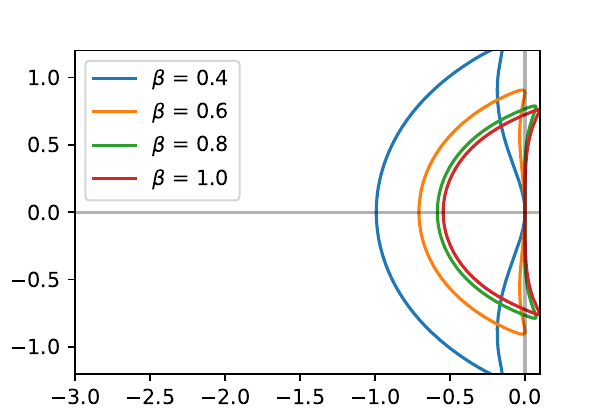}
    \caption{3\ts{rd} order}
    \label{fig:3rd_nest}
  \end{subfigure}
  \hspace{-0.03\textwidth}
  \begin{subfigure}{0.26\textwidth}
    \includegraphics[width=\textwidth]{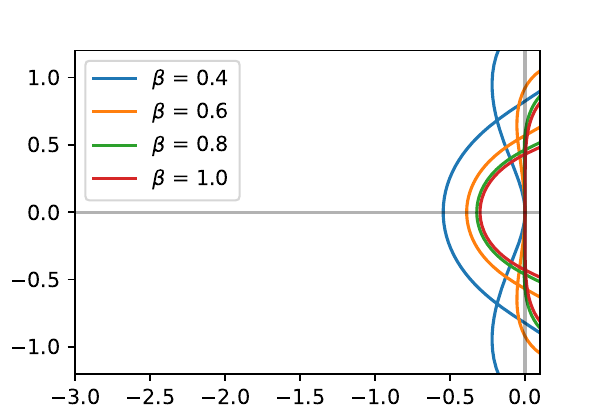}
    \caption{4\ts{th} order}
    \label{fig:4th_nest}
  \end{subfigure}
  \caption{Comparison of stability regions for different methods with different levels of Nesterov's momentum.}
  \label{fig:stability_nest}
\end{figure}

While generalizing the aggregated momentum method to higher-order methods is possible, it is no longer as straightforward as it is with the HB method. As an example, we will consider the 2\ts{nd}-order generalization of Nesterov's momentum method.

We begin by noting that $(\beta + (1-\beta)\Delta)v_{n+1} = \beta f(x_n)$. Our goal is to find polynomials $B$ and $C$ such that
\begin{align}
\Delta x_{n+1} = \delta (B(\Delta) v_{n+1} + C(\Delta) f(x_n)).
\end{align}
Multiplying both sides of the equation by $(\beta + (1-\beta)\Delta)$, we get
\begin{align}
(\beta + (1-\beta)\Delta)\Delta x_{n+1} = \delta (\beta B(\Delta) f(x_n) + (\beta + (1-\beta)\Delta) C(\Delta) f(x_n)).
\end{align}
Replace the left side with the first two terms from Equation \ref{eq:adam_bata}, we obtain
\begin{align}
\delta \left(\beta + \frac{2+\beta}{2}\Delta \right) f(x_n) = \delta (\beta B(\Delta) + (\beta + (1-\beta)\Delta) C(\Delta) ) f(x_n).
\end{align}

Let $B(\Delta) = b_0 + b_1 \Delta$ and $C(\Delta) = c_0$. Then, by balancing the coefficients of $\Delta$ on both sides,  we have $1 = b_0 + c_0$ and $\frac{2+\beta}{2} = \beta b_1 + (1-\beta) c_0$. We can now write the final formulation as follows:
\begin{align}
x_{n+1} = x_n + \delta (b_0 v_{n+1} + b_1 (v_{n+1} - v_n) + c_0 f(x_n)).
\end{align}
This suggests that there are countless different ways to expand the stability region of a numerical method, which offer many new research opportunities.

\section{Order of Convergence} \label{apx:conv}
When solving ODEs numerically, it is important to consider the accuracy of the method used. One way to measure accuracy is by considering the order method's of convergence of the. Suppose we have a numerical method of the form
\begin{align}
A(E)x_{n} = \delta B(E)f(x_n),
\end{align}
where $A(E) = a_0 + a_1 E + a_2 E^2 +...+a_s E^s$ and $B(E) = b_0 + b_1 E + ... + b_s E^s$. The method is said to be of $p^{th}$ order if and only if, for all sufficiently smooth functions $x$, we have that
\begin{align}
\sum^s_{m=0} a_m x(\sigma - m\delta) - \delta \sum^s_{m=0} b_m x'(\sigma - m\delta) = \mathcal{O}(\delta^{p+1}),
\end{align}
where $x'$ denotes the derivative of $x$.

To derive the order of convergence, we use Taylor expansion for both $x$ and $x'$, yielding
\begin{align*}
\text{L.H.S.} & = \sum^s_{m=0} a_m \sum^\infty_{k=0} \frac{(-m\delta)^k}{k!}x^{(k)}(\sigma) - \delta \sum^s_{m=0} b_m\sum^\infty_{k=0} \frac{(-m\delta)^k}{k!}x^{(k+1)}(\sigma)\\
& = \sum^\infty_{k=0}\left(\sum^s_{m=0} a_m\frac{(-m\delta)^k}{k!} \right)x^{(k)}(\sigma) - \delta \sum^\infty_{k=0}\left(\sum^s_{m=0} b_m\frac{(-m\delta)^k}{k!} \right)x^{(k+1)}(\sigma) \\
& = \sum^\infty_{k=0}\left(\sum^s_{m=0} a_m\frac{(-m\delta)^k}{k!} \right)x^{(k)}(\sigma) + \sum^\infty_{k=1}\left(\sum^s_{m=0} b_m\frac{m^{k-1}(-\delta)^k}{(k-1)!} \right)x^{k}(\sigma) \\
&= \sum^s_{m=0} a_m + \sum^\infty_{k=1}\left(\sum^s_{m=0} a_m\frac{m^k}{k!}+\sum^s_{m=0} b_m\frac{m^{k-1}}{(k-1)!} \right)x^{k}(\sigma)(-\delta)^k
\end{align*}
where $x^{(k)}(\sigma)$ denotes the $k^{th}$ derivative of $x$ evaluated at $\sigma$.

Therefore, the method has order $p$ if and only if the coefficients satisfy the conditions given by 
\begin{align} \label{eq:condi}
\sum^s_{m=0} a_m &= 0, \notag \\
\sum^s_{m=0} a_m\frac{m^k}{k!}+\sum^s_{m=0} b_m\frac{m^{k-1}}{(k-1)!} &= 0, \quad k = 0, 1, ..., p.
\end{align}

Now, we discuss the convergence order of any numerical method after HB momentum is applied to it. An example of such an algorithm is the modified PLMS method, presented in Algorithm \ref{algo:plms_hb}.

\begin{theorem}[Convergence order of numerical methods with HB momentum] \label{thm:hb}
Suppose that a $p$\ts{th}-order numerical method has the form $x_{n+1} = x_n + \delta \sum^s_{m=0} b_m f(x_{n-m})$, where $p\geq 1$.  The modified method that uses HB momentum can be expressed as follows:
\begin{align}
v_{n+1} &= (1-\beta) v_n + \beta \sum^s_{m=0} b_m f(x_{n-m}), \\
x_{n+1} &= x_n + \delta v_{n+1}.
\end{align}
It has first-order convergence.
\end{theorem}

\begin{proof}
From the condition given in \ref{eq:condi}, it follows that $\sum^s_{m=0} b_m = 1$. 
We can rewrite these equations as:
\begin{align} \label{eq:sing_hb}
x_{n+1} - x_n - (1-\beta) (x_n - x_{n-1}) = \delta \beta \sum^s_{m=0} b_m f(x_{n-m}).
\end{align}
To estimate the order of the modified method, we evaluate Equation \ref{eq:sing_hb} and obtain:
\begin{align}
\sum^s_{m=0} a_m &= 1 - 1 - (1-\beta) (1-1) = 0, \\
\sum^s_{m=0} a_m\frac{m^1}{1!}+\sum^s_{m=0} b_m\frac{m^{0}}{0!} &= 0 - 1 - (1-\beta) (1-2) + \beta\sum^s_{m=0} b_m \nonumber \\
&= - \beta + \beta = 0.
\end{align}
Therefore, we have shown that the modification to the method has first-order convergence.
\end{proof}

Next, we turn our attention to the GHVB method.

\begin{theorem}[Convergence order of the GHVB method] \label{thm:ghvb}
The $r$\ts{th}-order GHVB (Algorithm \ref{algo:ghvb}) has order of convergence of $r$.
\end{theorem}

\begin{proof}
We will use the 2\ts{nd}-order method as an example. Using Equation \ref{2nd_bAdam}, we can write an equivalent equation as:
\begin{align}
x_{n+1} - x_n - (1-\beta) (x_n - x_{n-1}) = \delta \left(\frac{2+\beta}{2} f(x_{n}) - \frac{2-\beta}{2} f(x_{n-1})\right).
\end{align}
To estimate the order of the modified method, we evaluate Equation \ref{eq:sing_hb} and obtain:
\begin{align*}
\sum^s_{m=0}& a_m = 1 - 1 - (1-\beta) (1-1) = 0, \\
\sum^s_{m=0}& a_m\frac{m^1}{1!}+\sum^s_{m=0} b_m\frac{m^{0}}{0!} = 0 - 1 - (1-\beta) (1-2) + \left(\frac{2+\beta}{2} - \frac{2-\beta}{2}\right) = 0,\\
\sum^s_{m=0}& a_m\frac{m^2}{2!}+\sum^s_{m=0} b_m\frac{m^1}{1!} = \frac{1}{2}(0 - 1^2 - (1-\beta) (1^2-2^2)) + \left(\frac{2+\beta}{2}1^1 - \frac{2-\beta}{2}2^1\right) \\
&= \frac{2-3\beta}{2} - \frac{2-3\beta}{2} = 0.
\end{align*}
Thus, the method  has a convergence order of two. Methods of other orders can be dealt with in a similar fashion.
\end{proof}

\section{Qualitative Comparisons} \label{apx:sample}
Figure \ref{fig:highlighted_image} compares our momentum-based methods, HB and GHVB, with two different diffusion solver methods, DPM-Solver++ \cite{lu2022dpmpp} and PLMS4 \cite{liu2022pseudo}, without momentum. The number of sampling steps is held constant while varying the guidance scale $s$ to intentionally induce divergence artifacts. (Note that the guidance scales that yield such artifacts are different between diffusion models.) The figure demonstrates that, under the difficult settings of low step counts and high guidance scales where the baseline methods produce artifacts, our proposed techniques can successfully eliminate them. 


We present additional qualitative results to show the effect of the damping parameter $\beta$ on the quality of images generated by methods modified with HB momentum. We use methods of varying orders, including DPLM-Solver++\cite{lu2022dpmpp}, UniPC\cite{zhao2023unipc}, and PLMS4\cite{liu2022pseudo}. The diffusion models utilized in our analysis are Realistic Vision v2.0\footnote{\url{https://huggingface.co/SG161222/Realistic\textunderscore Vision\textunderscore V2.0}}, Anything Diffusion v4.0\footnote{\url{https://huggingface.co/andite/anything-v4.0}}, Counterfeit Diffusion V2.5\footnote{\url{https://huggingface.co/gsdf/Counterfeit-V2.5}}, Pastel-Mix\footnote{\url{https://huggingface.co/andite/pastel-mix}}, Deliberate Diffusion\footnote{\url{https://huggingface.co/XpucT/Deliberate}}, and Dreamlink Diffusion V1.0\footnote{\url{https://huggingface.co/dreamlike-art/dreamlike-diffusion-1.0}}. The results are shown in Figure \ref{fig:highlighted_image2}. Notice that stronger momentum (lower $\beta$) leads to fewer and less severe artifacts.

\begin{figure}
    \centering
    \subcaptionbox{DPM-Solver++(2M) \cite{lu2022dpmpp}, a 2\ts{nd}-order method, using 15 steps. Prompt:"a tiny cute bunny"}{
    \begin{tabu} to \textwidth {@{}l@{\hspace{5pt}}c@{\hspace{2pt}}c@{\hspace{6pt}}c@{\hspace{2pt}}c@{\hspace{6pt}}c@{\hspace{2pt}}c@{}}  
        
        
        \shortstack[c]{\tiny Without \\ \tiny Momentum} &
        \noindent\parbox[c]{0.12\columnwidth}{\includegraphics[width=0.12\columnwidth]{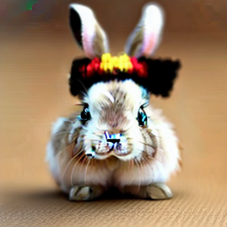}} & 
        \noindent\parbox[c]{0.12\columnwidth}{\includegraphics[width=0.12\columnwidth]{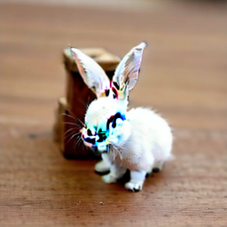}} &
        \noindent\parbox[c]{0.12\columnwidth}{\includegraphics[width=0.12\columnwidth]{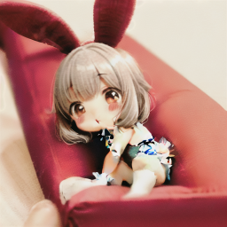}} &
        \noindent\parbox[c]{0.12\columnwidth}{\includegraphics[width=0.12\columnwidth]{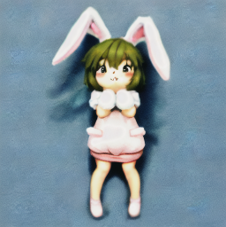}} &
        \noindent\parbox[c]{0.12\columnwidth}{\includegraphics[width=0.12\columnwidth]{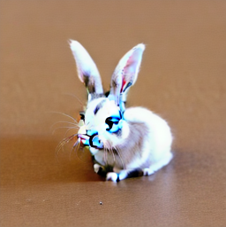}} &
        \noindent\parbox[c]{0.12\columnwidth}{\includegraphics[width=0.12\columnwidth]{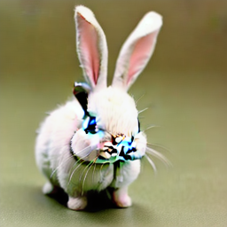}} \\

        \shortstack[c]{\tiny With \tiny HB 0.8 \\ \tiny (Ours)} &
        \noindent\parbox[c]{0.12\columnwidth}{\includegraphics[width=0.12\columnwidth]{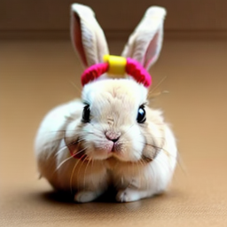}} & 
        \noindent\parbox[c]{0.12\columnwidth}{\includegraphics[width=0.12\columnwidth]{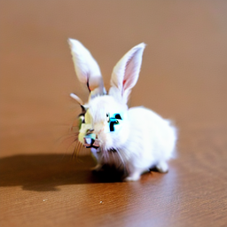}} &
        \noindent\parbox[c]{0.12\columnwidth}{\includegraphics[width=0.12\columnwidth]{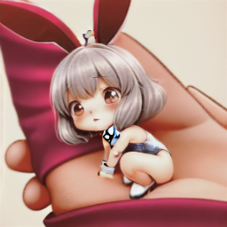}} &
        \noindent\parbox[c]{0.12\columnwidth}{\includegraphics[width=0.12\columnwidth]{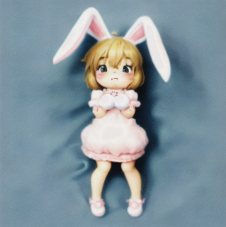}} &
        \noindent\parbox[c]{0.12\columnwidth}{\includegraphics[width=0.12\columnwidth]{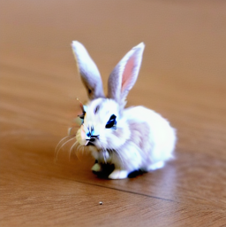}} &
        \noindent\parbox[c]{0.12\columnwidth}{\includegraphics[width=0.12\columnwidth]{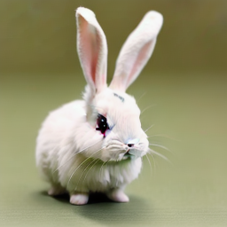}} \\

        \shortstack[c]{\tiny With \tiny HB 0.5 \\ \tiny (Ours)} &
        \noindent\parbox[c]{0.12\columnwidth}{\includegraphics[width=0.12\columnwidth]{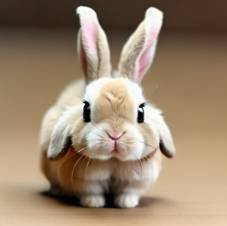}} & 
        \noindent\parbox[c]{0.12\columnwidth}{\includegraphics[width=0.12\columnwidth]{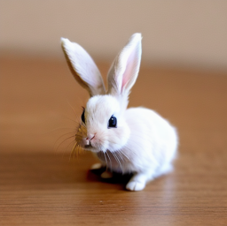}} &
        \noindent\parbox[c]{0.12\columnwidth}{\includegraphics[width=0.12\columnwidth]{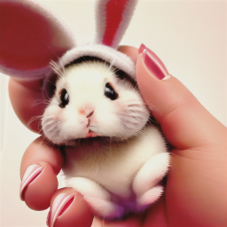}} &
        \noindent\parbox[c]{0.12\columnwidth}{\includegraphics[width=0.12\columnwidth]{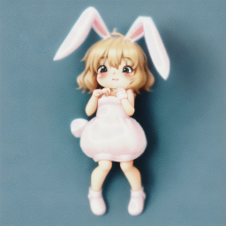}} &
        \noindent\parbox[c]{0.12\columnwidth}{\includegraphics[width=0.12\columnwidth]{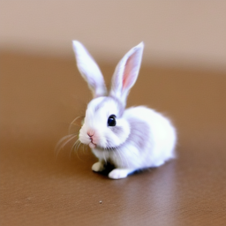}} &
        \noindent\parbox[c]{0.12\columnwidth}{\includegraphics[width=0.12\columnwidth]{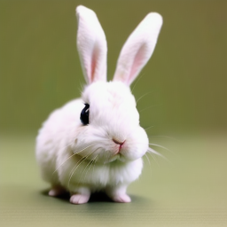}} \\
        
        & \multicolumn{2}{c}{\tiny Realistic Vision V2.0. ($s = 16$)} & \multicolumn{2}{c}{\tiny Anything V4.0 ($s=20$)} & \multicolumn{2}{c}{\tiny Deliberate ($s=16$)}
        
    \end{tabu}
    }
    \subcaptionbox{UniPC \cite{zhao2023unipc} method, a 3\ts{rd}-order method, using 8 steps. Prompt: "a cute kitty "}{
    \begin{tabu} to \textwidth {@{}l@{\hspace{5pt}}c@{\hspace{2pt}}c@{\hspace{6pt}}c@{\hspace{2pt}}c@{\hspace{6pt}}c@{\hspace{2pt}}c@{}}

        \shortstack[c]{\tiny Without \\ \tiny Momentum} &
        \noindent\parbox[c]{0.12\columnwidth}{\includegraphics[width=0.12\columnwidth]{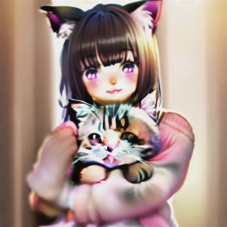}} & 
        \noindent\parbox[c]{0.12\columnwidth}{\includegraphics[width=0.12\columnwidth]{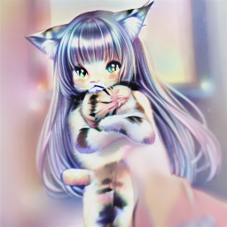}} &
        \noindent\parbox[c]{0.12\columnwidth}{\includegraphics[width=0.12\columnwidth]{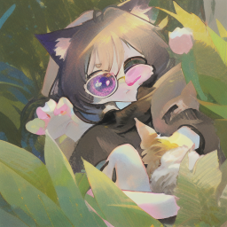}} &
        \noindent\parbox[c]{0.12\columnwidth}{\includegraphics[width=0.12\columnwidth]{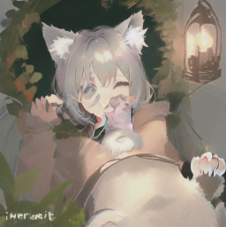}} &
        \noindent\parbox[c]{0.12\columnwidth}{\includegraphics[width=0.12\columnwidth]{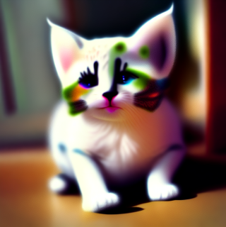}} &
        \noindent\parbox[c]{0.12\columnwidth}{\includegraphics[width=0.12\columnwidth]{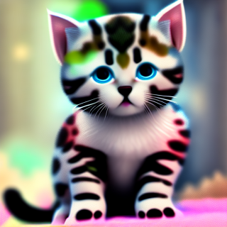}} \\

        \shortstack[c]{\tiny With \tiny HB 0.8 \\ \tiny (Ours)} &
        \noindent\parbox[c]{0.12\columnwidth}{\includegraphics[width=0.12\columnwidth]{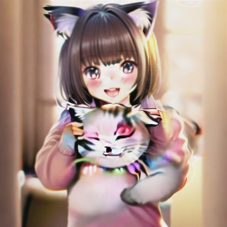}} & 
        \noindent\parbox[c]{0.12\columnwidth}{\includegraphics[width=0.12\columnwidth]{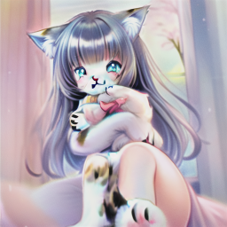}} &
        \noindent\parbox[c]{0.12\columnwidth}{\includegraphics[width=0.12\columnwidth]{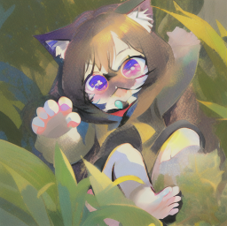}} &
        \noindent\parbox[c]{0.12\columnwidth}{\includegraphics[width=0.12\columnwidth]{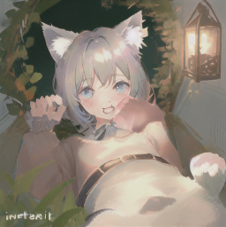}} &
        \noindent\parbox[c]{0.12\columnwidth}{\includegraphics[width=0.12\columnwidth]{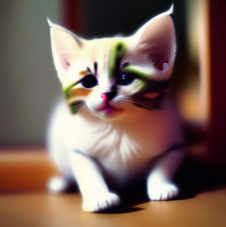}} &
        \noindent\parbox[c]{0.12\columnwidth}{\includegraphics[width=0.12\columnwidth]{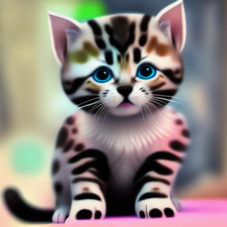}} \\

        \shortstack[c]{\tiny With \tiny HB 0.5 \\ \tiny (Ours)} &
        \noindent\parbox[c]{0.12\columnwidth}{\includegraphics[width=0.12\columnwidth]{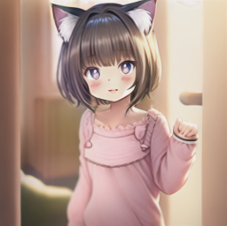}} & 
        \noindent\parbox[c]{0.12\columnwidth}{\includegraphics[width=0.12\columnwidth]{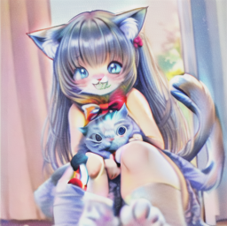}} &
        \noindent\parbox[c]{0.12\columnwidth}{\includegraphics[width=0.12\columnwidth]{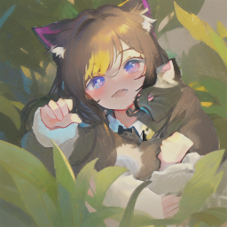}} &
        \noindent\parbox[c]{0.12\columnwidth}{\includegraphics[width=0.12\columnwidth]{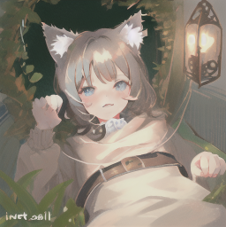}} &
        \noindent\parbox[c]{0.12\columnwidth}{\includegraphics[width=0.12\columnwidth]{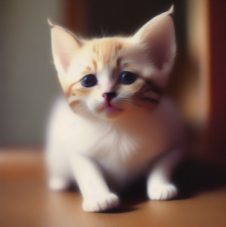}} &
        \noindent\parbox[c]{0.12\columnwidth}{\includegraphics[width=0.12\columnwidth]{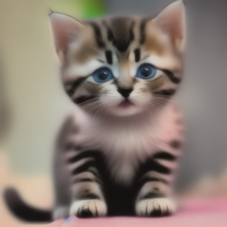}} \\
        
         & \multicolumn{2}{c}{\tiny Anything V4.0 ($s = 18$)} & \multicolumn{2}{c}{\tiny Pastel-mix V4.0 ($s=14$)} & \multicolumn{2}{c}{\tiny Dreamlike V1.0 ($s=10$)}
    \end{tabu}
    }
    \subcaptionbox{PLMS4 \cite{liu2022pseudo}, a 4\ts{th}-order method, using 15 steps, Prompt: "cute humanoid red panda"}{
    \begin{tabu} to \textwidth {@{}l@{\hspace{5pt}}c@{\hspace{2pt}}c@{\hspace{6pt}}c@{\hspace{2pt}}c@{\hspace{6pt}}c@{\hspace{2pt}}c@{}}

        \shortstack[c]{\tiny Without \\ \tiny Momentum} &
        \noindent\parbox[c]{0.12\columnwidth}{\includegraphics[width=0.12\columnwidth]{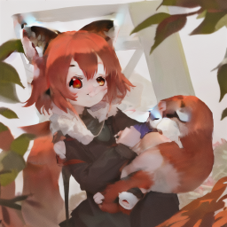}} & 
        \noindent\parbox[c]{0.12\columnwidth}{\includegraphics[width=0.12\columnwidth]{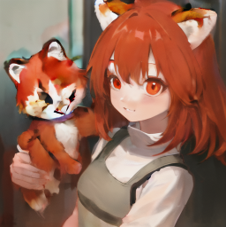}} &
        \noindent\parbox[c]{0.12\columnwidth}{\includegraphics[width=0.12\columnwidth]{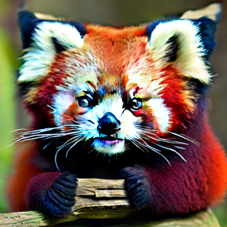}} &
        \noindent\parbox[c]{0.12\columnwidth}{\includegraphics[width=0.12\columnwidth]{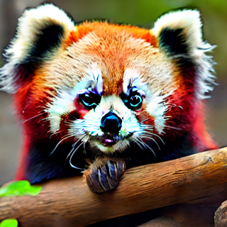}} &
        \noindent\parbox[c]{0.12\columnwidth}{\includegraphics[width=0.12\columnwidth]{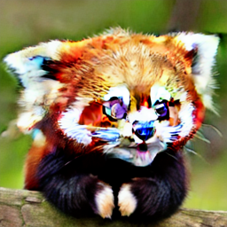}} &
        \noindent\parbox[c]{0.12\columnwidth}{\includegraphics[width=0.12\columnwidth]{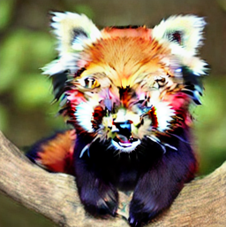}} \\

        \shortstack[c]{\tiny With \tiny HB 0.8 \\ \tiny (Ours)} &
        \noindent\parbox[c]{0.12\columnwidth}{\includegraphics[width=0.12\columnwidth]{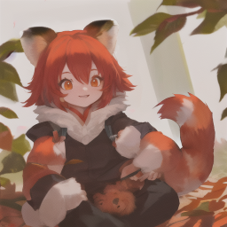}} & 
        \noindent\parbox[c]{0.12\columnwidth}{\includegraphics[width=0.12\columnwidth]{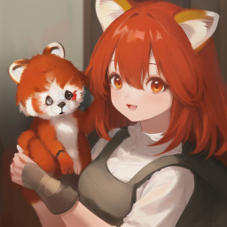}} &
        \noindent\parbox[c]{0.12\columnwidth}{\includegraphics[width=0.12\columnwidth]{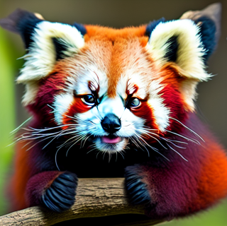}} &
        \noindent\parbox[c]{0.12\columnwidth}{\includegraphics[width=0.12\columnwidth]{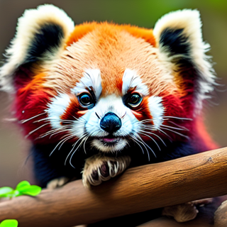}} &
        \noindent\parbox[c]{0.12\columnwidth}{\includegraphics[width=0.12\columnwidth]{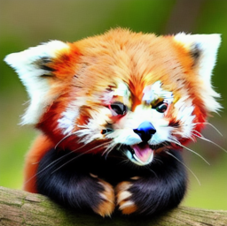}} &
        \noindent\parbox[c]{0.12\columnwidth}{\includegraphics[width=0.12\columnwidth]{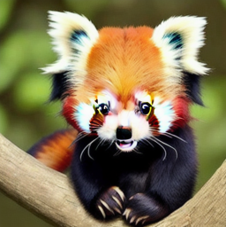}} \\

        \shortstack[c]{\tiny With \tiny HB 0.5 \\ \tiny (Ours)} &
        \noindent\parbox[c]{0.12\columnwidth}{\includegraphics[width=0.12\columnwidth]{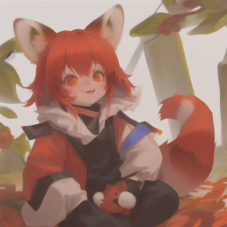}} & 
        \noindent\parbox[c]{0.12\columnwidth}{\includegraphics[width=0.12\columnwidth]{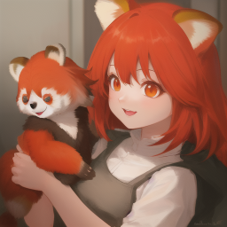}} &
        \noindent\parbox[c]{0.12\columnwidth}{\includegraphics[width=0.12\columnwidth]{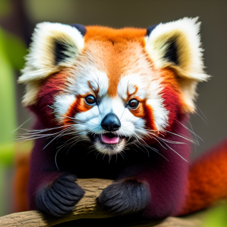}} &
        \noindent\parbox[c]{0.12\columnwidth}{\includegraphics[width=0.12\columnwidth]{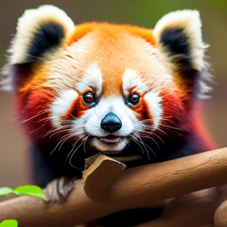}} &
        \noindent\parbox[c]{0.12\columnwidth}{\includegraphics[width=0.12\columnwidth]{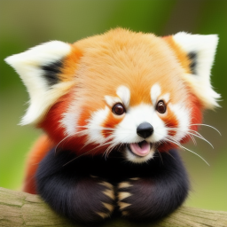}} &
        \noindent\parbox[c]{0.12\columnwidth}{\includegraphics[width=0.12\columnwidth]{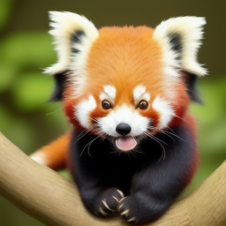}} \\
        
        & \multicolumn{2}{c}{\tiny Counterfeit V2.5 ($s = 7$)} & \multicolumn{2}{c}{\tiny Dreamlike V1.0 ($s=7$)} & \multicolumn{2}{c}{\tiny Deliberate ($s = 7$)}
    \end{tabu}
    }
    \caption{The impact of different damping coefficients $\beta$ on HB momentum for 2\ts{nd}-order (a), 3\ts{rd}-order (b), and 4\ts{th}-order (c) numerical methods. Notably, it is observed that incorporating higher momentum values (lower $\beta$) helps mitigate the occurrence of divergence artifacts.}
    \label{fig:highlighted_image2}
\end{figure}

\section{Experimental Details and Results of ADM} \label{apx:adm}

In this section, we present additional details and results for the ADM experiment in Section \ref{sec:exp_adm}. The primary objective of this experiment was to provide a quantitative evaluation of class-conditional diffusion sampling in the context of pixel-based images. The experiment was conducted using the pre-trained diffusion and classifier model at the following link: \footnote{\url{https://github.com/openai/guided-diffusion}}. The implementation used in our experiment was obtained directly from the official DPM-Solver GitHub repository \footnote{\url{https://github.com/LuChengTHU/dpm-solver}}.

To enhance the capabilities of DPM-Solver++, we incorporated HB momentum into DPM-Solver++ (just change a few lines of code) and implemented the splitting method LTSP both with and without HB momentum into the DPM-Solver code for comparative purposes. The experiment was done on four NVIDIA RTX A4000 GPUs and a 24-core AMD Threadripper 3960x CPU.

The results of the experiment, as measured by the full FID score, are presented in Table \ref{tab:results1} (as well as in Figrue \ref{fig:fid_adm}). Our technique is highlighted in grey within the table, while the best FID score for each number of steps is indicated in bold. Additionally, Figure \ref{fig:img_adm} showcases sample images from this experiment. 
As this experiment uses pixel-based diffusion models, the observed divergence artifacts differ from those in latent-based diffusion models. Specifically, the artifacts may display excessive brightness or darkness caused by pixel values nearing the maximum or minimum thresholds.


\begin{table}[ht] 
\centering 
\begin{tabular}{lrrrrrrrr} 
\toprule
&\multicolumn{4}{c}{Number of Steps} \\
  Method & 10 & 12 & 14 & 16 & 18 & 20 & 25 & 30 \\
\midrule
\whitebox{DPM-Solver++} & 66.77 & 46.77 & 34.56 & 26.97 & 21.87 & 19.48 & 16.31 & 15.63 \\
\graybox{DPM-Solver++ w/ HB 0.8} & 47.10 & 33.65 & 25.61 & 21.42 & 18.94 & 17.76 & 15.98 & 15.53 \\
\hline
\whitebox{LTSP [PLMS4, PLMS1]}  & 45.32 & 34.08 & 26.58 & 21.54 & 18.54 & 17.15 & \textbf{15.79} & \textbf{15.51} \\
\graybox{LTSP [PLMS4, GHVB 0.8]}  & \textbf{37.43} &\textbf{ 29.74} & \textbf{23.60} & \textbf{20.23} & \textbf{18.46} & \textbf{17.14} & 16.07 & 15.79 \\
\bottomrule
\end{tabular}
\vspace{0.1cm}
\caption{FID scores on classifier-guidance ADM models}
\label{tab:results1}
\end{table}

\begin{figure}[ht]
    \centering
    \begin{tabular}{cccc}
        DPM-Solver++\cite{lu2022dpmpp} & DPM-Solver++ & LTSP \cite{wizadwongsa2023accelerating} & LTSP\\
                     & with HB 0.8  & [PLMS4, PLMS1] & [PLMS4, GHVB 0.8]\\
                     &(Ours)  &       & (Ours))\\
        \includegraphics[width=0.20\linewidth]{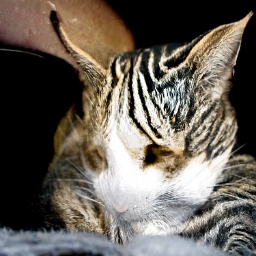} &
        \includegraphics[width=0.20\linewidth]{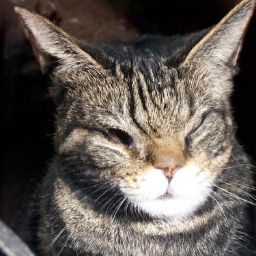} &
        \includegraphics[width=0.20\linewidth]{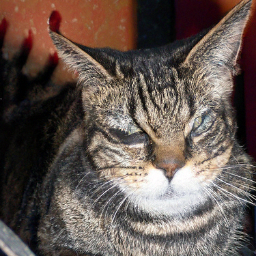} &
        \includegraphics[width=0.20\linewidth]{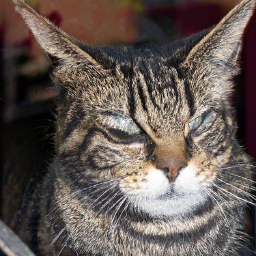} \\
        \includegraphics[width=0.20\linewidth]{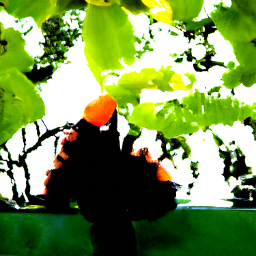} &
        \includegraphics[width=0.20\linewidth]{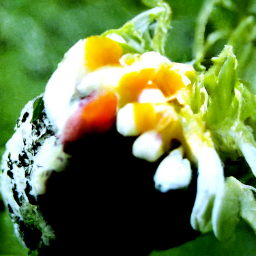} &
        \includegraphics[width=0.20\linewidth]{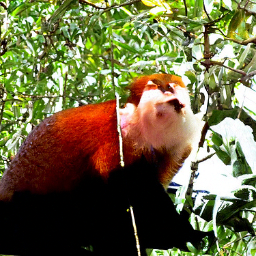} &
        \includegraphics[width=0.20\linewidth]{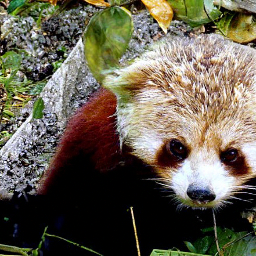} \\
        \includegraphics[width=0.20\linewidth]{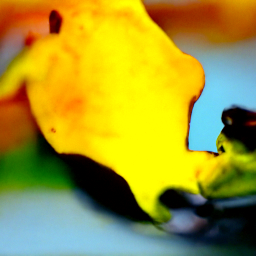} &
        \includegraphics[width=0.20\linewidth]{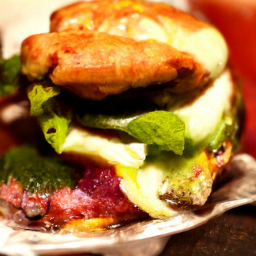} &
        \includegraphics[width=0.20\linewidth]{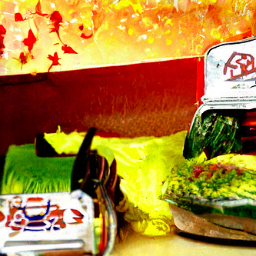} &
        \includegraphics[width=0.20\linewidth]{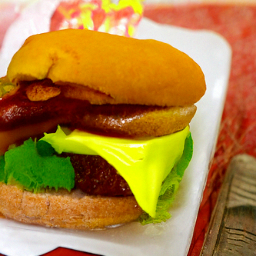} \\
        \includegraphics[width=0.20\linewidth]{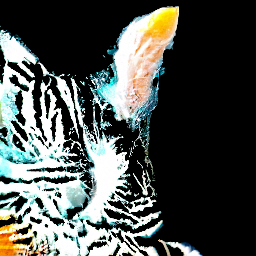} &
        \includegraphics[width=0.20\linewidth]{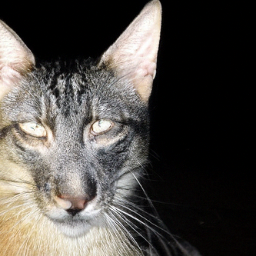} &
        \includegraphics[width=0.20\linewidth]{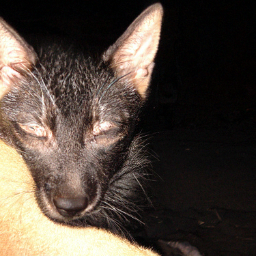} &
        \includegraphics[width=0.20\linewidth]{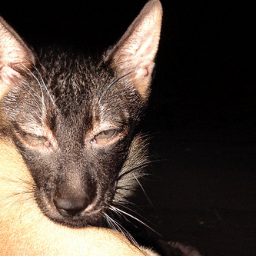} \\
        \includegraphics[width=0.20\linewidth]{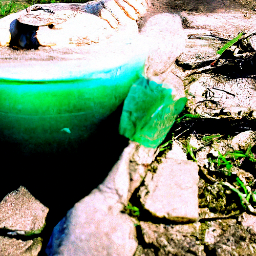} &
        \includegraphics[width=0.20\linewidth]{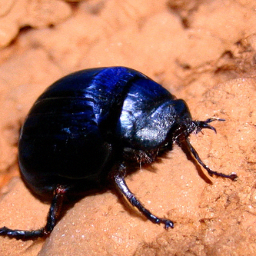} &
        \includegraphics[width=0.20\linewidth]{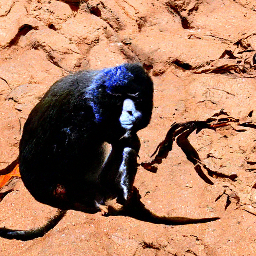} &
        \includegraphics[width=0.20\linewidth]{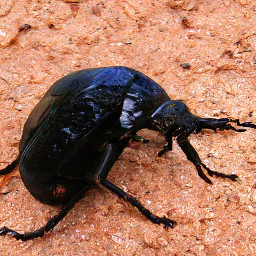} \\
        \includegraphics[width=0.20\linewidth]{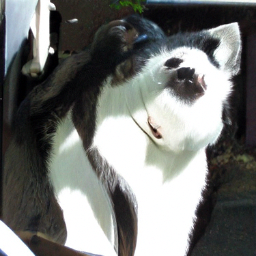} &
        \includegraphics[width=0.20\linewidth]{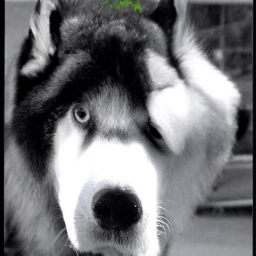} &
        \includegraphics[width=0.20\linewidth]{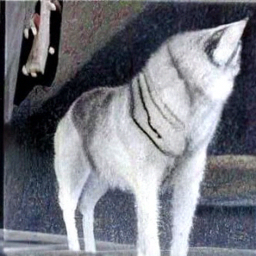} &
        \includegraphics[width=0.20\linewidth]{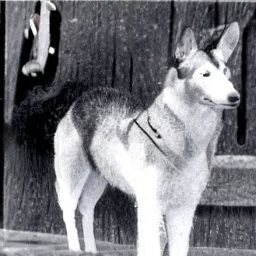} \\
        \includegraphics[width=0.20\linewidth]{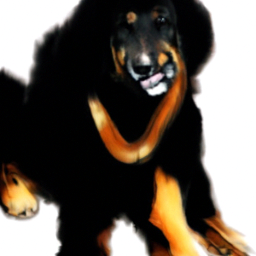} &
        \includegraphics[width=0.20\linewidth]{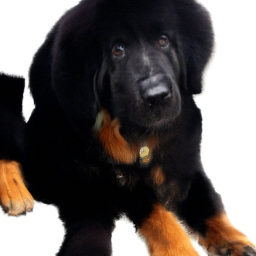} &
        \includegraphics[width=0.20\linewidth]{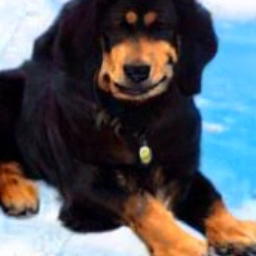} &
        \includegraphics[width=0.20\linewidth]{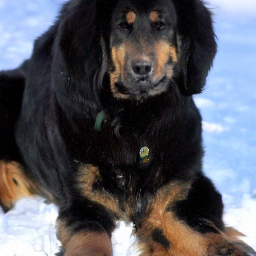} \\
    \end{tabular}
    \caption{Samples from various sampling methods employing classifier guidance diffusion with a guidance scale of 10 and 20 sampling steps.}
    \label{fig:img_adm}
\end{figure}

\section{Experimental Details and Results of DiT} \label{apx:dit}

This section presents supplementary details of the DiT experiment conducted in Figure \ref{fig:fid_dit} of Section \ref{sec:exp_dit} and further investigates the performance of the DiT model \cite{peebles2022scalable}. The code implementation and pre-trained DiT model were obtained directly from the official GitHub repository\footnote{\url{https://github.com/facebookresearch/DiT}}. The experiment was done on four NVIDIA GeForce RTX 2080 Ti GPUs and a 24-core AMD Threadripper 3960x CPU.

Our baselines include DDIM \cite{song2020denoising}, DPM-Solver++ \cite{lu2022dpmpp}, LTSP4 \cite{wizadwongsa2023accelerating}, and PLMS4 \cite{liu2022pseudo}. Based on the FID score, PLMS4 emerged as the most effective sampling method within the chosen context. As a result, only PLMS4 was included in Section \ref{sec:exp_dit} of our main paper. Our variations of the PLMS4 with HB 0.8 and 0.9, as well as GHVB 3.8 and 3.9. The results of the experiment are presented in Table \ref{tab:results}. Our technique is highlighted in grey in the table, and the best FID score for each number of steps is in bold. Notably, our method outperforms the others.

Moreover, we include the improved Precision and Recall metrics \cite{kynkaanniemi2019improved} in Tables \ref{tab:precision_dit} and \ref{tab:recall_dit}, respectively, where higher values indicate superior performance. Additionally, Figure \ref{fig:img_dit} displays sample images generated from different sampling methods.

\begin{table}[ht]
\centering
\begin{tabular}{lrrrrrrrr}
\toprule
&\multicolumn{8}{c}{Number of Steps}\\
Method & 6 & 7 & 8 & 9 & 10 & 15 & 20 & 25 \\
\midrule
\whitebox{DDIM} & 55.35 & 36.97 & 26.06 & 19.47 & 15.02 & 8.04 & 6.52 & 5.94 \\
\whitebox{DPM-Solver++}       & 18.60 & 10.80 & 7.93 & 6.72 & 6.13 & 5.49 & 5.30 & 5.24 \\
\whitebox{LTSP4 [PLMS4, PLMS1]} & 13.33 & 9.01  & 7.49 & 6.55 & 6.09 & 5.32 & 5.20 & \textbf{5.17} \\
\whitebox{PLMS4} & 13.10 & 8.94 & 7.31 & 6.51 & 6.03 & 5.32 & 5.21 & \textbf{5.17} \\
\hline
\graybox{PLMS4 w/ HB 0.8}
& 14.35 & 9.25 & 7.46 & 6.68 & 6.19 & 5.47 & 5.29 & 5.24 \\
\graybox{PLMS4 w/ HB 0.9} & 11.66 & 8.16 & 6.69 & 6.21 &\textbf{ 5.78} & \textbf{5.29} & \textbf{5.19} &\textbf{5.17} \\
\graybox{GHVB 3.8} &\textbf{ 10.99} & \textbf{7.93} & \textbf{6.63} & \textbf{6.19} & 5.80 & 5.31 & 5.22 & 5.18 \\
\graybox{GHVB 3.9}  & 11.67 & 8.29 & 6.83 & 6.30 & 5.87 & 5.31 & 5.22 & 5.18 \\
\bottomrule
\end{tabular}
\caption{FID scores on DiT-XL}
\label{tab:results}
\end{table}

\begin{table}[ht]
  \begin{minipage}{0.45\textwidth}
    \centering
    \begin{tabular}{lcccc}
      \toprule
      & \multicolumn{4}{c}{Number of Steps} \\
      Method & 6 & 8 & 10 & 20 \\
      \midrule
      \whitebox{DDIM} & 0.36 & 0.56 & 0.67 & 0.79 \\
      \whitebox{DPM-Solver++}  &    0.63 & 0.75 & 0.79 & \textbf{0.81} \\
      \whitebox{LTSP4} & 0.67 & 0.74 & 0.78 & \textbf{0.81} \\
      \whitebox{PLMS4} & 0.68 & 0.75 & 0.78 & \textbf{0.81} \\
\hline
      \graybox{PLMS4 w/ HB 0.8} & 0.68 & \textbf{0.77} & \textbf{0.79} & \textbf{0.81} \\
      \graybox{PLMS4 w/ HB 0.9} & 0.70 & \textbf{0.77} & \textbf{0.79} & \textbf{0.81} \\
      \graybox{GHVB3.8} & \textbf{0.71} & \textbf{0.77} & \textbf{0.79} & \textbf{0.81} \\
      \graybox{GHVB3.9} & 0.70 & 0.76 & 0.78 & \textbf{0.81} \\
      \bottomrule
    \end{tabular}
    \caption{Precision on DiT-XL} \label{tab:precision_dit}
  \end{minipage}\hfill
  \begin{minipage}{0.45\textwidth}
    \centering
    \begin{tabular}{lcccc}
      \toprule
      & \multicolumn{4}{c}{Number of Steps} \\
      Method & 6 & 8 & 10 & 20 \\
      \midrule
      \whitebox{DDIM} & 0.60 & 0.64 & 0.65 & 0.67 \\
      \whitebox{DPM-Solver++}    &   0.67 & 0.68 & 0.68 & \textbf{0.68} \\
      \whitebox{LTSP4} & \textbf{0.70} & \textbf{0.70} & \textbf{0.69} & \textbf{0.68} \\
      \whitebox{PLMS4} & \textbf{0.70} & \textbf{0.70} & \textbf{0.69} & \textbf{0.68} \\
\hline
      \graybox{PLMS4 w/ HB 0.8} & 0.68 & 0.68 & 0.67 & \textbf{0.68} \\
      \graybox{PLMS4 w/ HB 0.9} & 0.69 & 0.69 & 0.67 & \textbf{0.68} \\
      \graybox{GHVB3.8} & 0.69 & \textbf{0.70} & 0.68 & \textbf{0.68} \\
      \graybox{GHVB3.9} & \textbf{0.70} & \textbf{0.70} & \textbf{0.69} & \textbf{0.68} \\
      \bottomrule
    \end{tabular}
    \caption{Recall on DiT-XL} \label{tab:recall_dit}
  \end{minipage}
\end{table}

\tabulinesep=1pt
\begin{figure}
    \centering
    \begin{tabu} to \textwidth {@{}l@{\hspace{5pt}}c@{\hspace{2pt}}c@{\hspace{2pt}}c@{\hspace{2pt}}c@{\hspace{2pt}}c@{}}

        \multicolumn{1}{l}{\shortstack[l]{\scriptsize Number of steps}}
        & \multicolumn{1}{c}{\shortstack{\scriptsize 6}}
        & \multicolumn{1}{c}{\shortstack{\scriptsize 8}}
        & \multicolumn{1}{c}{\shortstack{\scriptsize 10}}
        & \multicolumn{1}{c}{\shortstack{\scriptsize 15}}
        & \multicolumn{1}{c}{\shortstack{\scriptsize 25}}
        \\
        
        \shortstack[l]{\tiny PLMS4 \cite{liu2022pseudo}} &
        \noindent\parbox[c]{0.16\columnwidth}{\includegraphics[width=0.16\textwidth]{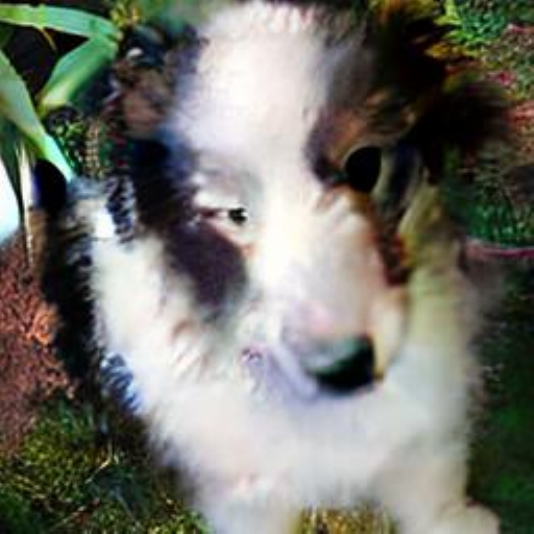}} & 
        \noindent\parbox[c]{0.16\columnwidth}{\includegraphics[width=0.16\textwidth]{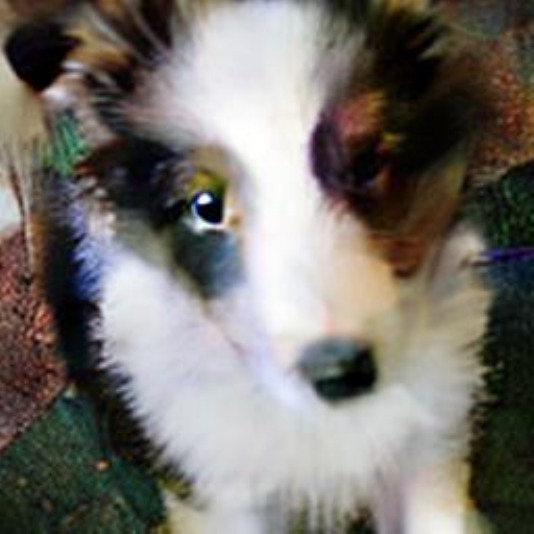}} & 
        \noindent\parbox[c]{0.16\columnwidth}{\includegraphics[width=0.16\textwidth]{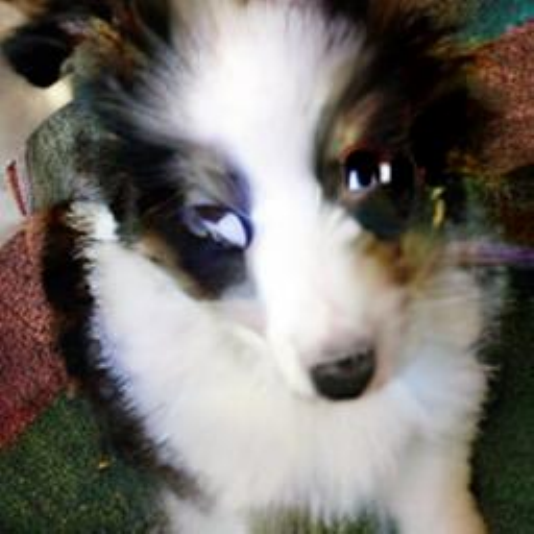}} & 
        \noindent\parbox[c]{0.16\columnwidth}{\includegraphics[width=0.16\textwidth]{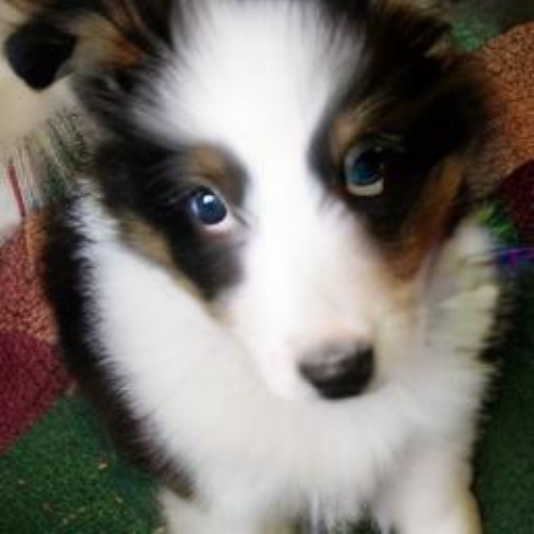}} & 
        \noindent\parbox[c]{0.16\columnwidth}{\includegraphics[width=0.16\textwidth]{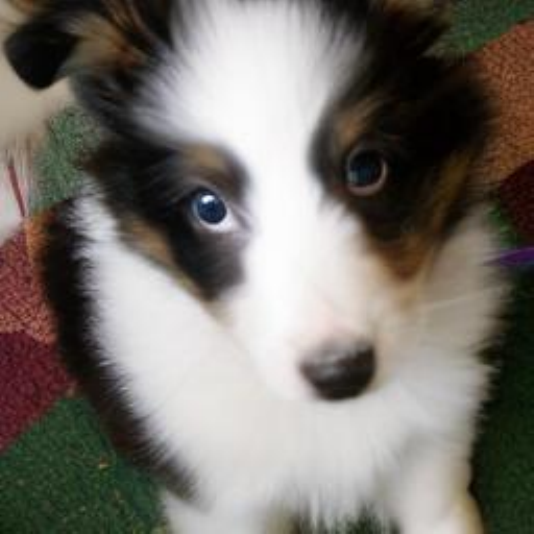}} \\

        \shortstack[l]{\tiny PLMS4 w/ HB 0.9} &
        \noindent\parbox[c]{0.16\columnwidth}{\includegraphics[width=0.16\textwidth]{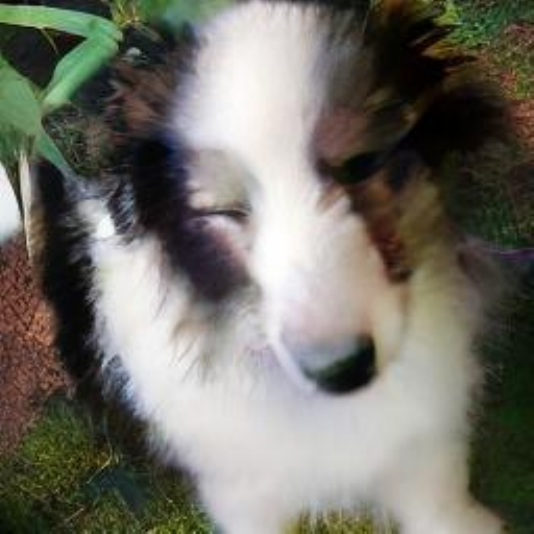}} & 
        \noindent\parbox[c]{0.16\columnwidth}{\includegraphics[width=0.16\textwidth]{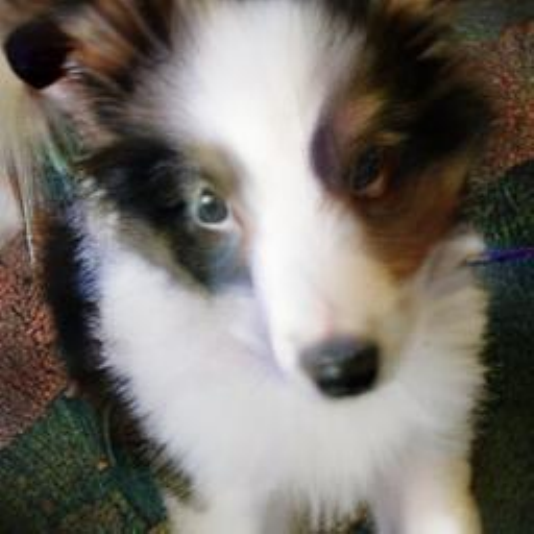}} & 
        \noindent\parbox[c]{0.16\columnwidth}{\includegraphics[width=0.16\textwidth]{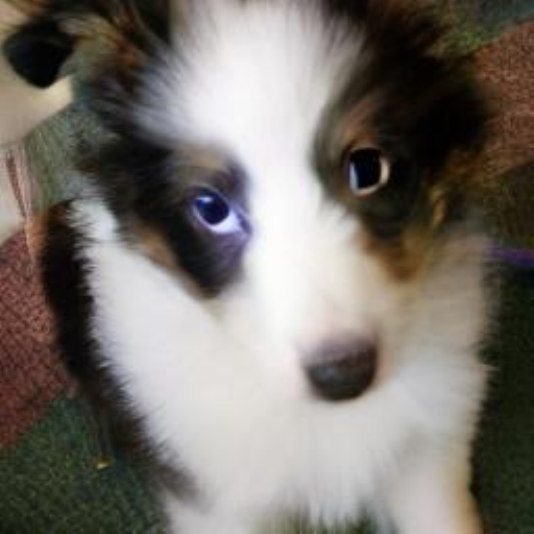}} & 
        \noindent\parbox[c]{0.16\columnwidth}{\includegraphics[width=0.16\textwidth]{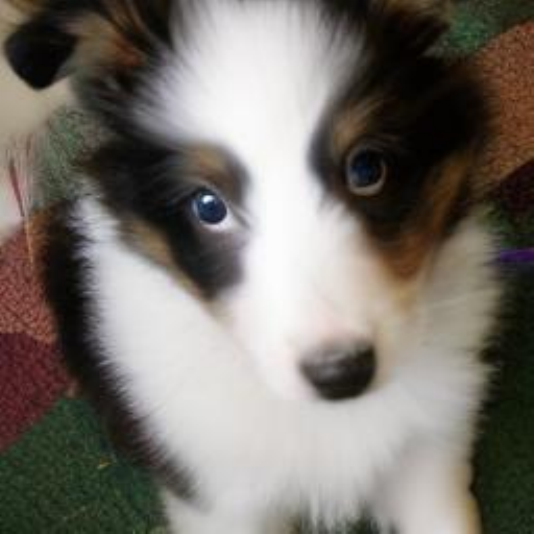}} & 
        \noindent\parbox[c]{0.16\columnwidth}{\includegraphics[width=0.16\textwidth]{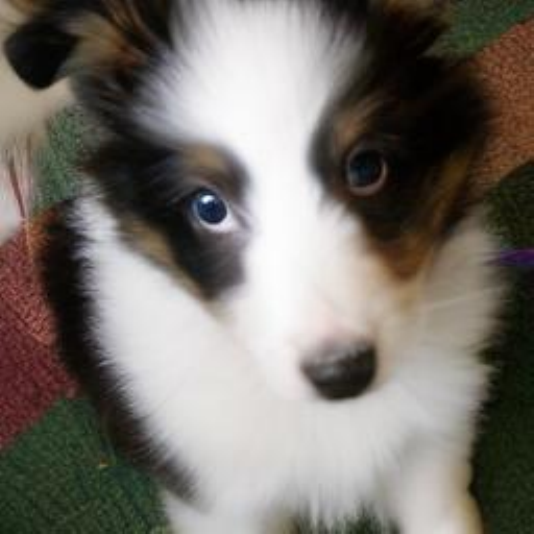}} \\

        \shortstack[l]{\tiny GHVB3.8} &
        \noindent\parbox[c]{0.16\columnwidth}{\includegraphics[width=0.16\textwidth]{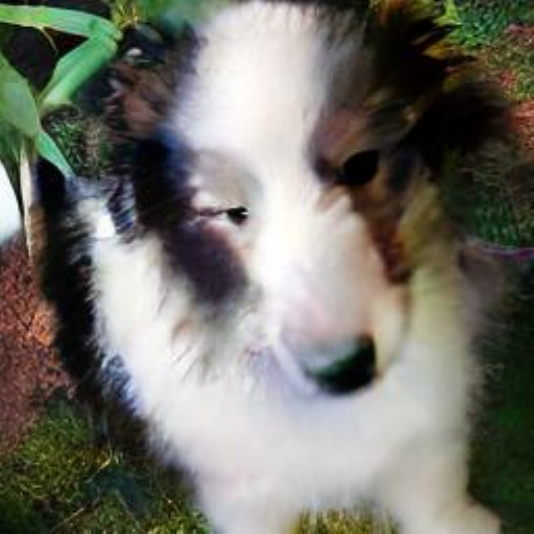}} & 
        \noindent\parbox[c]{0.16\columnwidth}{\includegraphics[width=0.16\textwidth]{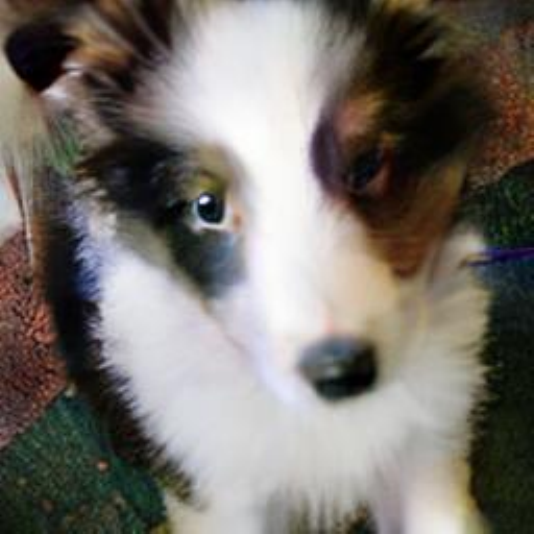}} & 
        \noindent\parbox[c]{0.16\columnwidth}{\includegraphics[width=0.16\textwidth]{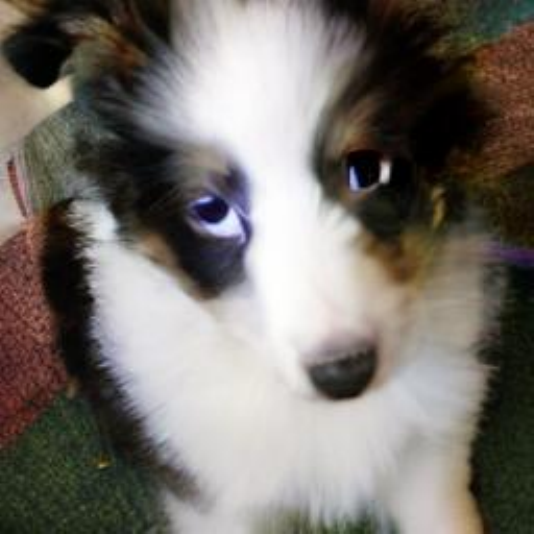}} & 
        \noindent\parbox[c]{0.16\columnwidth}{\includegraphics[width=0.16\textwidth]{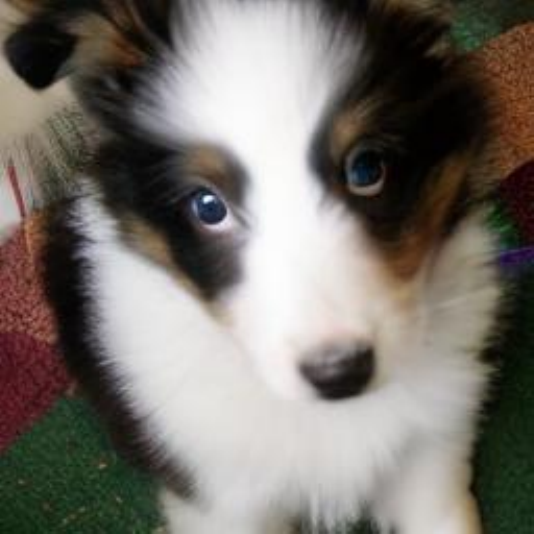}} & 
        \noindent\parbox[c]{0.16\columnwidth}{\includegraphics[width=0.16\textwidth]{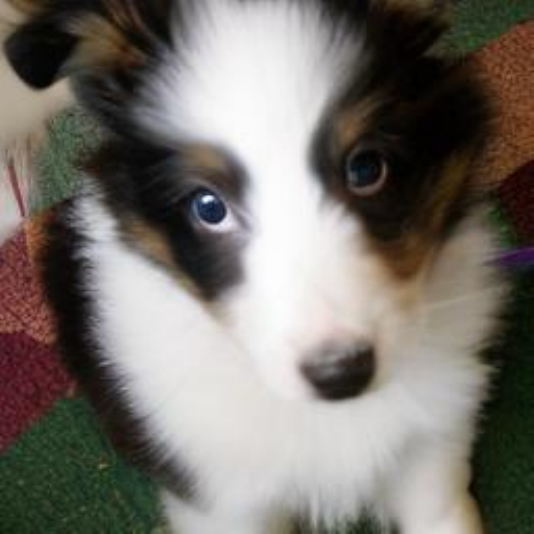}} \\
        
    \end{tabu}

    \medskip

    \begin{tabu} to \textwidth {@{}l@{\hspace{5pt}}c@{\hspace{2pt}}c@{\hspace{2pt}}c@{\hspace{2pt}}c@{\hspace{2pt}}c@{}}
        
        \shortstack[l]{\tiny PLMS4 \cite{liu2022pseudo}} &
        \noindent\parbox[c]{0.16\columnwidth}{\includegraphics[width=0.16\textwidth]{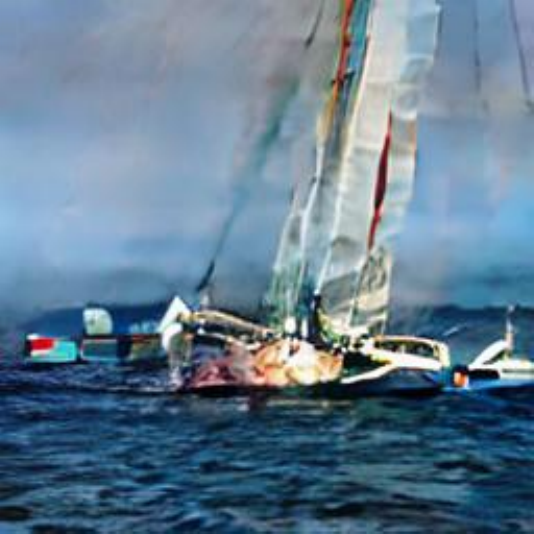}} & 
        \noindent\parbox[c]{0.16\columnwidth}{\includegraphics[width=0.16\textwidth]{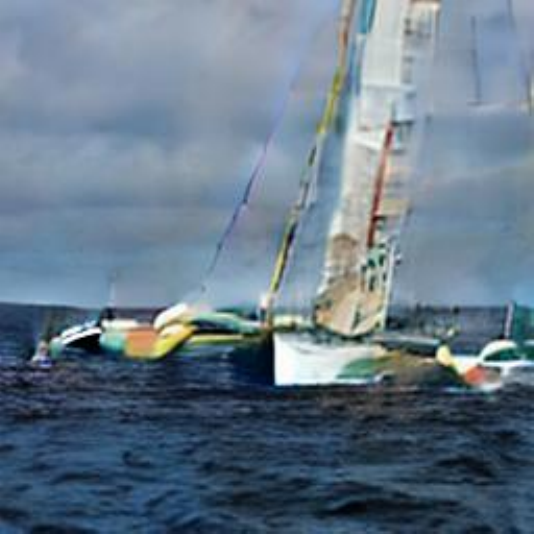}} & 
        \noindent\parbox[c]{0.16\columnwidth}{\includegraphics[width=0.16\textwidth]{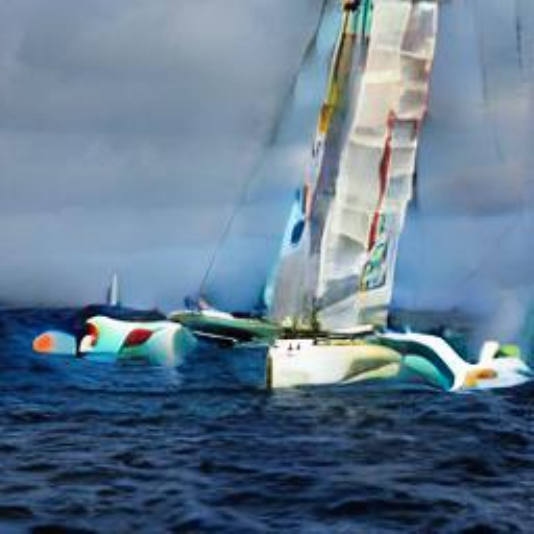}} & 
        \noindent\parbox[c]{0.16\columnwidth}{\includegraphics[width=0.16\textwidth]{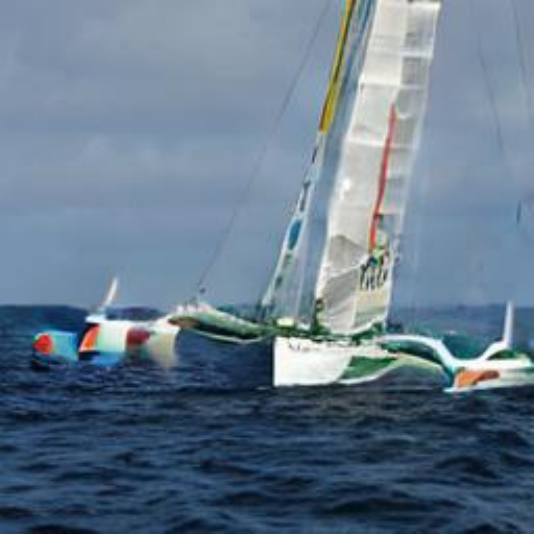}} & 
        \noindent\parbox[c]{0.16\columnwidth}{\includegraphics[width=0.16\textwidth]{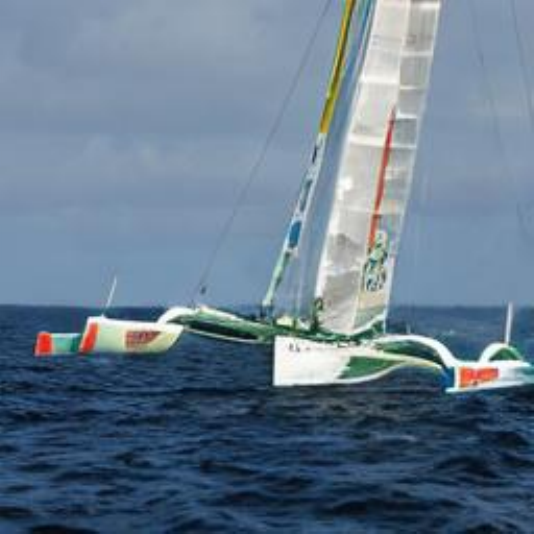}} \\

        \shortstack[l]{\tiny PLMS4 w/ HB 0.9} &
        \noindent\parbox[c]{0.16\columnwidth}{\includegraphics[width=0.16\textwidth]{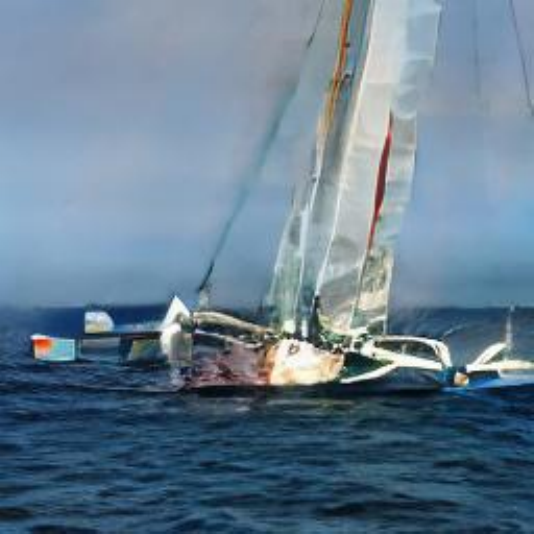}} & 
        \noindent\parbox[c]{0.16\columnwidth}{\includegraphics[width=0.16\textwidth]{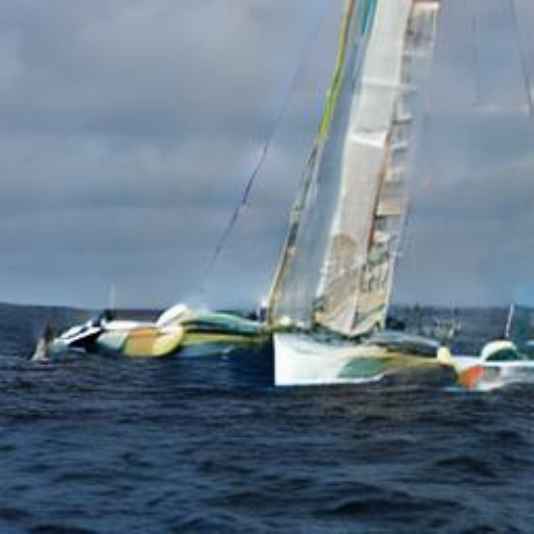}} & 
        \noindent\parbox[c]{0.16\columnwidth}{\includegraphics[width=0.16\textwidth]{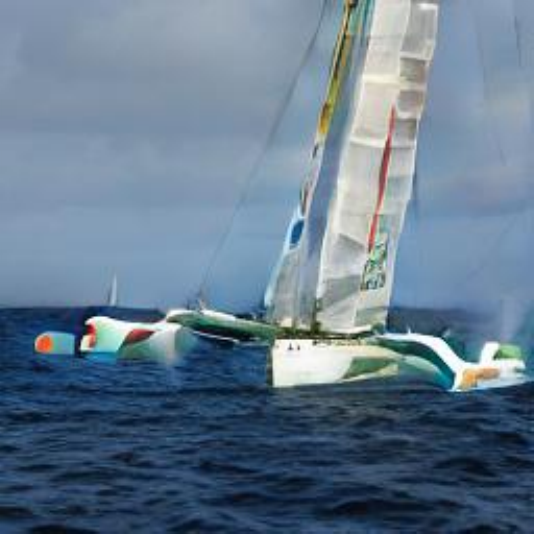}} & 
        \noindent\parbox[c]{0.16\columnwidth}{\includegraphics[width=0.16\textwidth]{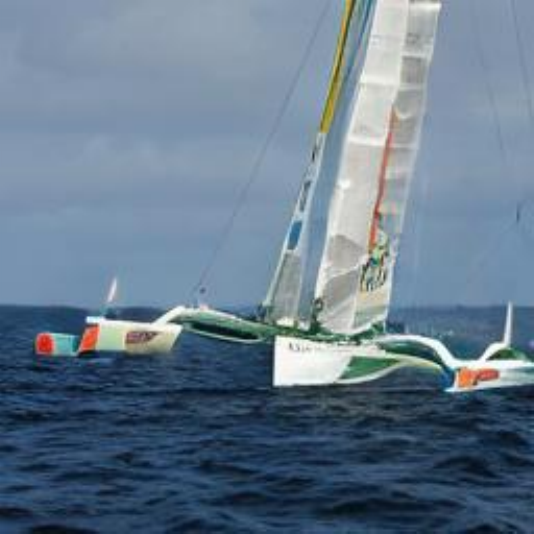}} & 
        \noindent\parbox[c]{0.16\columnwidth}{\includegraphics[width=0.16\textwidth]{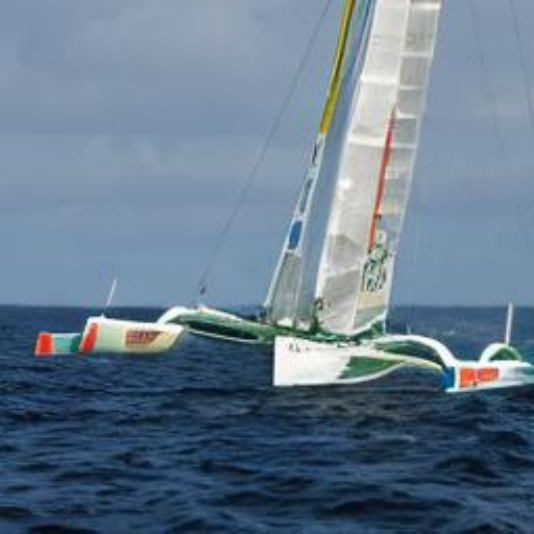}} \\

        \shortstack[l]{\tiny GHVB3.8} &
        \noindent\parbox[c]{0.16\columnwidth}{\includegraphics[width=0.16\textwidth]{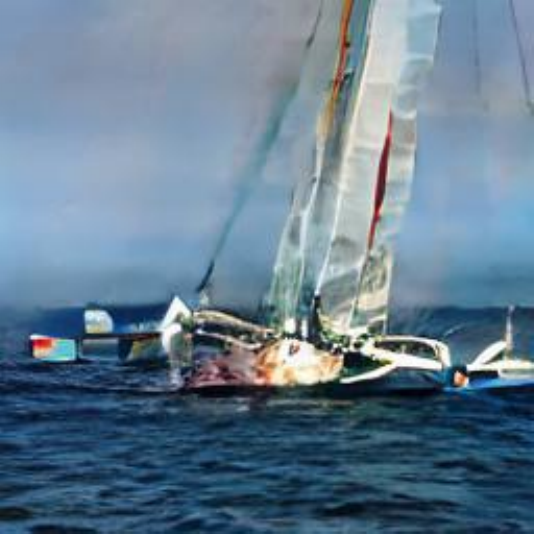}} & 
        \noindent\parbox[c]{0.16\columnwidth}{\includegraphics[width=0.16\textwidth]{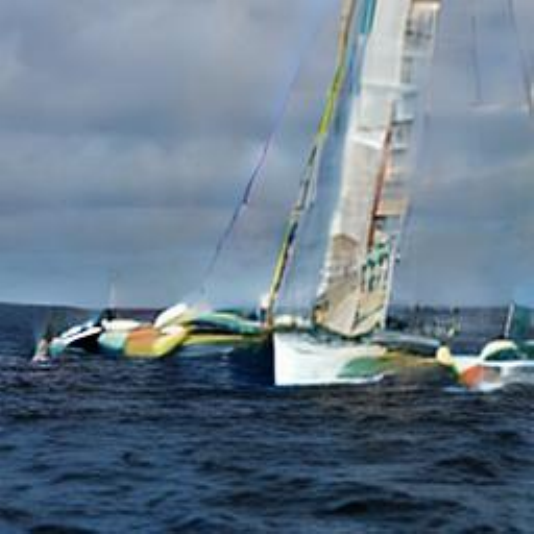}} & 
        \noindent\parbox[c]{0.16\columnwidth}{\includegraphics[width=0.16\textwidth]{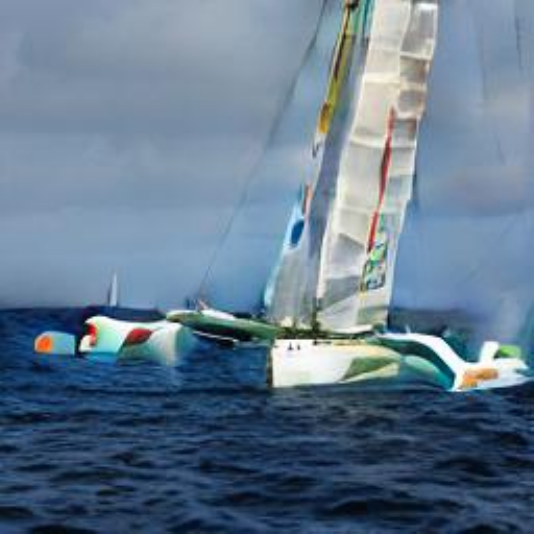}} & 
        \noindent\parbox[c]{0.16\columnwidth}{\includegraphics[width=0.16\textwidth]{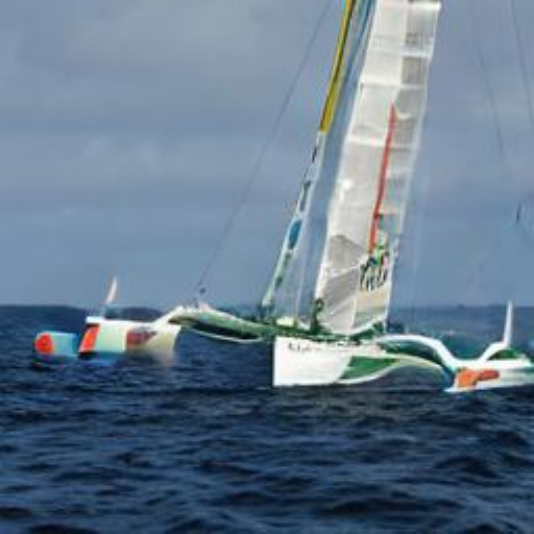}} & 
        \noindent\parbox[c]{0.16\columnwidth}{\includegraphics[width=0.16\textwidth]{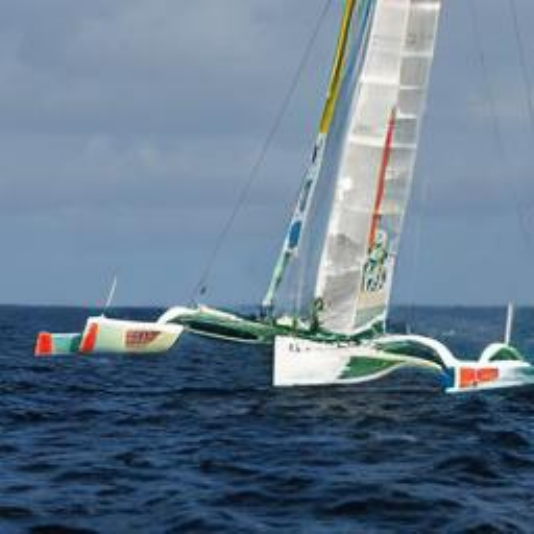}} \\
        
    \end{tabu}

    \caption{Comparison of samples generated from DiT-XL with a guidance scale of 3, using different sampling methods and sampling steps.
    }

    \label{fig:img_dit}
\end{figure}

\section{Extended Comparison on Text-to-Image Comparison} \label{apx:sota_sd}

To provide a more comprehensive evaluation of the methods discussed in Section \ref{sec:exp_sd}, we utilize a fine-tuned variant of Stable-Diffusion called Anything V4.  We consider full-path samples generated by PLMS4 \cite{liu2022pseudo} at 1,000 steps as reference solutions. The performance of each method is evaluated by measuring the image similarity between the generated samples produced using a reduced number of steps and the reference samples. Importantly, both sets of samples originate from identical initial noise maps. This comparison allows us to assess how well the solution from each configuration matches the full-path reference solution.

To quantify image similarity, we use the Learned Perceptual Image Patch Similarity (LPIPS) \cite{zhang2018unreasonable}, where lower values indicate higher similarity. Additionally, we measure similarity using the L2 norm in the latent space, as discussed in Section \ref{sec:ghvb}, again with lower values indicating higher similarity. The outcomes of these analyses are visually presented in Figure \ref{fig:sota_sd}.

Discrepancies between the numerical solutions and the 1,000-step reference solution can arise from two primary factors: the accuracy of the employed method and the presence of divergence artifacts. Notably, in this particular context, divergence artifacts tend to outweigh errors stemming from method accuracy. Consequently, higher-order methods exhibit greater discrepancies in both LPIPS and L2 similarity measurements. It is worth highlighting that our techniques demonstrate a remarkable ability to minimize deviations from the reference solution in the majority of cases. Furthermore, Figure \ref{fig:sota_sd} consistently demonstrates the superiority of our techniques compared to other methods, as evidenced by the lower LPIPS and L2 similarity scores. These results indicate that our techniques effectively both reduce divergence artifacts and handle errors related to method accuracy. 

\begin{figure}
    \centering
    \includegraphics[width=0.24\textwidth]{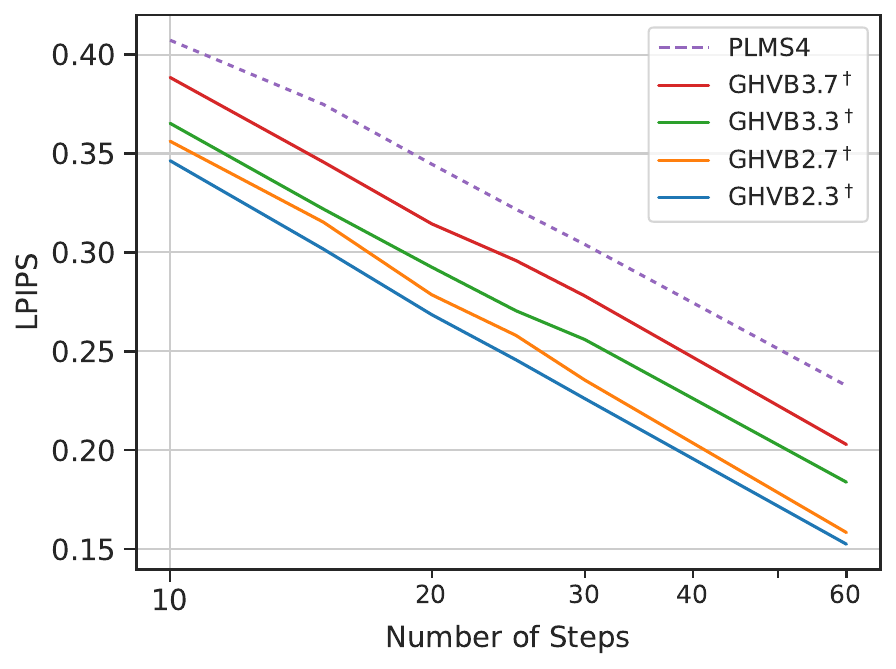}
    \includegraphics[width=0.24\textwidth]{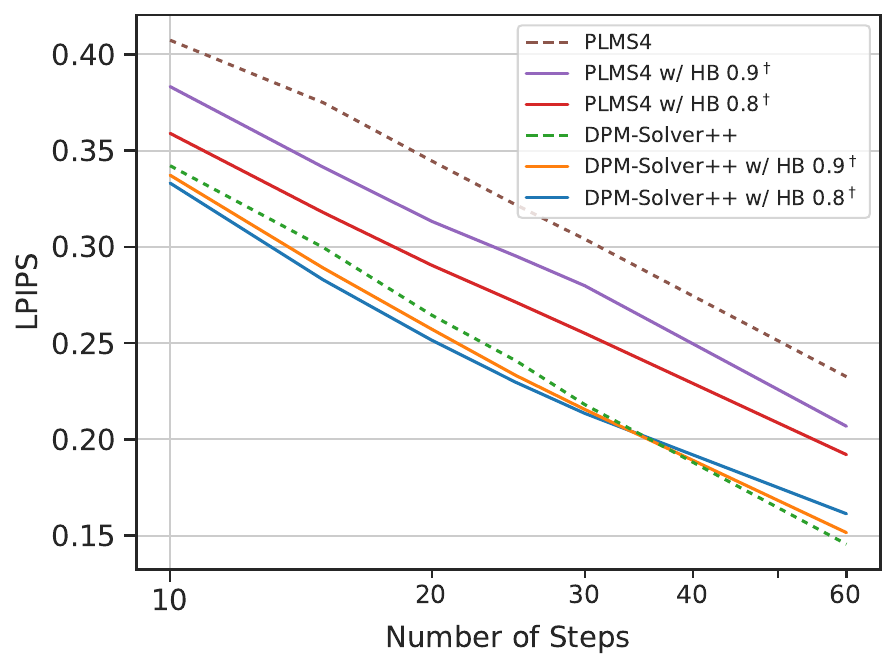}
    \includegraphics[width=0.24\textwidth]{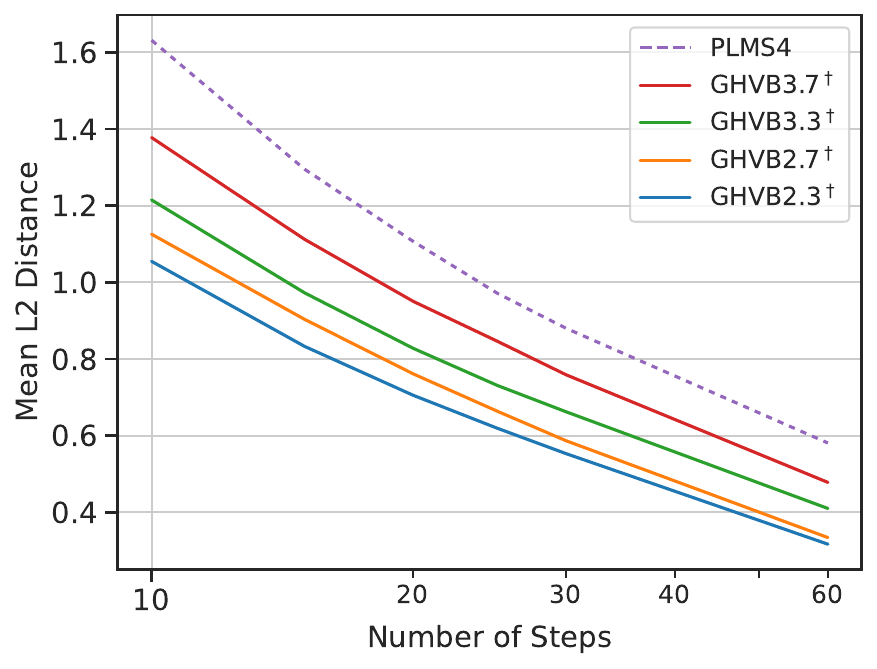}
    \includegraphics[width=0.24\textwidth]{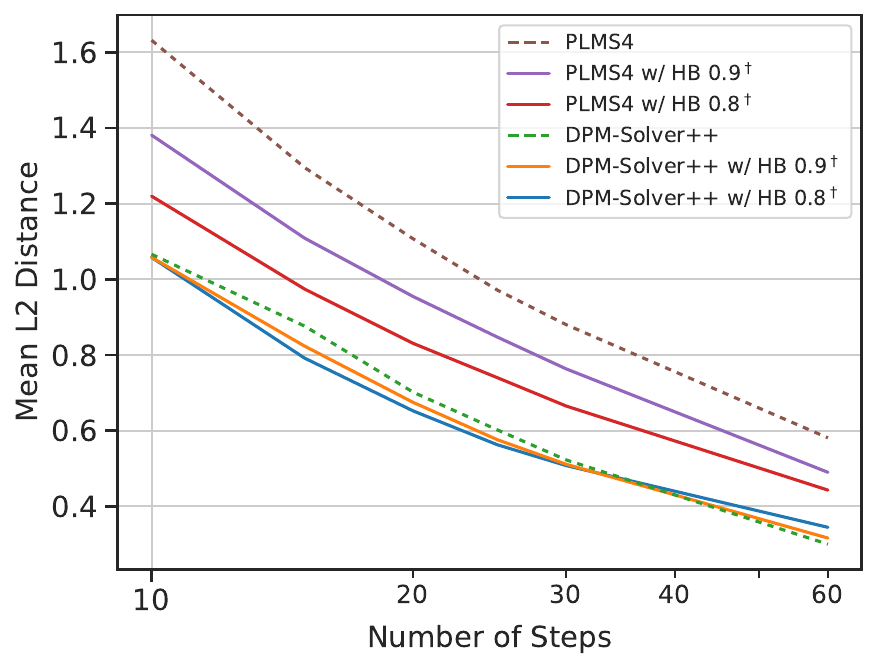}
    \caption{Comparison of LPIPS in the image space and L2 distance in the latent space across different sampling methods, with and without the utilization of our momentum techniques. The experimental setting is similar to that in Section \ref{sec:exp_sd}.}
    
    \label{fig:sota_sd}
\end{figure}

\section{Factors Contributing to Artifact Occurrence in Fine-tuned Diffusion Models} \label{apx:artifact_factors}

This section investigates factors that influence the occurrence of divergence artifacts in diffusion sampling, namely the number of steps, guidance scale, and the choice of diffusion models. The analysis includes a qualitative assessment that compares the results obtained from Stable Diffusion 1.5 (original) (Figure \ref{fig:scale_step_sd15}) with three fine-tuned diffusion models designed for specific purposes: generating Midjourney-style images (Figure \ref{fig:scale_step_openjourney}), Japanese animation-style images (Figure \ref{fig:scale_step_waifu}), and photorealistic images (Figure \ref{fig:scale_step_dreamlike}).

Our observations reveal that several factors contribute to the occurrence of artifacts in diffusion sampling, including the number of steps, guidance scale, and choice of diffusion models. Insufficient numbers of steps and high guidance scales positively correlate with the presence of divergence artifacts in the generated samples. Fine-tuned models exhibit a higher sensitivity to these factors, resulting in a greater incidence of artifacts compared to Stable Diffusion 1.5. Consistent with the findings presented in Section \ref{artifact_sampling}, reducing the number of steps increases the likelihood of artifact occurrence. Furthermore, increasing the guidance scale amplifies the magnitude of eigenvalues, contributing to the presence of artifacts.

The choice of diffusion model also has an impact on artifact occurrence. Stable Diffusion 1.5 exhibits the fewest artifacts compared to the fine-tuned models, which demonstrate a higher incidence of divergence artifacts. Among the fine-tuned models, Openjourney demonstrates the lowest occurrence of artifacts, while also producing results that look similar to those obtained using the original Stable Diffusion 1.5 model. This suggests that extensive changes to the model may alter the eigenvalues and result in an increased presence of artifacts.

Additionally, we present the results of our techniques for handling divergence artifacts in Figure \ref{fig:scale_step_sd15_hb}, \ref{fig:scale_step_openjourney_hb}, \ref{fig:scale_step_waifu_hb}, and \ref{fig:scale_step_dreamlike_hb}. The choice of the parameter $\beta$ plays a crucial role in achieving a balance between reducing artifacts and maintaining accuracy, with its optimal value being influenced by the chosen guidance scale and diffusion model.

\tabulinesep=1pt
\begin{figure}
    \centering
    \begin{tabu} to \textwidth {
        @{}l
        @{\hspace{5pt}}c
        @{\hspace{2pt}}c
        @{\hspace{2pt}}c
        @{\hspace{2pt}}c
        @{\hspace{2pt}}c
        @{}
    }

        \multicolumn{1}{c}{\shortstack[l]{\scriptsize Number \\ \scriptsize of steps}}
        & \multicolumn{1}{c}{\scriptsize 9}
        & \multicolumn{1}{c}{\scriptsize 12}
        & \multicolumn{1}{c}{\scriptsize 15}
        & \multicolumn{1}{c}{\scriptsize 18}
        & \multicolumn{1}{c}{\scriptsize 21}
        \\
        
        \shortstack[l]{\tiny PLMS1 (Euler)} &
        \noindent\parbox[c]{0.17\columnwidth}{\includegraphics[width=0.17\columnwidth]{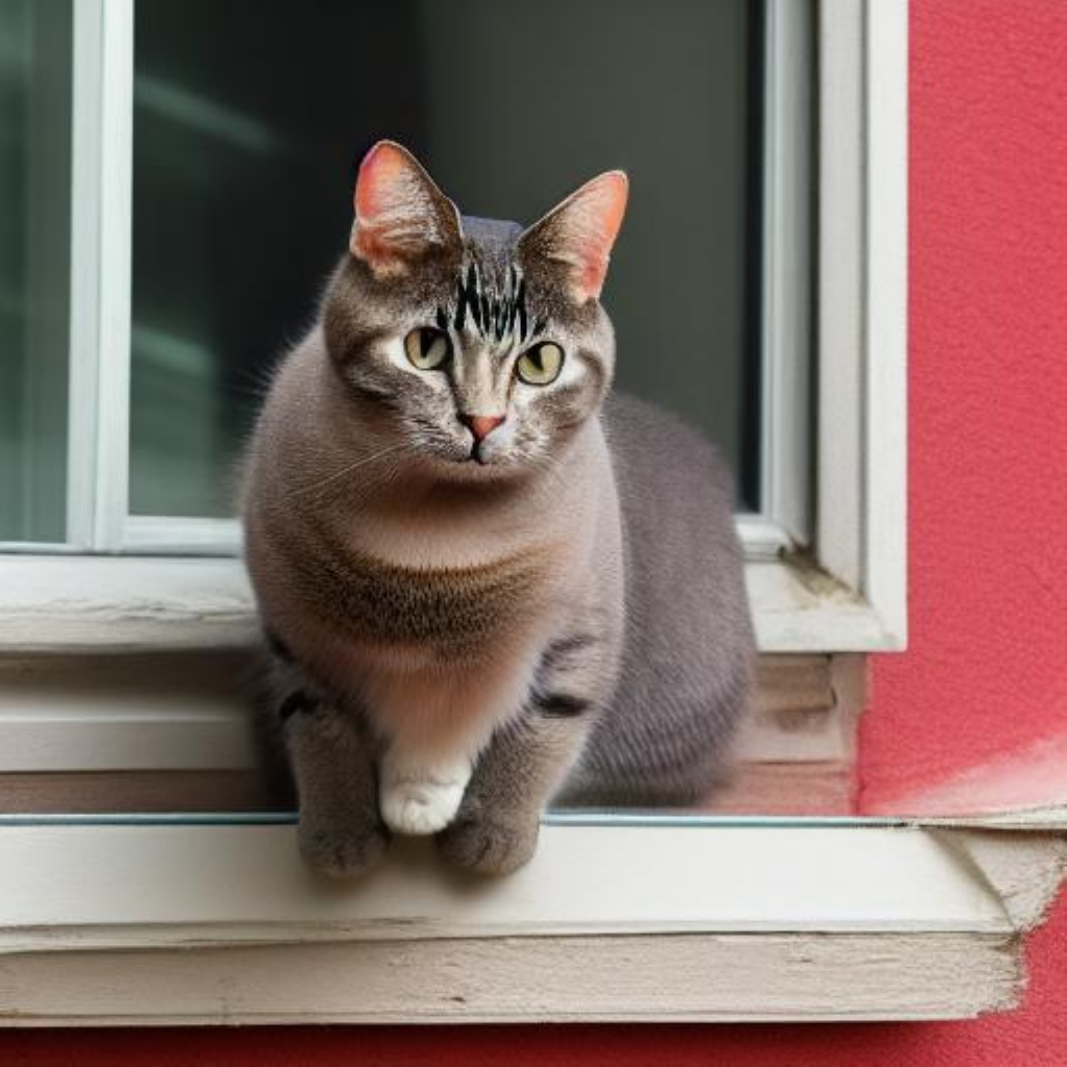}} & 
        \noindent\parbox[c]{0.17\columnwidth}{\includegraphics[width=0.17\columnwidth]{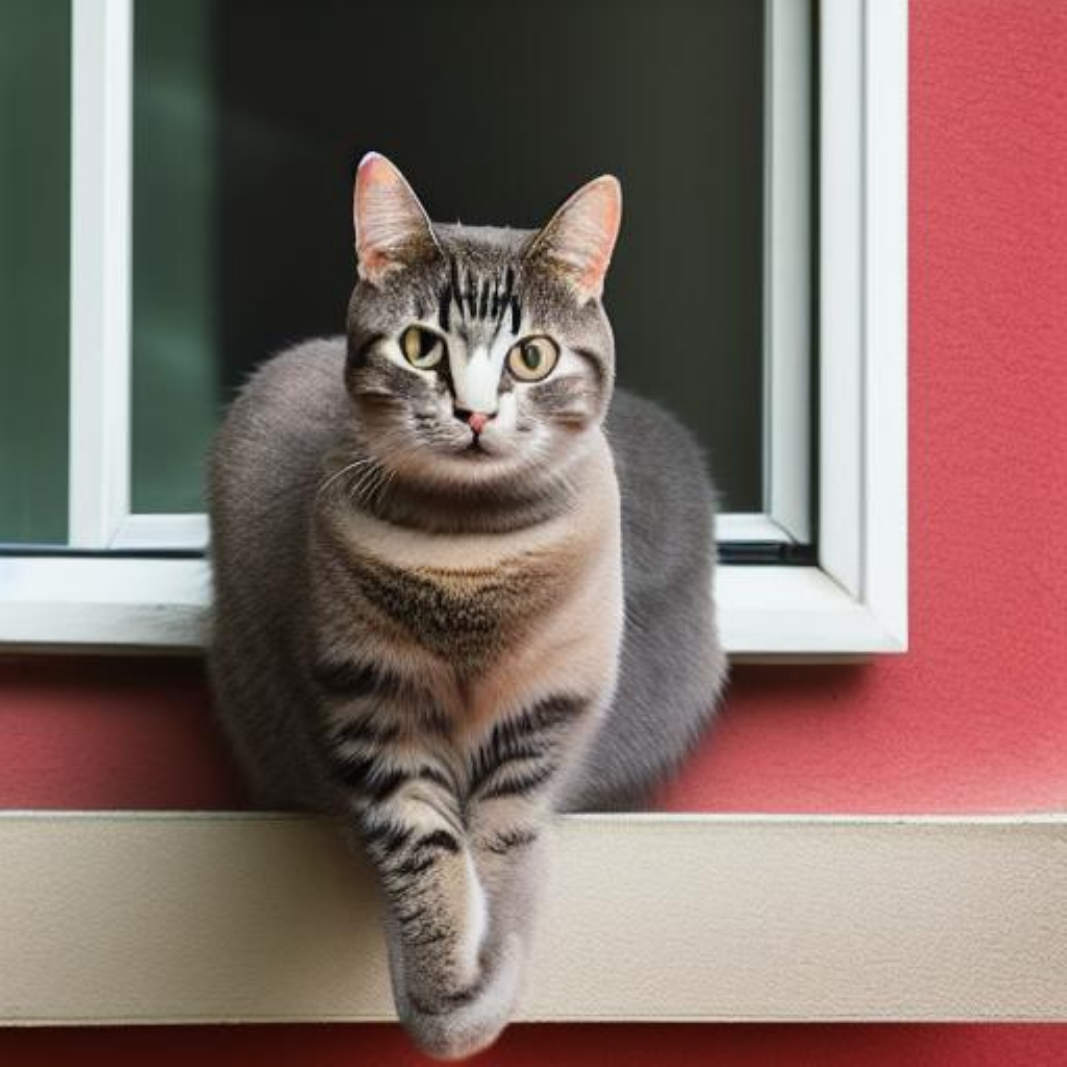}} & 
        \noindent\parbox[c]{0.17\columnwidth}{\includegraphics[width=0.17\columnwidth]{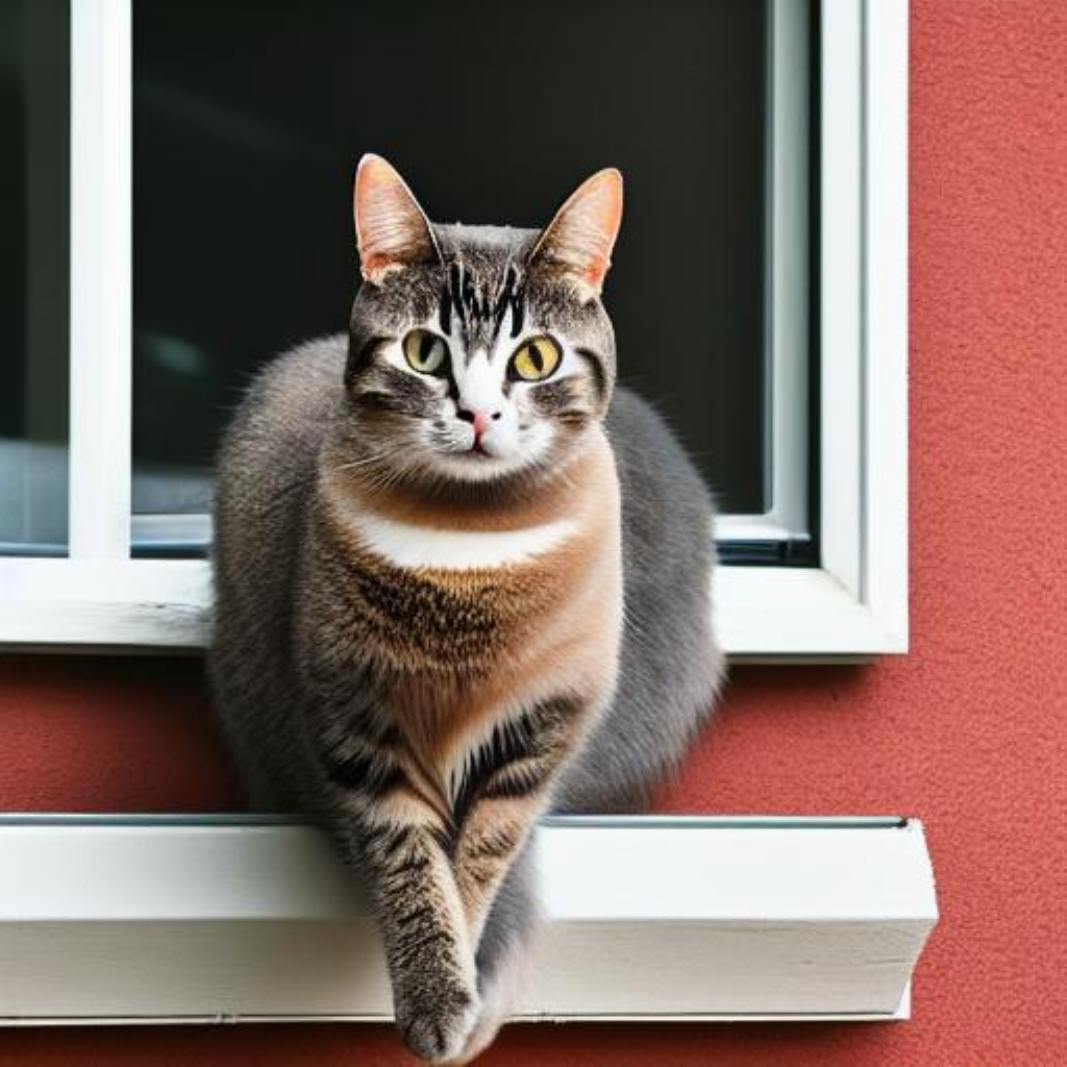}} & 
        \noindent\parbox[c]{0.17\columnwidth}{\includegraphics[width=0.17\columnwidth]{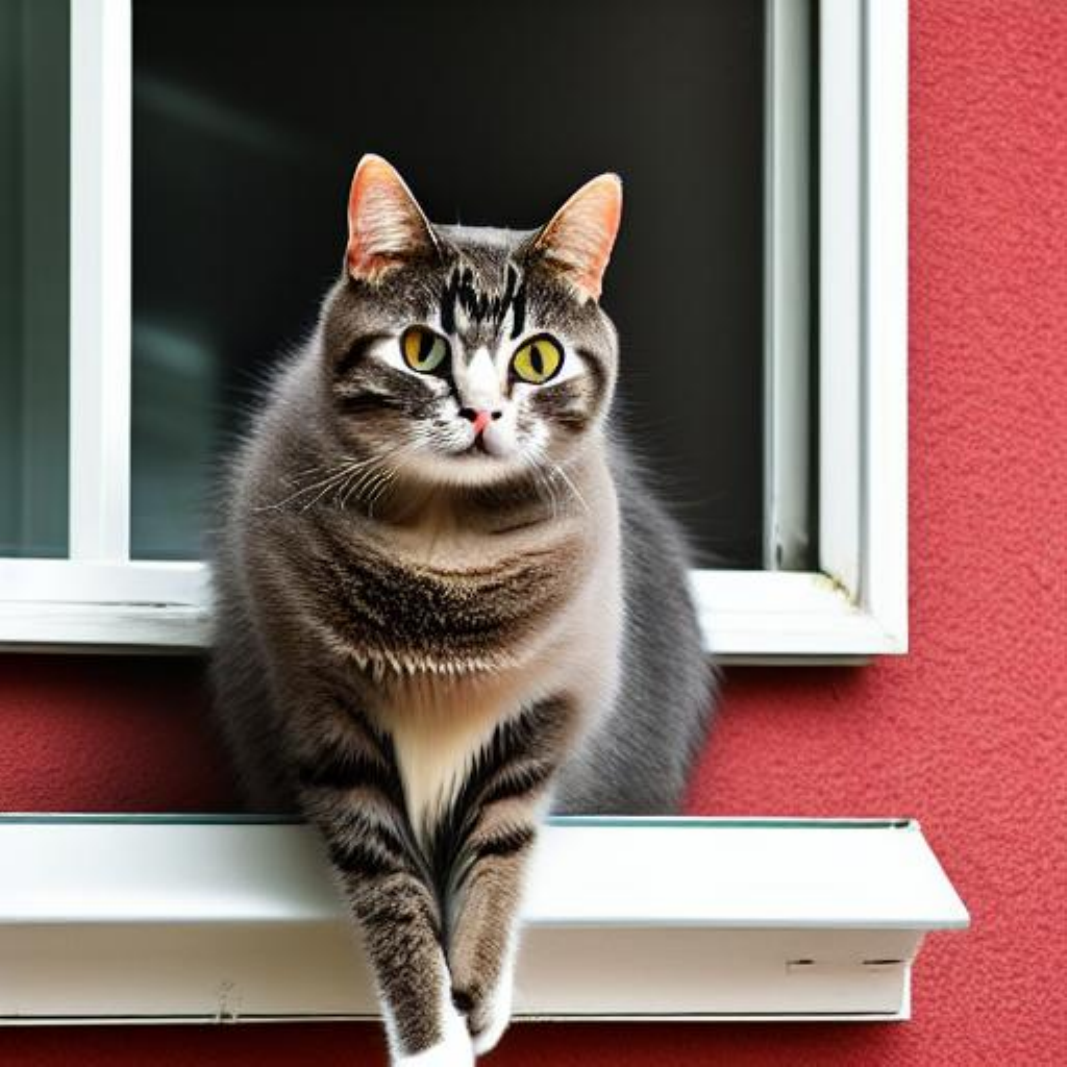}} & 
        \noindent\parbox[c]{0.17\columnwidth}{\includegraphics[width=0.17\columnwidth]{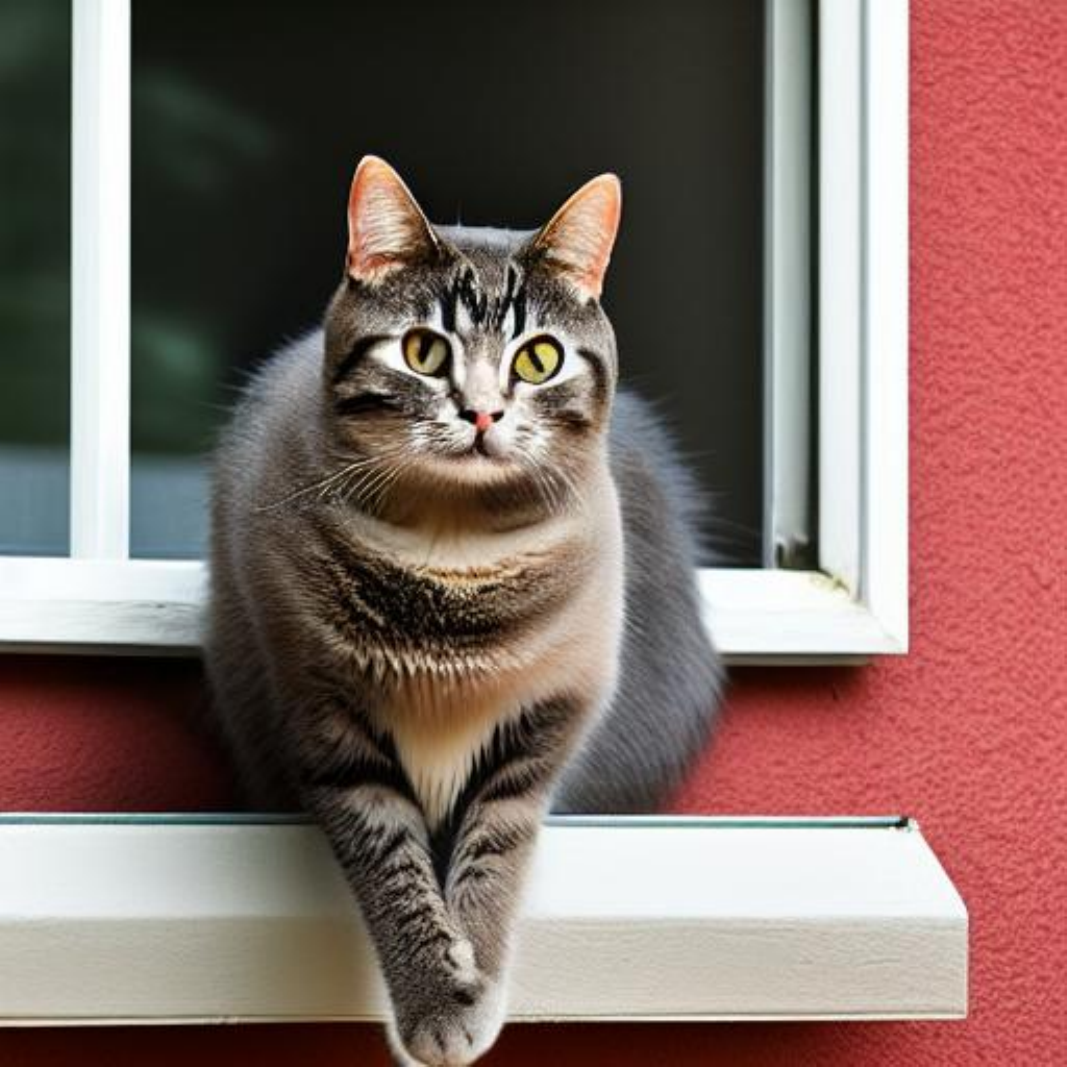}} \\

        \shortstack[l]{\tiny PLMS2} &
        \noindent\parbox[c]{0.17\columnwidth}{\includegraphics[width=0.17\columnwidth]{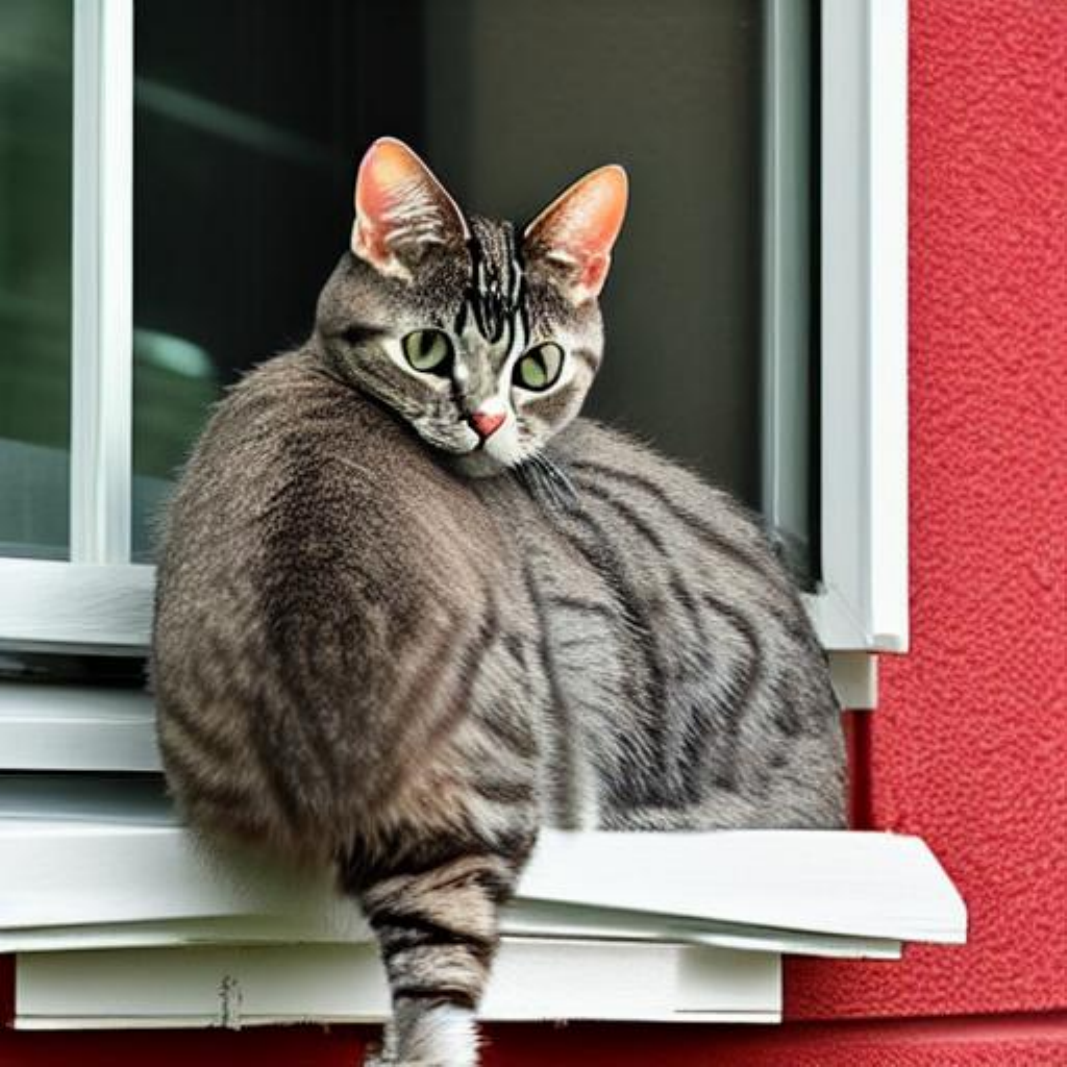}} & 
        \noindent\parbox[c]{0.17\columnwidth}{\includegraphics[width=0.17\columnwidth]{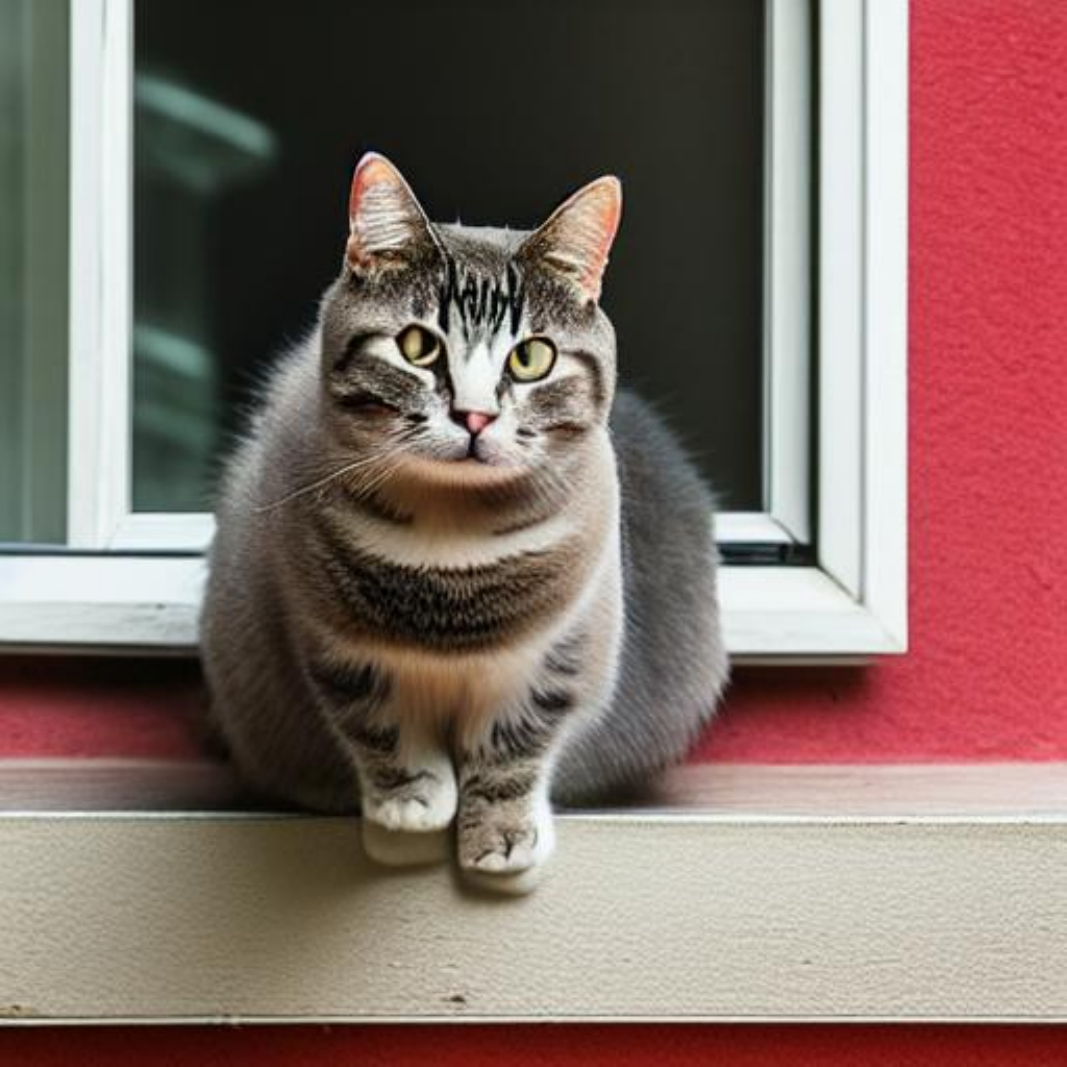}} & 
        \noindent\parbox[c]{0.17\columnwidth}{\includegraphics[width=0.17\columnwidth]{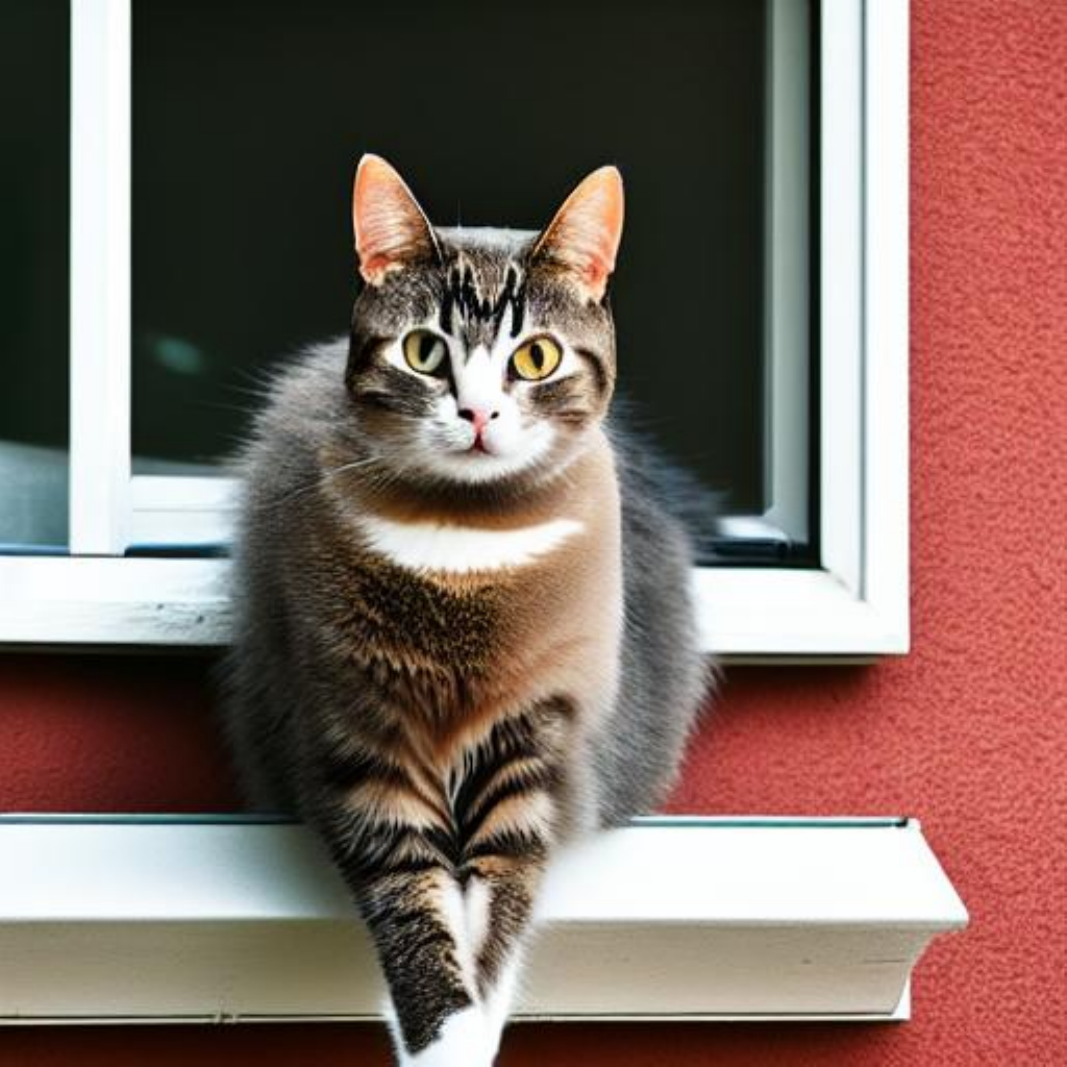}} & 
        \noindent\parbox[c]{0.17\columnwidth}{\includegraphics[width=0.17\columnwidth]{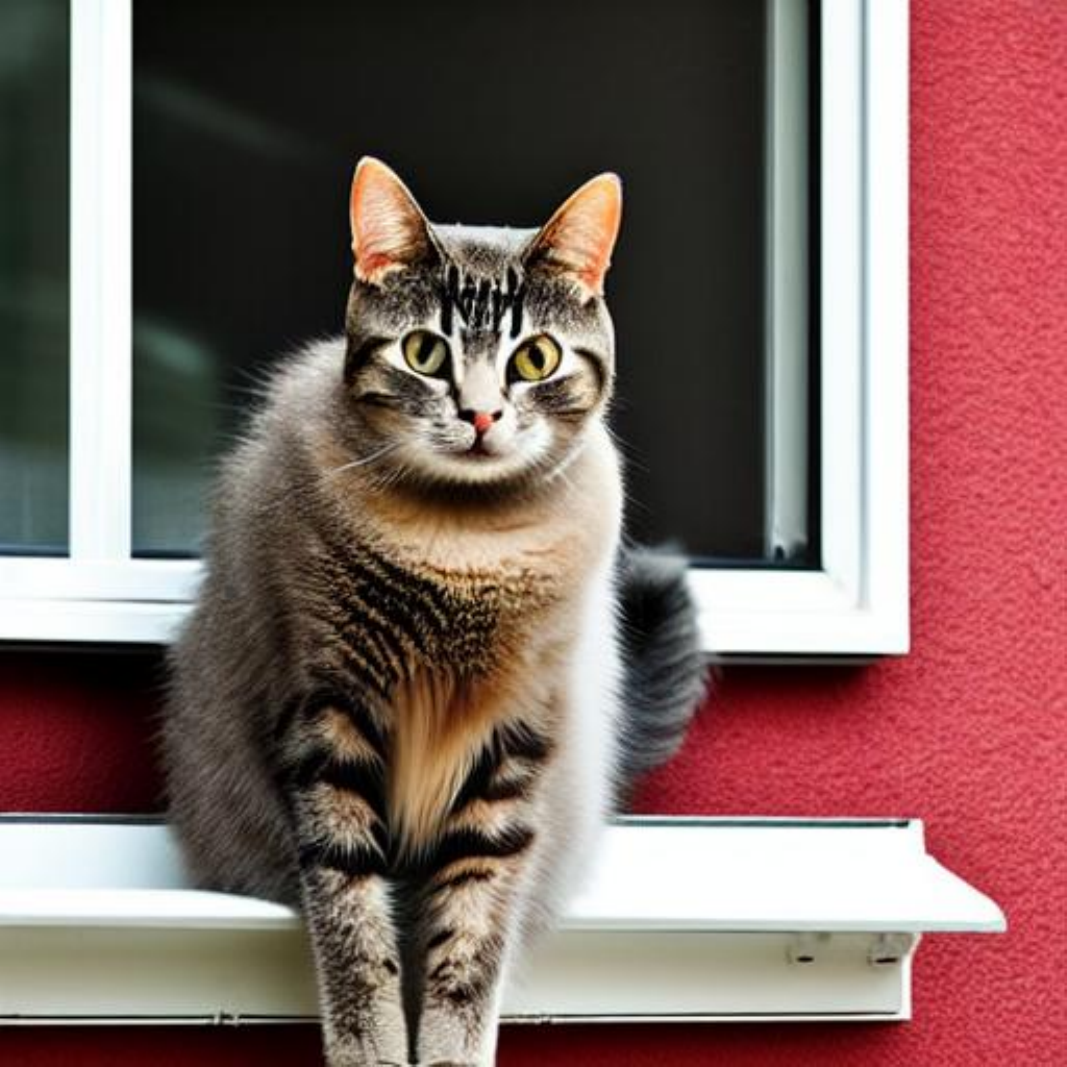}} & 
        \noindent\parbox[c]{0.17\columnwidth}{\includegraphics[width=0.17\columnwidth]{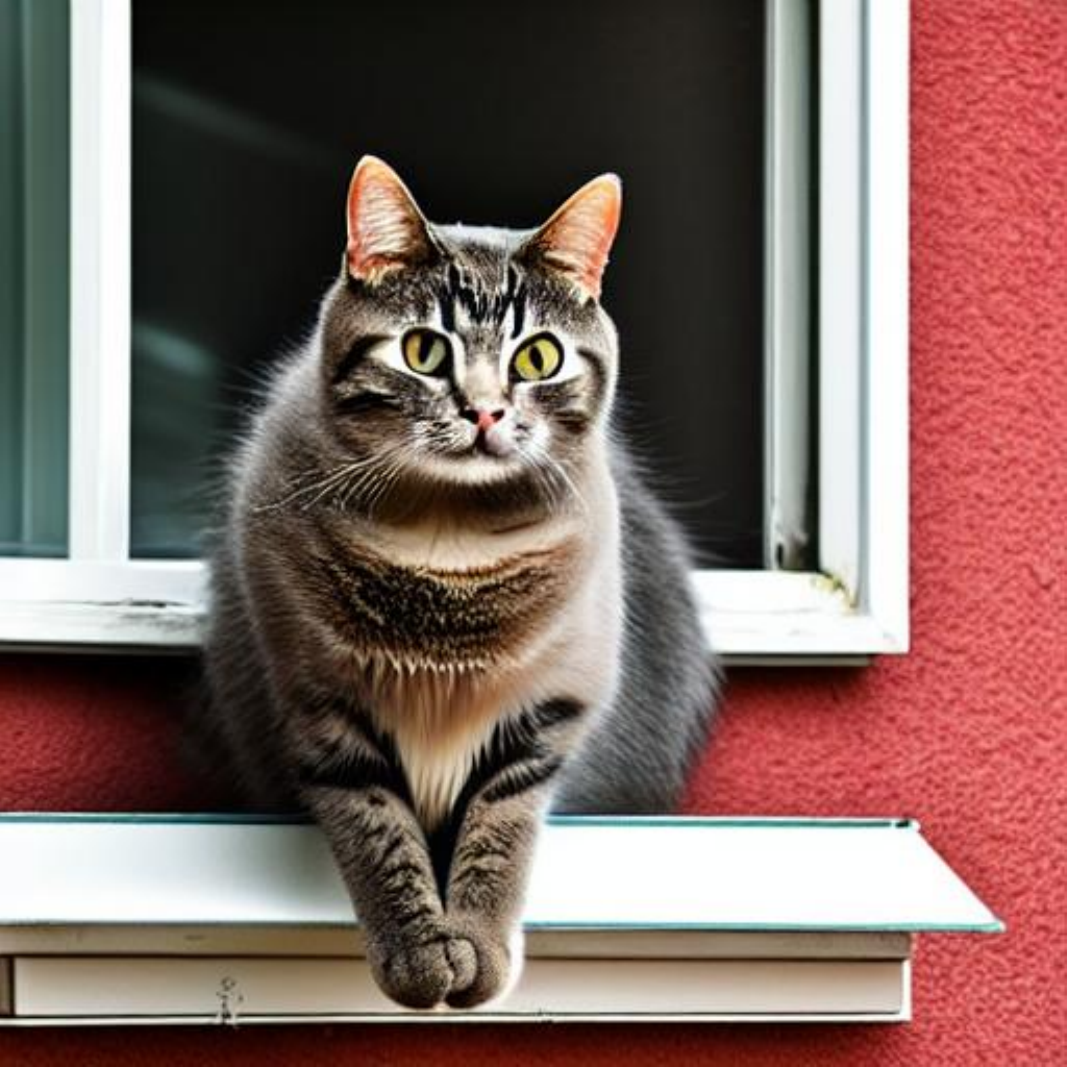}} \\

        \shortstack[l]{\tiny PLMS3} &
        \noindent\parbox[c]{0.17\columnwidth}{\includegraphics[width=0.17\columnwidth]{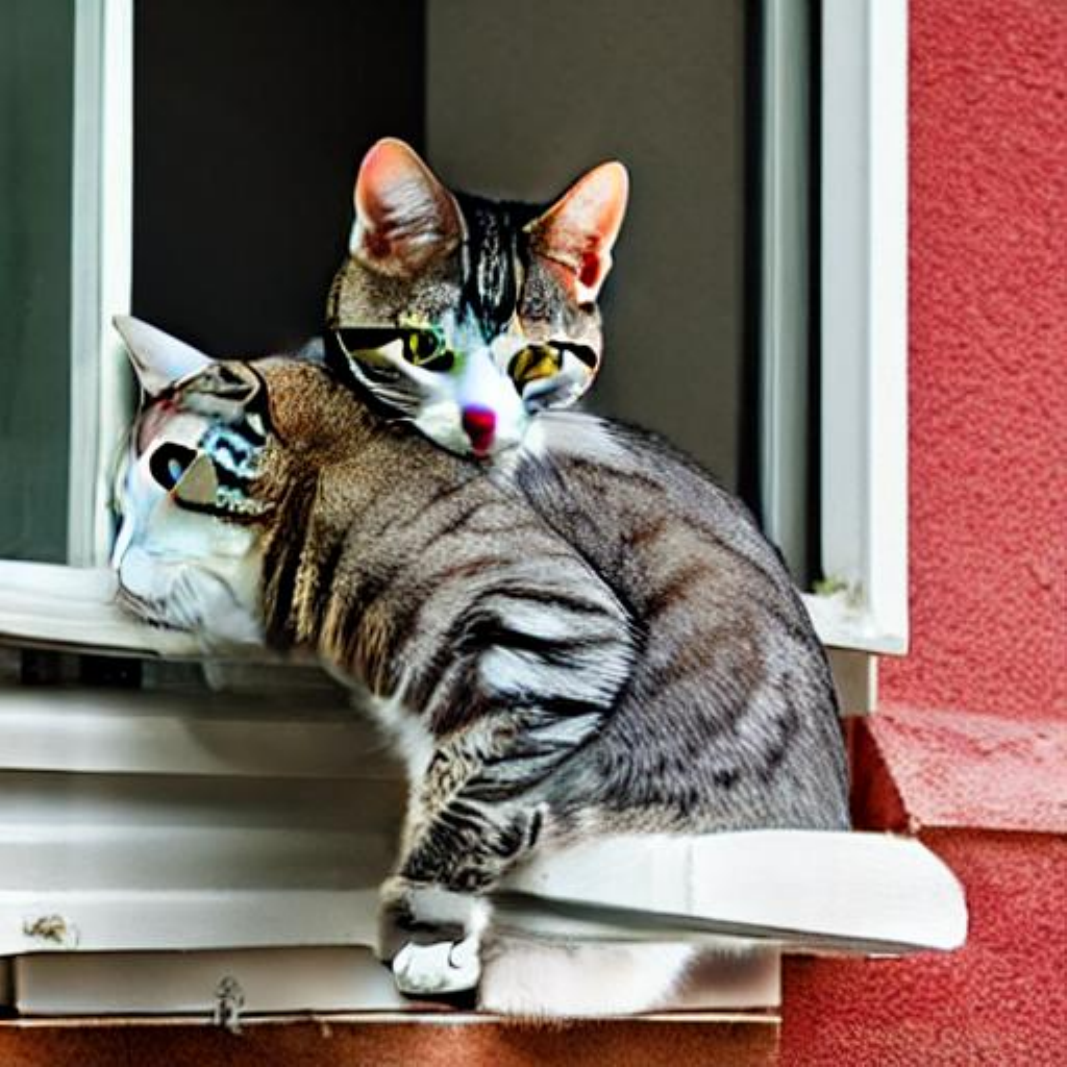}} & 
        \noindent\parbox[c]{0.17\columnwidth}{\includegraphics[width=0.17\columnwidth]{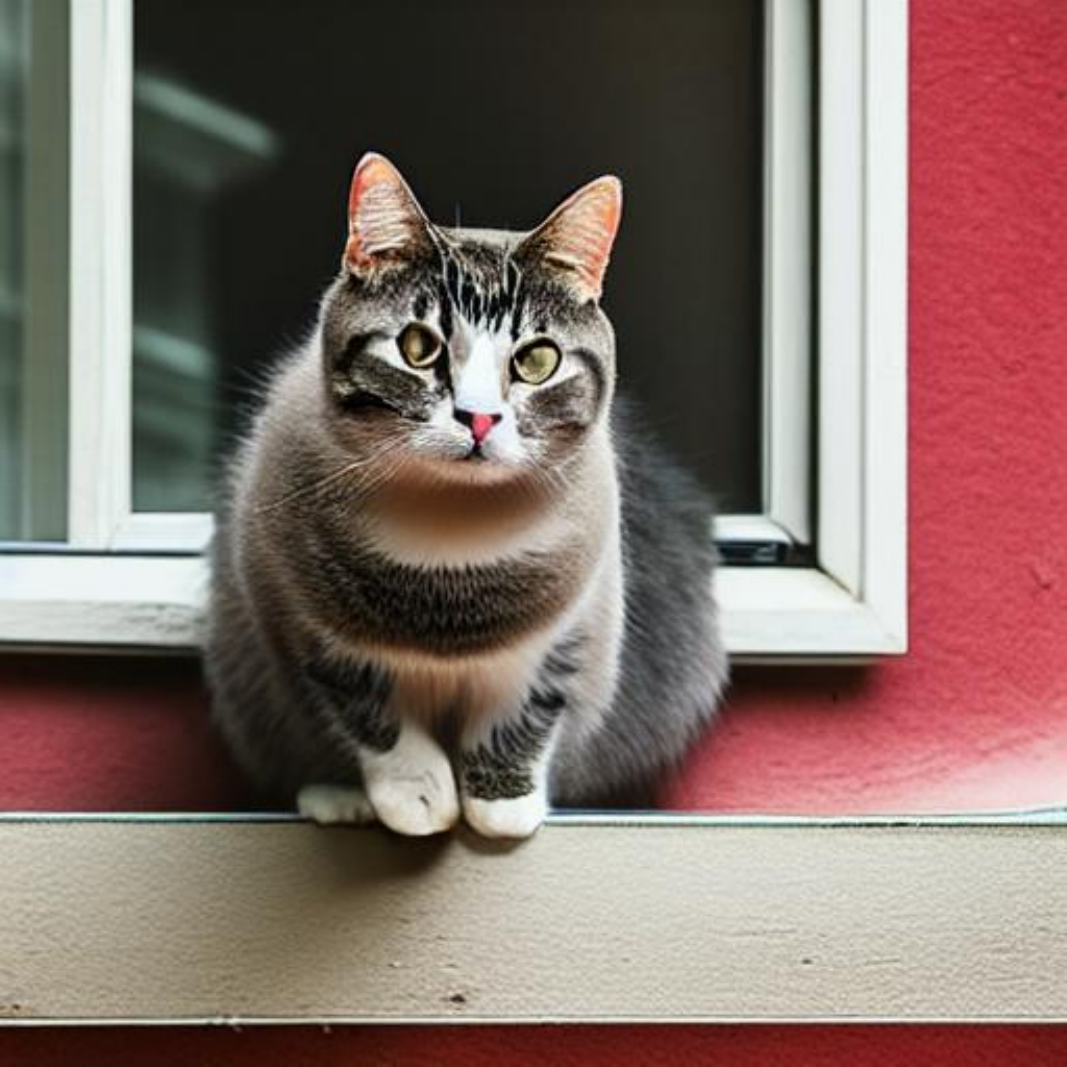}} & 
        \noindent\parbox[c]{0.17\columnwidth}{\includegraphics[width=0.17\columnwidth]{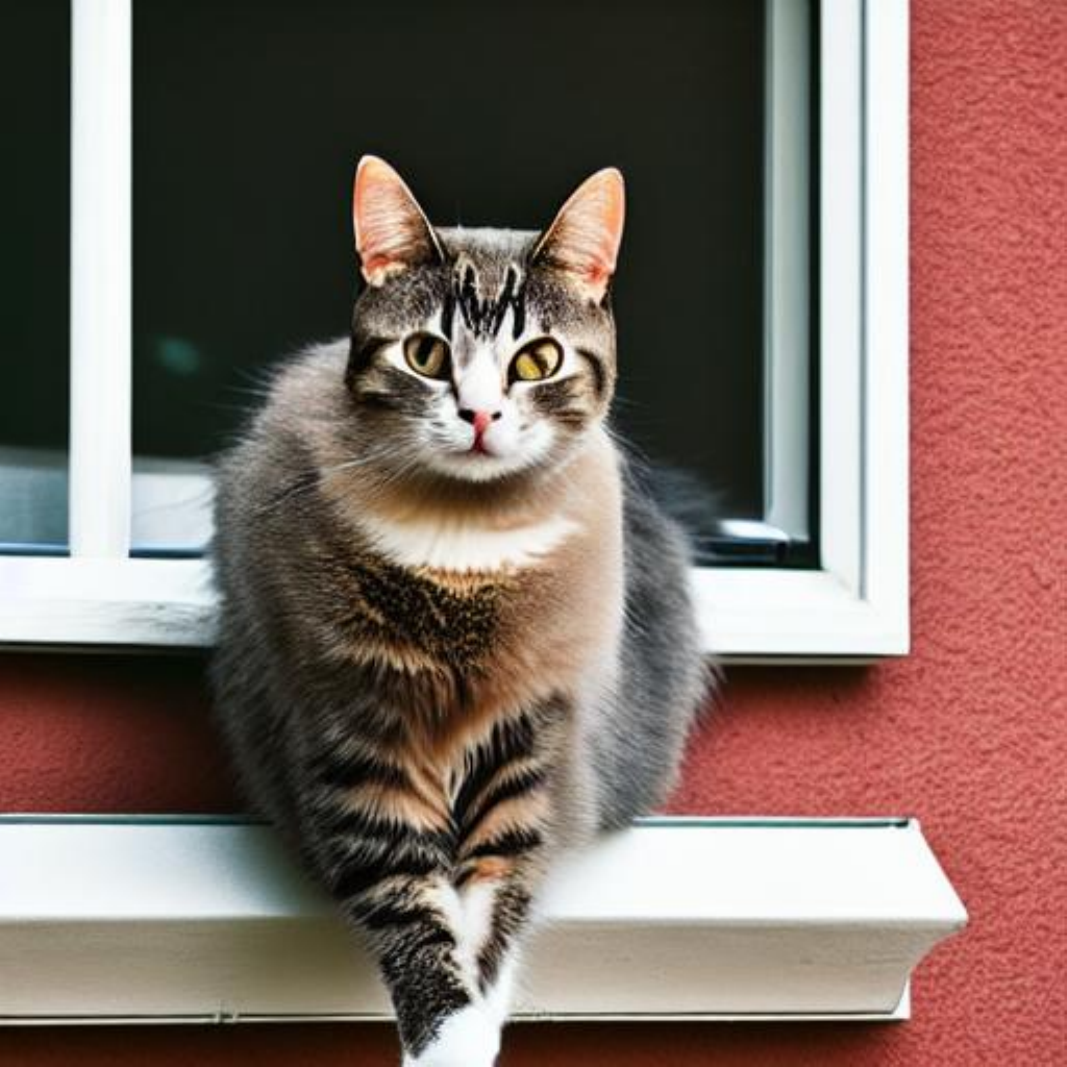}} & 
        \noindent\parbox[c]{0.17\columnwidth}{\includegraphics[width=0.17\columnwidth]{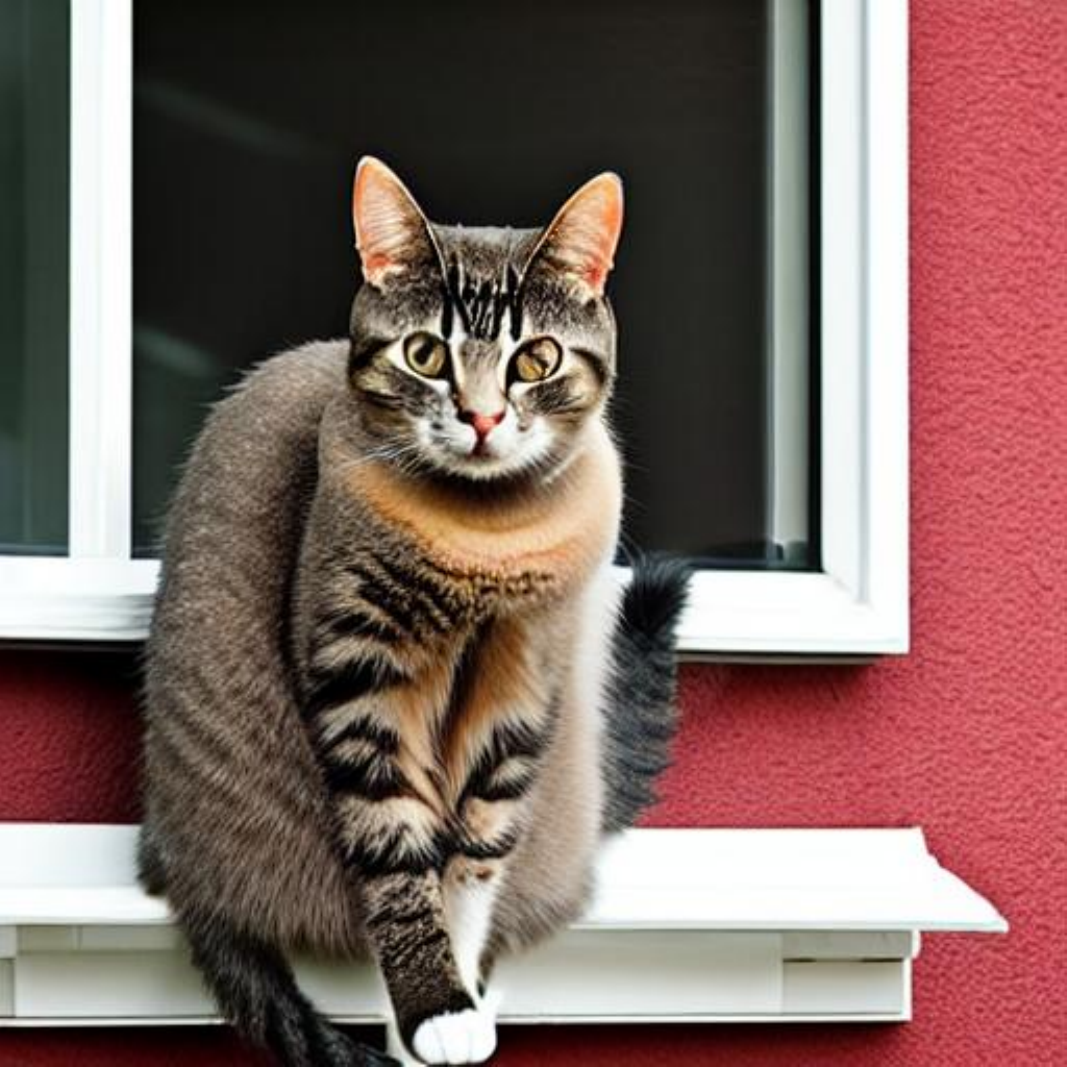}} & 
        \noindent\parbox[c]{0.17\columnwidth}{\includegraphics[width=0.17\columnwidth]{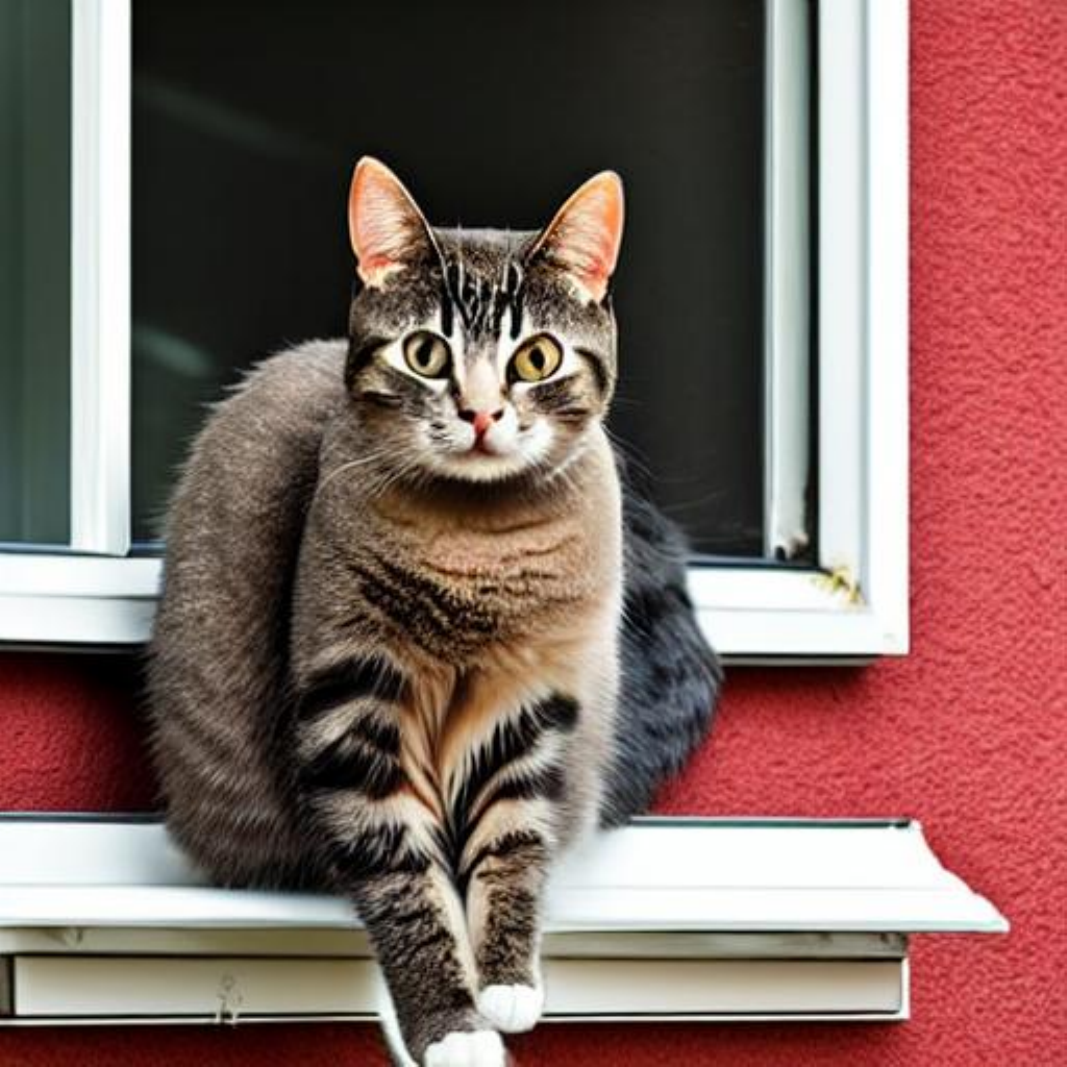}} \\

        \shortstack[l]{\tiny PLMS4} &
        \noindent\parbox[c]{0.17\columnwidth}{\includegraphics[width=0.17\columnwidth]{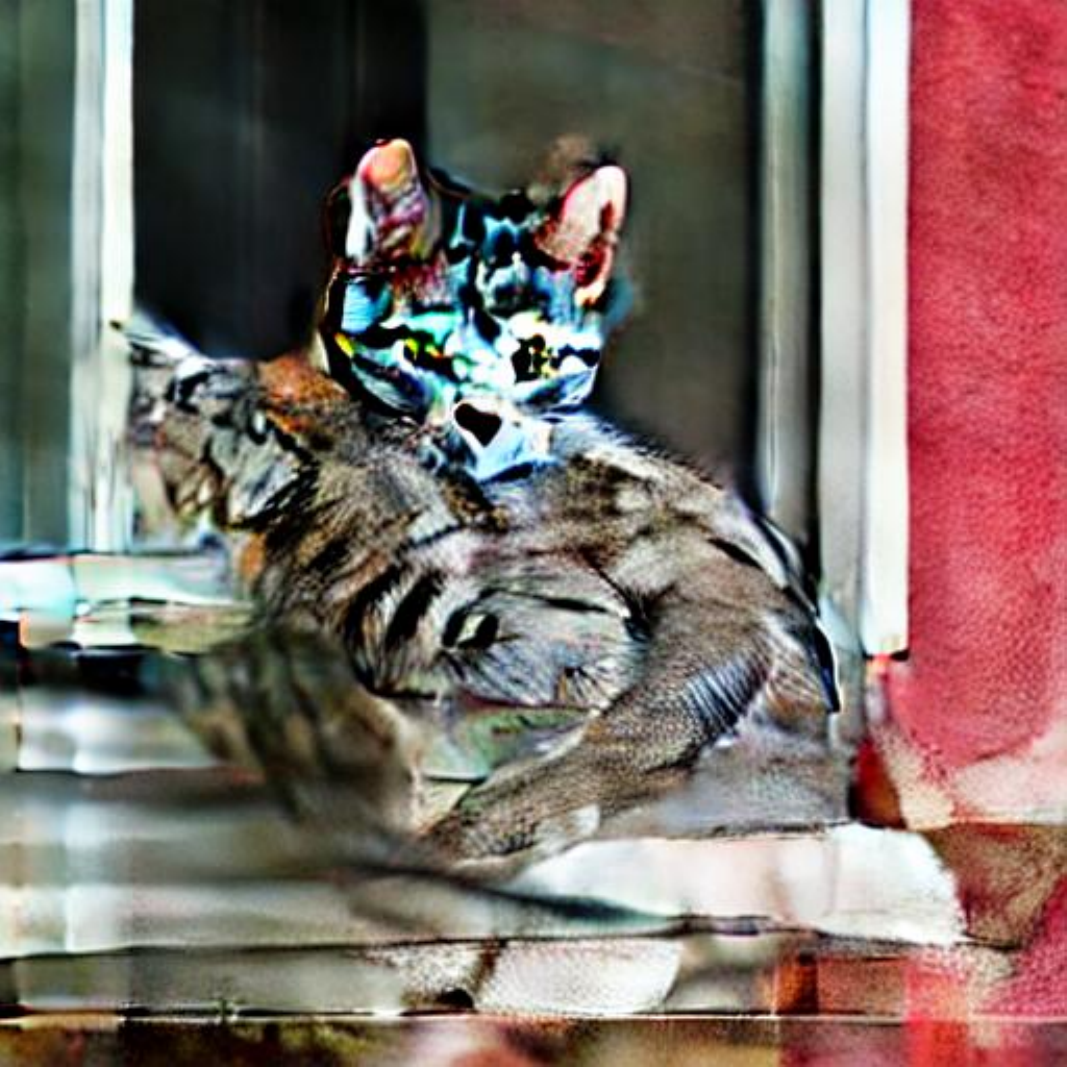}} & 
        \noindent\parbox[c]{0.17\columnwidth}{\includegraphics[width=0.17\columnwidth]{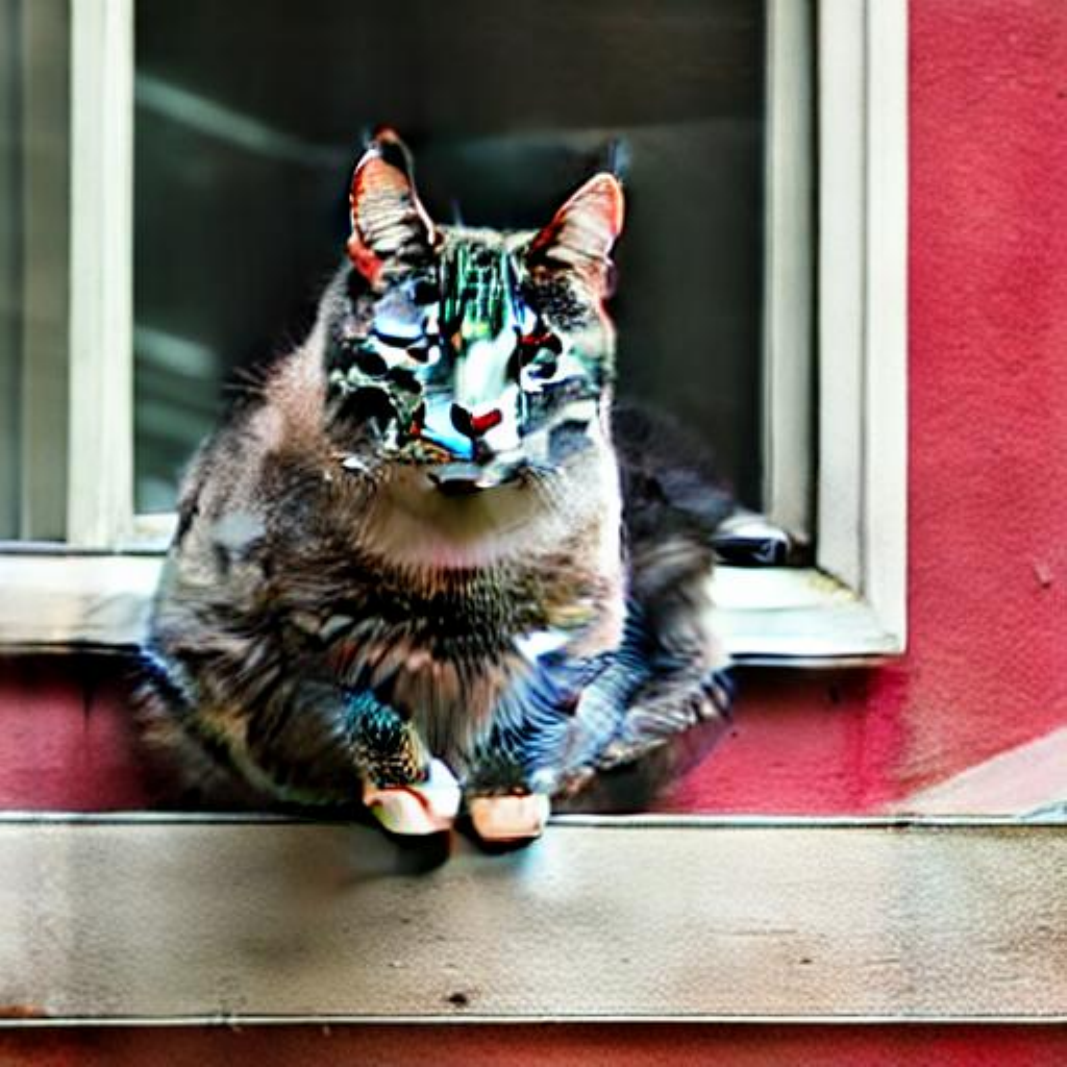}} & 
        \noindent\parbox[c]{0.17\columnwidth}{\includegraphics[width=0.17\columnwidth]{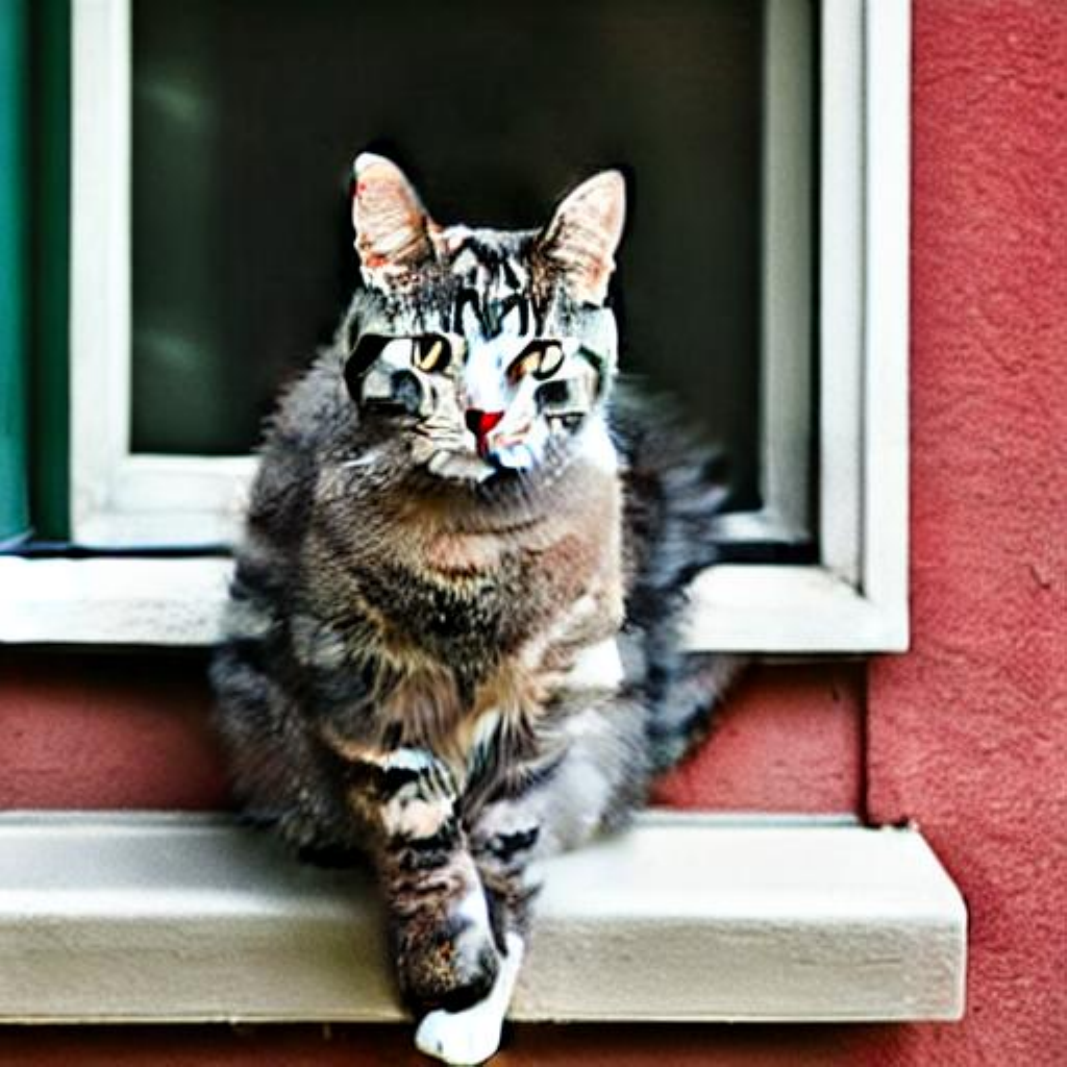}} & 
        \noindent\parbox[c]{0.17\columnwidth}{\includegraphics[width=0.17\columnwidth]{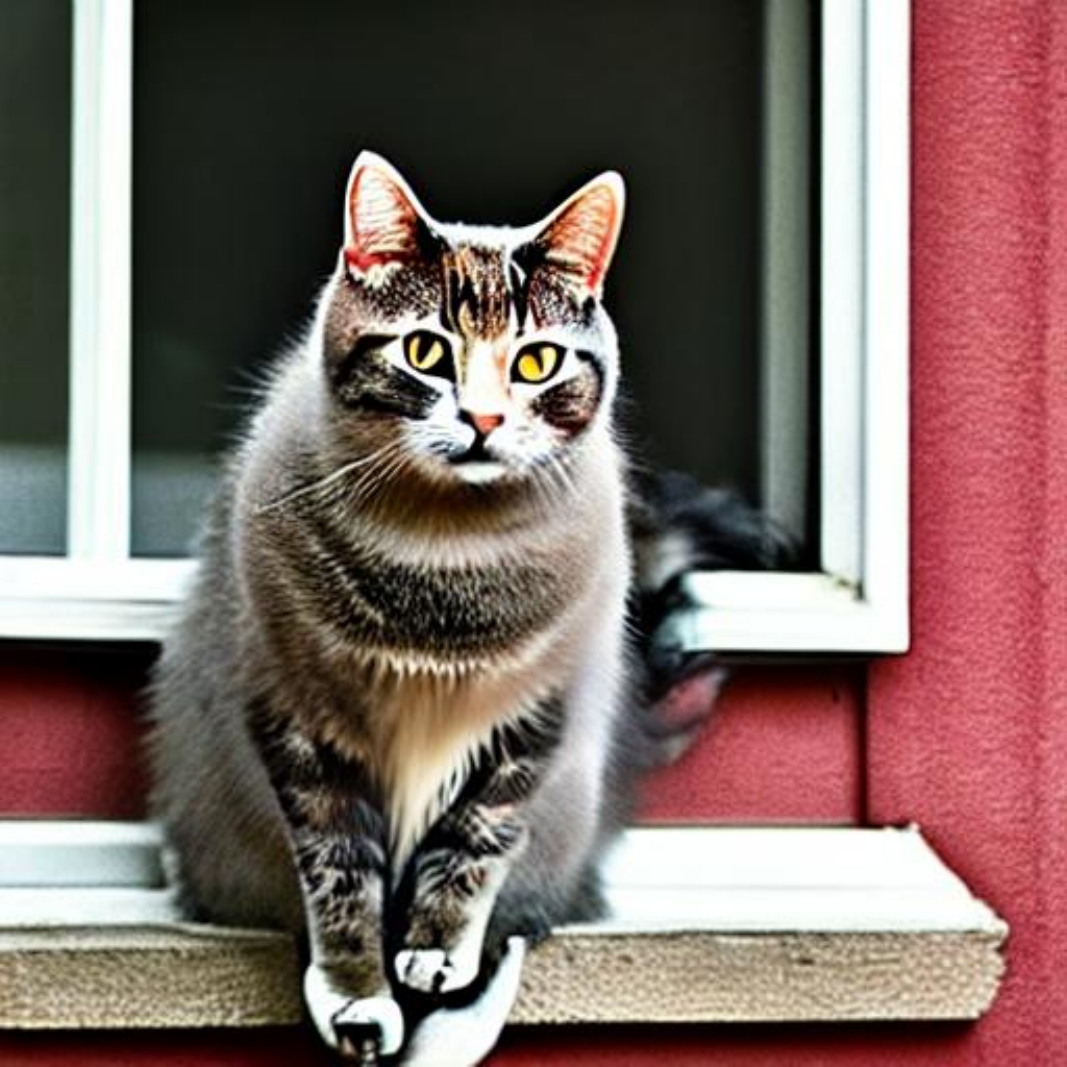}} & 
        \noindent\parbox[c]{0.17\columnwidth}{\includegraphics[width=0.17\columnwidth]{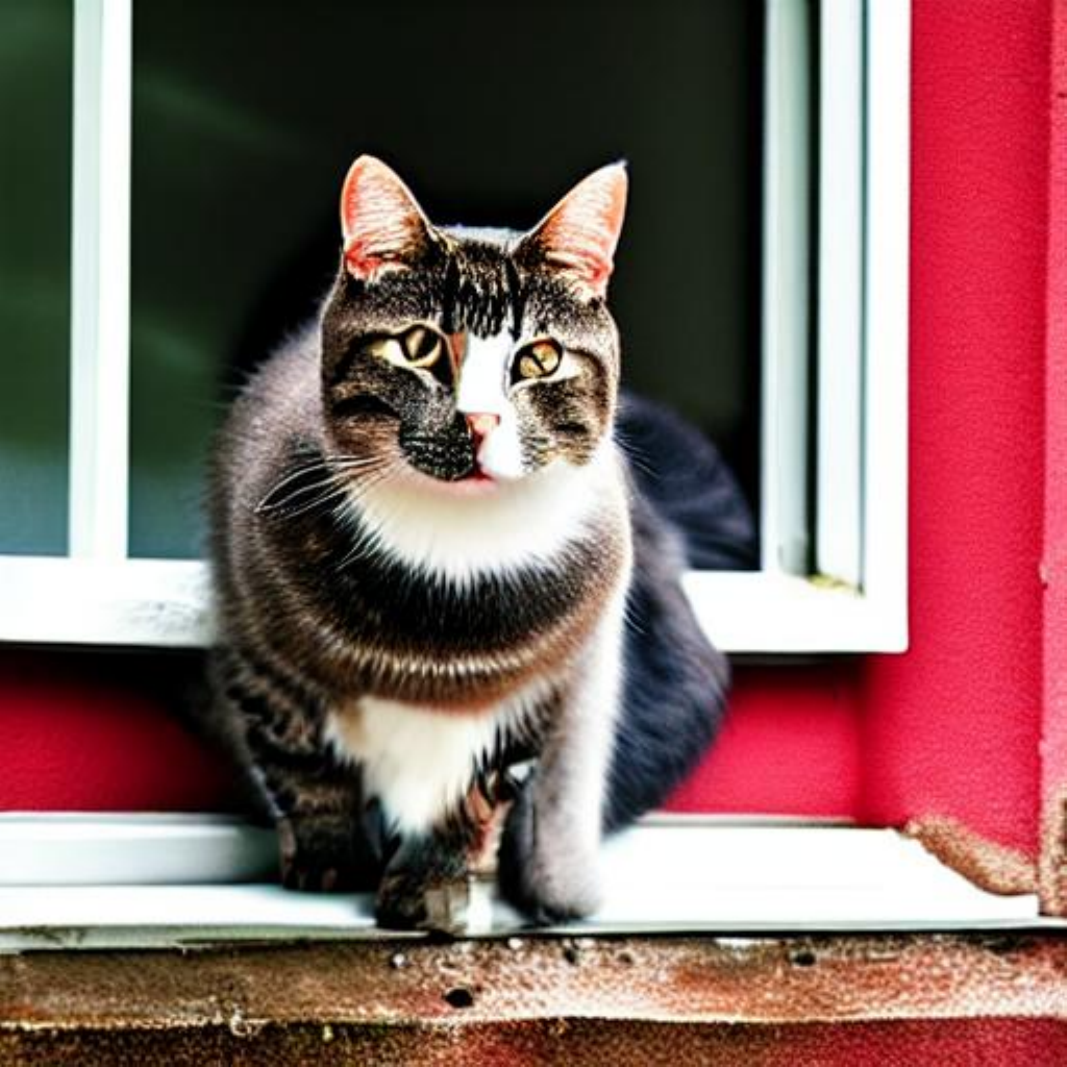}} \\

    \end{tabu}
    \caption{Comparison of samples generated from Stable Diffusion 1.5 \protect\footnotemark with different sampling steps and method orders. The guidance scale is fixed at 10. Prompt: "A cat sitting on a window sill"}
    \label{fig:order_step_sd15}
\end{figure}

\footnotetext{\url{https://huggingface.co/runwayml/stable-diffusion-v1-5}}

\pagebreak
\tabulinesep=1pt
\begin{figure}
    \centering
    \begin{tabu} to \textwidth {@{}l@{\hspace{5pt}}c@{\hspace{2pt}}c@{\hspace{2pt}}c@{\hspace{4pt}}c@{\hspace{2pt}}c@{\hspace{2pt}}c@{}}
        & \multicolumn{3}{c}{\shortstack{\scriptsize "A post-apocalyptic world with ruined \\ \scriptsize buildings, overgrown vegetation, and a red sky"}}
        & \multicolumn{3}{c}{\shortstack{\scriptsize "A girl standing in a park in \\ \scriptsize Japanese animation style"}} \\

        & \multicolumn{1}{c}{\shortstack{\scriptsize $s = 7.5$}}
        & \multicolumn{1}{c}{\shortstack{\scriptsize $s = 15$}}
        & \multicolumn{1}{c}{\shortstack{\scriptsize $s = 22.5$}}
        & \multicolumn{1}{c}{\shortstack{\scriptsize $s = 7.5$}}
        & \multicolumn{1}{c}{\shortstack{\scriptsize $s = 15$}}
        & \multicolumn{1}{c}{\shortstack{\scriptsize $s = 22.5$}}
        \\
        
        \shortstack[l]{\tiny 10 steps} &
        \noindent\parbox[c]{0.14\columnwidth}{\includegraphics[width=0.14\columnwidth]{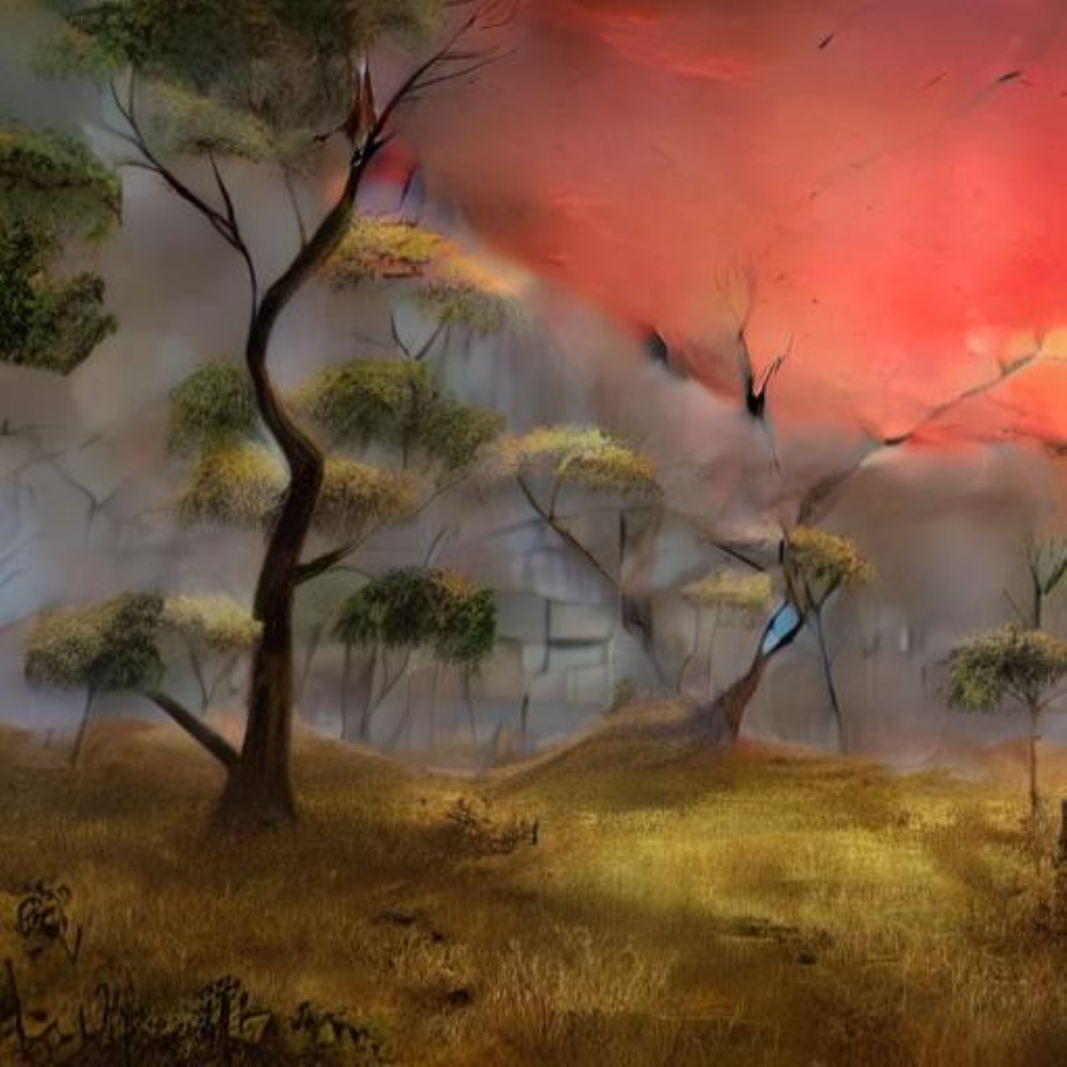}} & 
        \noindent\parbox[c]{0.14\columnwidth}{\includegraphics[width=0.14\columnwidth]{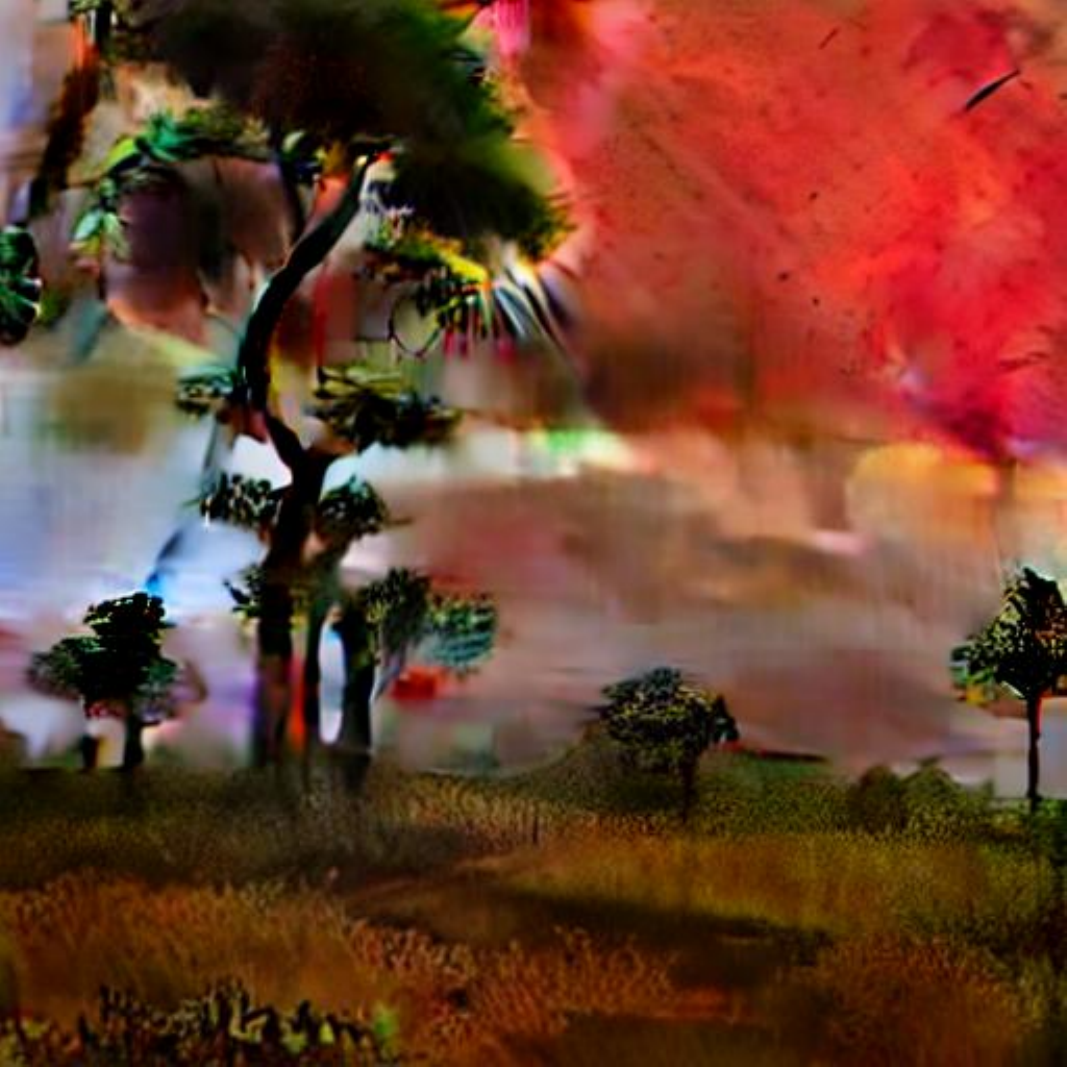}} & 
        \noindent\parbox[c]{0.14\columnwidth}{\includegraphics[width=0.14\columnwidth]{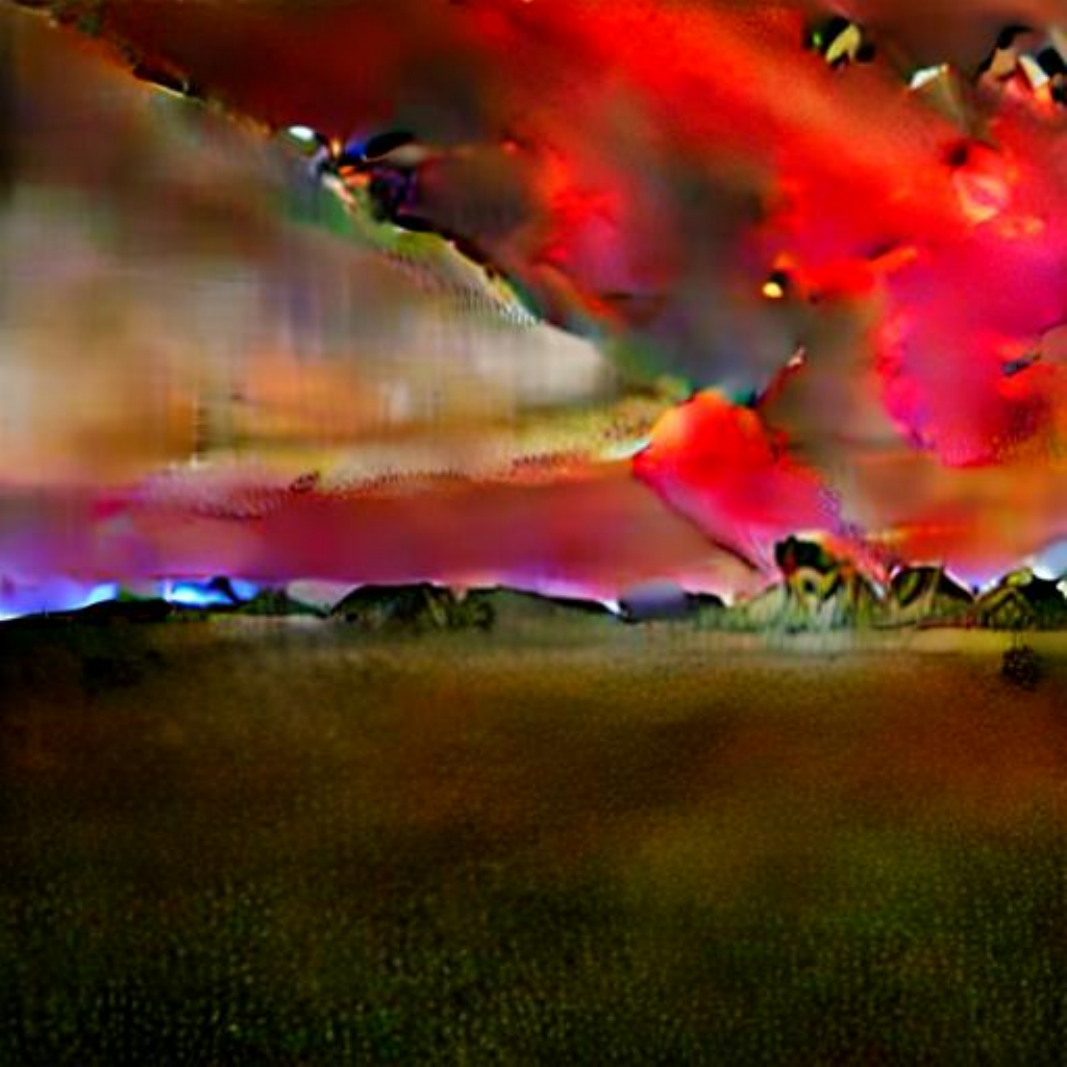}} & 
        \noindent\parbox[c]{0.14\columnwidth}{\includegraphics[width=0.14\columnwidth]{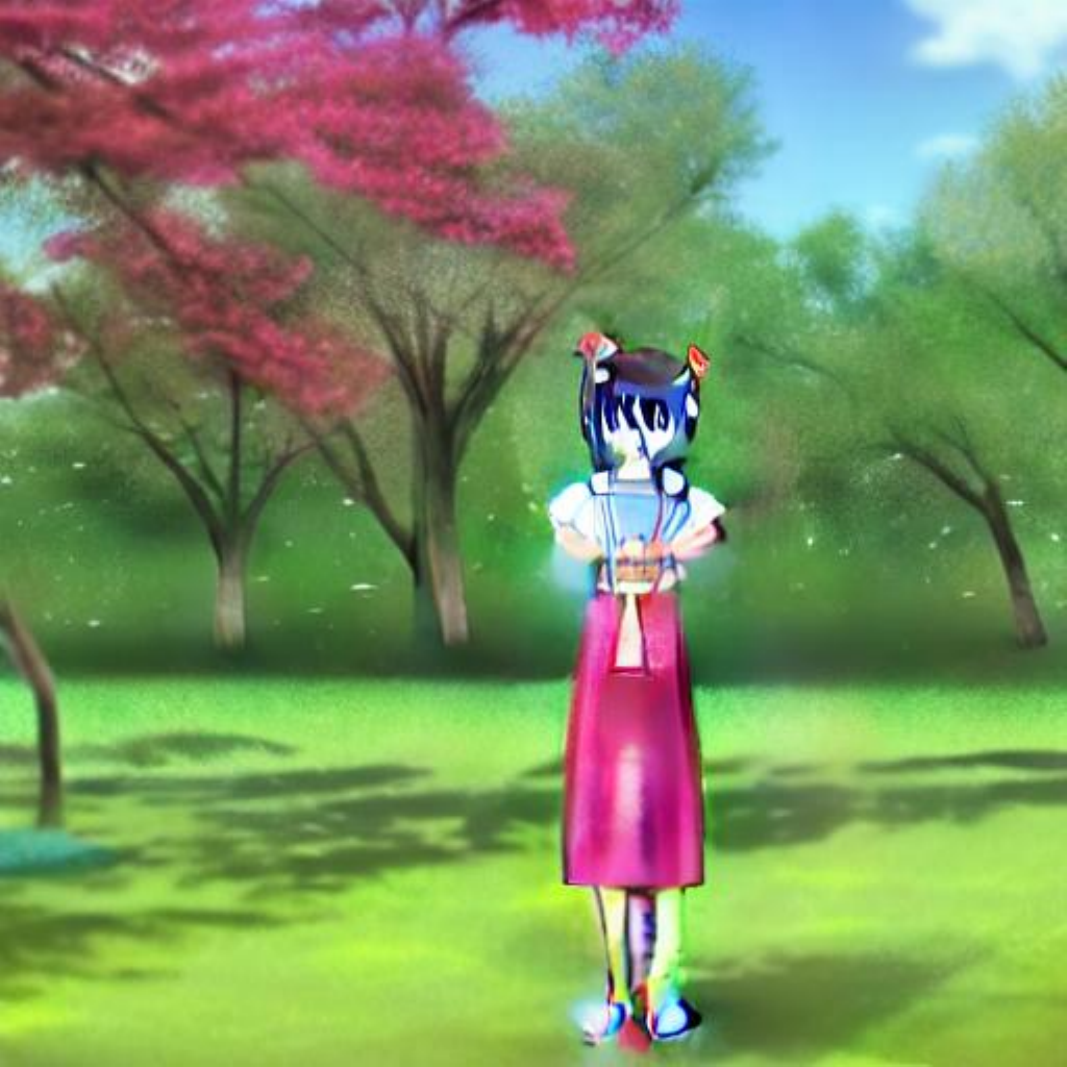}} & 
        \noindent\parbox[c]{0.14\columnwidth}{\includegraphics[width=0.14\columnwidth]{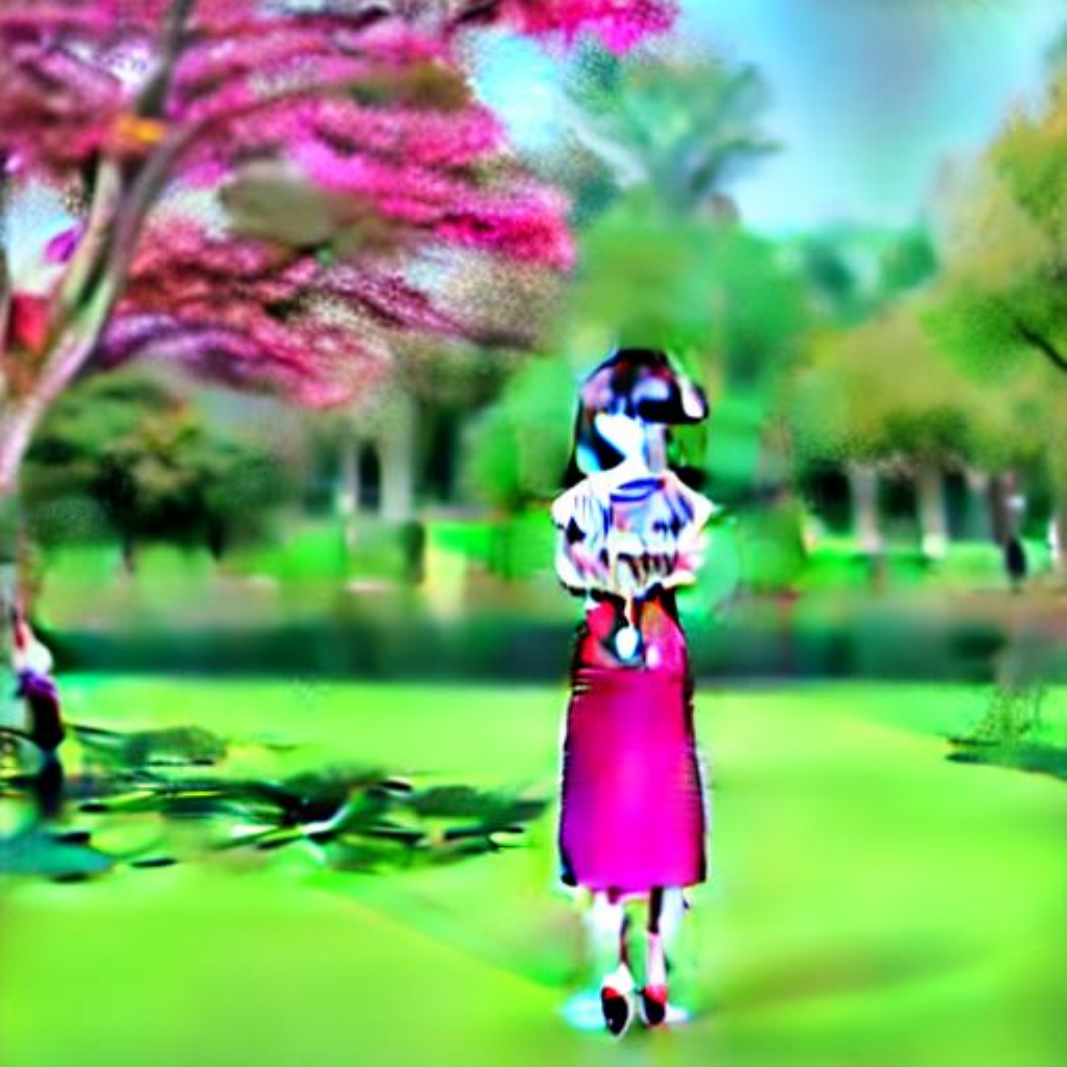}} & 
        \noindent\parbox[c]{0.14\columnwidth}{\includegraphics[width=0.14\columnwidth]{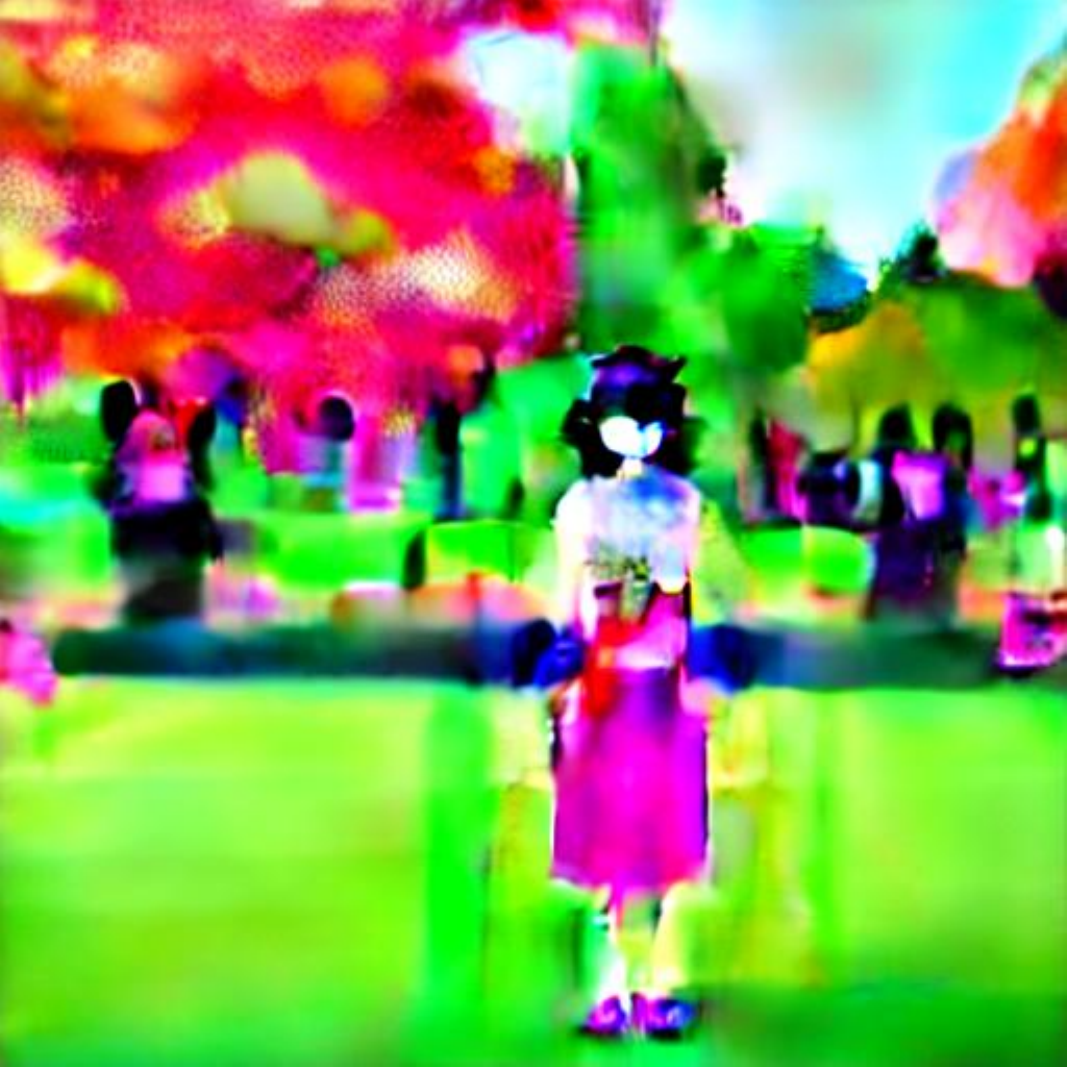}} \\

        \shortstack[l]{\tiny 15 steps} &
        \noindent\parbox[c]{0.14\columnwidth}{\includegraphics[width=0.14\columnwidth]{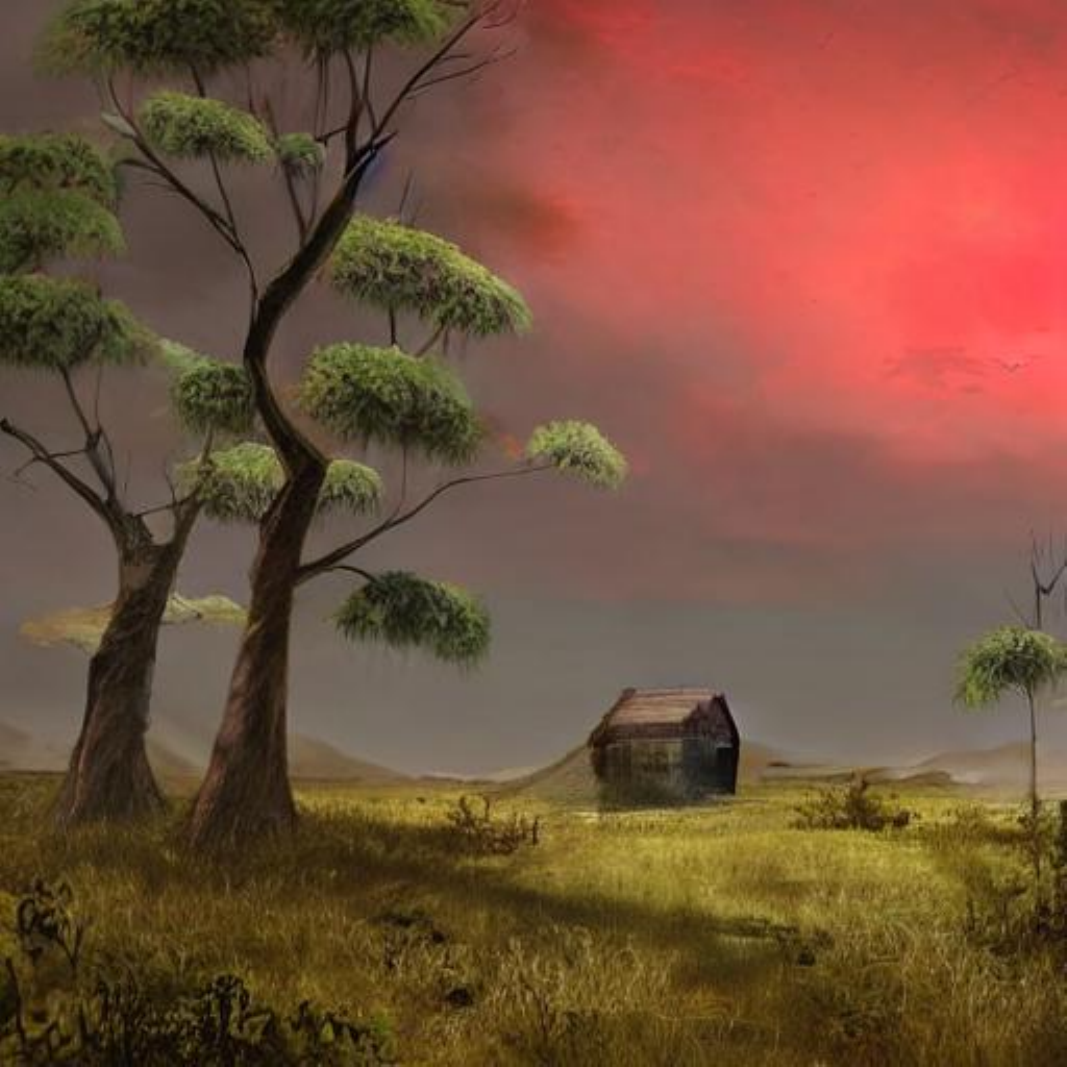}} & 
        \noindent\parbox[c]{0.14\columnwidth}{\includegraphics[width=0.14\columnwidth]{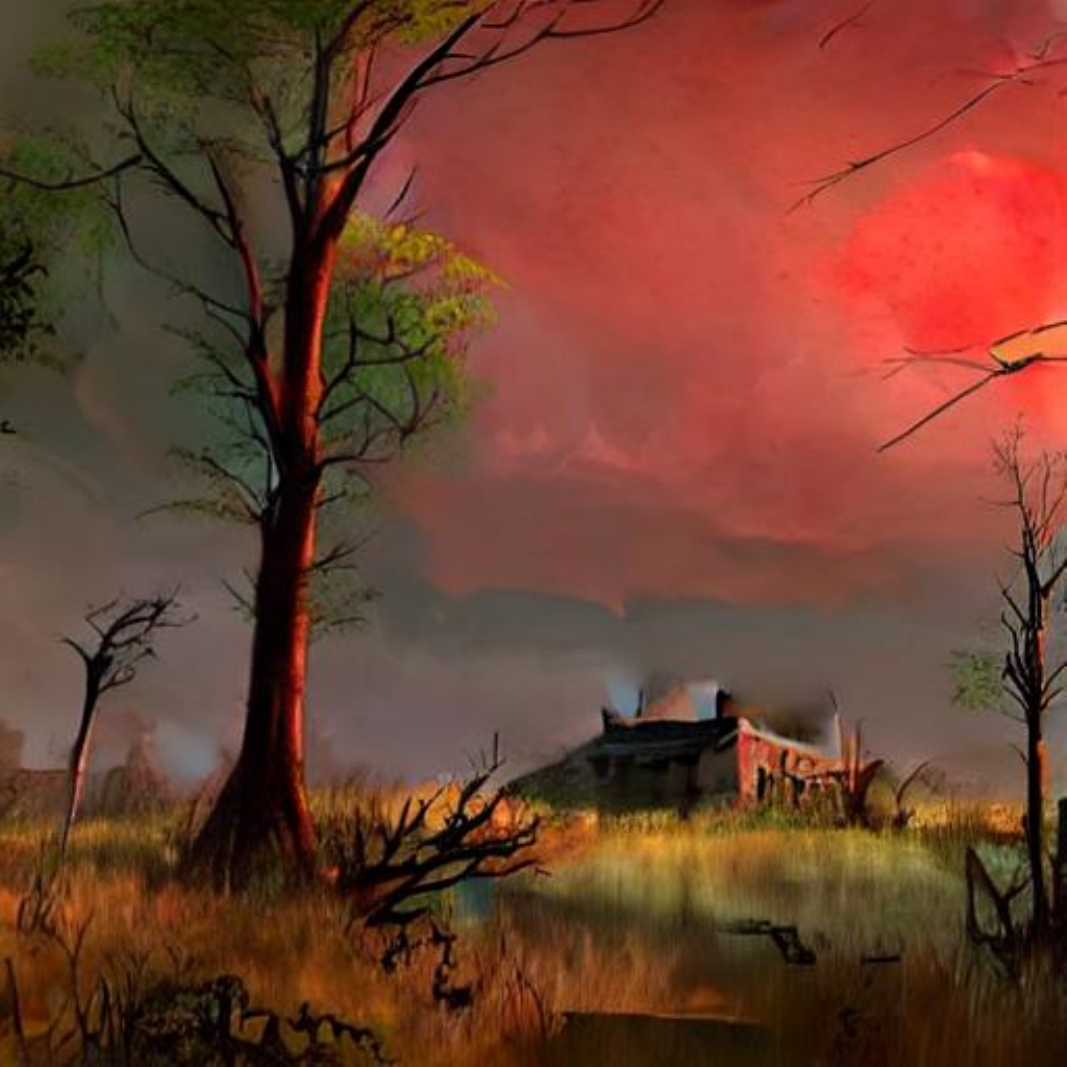}} & 
        \noindent\parbox[c]{0.14\columnwidth}{\includegraphics[width=0.14\columnwidth]{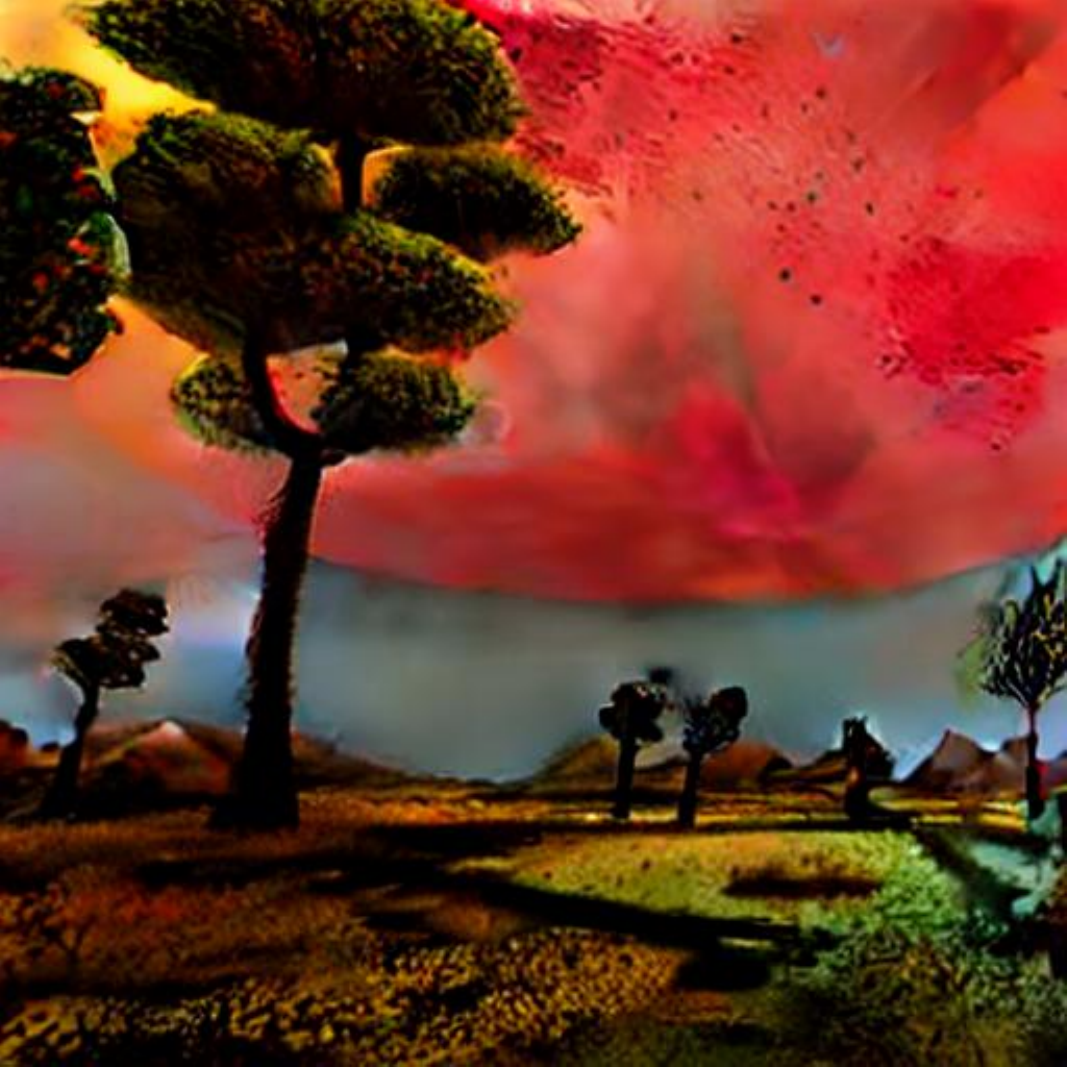}} & 
        \noindent\parbox[c]{0.14\columnwidth}{\includegraphics[width=0.14\columnwidth]{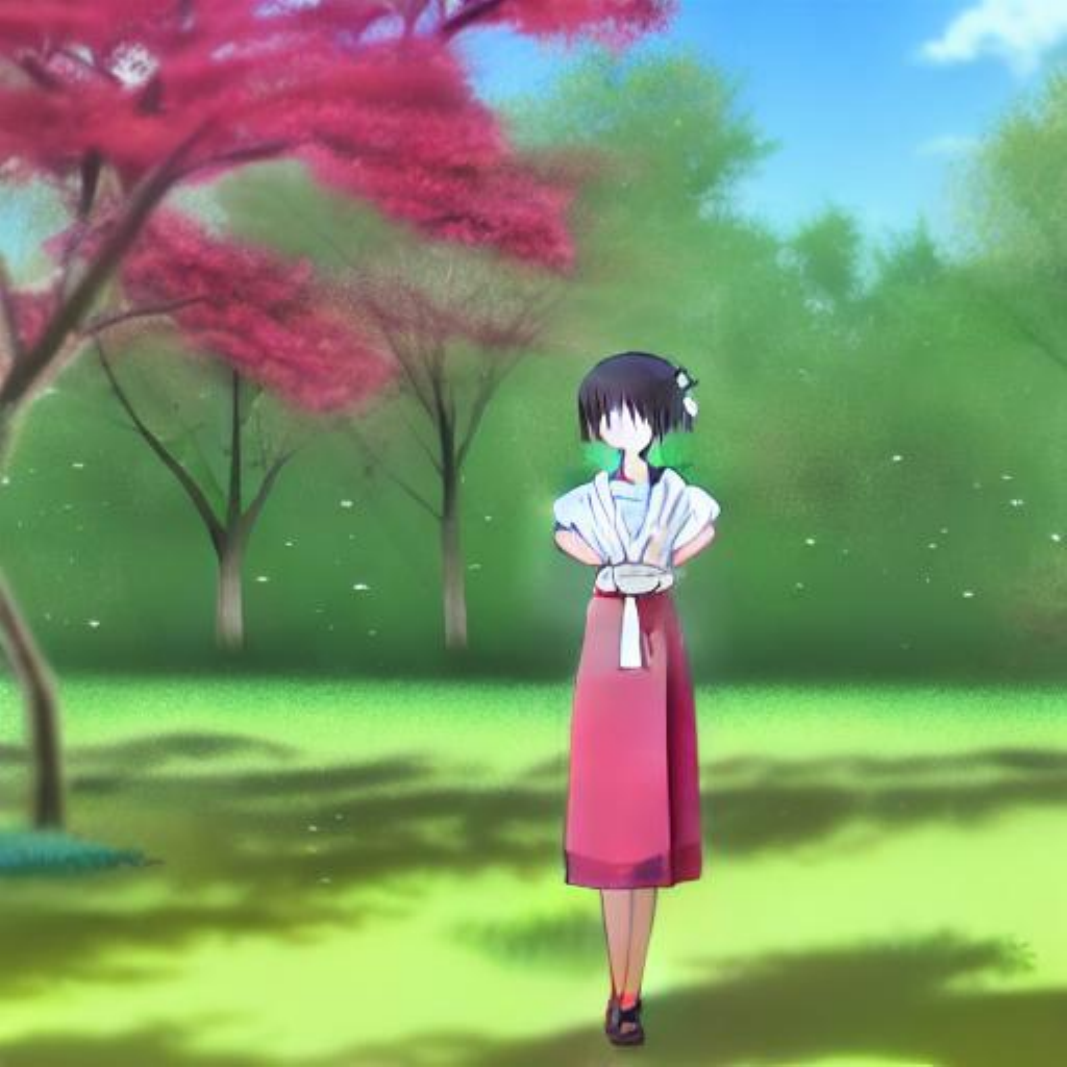}} & 
        \noindent\parbox[c]{0.14\columnwidth}{\includegraphics[width=0.14\columnwidth]{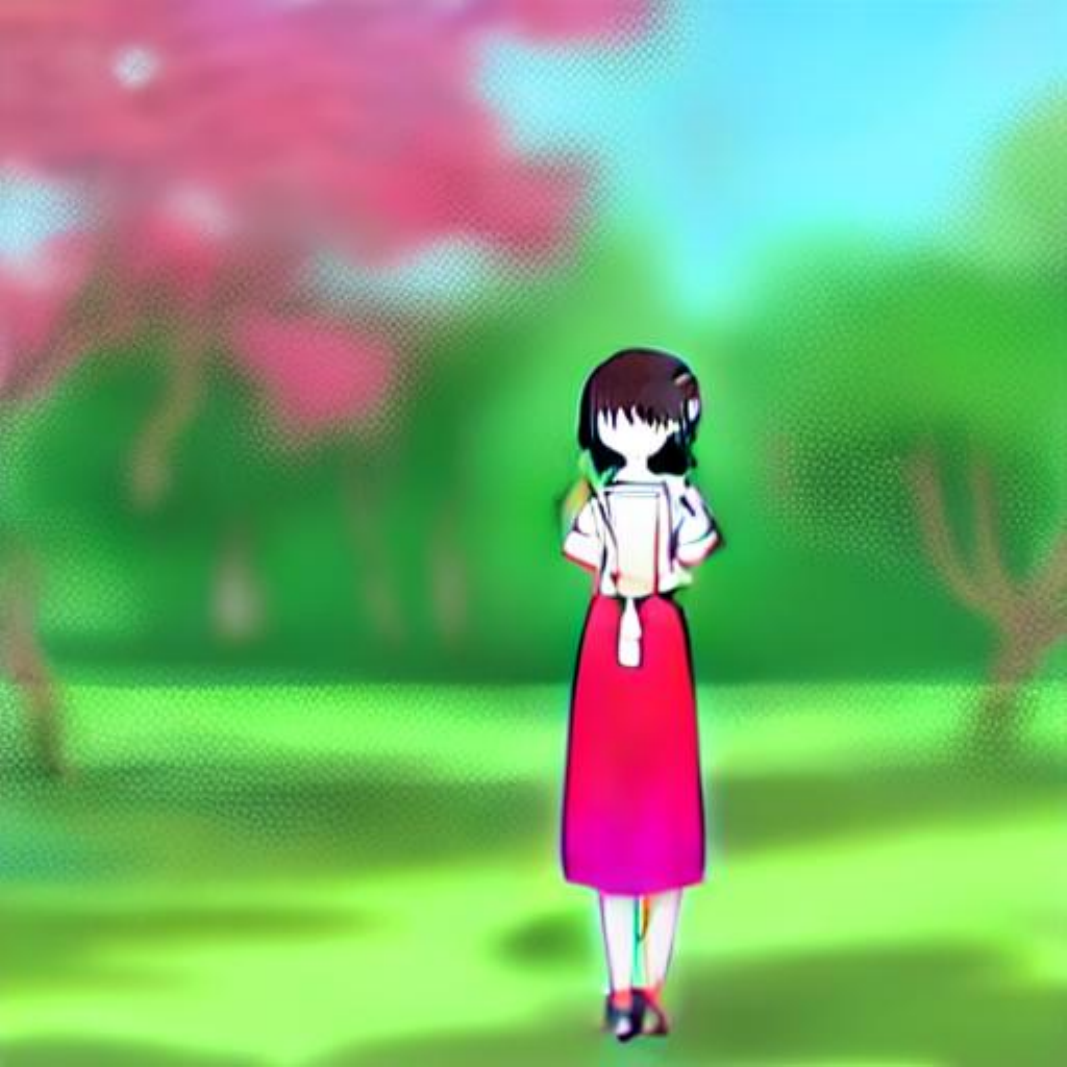}} & 
        \noindent\parbox[c]{0.14\columnwidth}{\includegraphics[width=0.14\columnwidth]{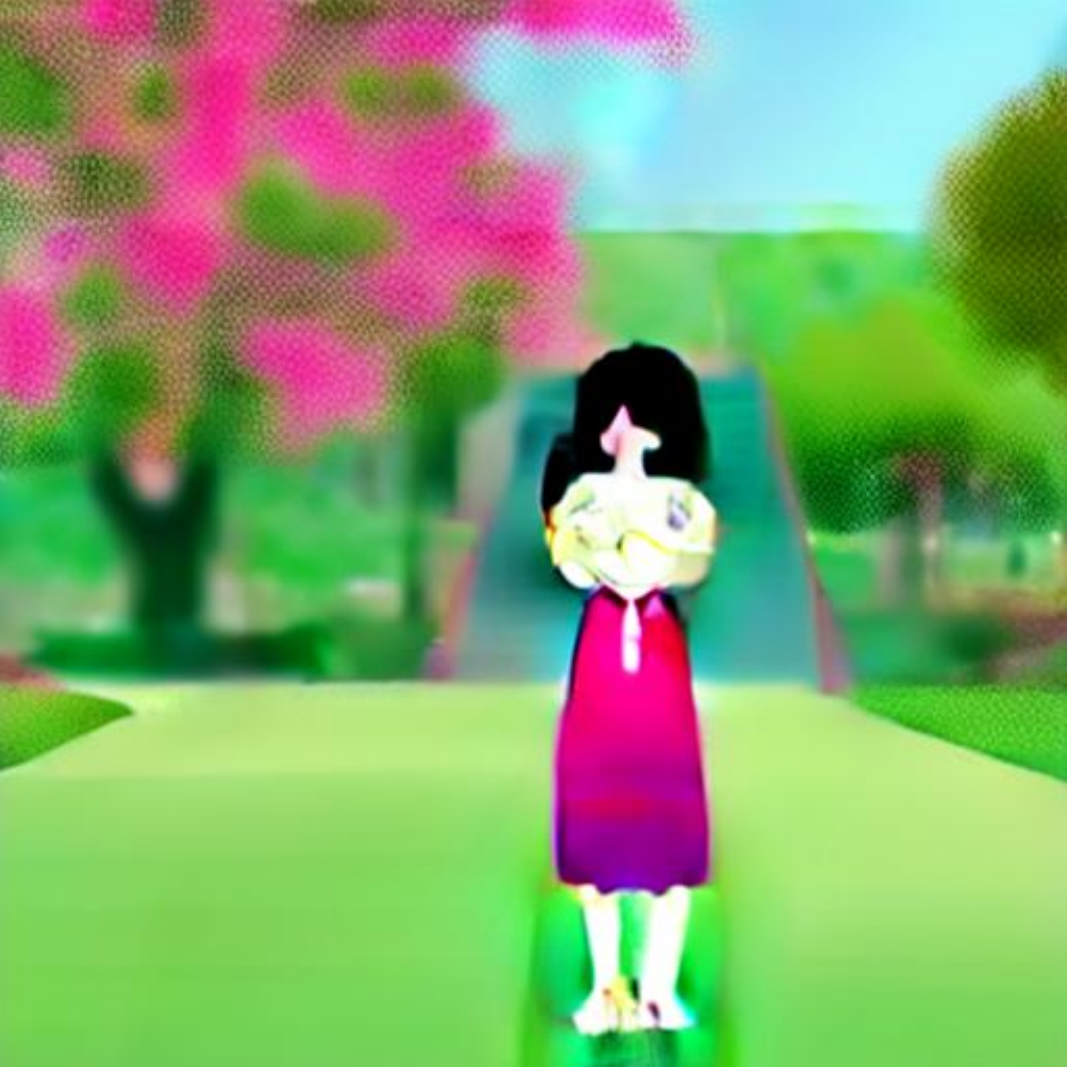}} \\

        \shortstack[l]{\tiny 20 steps} &
        \noindent\parbox[c]{0.14\columnwidth}{\includegraphics[width=0.14\columnwidth]{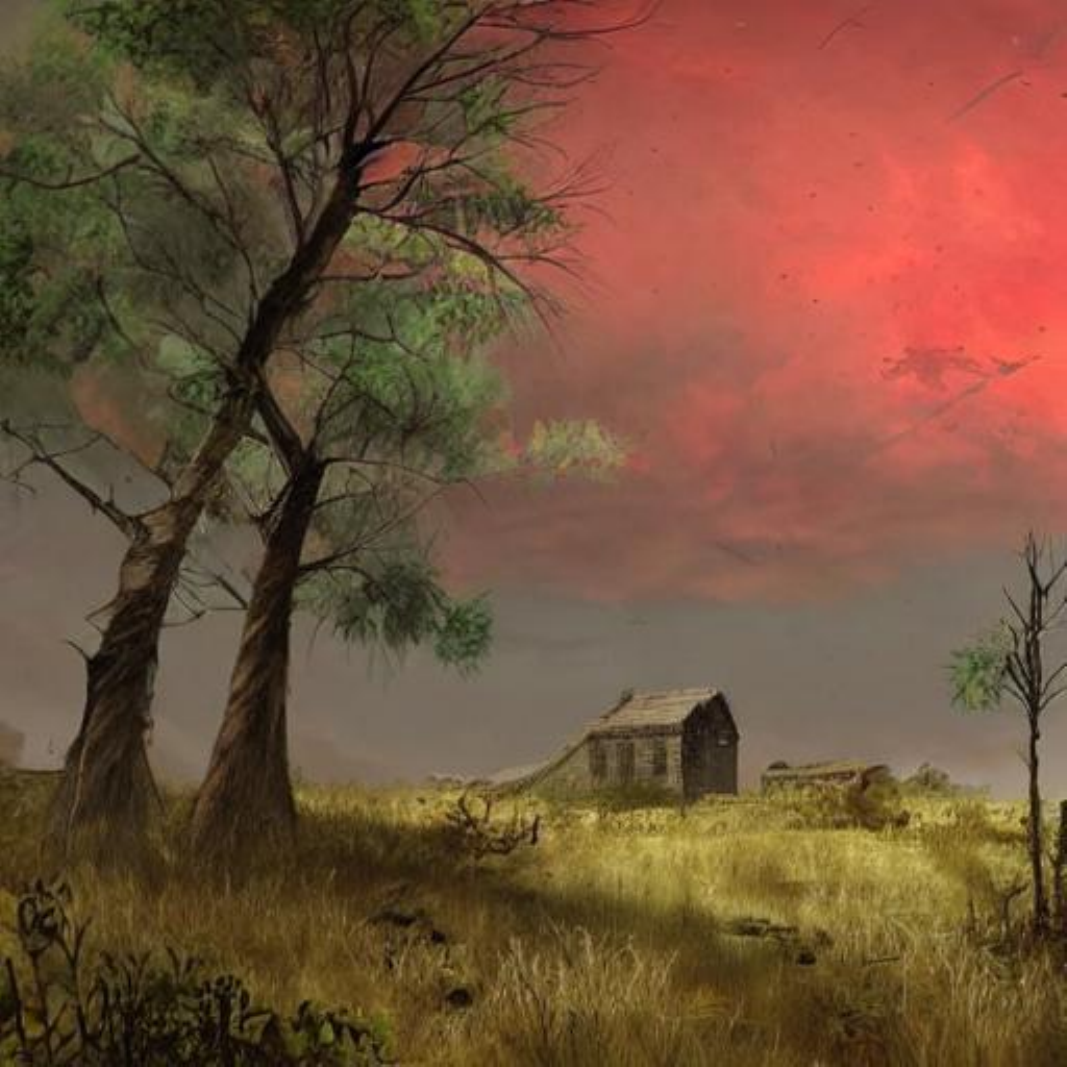}} & 
        \noindent\parbox[c]{0.14\columnwidth}{\includegraphics[width=0.14\columnwidth]{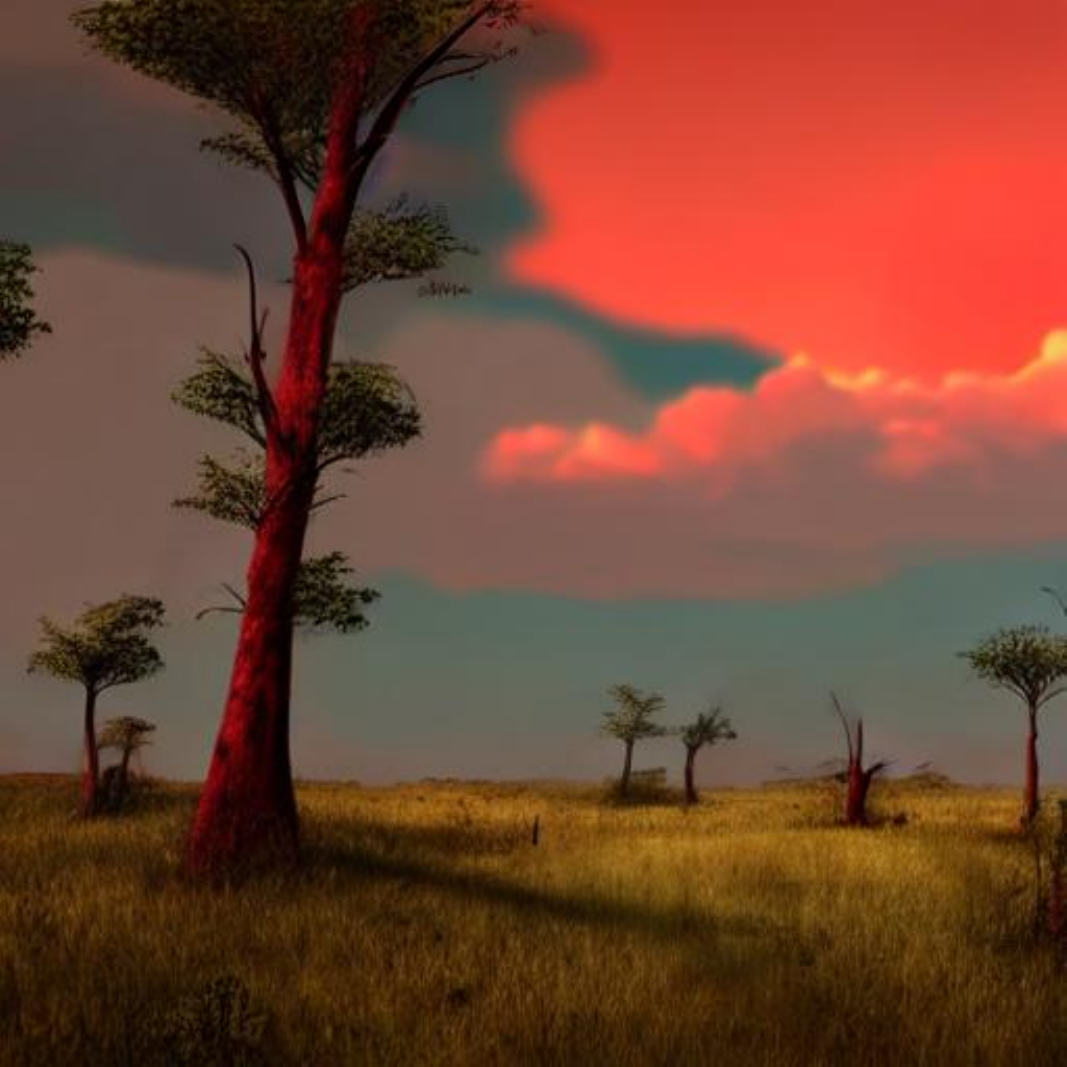}} & 
        \noindent\parbox[c]{0.14\columnwidth}{\includegraphics[width=0.14\columnwidth]{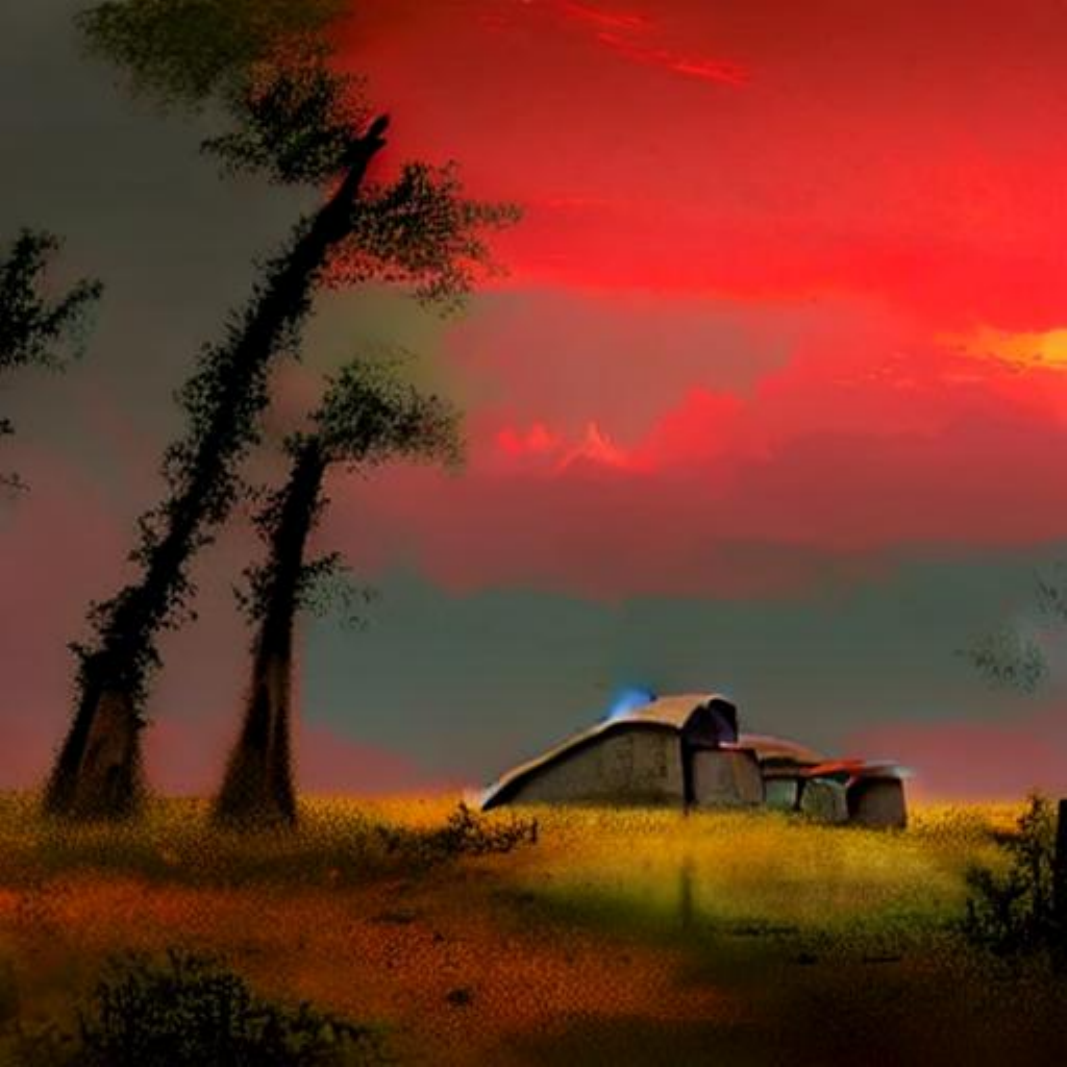}} & 
        \noindent\parbox[c]{0.14\columnwidth}{\includegraphics[width=0.14\columnwidth]{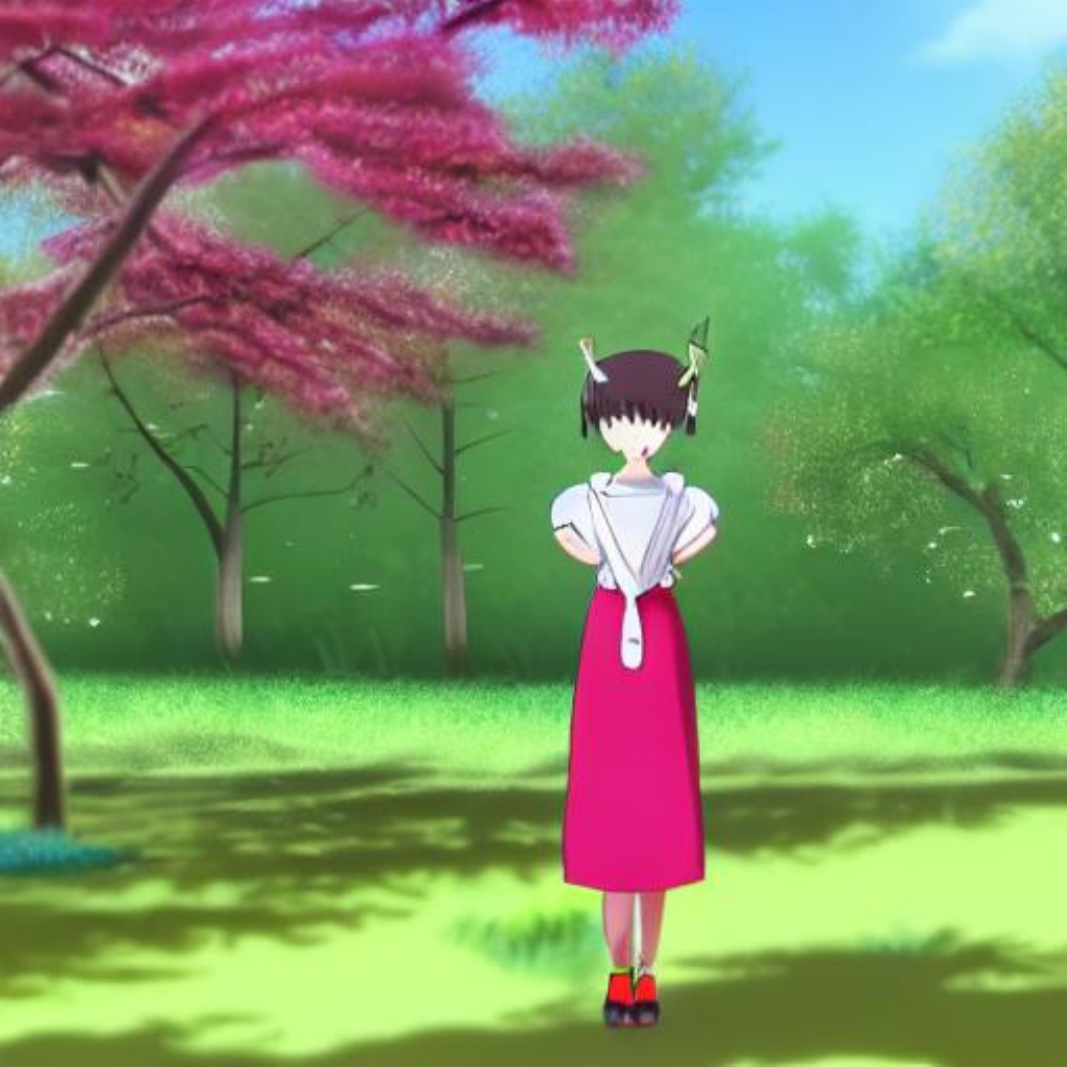}} & 
        \noindent\parbox[c]{0.14\columnwidth}{\includegraphics[width=0.14\columnwidth]{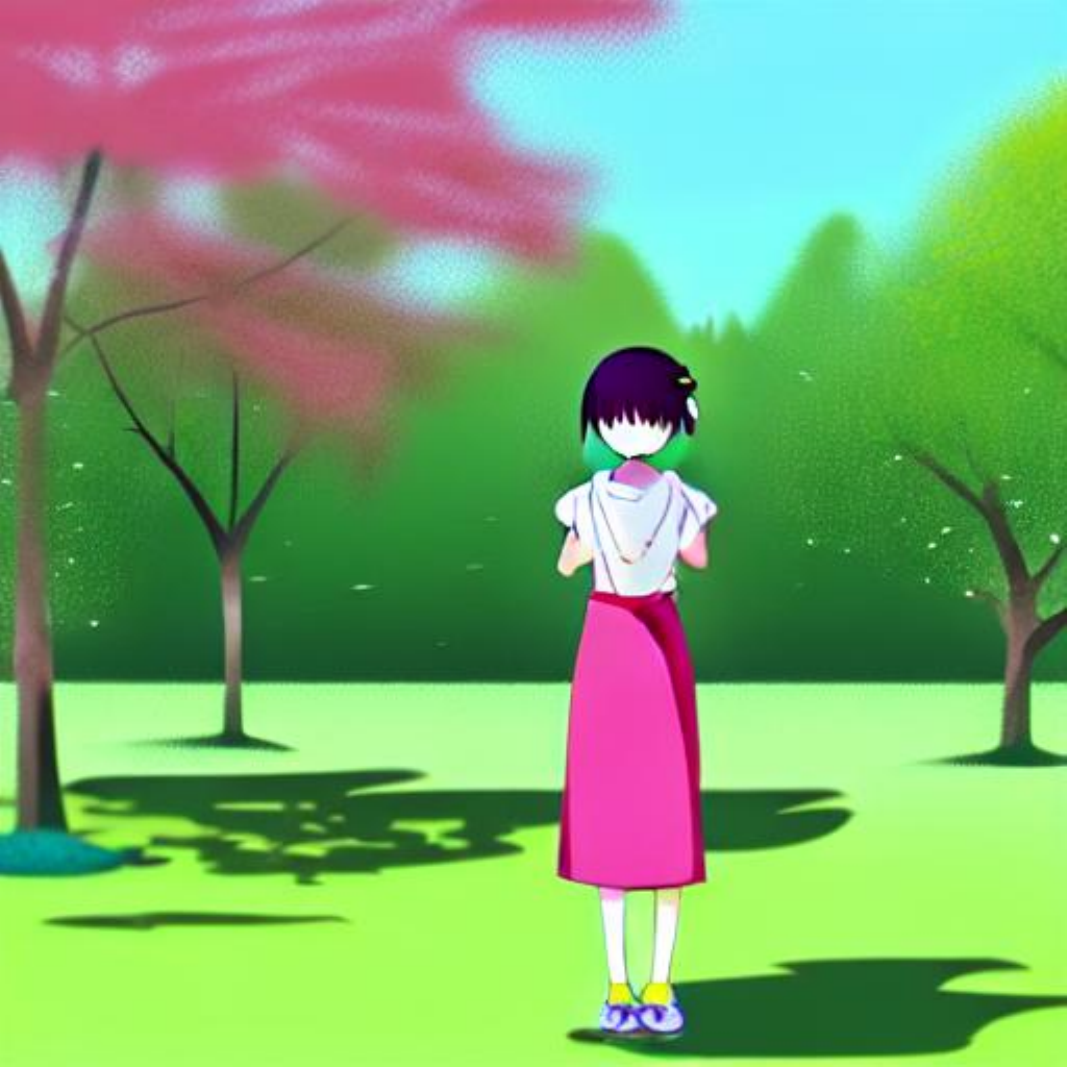}} & 
        \noindent\parbox[c]{0.14\columnwidth}{\includegraphics[width=0.14\columnwidth]{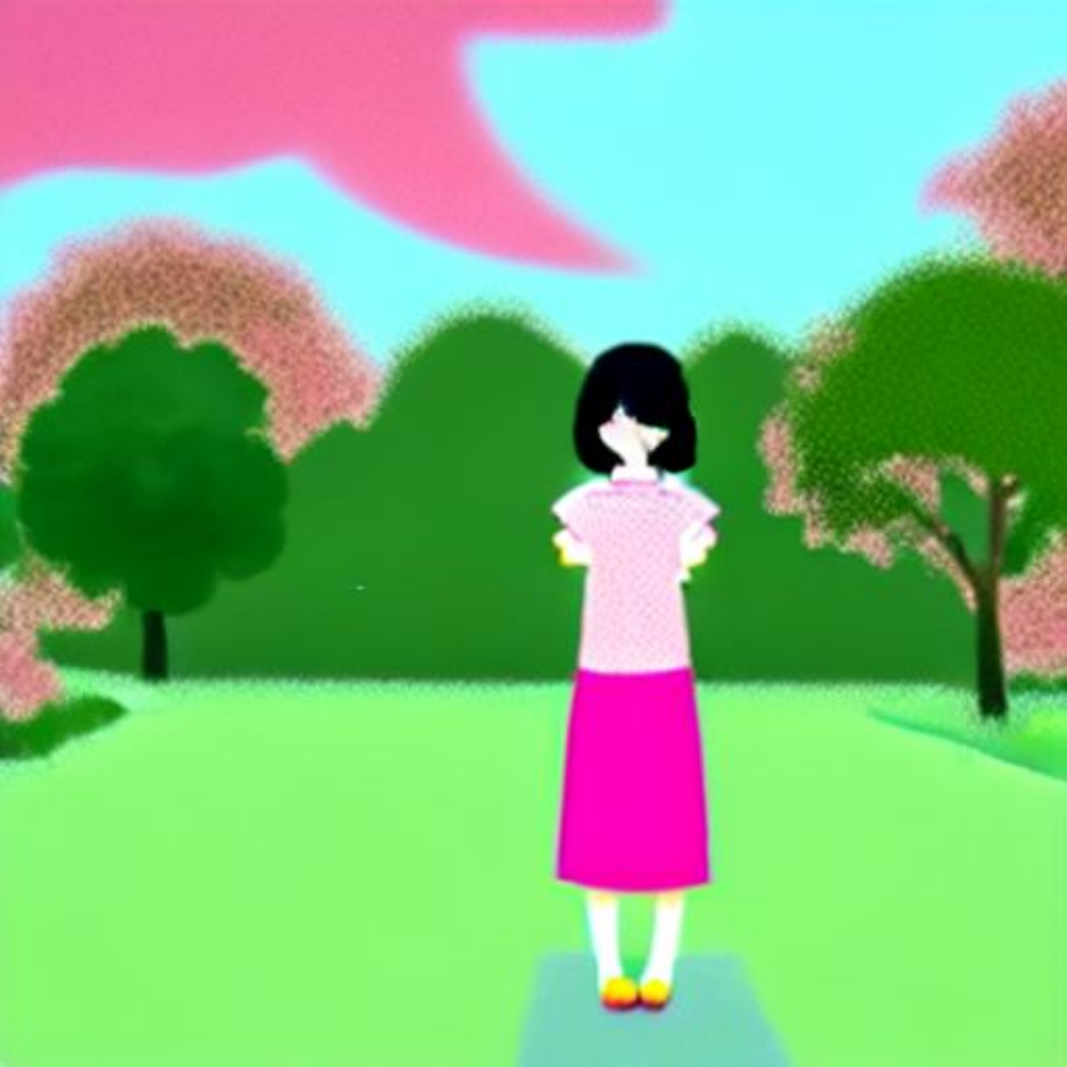}} \\

        \shortstack[l]{\tiny 40 steps} &
        \noindent\parbox[c]{0.14\columnwidth}{\includegraphics[width=0.14\columnwidth]{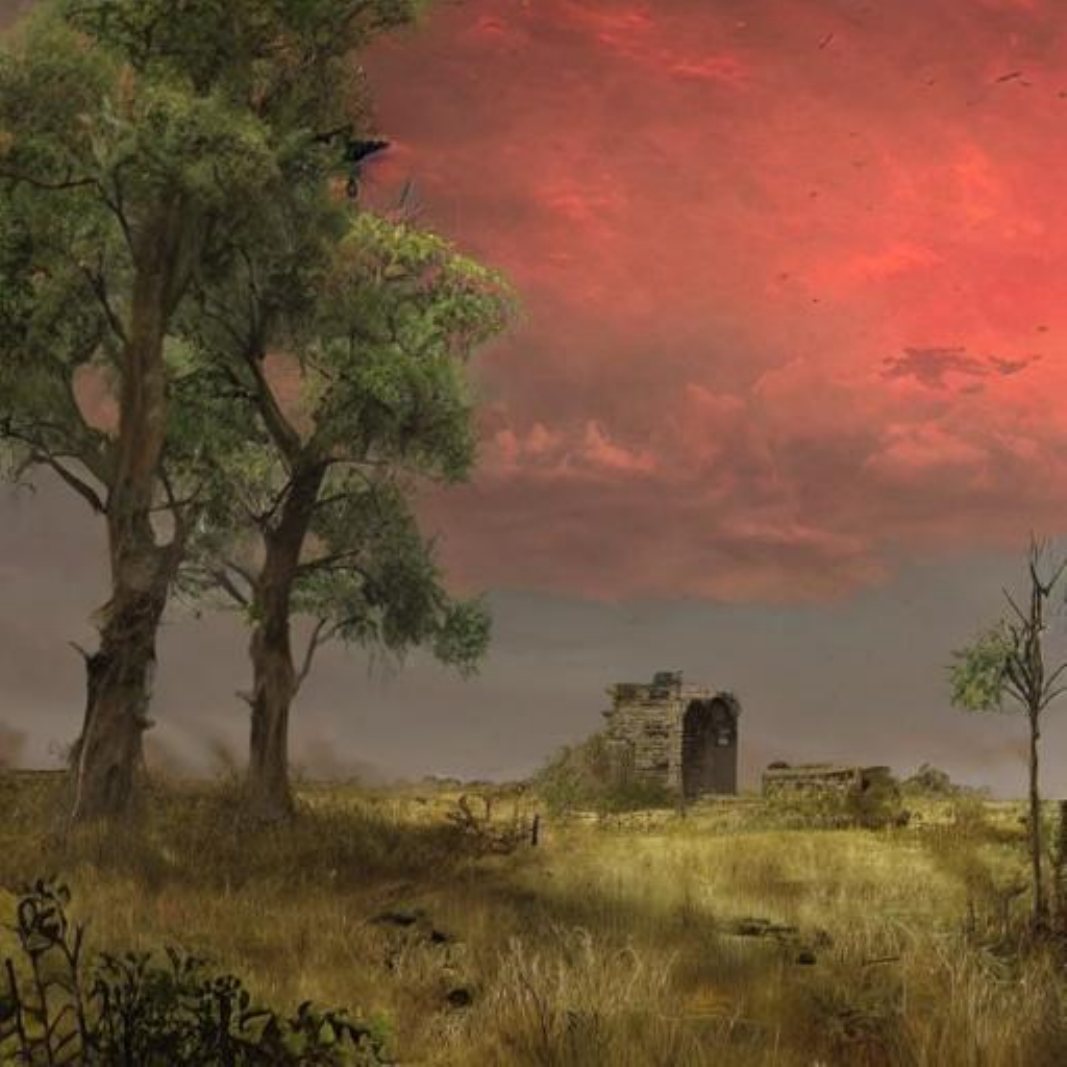}} & 
        \noindent\parbox[c]{0.14\columnwidth}{\includegraphics[width=0.14\columnwidth]{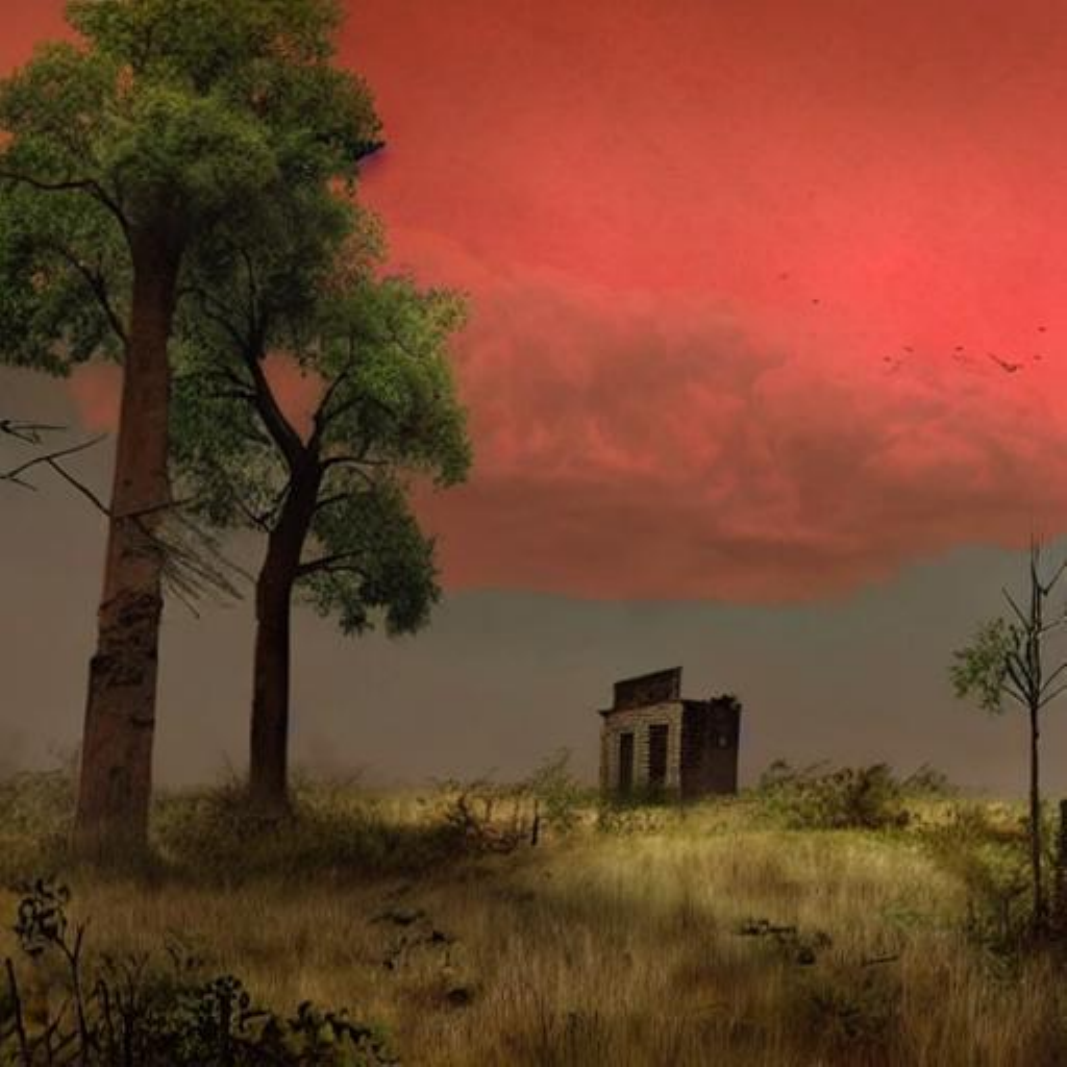}} & 
        \noindent\parbox[c]{0.14\columnwidth}{\includegraphics[width=0.14\columnwidth]{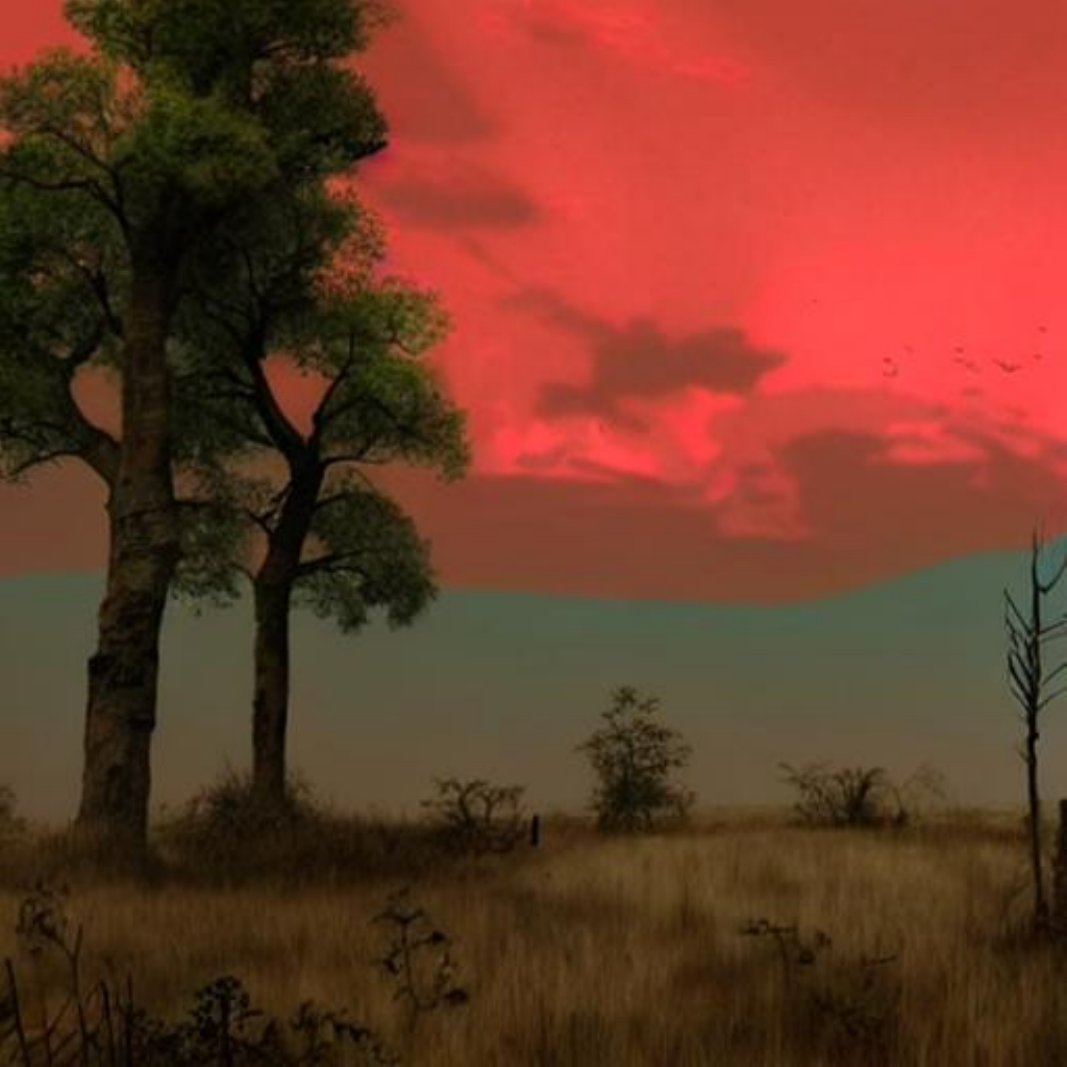}} & 
        \noindent\parbox[c]{0.14\columnwidth}{\includegraphics[width=0.14\columnwidth]{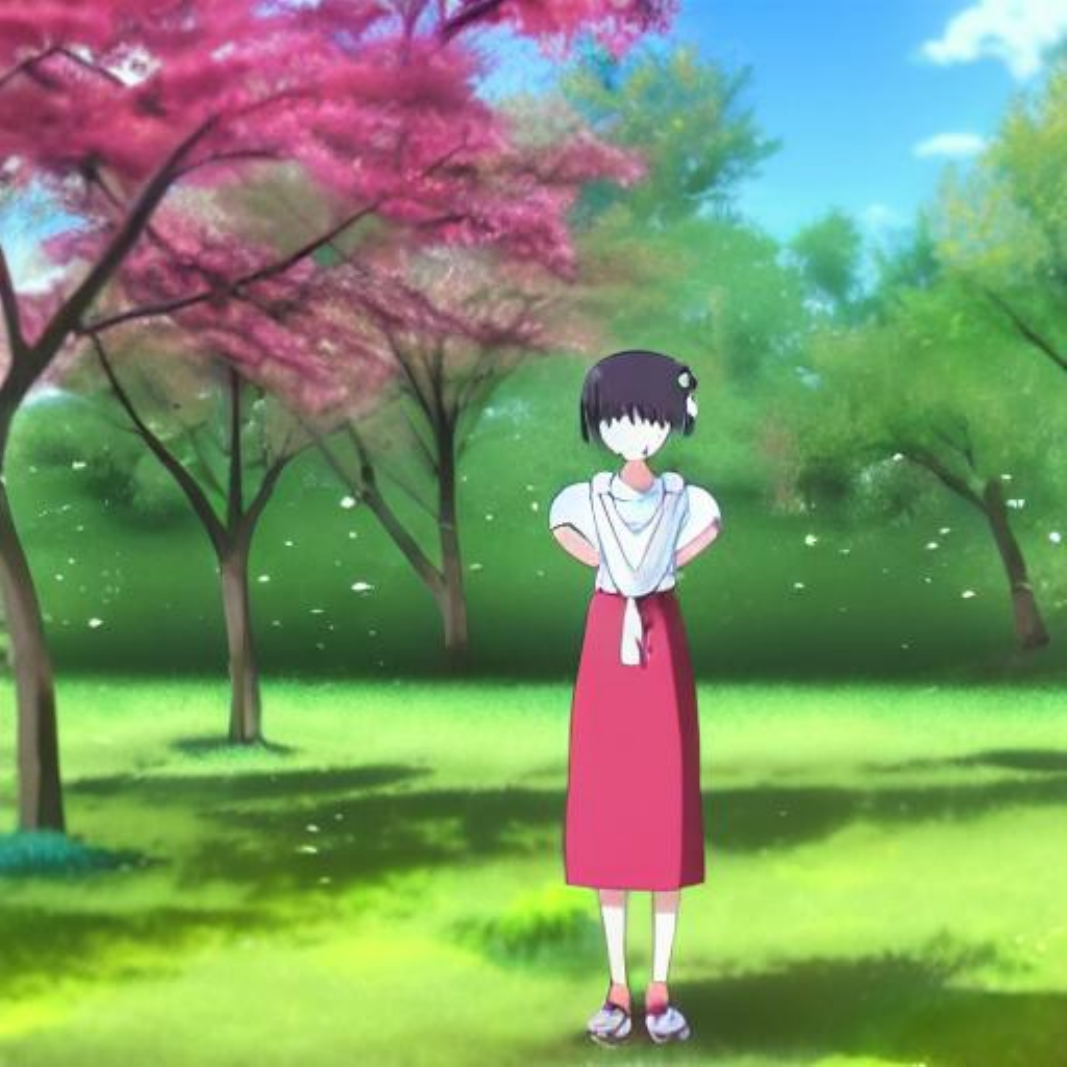}} & 
        \noindent\parbox[c]{0.14\columnwidth}{\includegraphics[width=0.14\columnwidth]{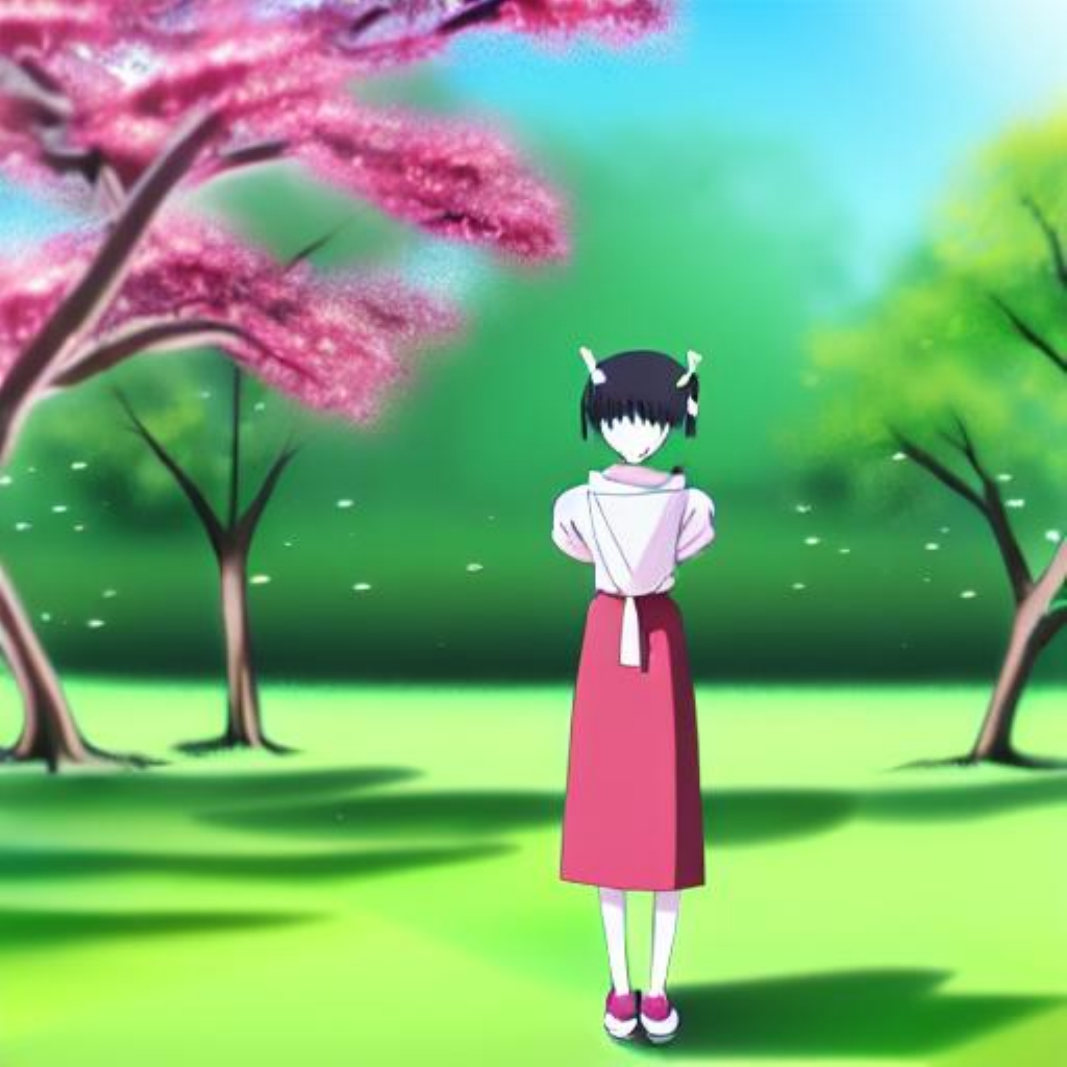}} & 
        \noindent\parbox[c]{0.14\columnwidth}{\includegraphics[width=0.14\columnwidth]{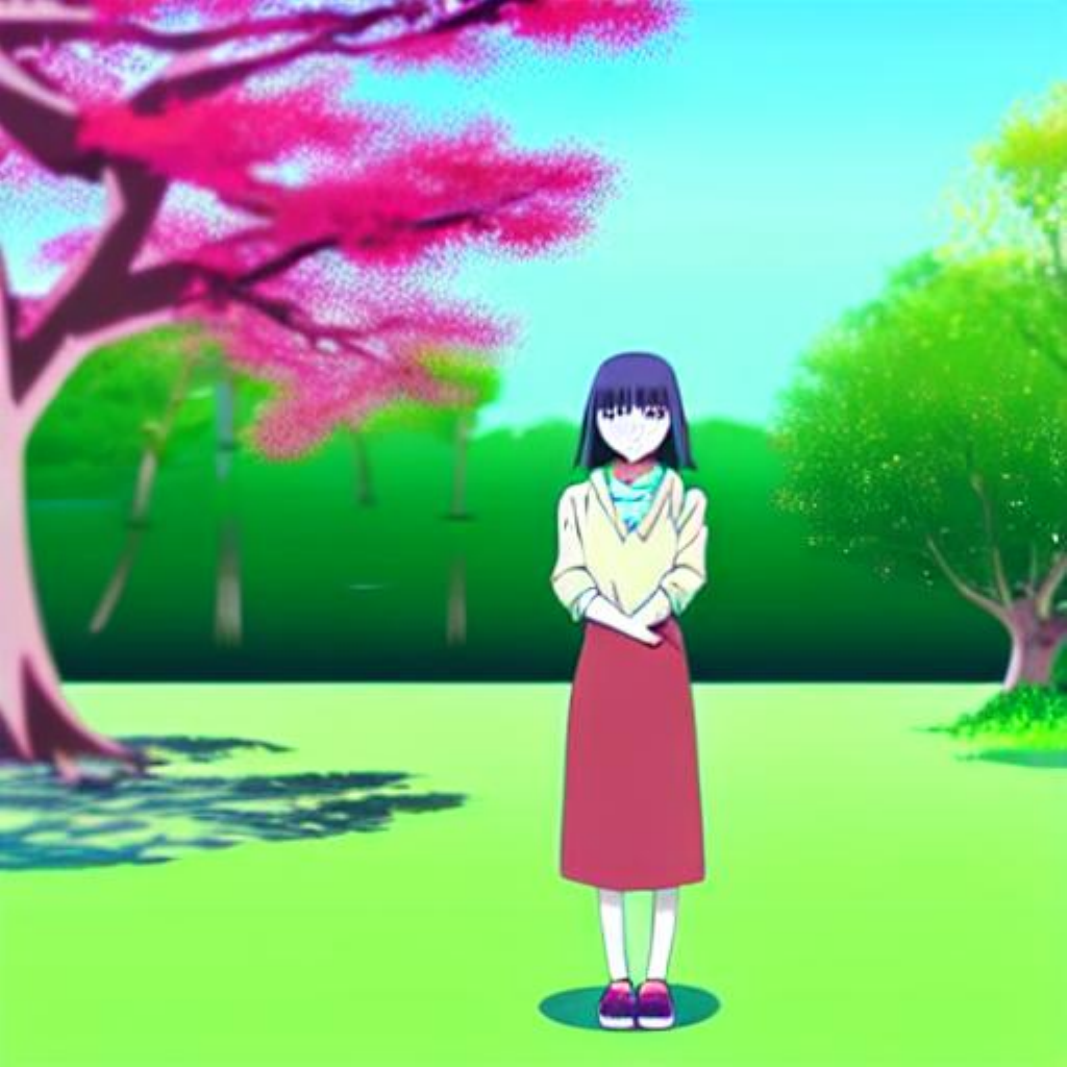}} \\
    \end{tabu}
    \caption{Comparison of samples generated from Stable Diffusion 1.5 \protect\footnotemark using PLMS4 \cite{liu2022pseudo} with different sampling steps and guidance scale.}
    \label{fig:scale_step_sd15}
\end{figure}

\footnotetext{\url{https://huggingface.co/runwayml/stable-diffusion-v1-5}}


\tabulinesep=1pt
\begin{figure}
    \centering
    \begin{tabu} to \textwidth {@{}l@{\hspace{5pt}}c@{\hspace{2pt}}c@{\hspace{2pt}}c@{\hspace{4pt}}c@{\hspace{2pt}}c@{\hspace{2pt}}c@{}}
        & \multicolumn{3}{c}{\shortstack{\scriptsize "A post-apocalyptic world with ruined \\ \scriptsize buildings, overgrown vegetation, and a red sky"}}
        & \multicolumn{3}{c}{\shortstack{\scriptsize "A girl standing in a park in \\ \scriptsize Japanese animation style"}} \\


        & \multicolumn{1}{c}{\shortstack{\scriptsize $s = 7.5$}}
        & \multicolumn{1}{c}{\shortstack{\scriptsize $s = 15$}}
        & \multicolumn{1}{c}{\shortstack{\scriptsize $s = 22.5$}}
        & \multicolumn{1}{c}{\shortstack{\scriptsize $s = 7.5$}}
        & \multicolumn{1}{c}{\shortstack{\scriptsize $s = 15$}}
        & \multicolumn{1}{c}{\shortstack{\scriptsize $s = 22.5$}}
        \\
        
        \shortstack[l]{\tiny 10 steps} &
        \noindent\parbox[c]{0.14\columnwidth}{\includegraphics[width=0.14\columnwidth]{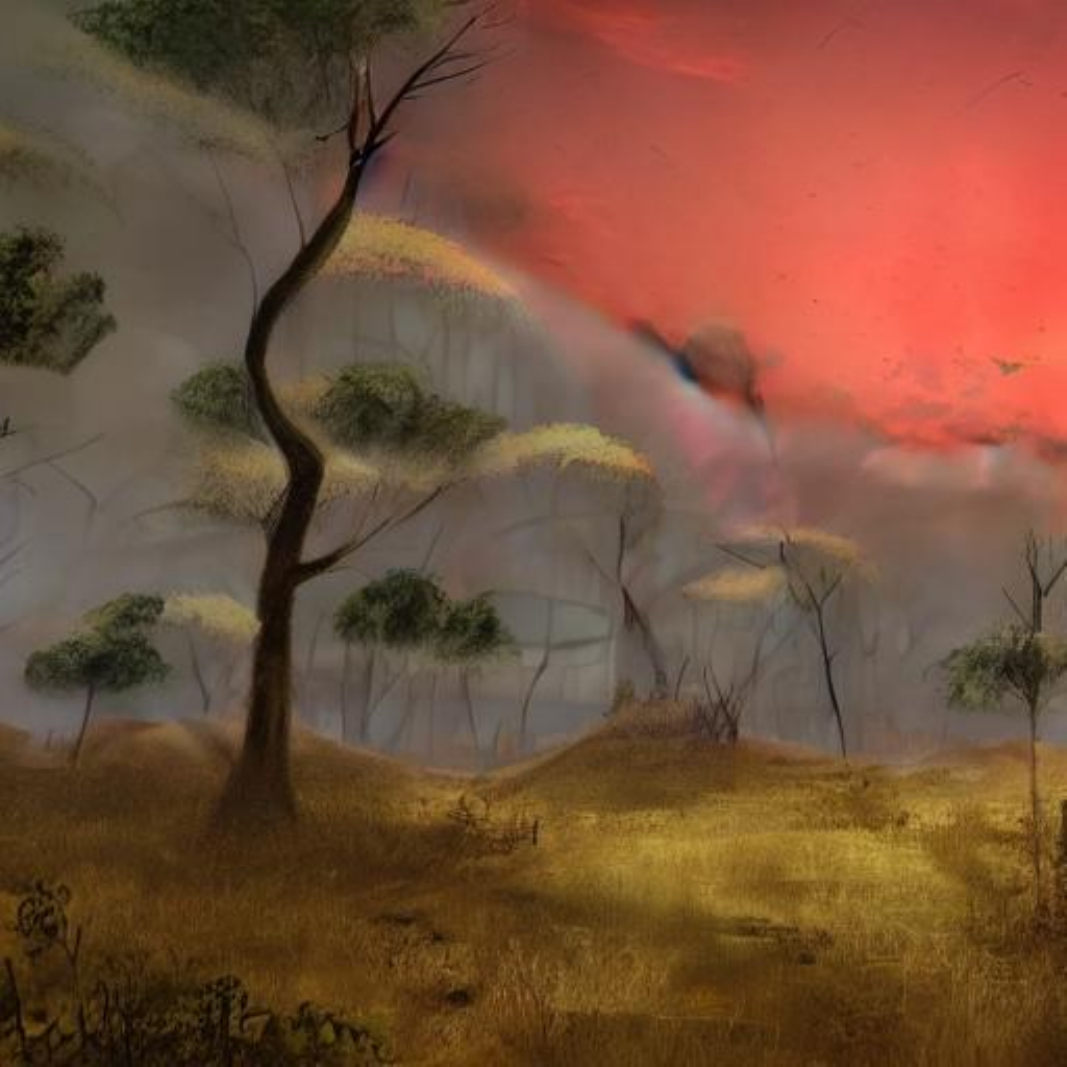}} & 
        \noindent\parbox[c]{0.14\columnwidth}{\includegraphics[width=0.14\columnwidth]{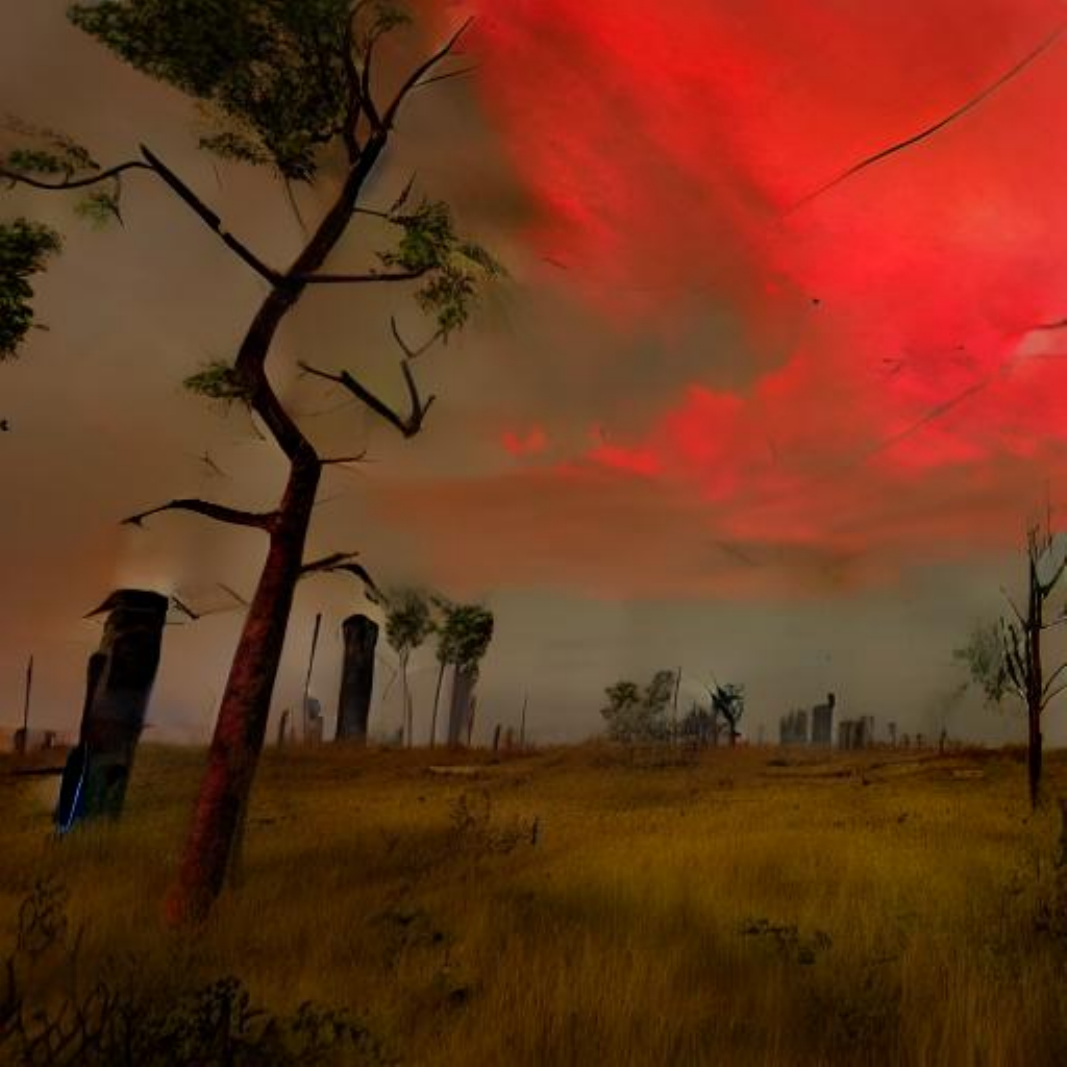}} & 
        \noindent\parbox[c]{0.14\columnwidth}{\includegraphics[width=0.14\columnwidth]{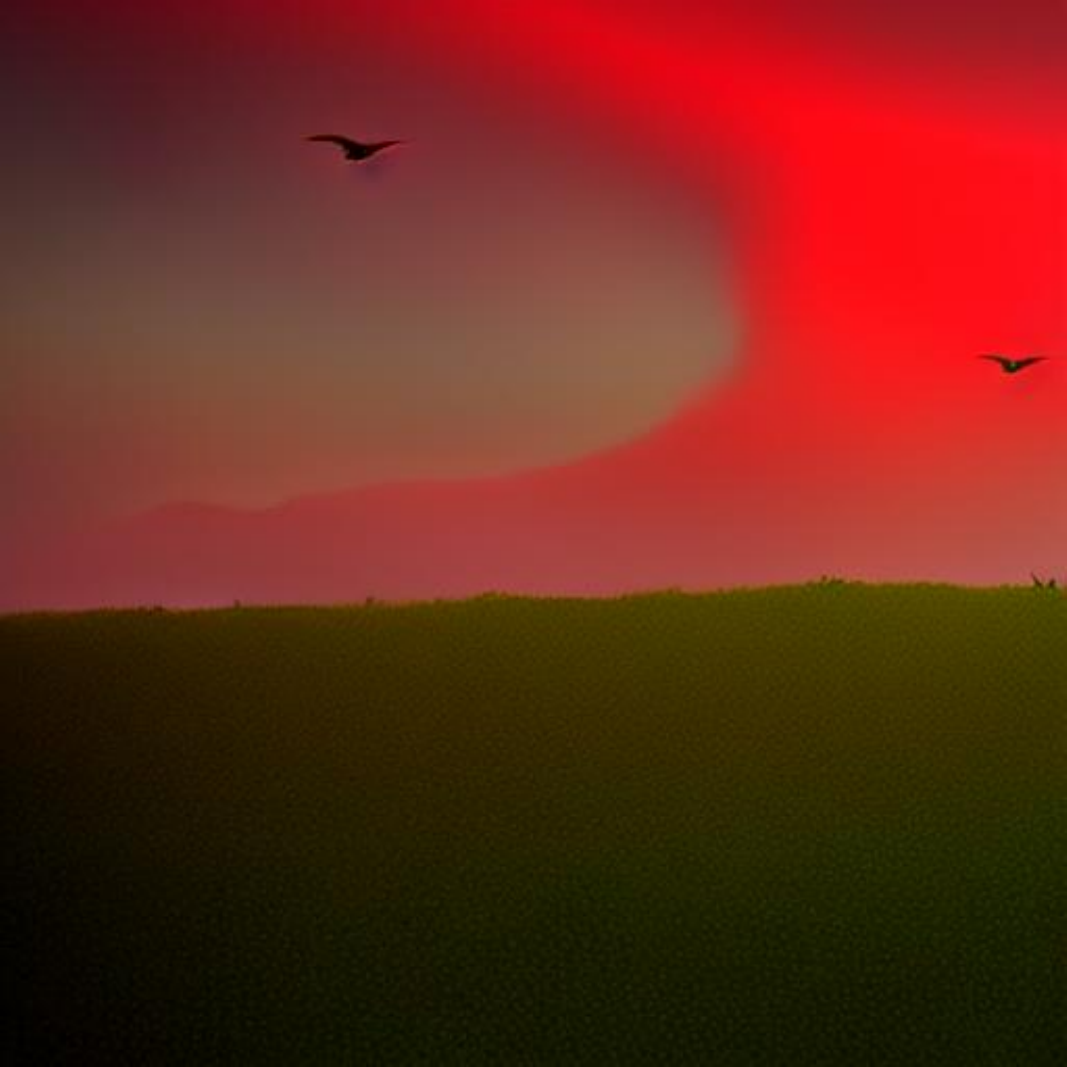}} & 
        \noindent\parbox[c]{0.14\columnwidth}{\includegraphics[width=0.14\columnwidth]{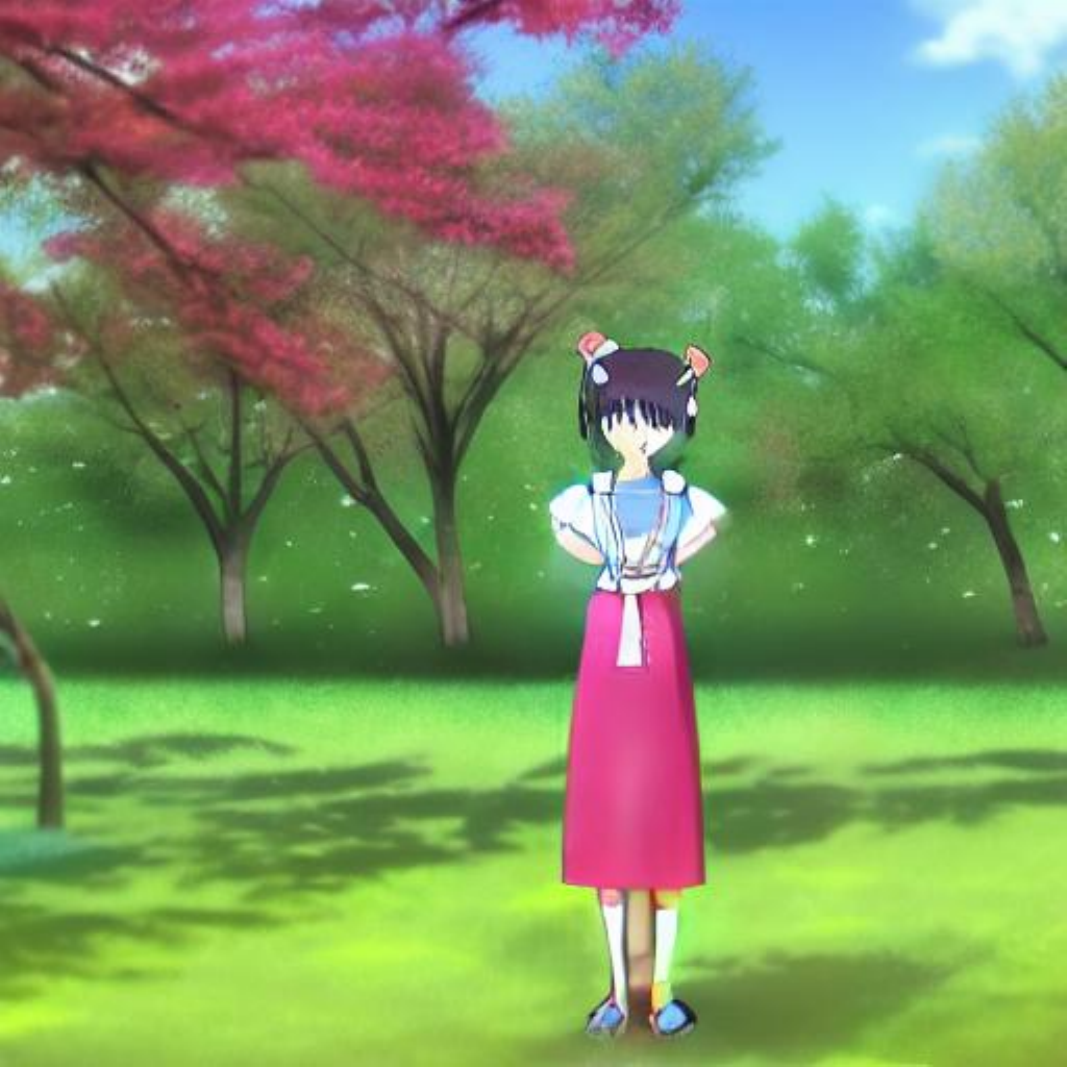}} & 
        \noindent\parbox[c]{0.14\columnwidth}{\includegraphics[width=0.14\columnwidth]{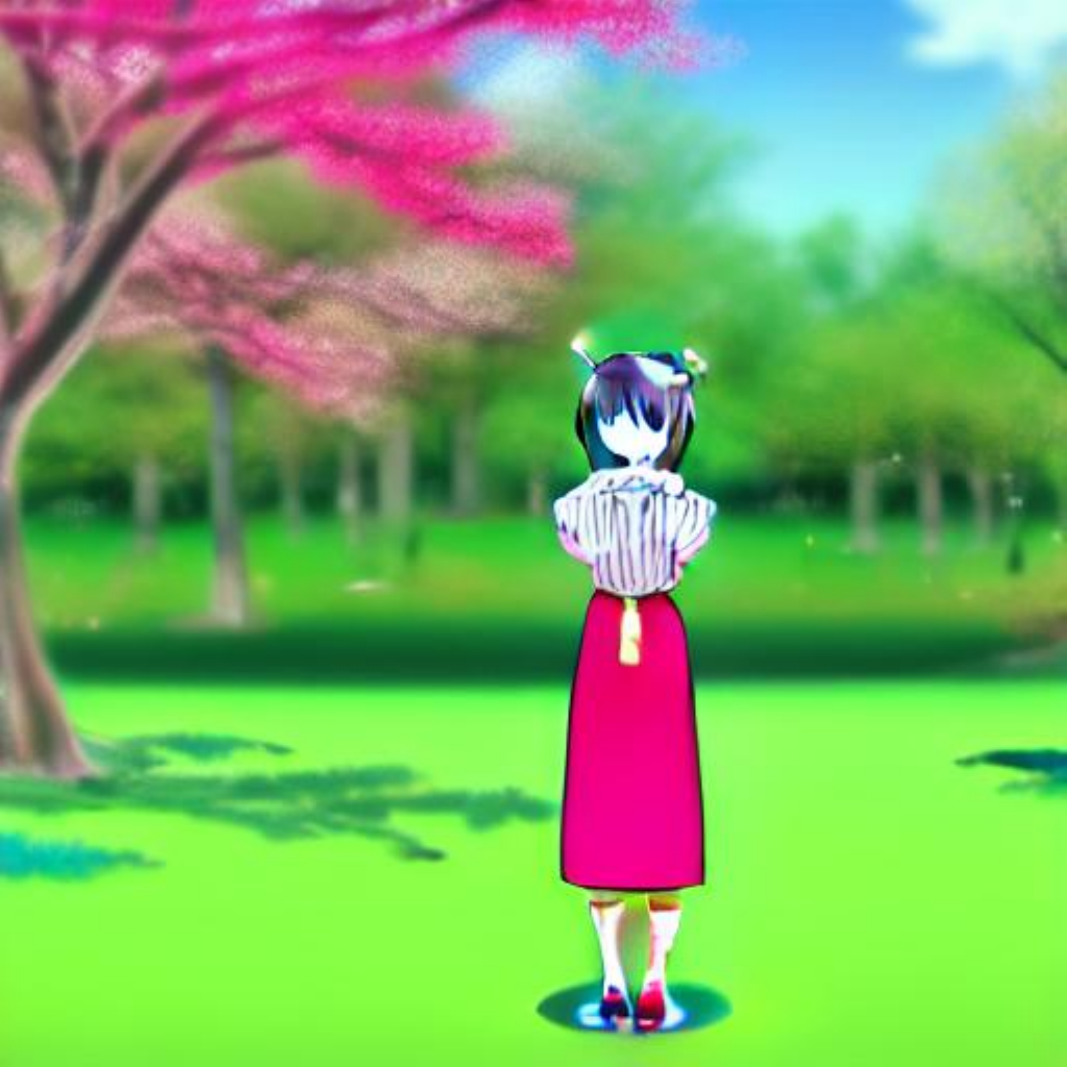}} & 
        \noindent\parbox[c]{0.14\columnwidth}{\includegraphics[width=0.14\columnwidth]{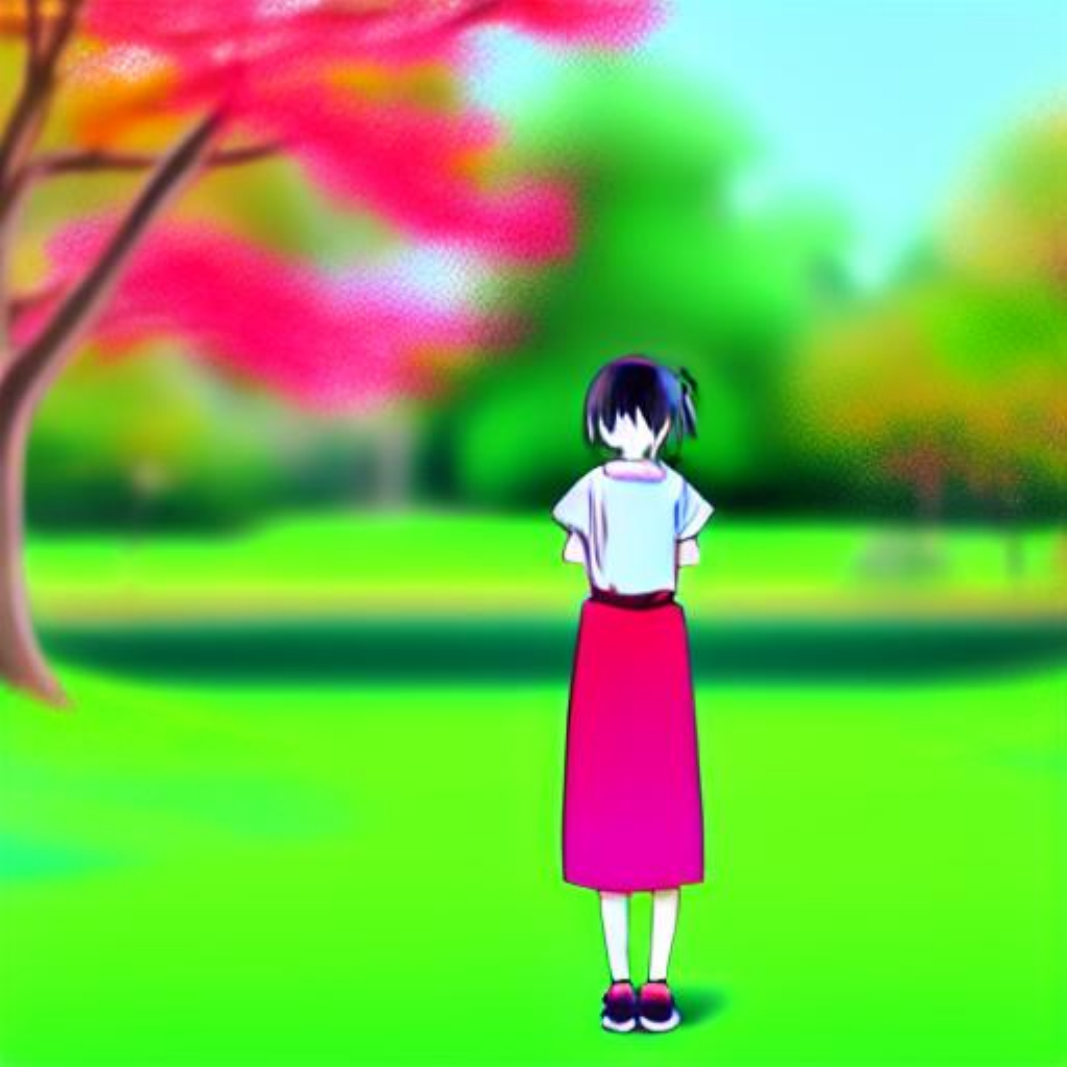}} \\

        \shortstack[l]{\tiny 15 steps} &
        \noindent\parbox[c]{0.14\columnwidth}{\includegraphics[width=0.14\columnwidth]{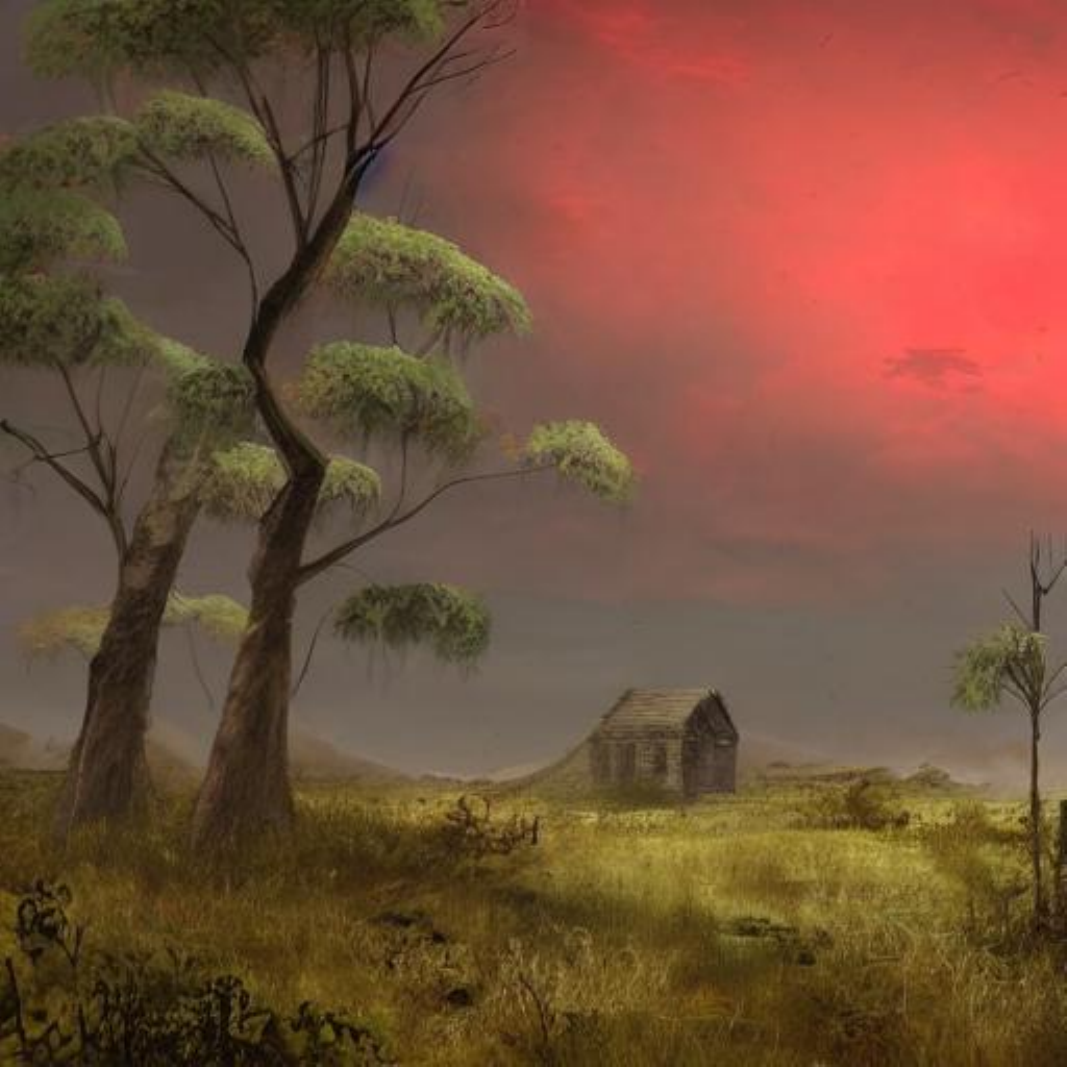}} & 
        \noindent\parbox[c]{0.14\columnwidth}{\includegraphics[width=0.14\columnwidth]{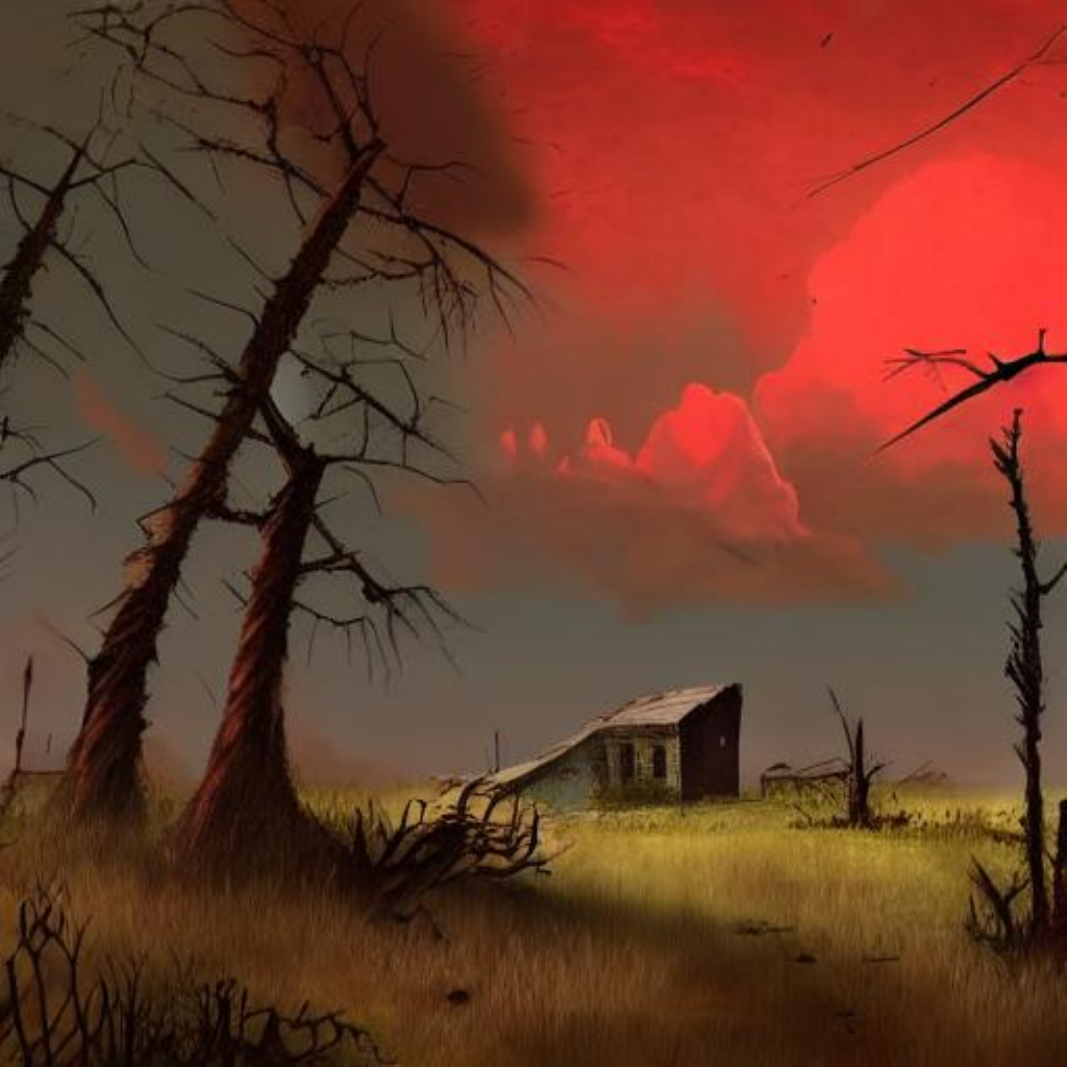}} & 
        \noindent\parbox[c]{0.14\columnwidth}{\includegraphics[width=0.14\columnwidth]{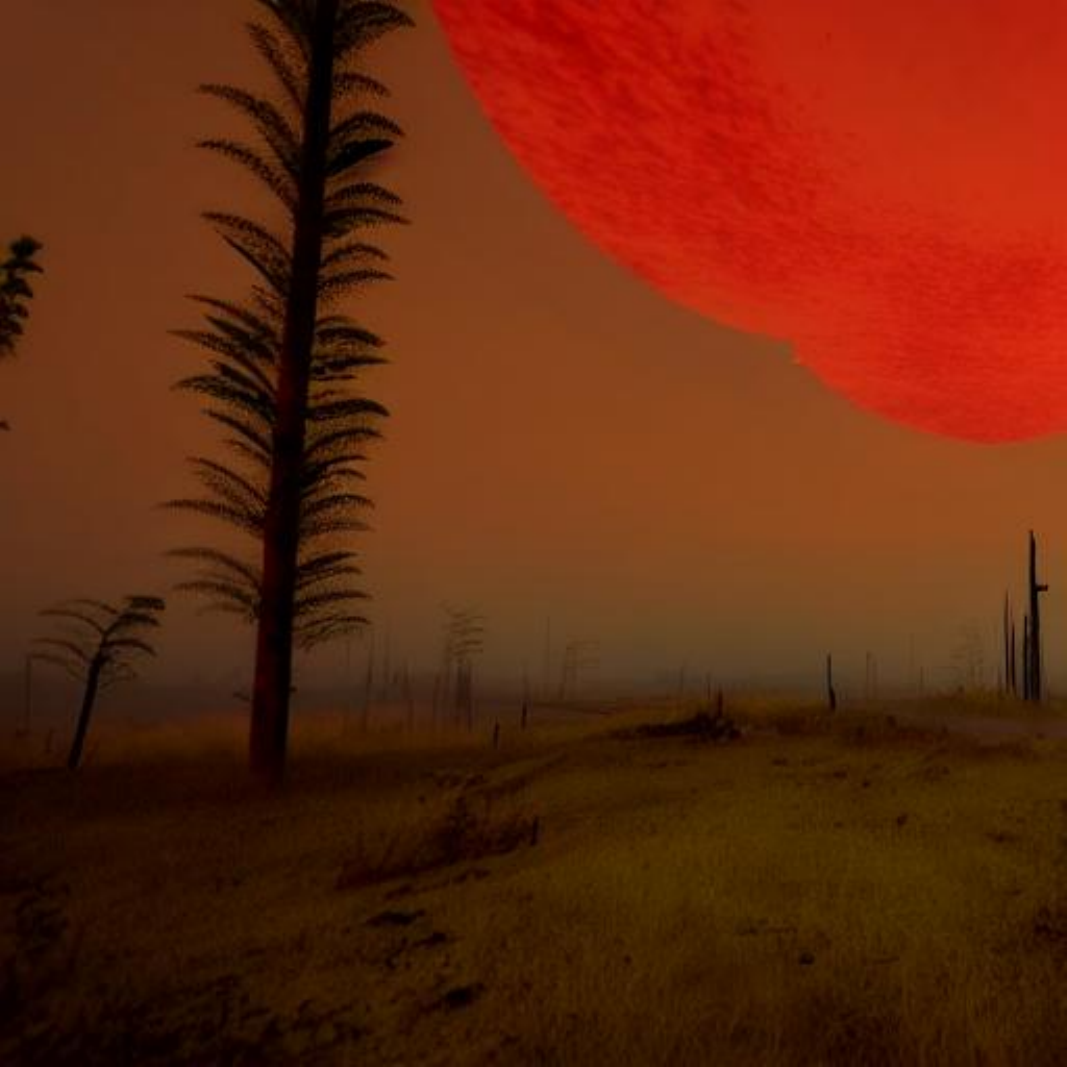}} & 
        \noindent\parbox[c]{0.14\columnwidth}{\includegraphics[width=0.14\columnwidth]{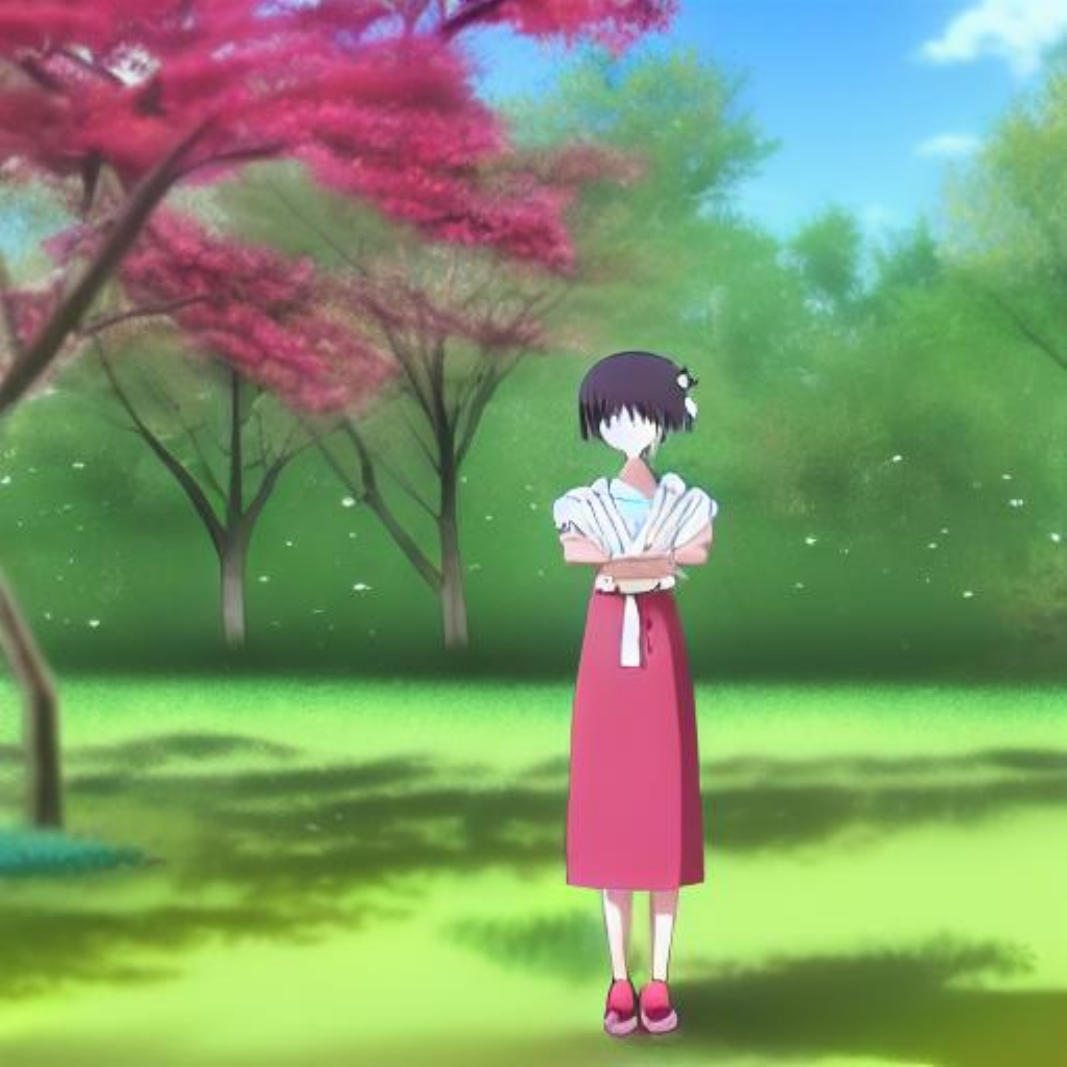}} & 
        \noindent\parbox[c]{0.14\columnwidth}{\includegraphics[width=0.14\columnwidth]{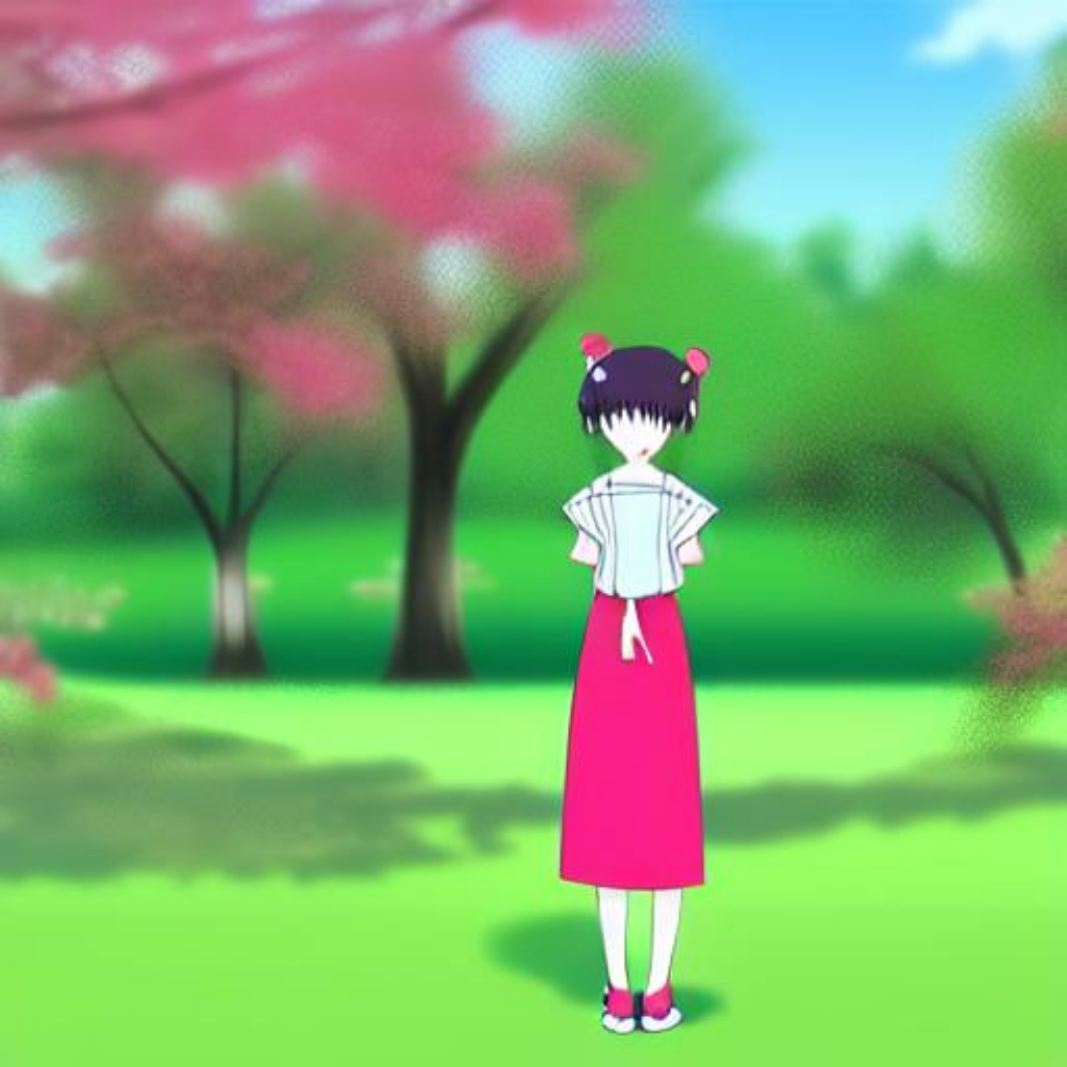}} & 
        \noindent\parbox[c]{0.14\columnwidth}{\includegraphics[width=0.14\columnwidth]{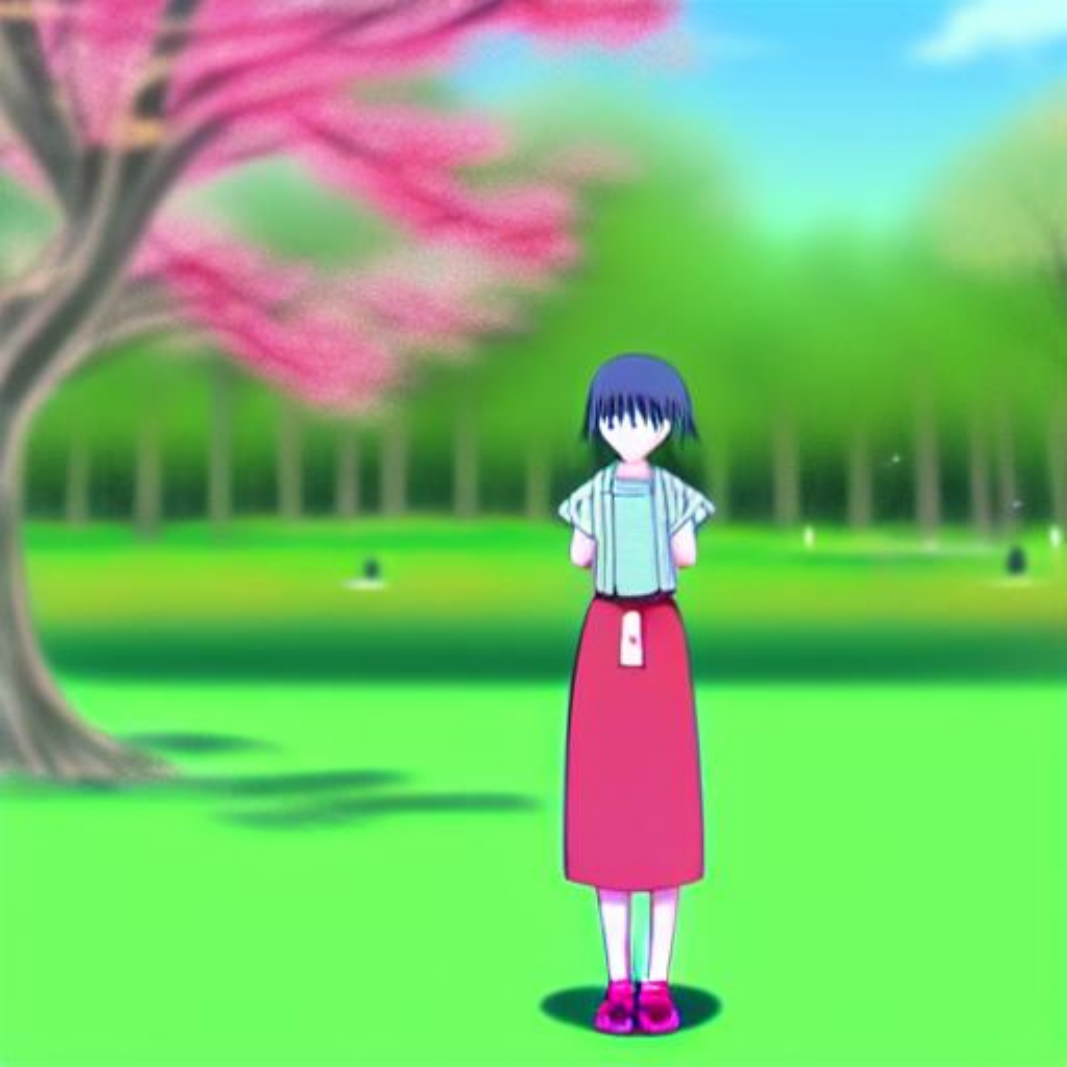}} \\

        \shortstack[l]{\tiny 20 steps} &
        \noindent\parbox[c]{0.14\columnwidth}{\includegraphics[width=0.14\columnwidth]{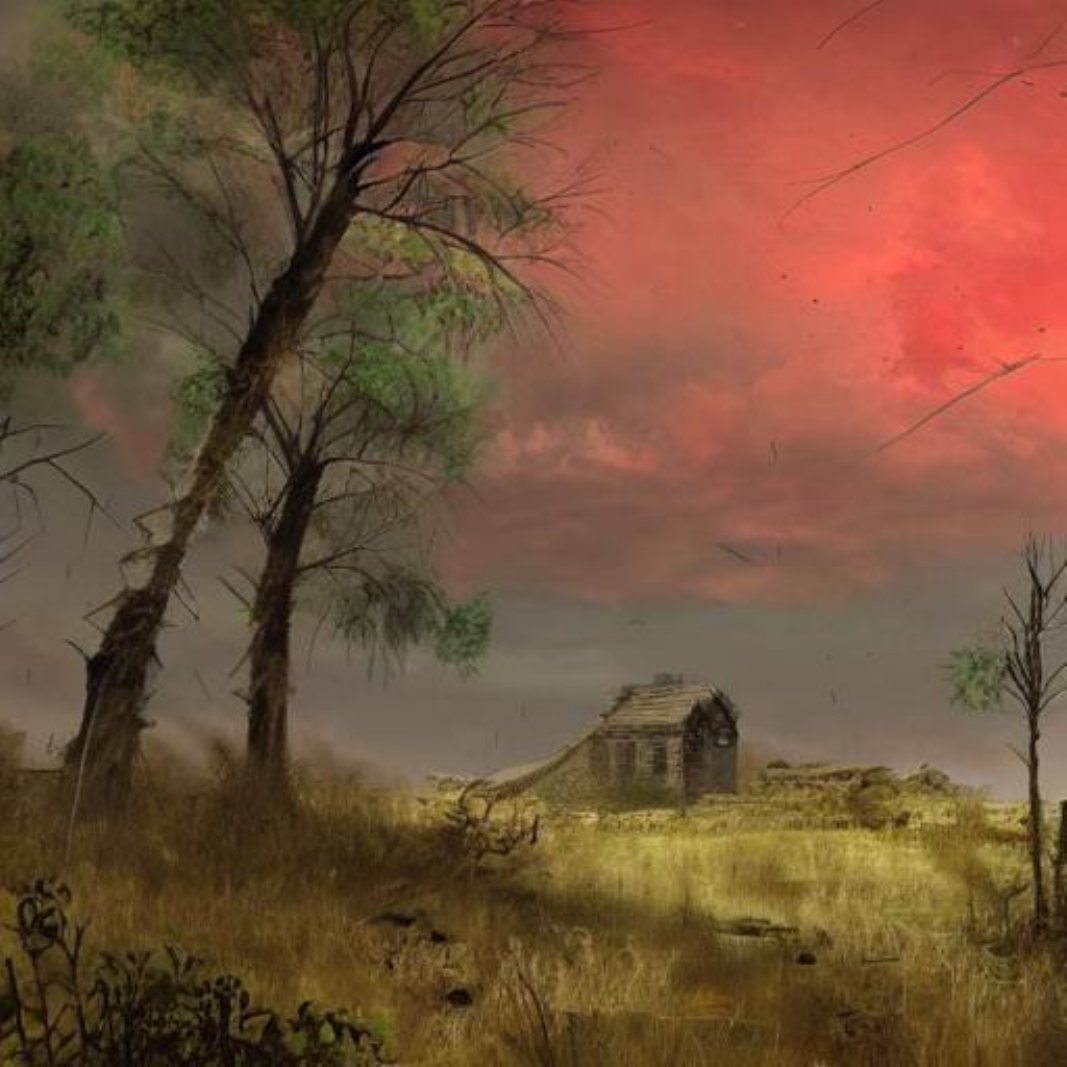}} & 
        \noindent\parbox[c]{0.14\columnwidth}{\includegraphics[width=0.14\columnwidth]{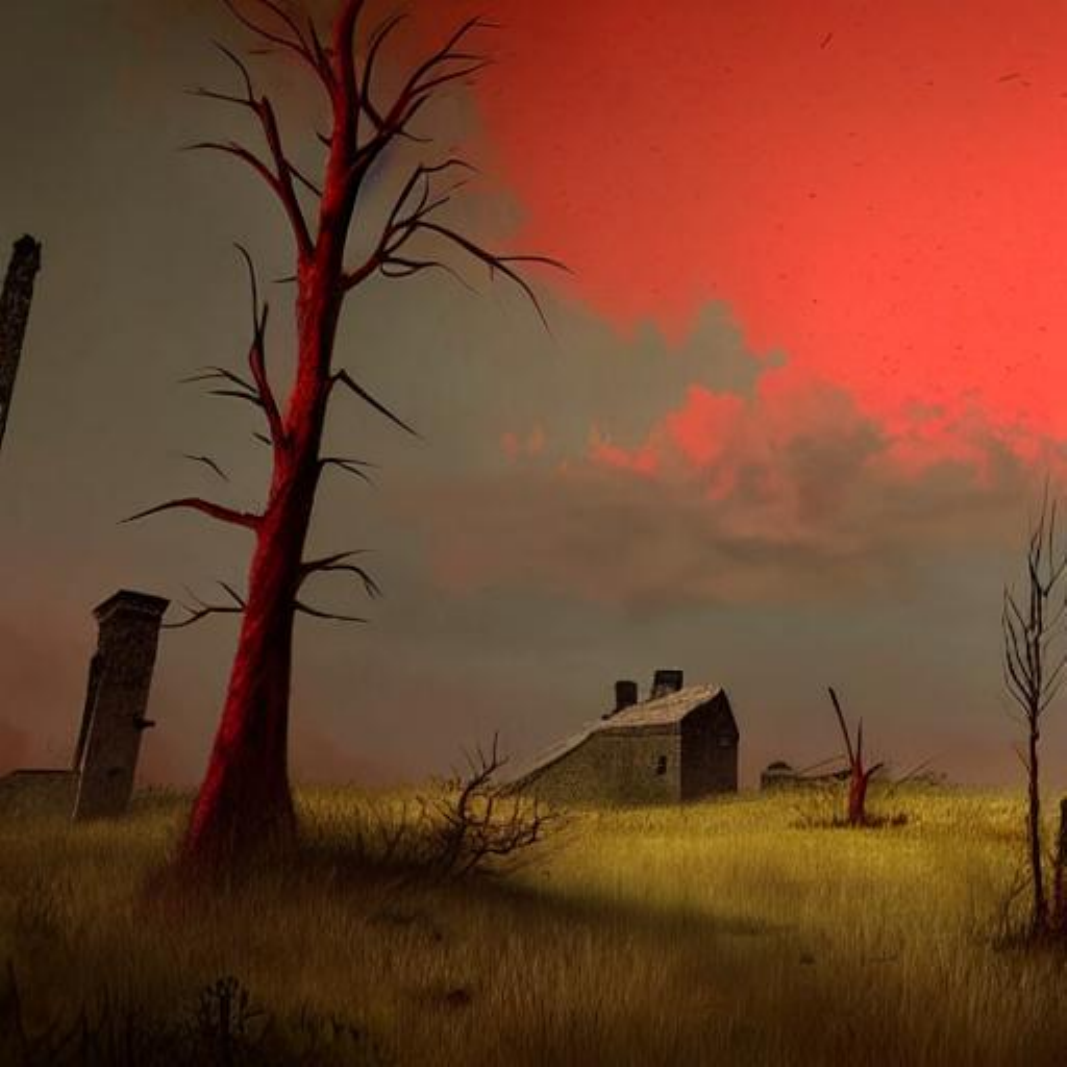}} & 
        \noindent\parbox[c]{0.14\columnwidth}{\includegraphics[width=0.14\columnwidth]{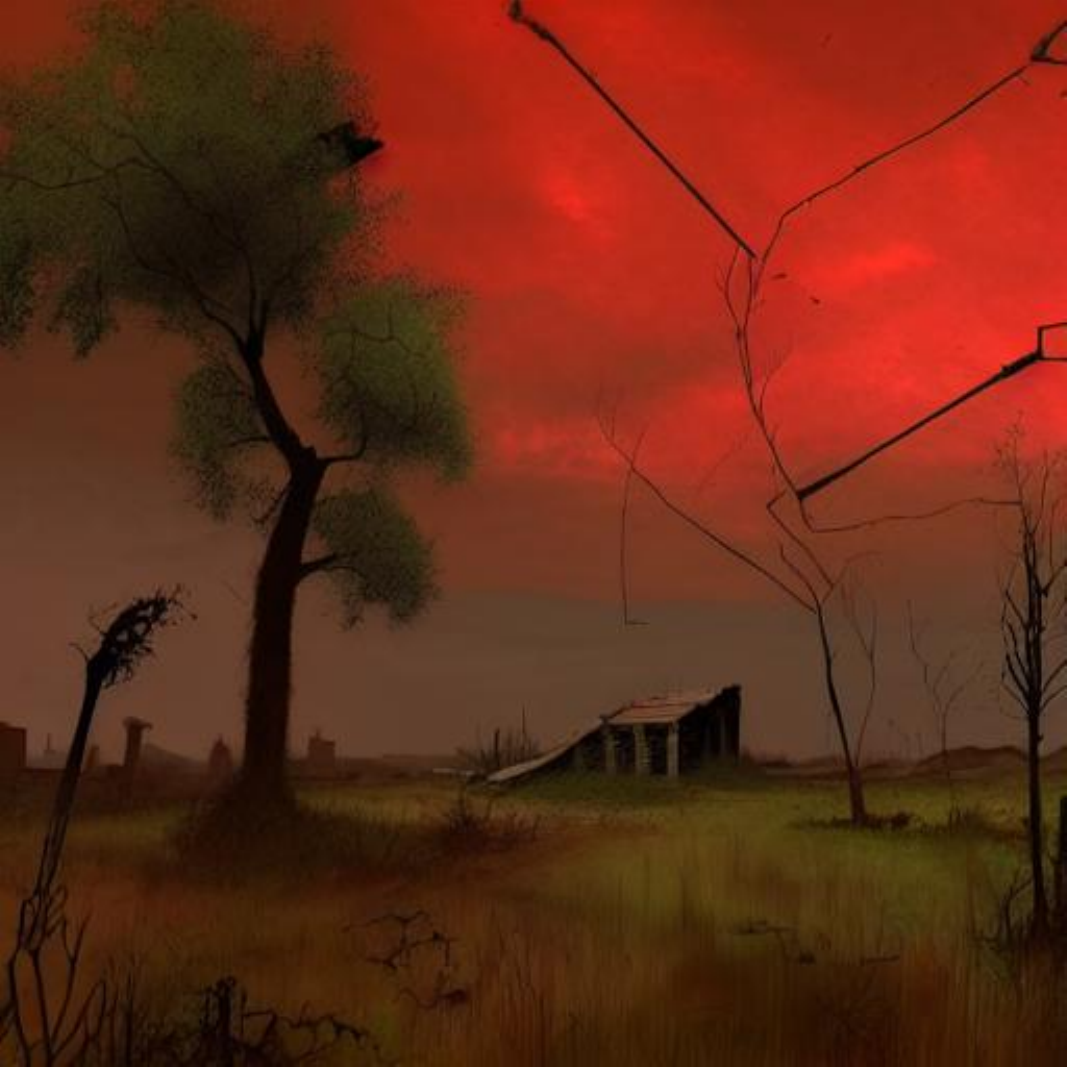}} & 
        \noindent\parbox[c]{0.14\columnwidth}{\includegraphics[width=0.14\columnwidth]{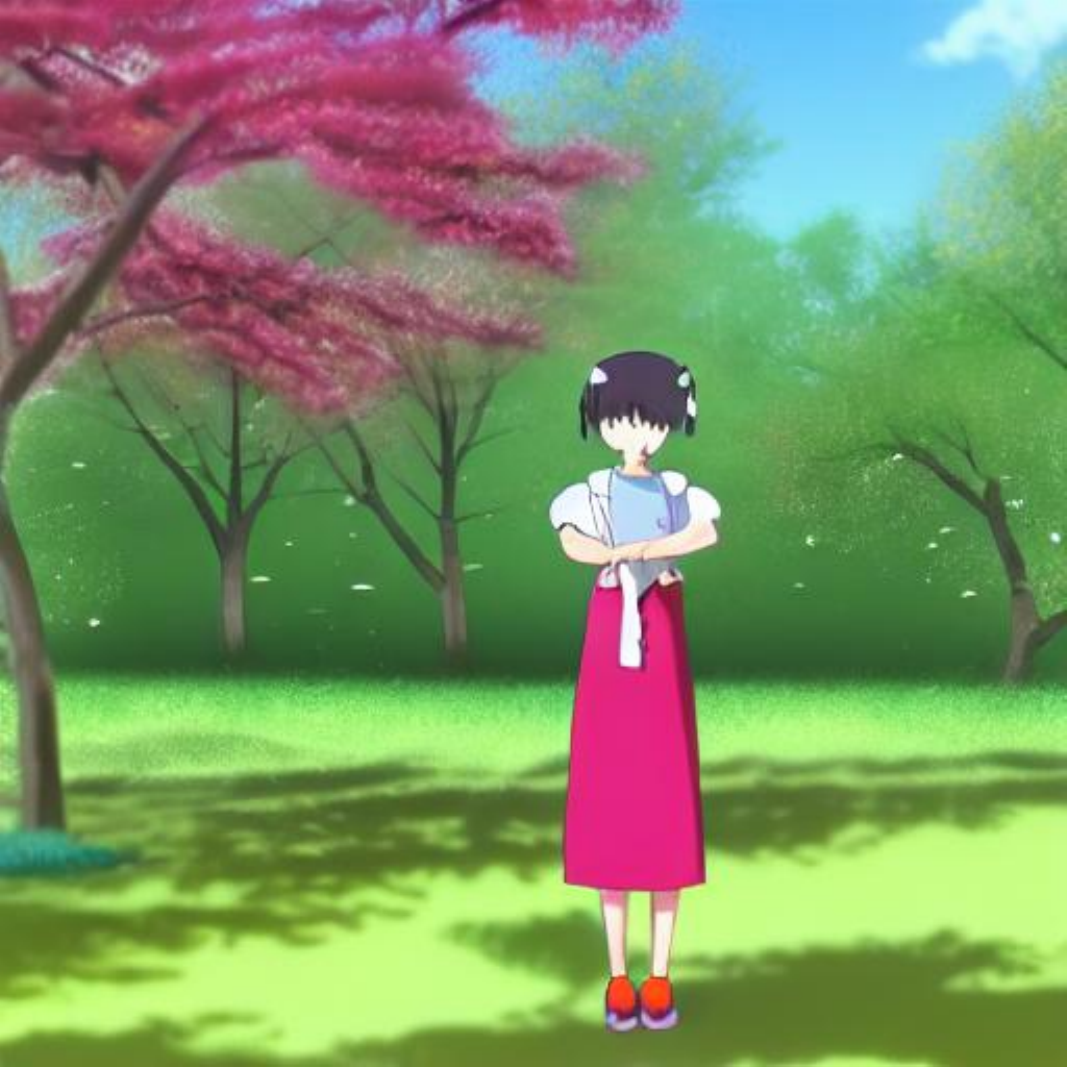}} & 
        \noindent\parbox[c]{0.14\columnwidth}{\includegraphics[width=0.14\columnwidth]{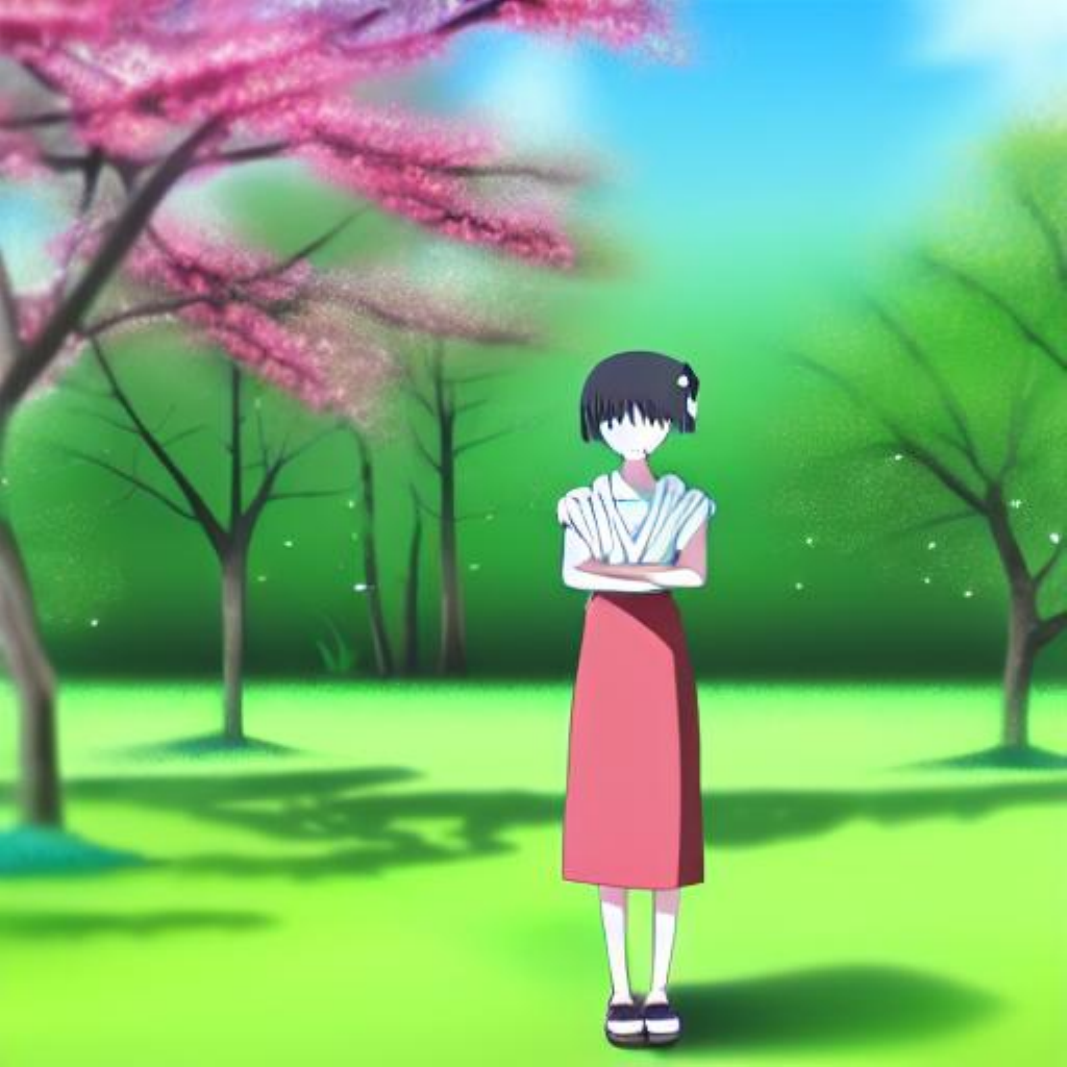}} & 
        \noindent\parbox[c]{0.14\columnwidth}{\includegraphics[width=0.14\columnwidth]{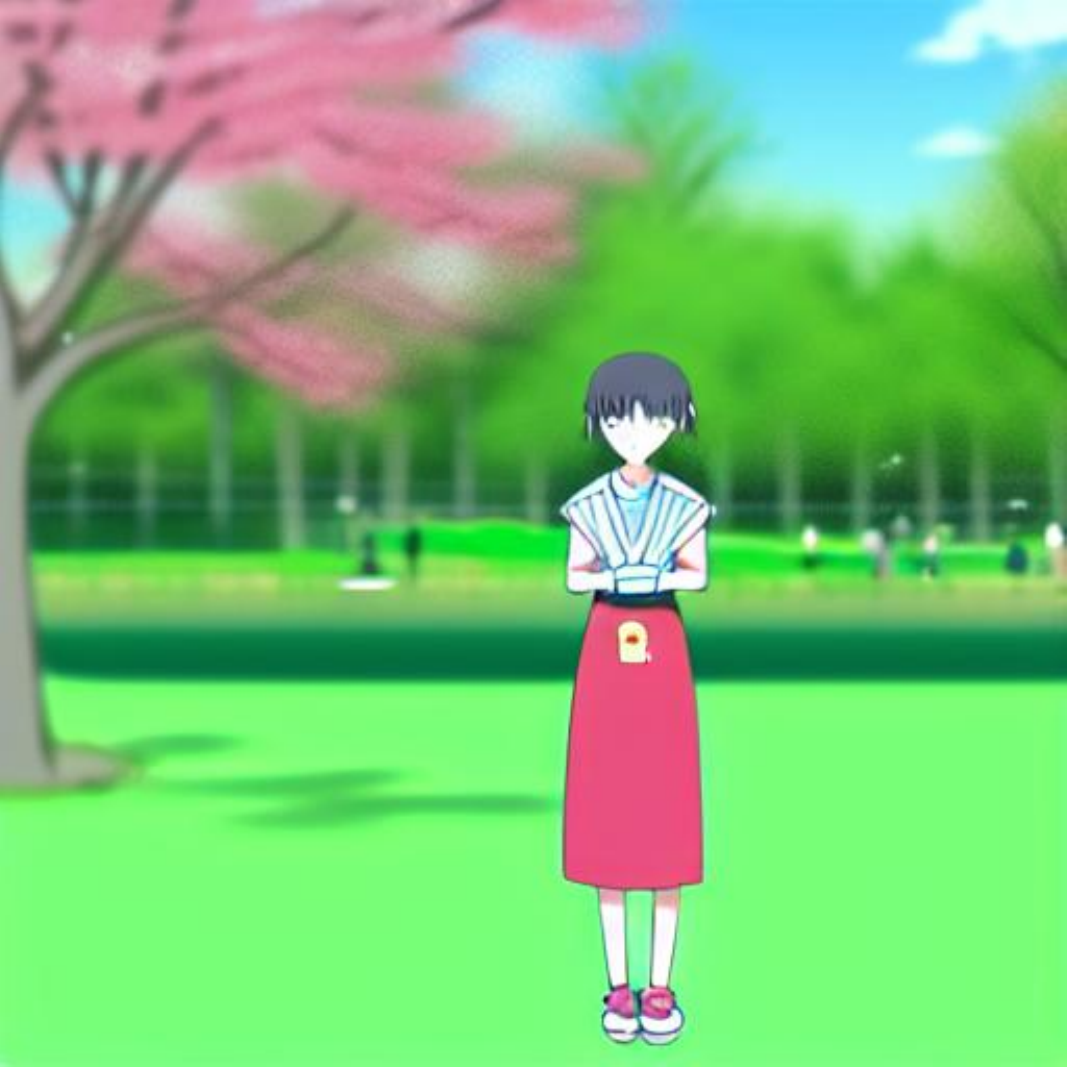}} \\

        \shortstack[l]{\tiny 40 steps} &
        \noindent\parbox[c]{0.14\columnwidth}{\includegraphics[width=0.14\columnwidth]{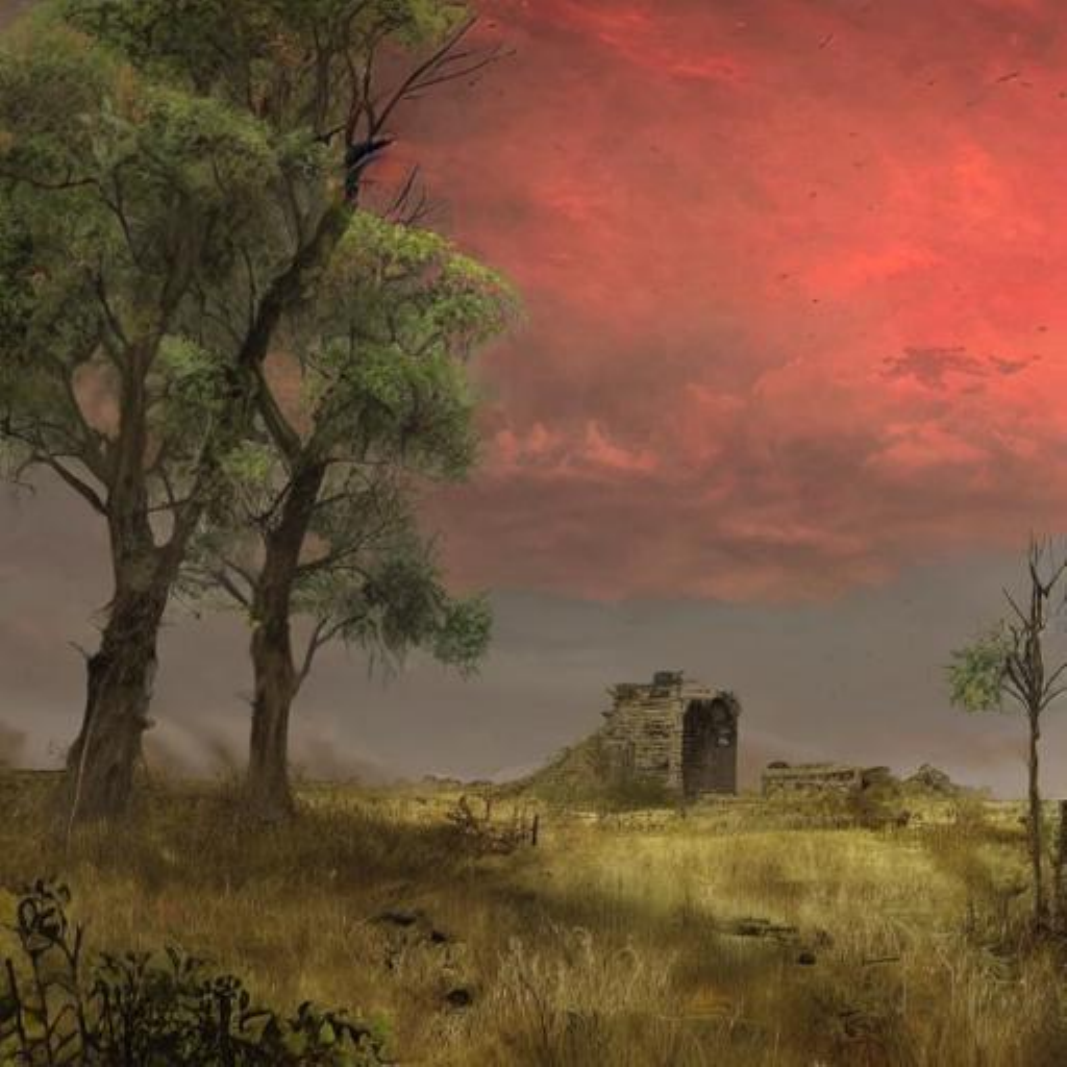}} & 
        \noindent\parbox[c]{0.14\columnwidth}{\includegraphics[width=0.14\columnwidth]{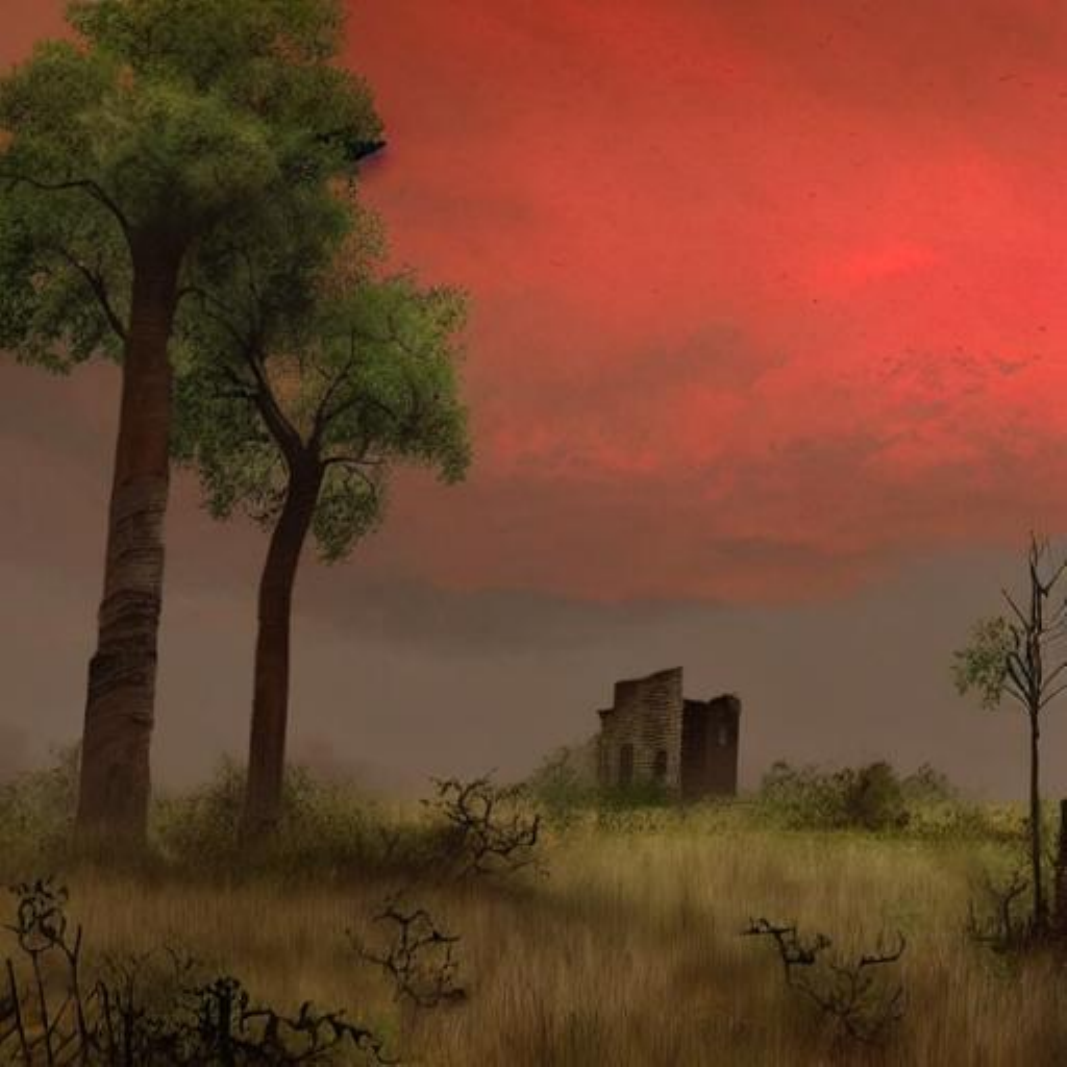}} & 
        \noindent\parbox[c]{0.14\columnwidth}{\includegraphics[width=0.14\columnwidth]{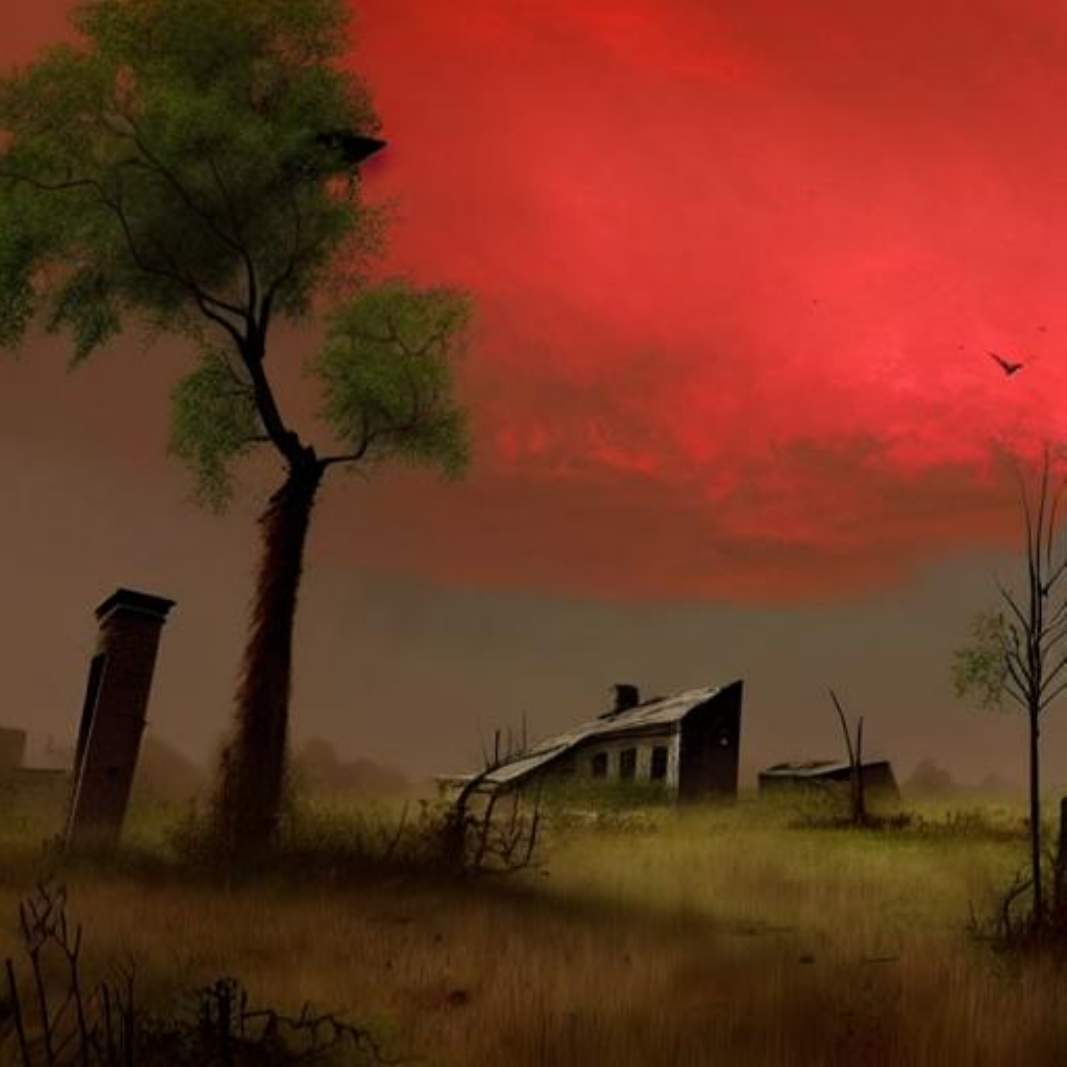}} & 
        \noindent\parbox[c]{0.14\columnwidth}{\includegraphics[width=0.14\columnwidth]{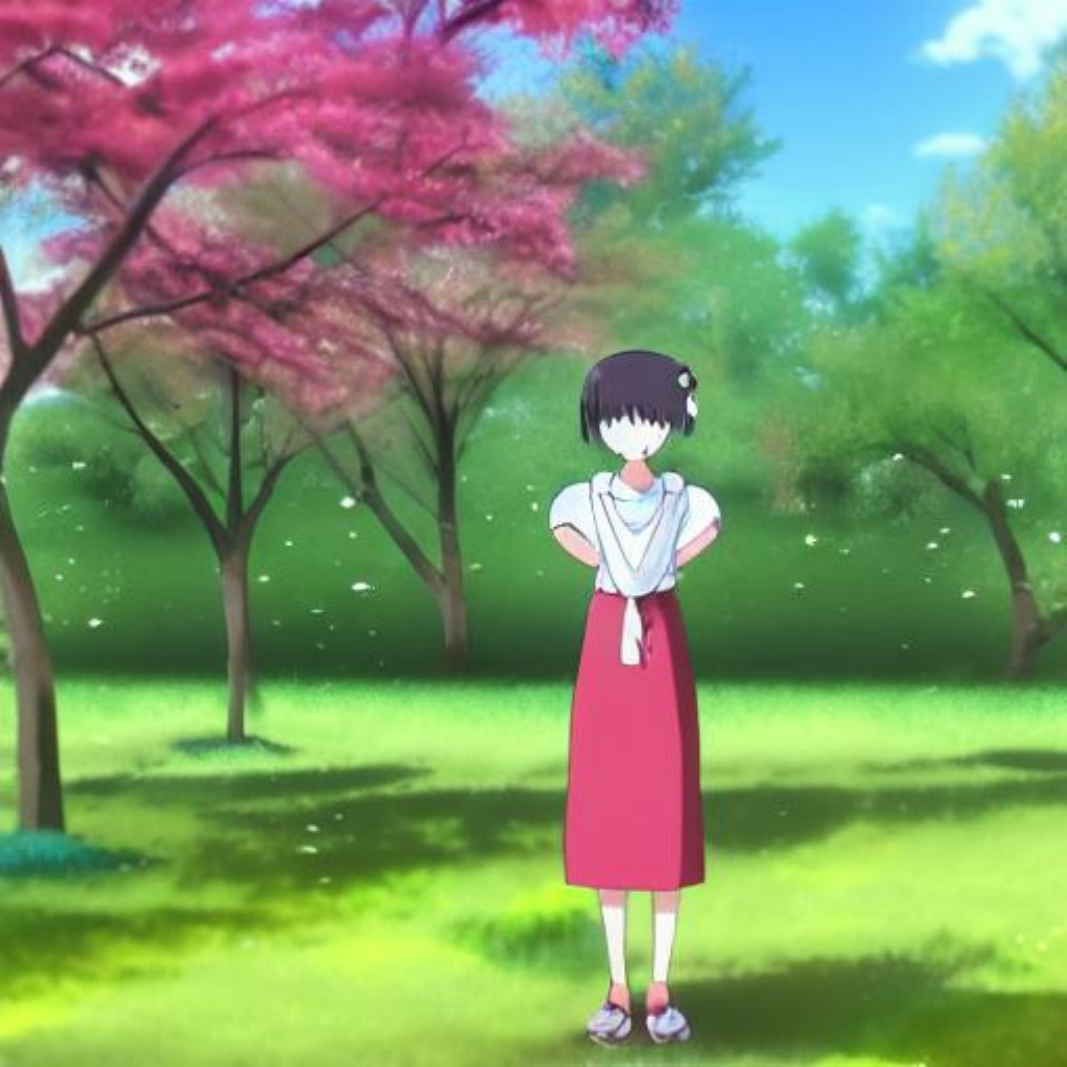}} & 
        \noindent\parbox[c]{0.14\columnwidth}{\includegraphics[width=0.14\columnwidth]{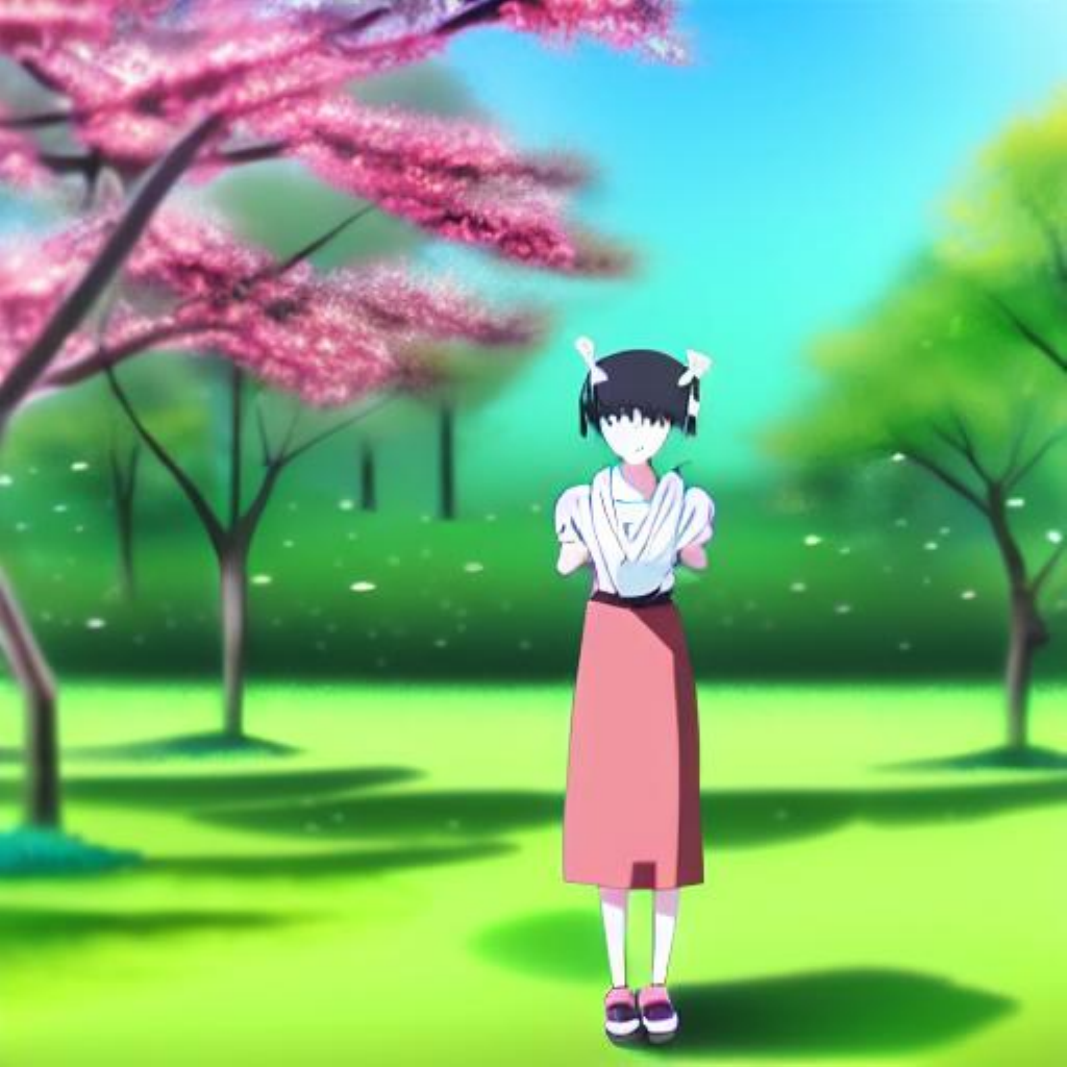}} & 
        \noindent\parbox[c]{0.14\columnwidth}{\includegraphics[width=0.14\columnwidth]{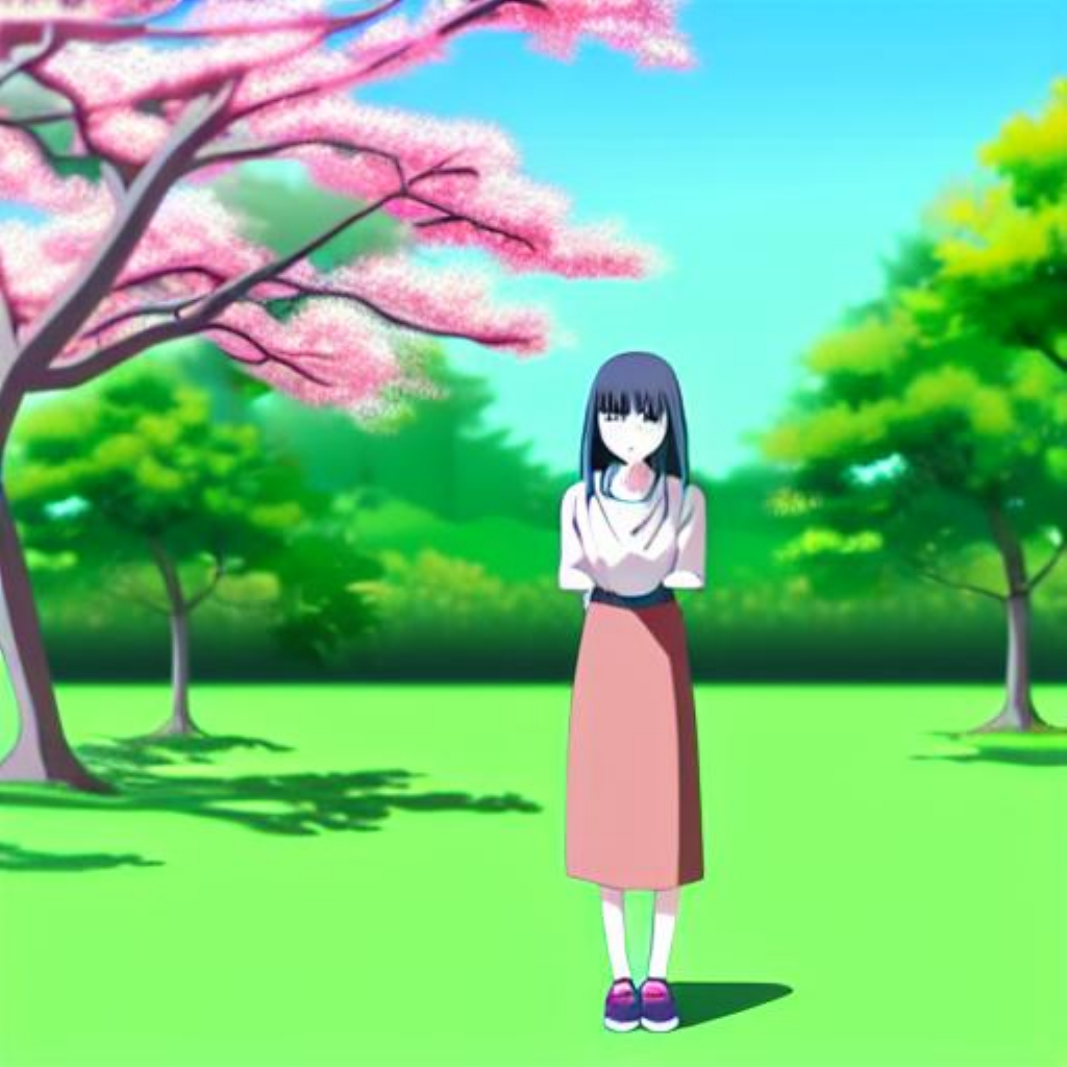}} \\
    \end{tabu}
    \caption{Comparison of samples generated from Stable Diffusion 1.5 using PLMS4 with HB $\beta$ under various sampling steps and guidance scale $s$. Specifically, we employ $\beta = 0.9$ for $s = 7.5$, $\beta = 0.8$ for $s = 15$, and $\beta = 0.6$ for $s = 22.5$ to account for the varying degrees of artifact manifestation associated with each guidance scale.}
    \label{fig:scale_step_sd15_hb}
\end{figure}


\pagebreak

\tabulinesep=1pt
\begin{figure}
    \centering
    \begin{tabu} to \textwidth {@{}l@{\hspace{5pt}}c@{\hspace{2pt}}c@{\hspace{2pt}}c@{\hspace{4pt}}c@{\hspace{2pt}}c@{\hspace{2pt}}c@{}}
        & \multicolumn{3}{c}{\shortstack{\scriptsize "A post-apocalyptic world with ruined \\ \scriptsize buildings, overgrown vegetation, and a red sky"}}
        & \multicolumn{3}{c}{\shortstack{\scriptsize "A girl standing in a park in \\ \scriptsize Japanese animation style"}} \\

        & \multicolumn{1}{c}{\shortstack{\scriptsize $s = 7.5$}}
        & \multicolumn{1}{c}{\shortstack{\scriptsize $s = 15$}}
        & \multicolumn{1}{c}{\shortstack{\scriptsize $s = 22.5$}}
        & \multicolumn{1}{c}{\shortstack{\scriptsize $s = 7.5$}}
        & \multicolumn{1}{c}{\shortstack{\scriptsize $s = 15$}}
        & \multicolumn{1}{c}{\shortstack{\scriptsize $s = 22.5$}}
        \\
        
        \shortstack[l]{\tiny 10 steps} &
        \noindent\parbox[c]{0.14\columnwidth}{\includegraphics[width=0.14\columnwidth]{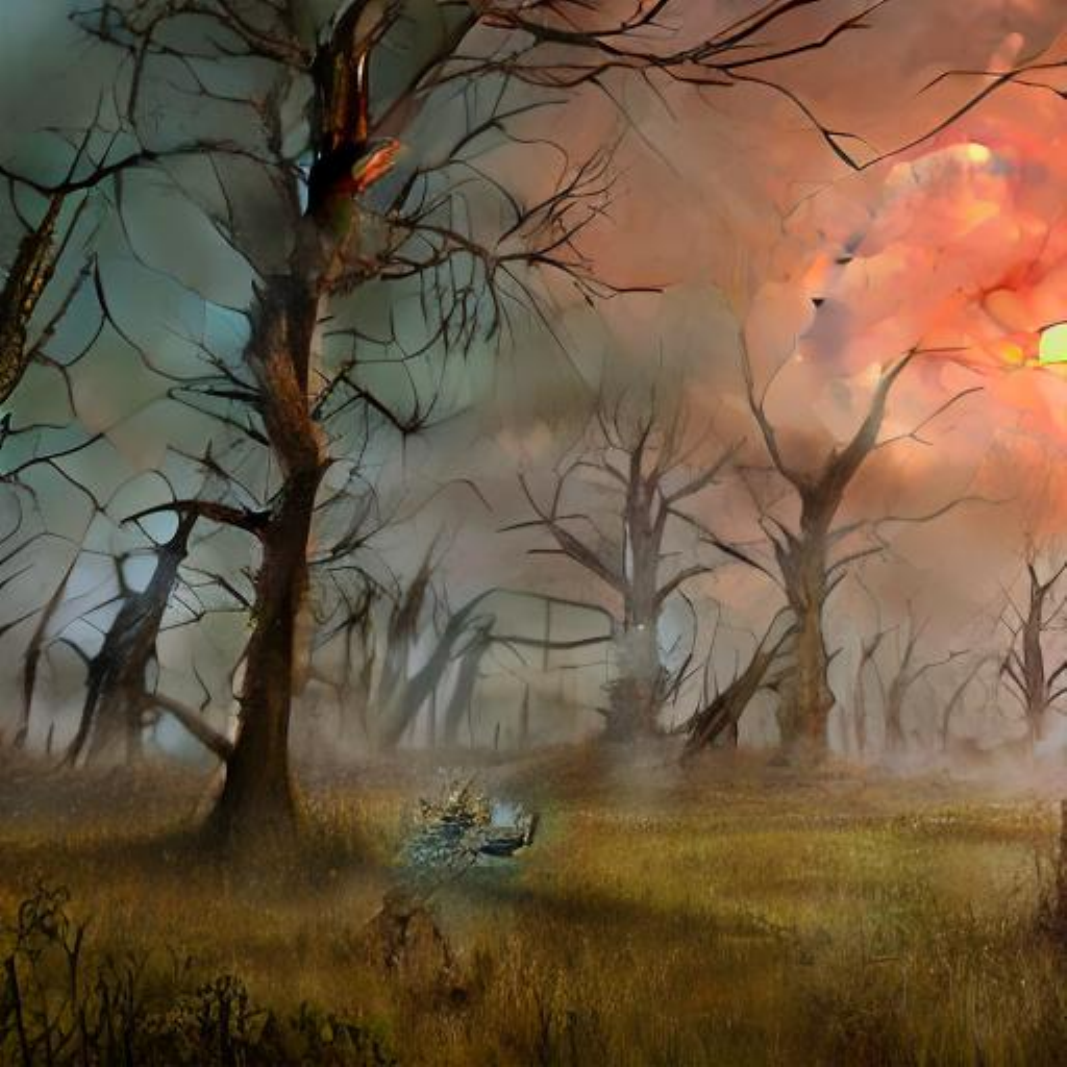}} & 
        \noindent\parbox[c]{0.14\columnwidth}{\includegraphics[width=0.14\columnwidth]{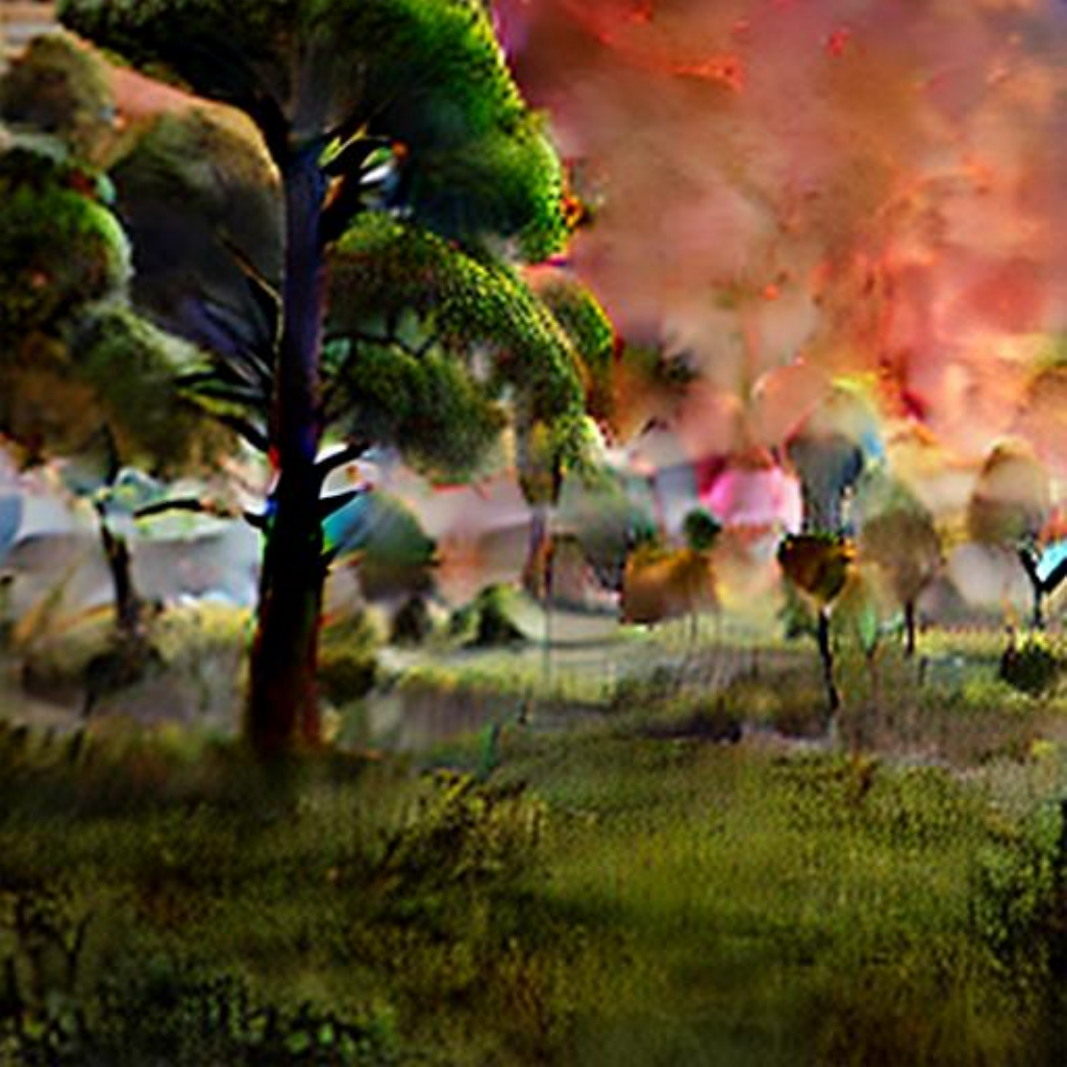}} & 
        \noindent\parbox[c]{0.14\columnwidth}{\includegraphics[width=0.14\columnwidth]{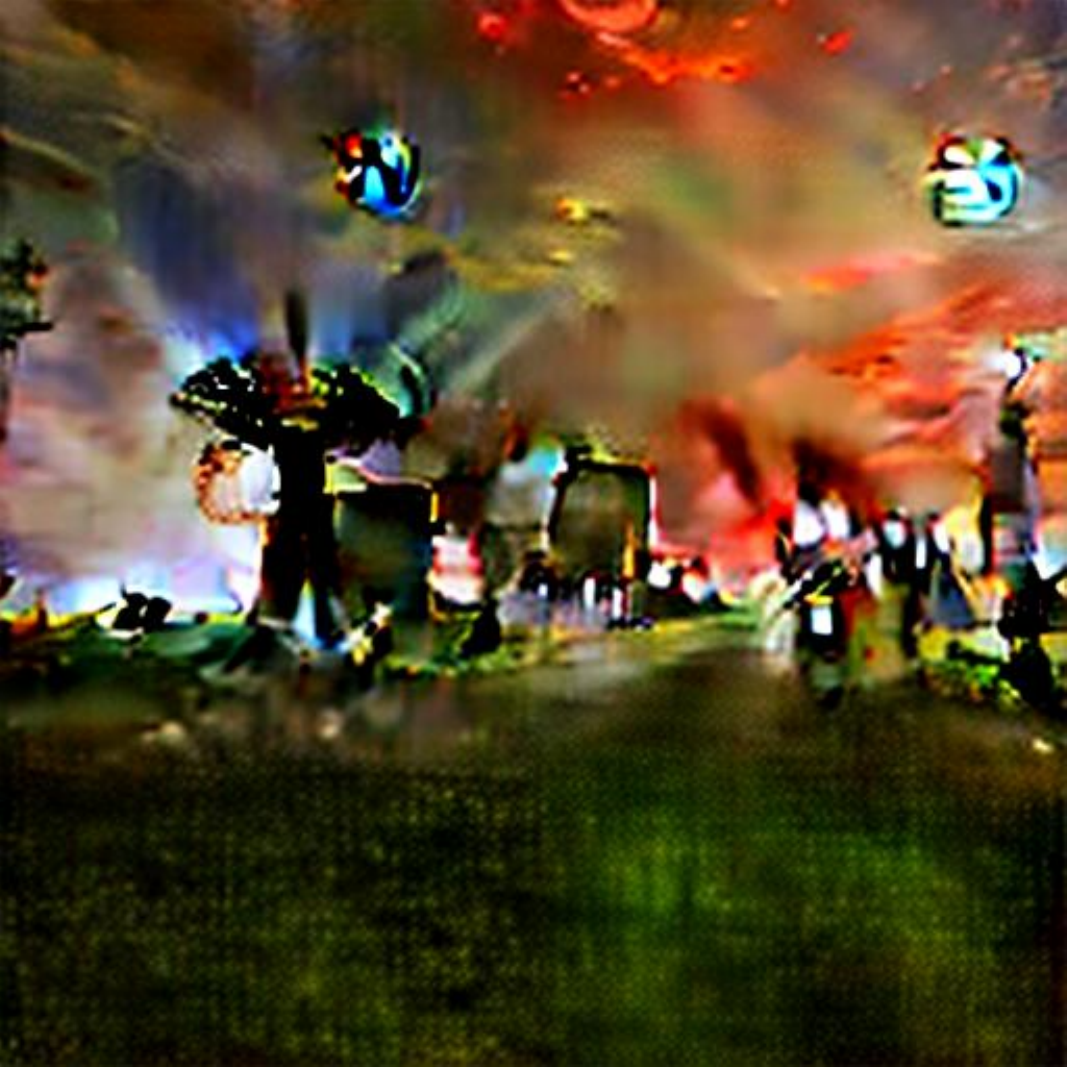}} & 
        \noindent\parbox[c]{0.14\columnwidth}{\includegraphics[width=0.14\columnwidth]{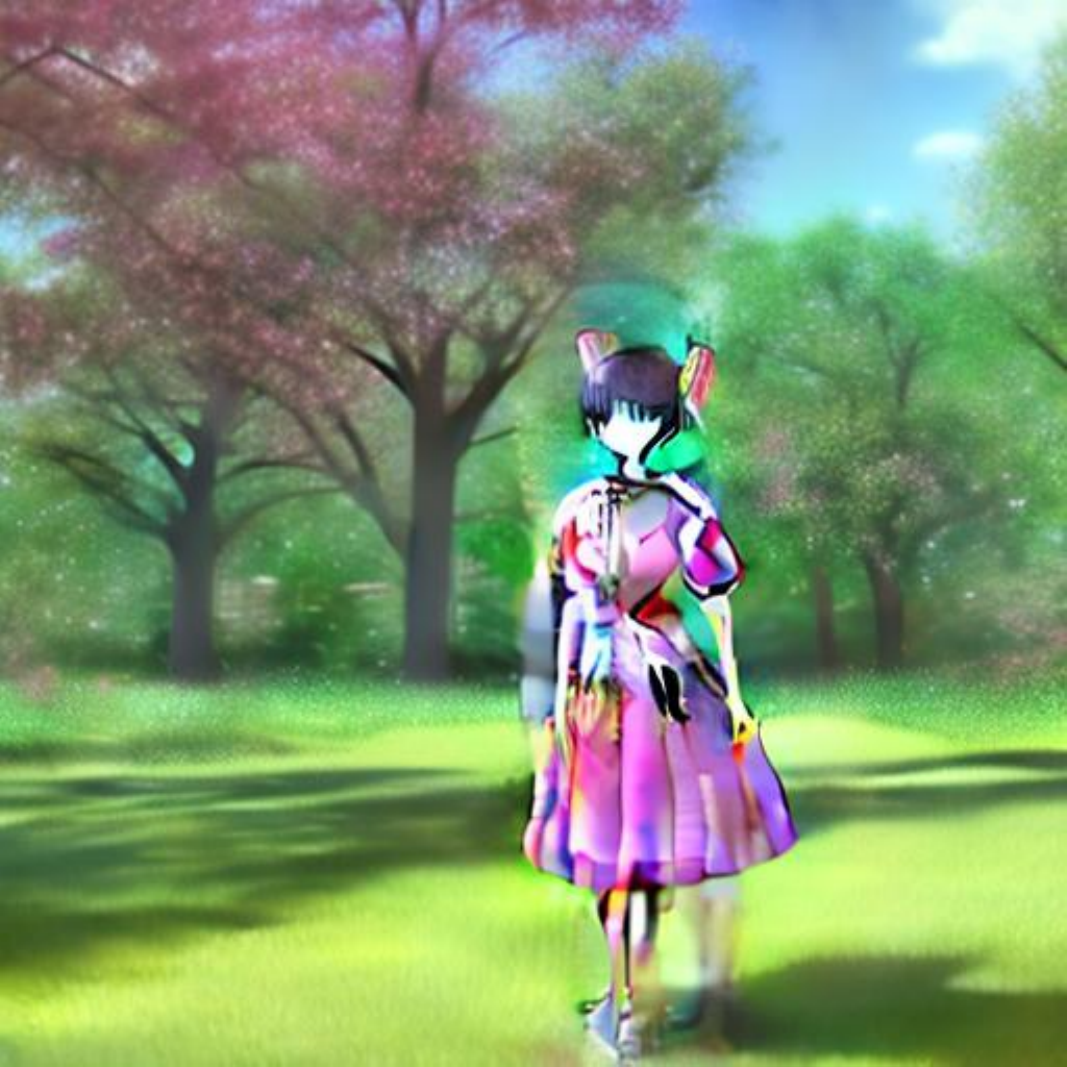}} & 
        \noindent\parbox[c]{0.14\columnwidth}{\includegraphics[width=0.14\columnwidth]{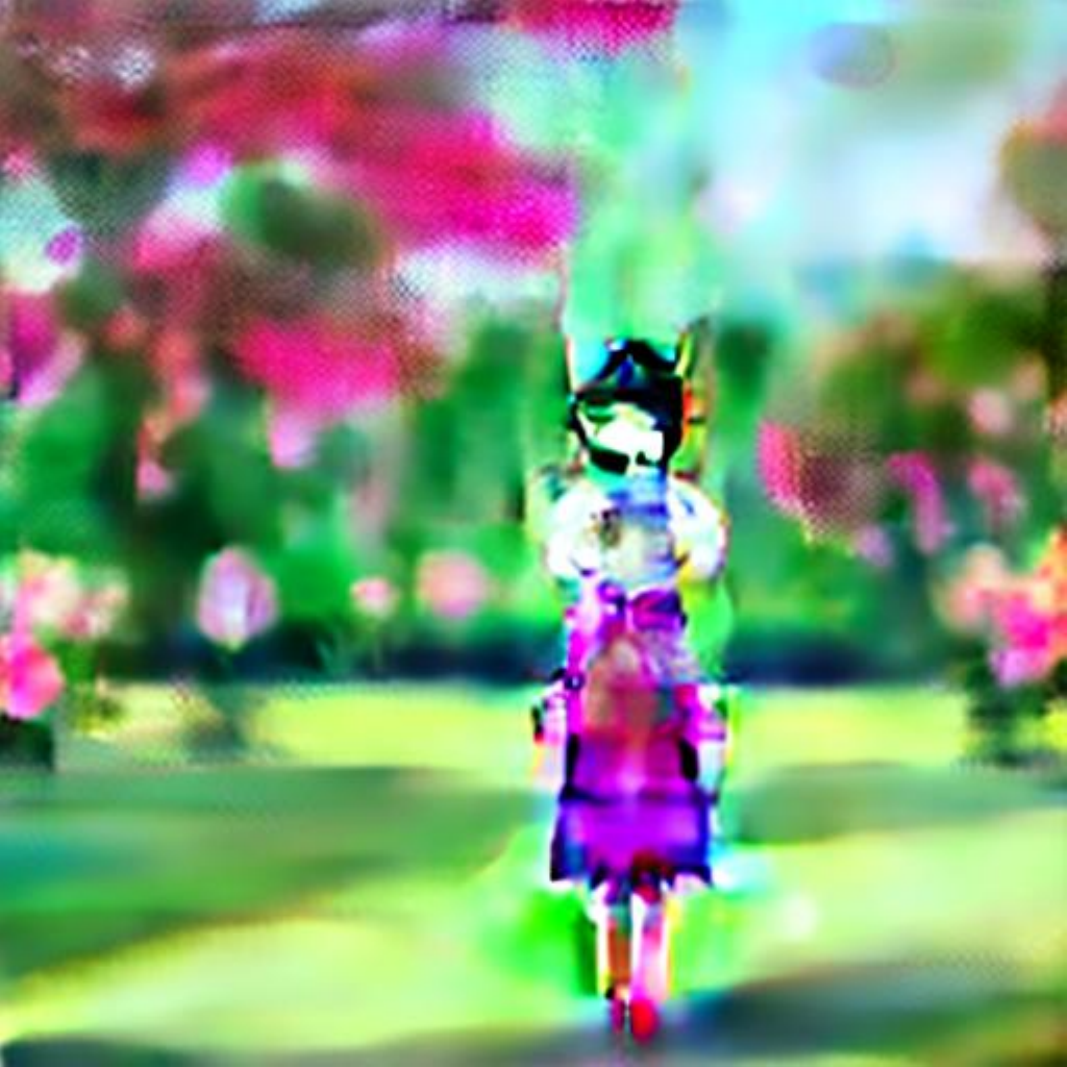}} & 
        \noindent\parbox[c]{0.14\columnwidth}{\includegraphics[width=0.14\columnwidth]{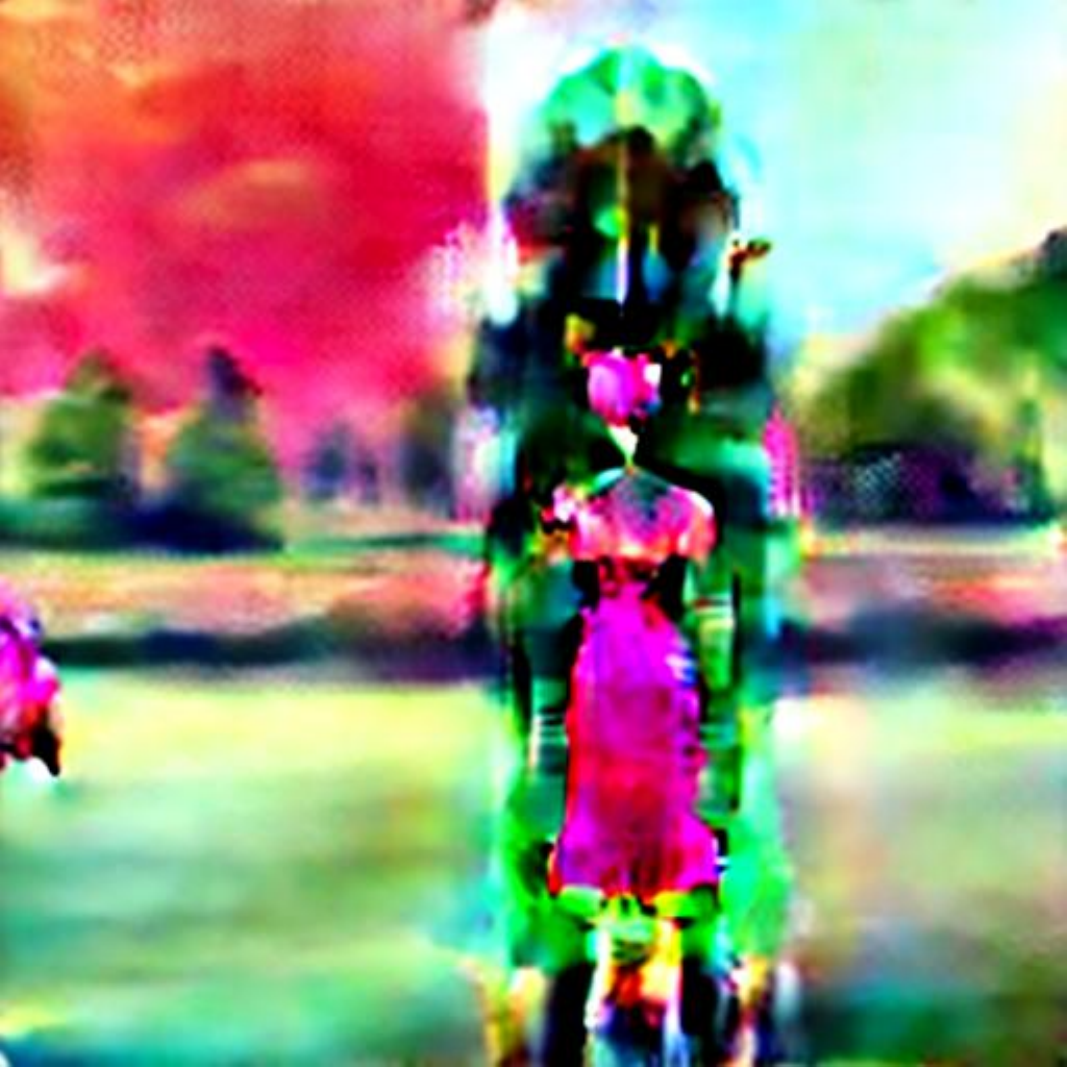}} \\

        \shortstack[l]{\tiny 15 steps} &
        \noindent\parbox[c]{0.14\columnwidth}{\includegraphics[width=0.14\columnwidth]{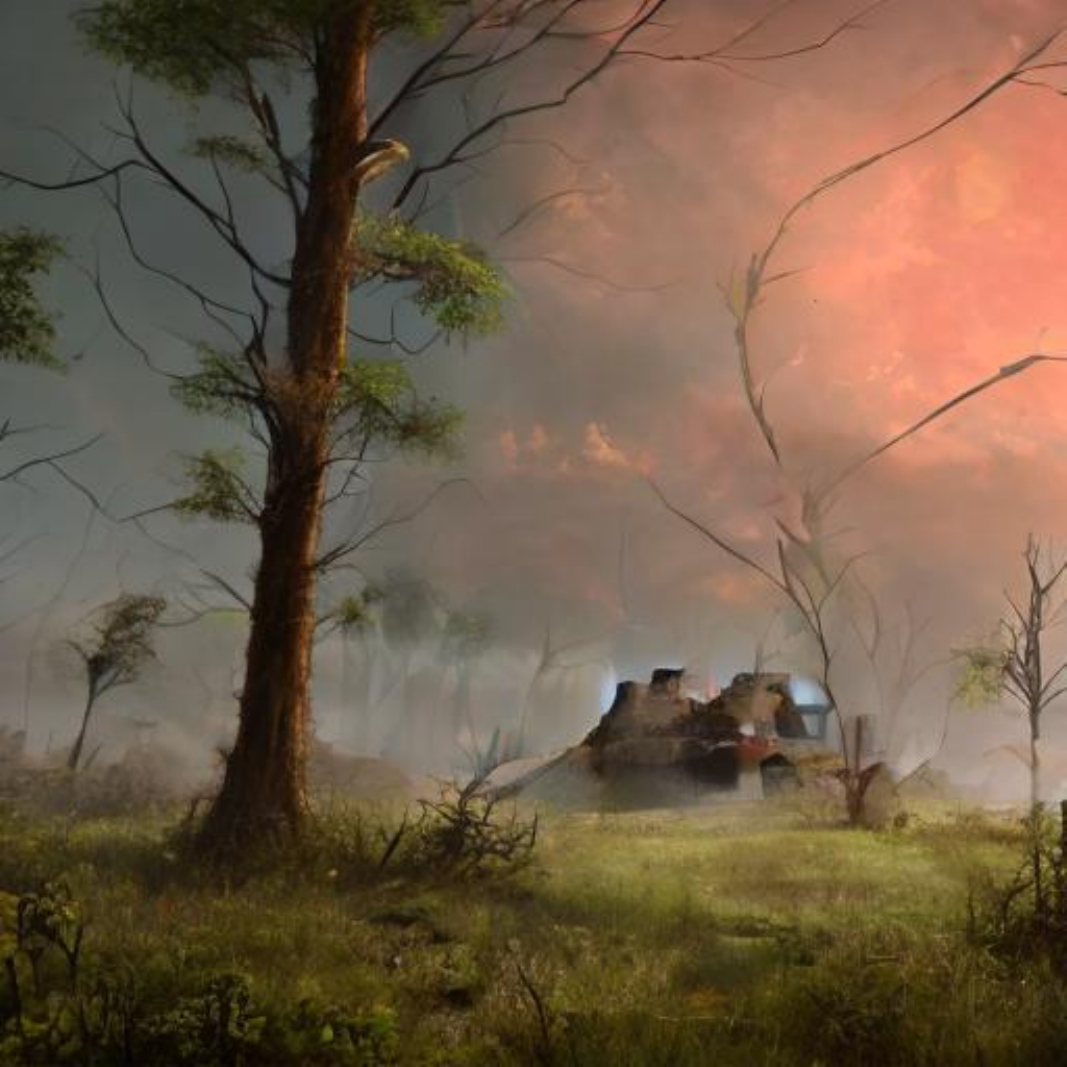}} & 
        \noindent\parbox[c]{0.14\columnwidth}{\includegraphics[width=0.14\columnwidth]{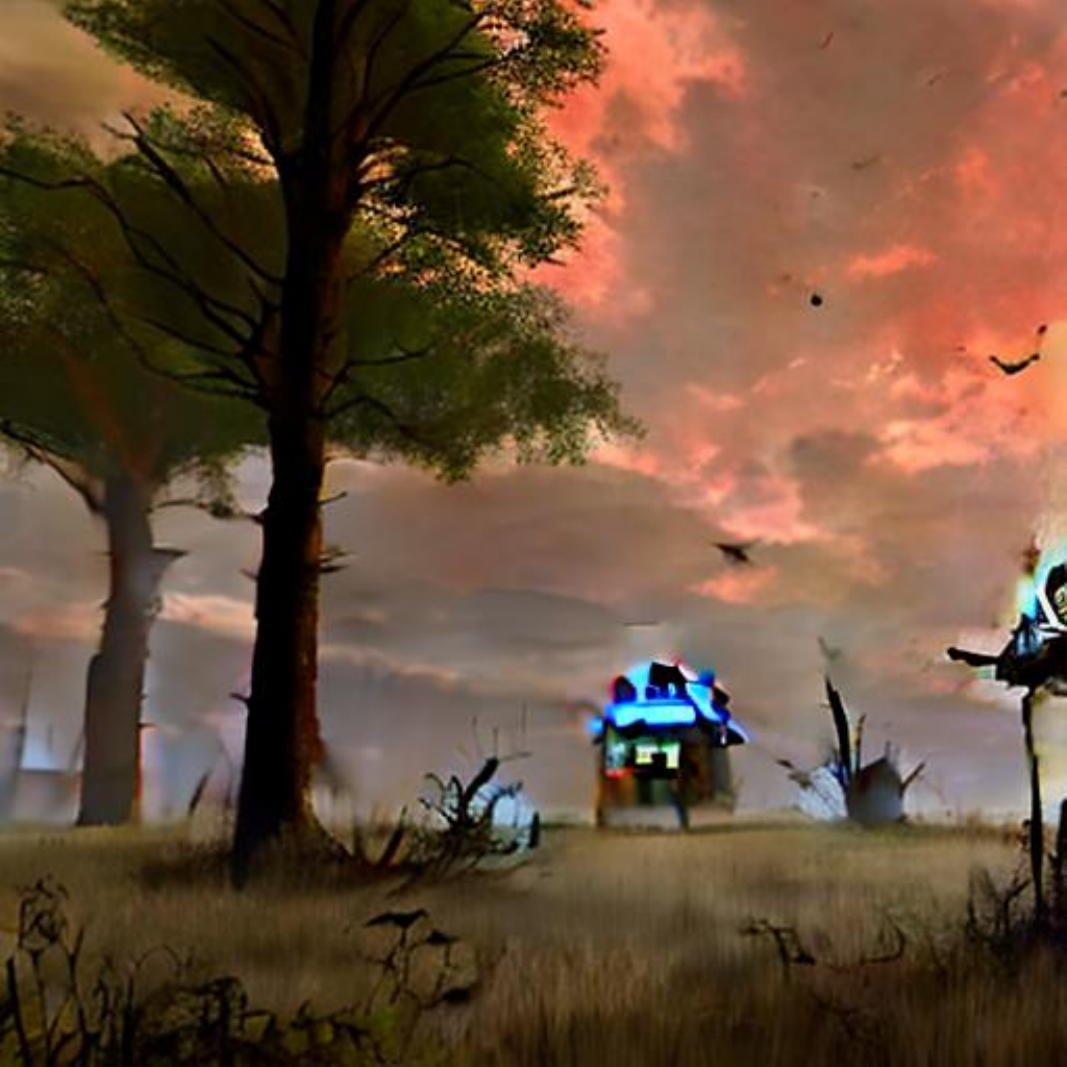}} & 
        \noindent\parbox[c]{0.14\columnwidth}{\includegraphics[width=0.14\columnwidth]{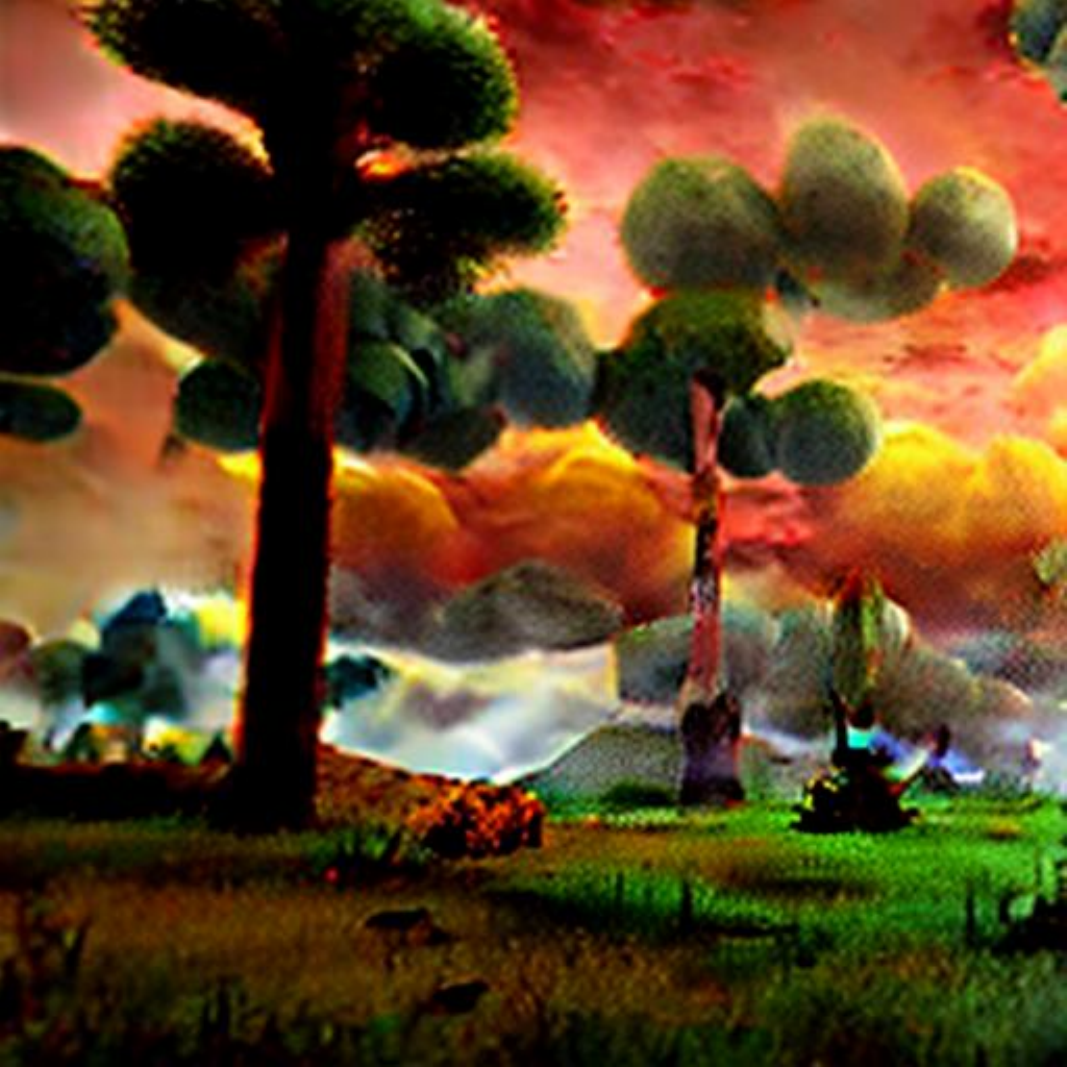}} & 
        \noindent\parbox[c]{0.14\columnwidth}{\includegraphics[width=0.14\columnwidth]{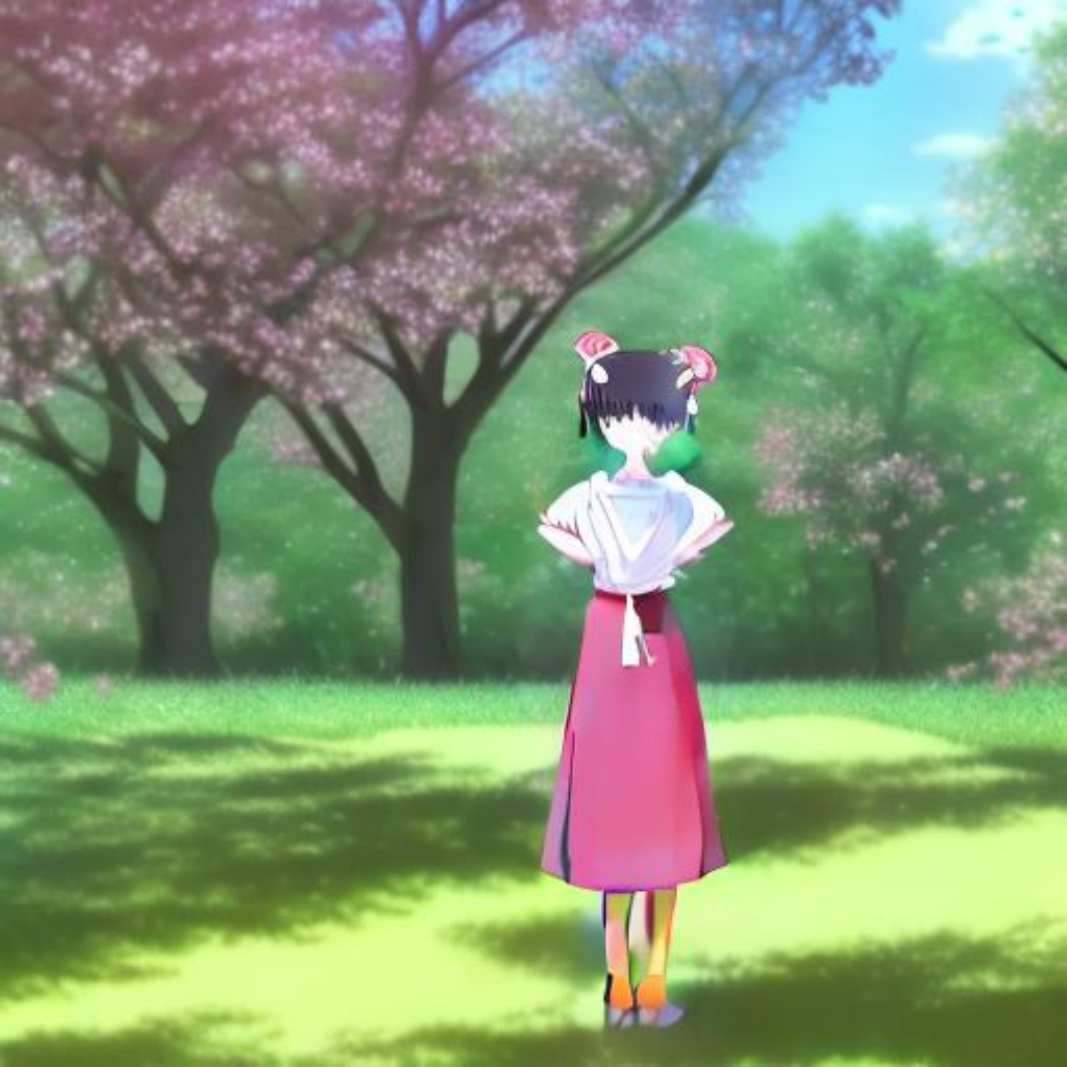}} & 
        \noindent\parbox[c]{0.14\columnwidth}{\includegraphics[width=0.14\columnwidth]{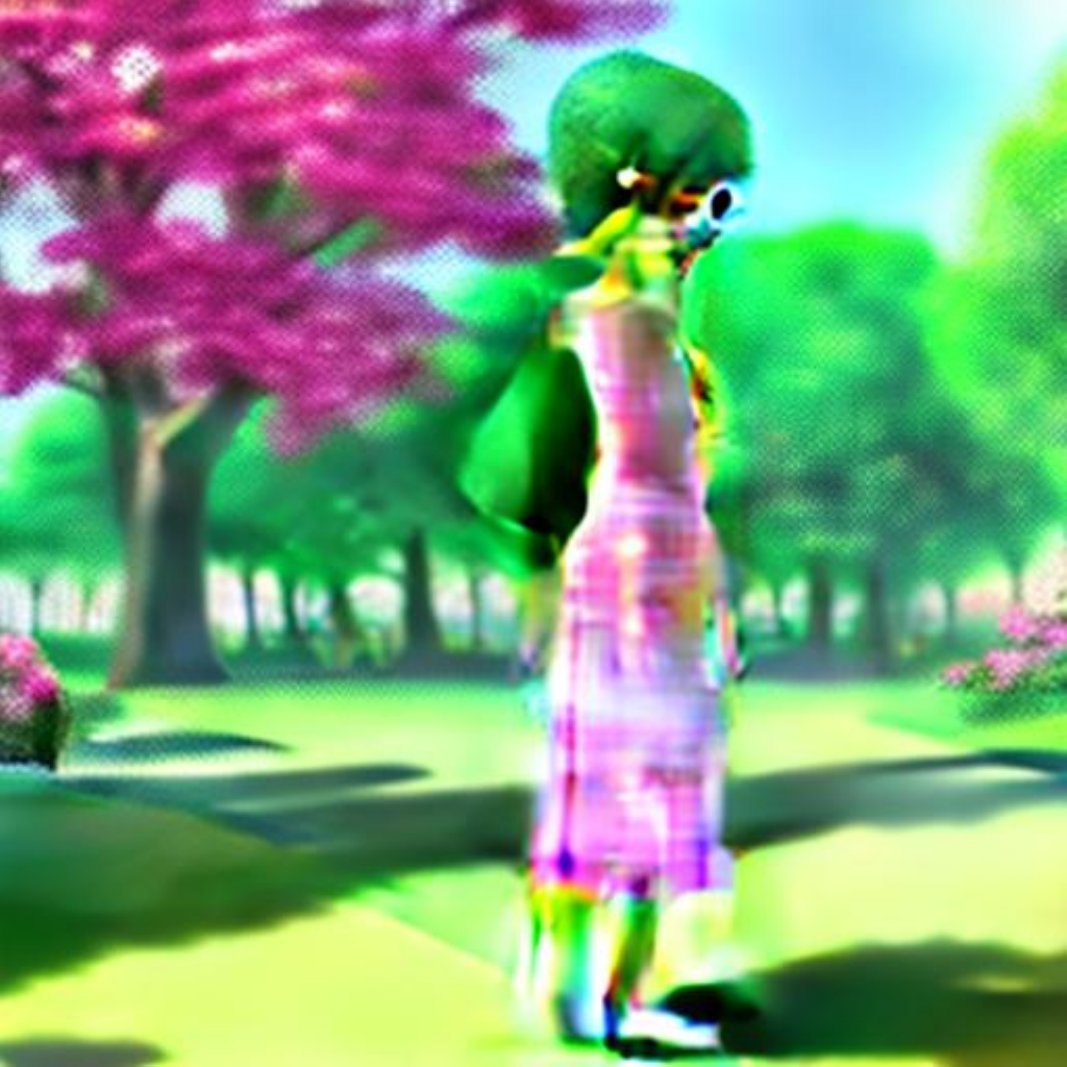}} & 
        \noindent\parbox[c]{0.14\columnwidth}{\includegraphics[width=0.14\columnwidth]{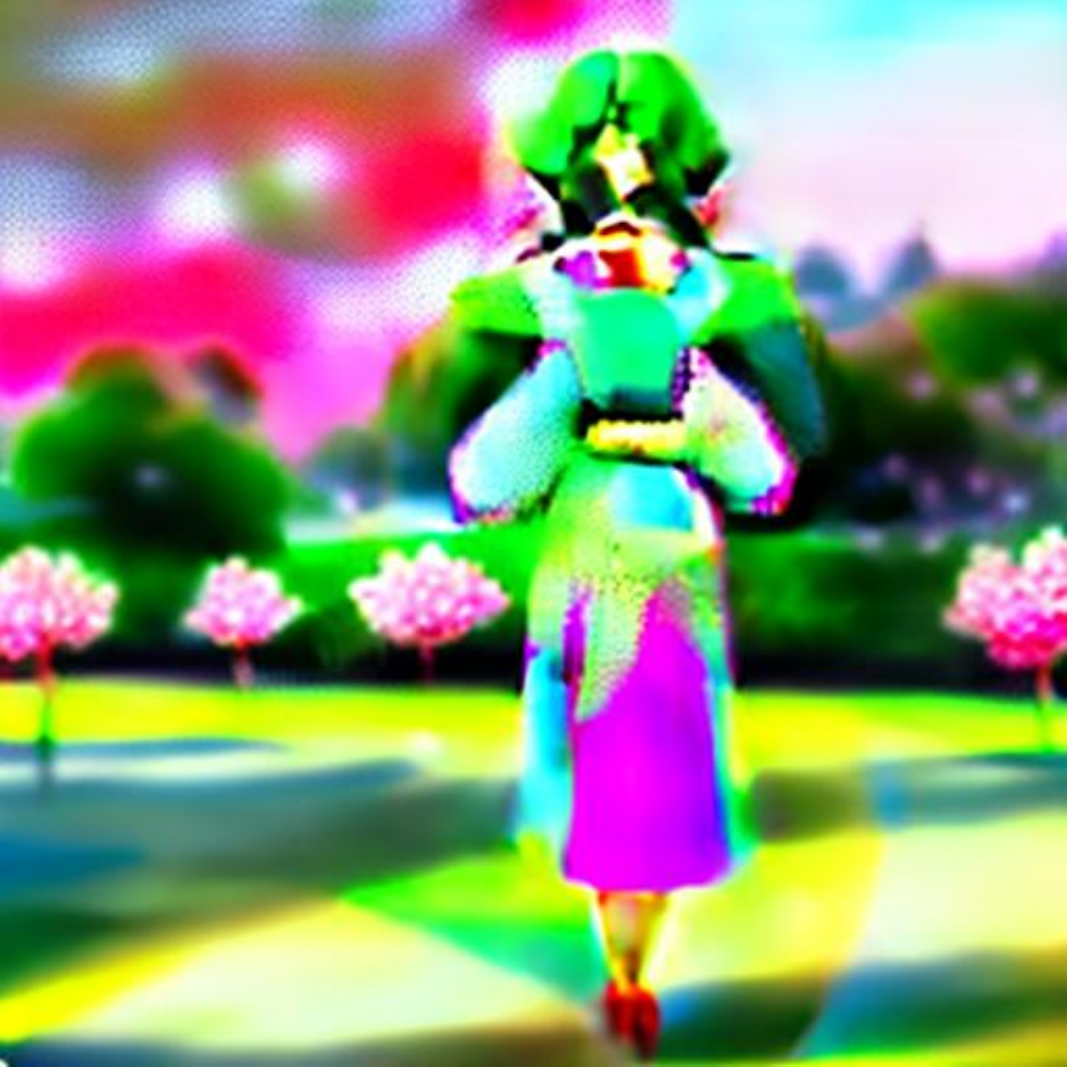}} \\

        \shortstack[l]{\tiny 20 steps} &
        \noindent\parbox[c]{0.14\columnwidth}{\includegraphics[width=0.14\columnwidth]{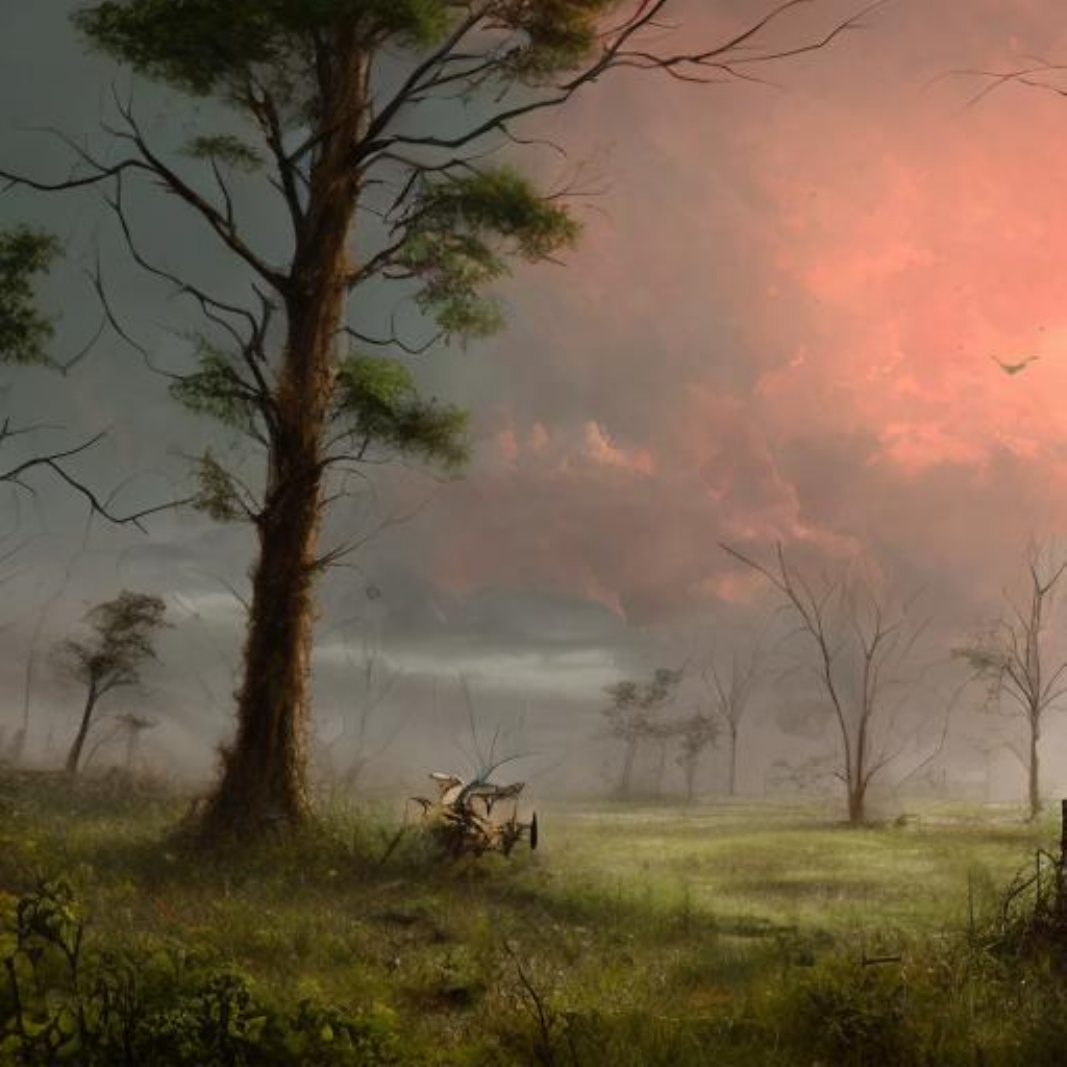}} & 
        \noindent\parbox[c]{0.14\columnwidth}{\includegraphics[width=0.14\columnwidth]{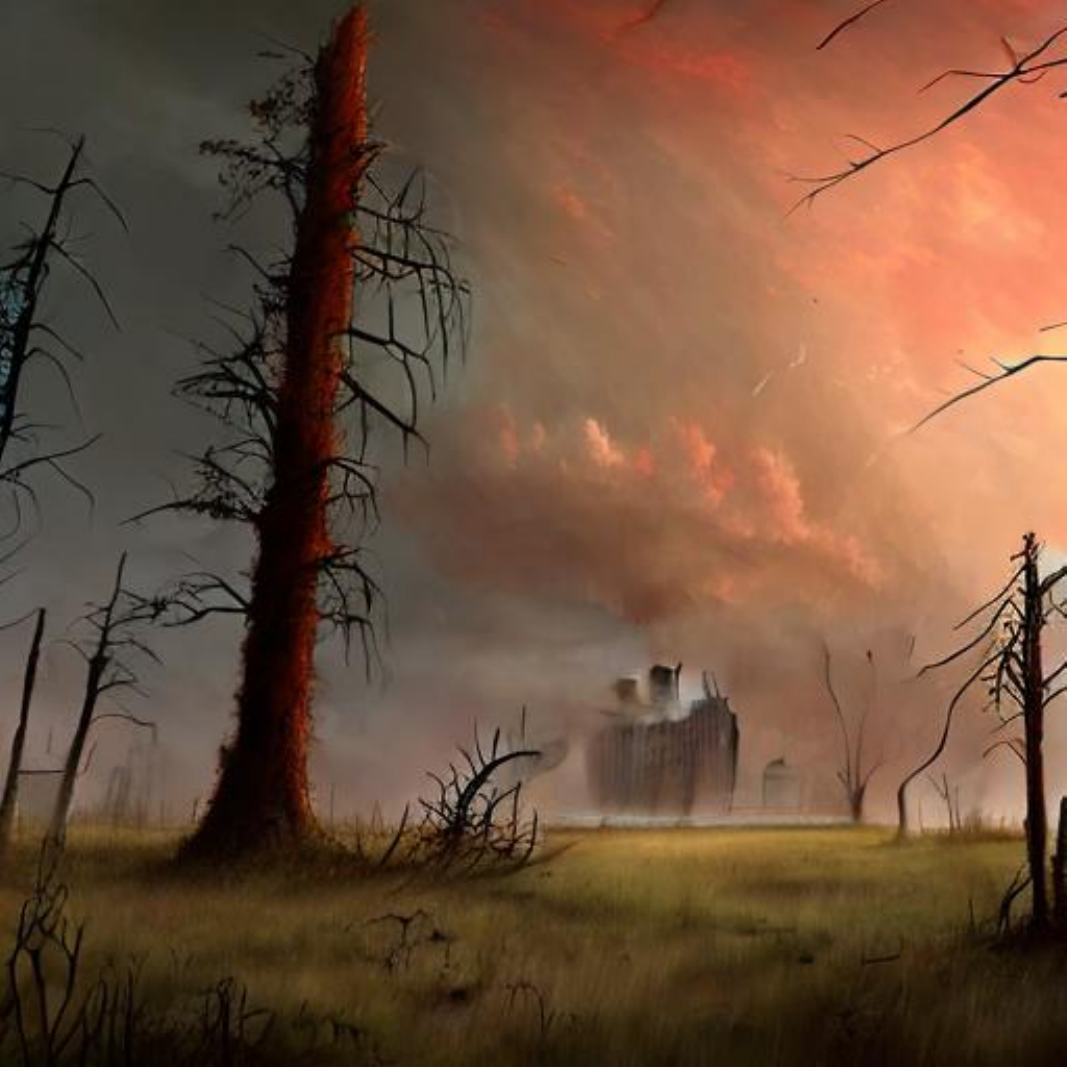}} & 
        \noindent\parbox[c]{0.14\columnwidth}{\includegraphics[width=0.14\columnwidth]{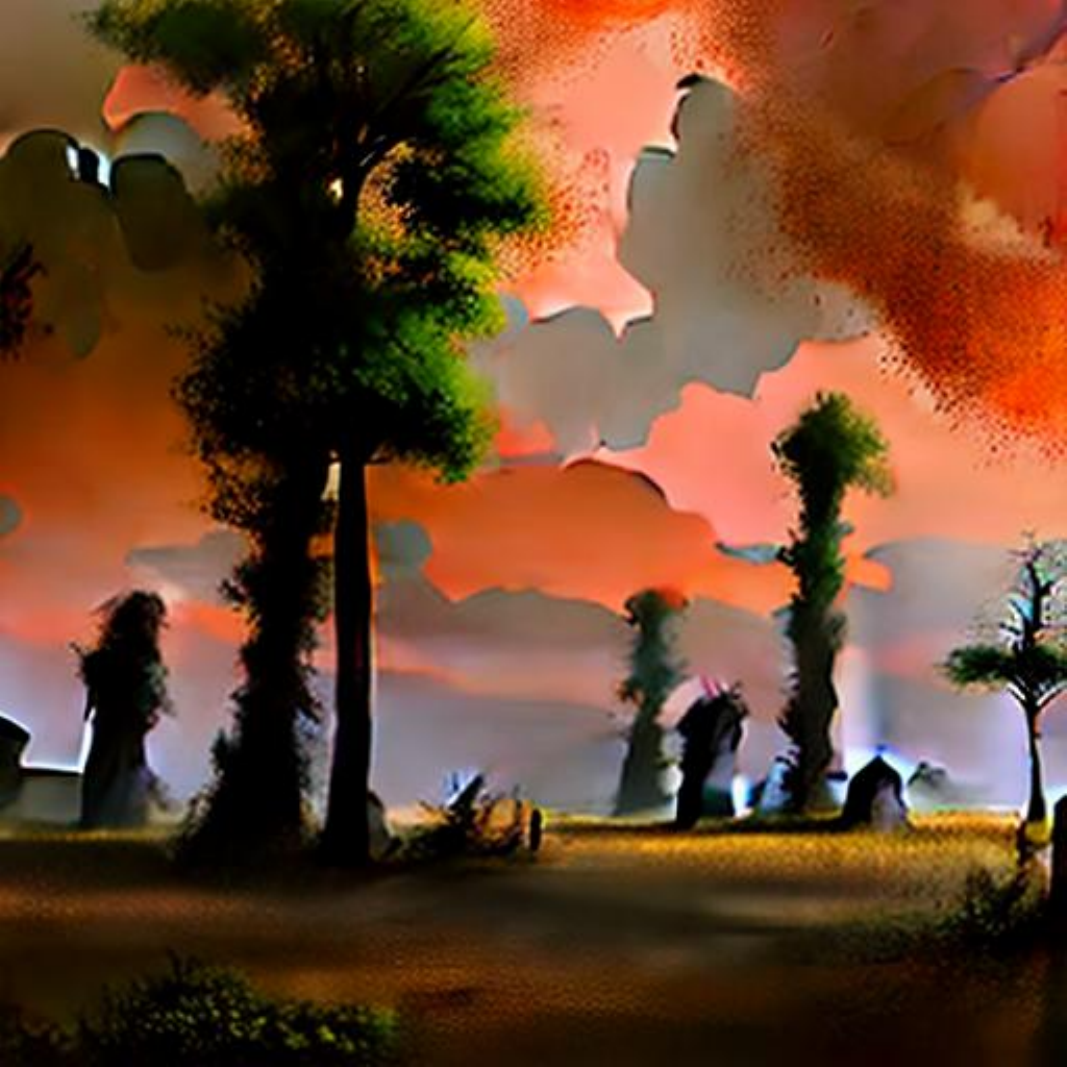}} & 
        \noindent\parbox[c]{0.14\columnwidth}{\includegraphics[width=0.14\columnwidth]{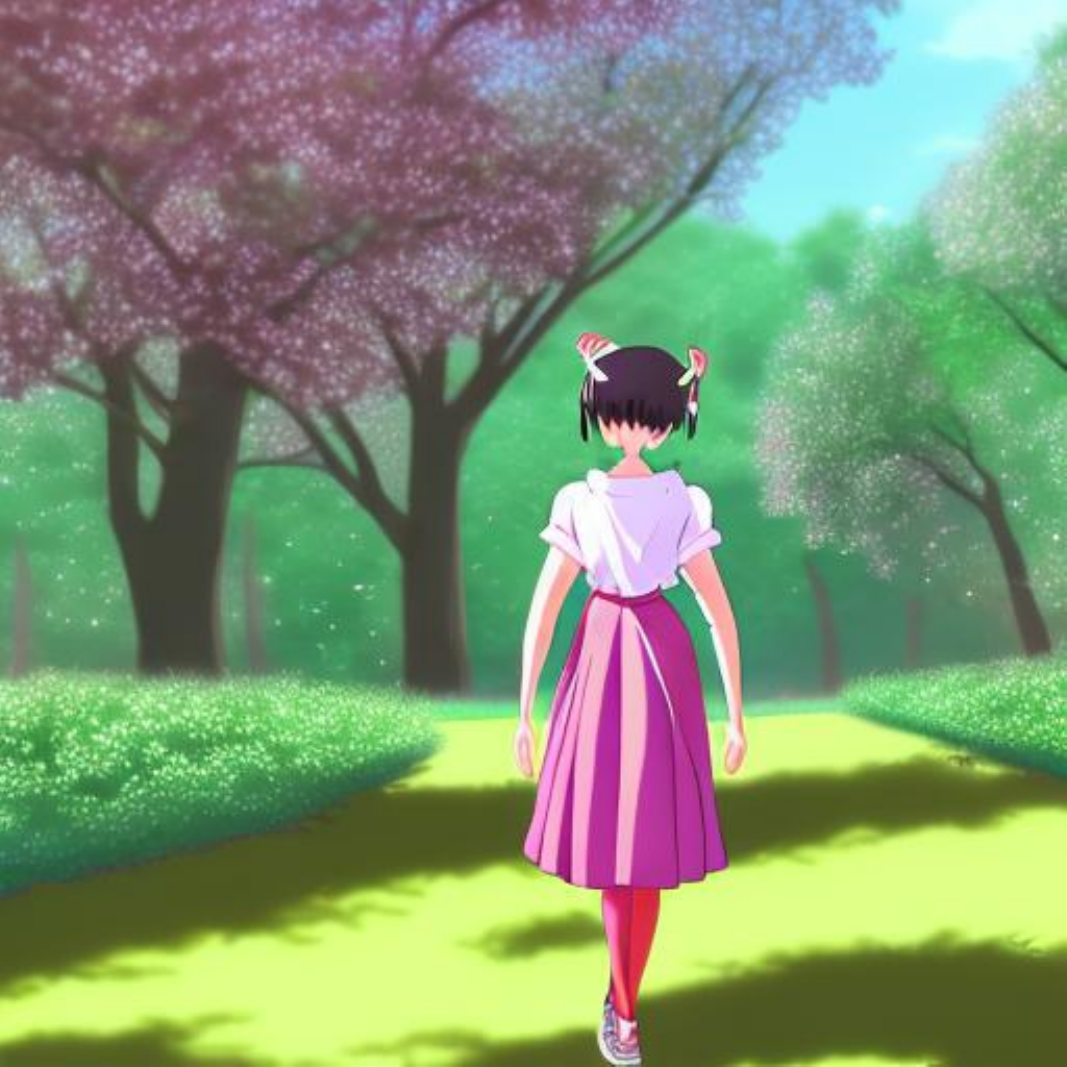}} & 
        \noindent\parbox[c]{0.14\columnwidth}{\includegraphics[width=0.14\columnwidth]{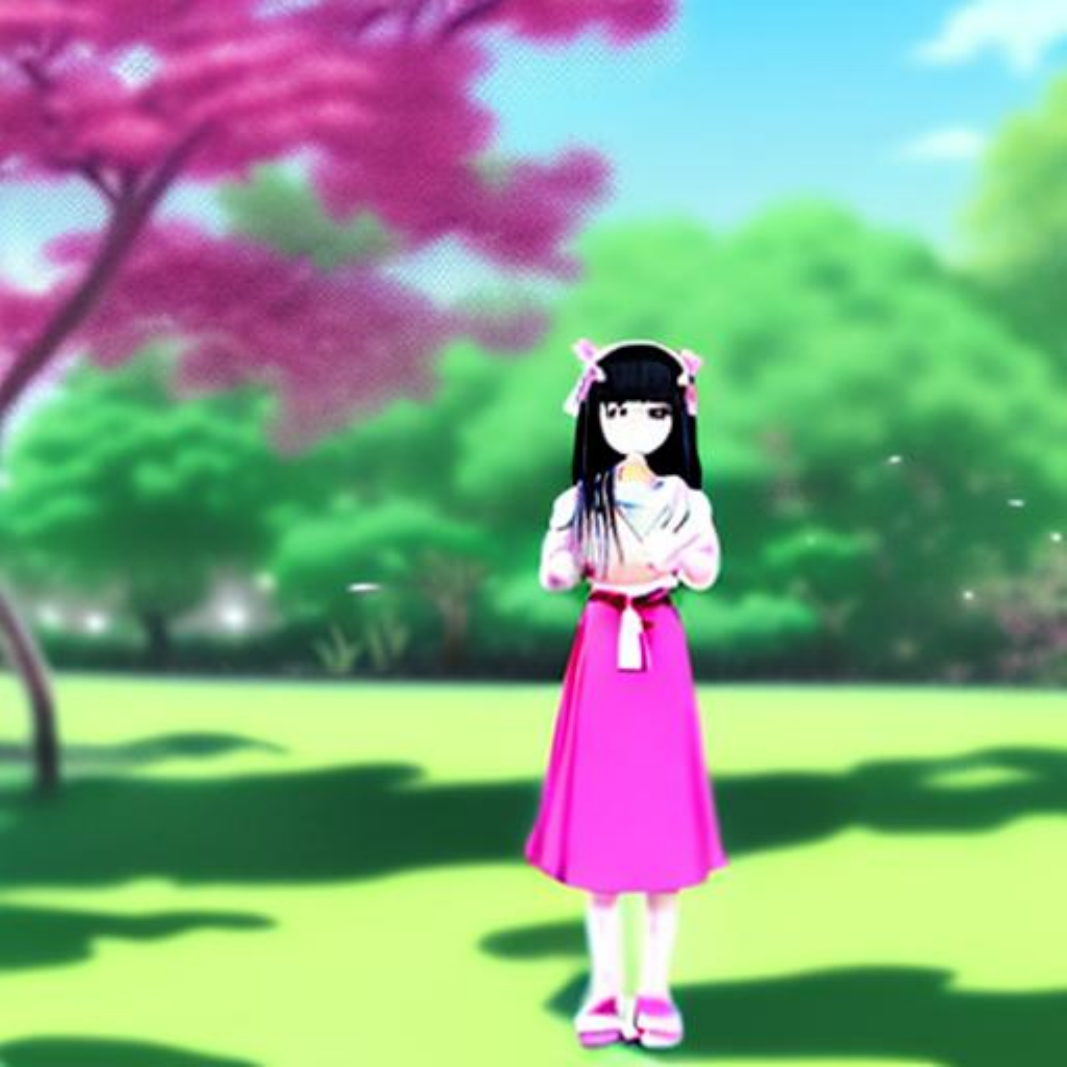}} & 
        \noindent\parbox[c]{0.14\columnwidth}{\includegraphics[width=0.14\columnwidth]{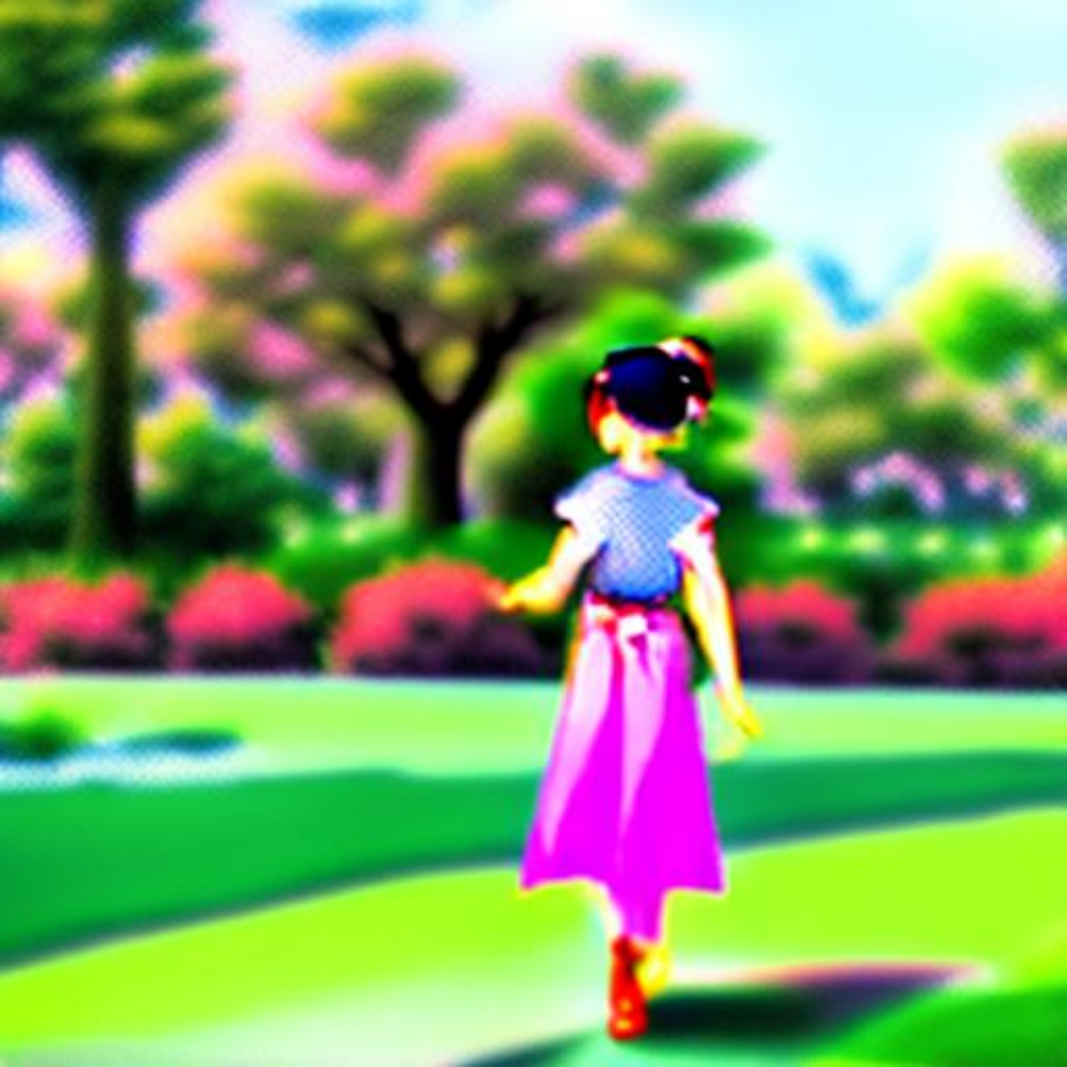}} \\

        \shortstack[l]{\tiny 40 steps} &
        \noindent\parbox[c]{0.14\columnwidth}{\includegraphics[width=0.14\columnwidth]{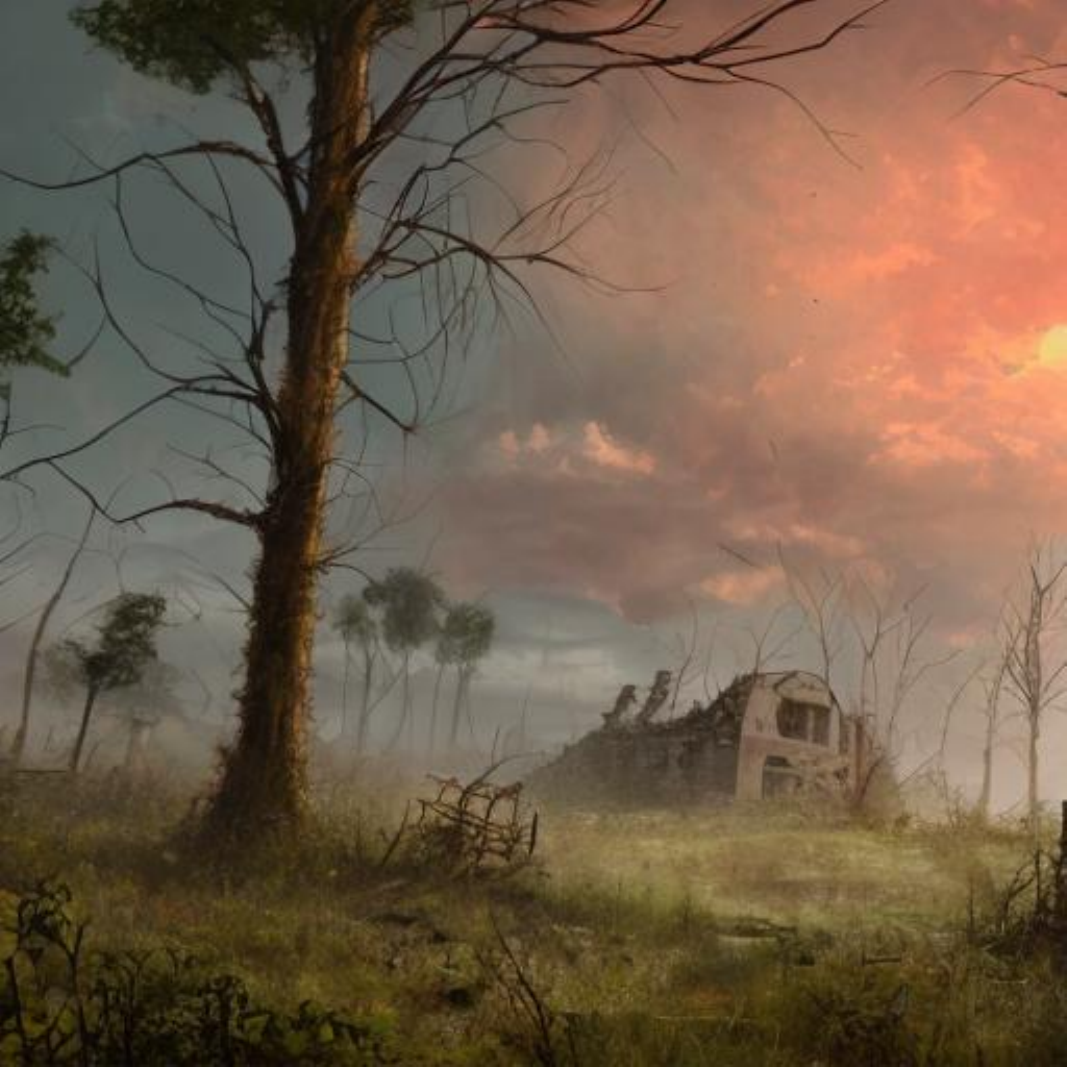}} & 
        \noindent\parbox[c]{0.14\columnwidth}{\includegraphics[width=0.14\columnwidth]{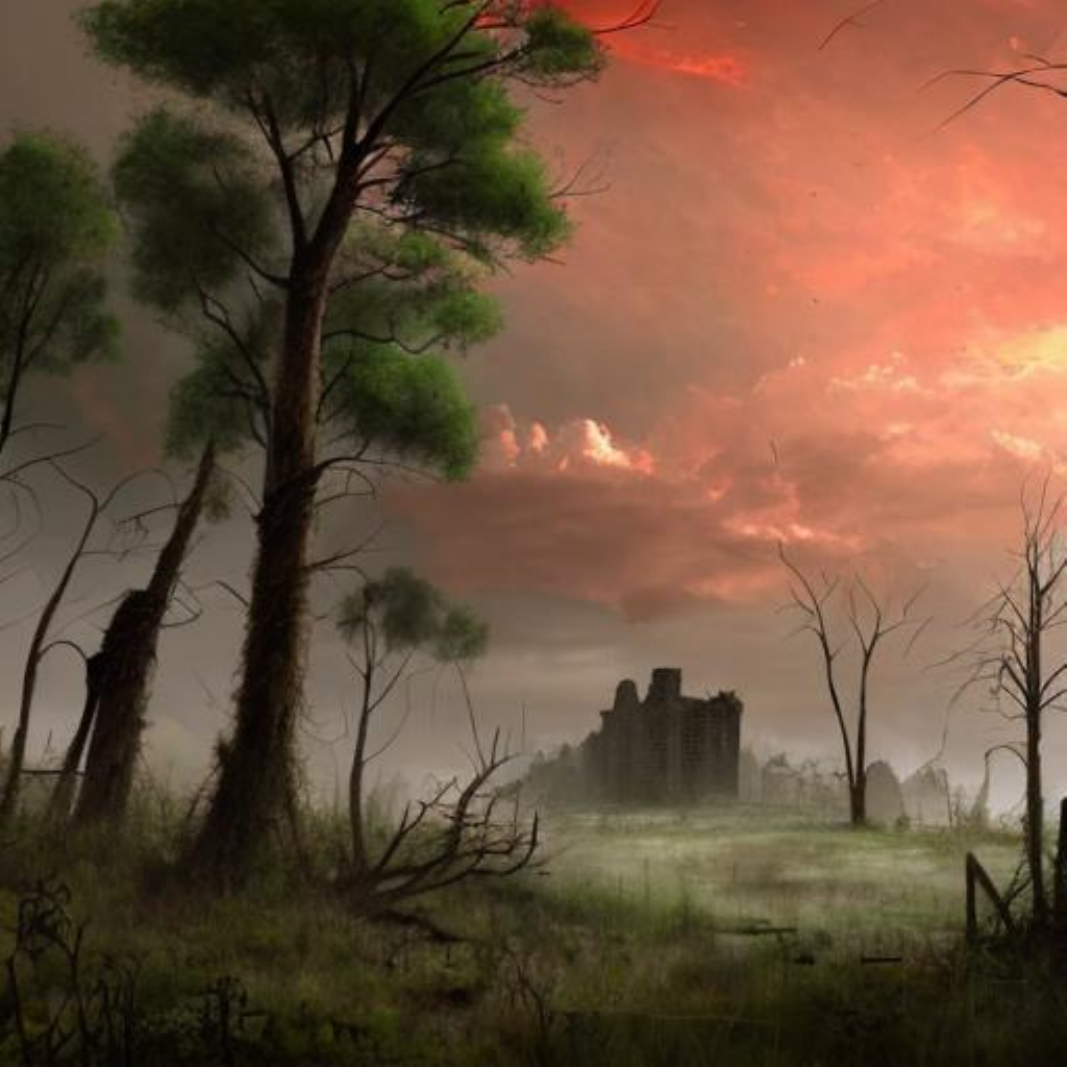}} & 
        \noindent\parbox[c]{0.14\columnwidth}{\includegraphics[width=0.14\columnwidth]{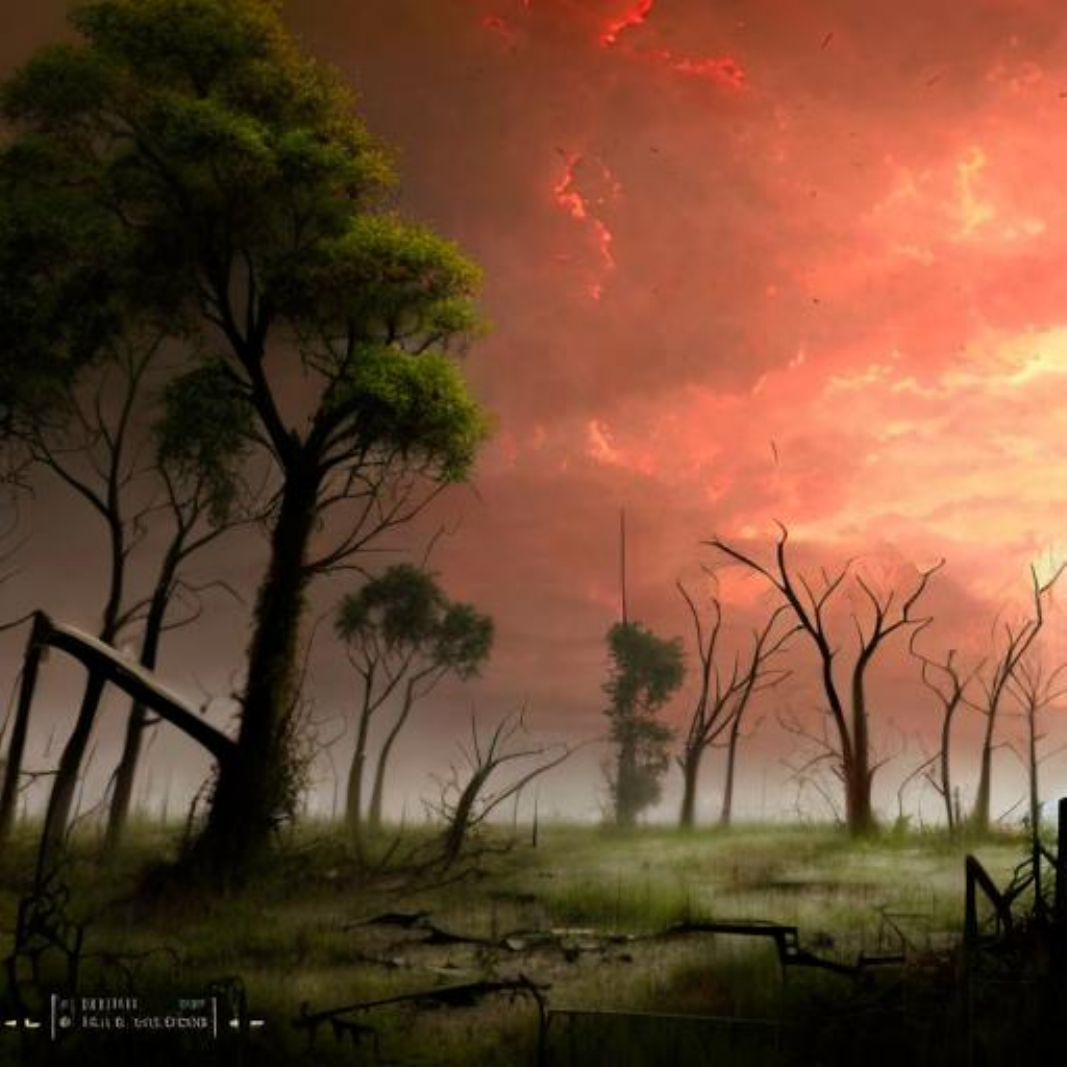}} & 
        \noindent\parbox[c]{0.14\columnwidth}{\includegraphics[width=0.14\columnwidth]{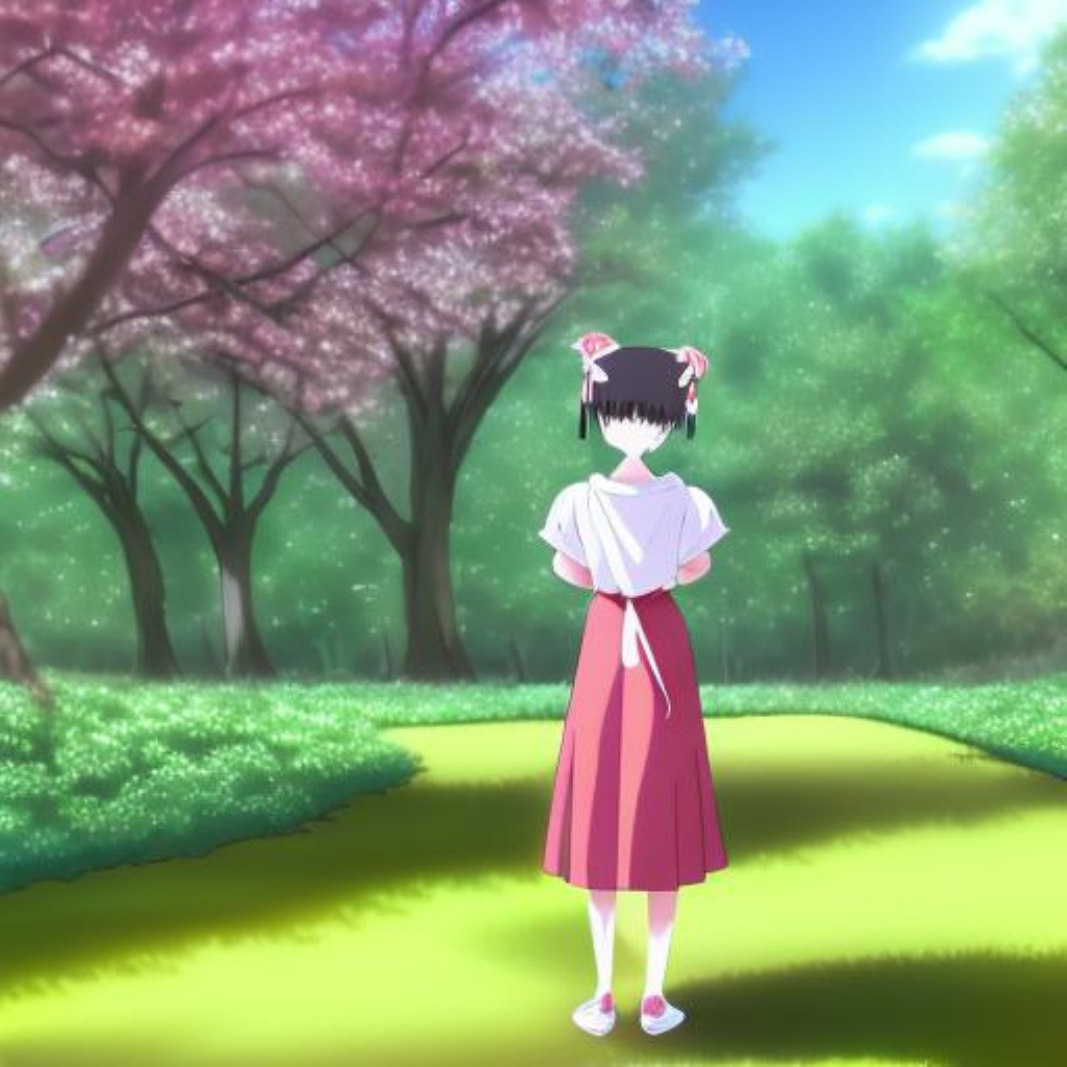}} & 
        \noindent\parbox[c]{0.14\columnwidth}{\includegraphics[width=0.14\columnwidth]{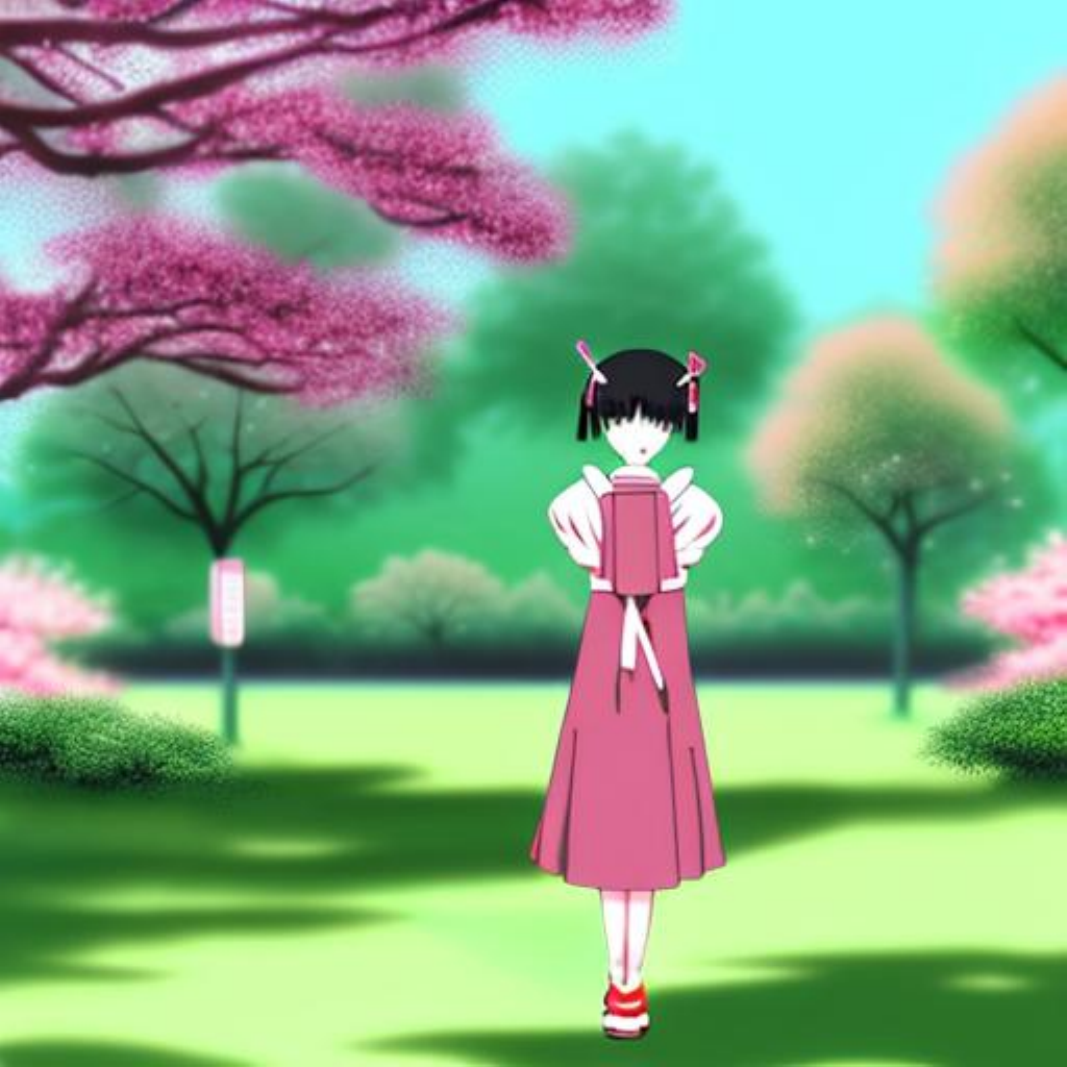}} & 
        \noindent\parbox[c]{0.14\columnwidth}{\includegraphics[width=0.14\columnwidth]{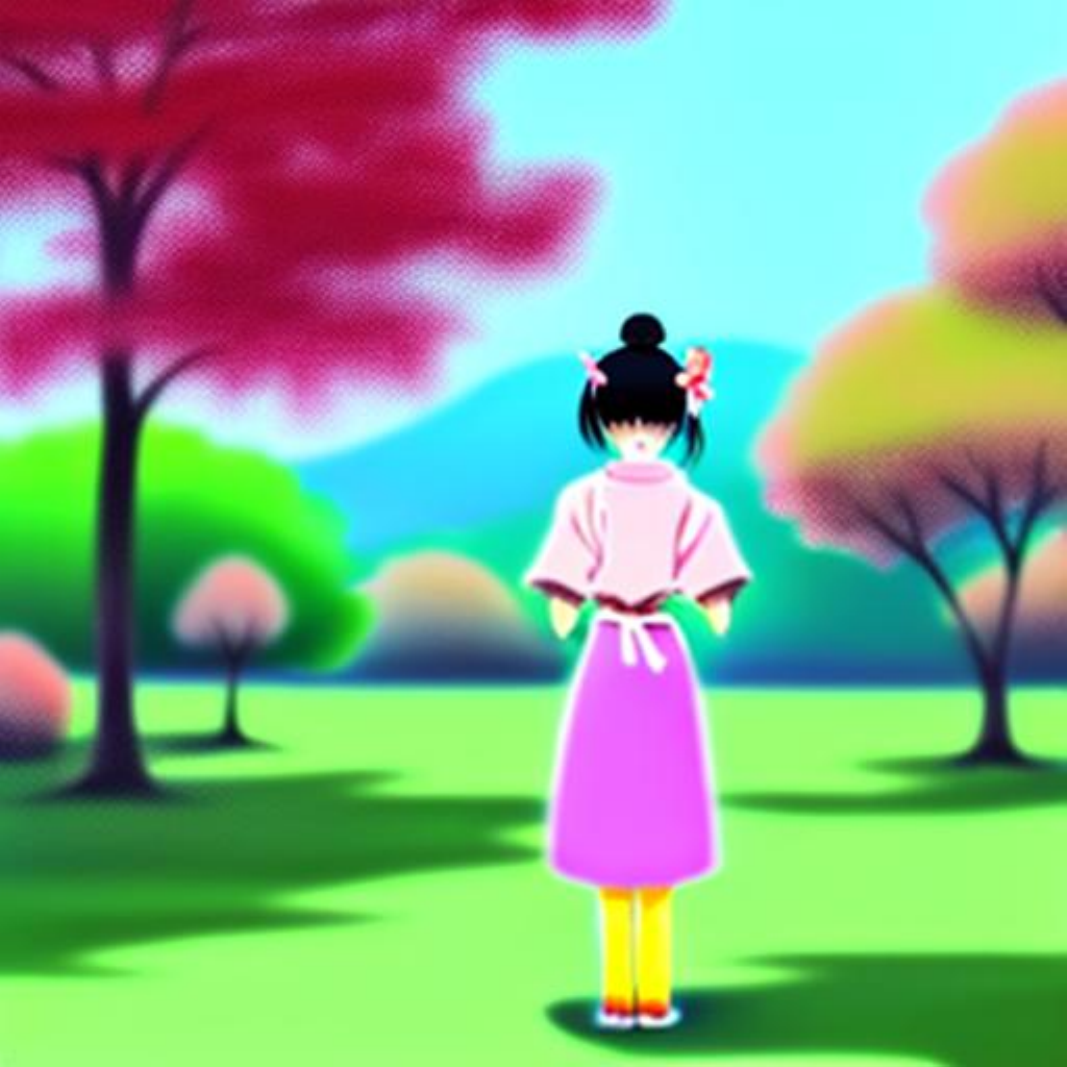}} \\
    \end{tabu}
    \caption{Comparison of samples generated from Openjourney \protect\footnotemark using PLMS4 \cite{liu2022pseudo} with different sampling steps and guidance scale.}
    \label{fig:scale_step_openjourney}
\end{figure}

\footnotetext{\url{https://huggingface.co/prompthero/openjourney}}


\tabulinesep=1pt
\begin{figure}
    \centering
    \begin{tabu} to \textwidth {@{}l@{\hspace{5pt}}c@{\hspace{2pt}}c@{\hspace{2pt}}c@{\hspace{4pt}}c@{\hspace{2pt}}c@{\hspace{2pt}}c@{}}
        & \multicolumn{3}{c}{\shortstack{\scriptsize "A post-apocalyptic world with ruined \\ \scriptsize buildings, overgrown vegetation, and a red sky"}}
        & \multicolumn{3}{c}{\shortstack{\scriptsize "A girl standing in a park in \\ \scriptsize Japanese animation style"}} \\

        & \multicolumn{1}{c}{\shortstack{\scriptsize $s = 7.5$}}
        & \multicolumn{1}{c}{\shortstack{\scriptsize $s = 15$}}
        & \multicolumn{1}{c}{\shortstack{\scriptsize $s = 22.5$}}
        & \multicolumn{1}{c}{\shortstack{\scriptsize $s = 7.5$}}
        & \multicolumn{1}{c}{\shortstack{\scriptsize $s = 15$}}
        & \multicolumn{1}{c}{\shortstack{\scriptsize $s = 22.5$}}
        \\
        
        \shortstack[l]{\tiny 10 steps} &
        \noindent\parbox[c]{0.14\columnwidth}{\includegraphics[width=0.14\columnwidth]{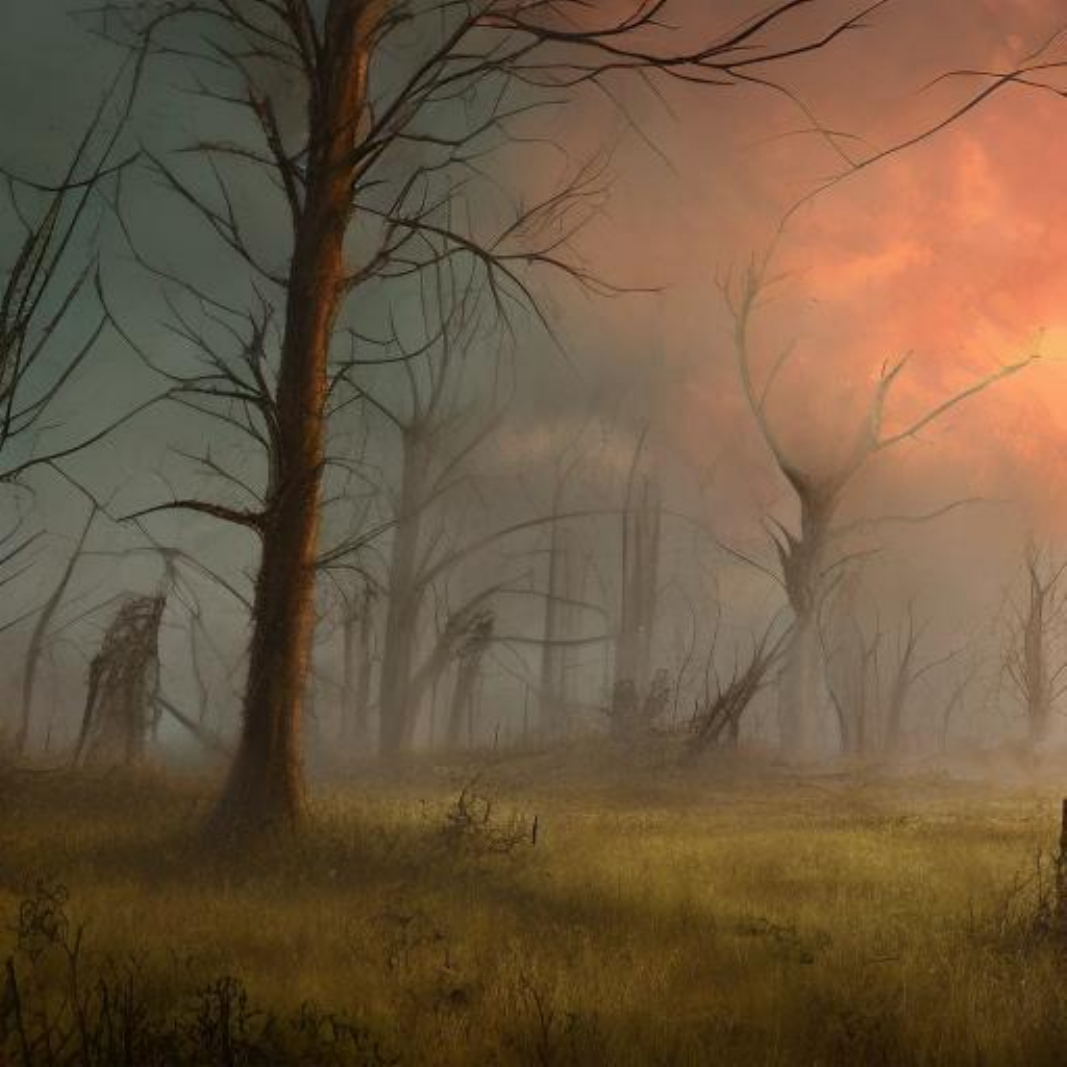}} & 
        \noindent\parbox[c]{0.14\columnwidth}{\includegraphics[width=0.14\columnwidth]{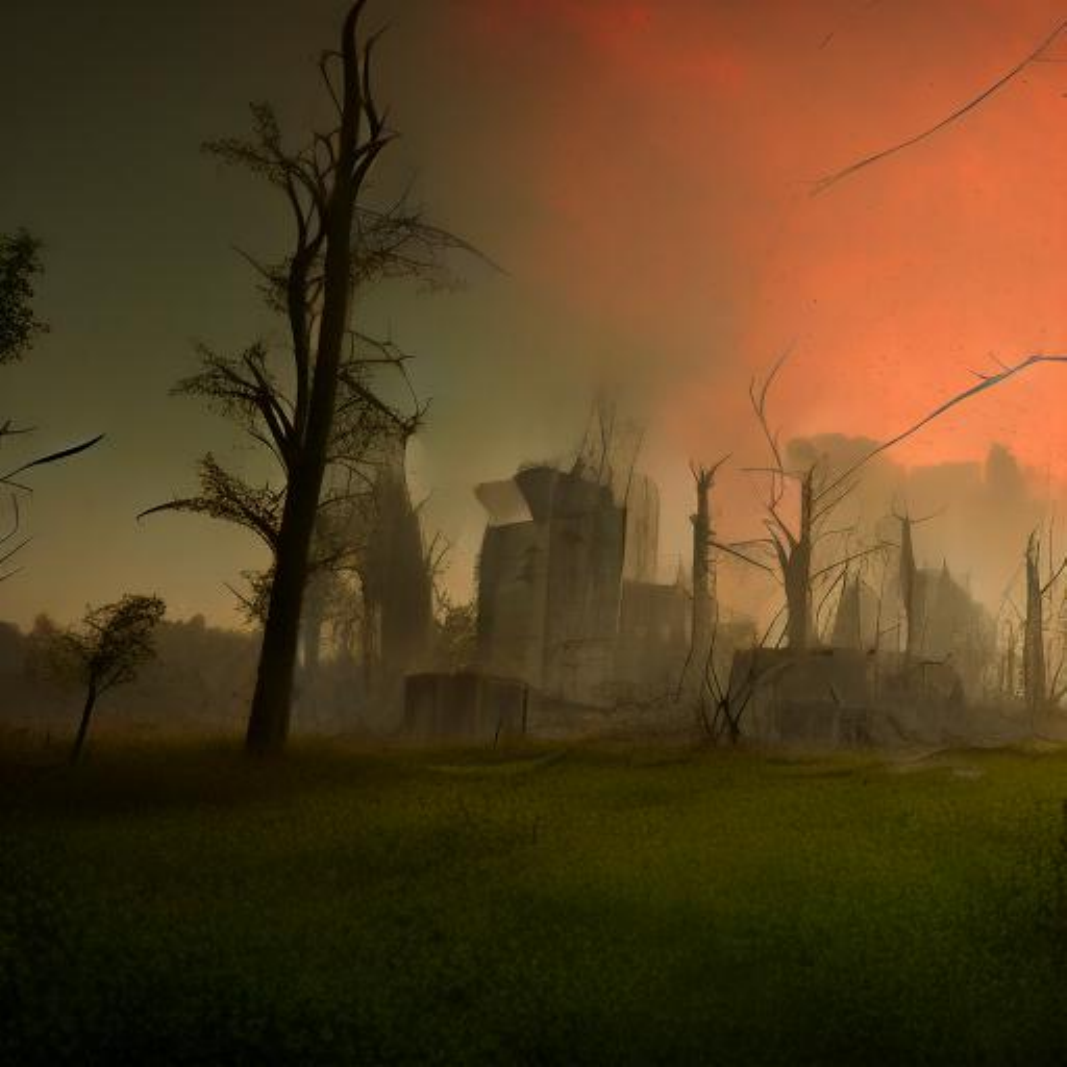}} & 
        \noindent\parbox[c]{0.14\columnwidth}{\includegraphics[width=0.14\columnwidth]{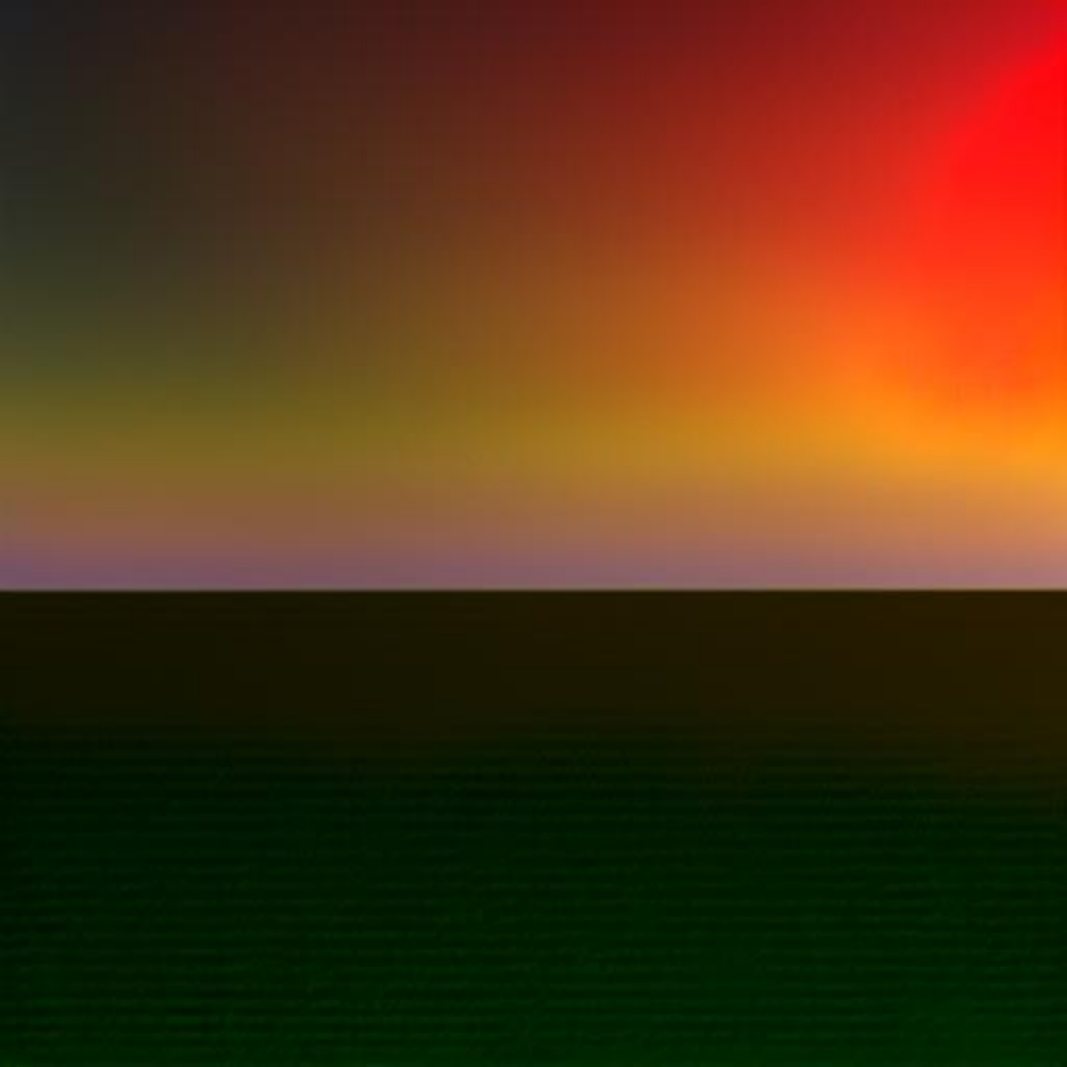}} & 
        \noindent\parbox[c]{0.14\columnwidth}{\includegraphics[width=0.14\columnwidth]{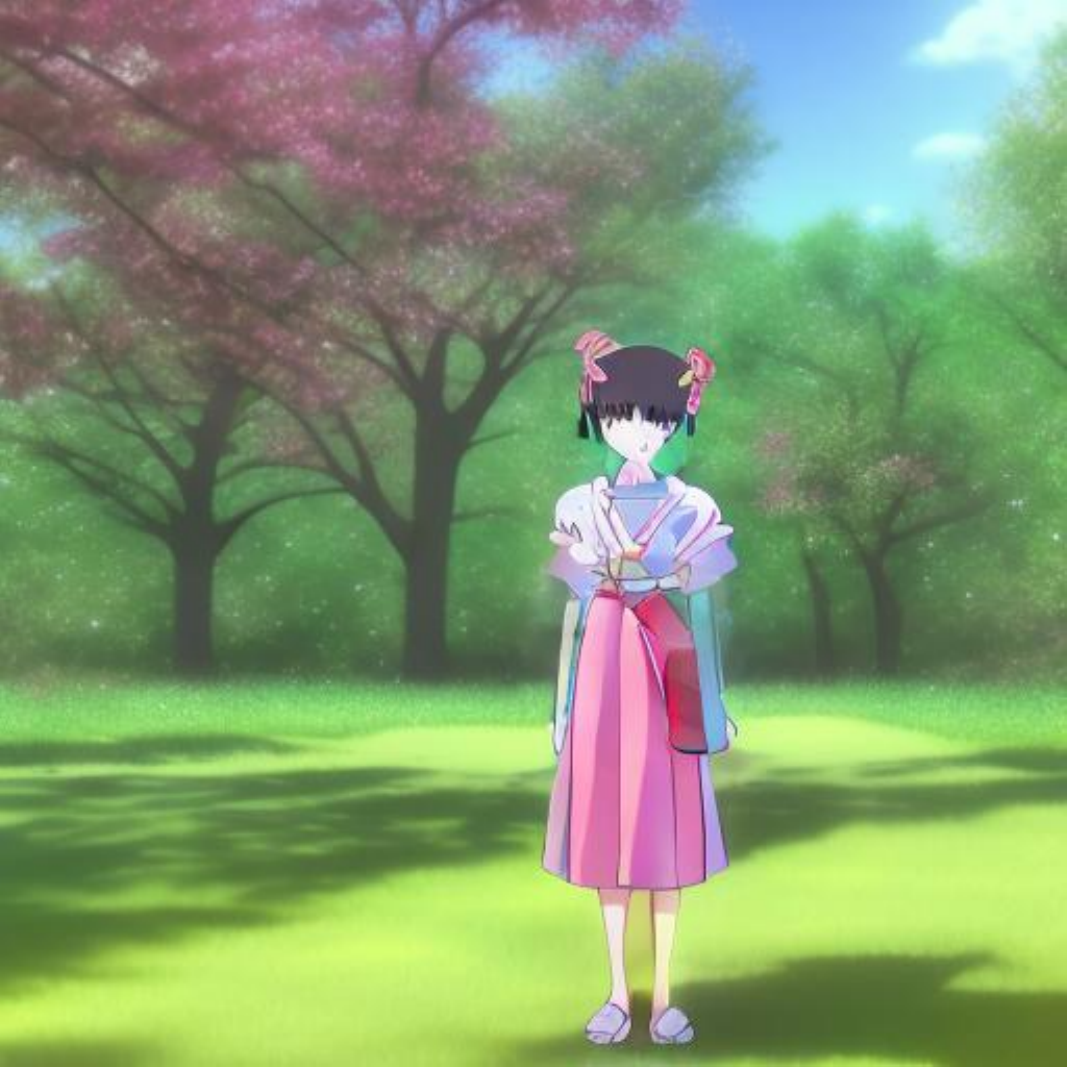}} & 
        \noindent\parbox[c]{0.14\columnwidth}{\includegraphics[width=0.14\columnwidth]{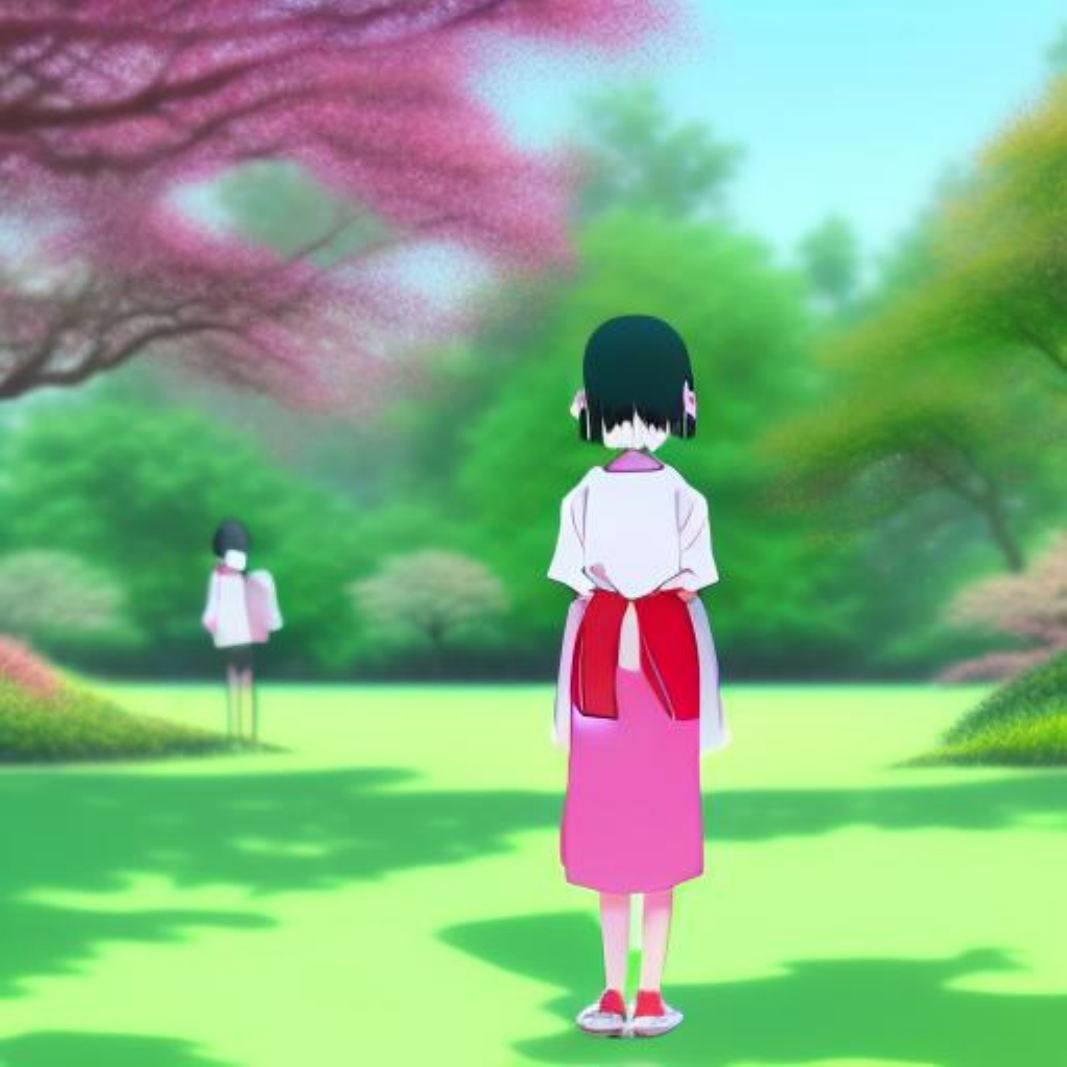}} & 
        \noindent\parbox[c]{0.14\columnwidth}{\includegraphics[width=0.14\columnwidth]{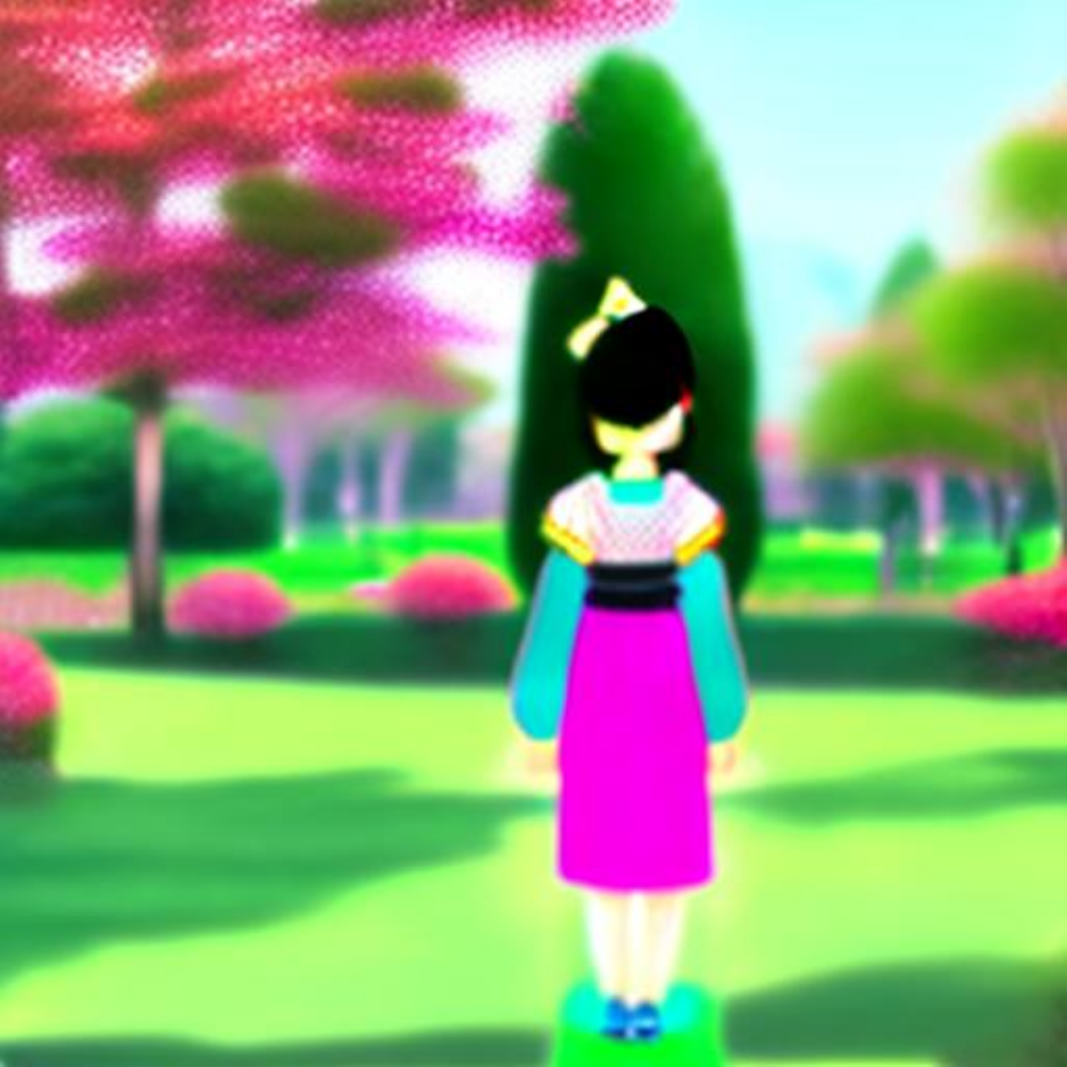}} \\

        \shortstack[l]{\tiny 15 steps} &
        \noindent\parbox[c]{0.14\columnwidth}{\includegraphics[width=0.14\columnwidth]{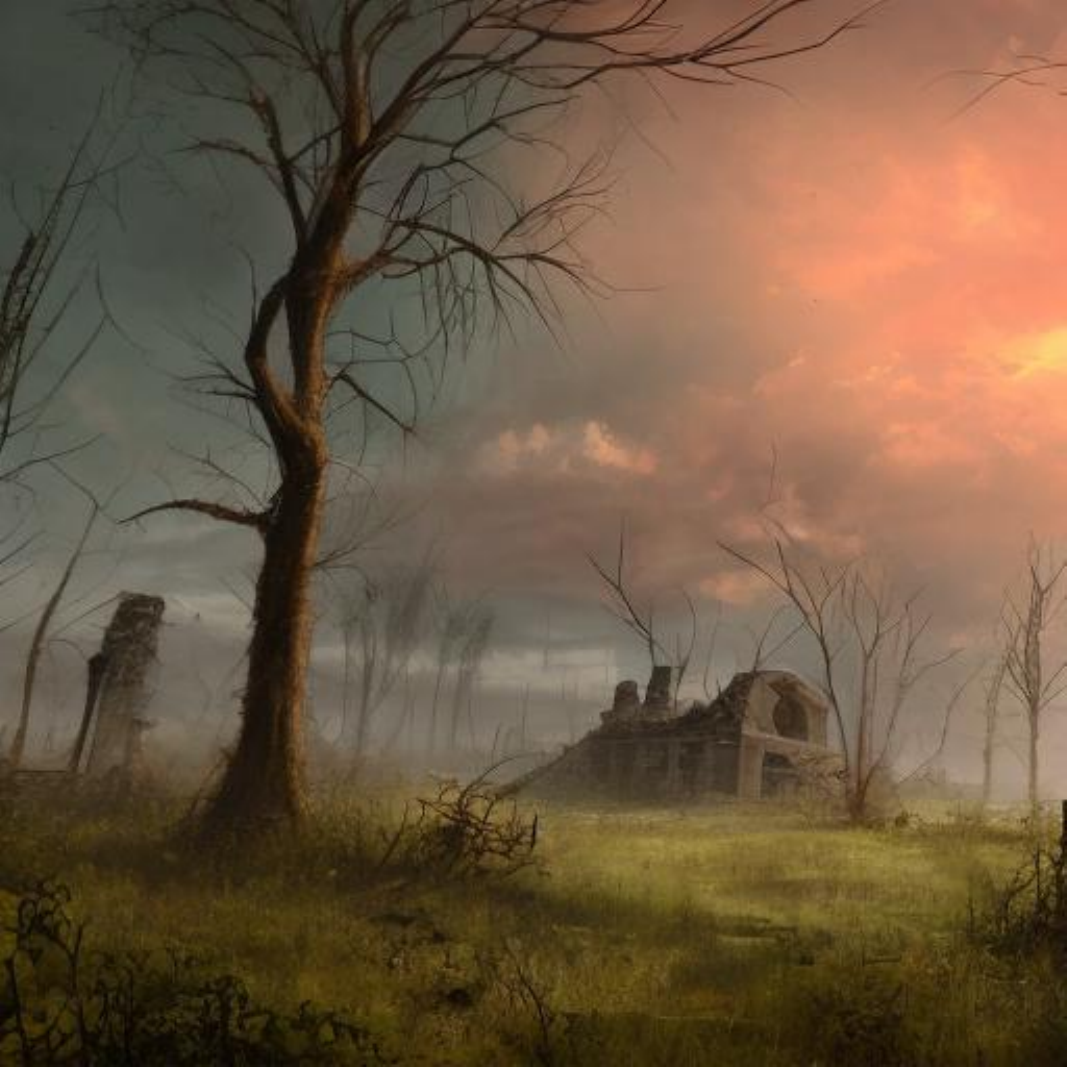}} & 
        \noindent\parbox[c]{0.14\columnwidth}{\includegraphics[width=0.14\columnwidth]{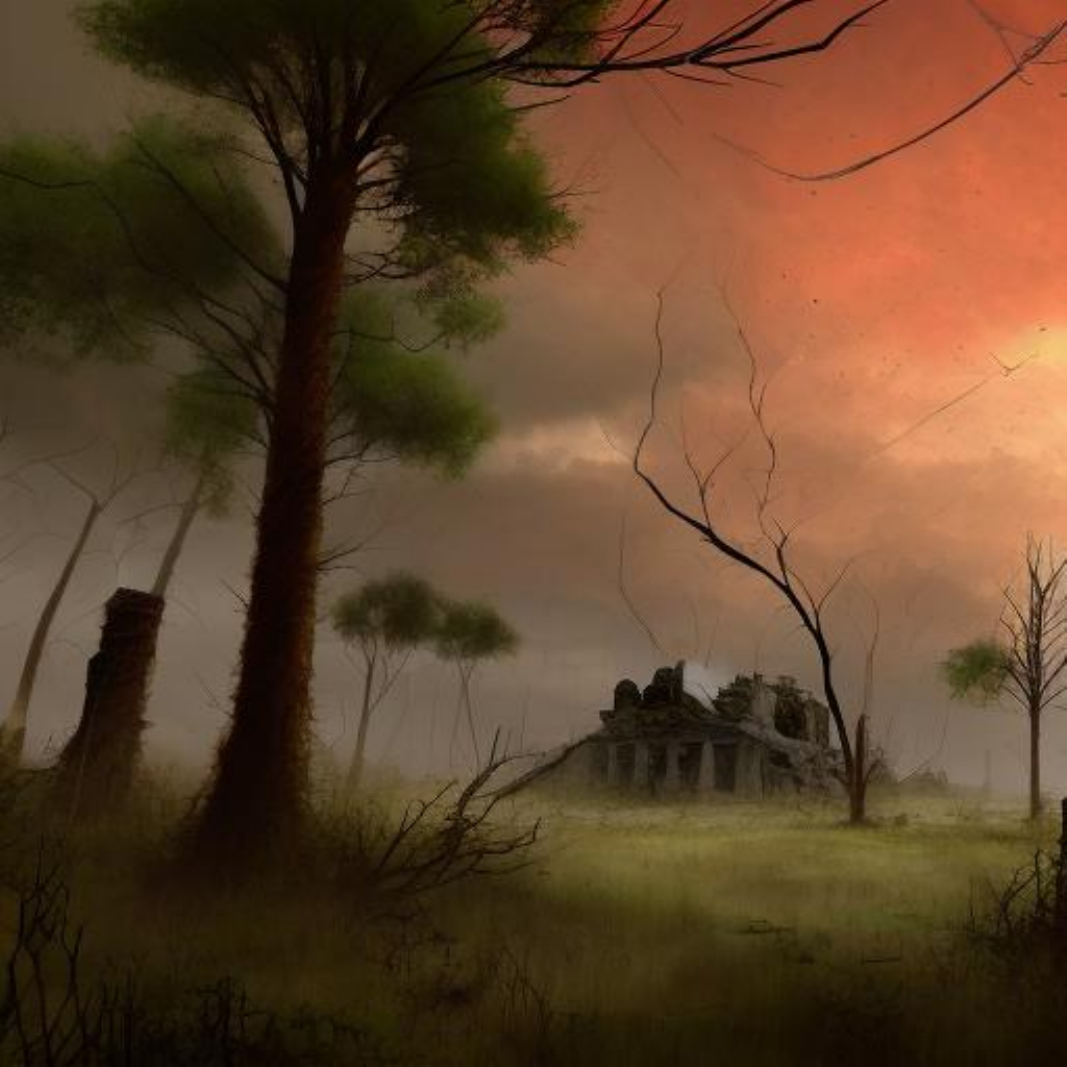}} & 
        \noindent\parbox[c]{0.14\columnwidth}{\includegraphics[width=0.14\columnwidth]{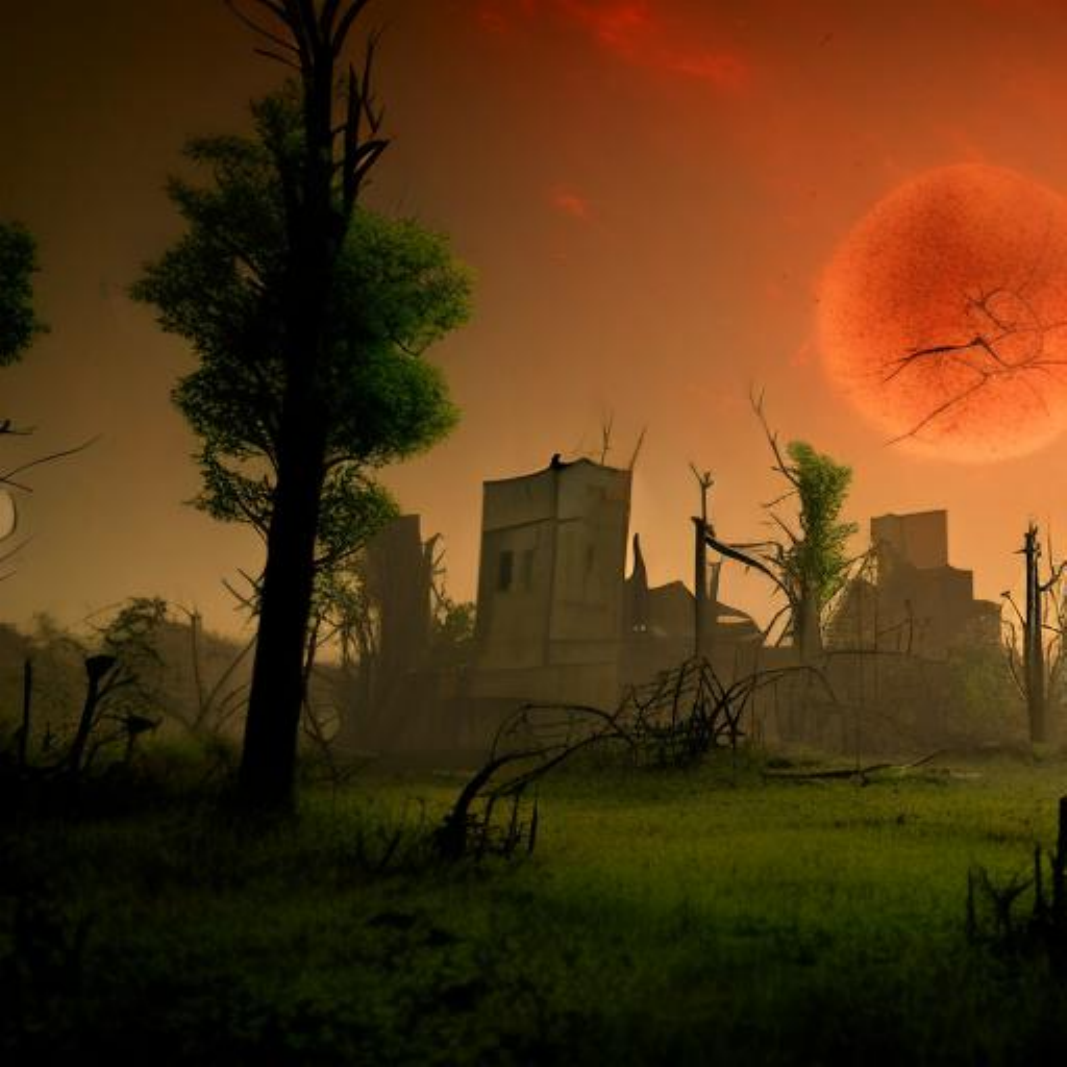}} & 
        \noindent\parbox[c]{0.14\columnwidth}{\includegraphics[width=0.14\columnwidth]{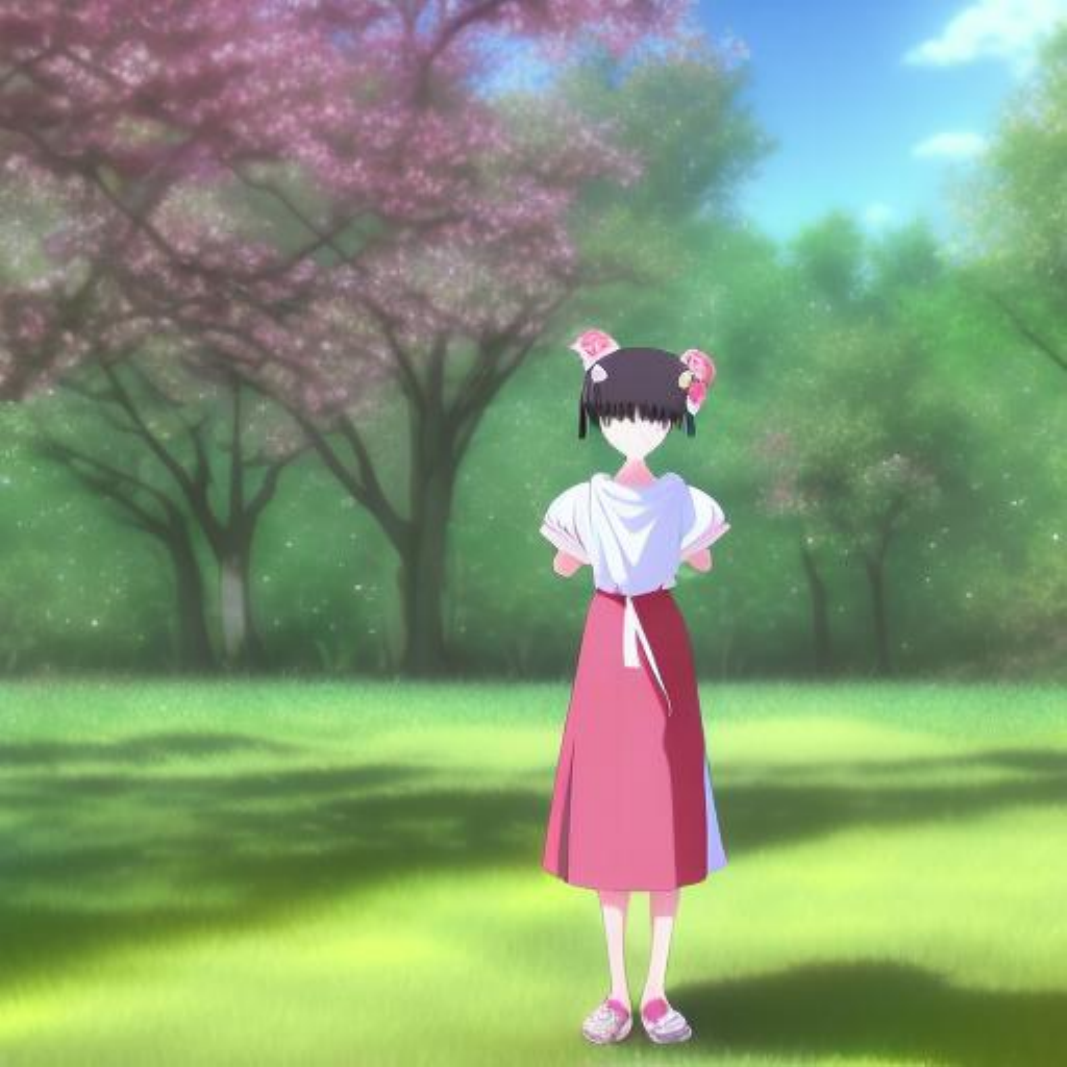}} & 
        \noindent\parbox[c]{0.14\columnwidth}{\includegraphics[width=0.14\columnwidth]{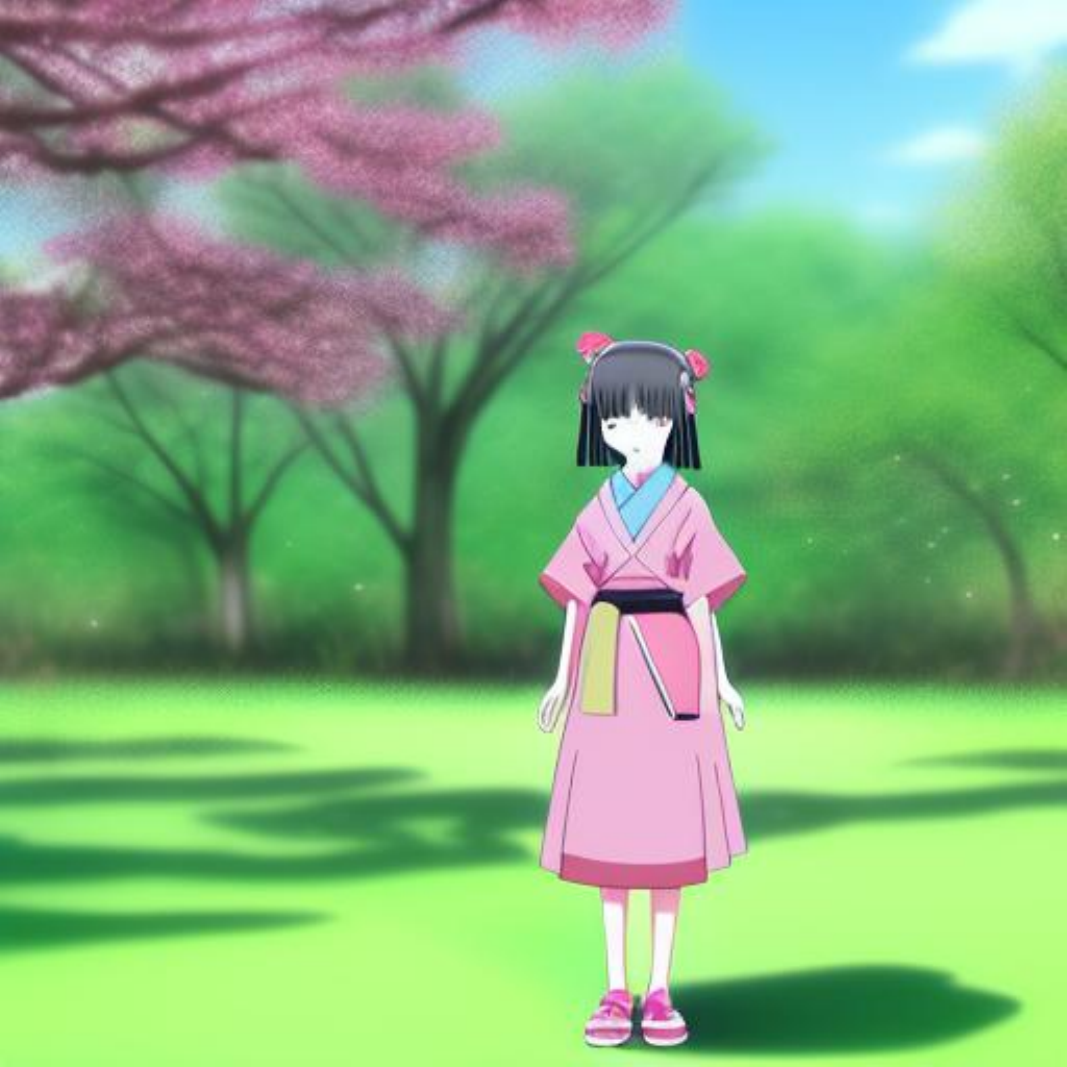}} & 
        \noindent\parbox[c]{0.14\columnwidth}{\includegraphics[width=0.14\columnwidth]{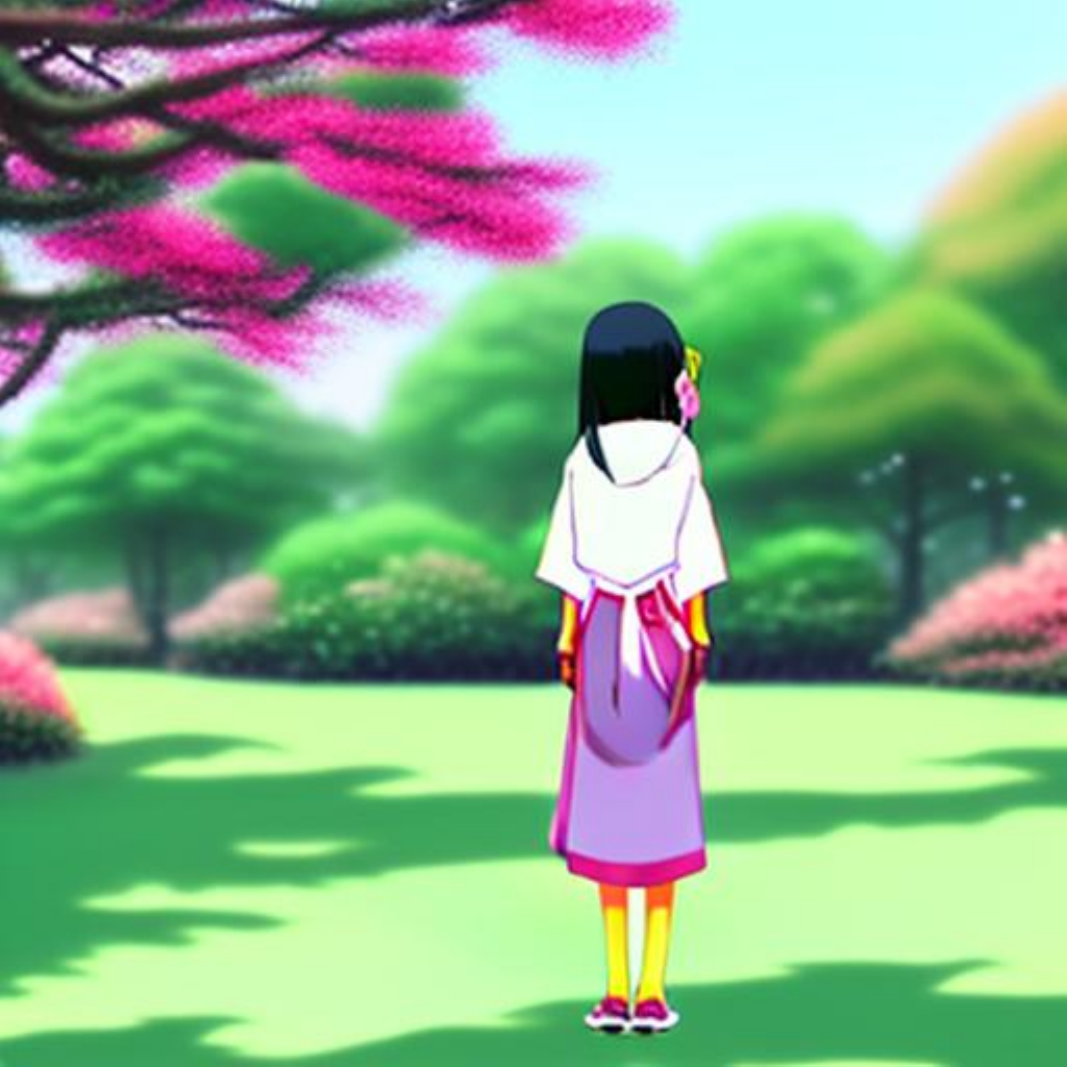}} \\

        \shortstack[l]{\tiny 20 steps} &
        \noindent\parbox[c]{0.14\columnwidth}{\includegraphics[width=0.14\columnwidth]{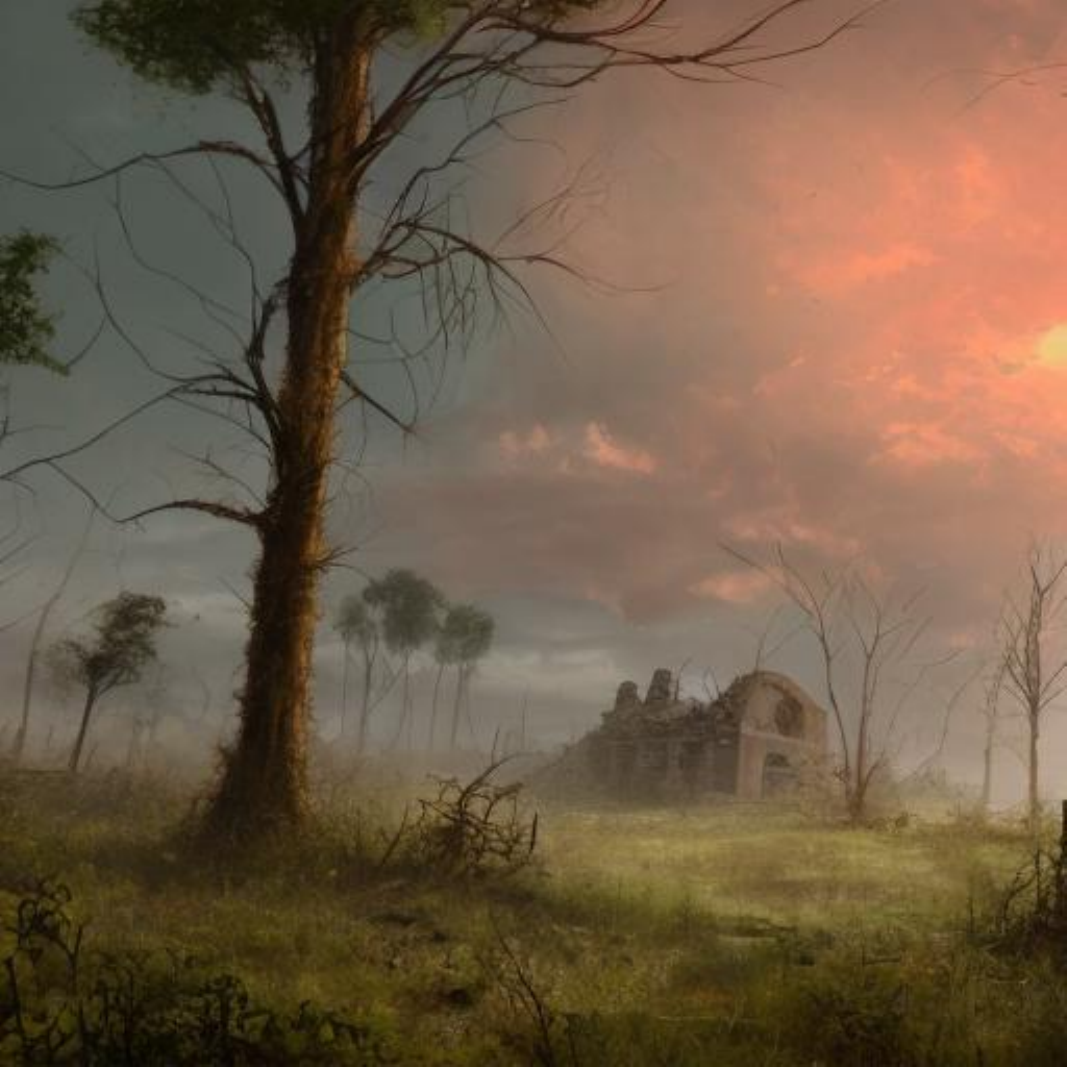}} & 
        \noindent\parbox[c]{0.14\columnwidth}{\includegraphics[width=0.14\columnwidth]{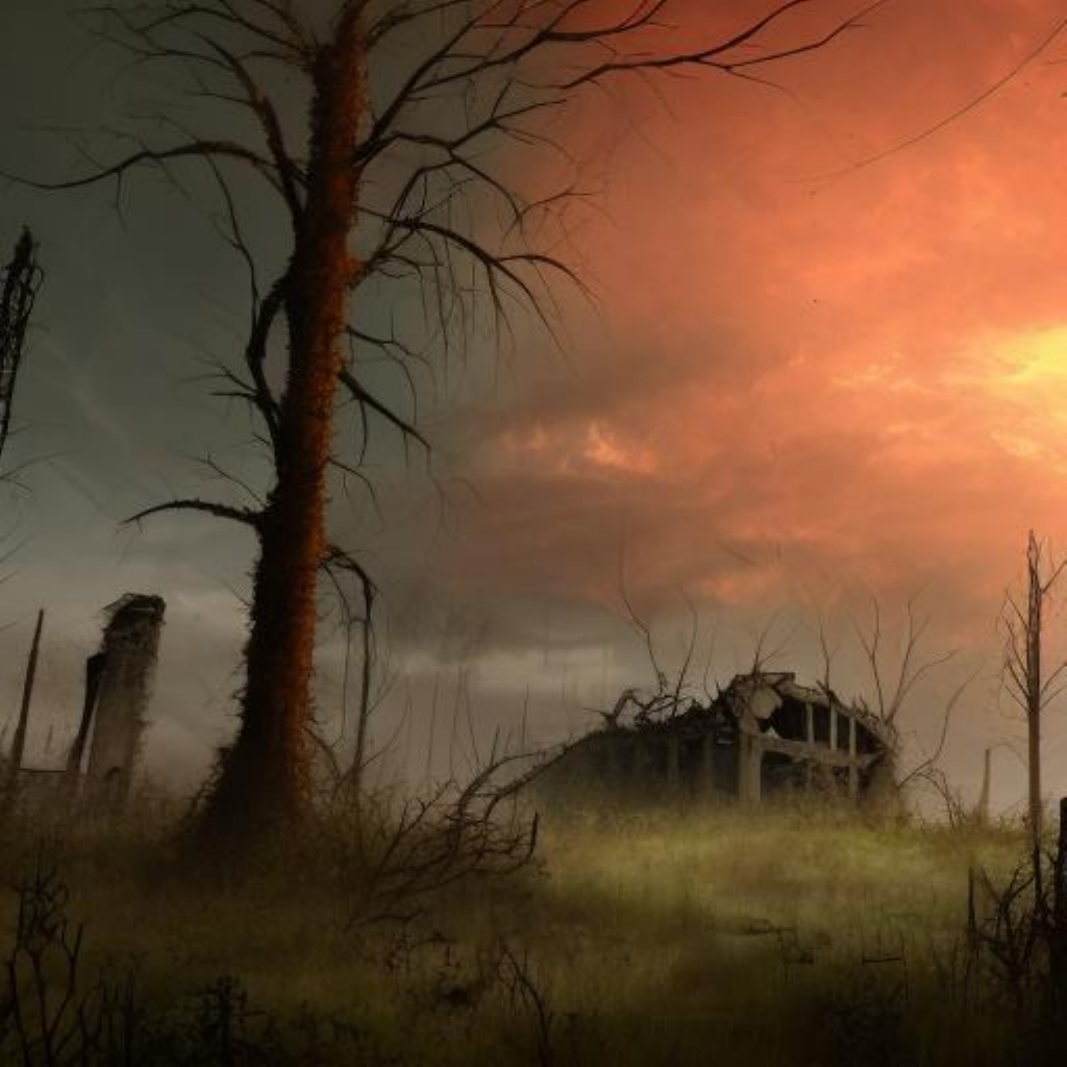}} & 
        \noindent\parbox[c]{0.14\columnwidth}{\includegraphics[width=0.14\columnwidth]{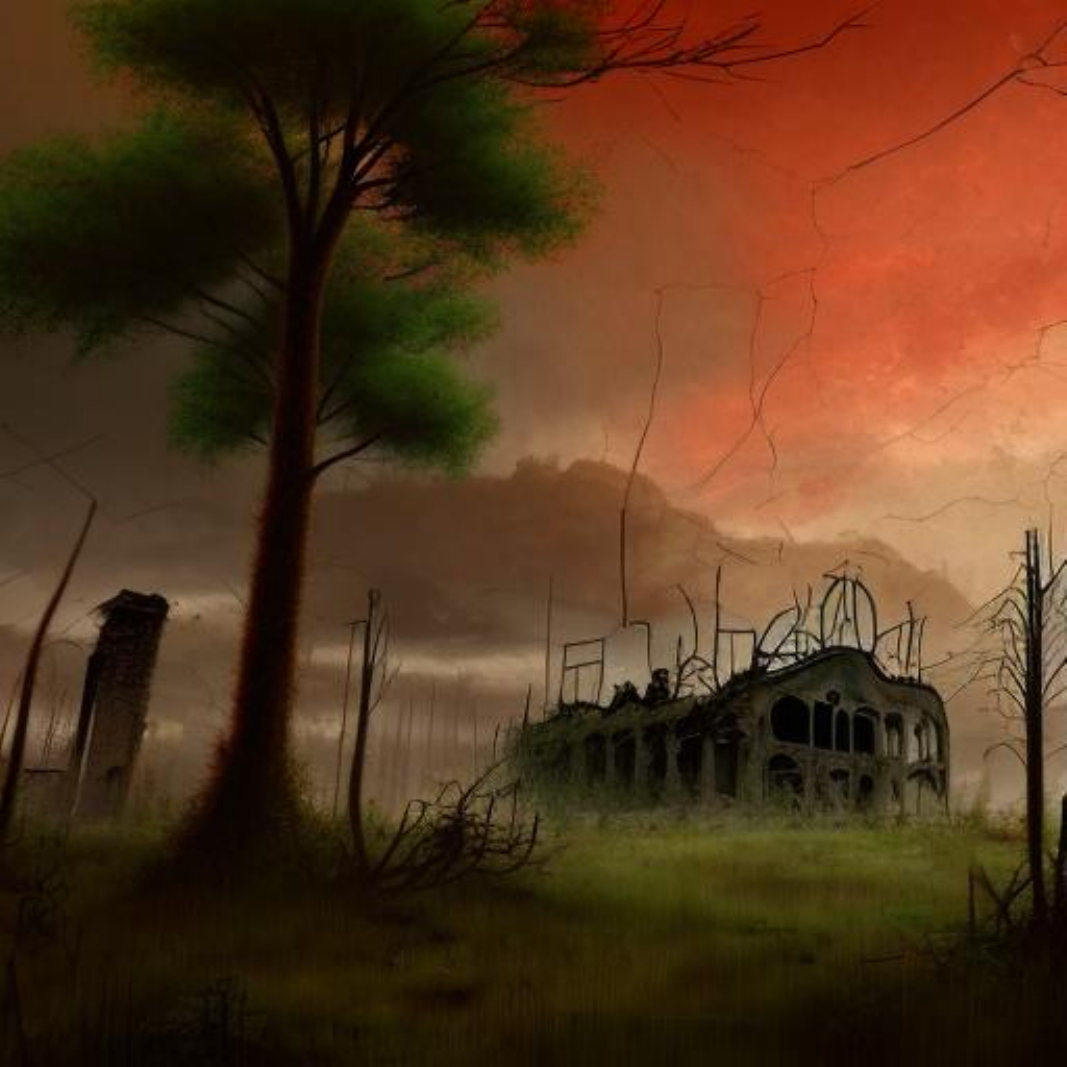}} & 
        \noindent\parbox[c]{0.14\columnwidth}{\includegraphics[width=0.14\columnwidth]{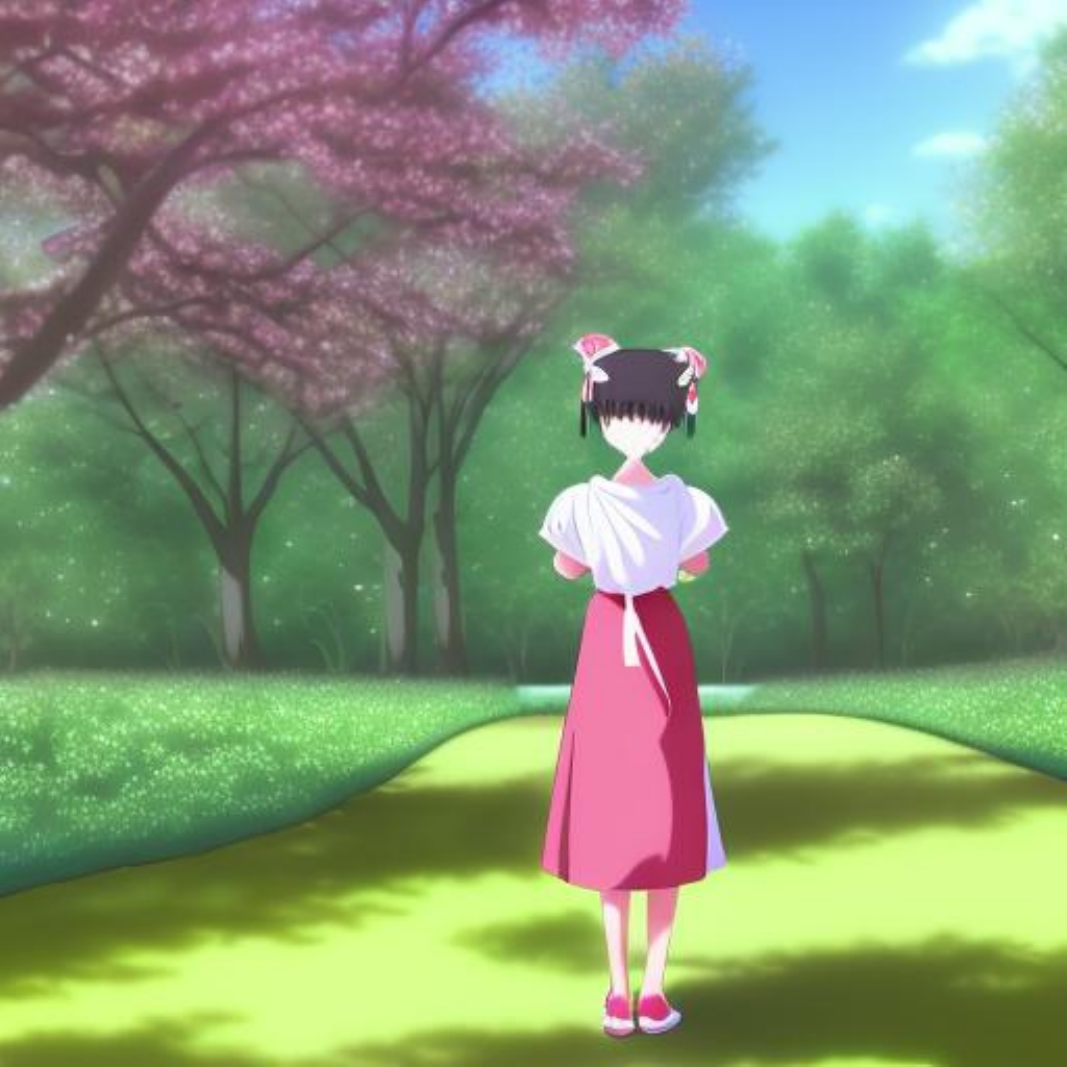}} & 
        \noindent\parbox[c]{0.14\columnwidth}{\includegraphics[width=0.14\columnwidth]{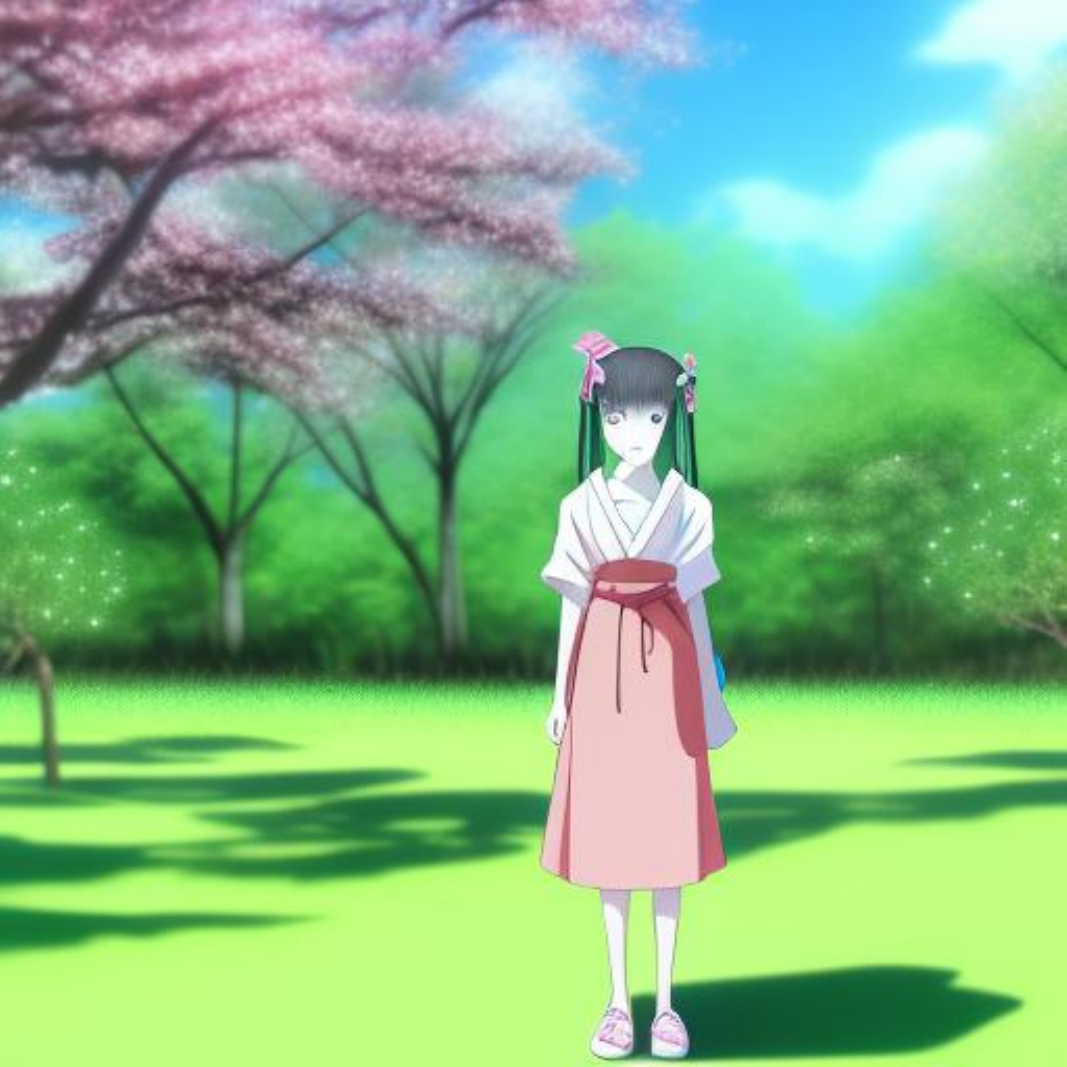}} & 
        \noindent\parbox[c]{0.14\columnwidth}{\includegraphics[width=0.14\columnwidth]{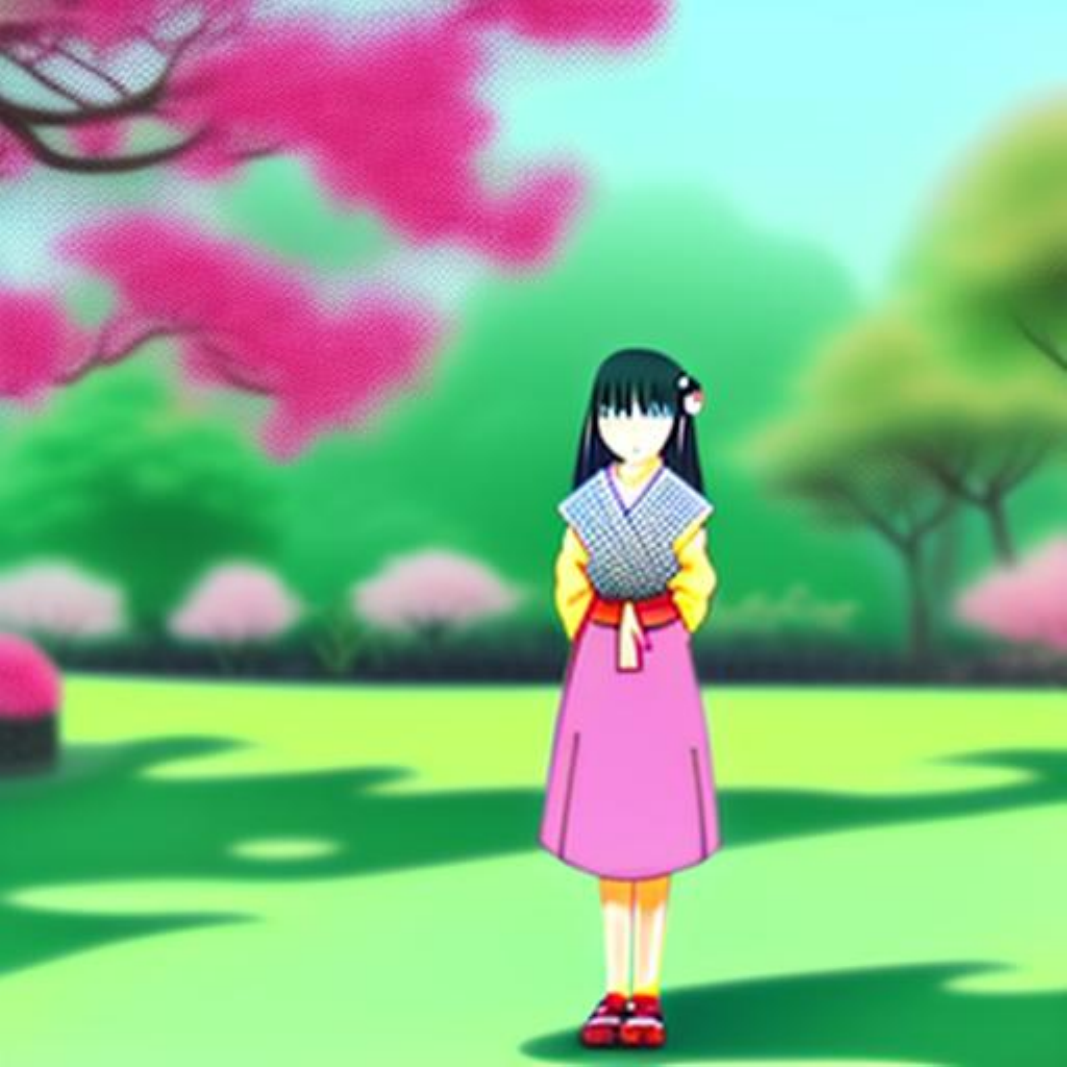}} \\

        \shortstack[l]{\tiny 40 steps} &
        \noindent\parbox[c]{0.14\columnwidth}{\includegraphics[width=0.14\columnwidth]{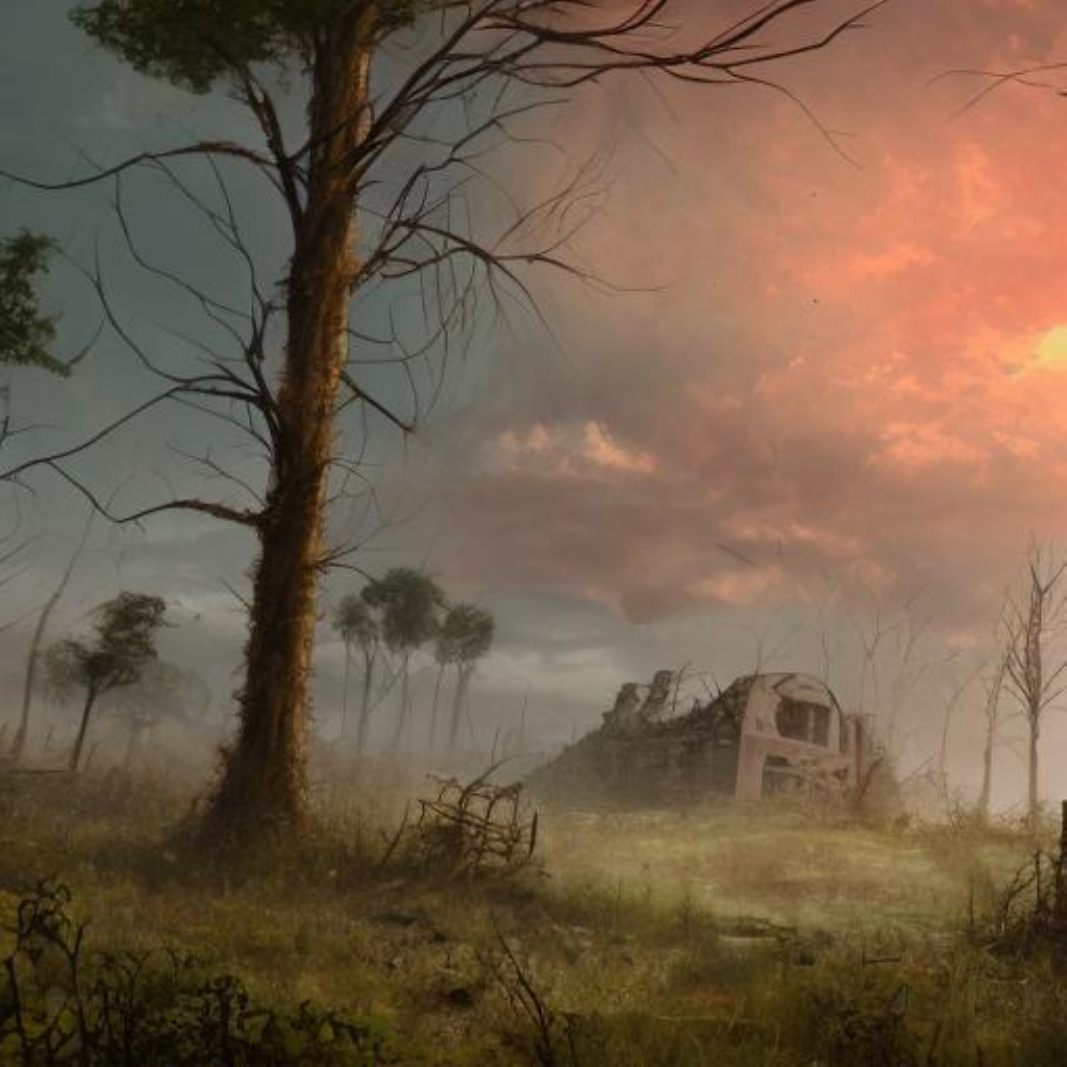}} & 
        \noindent\parbox[c]{0.14\columnwidth}{\includegraphics[width=0.14\columnwidth]{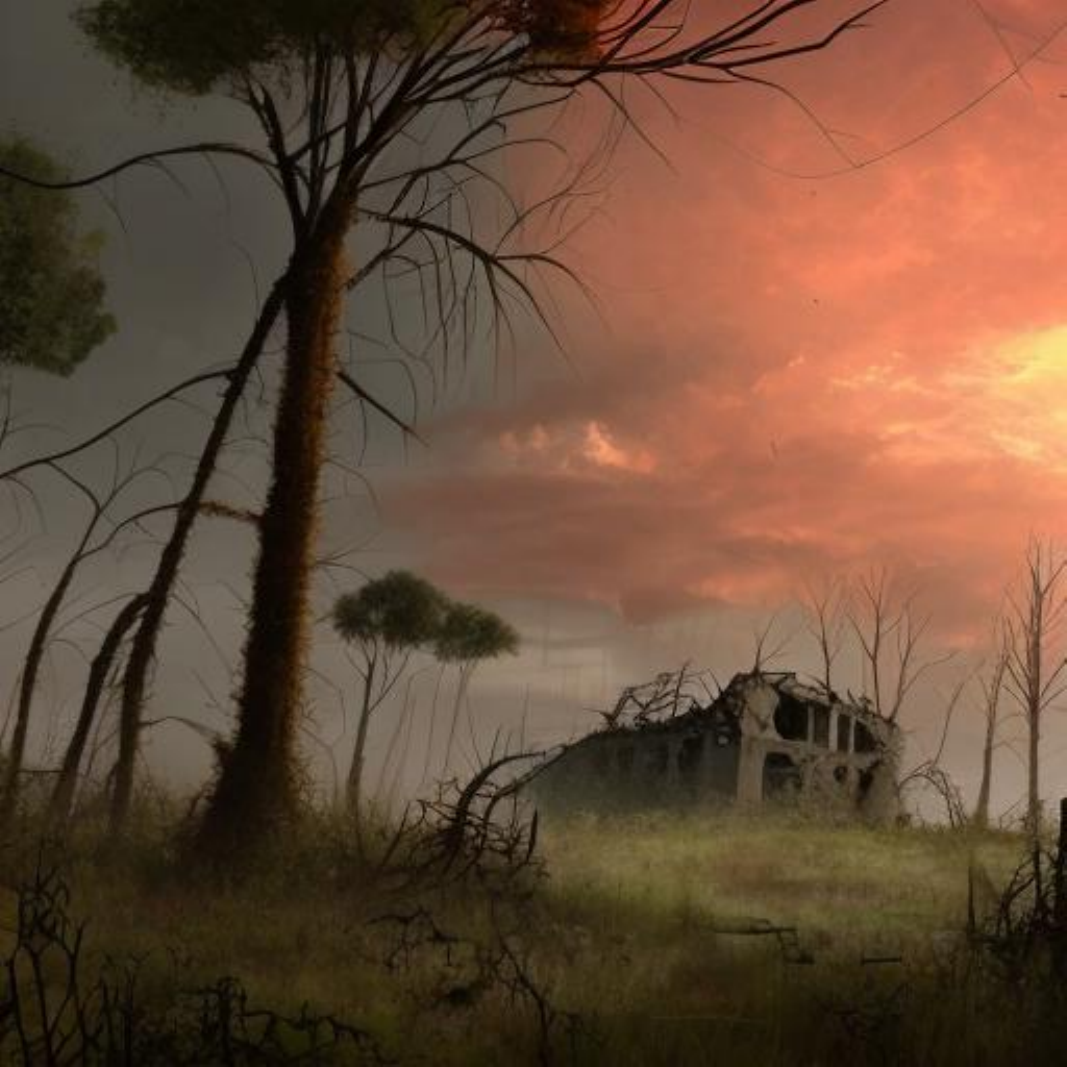}} & 
        \noindent\parbox[c]{0.14\columnwidth}{\includegraphics[width=0.14\columnwidth]{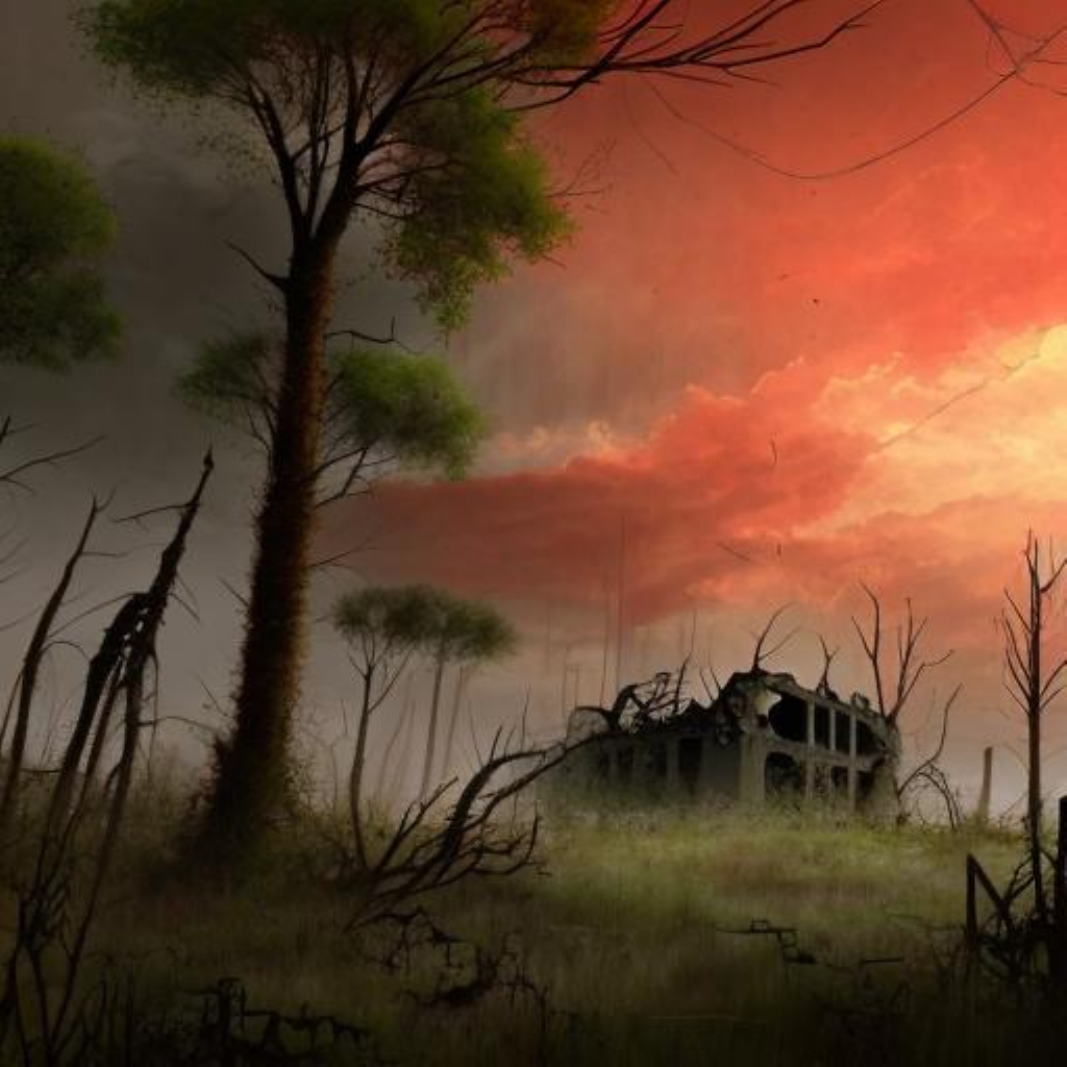}} & 
        \noindent\parbox[c]{0.14\columnwidth}{\includegraphics[width=0.14\columnwidth]{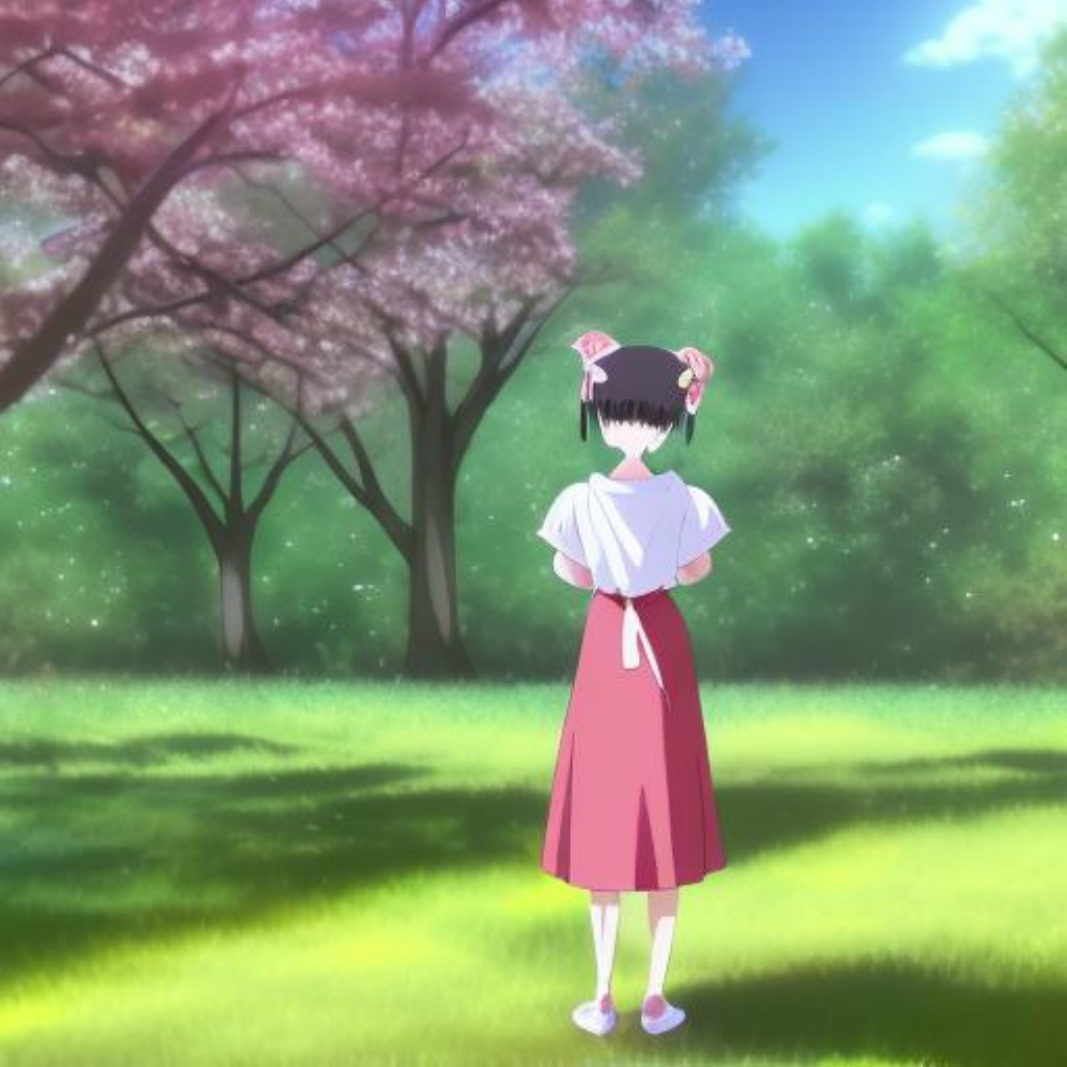}} & 
        \noindent\parbox[c]{0.14\columnwidth}{\includegraphics[width=0.14\columnwidth]{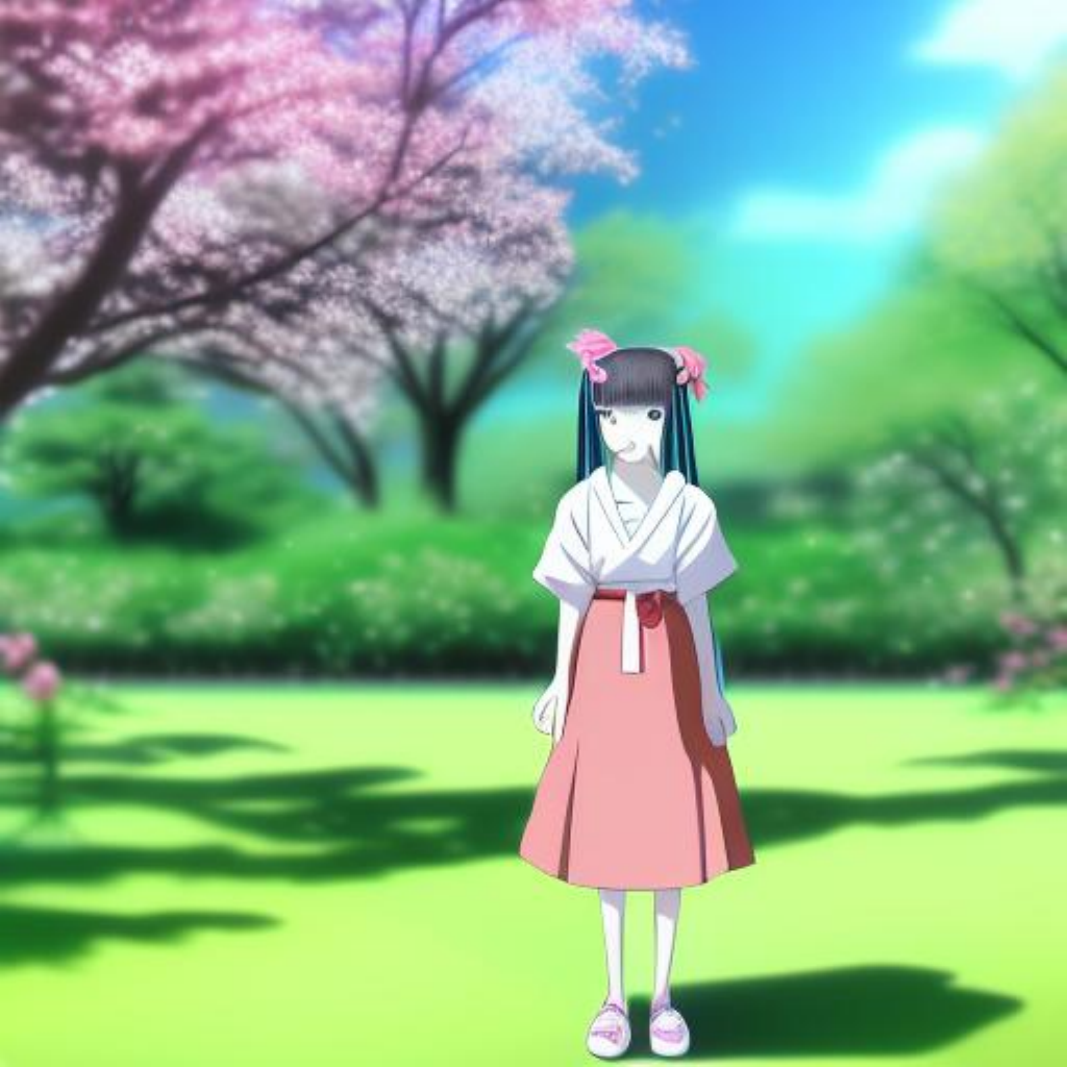}} & 
        \noindent\parbox[c]{0.14\columnwidth}{\includegraphics[width=0.14\columnwidth]{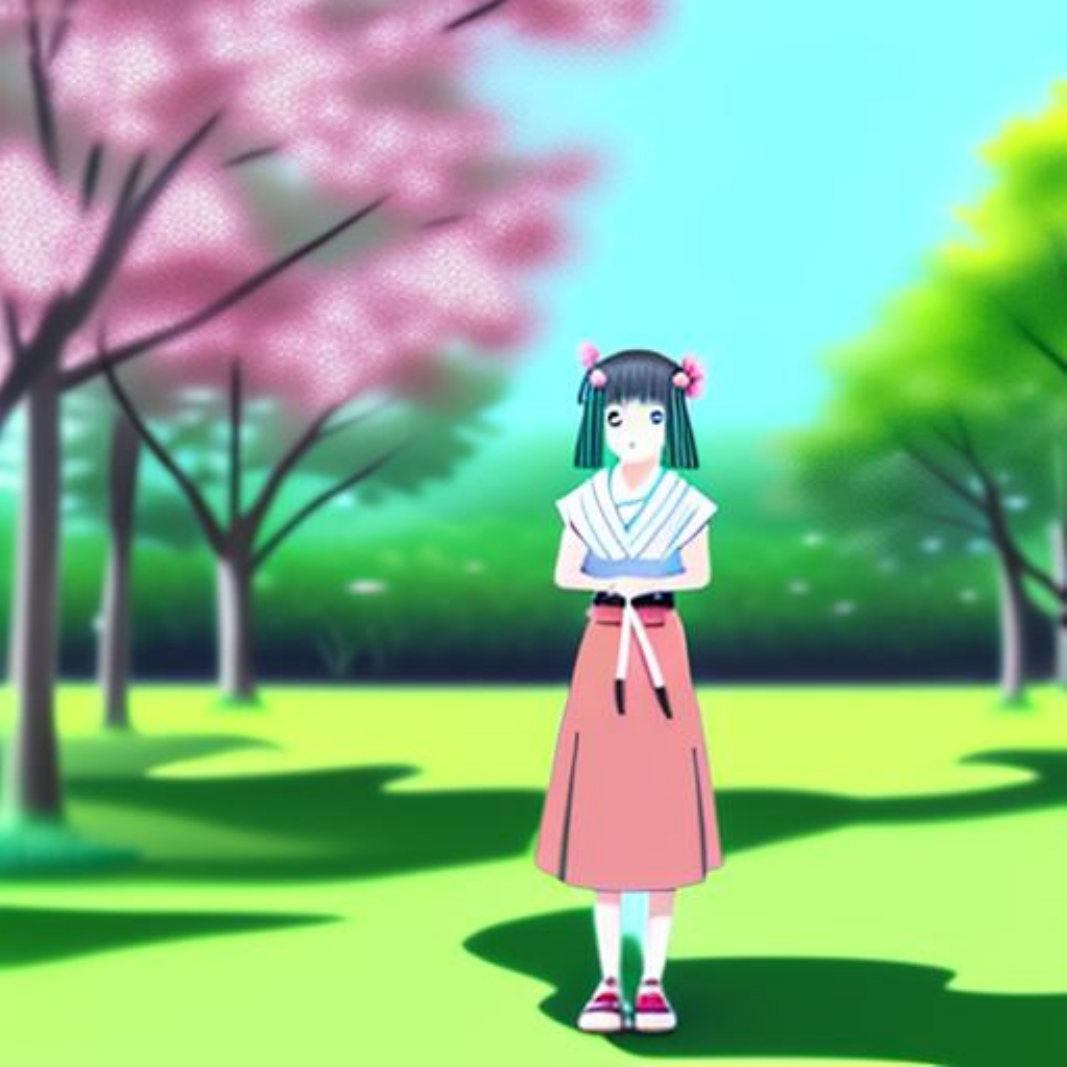}} \\
    \end{tabu}
    \caption{Comparison of samples generated from Openjourney using PLMS4 with HB $\beta$ under various sampling steps and guidance scale $s$. Specifically, we employ $\beta = 0.8$ for $s = 7.5$, $\beta = 0.6$ for $s = 15$, and $\beta = 0.6$ for $s = 22.5$ to account for the varying degrees of artifact manifestation associated with each guidance scale.}
    \label{fig:scale_step_openjourney_hb}
\end{figure}


\pagebreak

\tabulinesep=1pt
\begin{figure}
    \centering
    \begin{tabu} to \textwidth {@{}l@{\hspace{5pt}}c@{\hspace{2pt}}c@{\hspace{2pt}}c@{\hspace{4pt}}c@{\hspace{2pt}}c@{\hspace{2pt}}c@{}}
        & \multicolumn{3}{c}{\shortstack{\scriptsize "A post-apocalyptic world with ruined \\ \scriptsize buildings, overgrown vegetation, and a red sky"}}
        & \multicolumn{3}{c}{\shortstack{\scriptsize "A girl standing in a park in \\ \scriptsize Japanese animation style"}} \\

        & \multicolumn{1}{c}{\shortstack{\scriptsize $s = 7.5$}}
        & \multicolumn{1}{c}{\shortstack{\scriptsize $s = 15$}}
        & \multicolumn{1}{c}{\shortstack{\scriptsize $s = 22.5$}}
        & \multicolumn{1}{c}{\shortstack{\scriptsize $s = 7.5$}}
        & \multicolumn{1}{c}{\shortstack{\scriptsize $s = 15$}}
        & \multicolumn{1}{c}{\shortstack{\scriptsize $s = 22.5$}}
        \\
        
        \shortstack[l]{\tiny 10 steps} &
        \noindent\parbox[c]{0.14\columnwidth}{\includegraphics[width=0.14\columnwidth]{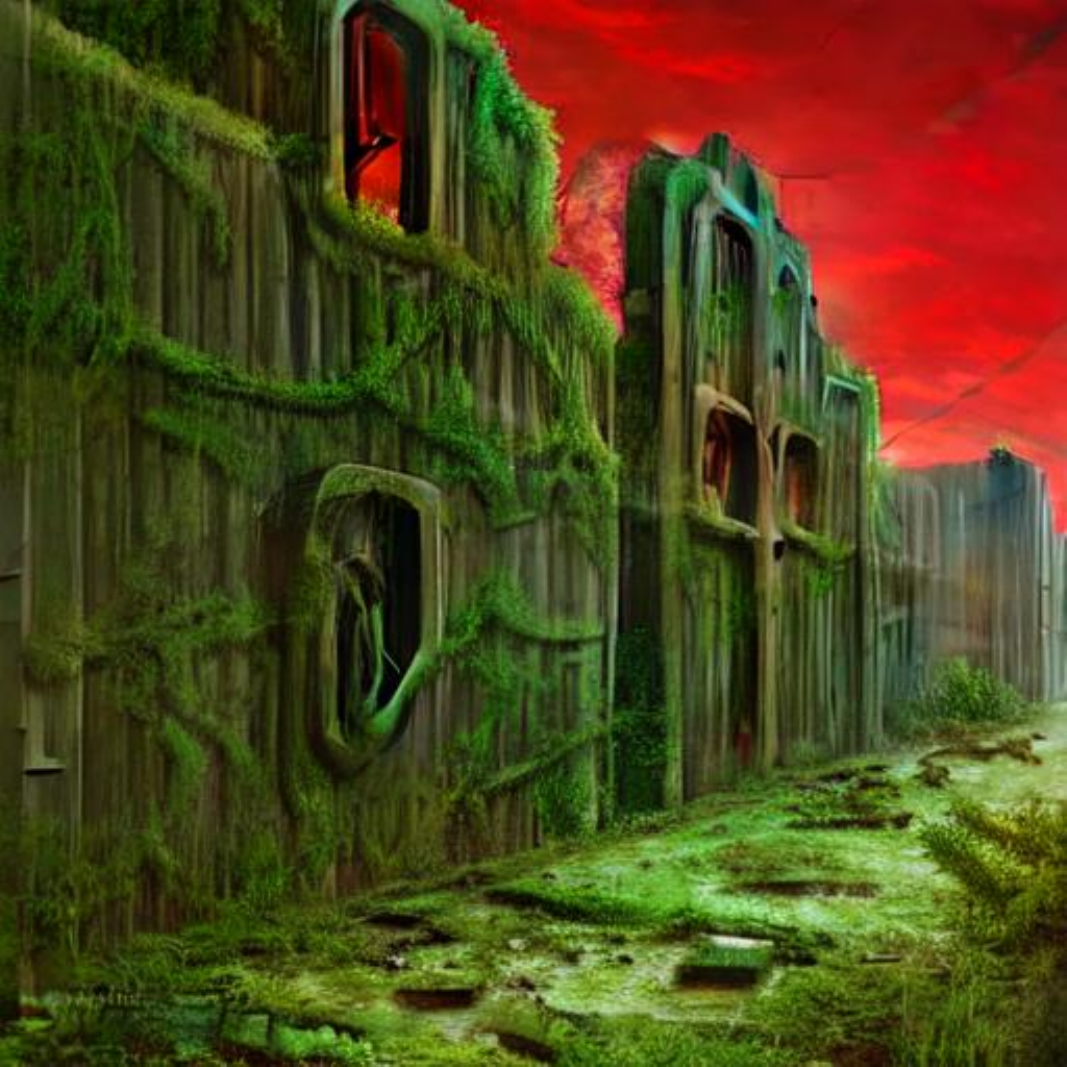}} & 
        \noindent\parbox[c]{0.14\columnwidth}{\includegraphics[width=0.14\columnwidth]{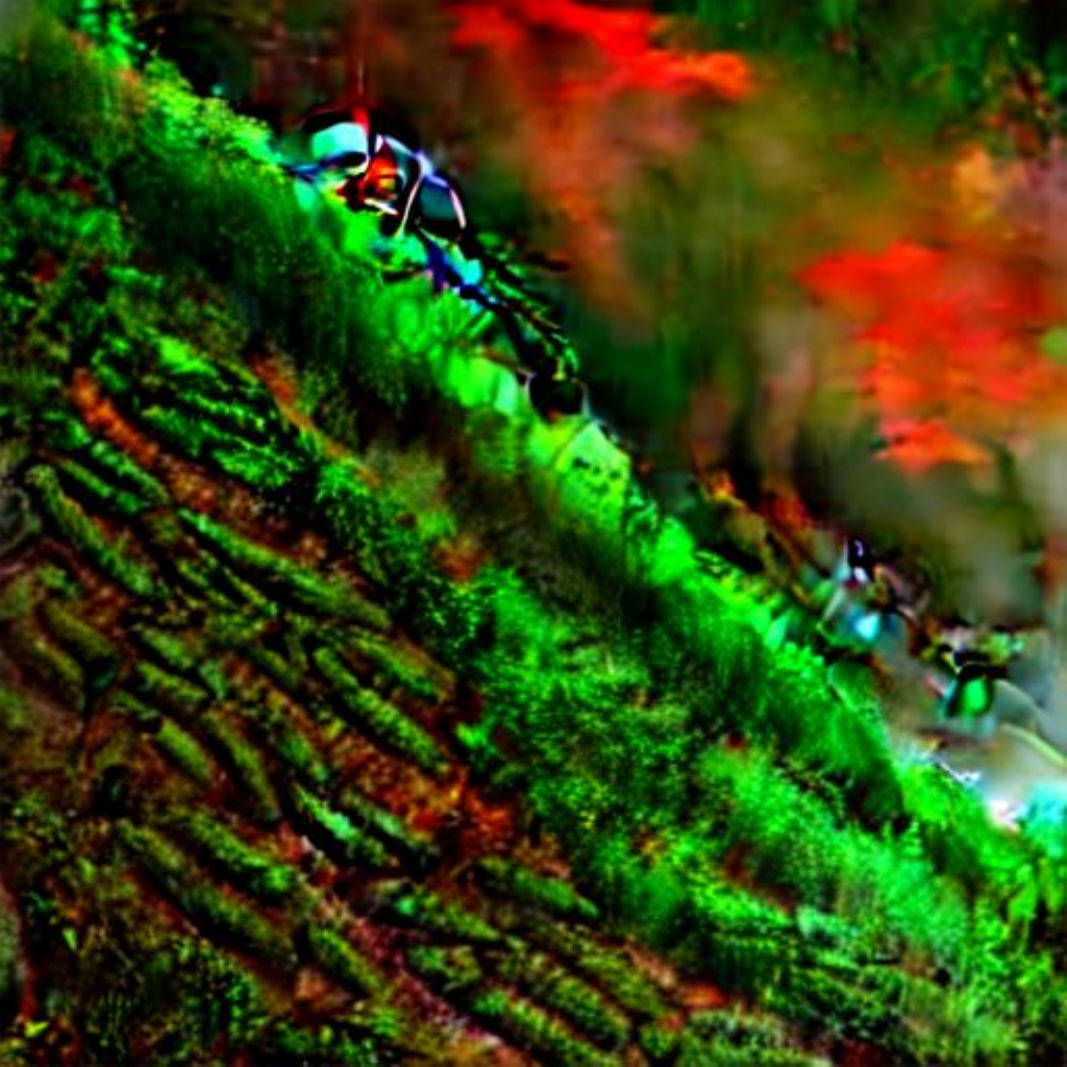}} & 
        \noindent\parbox[c]{0.14\columnwidth}{\includegraphics[width=0.14\columnwidth]{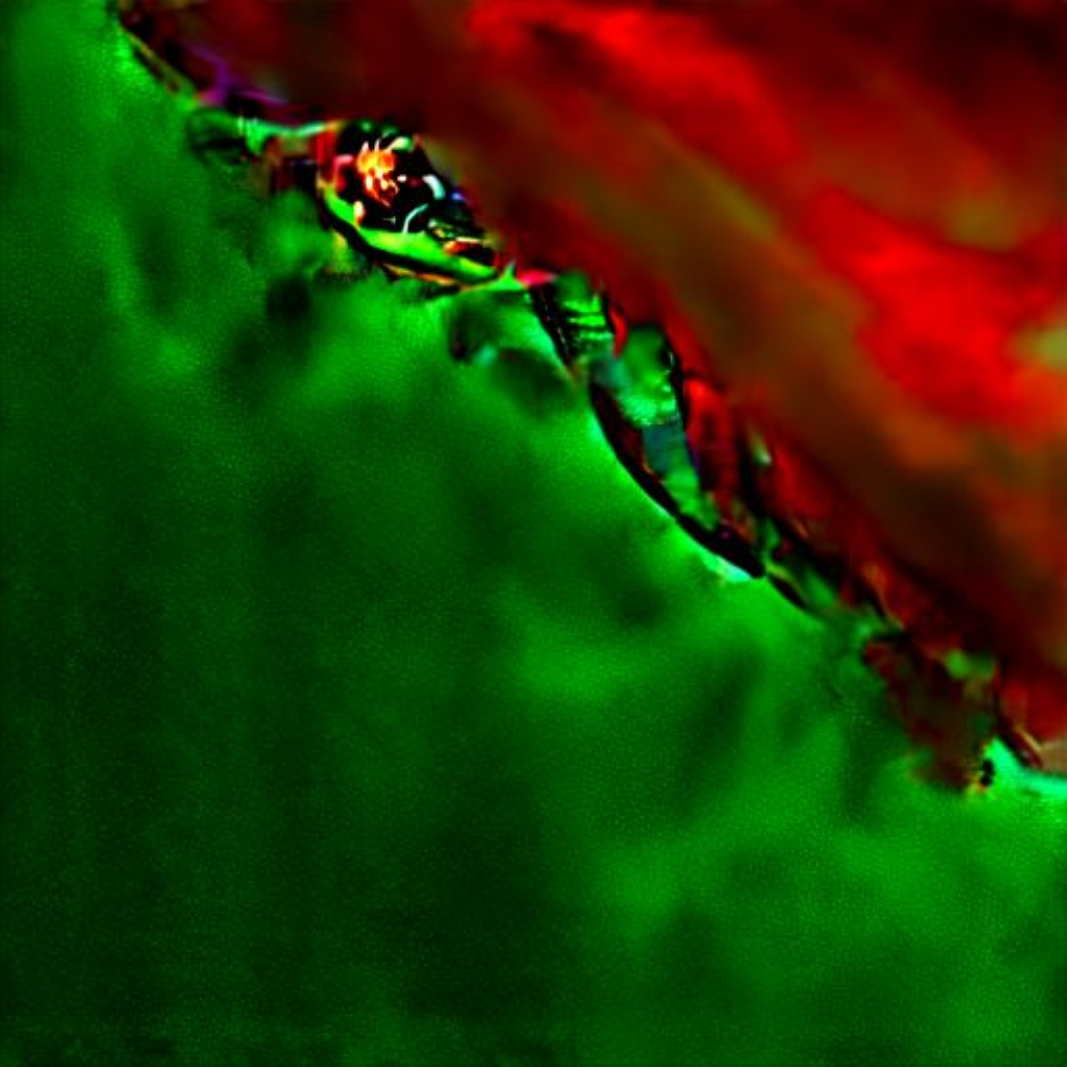}} & 
        \noindent\parbox[c]{0.14\columnwidth}{\includegraphics[width=0.14\columnwidth]{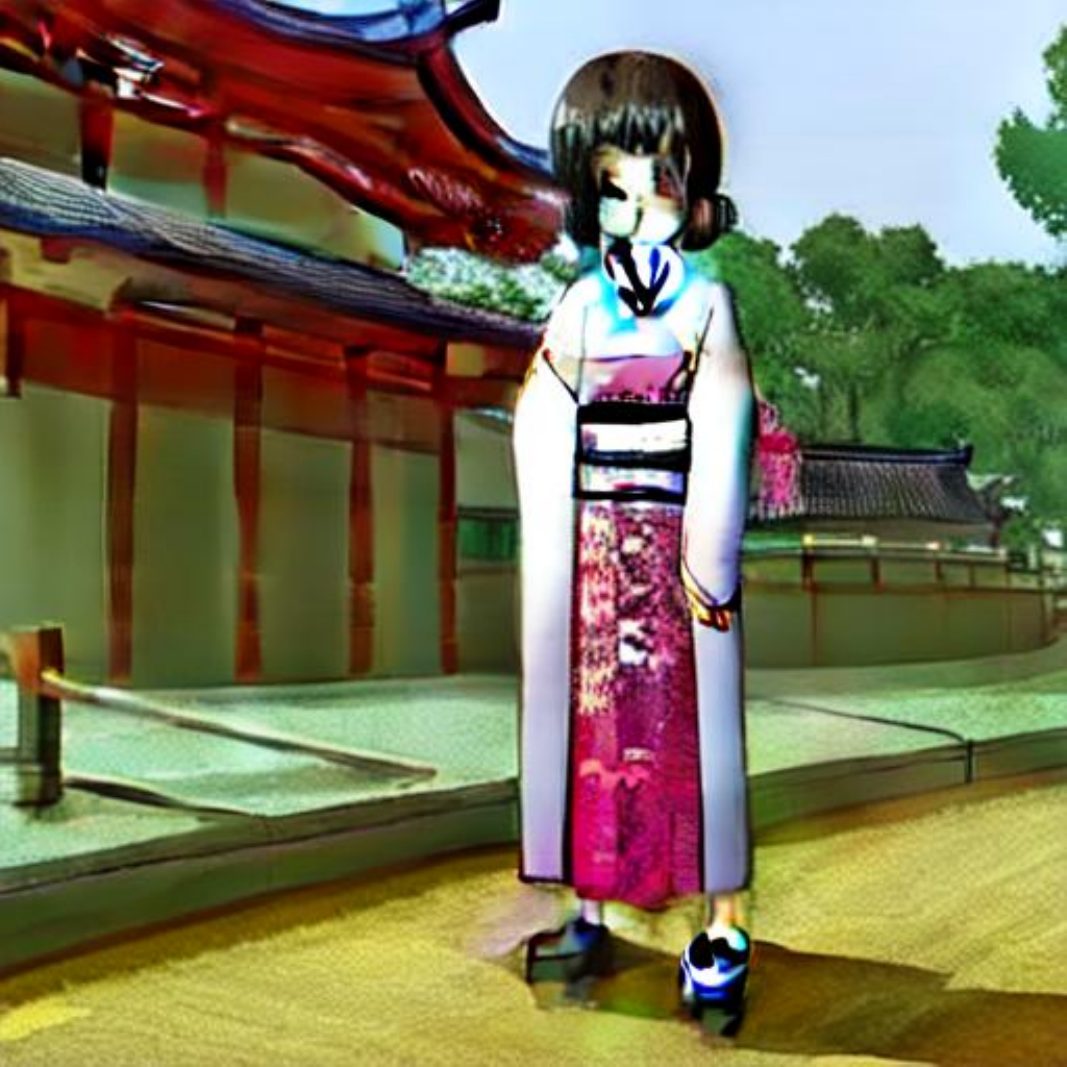}} & 
        \noindent\parbox[c]{0.14\columnwidth}{\includegraphics[width=0.14\columnwidth]{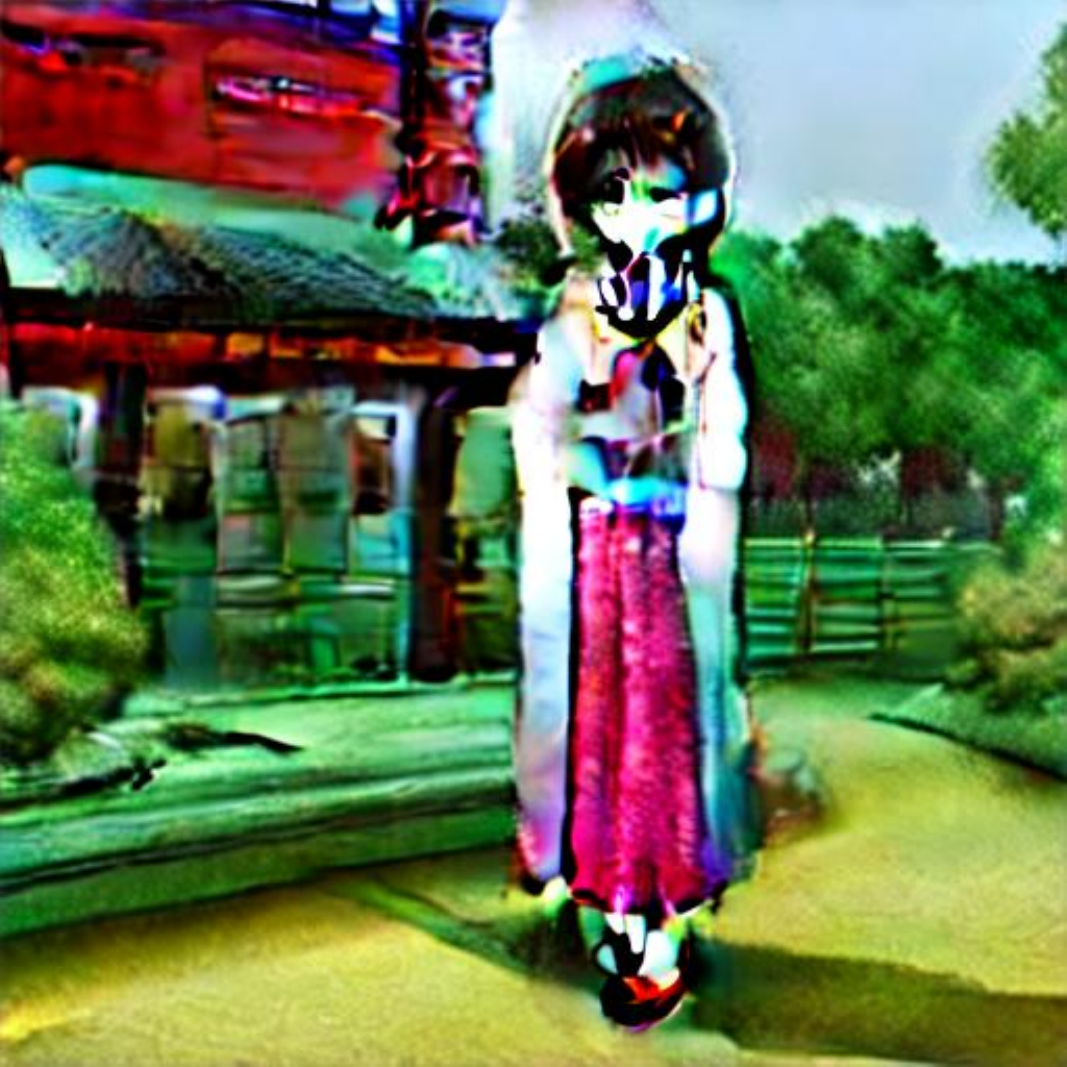}} & 
        \noindent\parbox[c]{0.14\columnwidth}{\includegraphics[width=0.14\columnwidth]{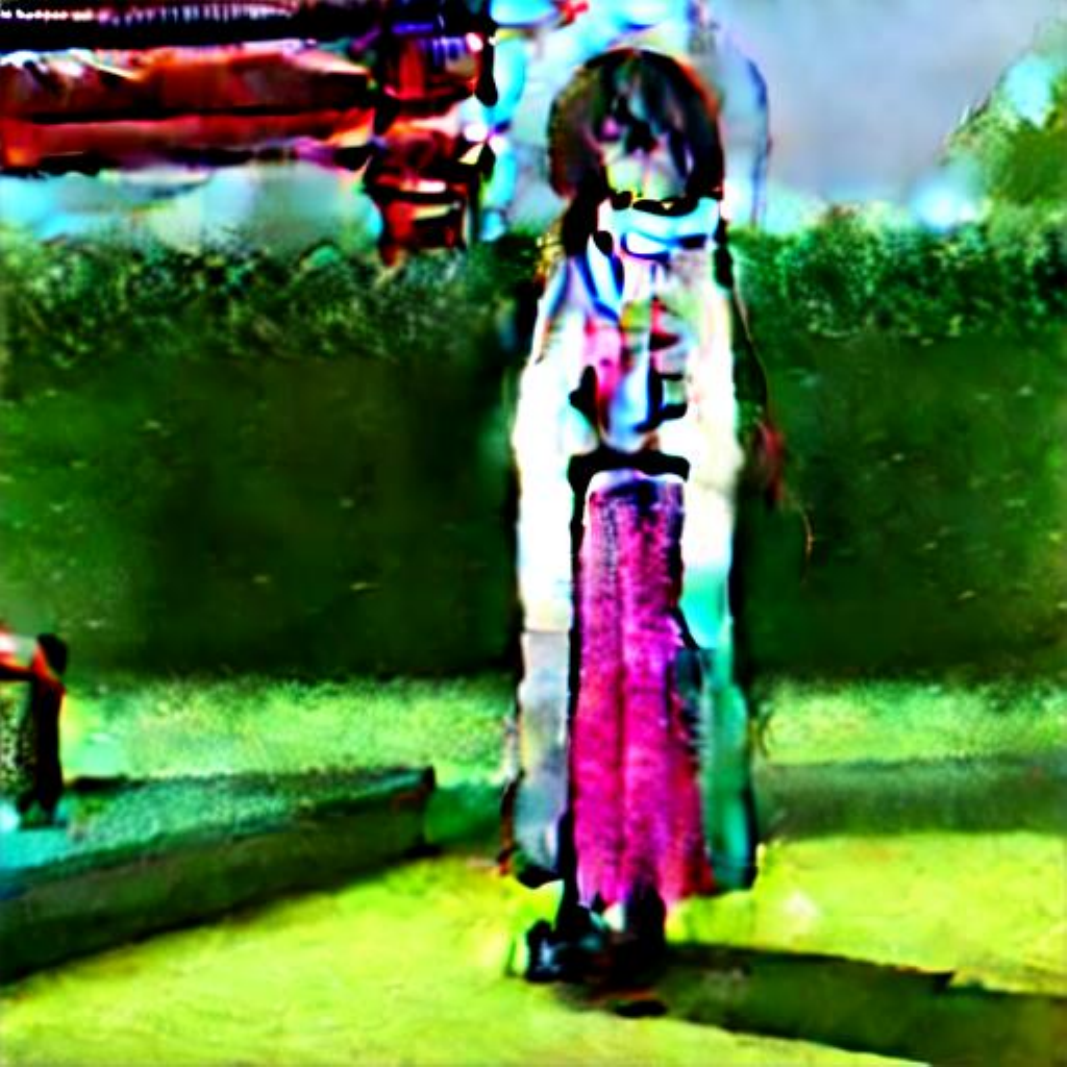}} \\

        \shortstack[l]{\tiny 15 steps} &
        \noindent\parbox[c]{0.14\columnwidth}{\includegraphics[width=0.14\columnwidth]{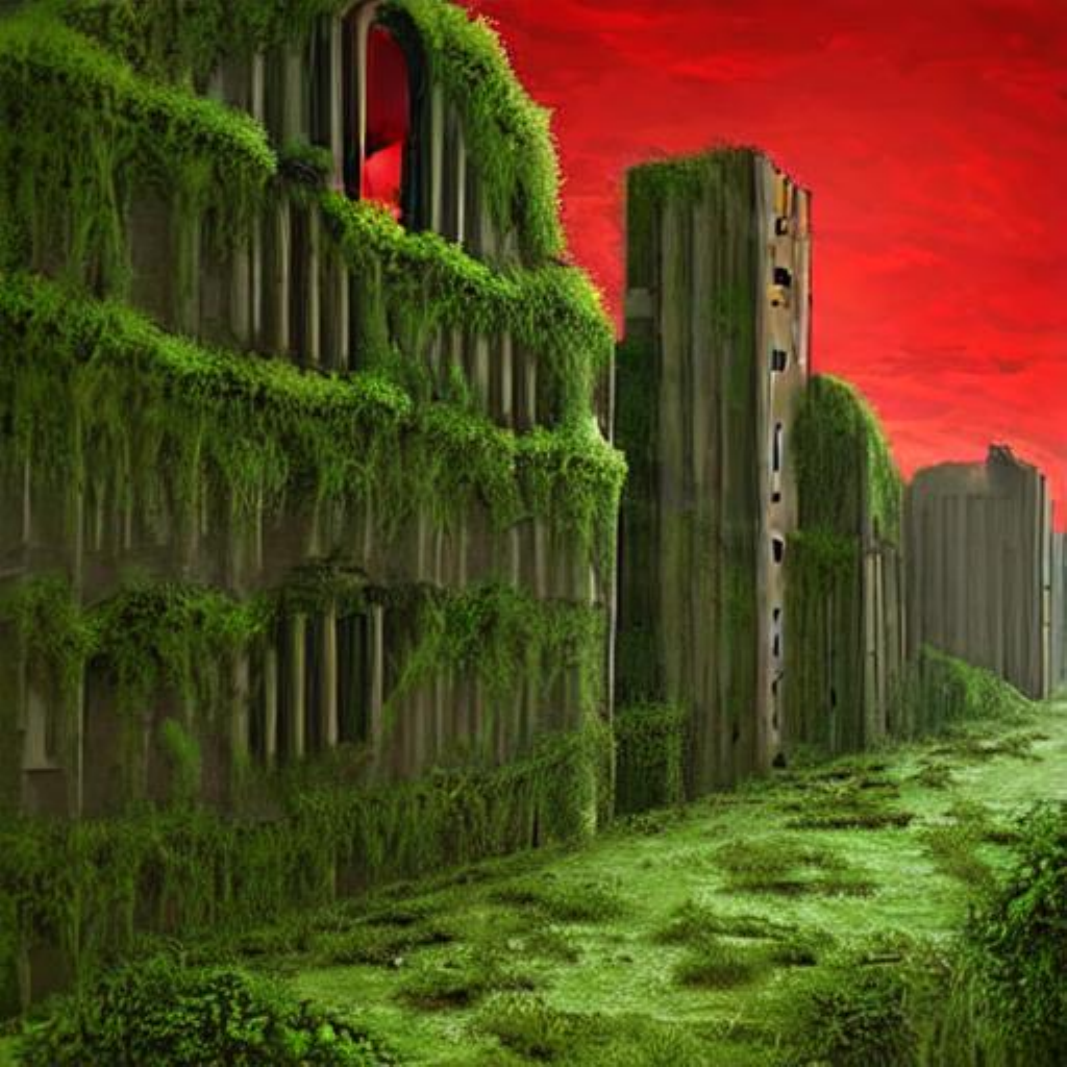}} & 
        \noindent\parbox[c]{0.14\columnwidth}{\includegraphics[width=0.14\columnwidth]{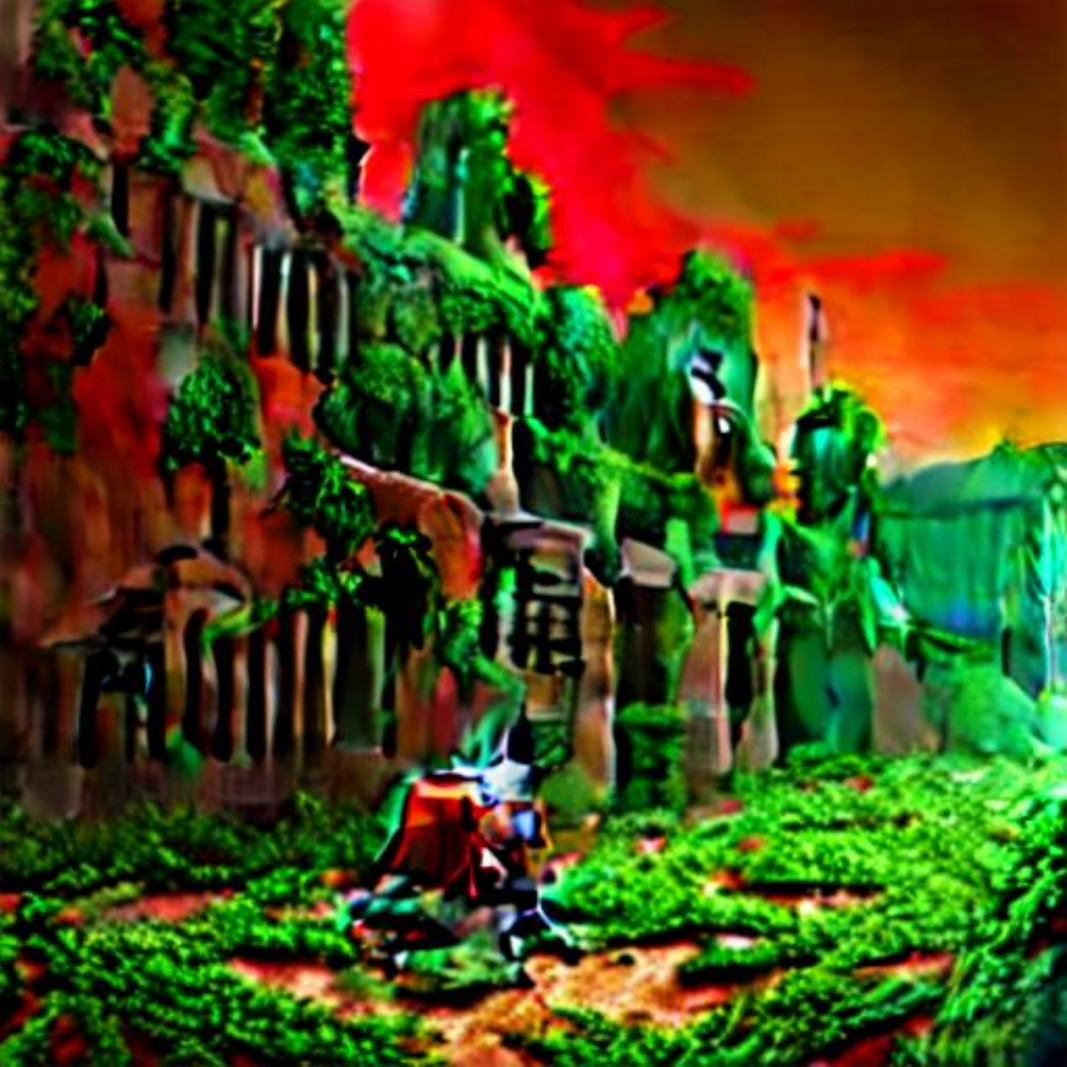}} & 
        \noindent\parbox[c]{0.14\columnwidth}{\includegraphics[width=0.14\columnwidth]{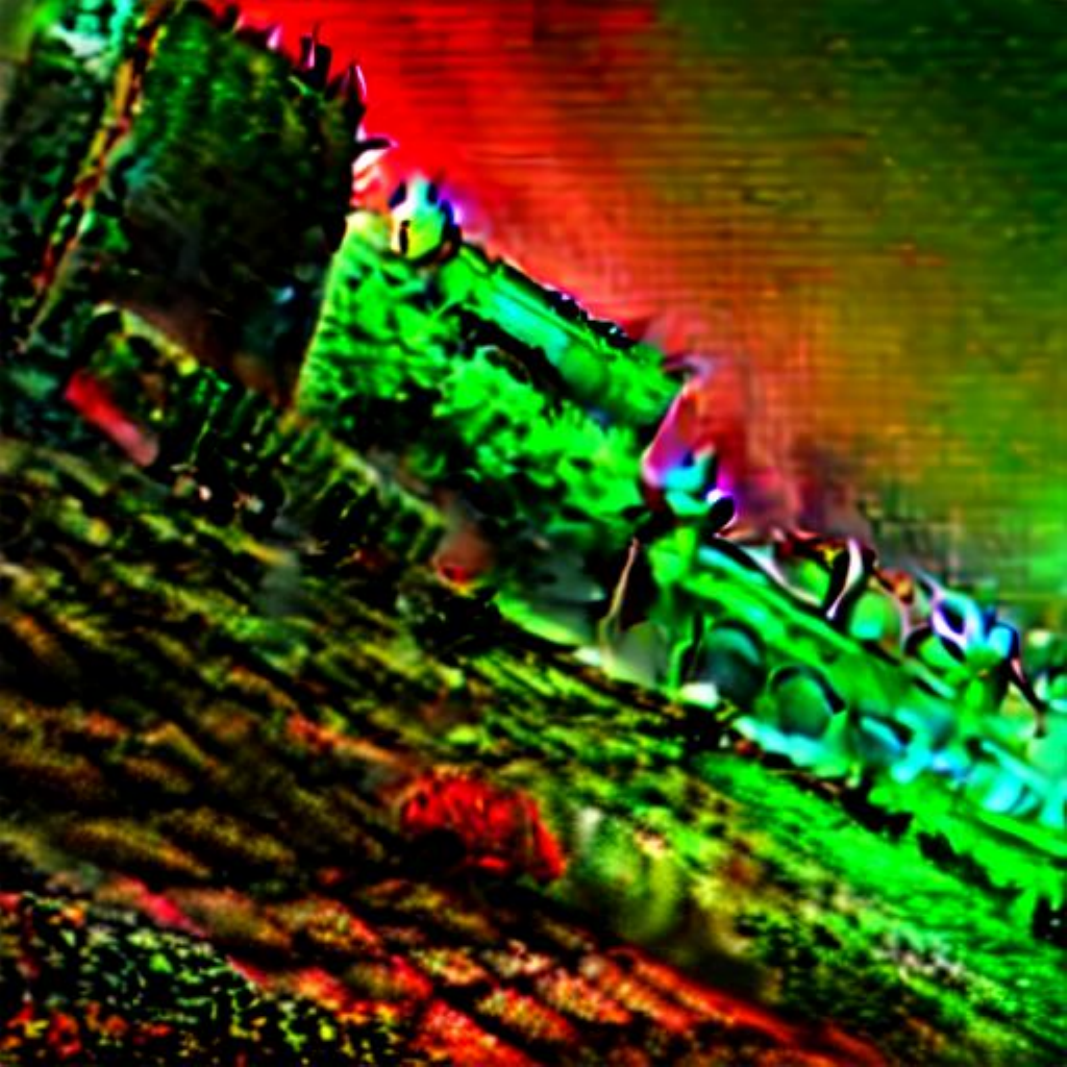}} & 
        \noindent\parbox[c]{0.14\columnwidth}{\includegraphics[width=0.14\columnwidth]{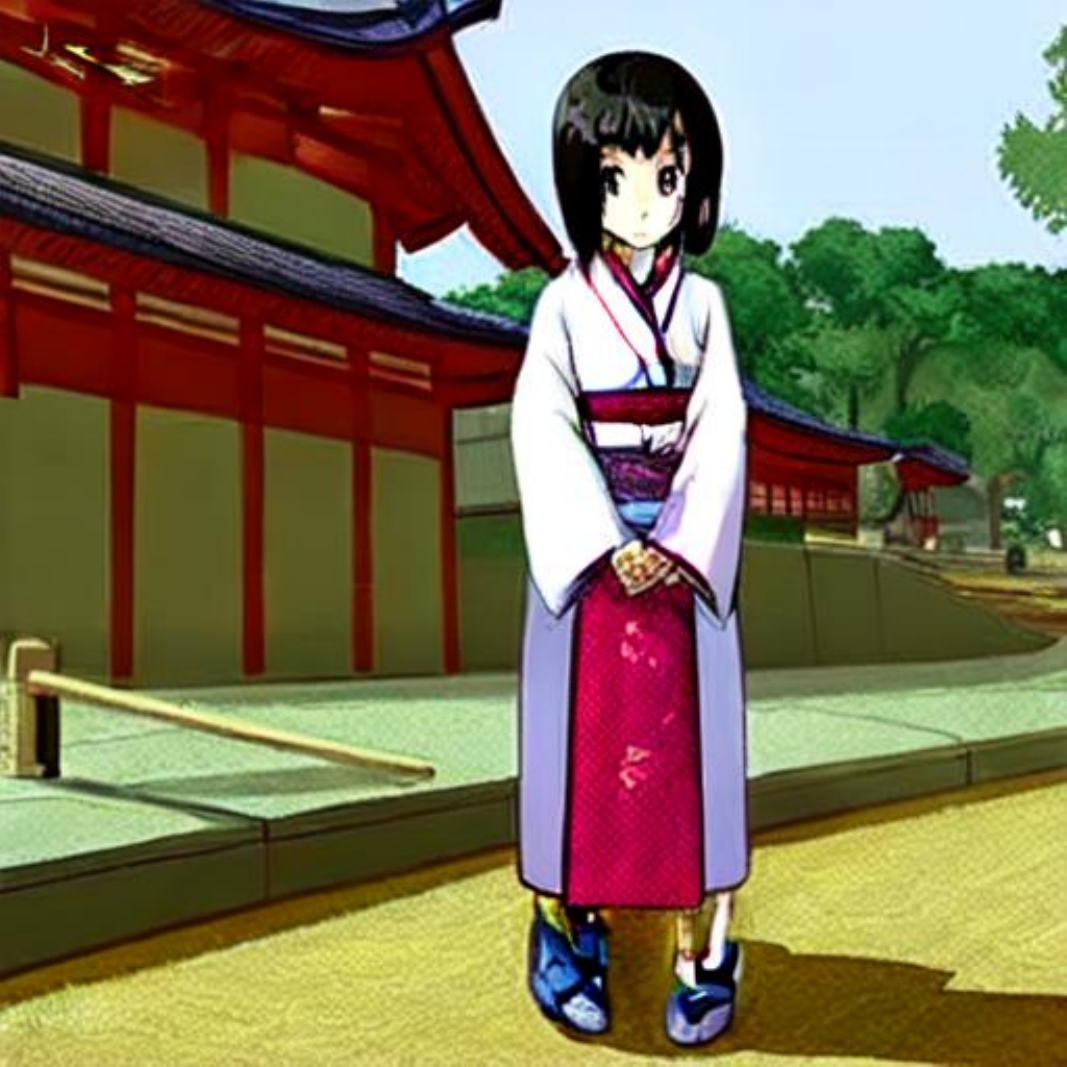}} & 
        \noindent\parbox[c]{0.14\columnwidth}{\includegraphics[width=0.14\columnwidth]{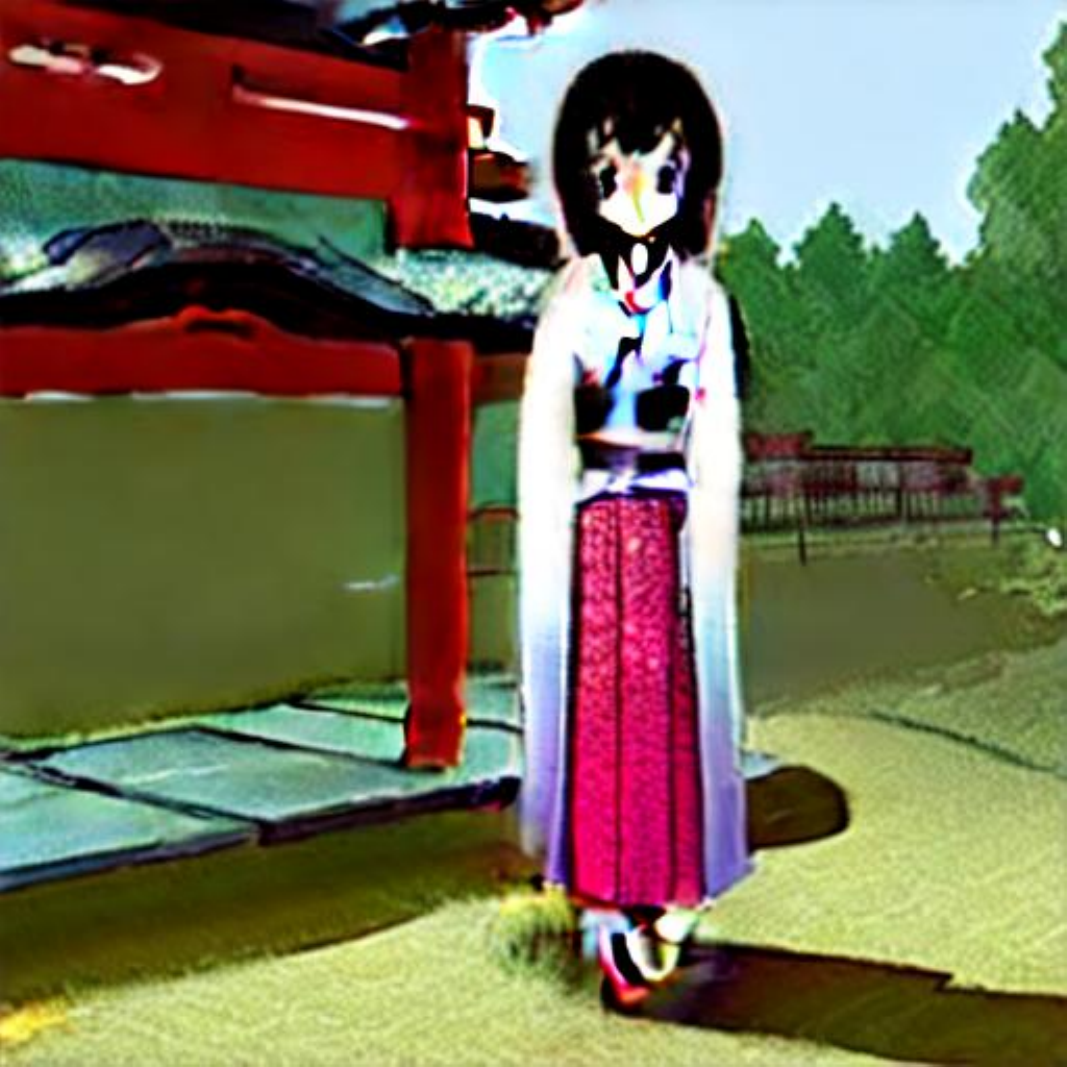}} & 
        \noindent\parbox[c]{0.14\columnwidth}{\includegraphics[width=0.14\columnwidth]{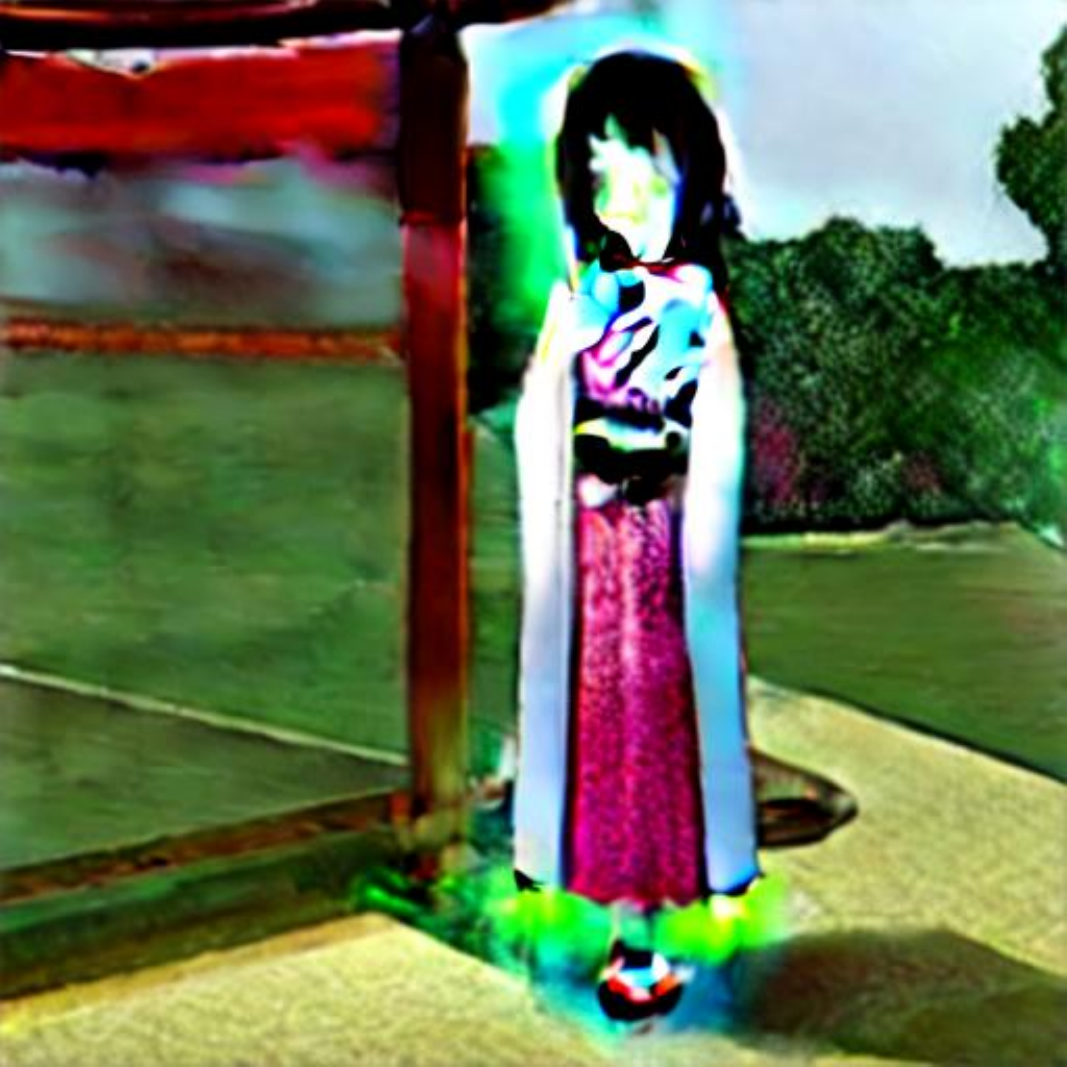}} \\

        \shortstack[l]{\tiny 20 steps} &
        \noindent\parbox[c]{0.14\columnwidth}{\includegraphics[width=0.14\columnwidth]{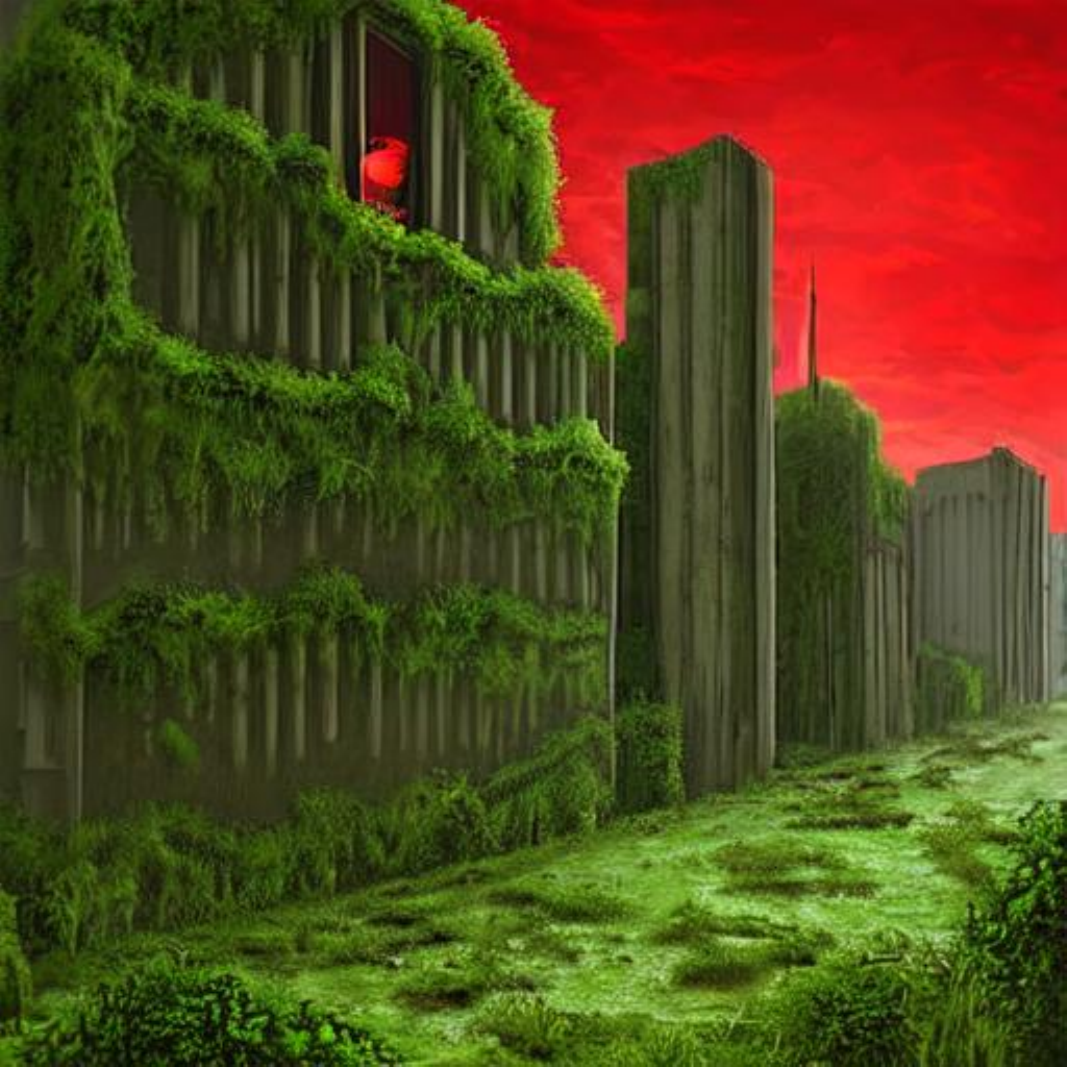}} & 
        \noindent\parbox[c]{0.14\columnwidth}{\includegraphics[width=0.14\columnwidth]{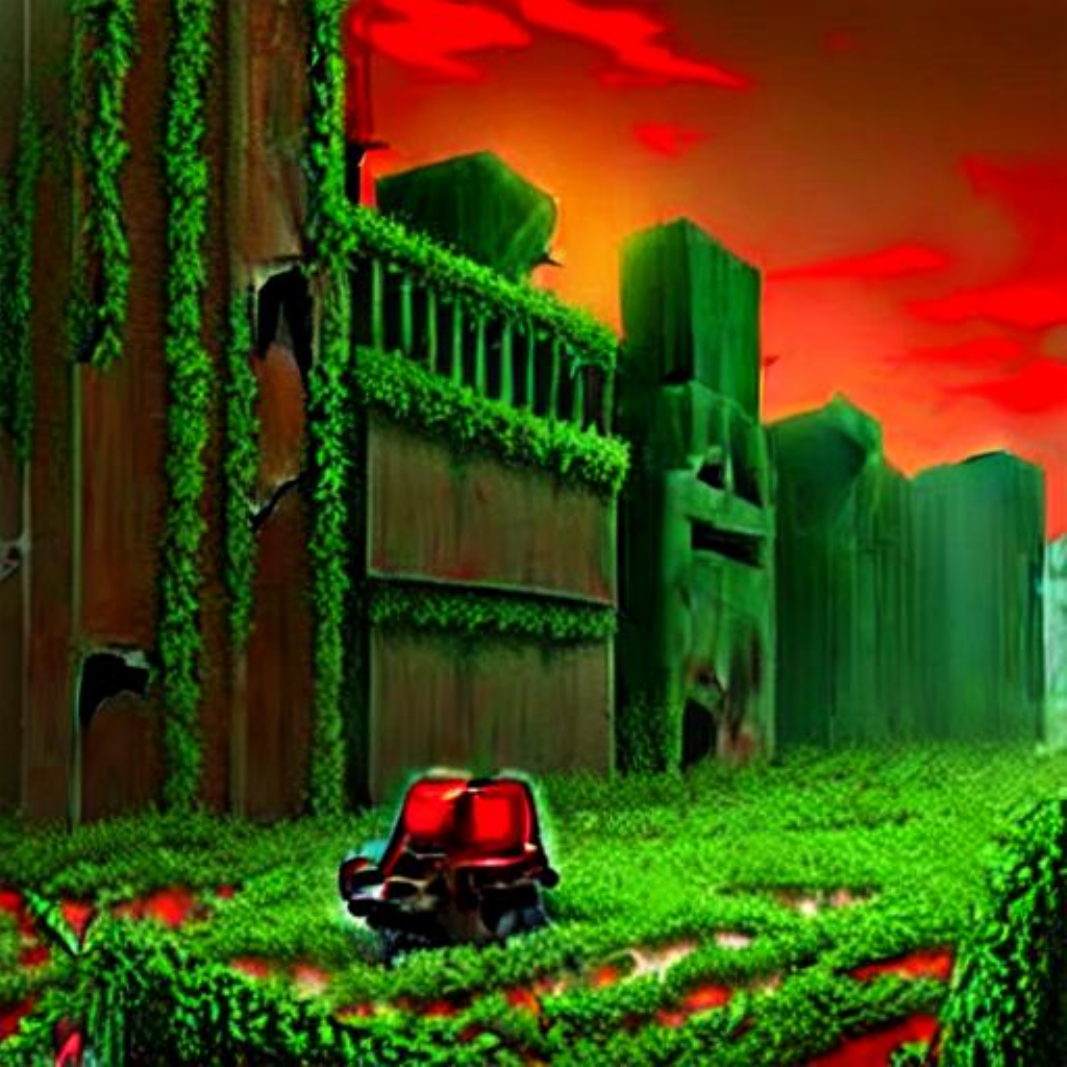}} & 
        \noindent\parbox[c]{0.14\columnwidth}{\includegraphics[width=0.14\columnwidth]{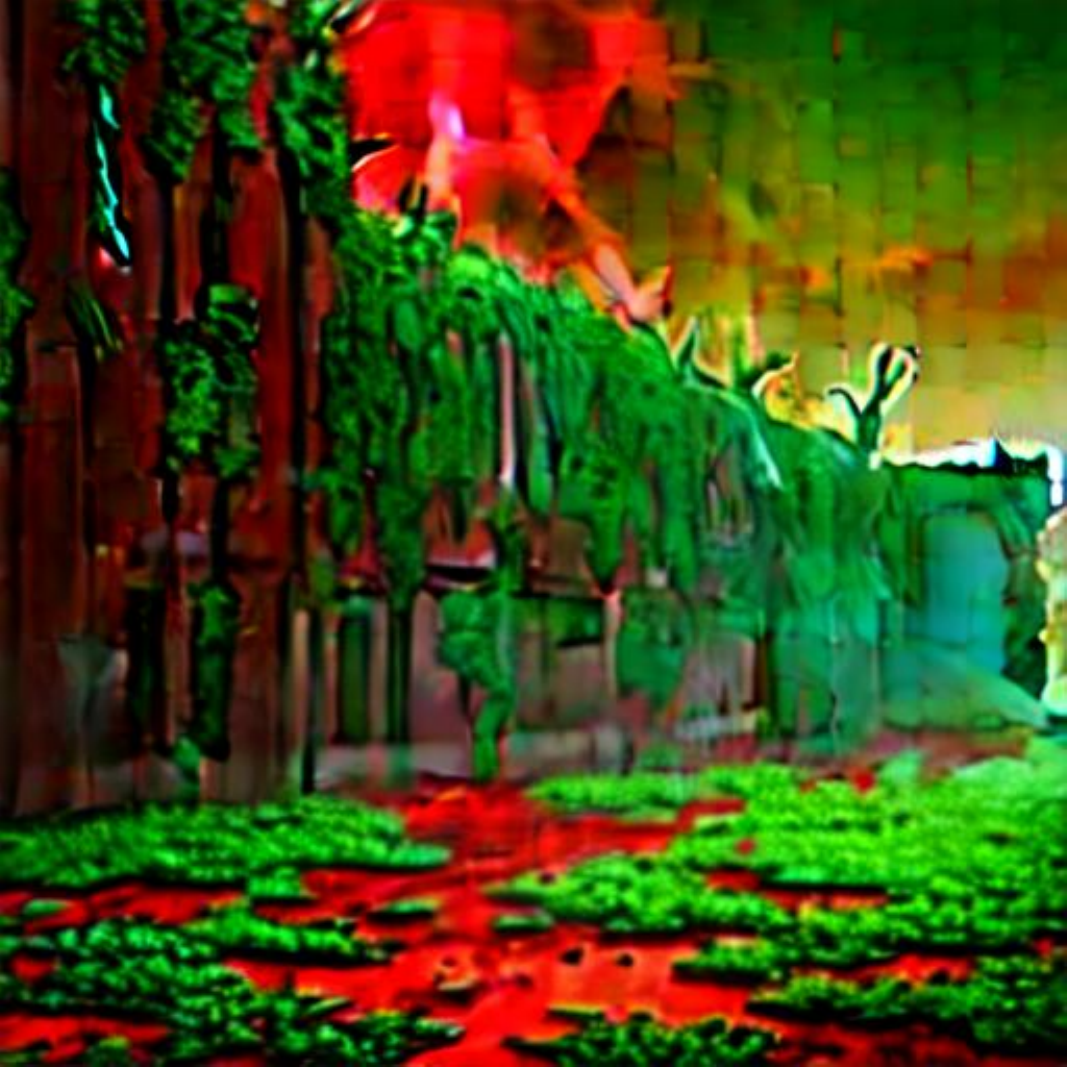}} & 
        \noindent\parbox[c]{0.14\columnwidth}{\includegraphics[width=0.14\columnwidth]{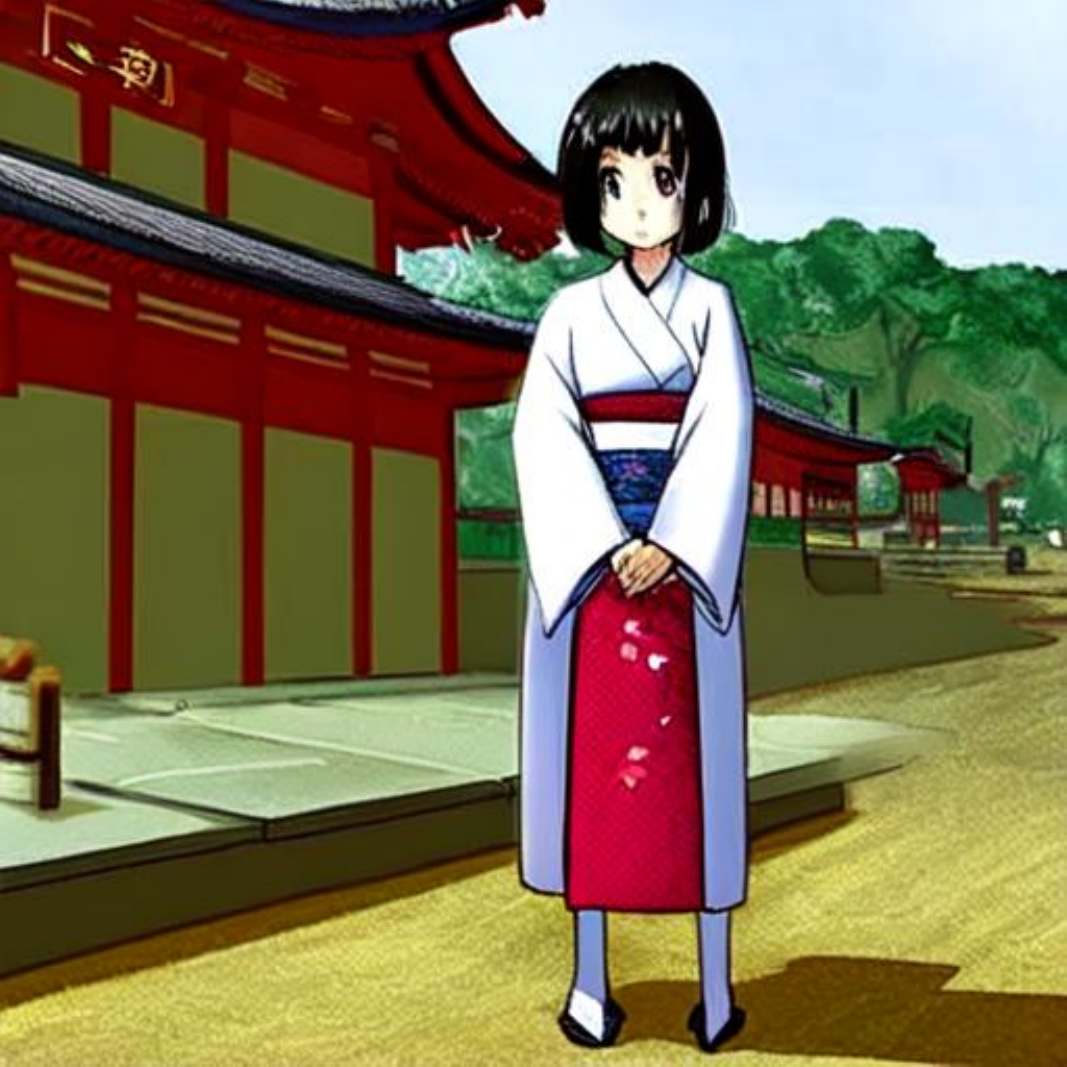}} & 
        \noindent\parbox[c]{0.14\columnwidth}{\includegraphics[width=0.14\columnwidth]{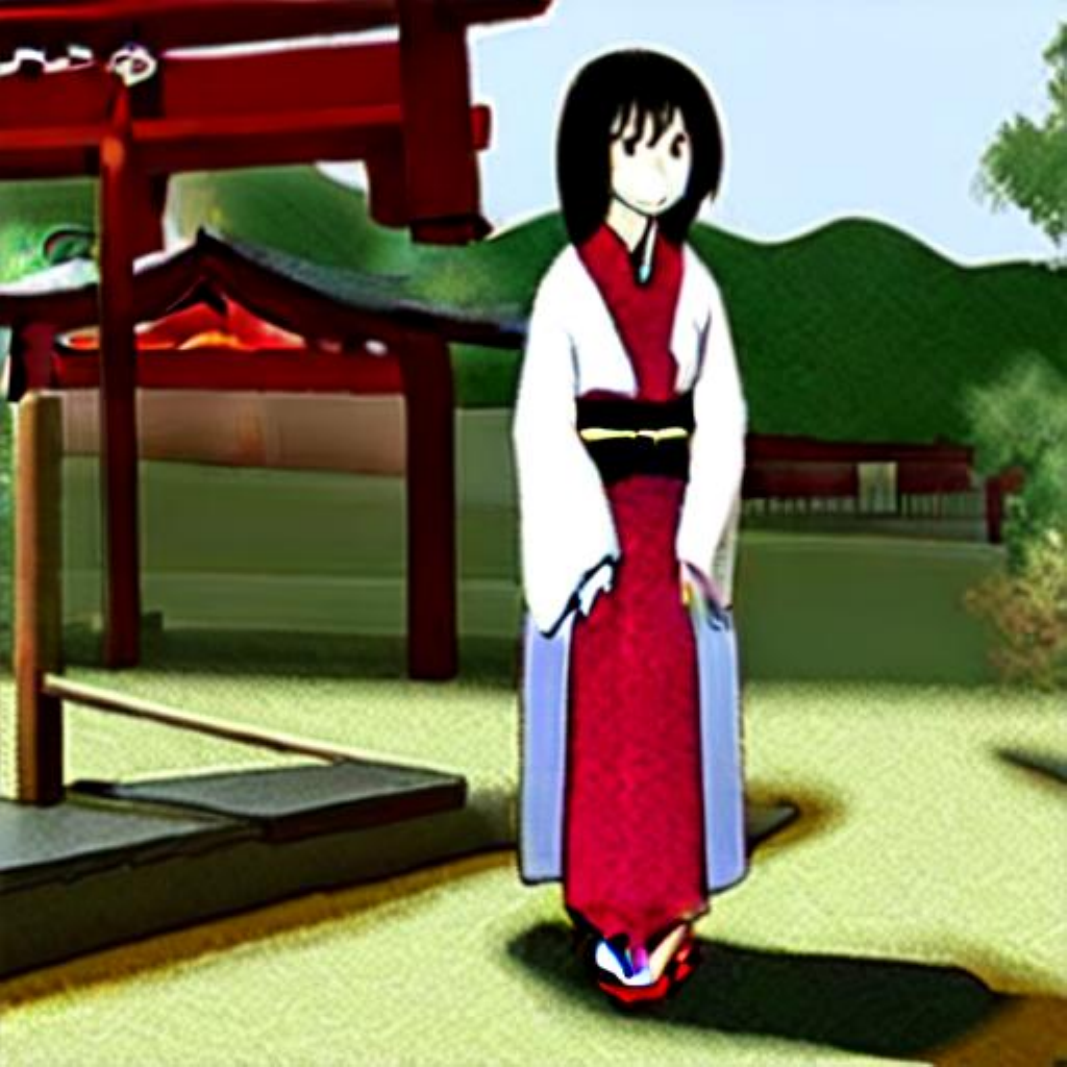}} & 
        \noindent\parbox[c]{0.14\columnwidth}{\includegraphics[width=0.14\columnwidth]{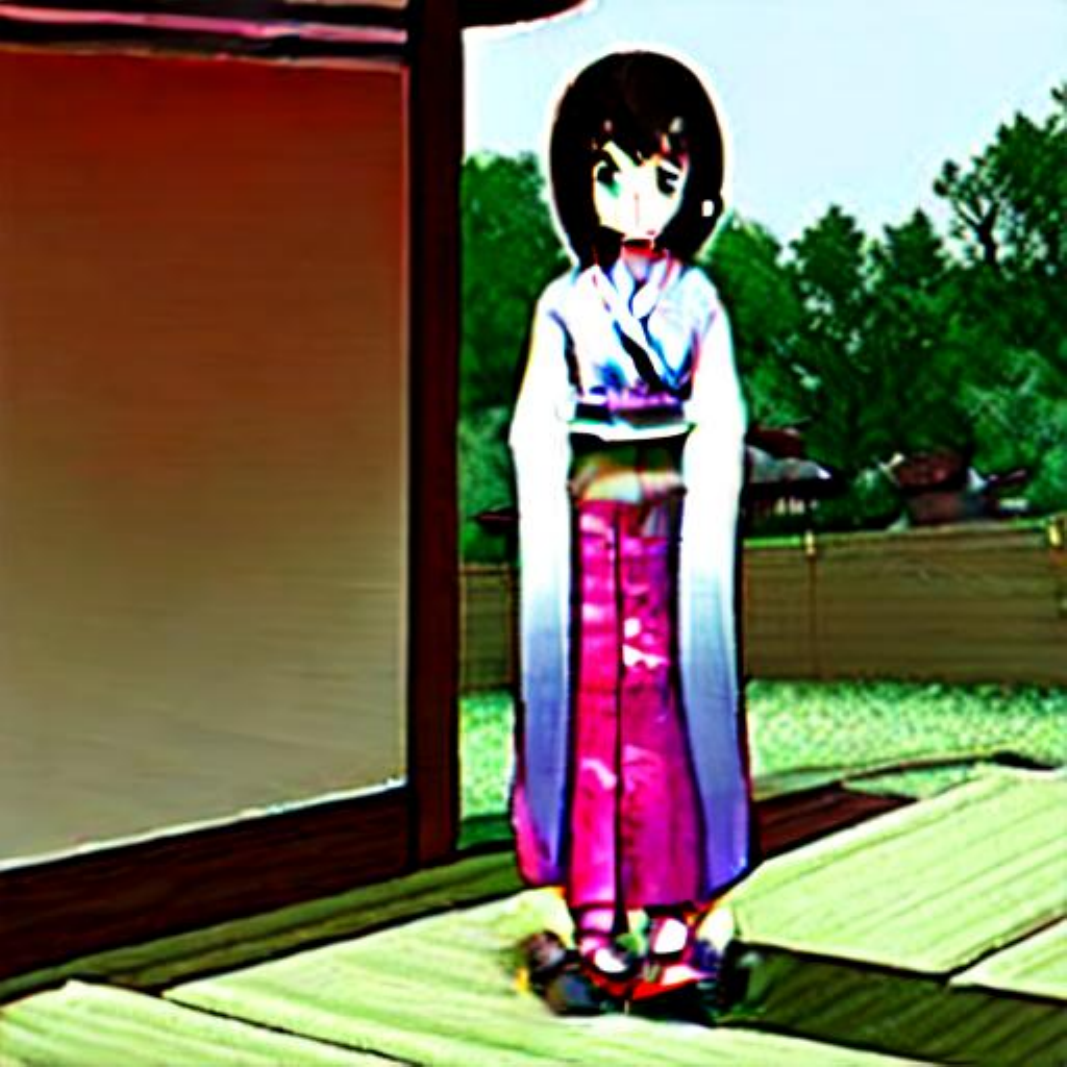}} \\

        \shortstack[l]{\tiny 40 steps} &
        \noindent\parbox[c]{0.14\columnwidth}{\includegraphics[width=0.14\columnwidth]{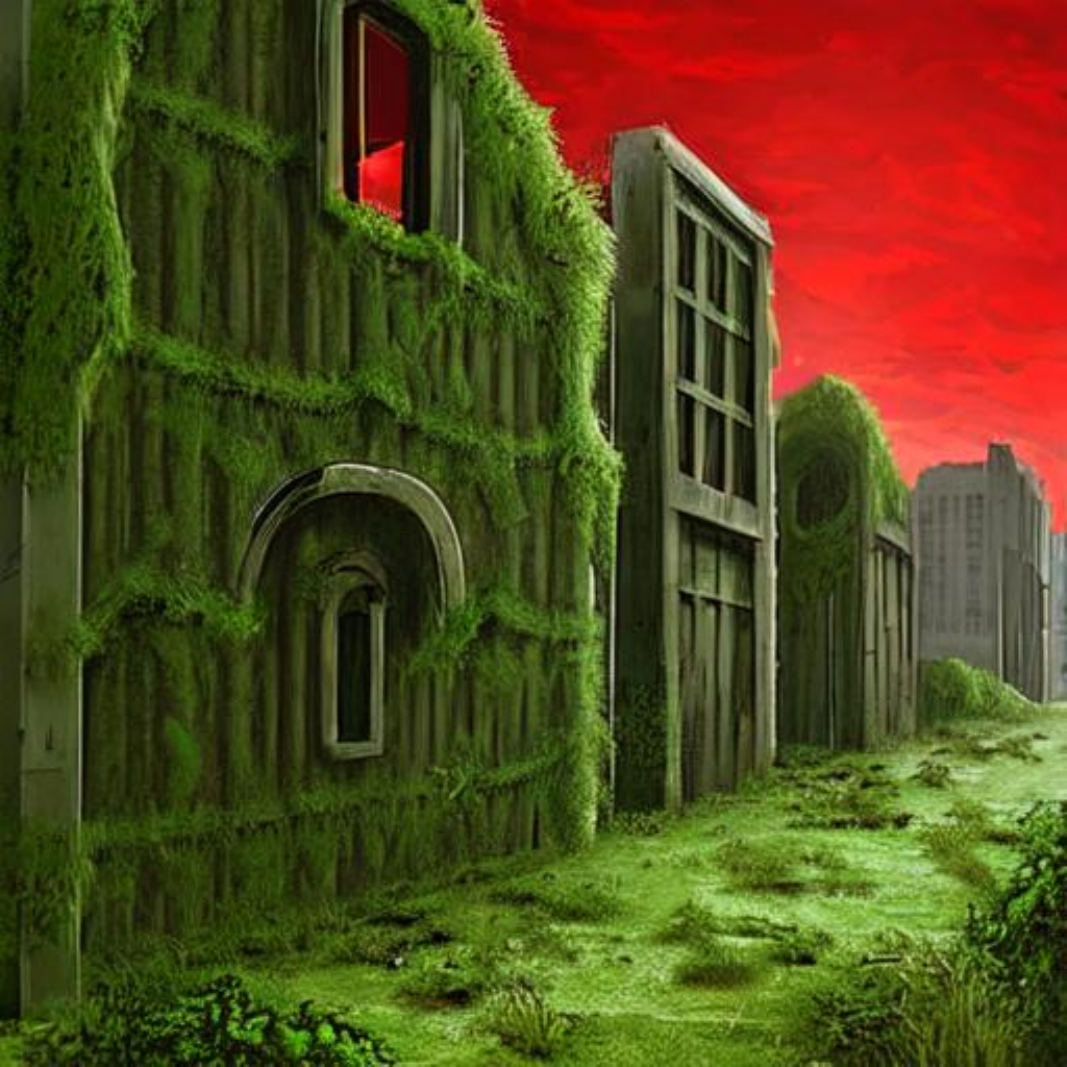}} & 
        \noindent\parbox[c]{0.14\columnwidth}{\includegraphics[width=0.14\columnwidth]{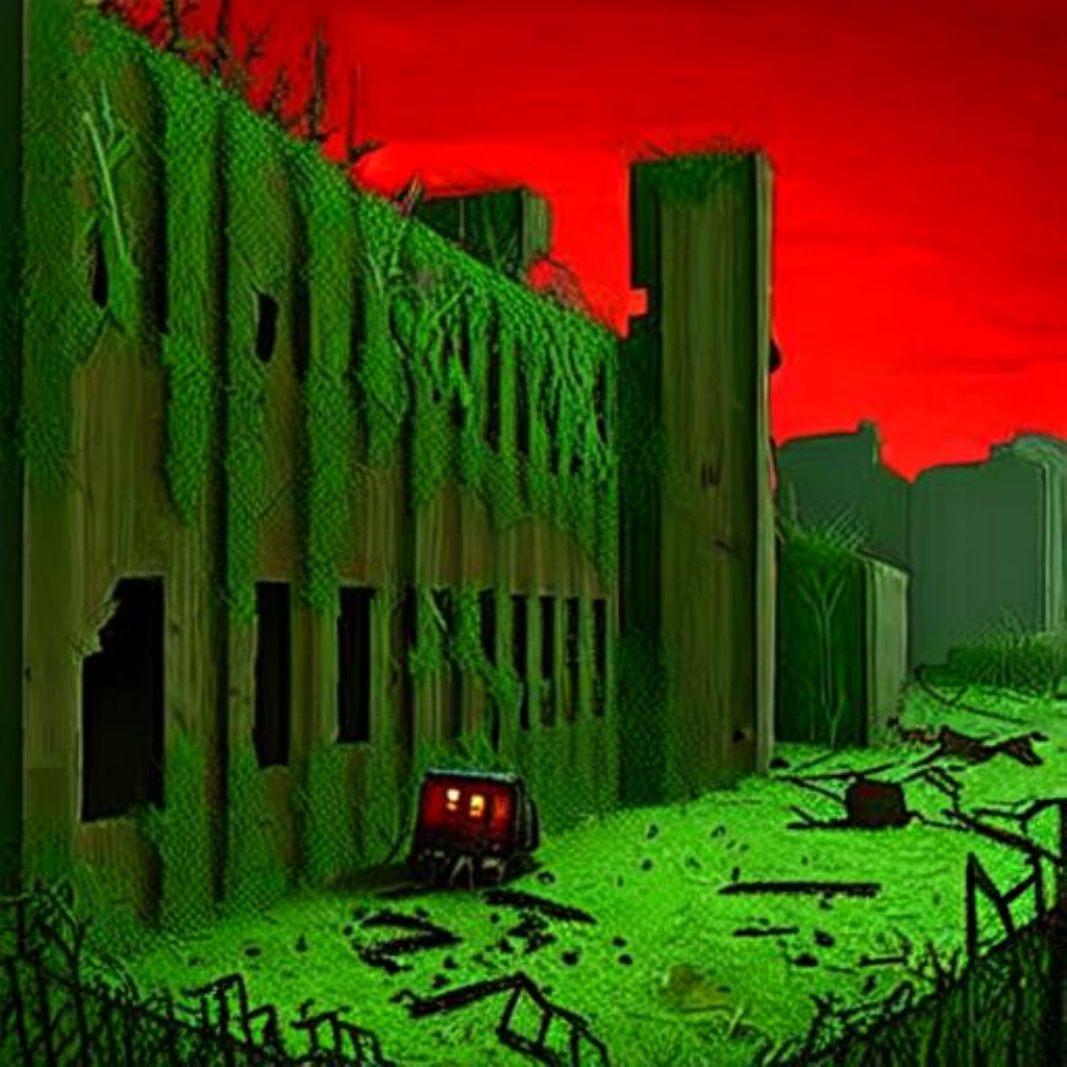}} & 
        \noindent\parbox[c]{0.14\columnwidth}{\includegraphics[width=0.14\columnwidth]{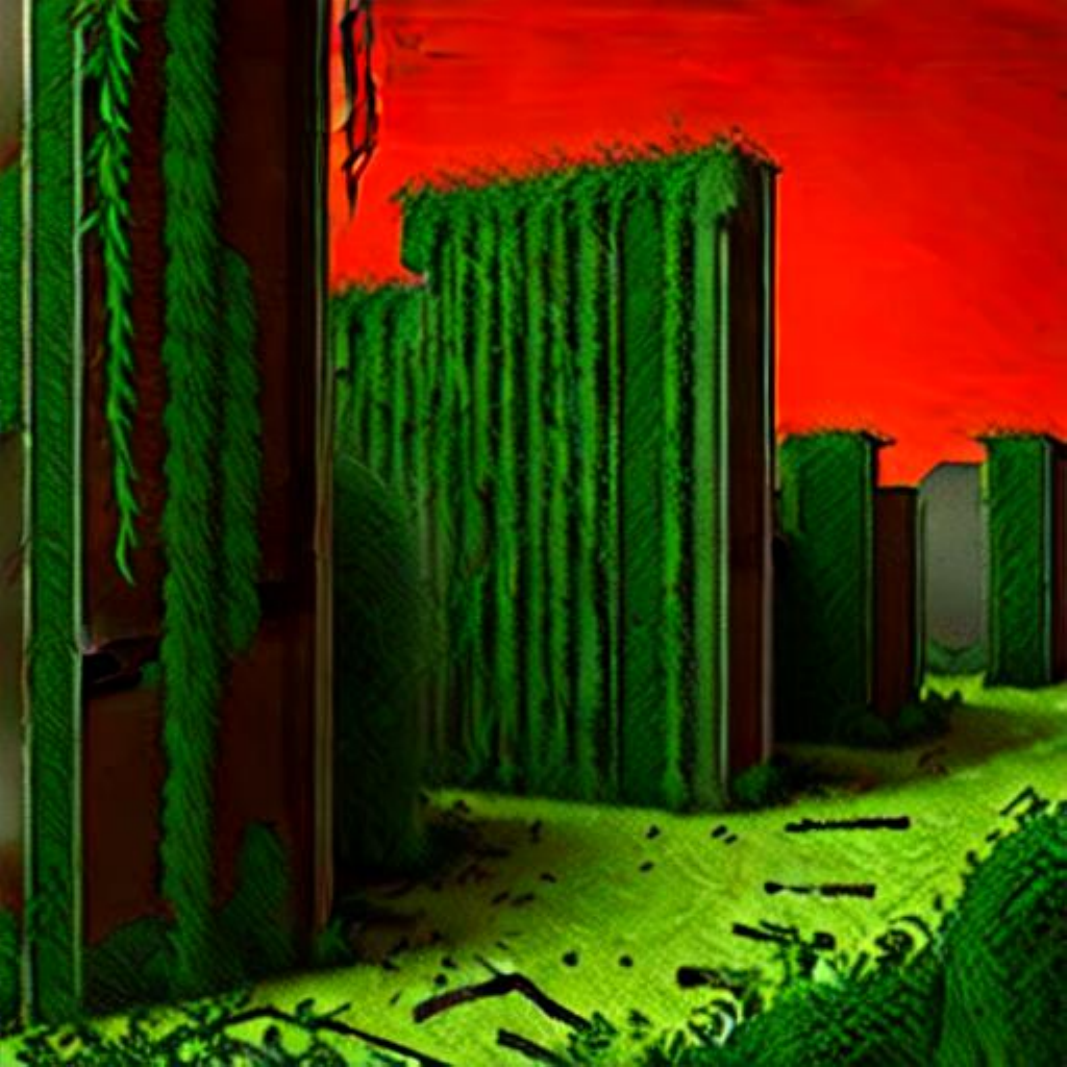}} & 
        \noindent\parbox[c]{0.14\columnwidth}{\includegraphics[width=0.14\columnwidth]{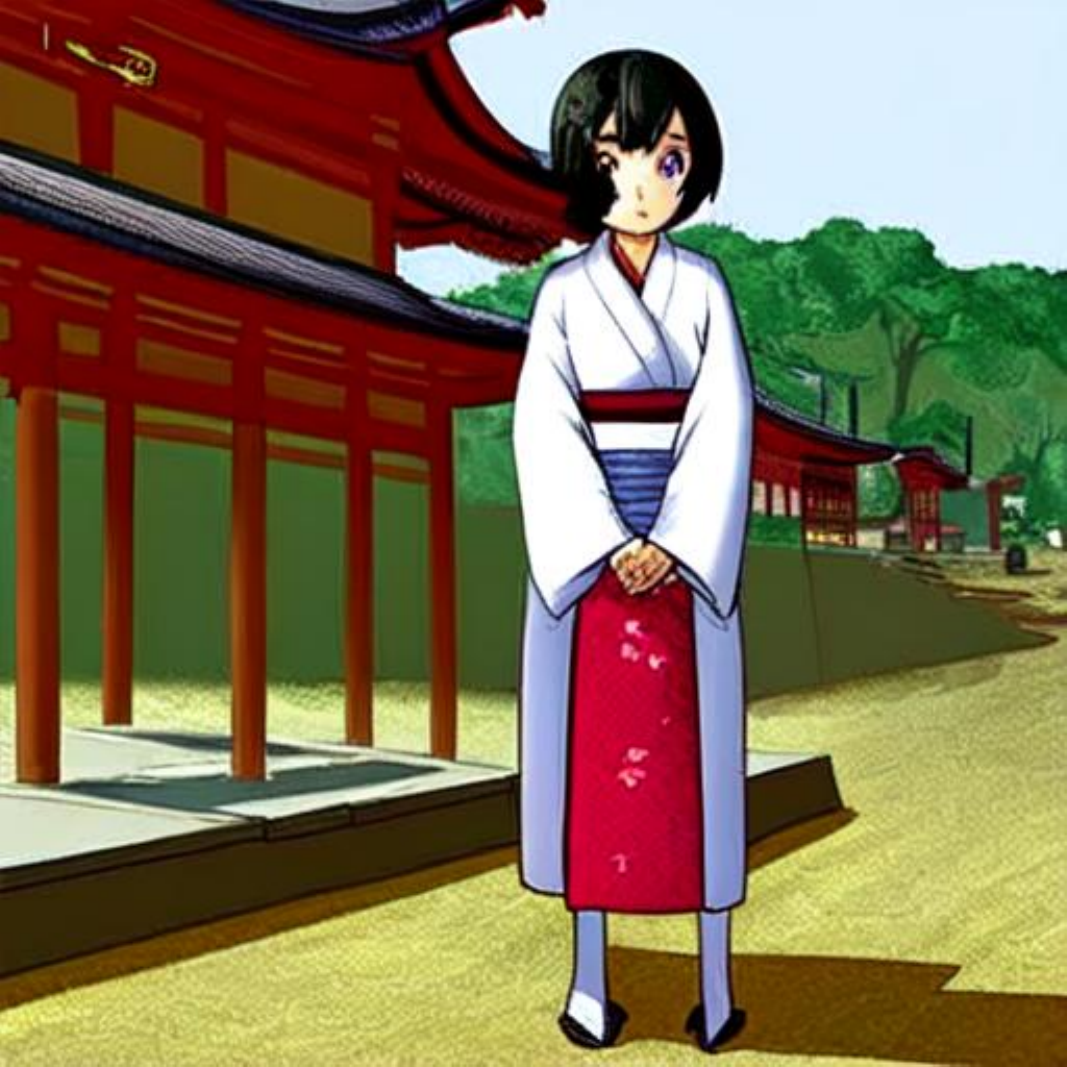}} & 
        \noindent\parbox[c]{0.14\columnwidth}{\includegraphics[width=0.14\columnwidth]{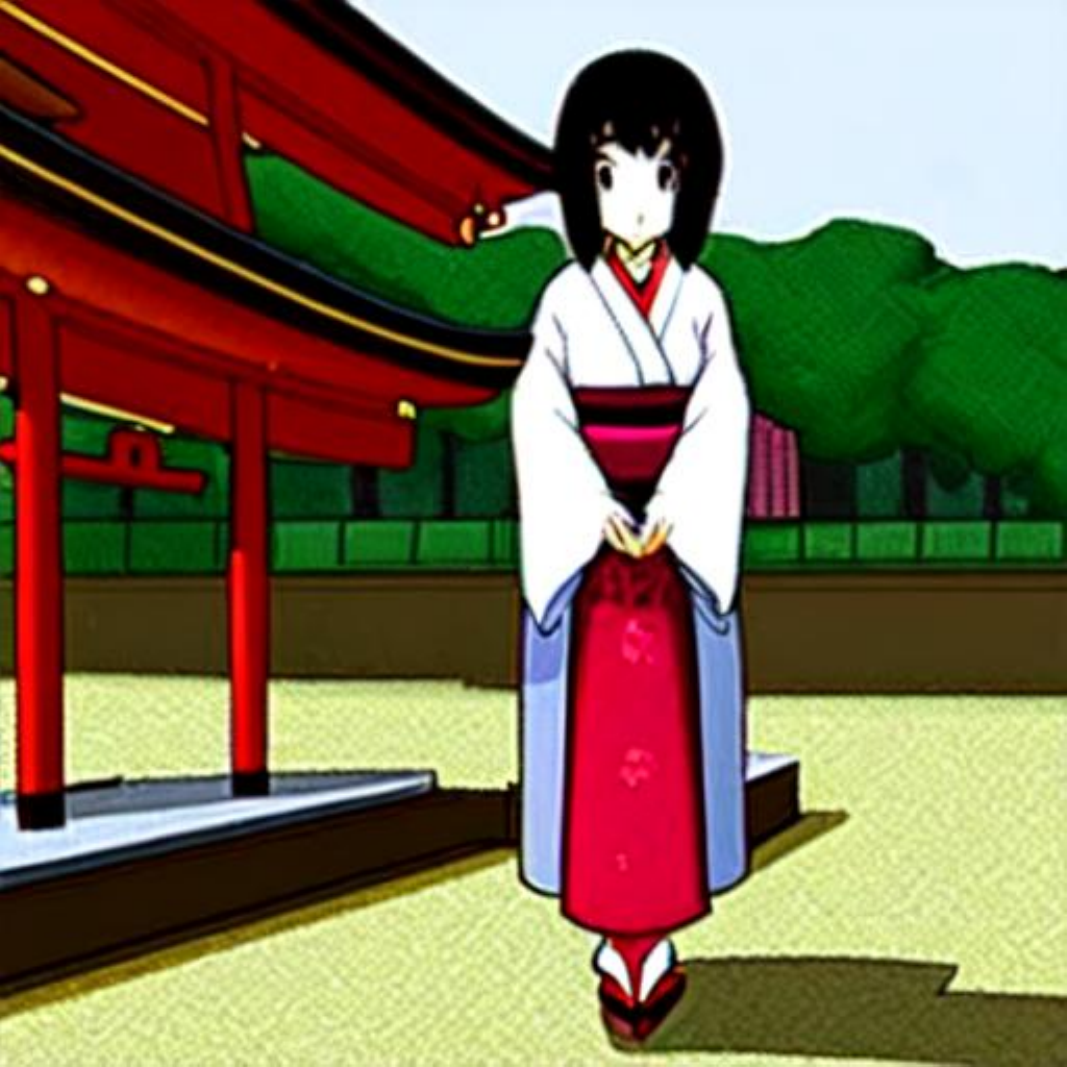}} & 
        \noindent\parbox[c]{0.14\columnwidth}{\includegraphics[width=0.14\columnwidth]{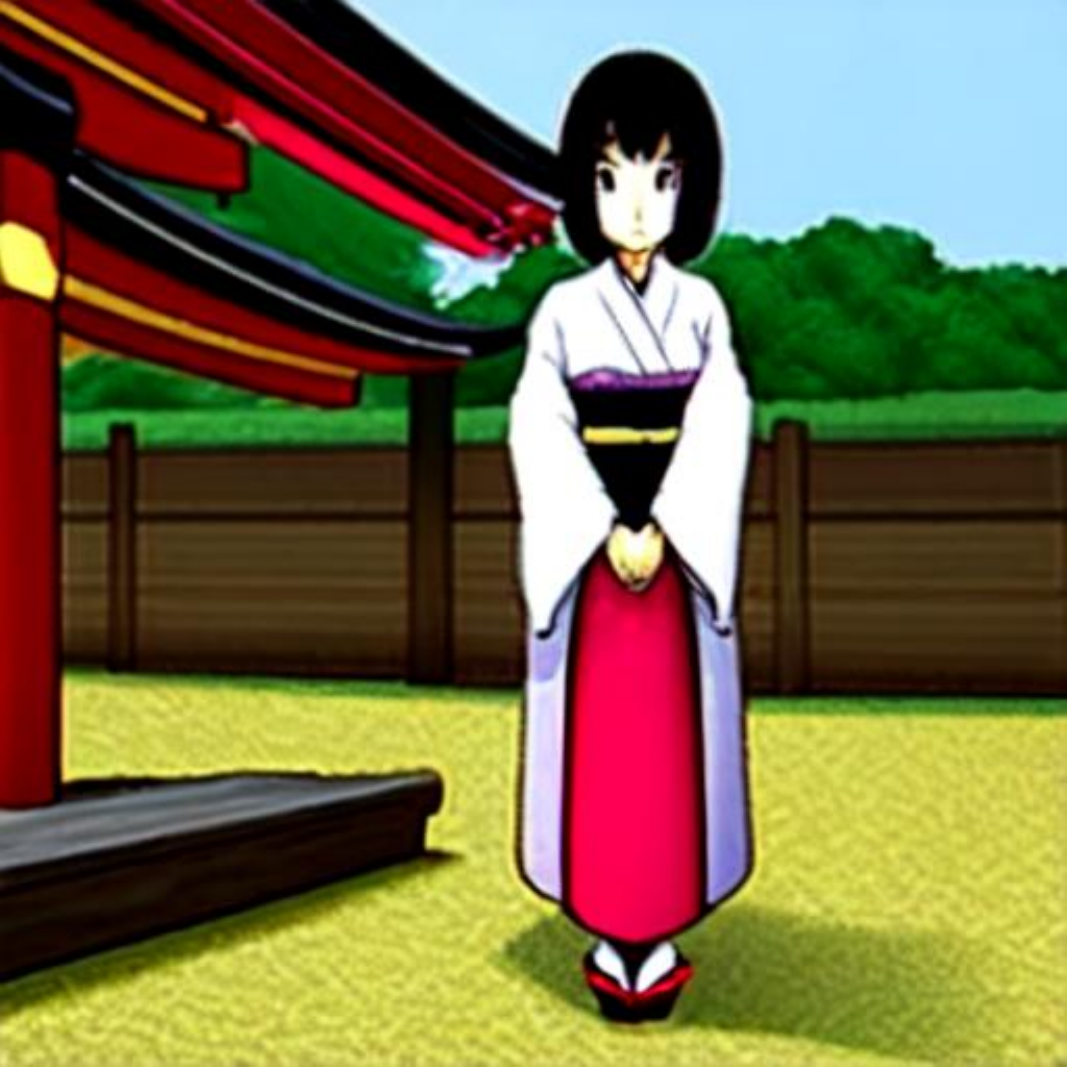}} \\
    \end{tabu}
    \caption{Comparison of samples generated from Waifu Diffusion V1.4 \protect\footnotemark using PLMS4 \cite{liu2022pseudo} with different sampling steps and guidance scale.}
    \label{fig:scale_step_waifu}
\end{figure}

\footnotetext{\url{https://huggingface.co/hakurei/waifu-diffusion}}


\tabulinesep=1pt
\begin{figure}
    \centering
    \begin{tabu} to \textwidth {@{}l@{\hspace{5pt}}c@{\hspace{2pt}}c@{\hspace{2pt}}c@{\hspace{4pt}}c@{\hspace{2pt}}c@{\hspace{2pt}}c@{}}
        & \multicolumn{3}{c}{\shortstack{\scriptsize "A post-apocalyptic world with ruined \\ \scriptsize buildings, overgrown vegetation, and a red sky"}}
        & \multicolumn{3}{c}{\shortstack{\scriptsize "A girl standing in a park in \\ \scriptsize Japanese animation style"}} \\

        & \multicolumn{1}{c}{\shortstack{\scriptsize $s = 7.5$}}
        & \multicolumn{1}{c}{\shortstack{\scriptsize $s = 15$}}
        & \multicolumn{1}{c}{\shortstack{\scriptsize $s = 22.5$}}
        & \multicolumn{1}{c}{\shortstack{\scriptsize $s = 7.5$}}
        & \multicolumn{1}{c}{\shortstack{\scriptsize $s = 15$}}
        & \multicolumn{1}{c}{\shortstack{\scriptsize $s = 22.5$}}
        \\
        
        \shortstack[l]{\tiny 10 steps} &
        \noindent\parbox[c]{0.14\columnwidth}{\includegraphics[width=0.14\columnwidth]{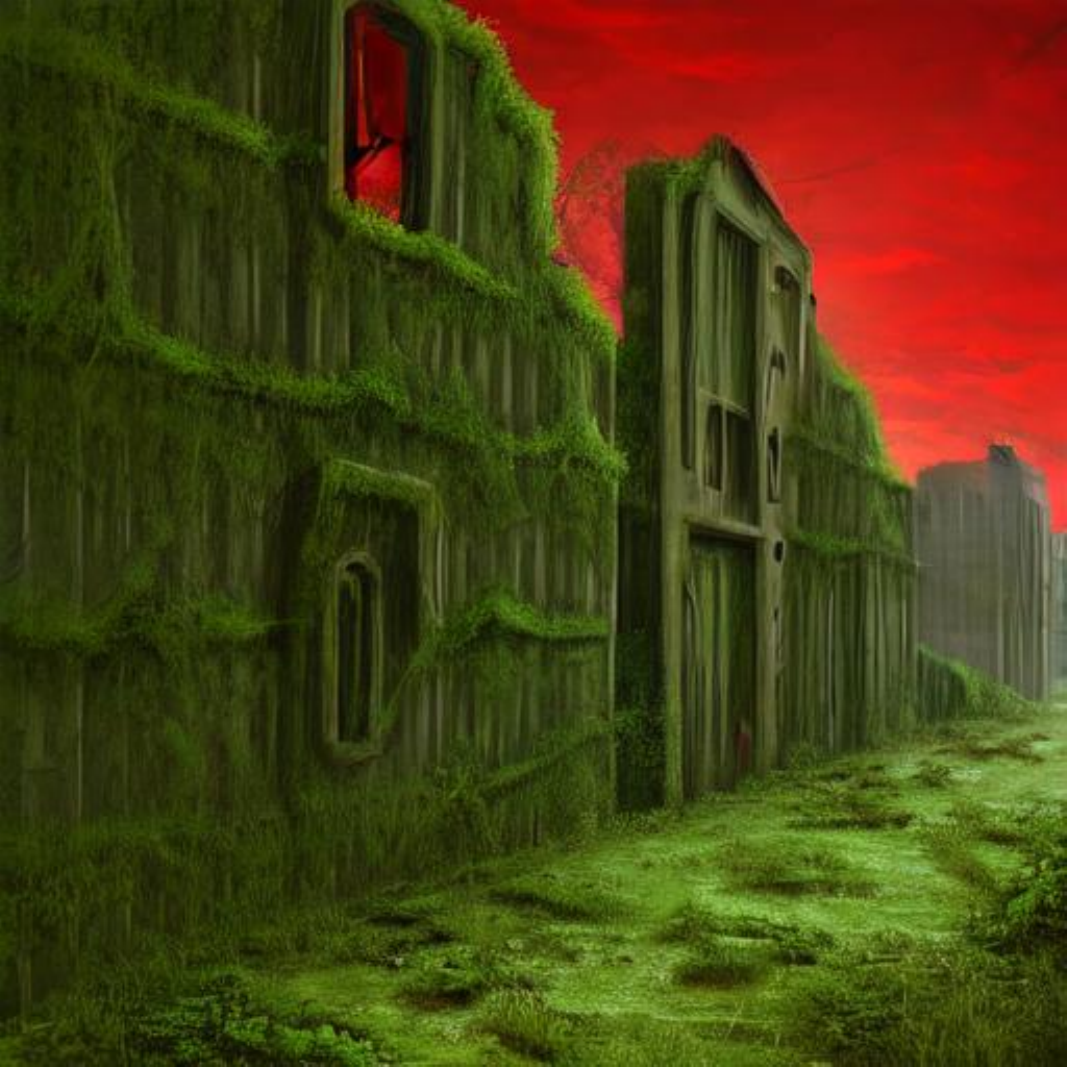}} & 
        \noindent\parbox[c]{0.14\columnwidth}{\includegraphics[width=0.14\columnwidth]{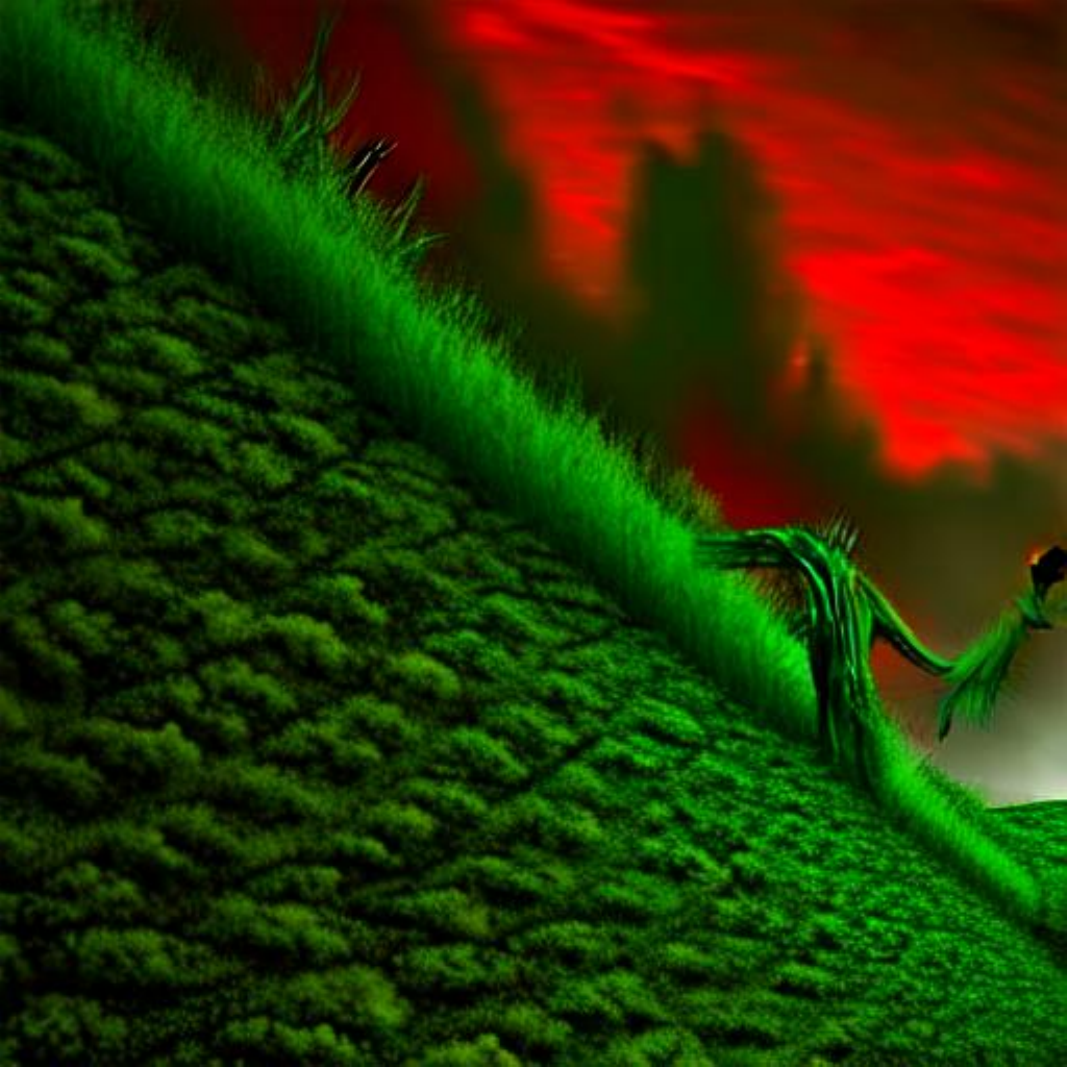}} & 
        \noindent\parbox[c]{0.14\columnwidth}{\includegraphics[width=0.14\columnwidth]{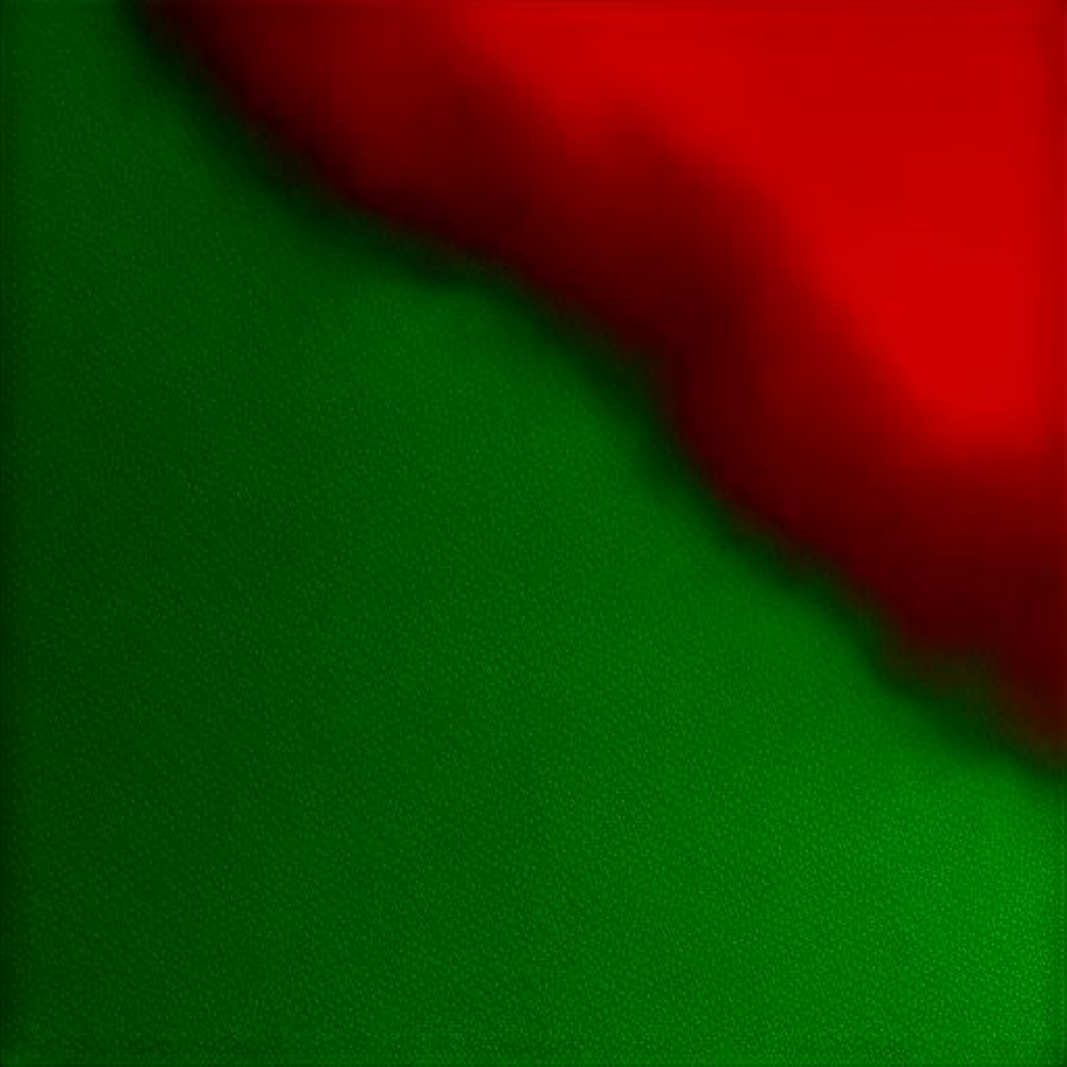}} & 
        \noindent\parbox[c]{0.14\columnwidth}{\includegraphics[width=0.14\columnwidth]{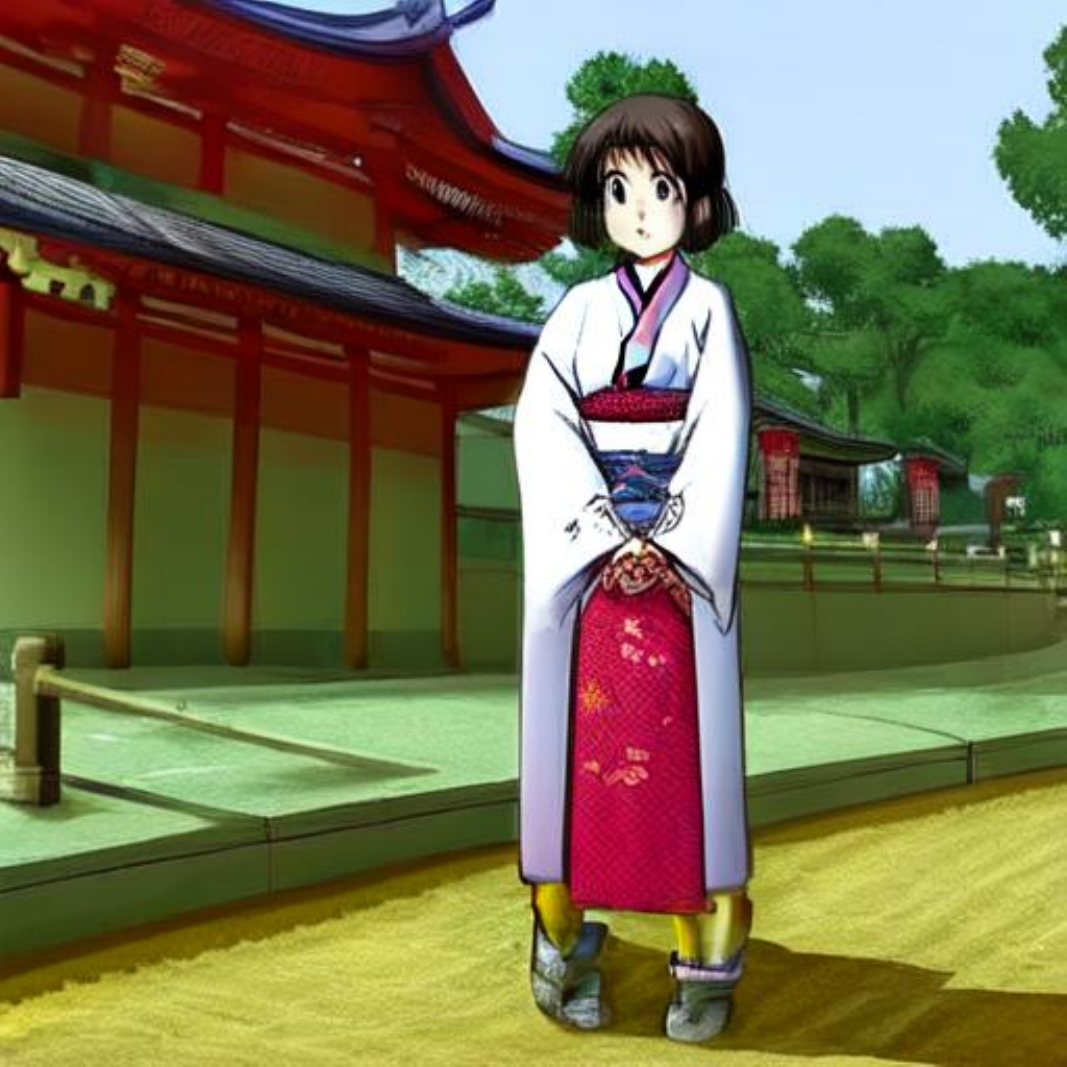}} & 
        \noindent\parbox[c]{0.14\columnwidth}{\includegraphics[width=0.14\columnwidth]{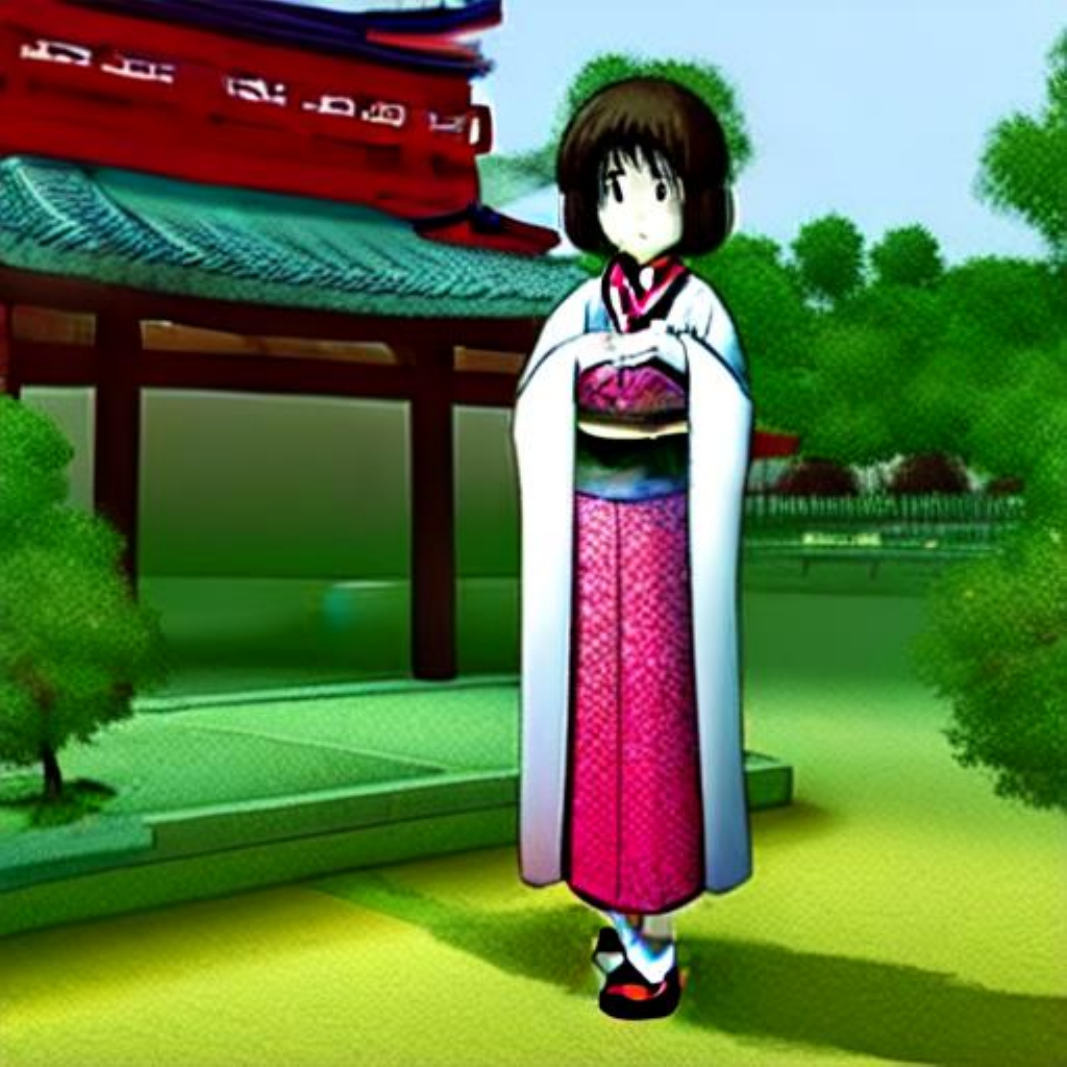}} & 
        \noindent\parbox[c]{0.14\columnwidth}{\includegraphics[width=0.14\columnwidth]{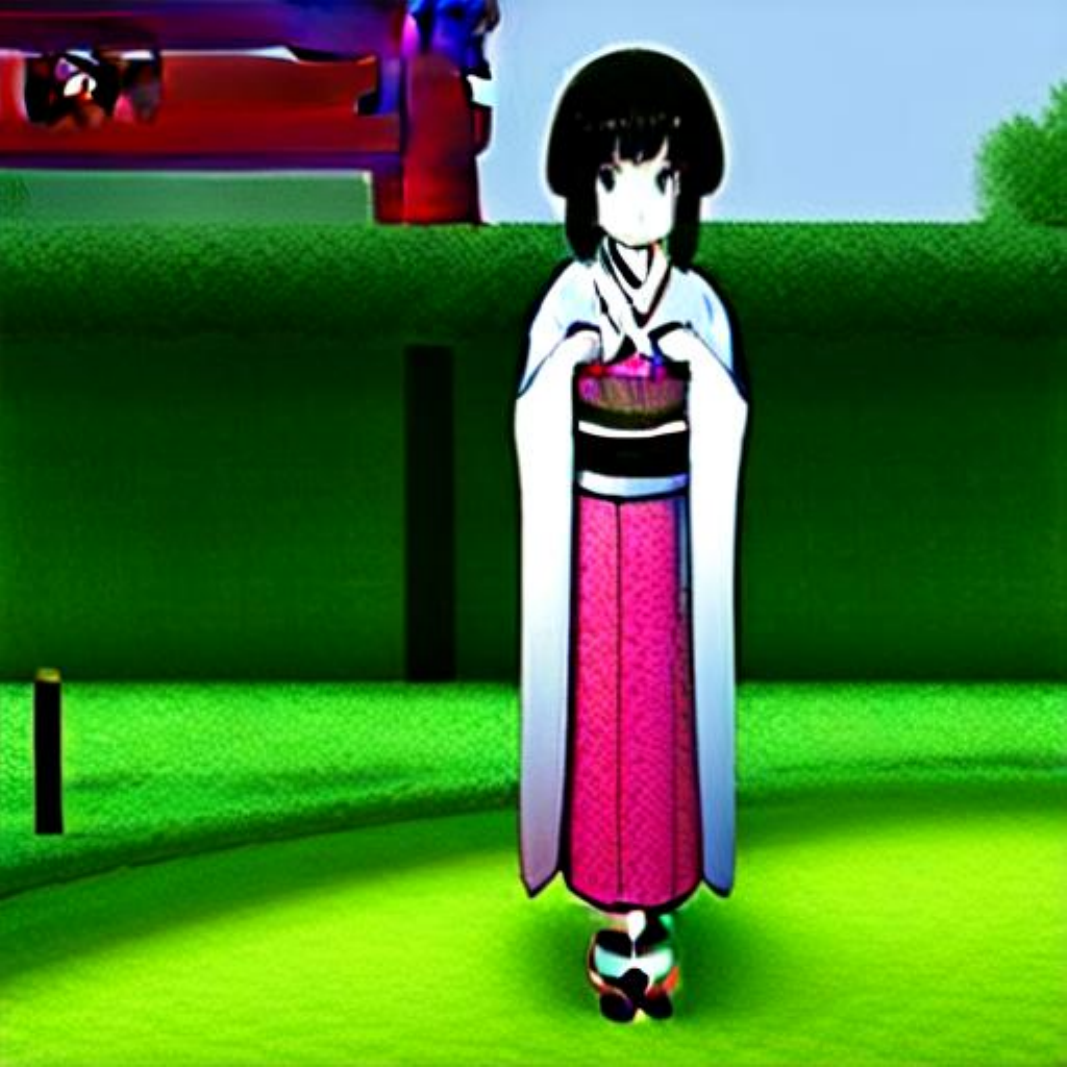}} \\

        \shortstack[l]{\tiny 15 steps} &
        \noindent\parbox[c]{0.14\columnwidth}{\includegraphics[width=0.14\columnwidth]{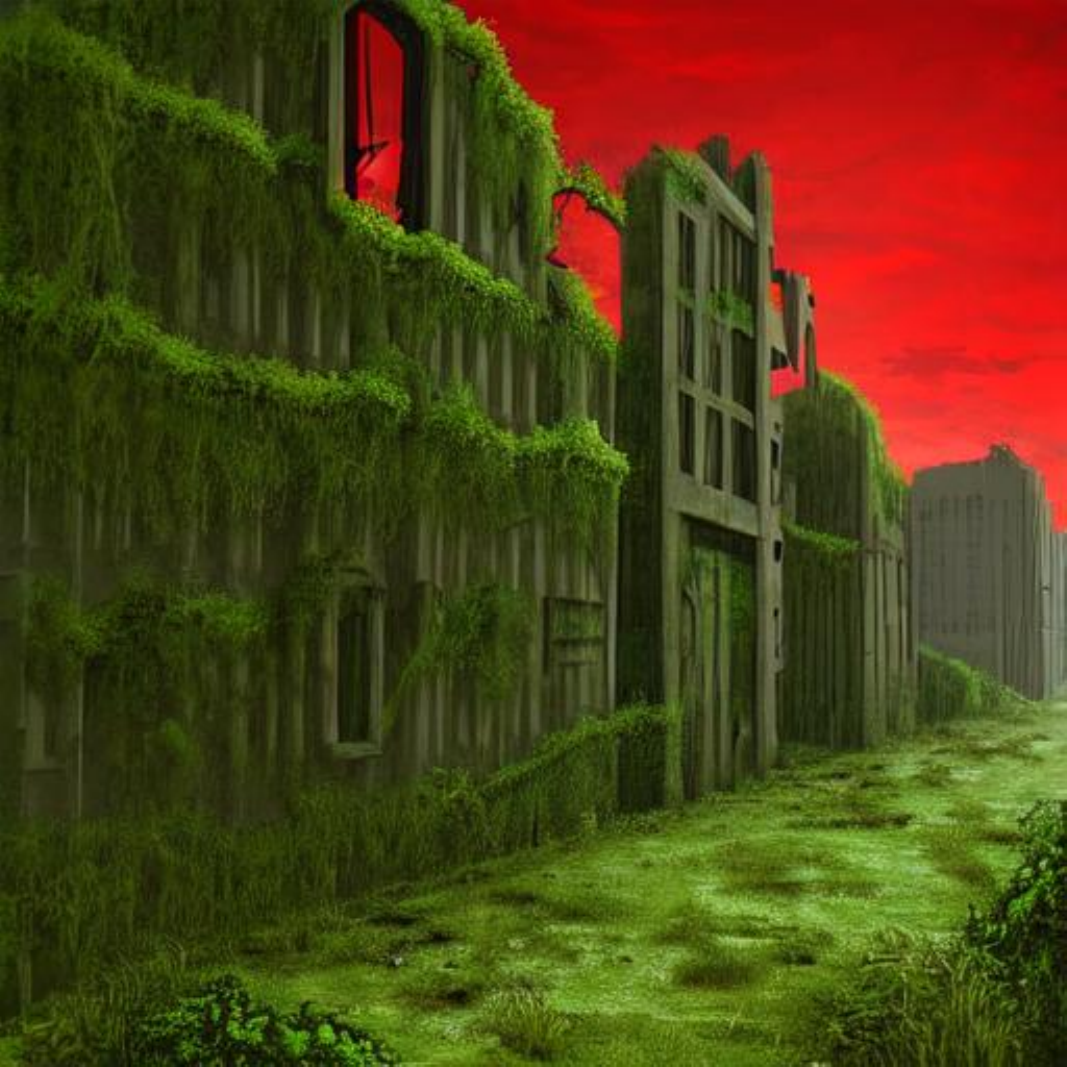}} & 
        \noindent\parbox[c]{0.14\columnwidth}{\includegraphics[width=0.14\columnwidth]{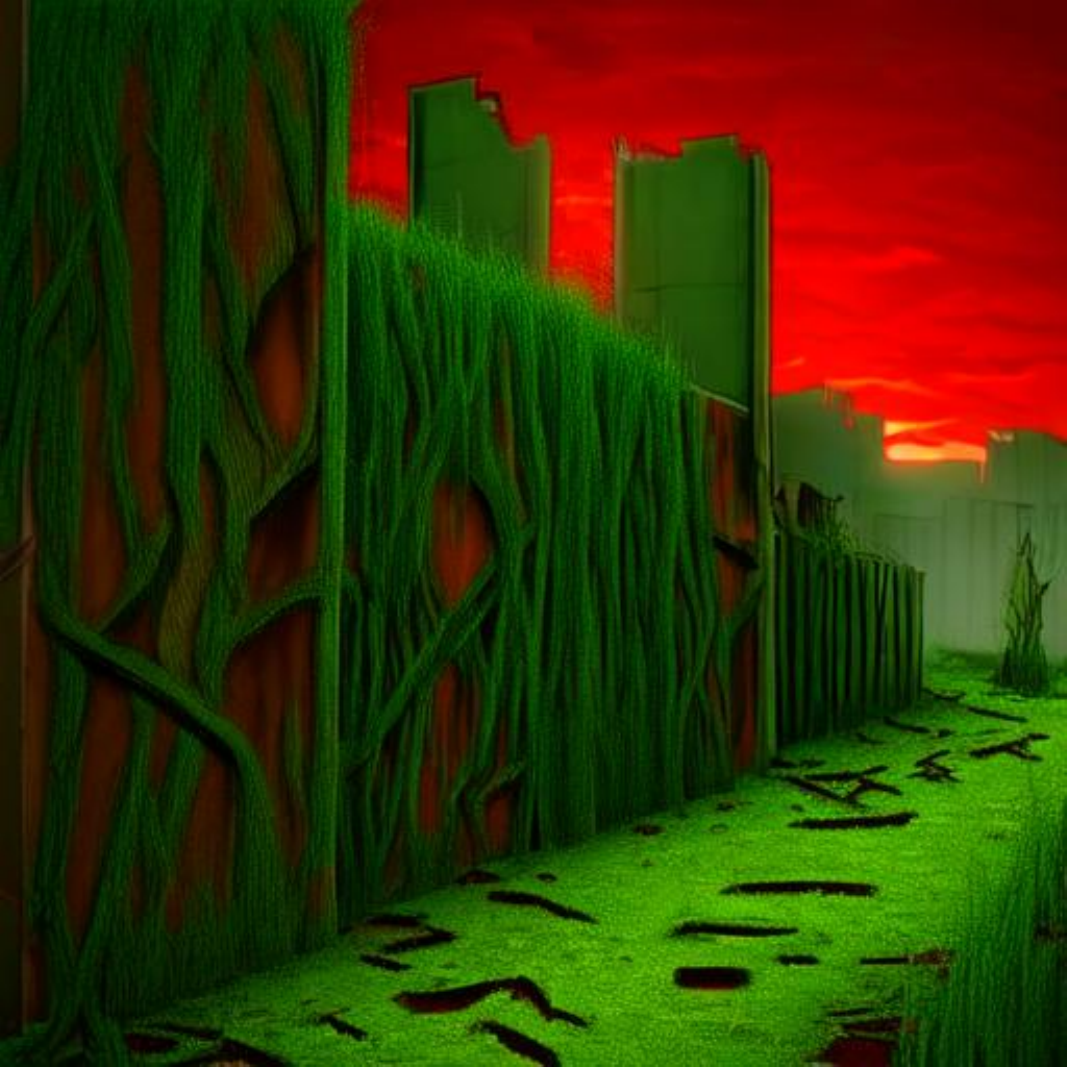}} & 
        \noindent\parbox[c]{0.14\columnwidth}{\includegraphics[width=0.14\columnwidth]{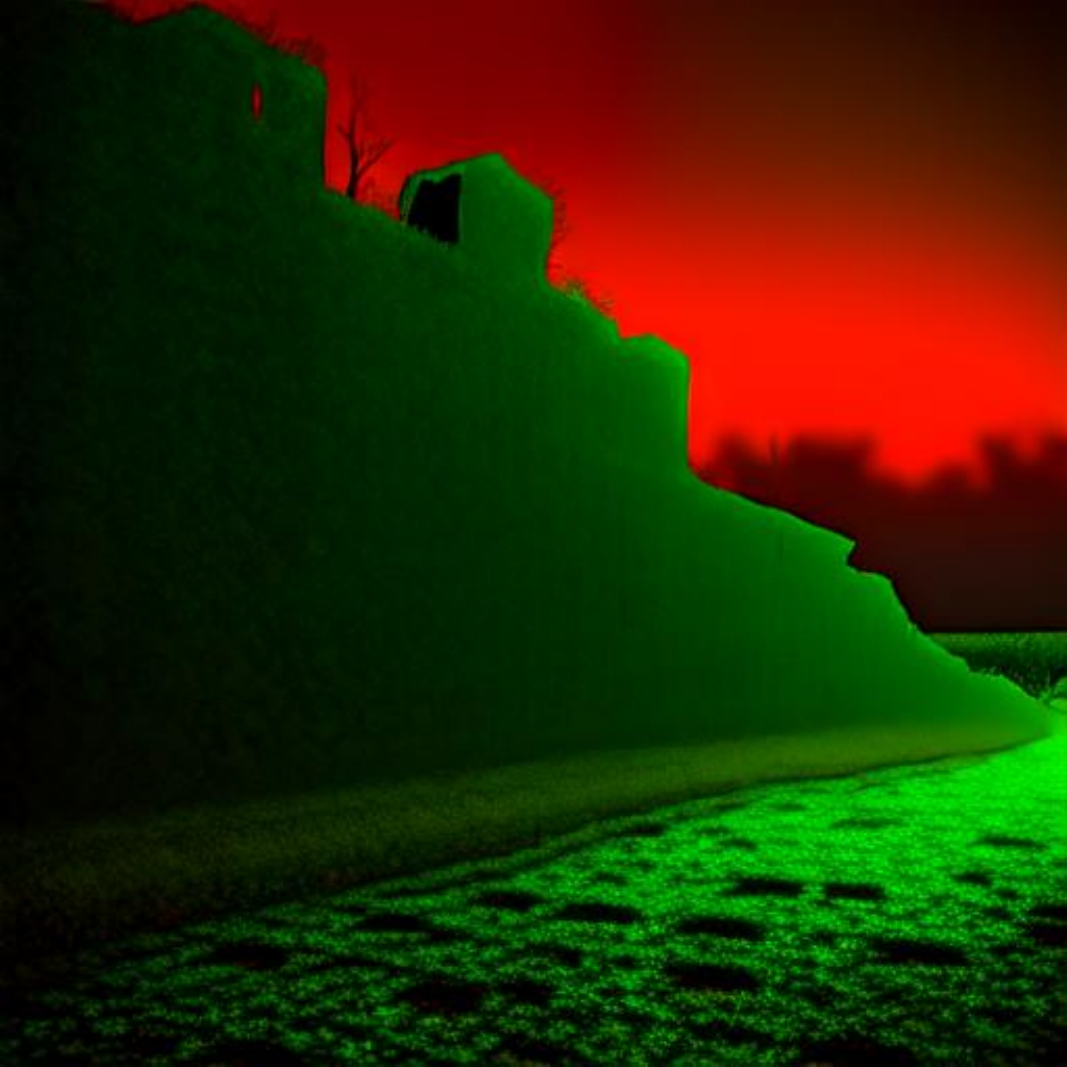}} & 
        \noindent\parbox[c]{0.14\columnwidth}{\includegraphics[width=0.14\columnwidth]{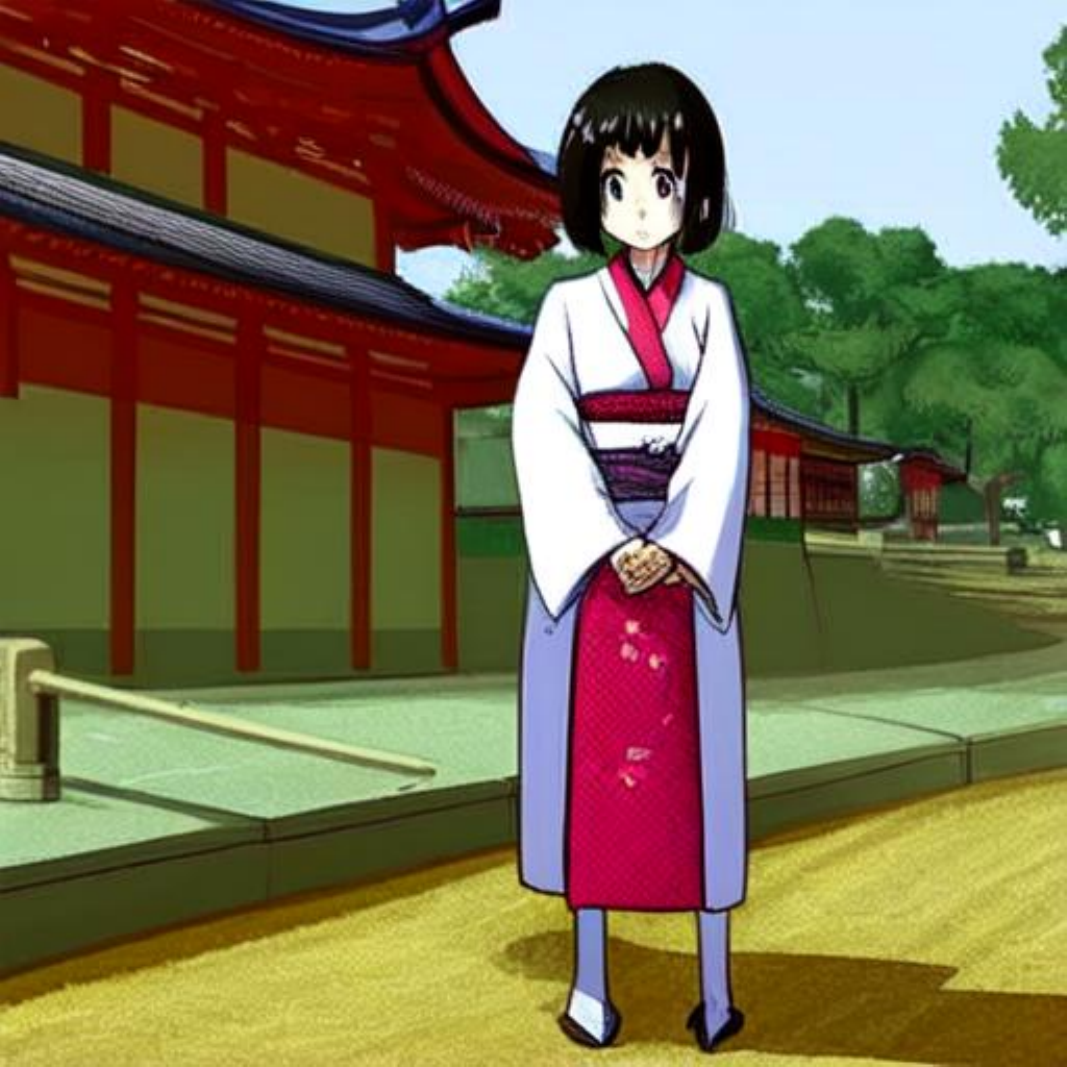}} & 
        \noindent\parbox[c]{0.14\columnwidth}{\includegraphics[width=0.14\columnwidth]{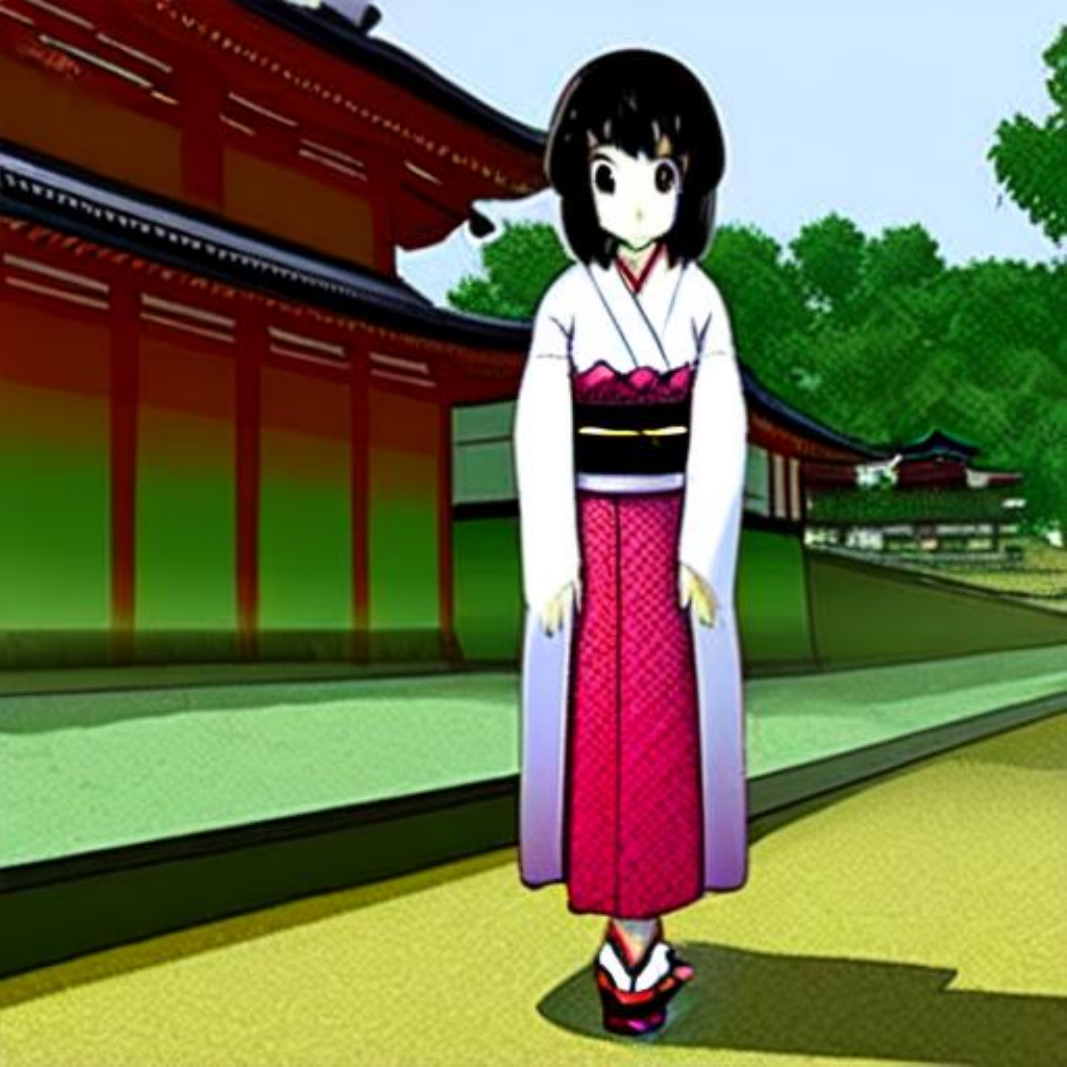}} & 
        \noindent\parbox[c]{0.14\columnwidth}{\includegraphics[width=0.14\columnwidth]{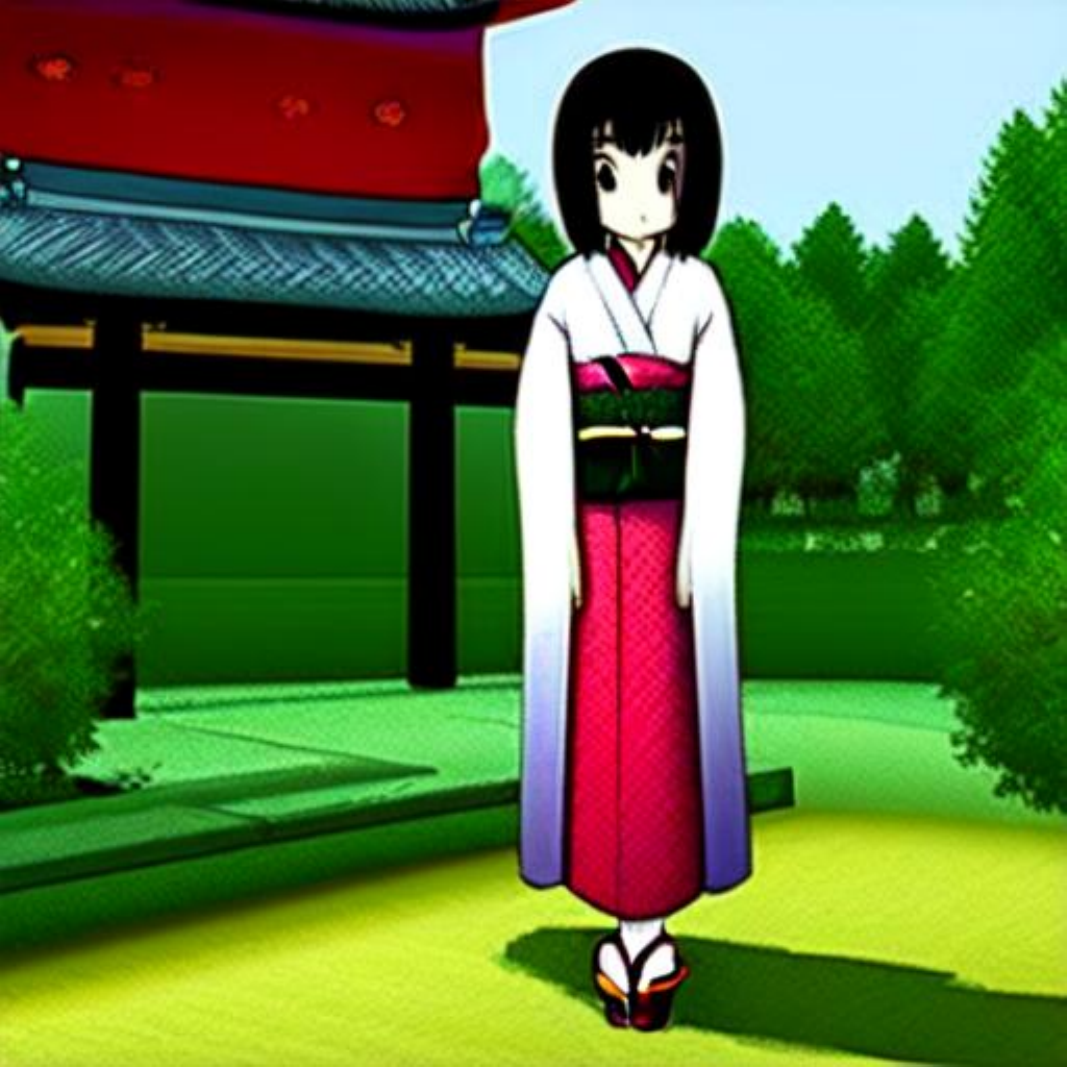}} \\

        \shortstack[l]{\tiny 20 steps} &
        \noindent\parbox[c]{0.14\columnwidth}{\includegraphics[width=0.14\columnwidth]{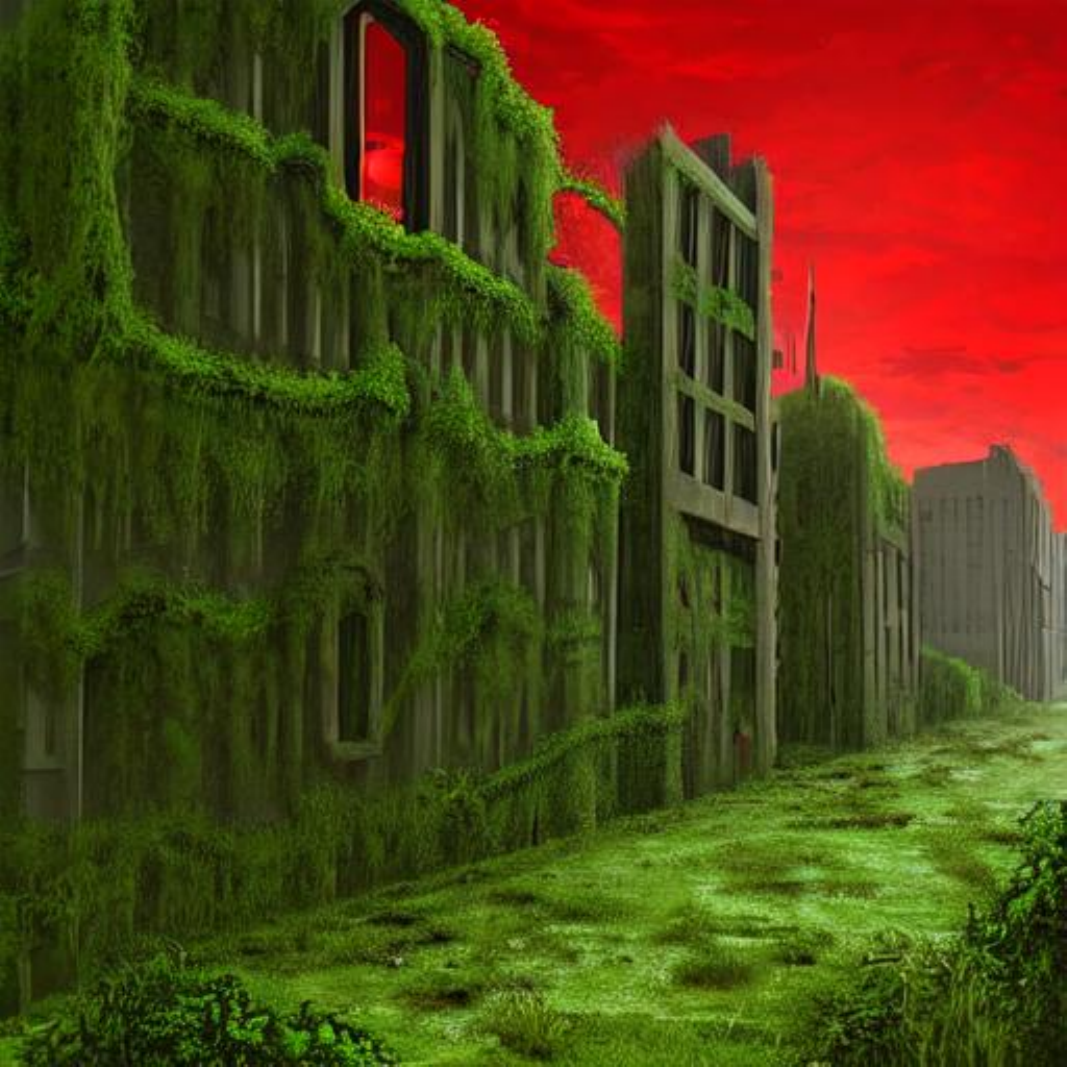}} & 
        \noindent\parbox[c]{0.14\columnwidth}{\includegraphics[width=0.14\columnwidth]{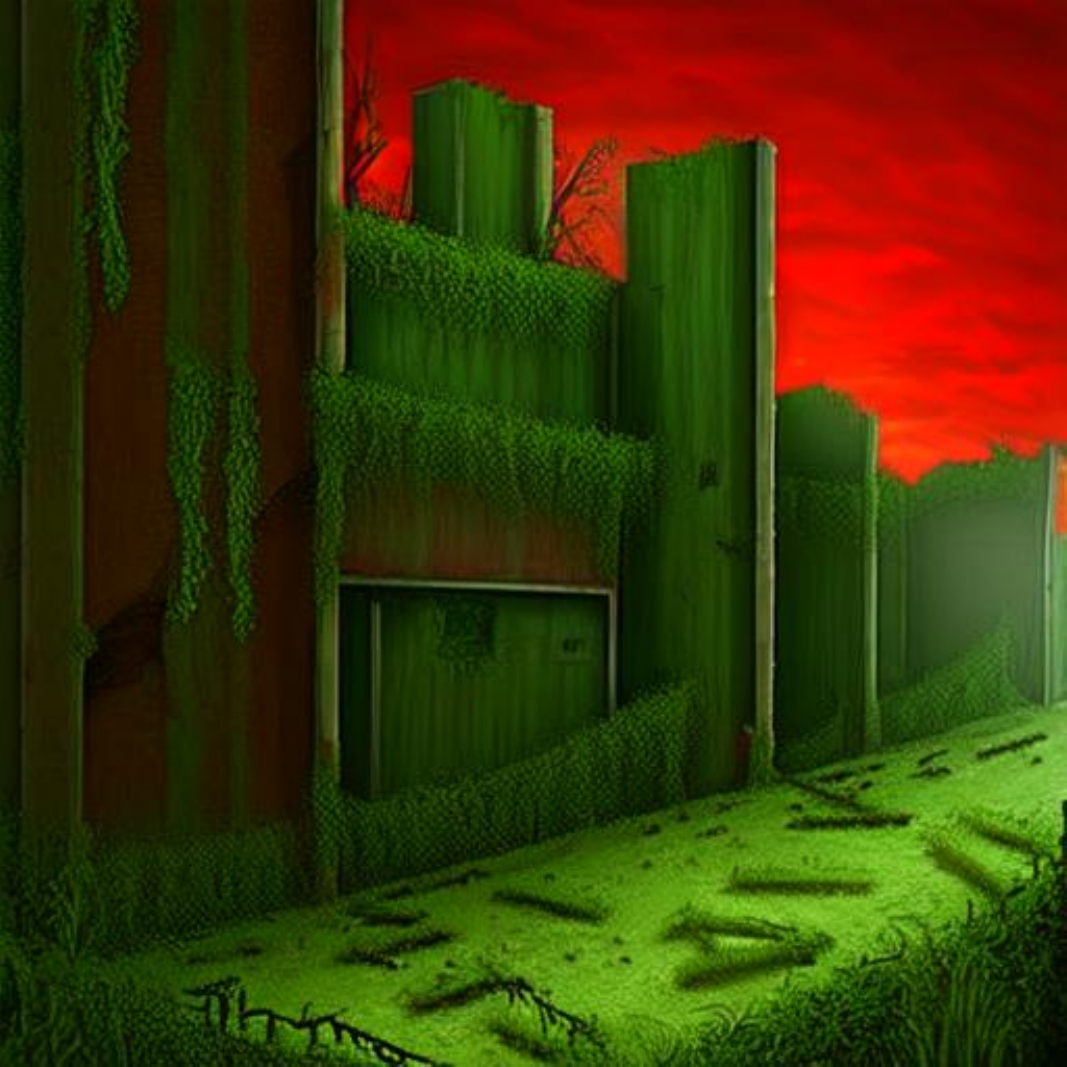}} & 
        \noindent\parbox[c]{0.14\columnwidth}{\includegraphics[width=0.14\columnwidth]{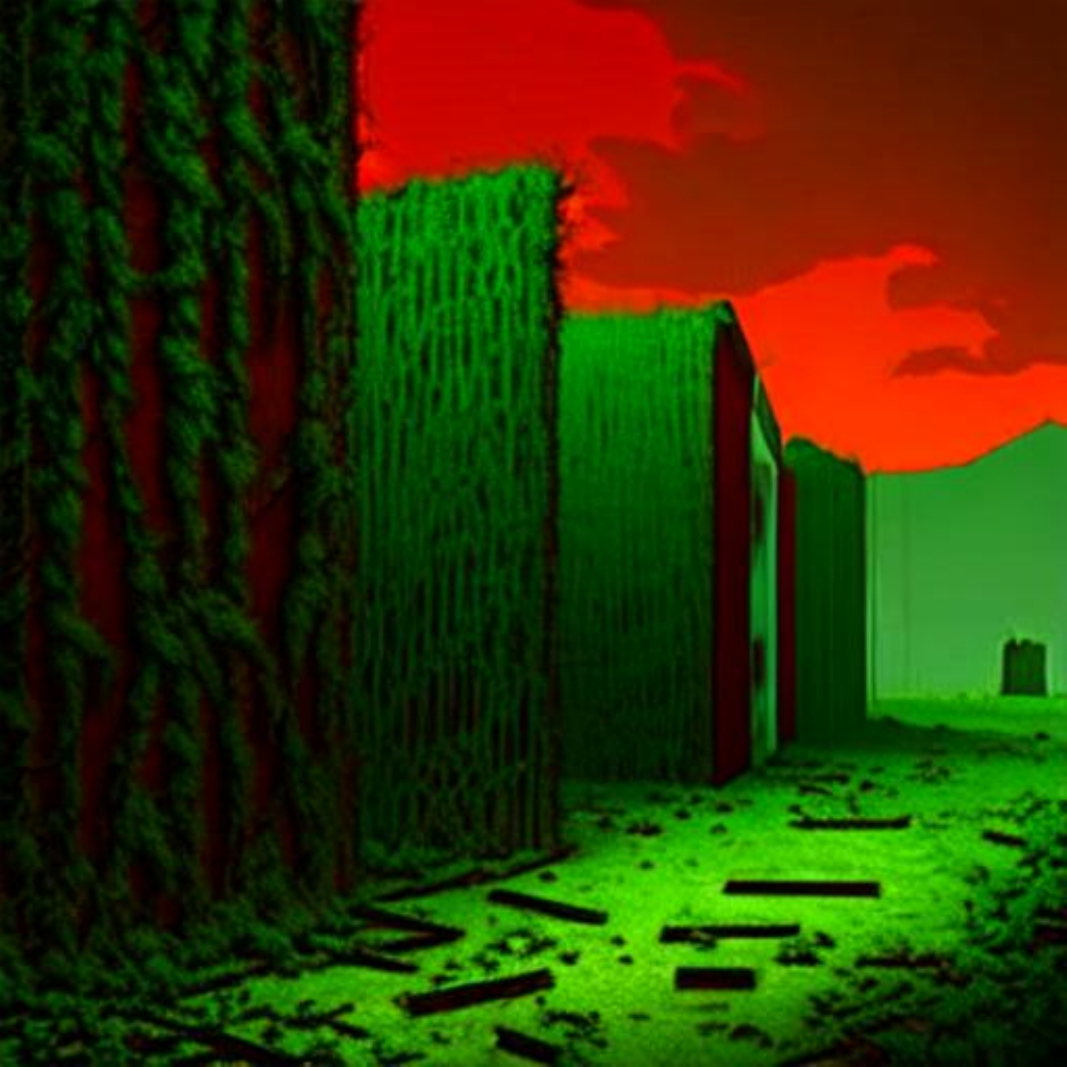}} & 
        \noindent\parbox[c]{0.14\columnwidth}{\includegraphics[width=0.14\columnwidth]{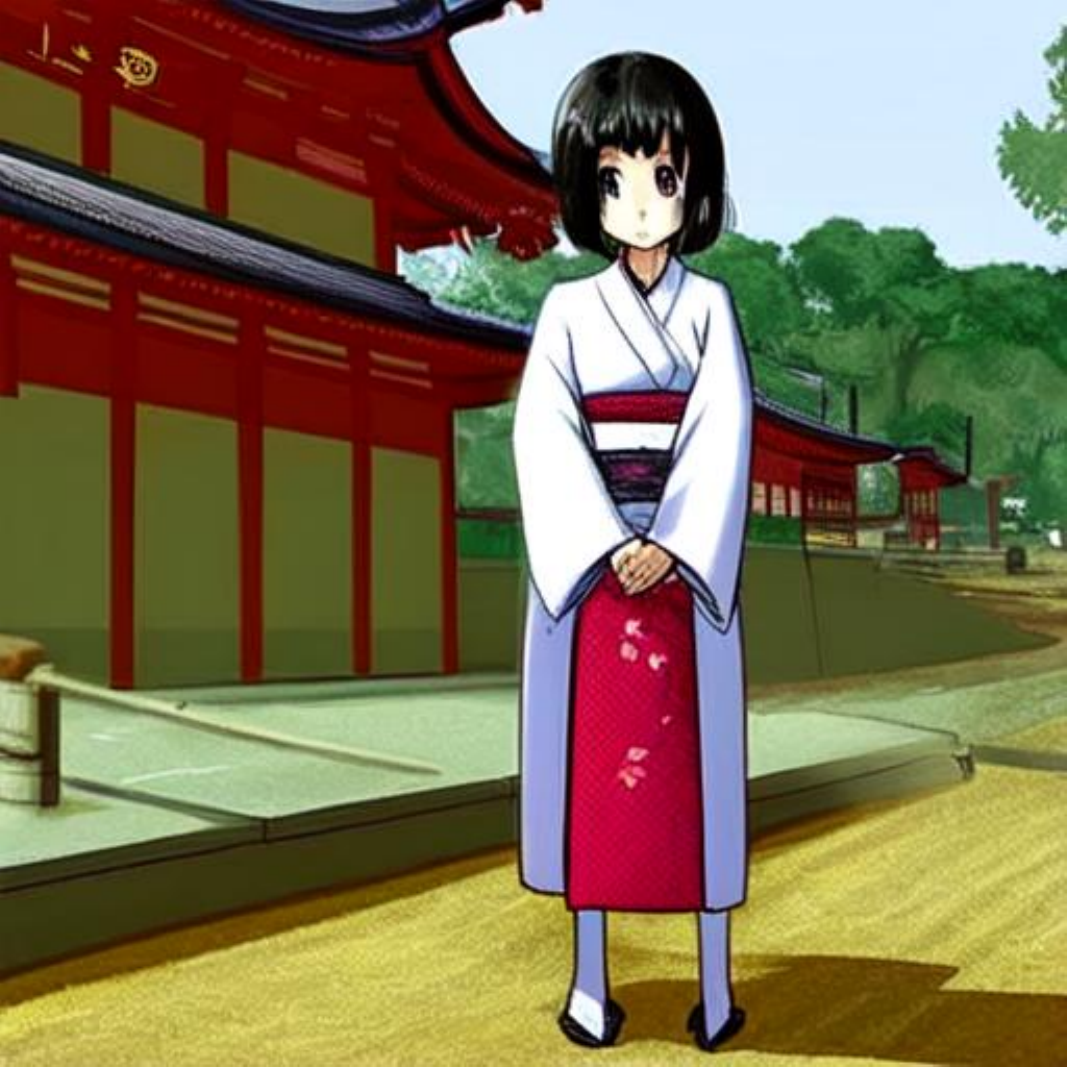}} & 
        \noindent\parbox[c]{0.14\columnwidth}{\includegraphics[width=0.14\columnwidth]{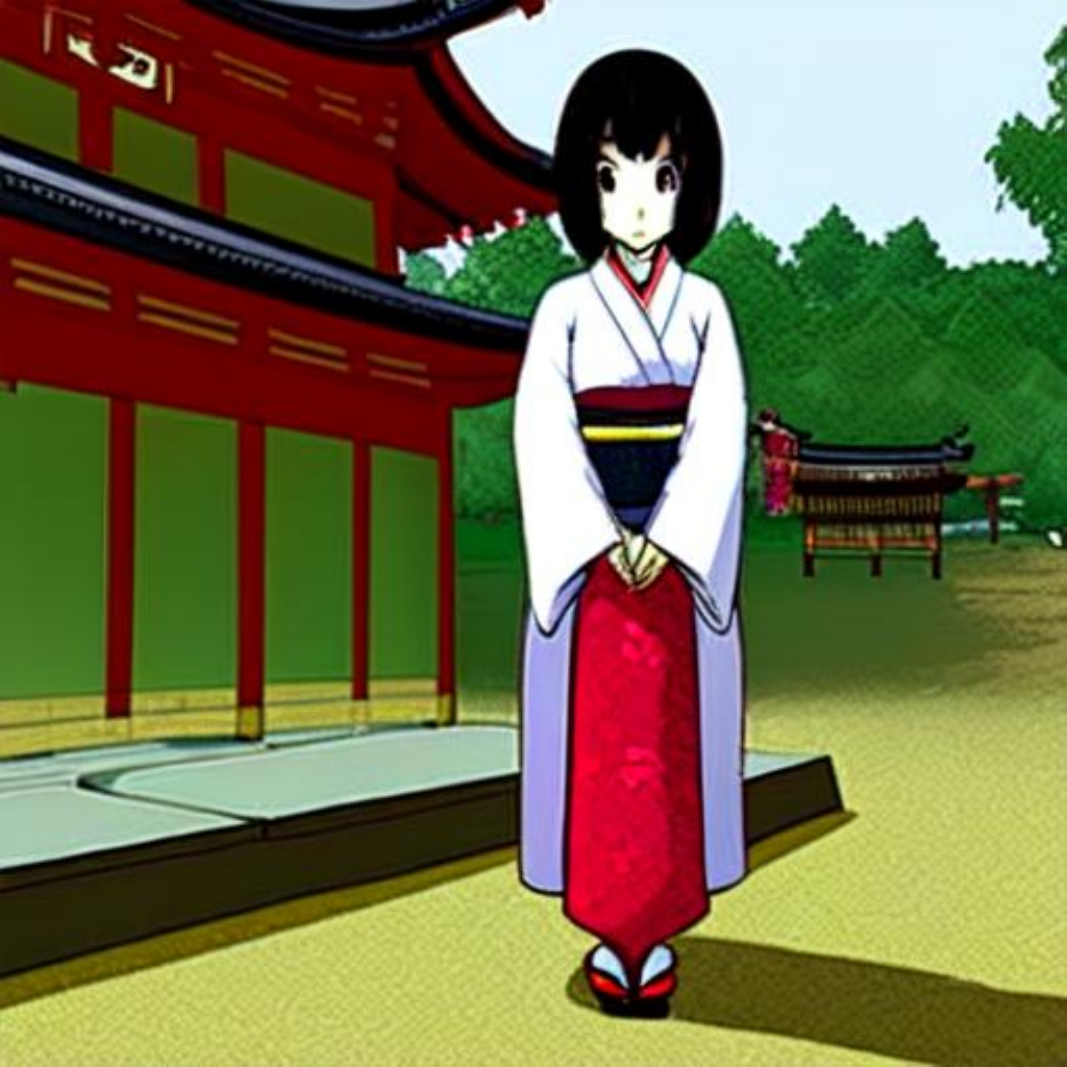}} & 
        \noindent\parbox[c]{0.14\columnwidth}{\includegraphics[width=0.14\columnwidth]{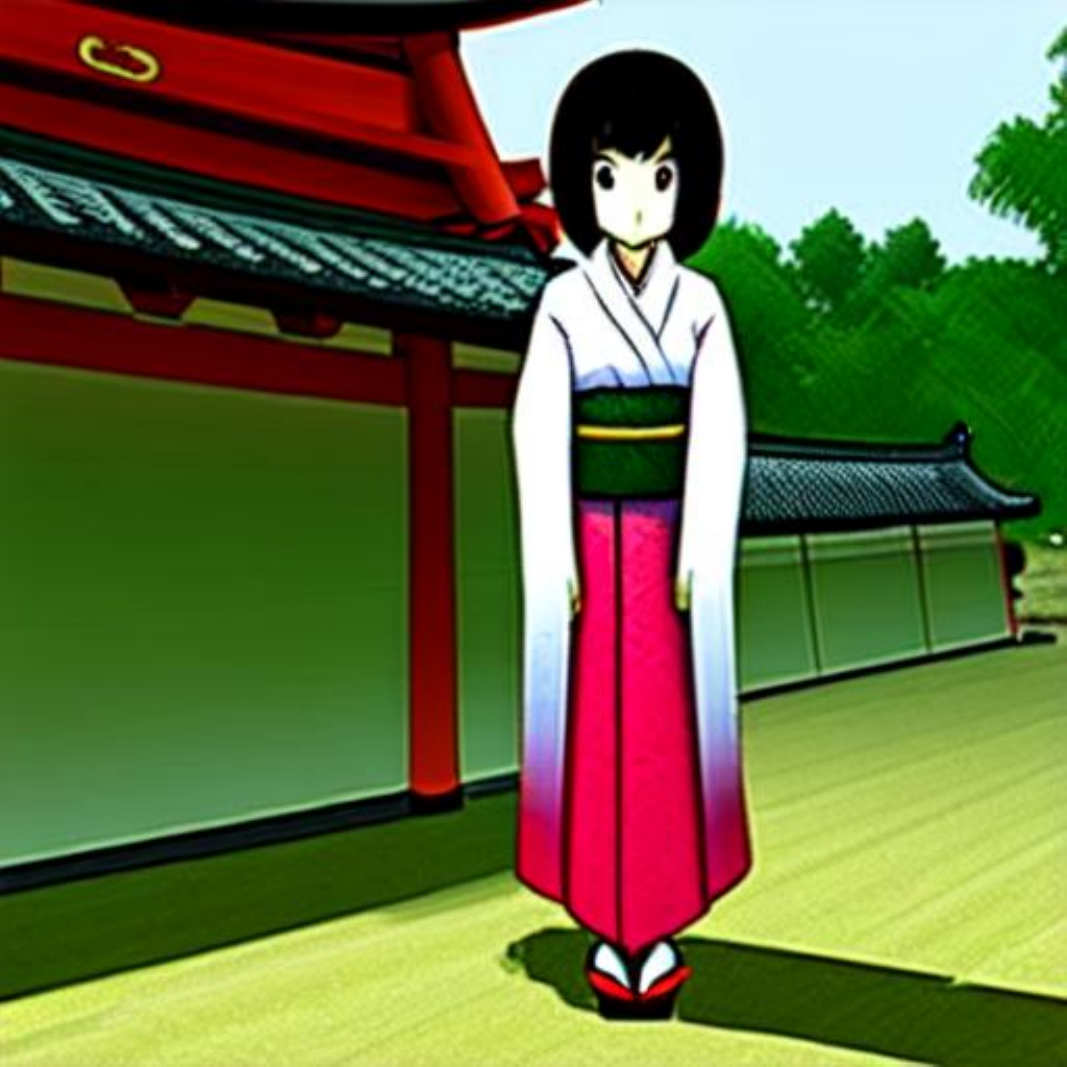}} \\

        \shortstack[l]{\tiny 40 steps} &
        \noindent\parbox[c]{0.14\columnwidth}{\includegraphics[width=0.14\columnwidth]{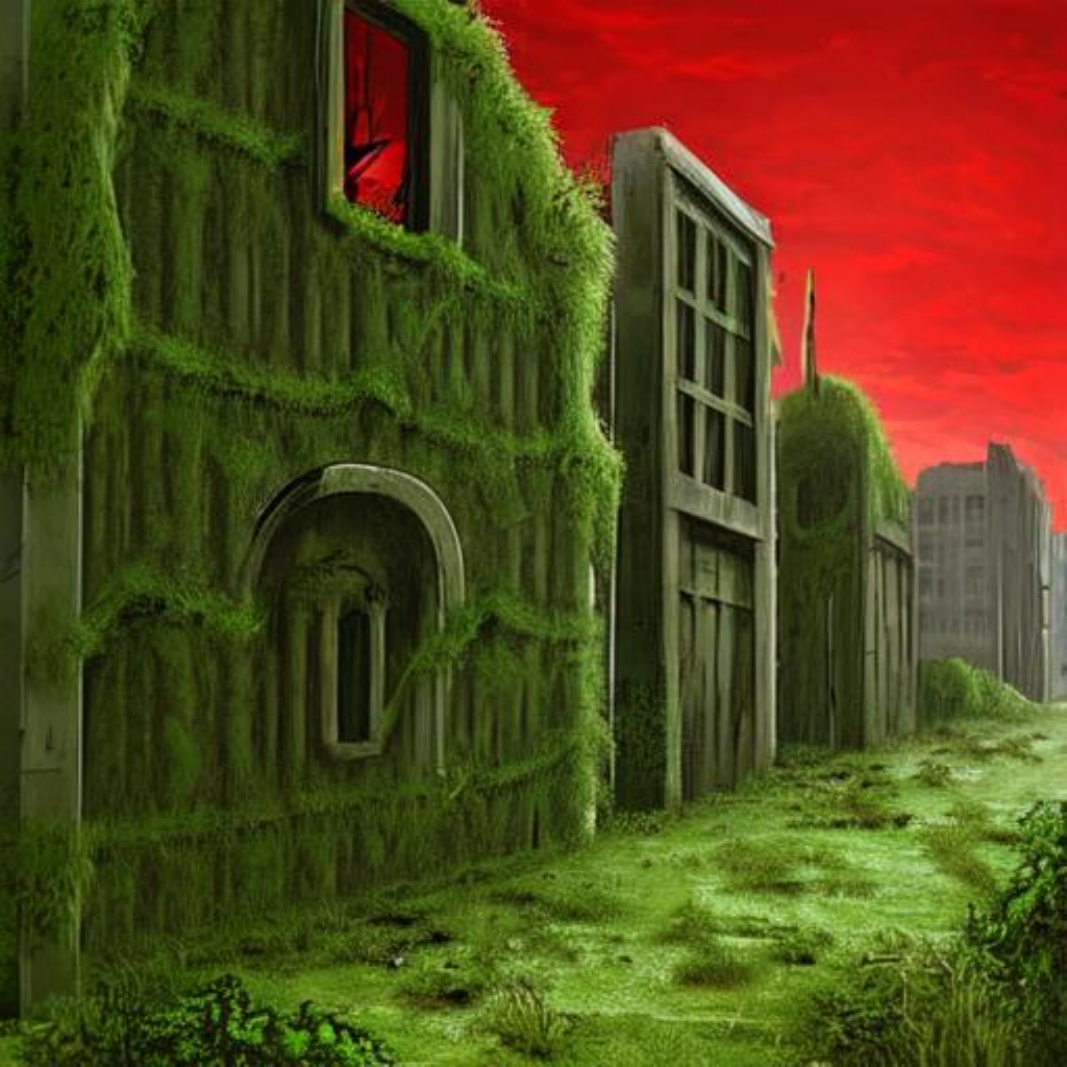}} & 
        \noindent\parbox[c]{0.14\columnwidth}{\includegraphics[width=0.14\columnwidth]{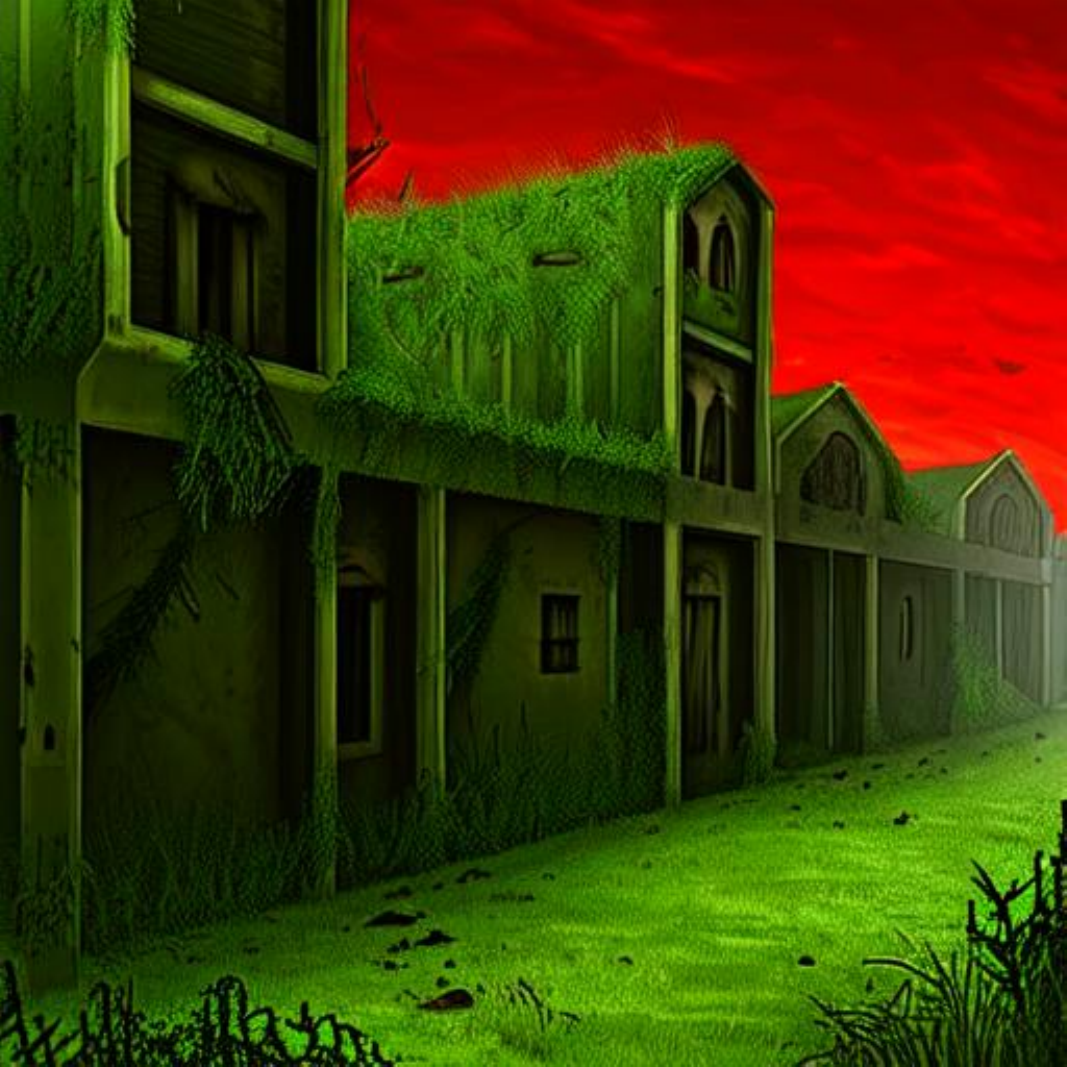}} & 
        \noindent\parbox[c]{0.14\columnwidth}{\includegraphics[width=0.14\columnwidth]{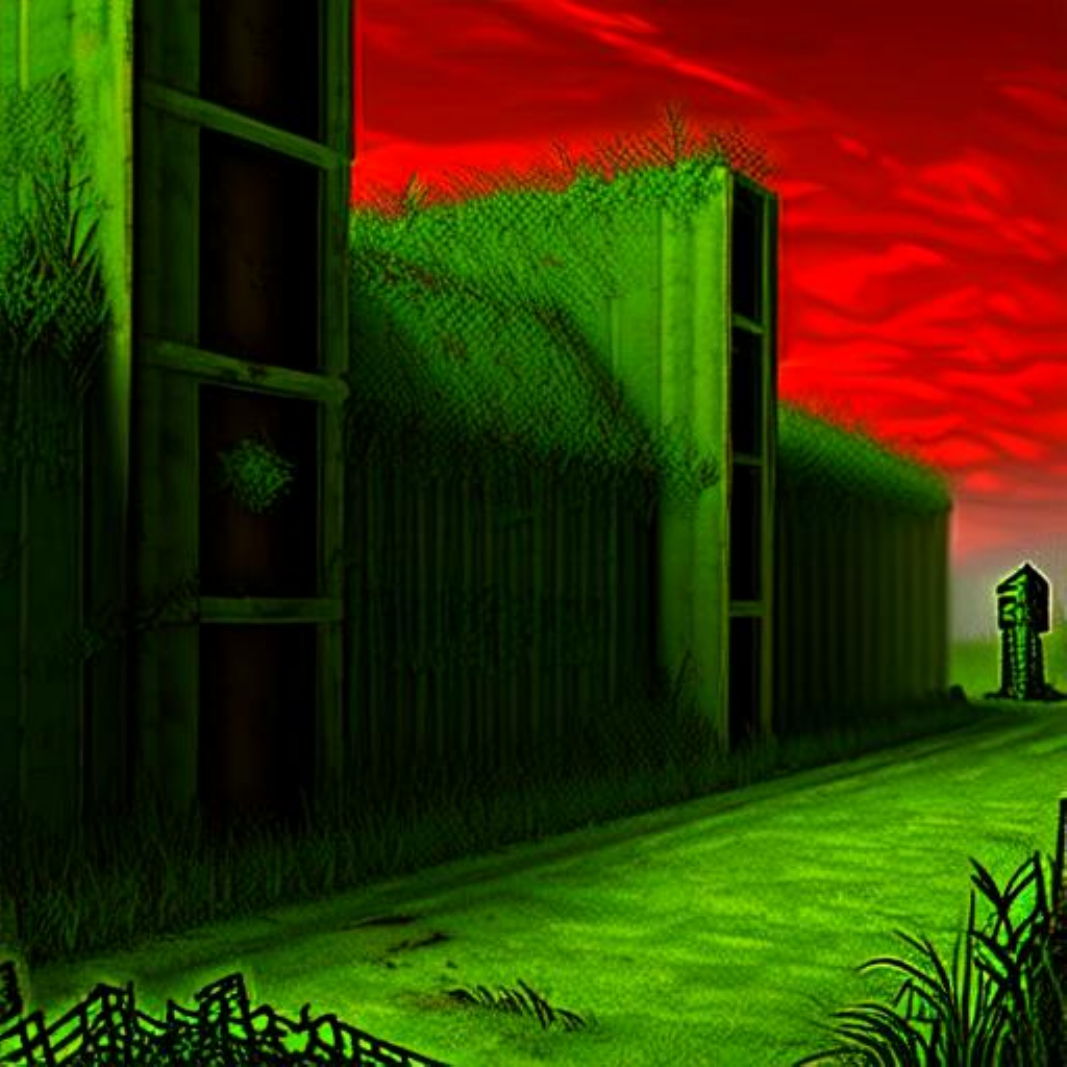}} & 
        \noindent\parbox[c]{0.14\columnwidth}{\includegraphics[width=0.14\columnwidth]{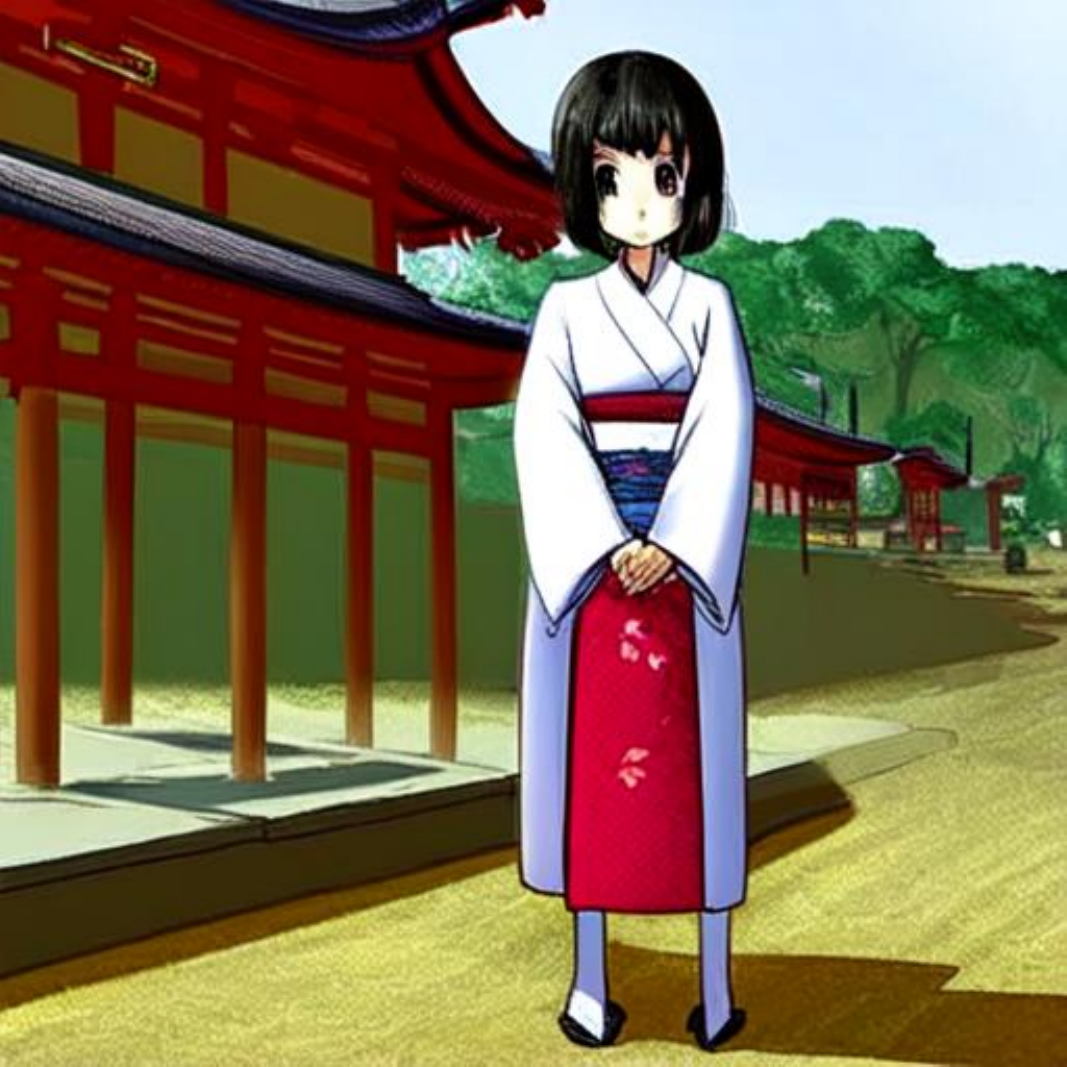}} & 
        \noindent\parbox[c]{0.14\columnwidth}{\includegraphics[width=0.14\columnwidth]{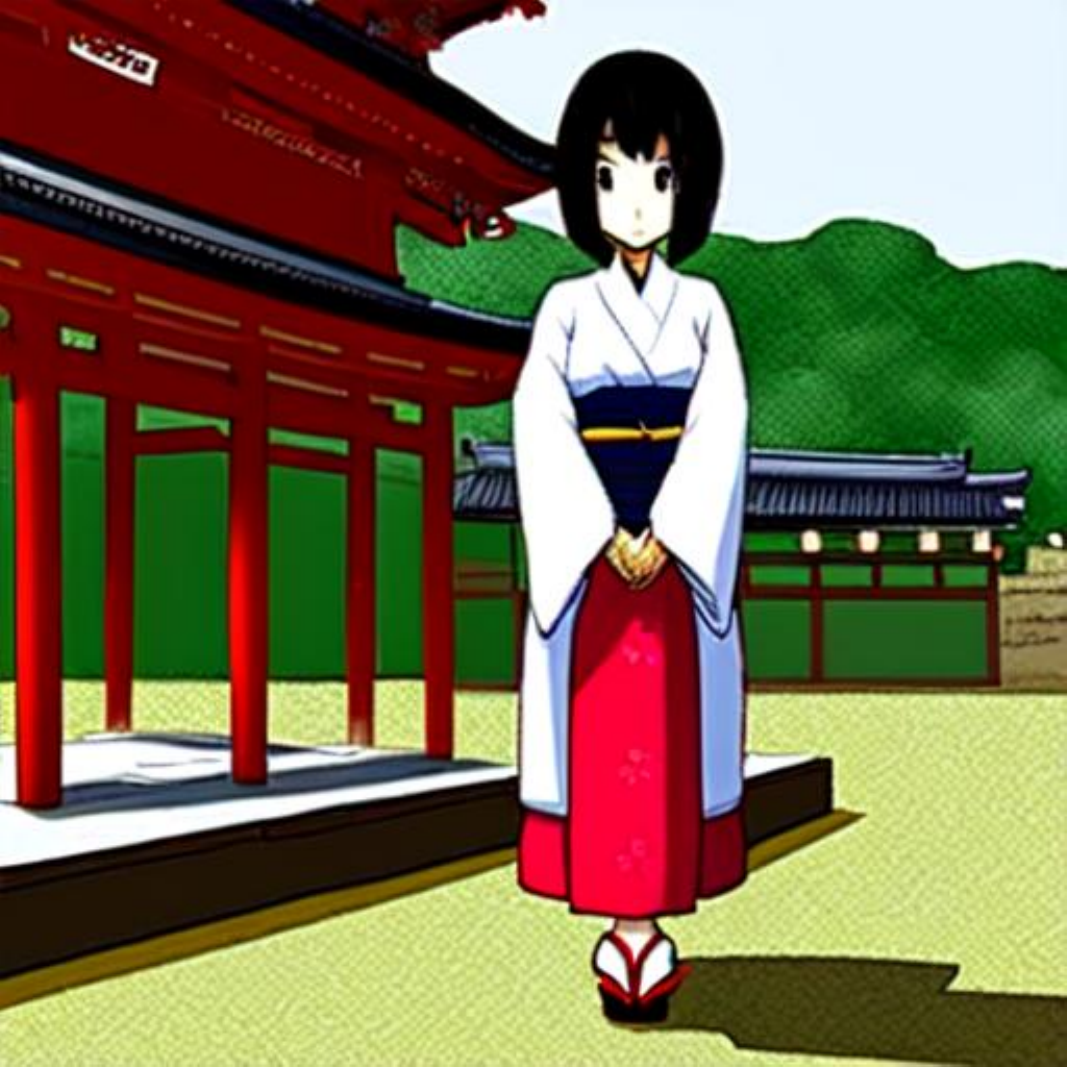}} & 
        \noindent\parbox[c]{0.14\columnwidth}{\includegraphics[width=0.14\columnwidth]{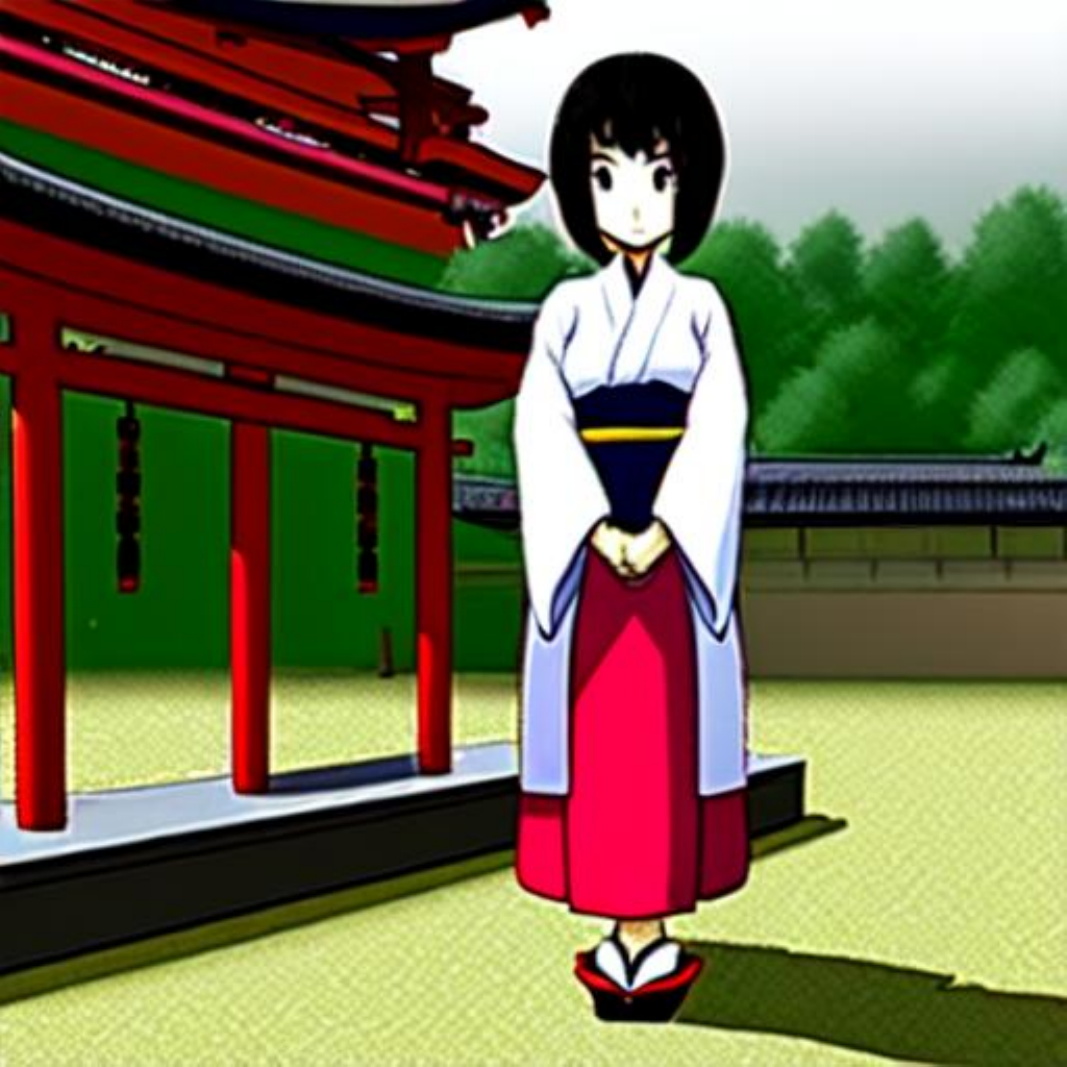}} \\
    \end{tabu}
    \caption{Comparison of samples generated from Waifu Diffusion V1.4 using PLMS4 with HB $\beta$ under various sampling steps and guidance scale $s$. Specifically, we employ $\beta = 0.8$ for $s = 7.5$, $\beta = 0.7$ for $s = 15$, and $\beta = 0.6$ for $s = 22.5$ to account for the varying degrees of artifact manifestation associated with each guidance scale.}
    \label{fig:scale_step_waifu_hb}
\end{figure}


\pagebreak

\tabulinesep=1pt
\begin{figure}
    \centering
    \begin{tabu} to \textwidth {@{}l@{\hspace{5pt}}c@{\hspace{2pt}}c@{\hspace{2pt}}c@{\hspace{4pt}}c@{\hspace{2pt}}c@{\hspace{2pt}}c@{}}
        & \multicolumn{3}{c}{\shortstack{\scriptsize "A post-apocalyptic world with ruined \\ \scriptsize buildings, overgrown vegetation, and a red sky"}}
        & \multicolumn{3}{c}{\shortstack{\scriptsize "A girl standing in a park in \\ \scriptsize Japanese animation style"}} \\

        & \multicolumn{1}{c}{\shortstack{\scriptsize $s = 7.5$}}
        & \multicolumn{1}{c}{\shortstack{\scriptsize $s = 15$}}
        & \multicolumn{1}{c}{\shortstack{\scriptsize $s = 22.5$}}
        & \multicolumn{1}{c}{\shortstack{\scriptsize $s = 7.5$}}
        & \multicolumn{1}{c}{\shortstack{\scriptsize $s = 15$}}
        & \multicolumn{1}{c}{\shortstack{\scriptsize $s = 22.5$}}
        \\
        
        \shortstack[l]{\tiny 10 steps} &
        \noindent\parbox[c]{0.14\columnwidth}{\includegraphics[width=0.14\columnwidth]{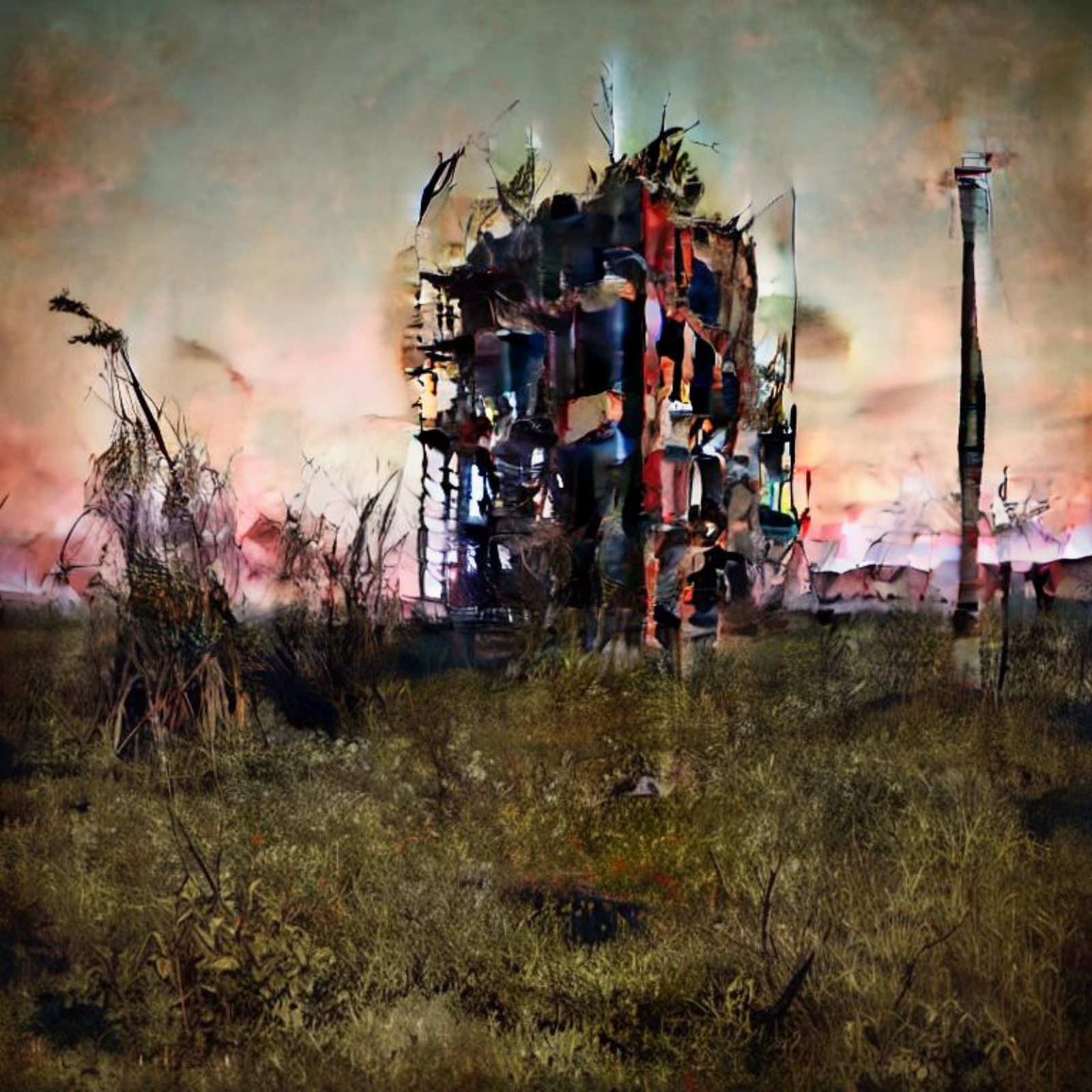}} & 
        \noindent\parbox[c]{0.14\columnwidth}{\includegraphics[width=0.14\columnwidth]{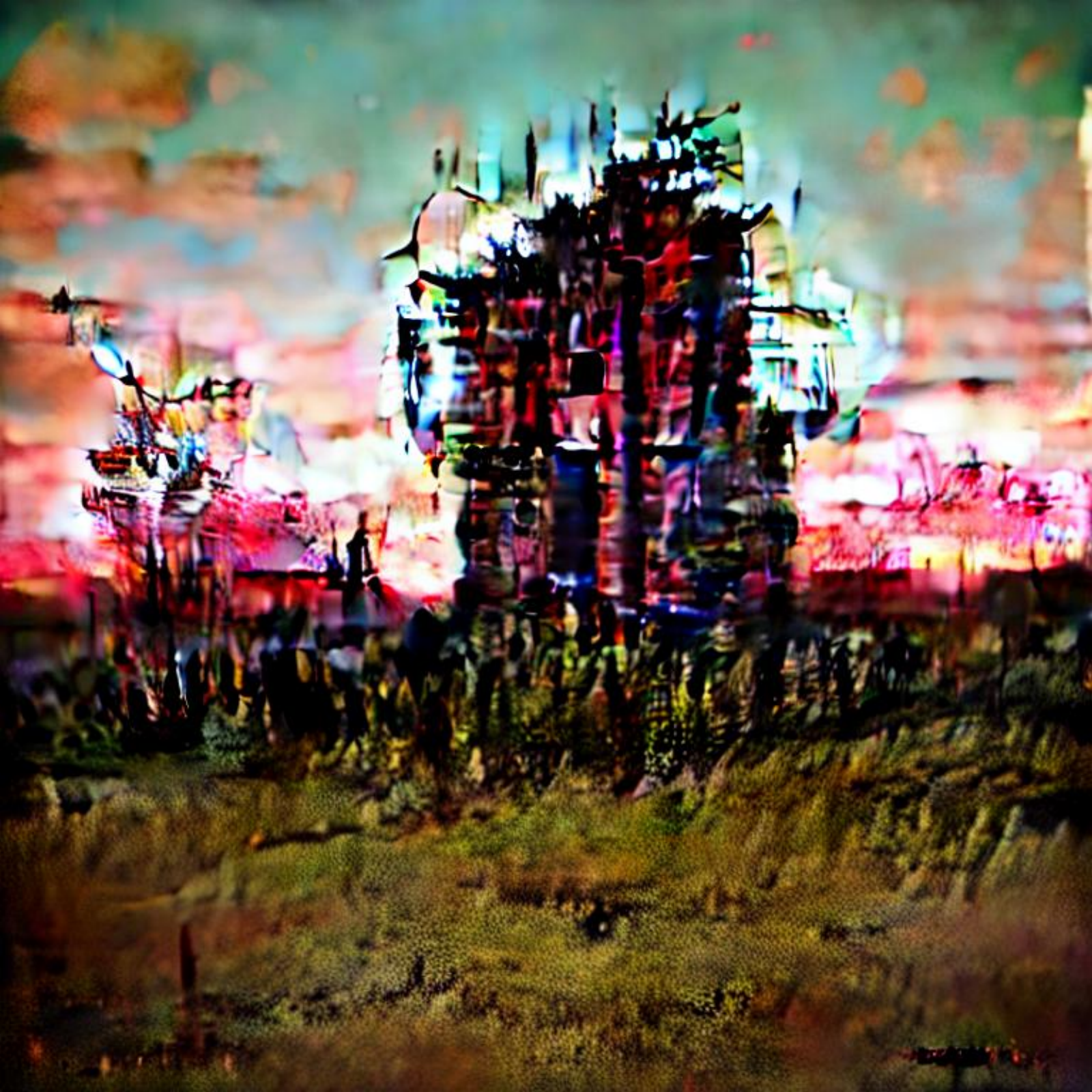}} & 
        \noindent\parbox[c]{0.14\columnwidth}{\includegraphics[width=0.14\columnwidth]{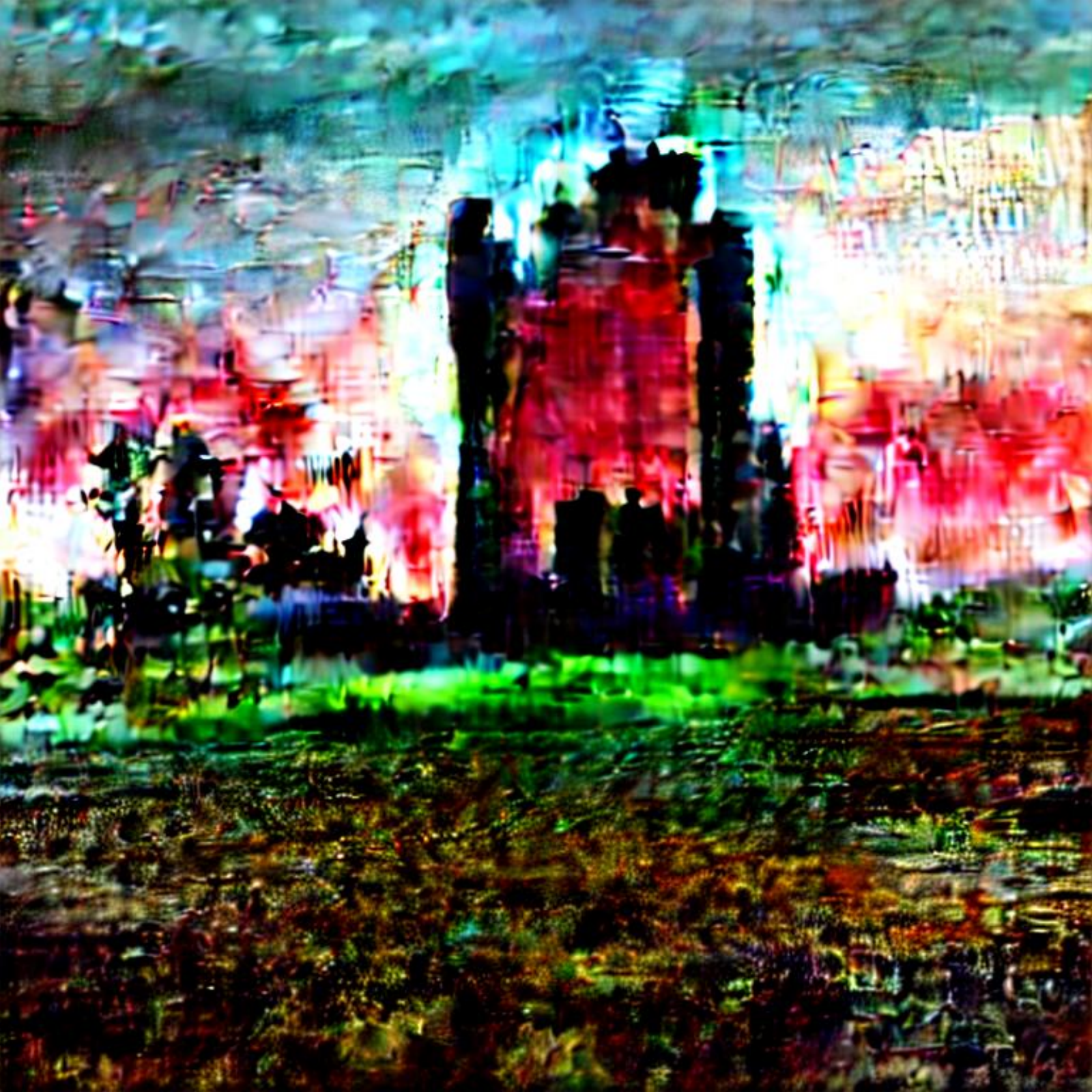}} & 
        \noindent\parbox[c]{0.14\columnwidth}{\includegraphics[width=0.14\columnwidth]{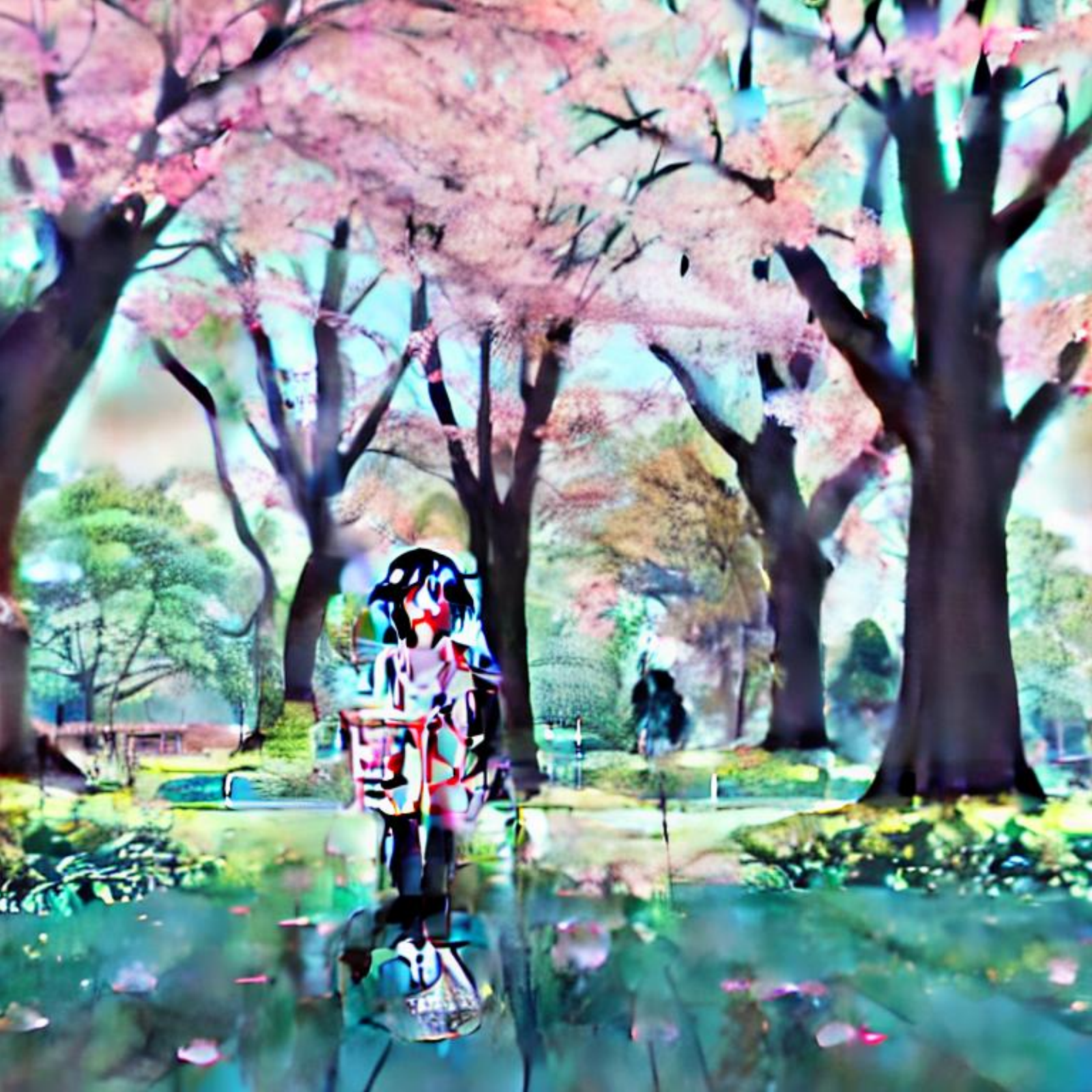}} & 
        \noindent\parbox[c]{0.14\columnwidth}{\includegraphics[width=0.14\columnwidth]{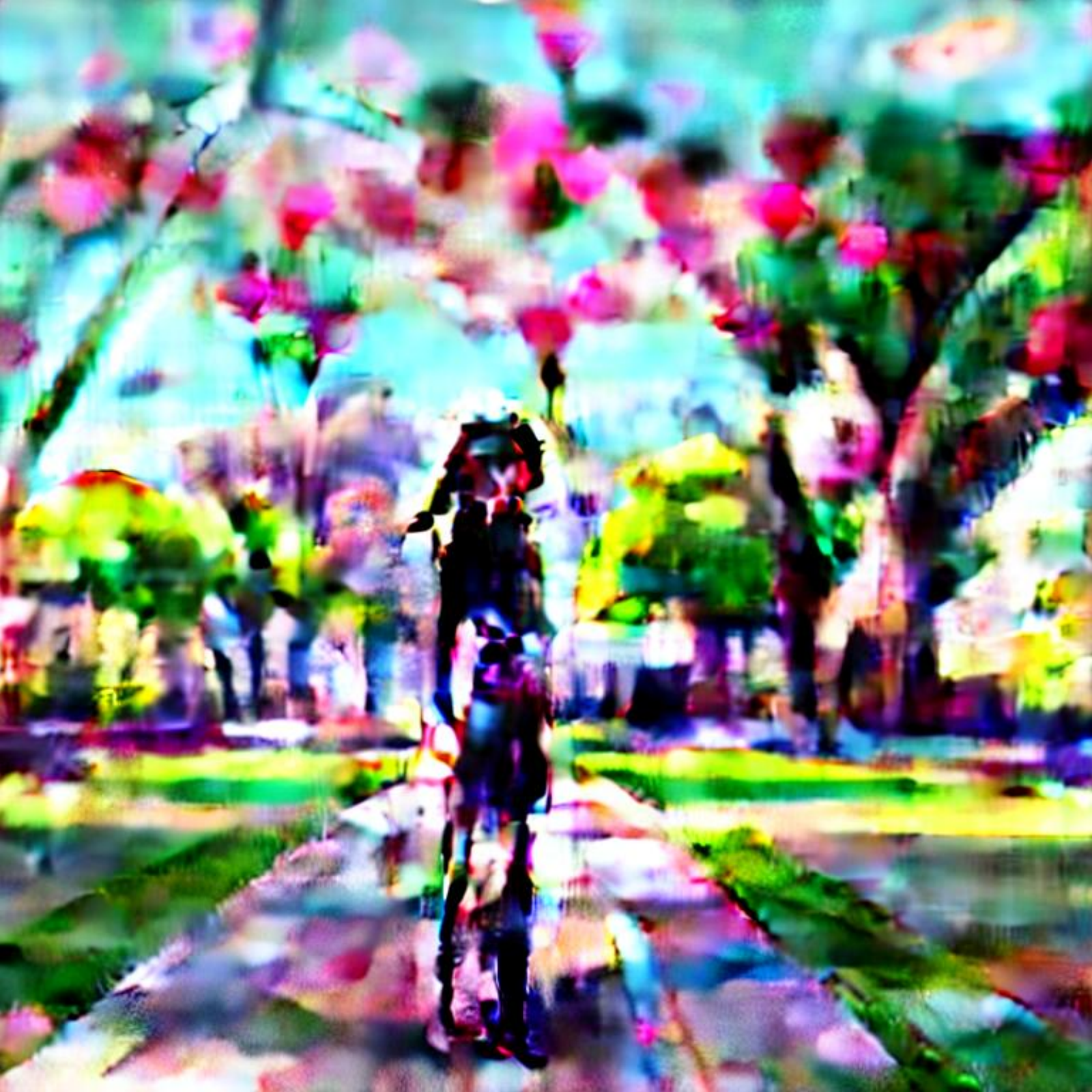}} & 
        \noindent\parbox[c]{0.14\columnwidth}{\includegraphics[width=0.14\columnwidth]{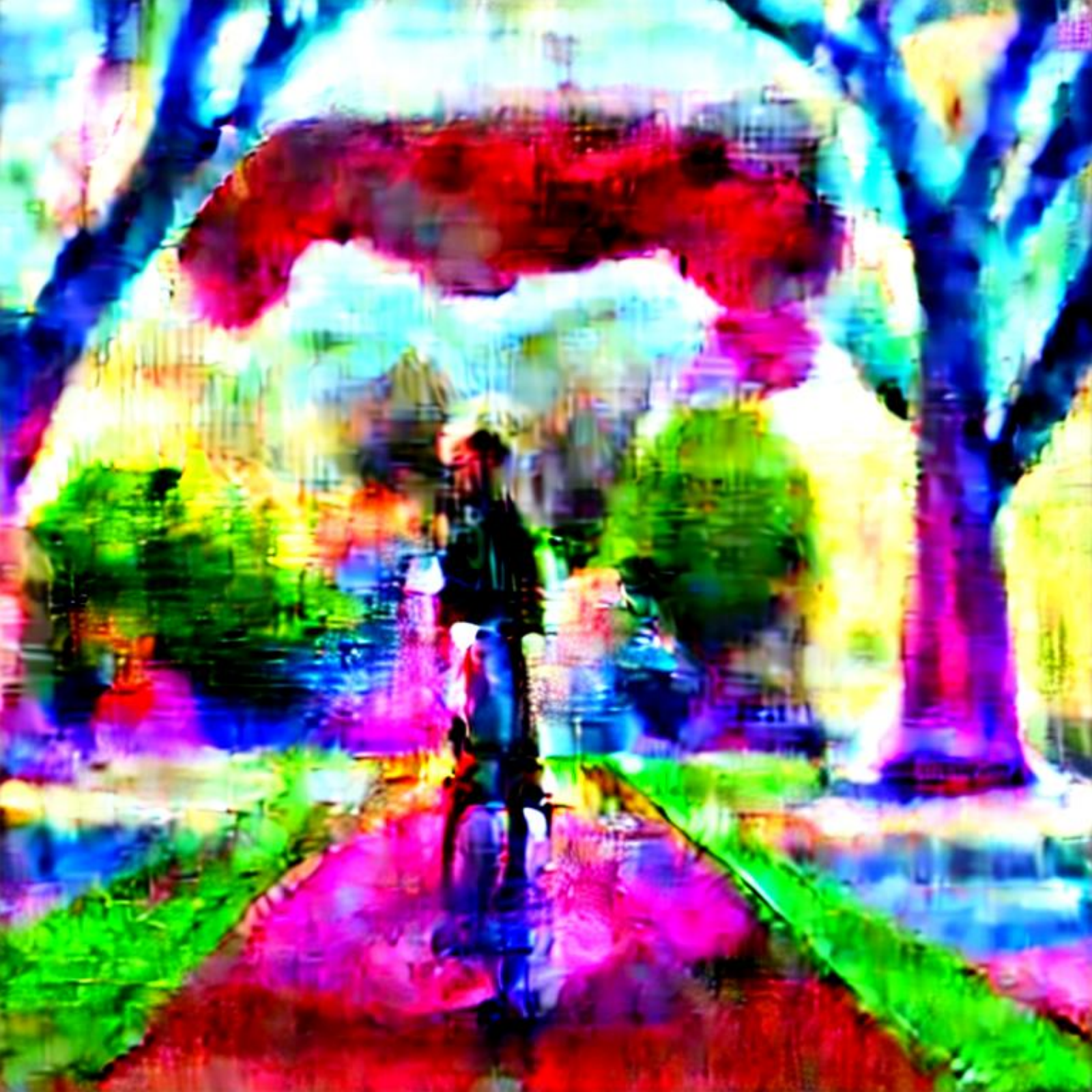}} \\

        \shortstack[l]{\tiny 15 steps} &
        \noindent\parbox[c]{0.14\columnwidth}{\includegraphics[width=0.14\columnwidth]{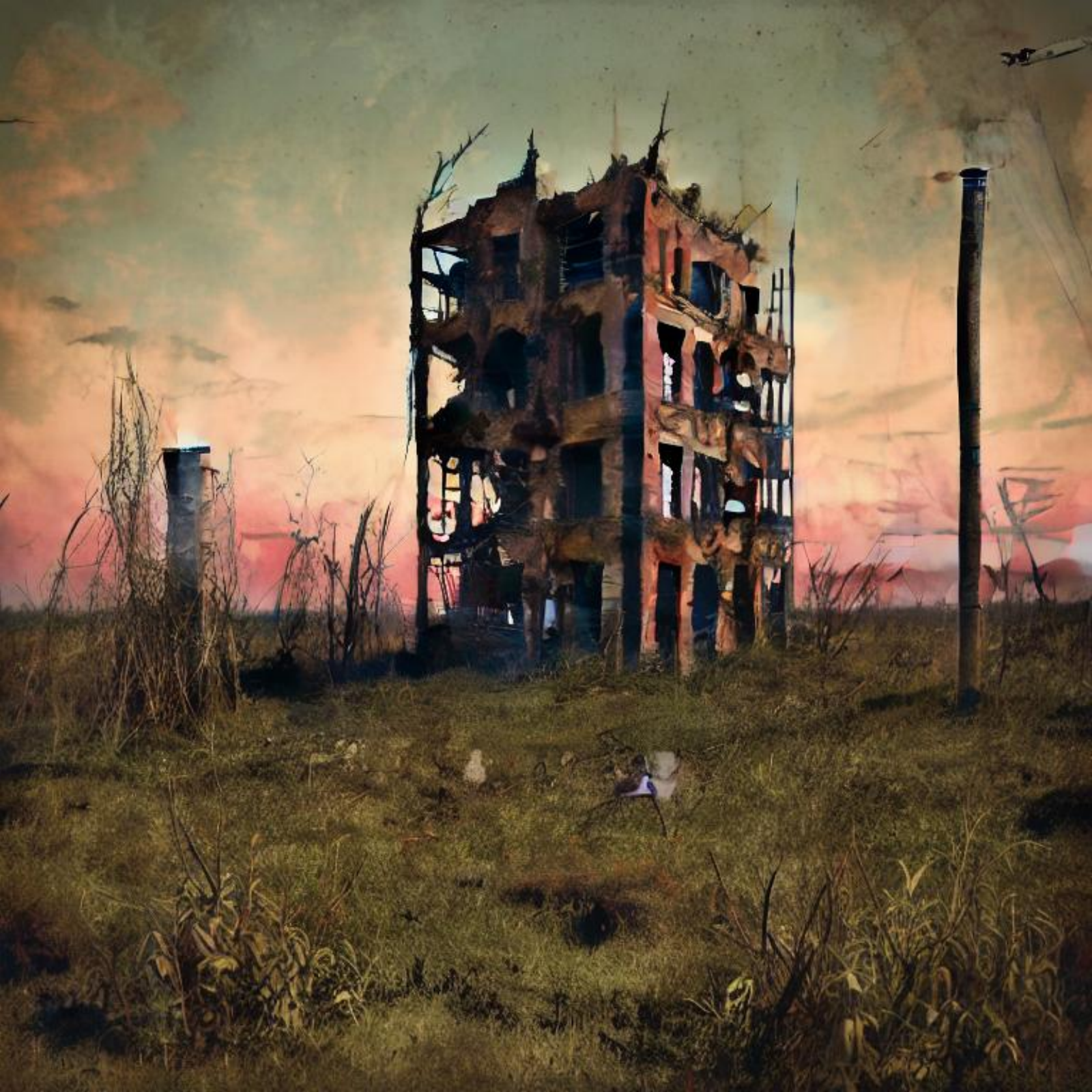}} & 
        \noindent\parbox[c]{0.14\columnwidth}{\includegraphics[width=0.14\columnwidth]{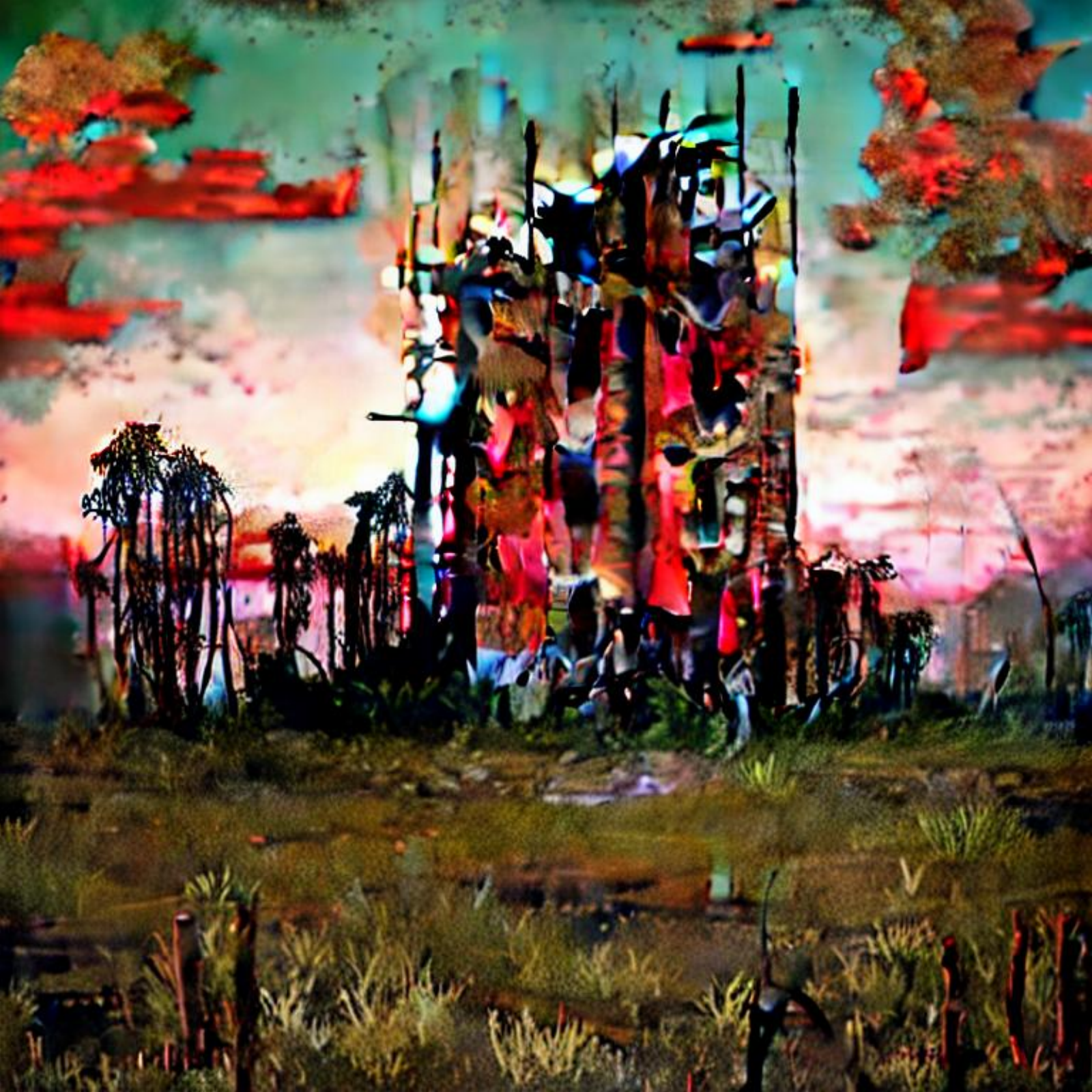}} & 
        \noindent\parbox[c]{0.14\columnwidth}{\includegraphics[width=0.14\columnwidth]{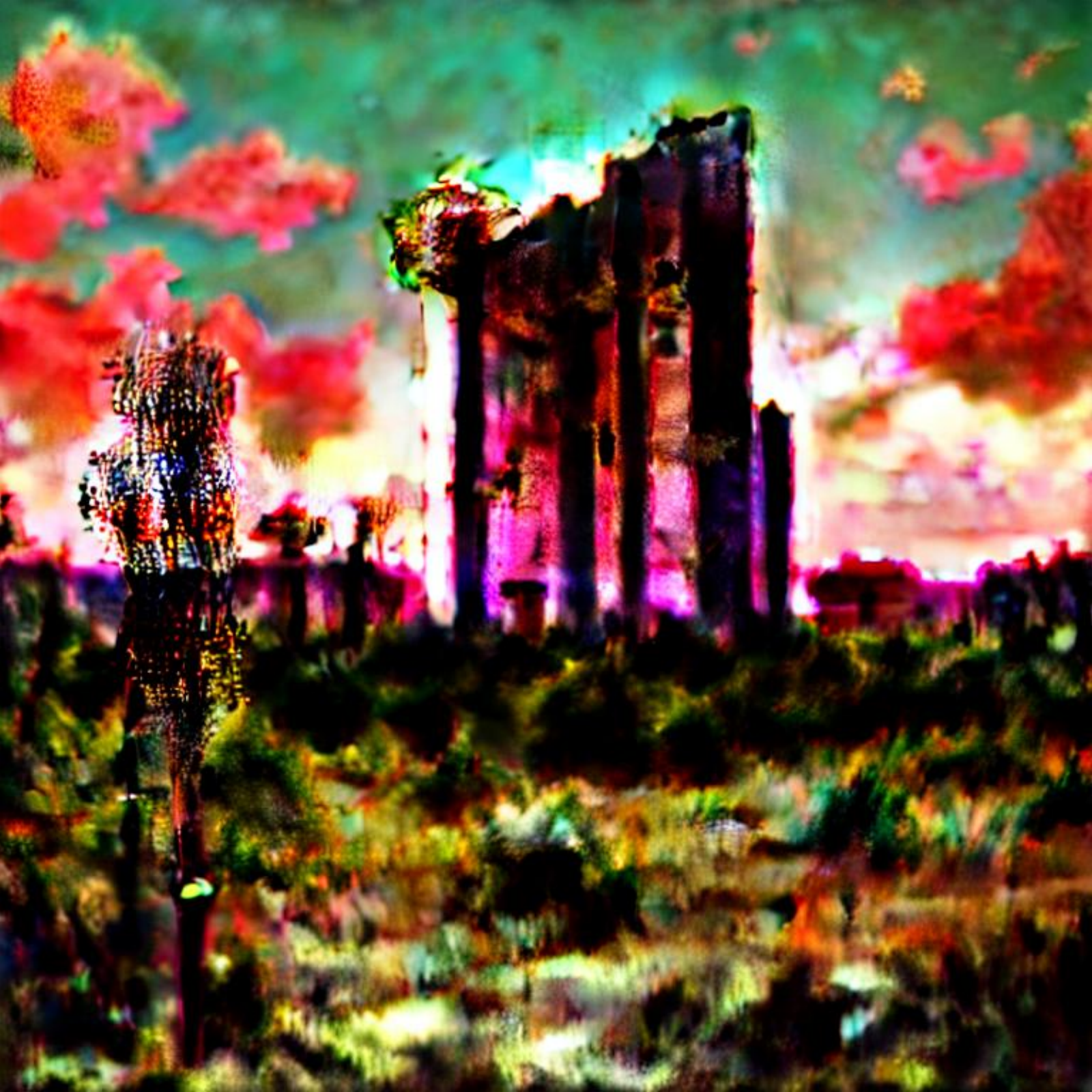}} & 
        \noindent\parbox[c]{0.14\columnwidth}{\includegraphics[width=0.14\columnwidth]{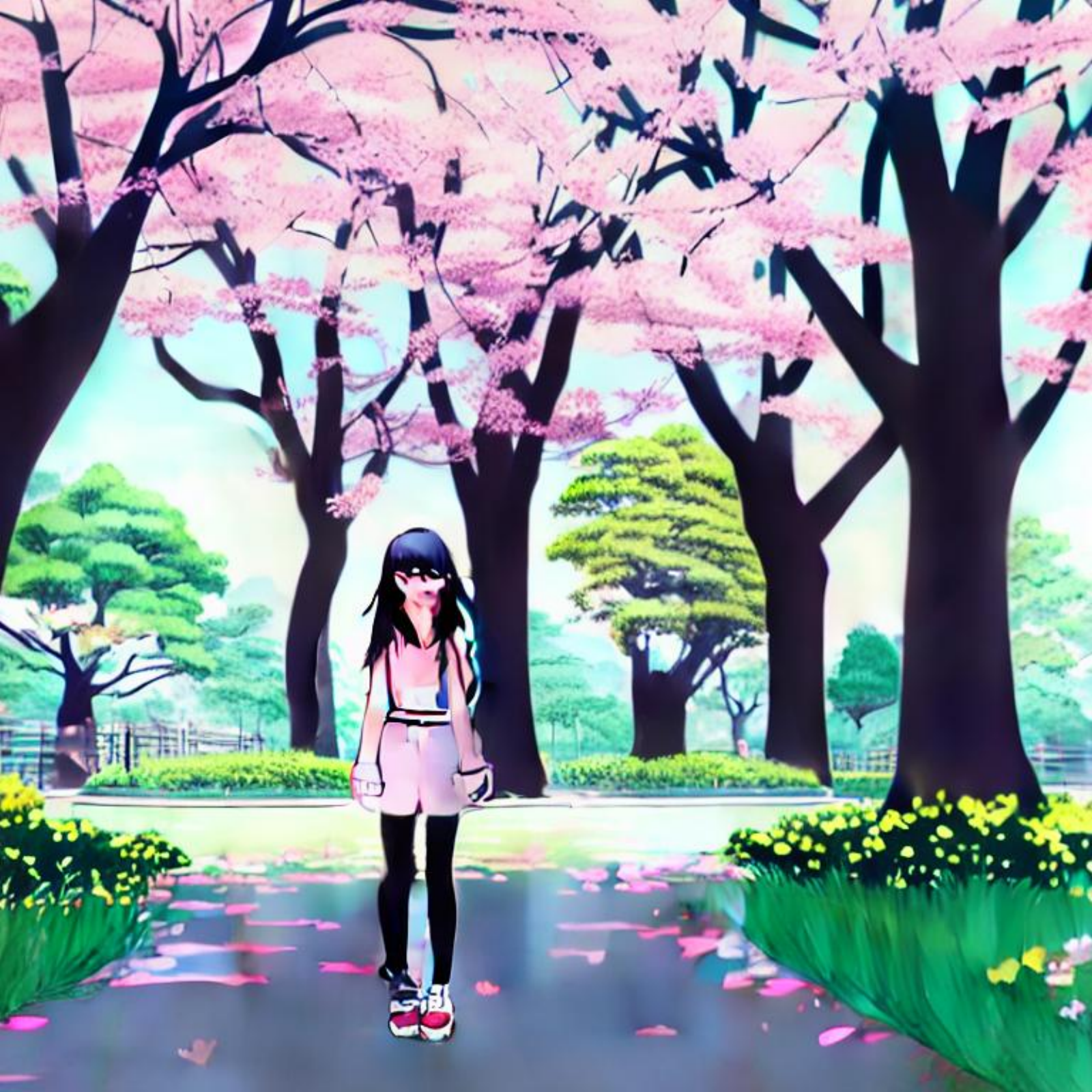}} & 
        \noindent\parbox[c]{0.14\columnwidth}{\includegraphics[width=0.14\columnwidth]{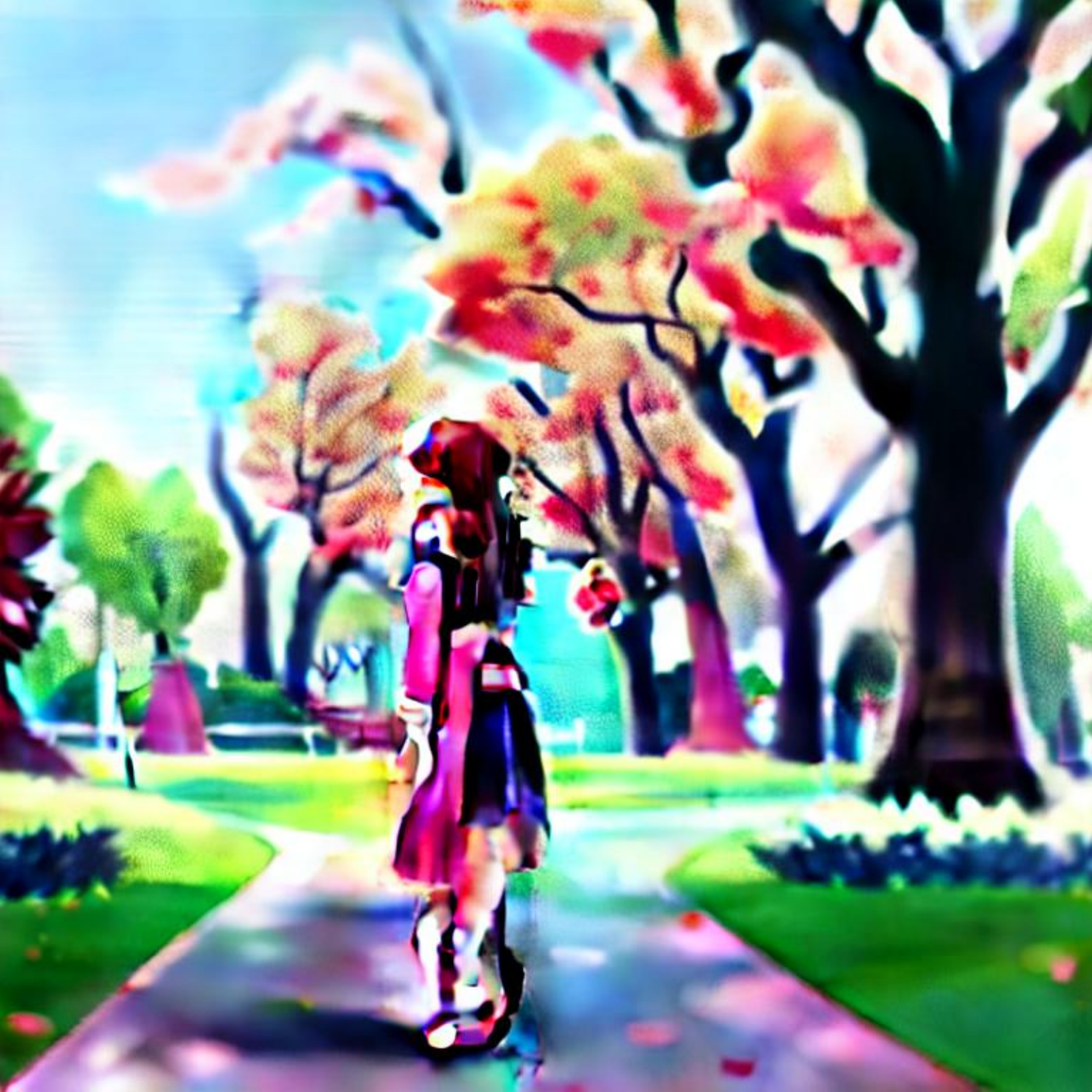}} & 
        \noindent\parbox[c]{0.14\columnwidth}{\includegraphics[width=0.14\columnwidth]{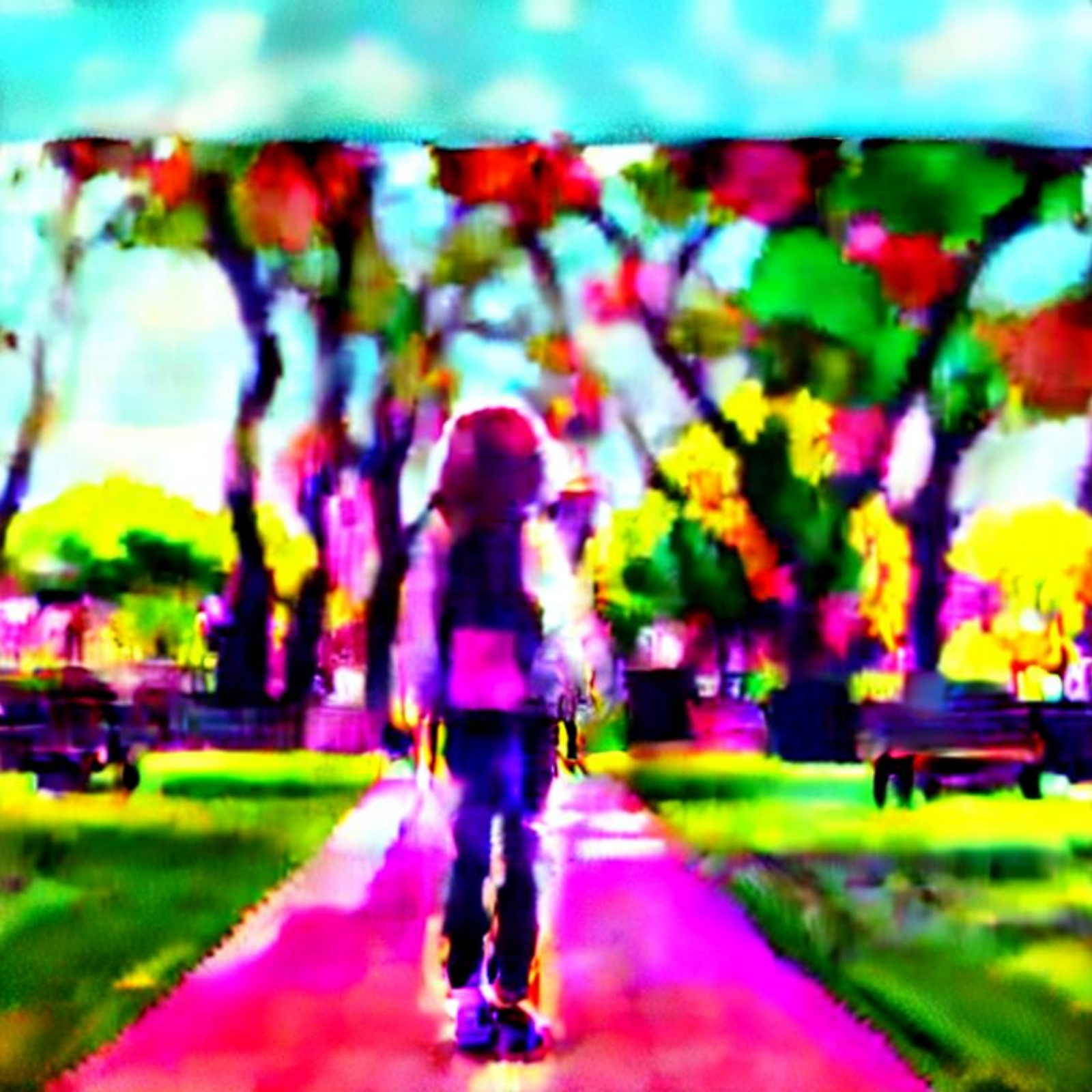}} \\

        \shortstack[l]{\tiny 20 steps} &
        \noindent\parbox[c]{0.14\columnwidth}{\includegraphics[width=0.14\columnwidth]{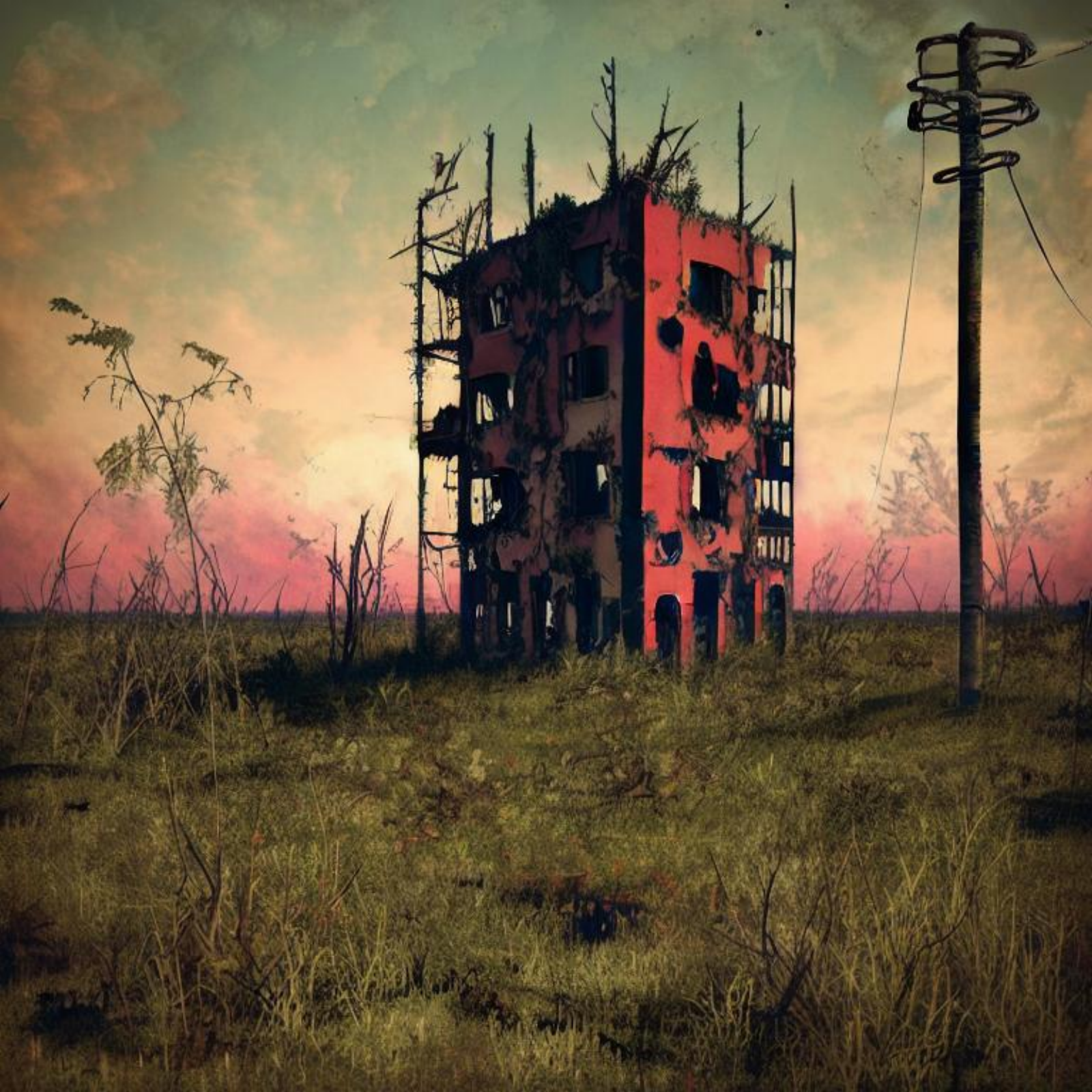}} & 
        \noindent\parbox[c]{0.14\columnwidth}{\includegraphics[width=0.14\columnwidth]{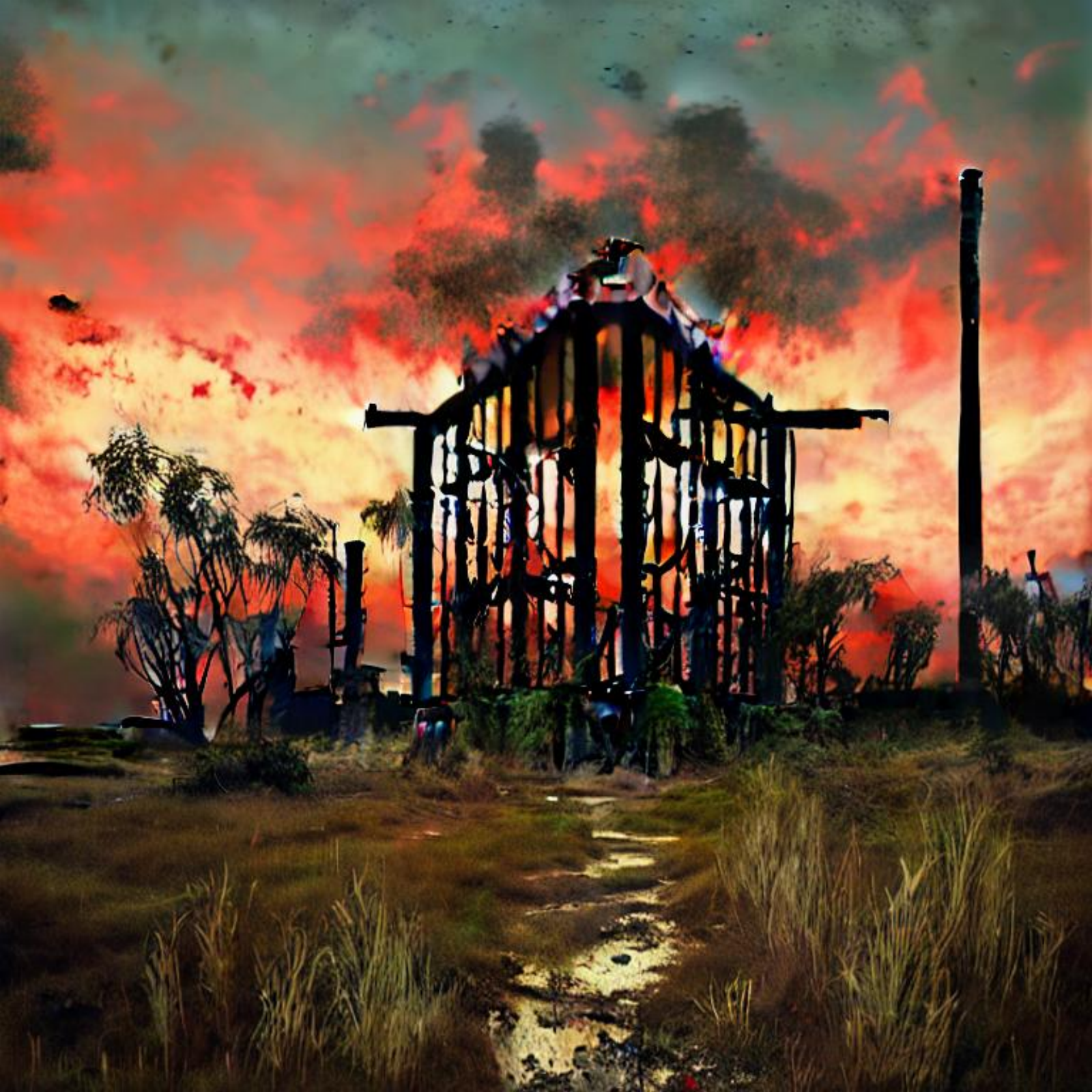}} & 
        \noindent\parbox[c]{0.14\columnwidth}{\includegraphics[width=0.14\columnwidth]{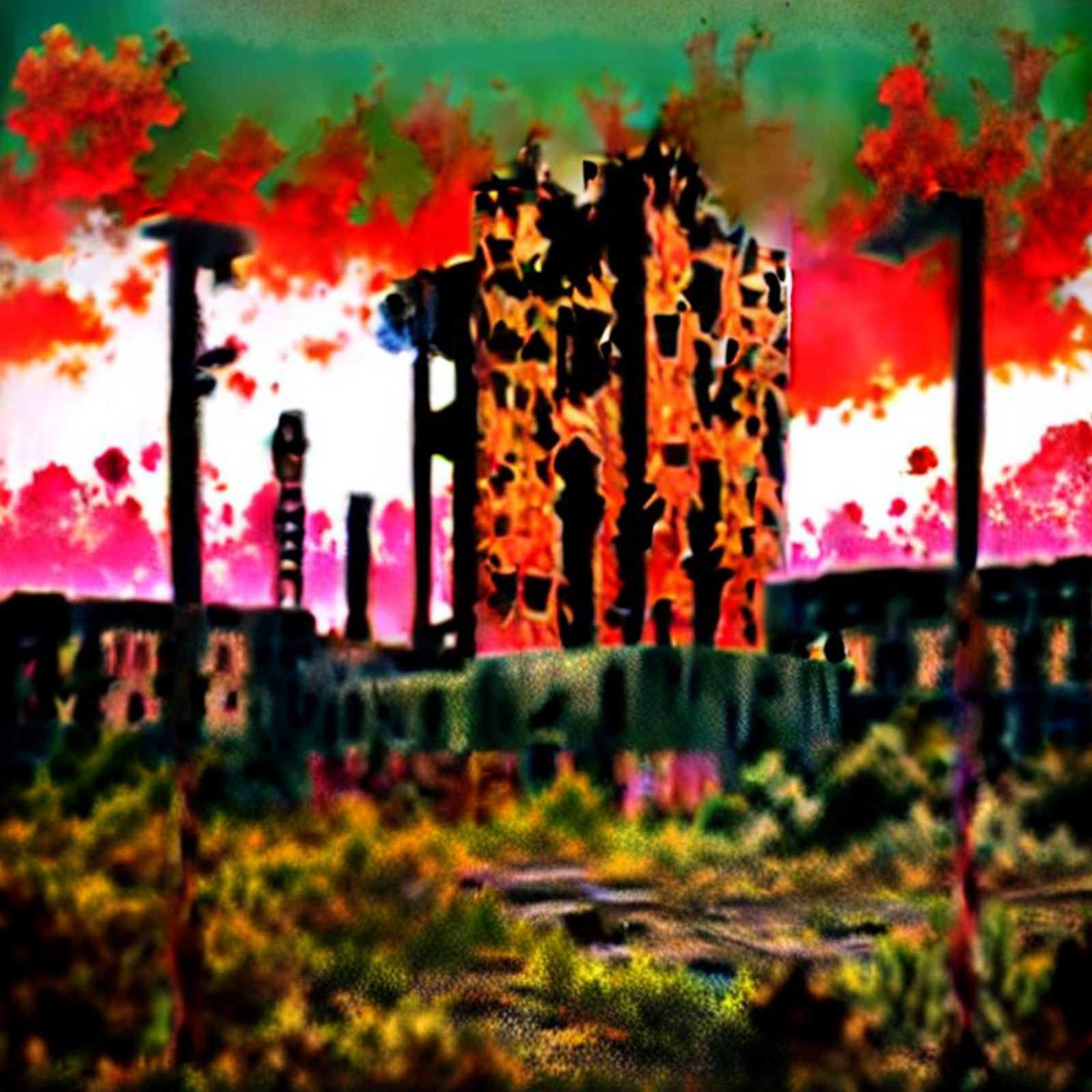}} & 
        \noindent\parbox[c]{0.14\columnwidth}{\includegraphics[width=0.14\columnwidth]{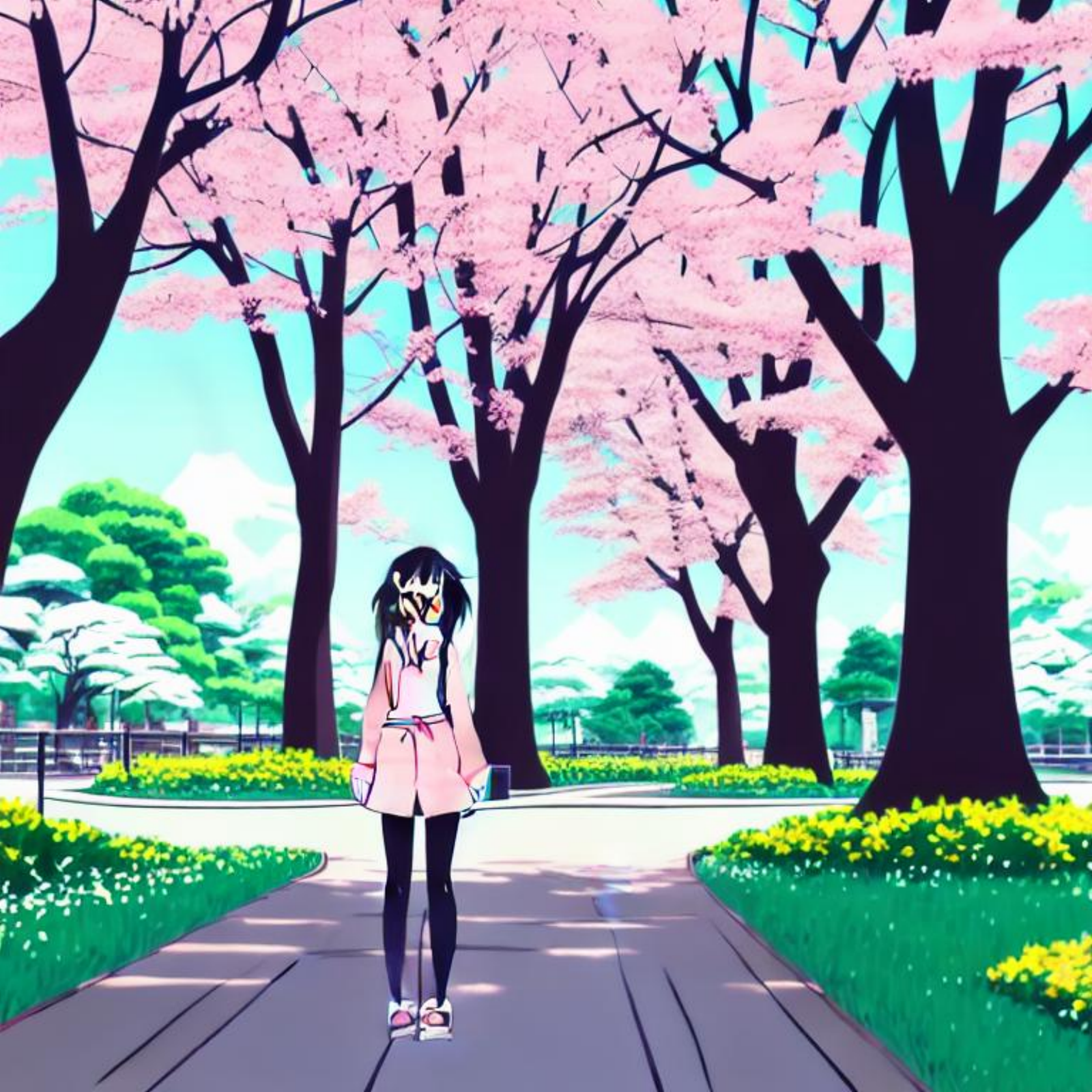}} & 
        \noindent\parbox[c]{0.14\columnwidth}{\includegraphics[width=0.14\columnwidth]{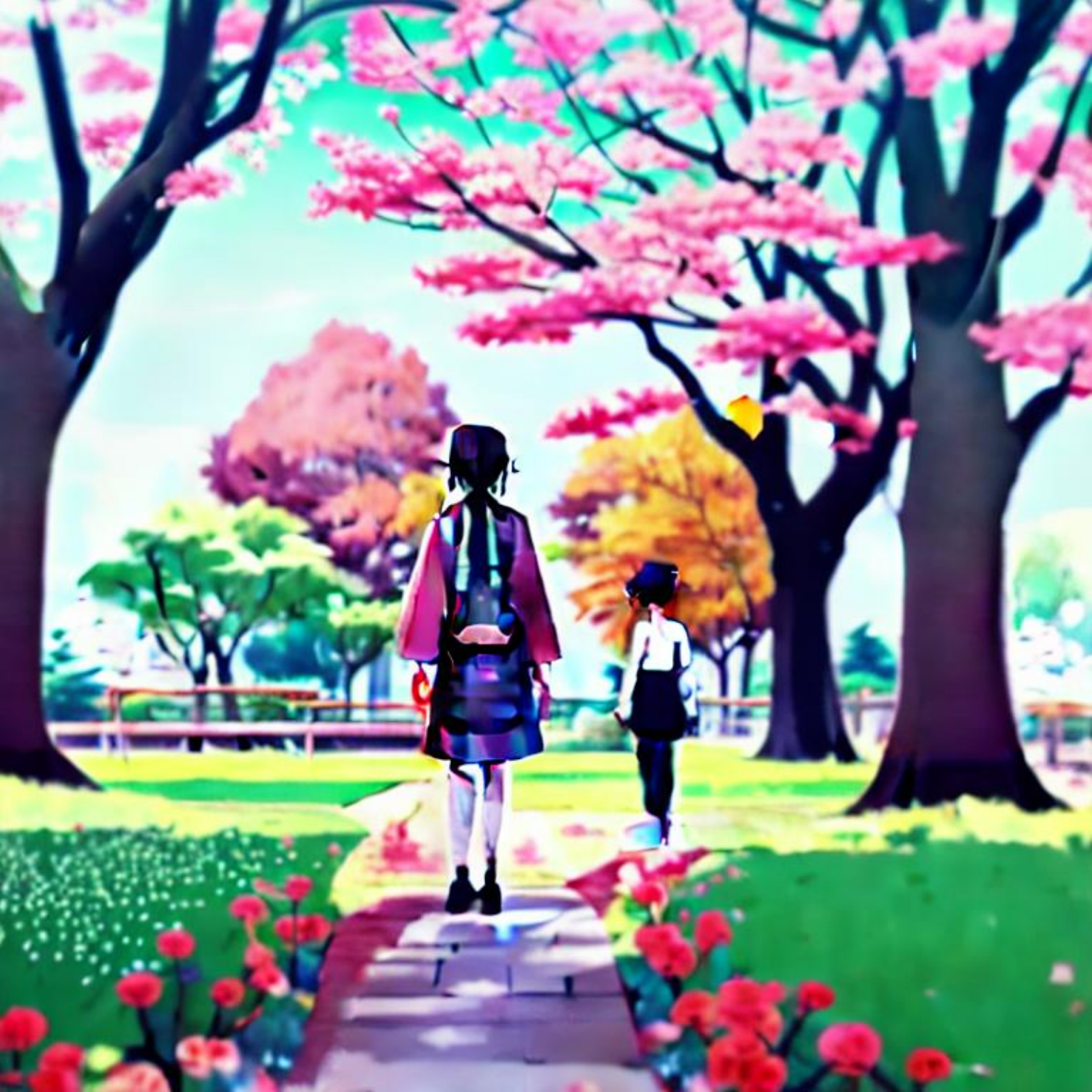}} & 
        \noindent\parbox[c]{0.14\columnwidth}{\includegraphics[width=0.14\columnwidth]{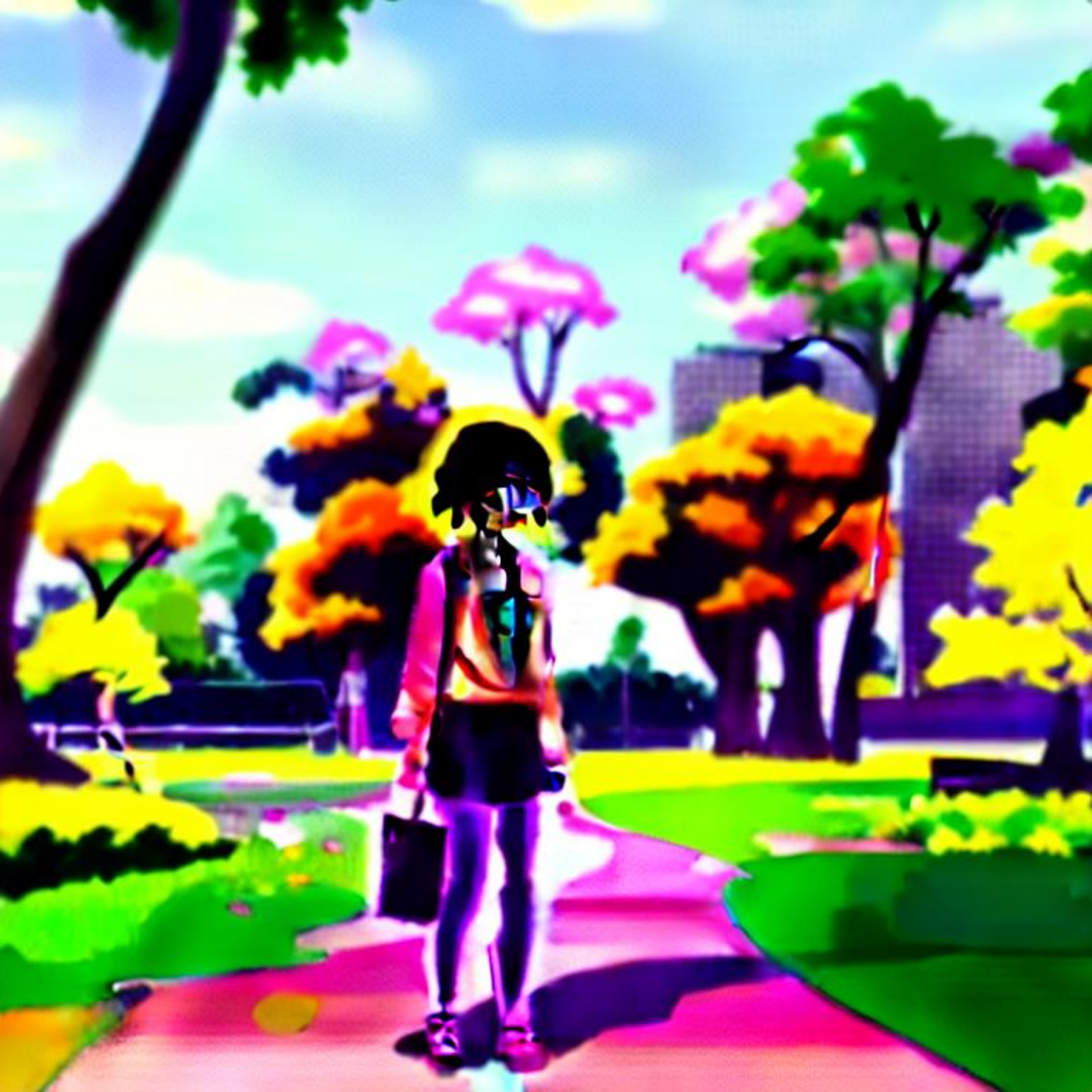}} \\

        \shortstack[l]{\tiny 40 steps} &
        \noindent\parbox[c]{0.14\columnwidth}{\includegraphics[width=0.14\columnwidth]{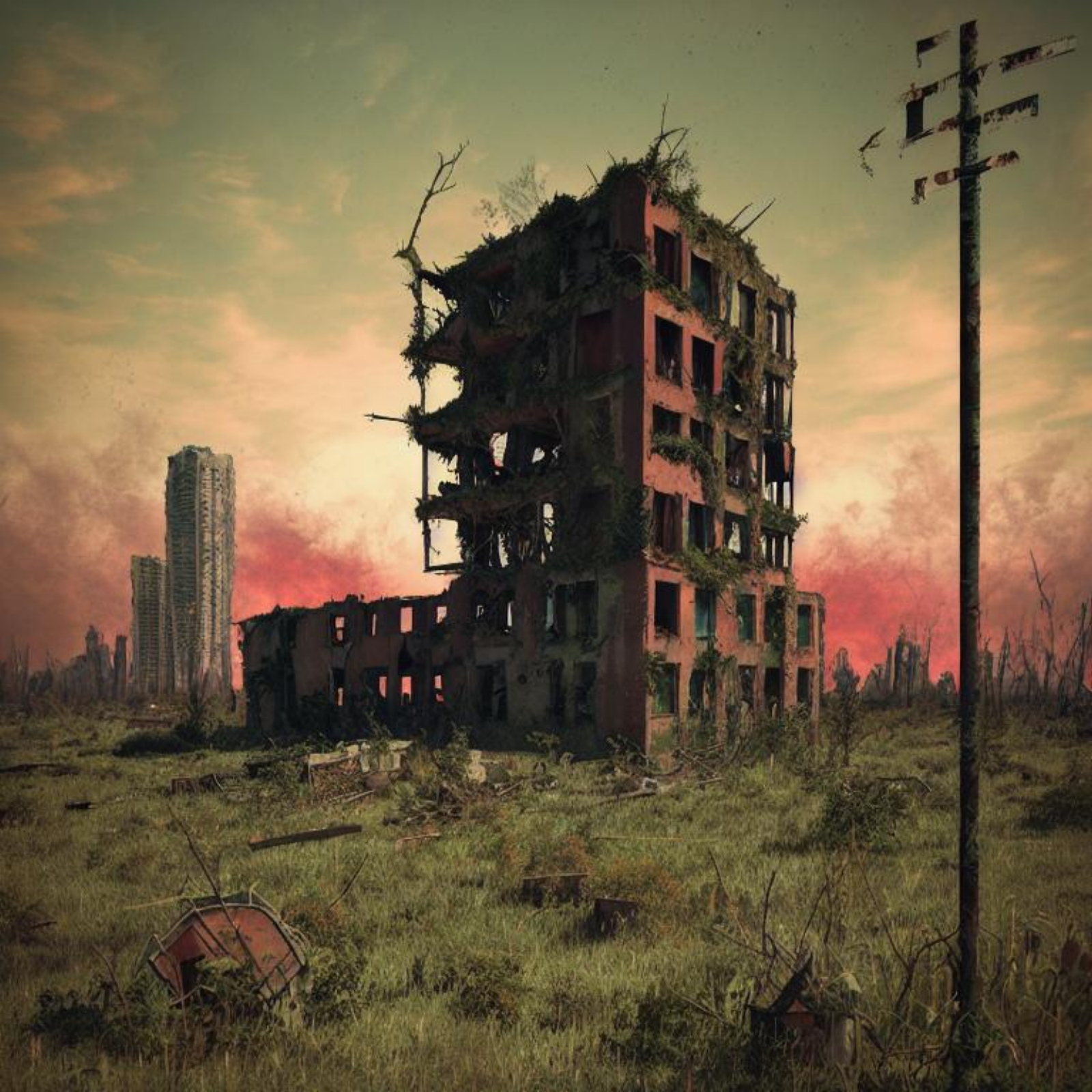}} & 
        \noindent\parbox[c]{0.14\columnwidth}{\includegraphics[width=0.14\columnwidth]{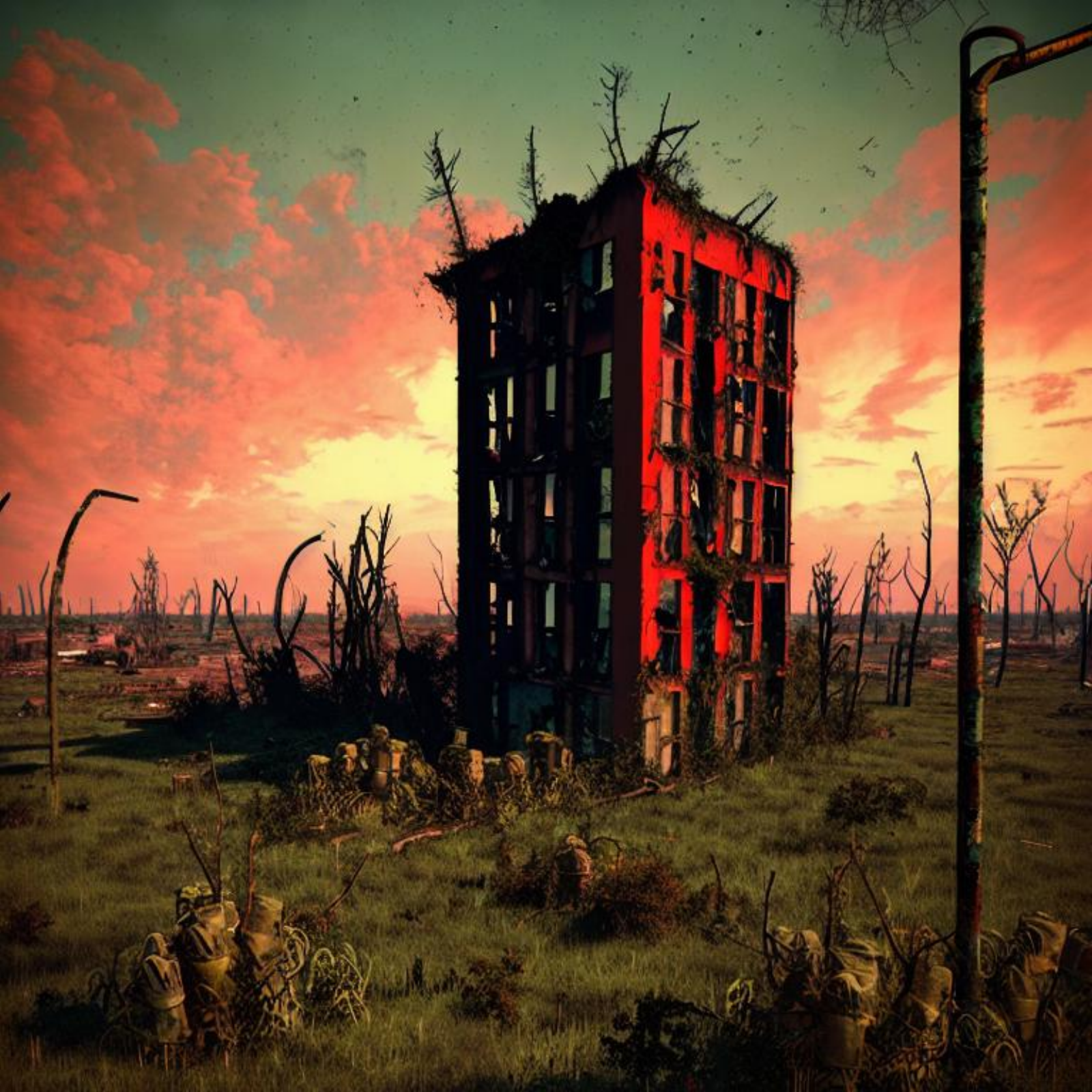}} & 
        \noindent\parbox[c]{0.14\columnwidth}{\includegraphics[width=0.14\columnwidth]{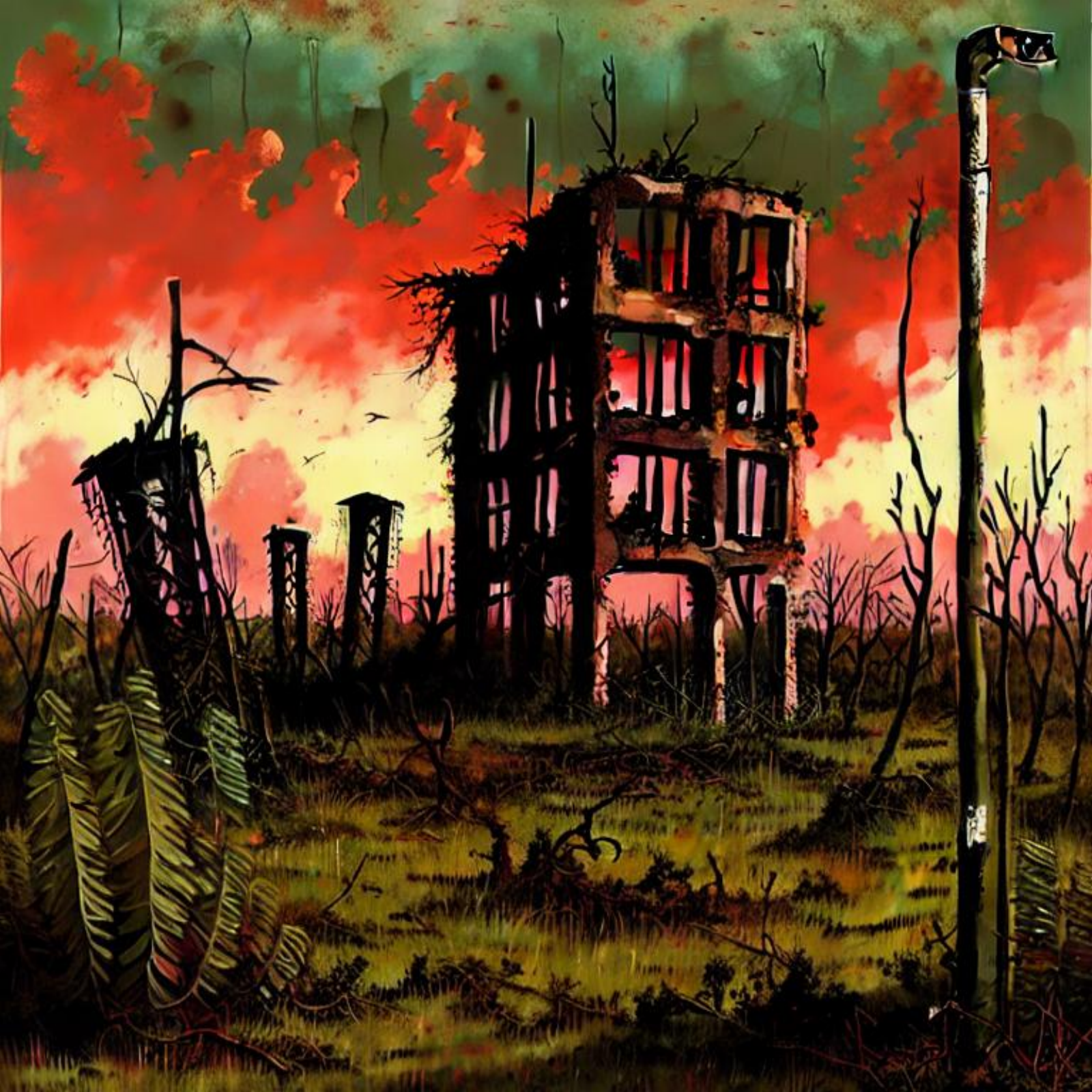}} & 
        \noindent\parbox[c]{0.14\columnwidth}{\includegraphics[width=0.14\columnwidth]{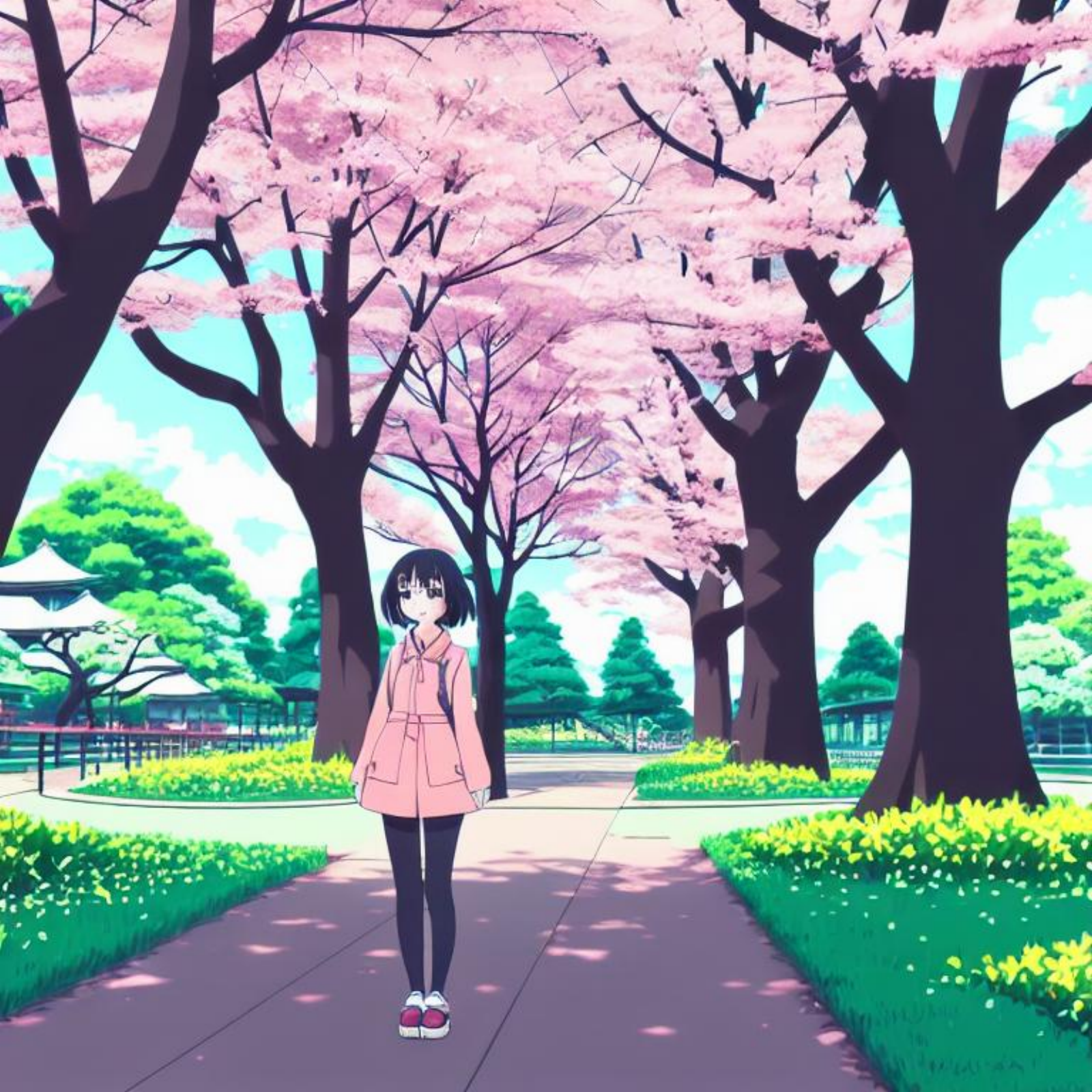}} & 
        \noindent\parbox[c]{0.14\columnwidth}{\includegraphics[width=0.14\columnwidth]{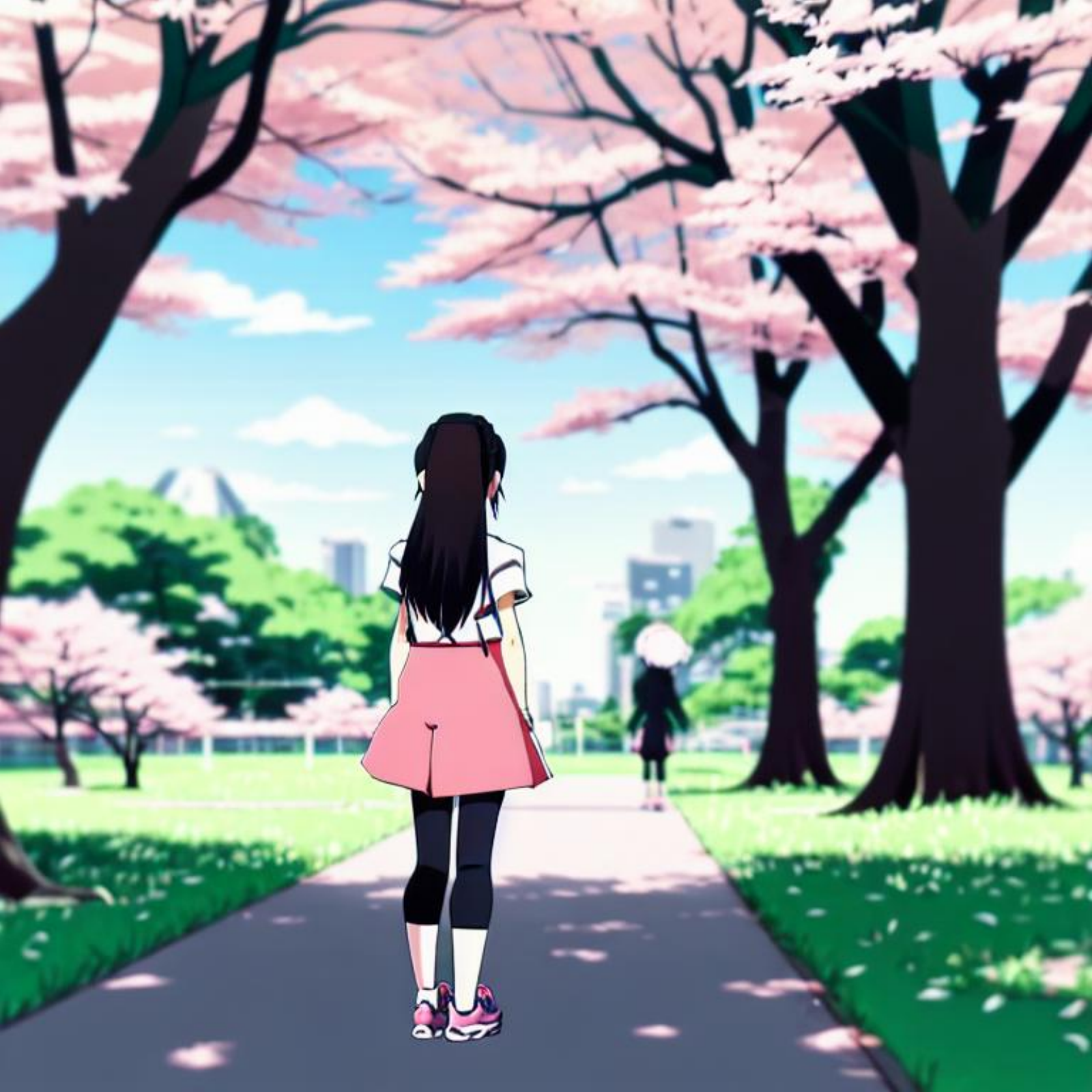}} & 
        \noindent\parbox[c]{0.14\columnwidth}{\includegraphics[width=0.14\columnwidth]{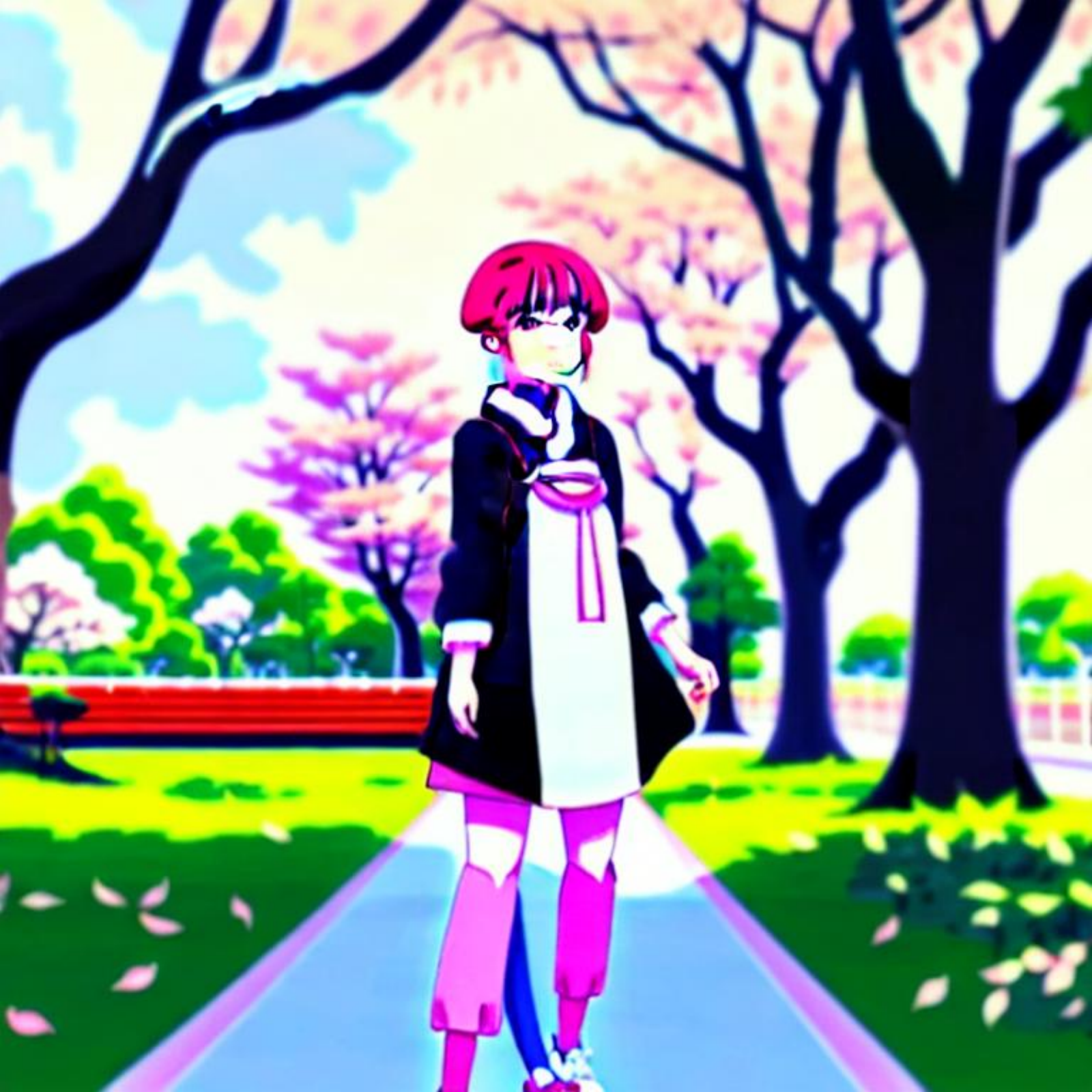}} \\
    \end{tabu}
    \caption{Comparison of samples generated from Dreamlike Photoreal 2.0 \protect\footnotemark using PLMS4 \cite{liu2022pseudo} with different sampling steps and guidance scale.}
    \label{fig:scale_step_dreamlike}
\end{figure}

\footnotetext{\url{https://huggingface.co/dreamlike-art/dreamlike-photoreal-2.0}}


\tabulinesep=1pt
\begin{figure}
    \centering
    \begin{tabu} to \textwidth {@{}l@{\hspace{5pt}}c@{\hspace{2pt}}c@{\hspace{2pt}}c@{\hspace{4pt}}c@{\hspace{2pt}}c@{\hspace{2pt}}c@{}}
        & \multicolumn{3}{c}{\shortstack{\scriptsize "A post-apocalyptic world with ruined \\ \scriptsize buildings, overgrown vegetation, and a red sky"}}
        & \multicolumn{3}{c}{\shortstack{\scriptsize "A girl standing in a park in \\ \scriptsize Japanese animation style"}} \\

        & \multicolumn{1}{c}{\shortstack{\scriptsize $s = 7.5$}}
        & \multicolumn{1}{c}{\shortstack{\scriptsize $s = 15$}}
        & \multicolumn{1}{c}{\shortstack{\scriptsize $s = 22.5$}}
        & \multicolumn{1}{c}{\shortstack{\scriptsize $s = 7.5$}}
        & \multicolumn{1}{c}{\shortstack{\scriptsize $s = 15$}}
        & \multicolumn{1}{c}{\shortstack{\scriptsize $s = 22.5$}}
        \\
        
        \shortstack[l]{\tiny 10 steps} &
        \noindent\parbox[c]{0.14\columnwidth}{\includegraphics[width=0.14\columnwidth]{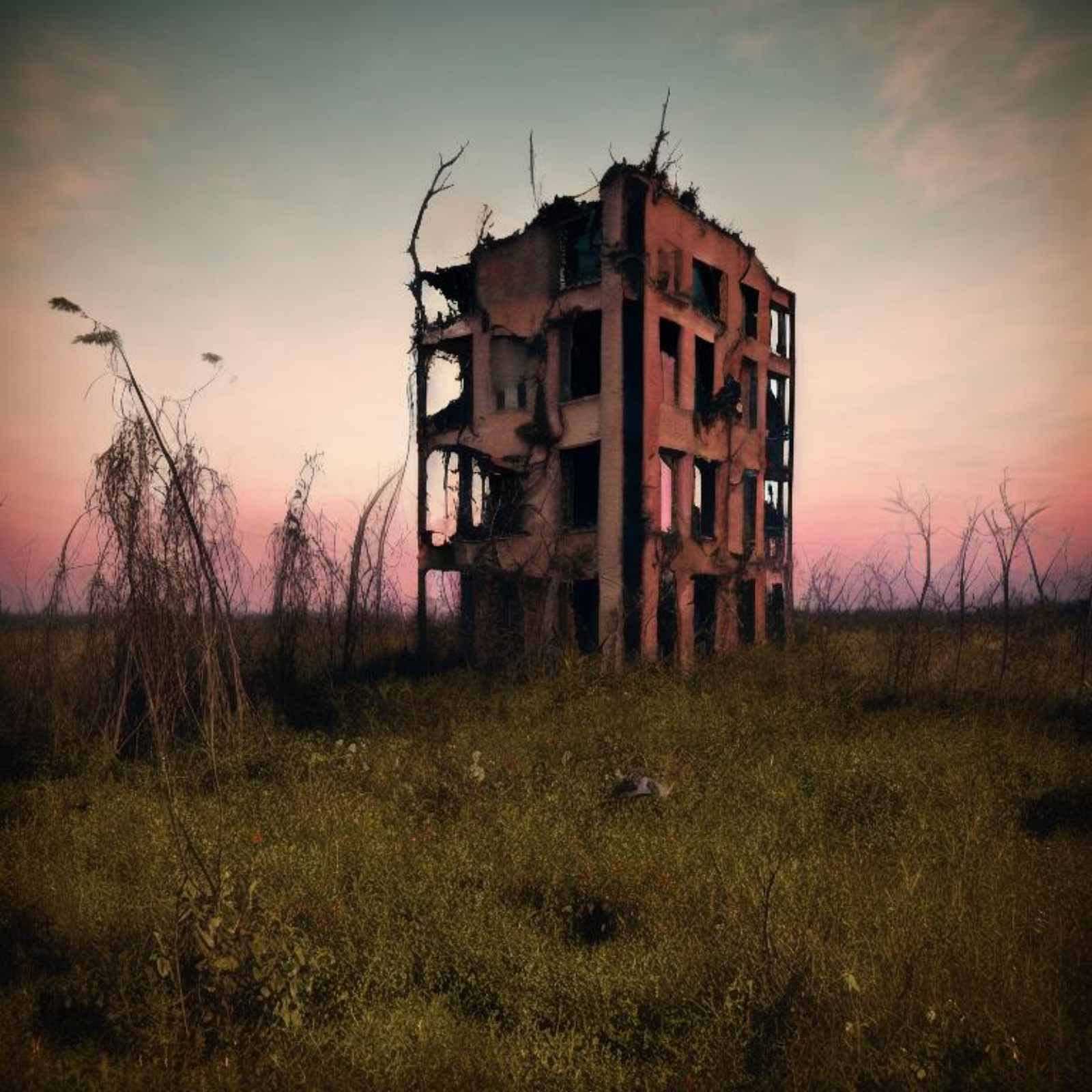}} & 
        \noindent\parbox[c]{0.14\columnwidth}{\includegraphics[width=0.14\columnwidth]{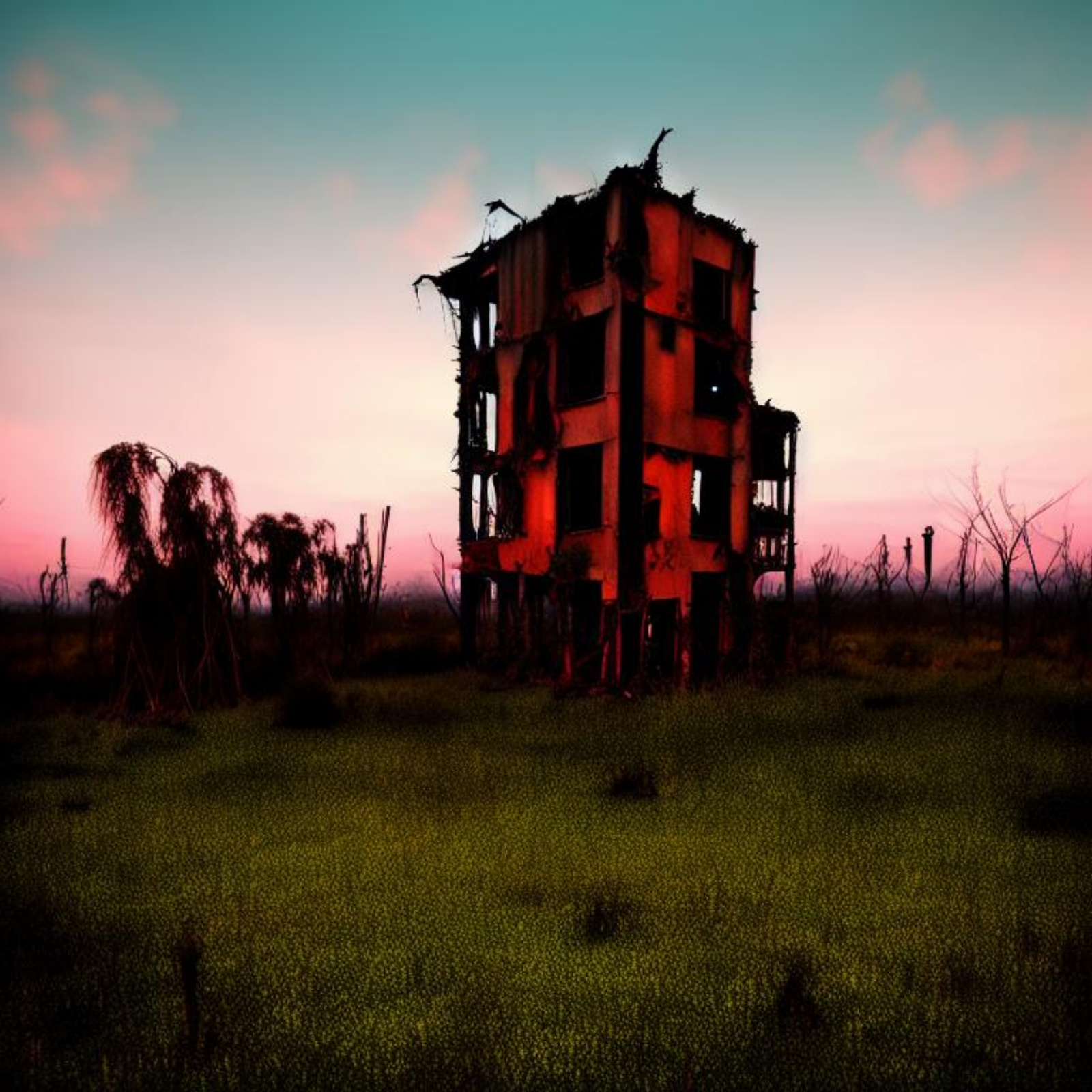}} & 
        \noindent\parbox[c]{0.14\columnwidth}{\includegraphics[width=0.14\columnwidth]{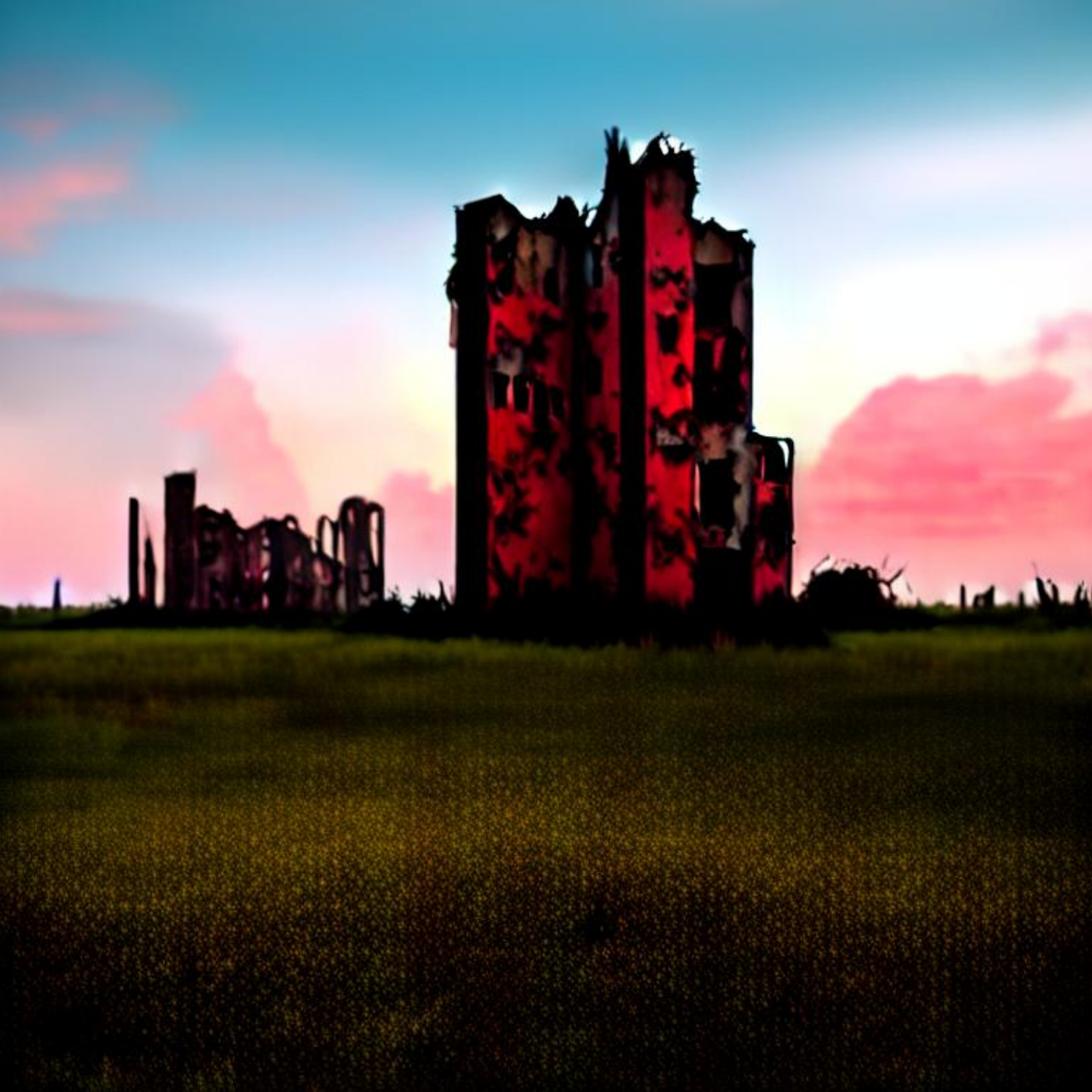}} & 
        \noindent\parbox[c]{0.14\columnwidth}{\includegraphics[width=0.14\columnwidth]{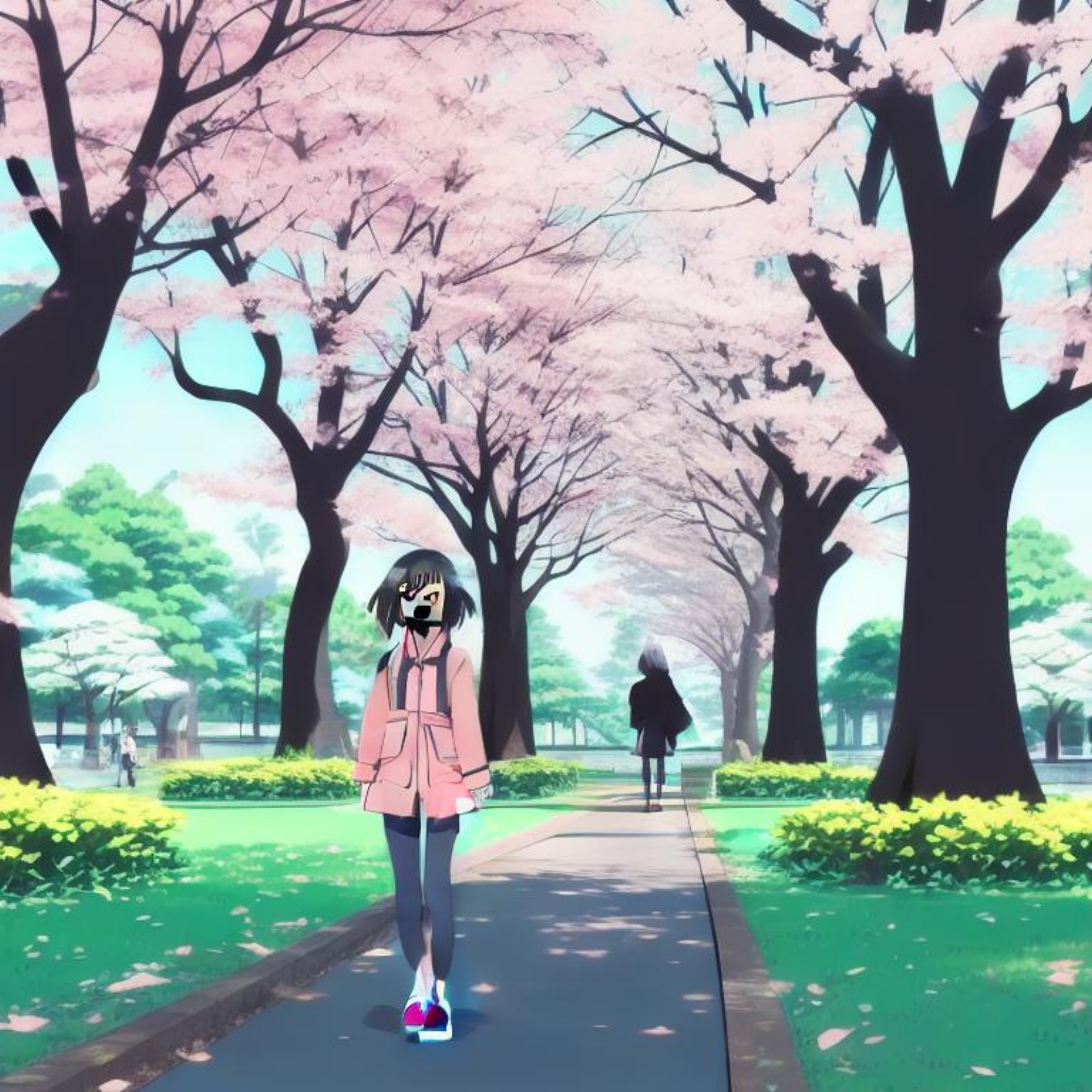}} & 
        \noindent\parbox[c]{0.14\columnwidth}{\includegraphics[width=0.14\columnwidth]{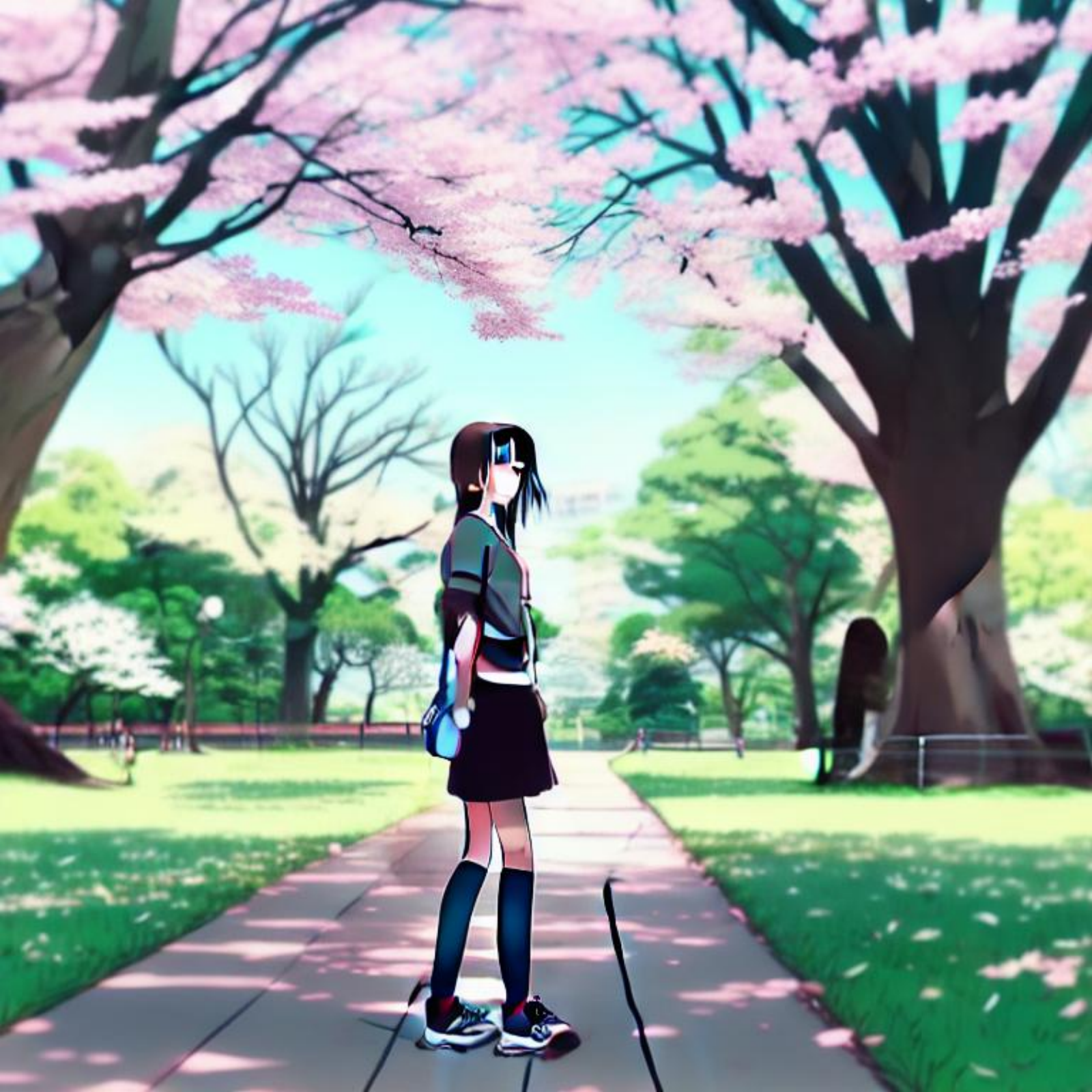}} & 
        \noindent\parbox[c]{0.14\columnwidth}{\includegraphics[width=0.14\columnwidth]{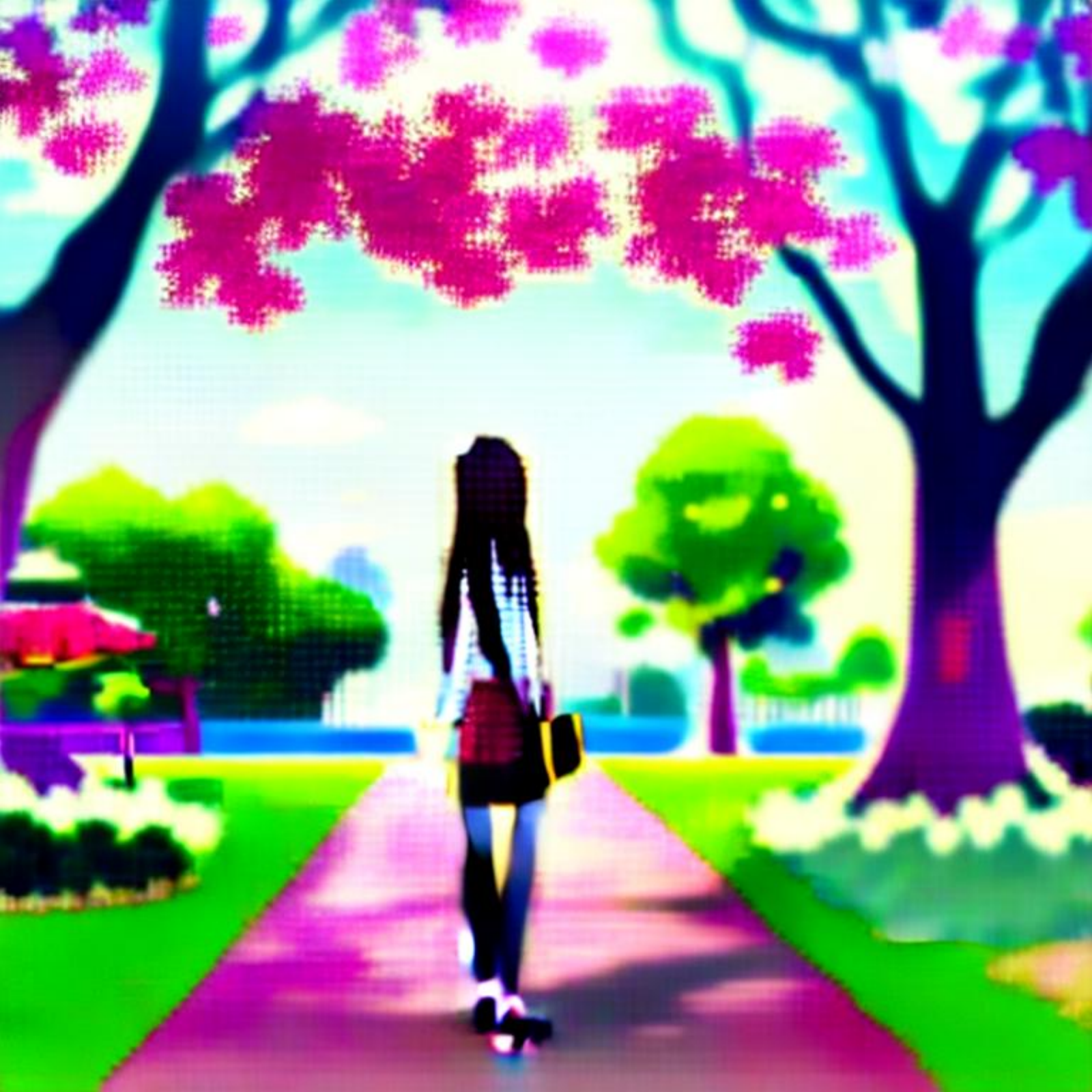}} \\

        \shortstack[l]{\tiny 15 steps} &
        \noindent\parbox[c]{0.14\columnwidth}{\includegraphics[width=0.14\columnwidth]{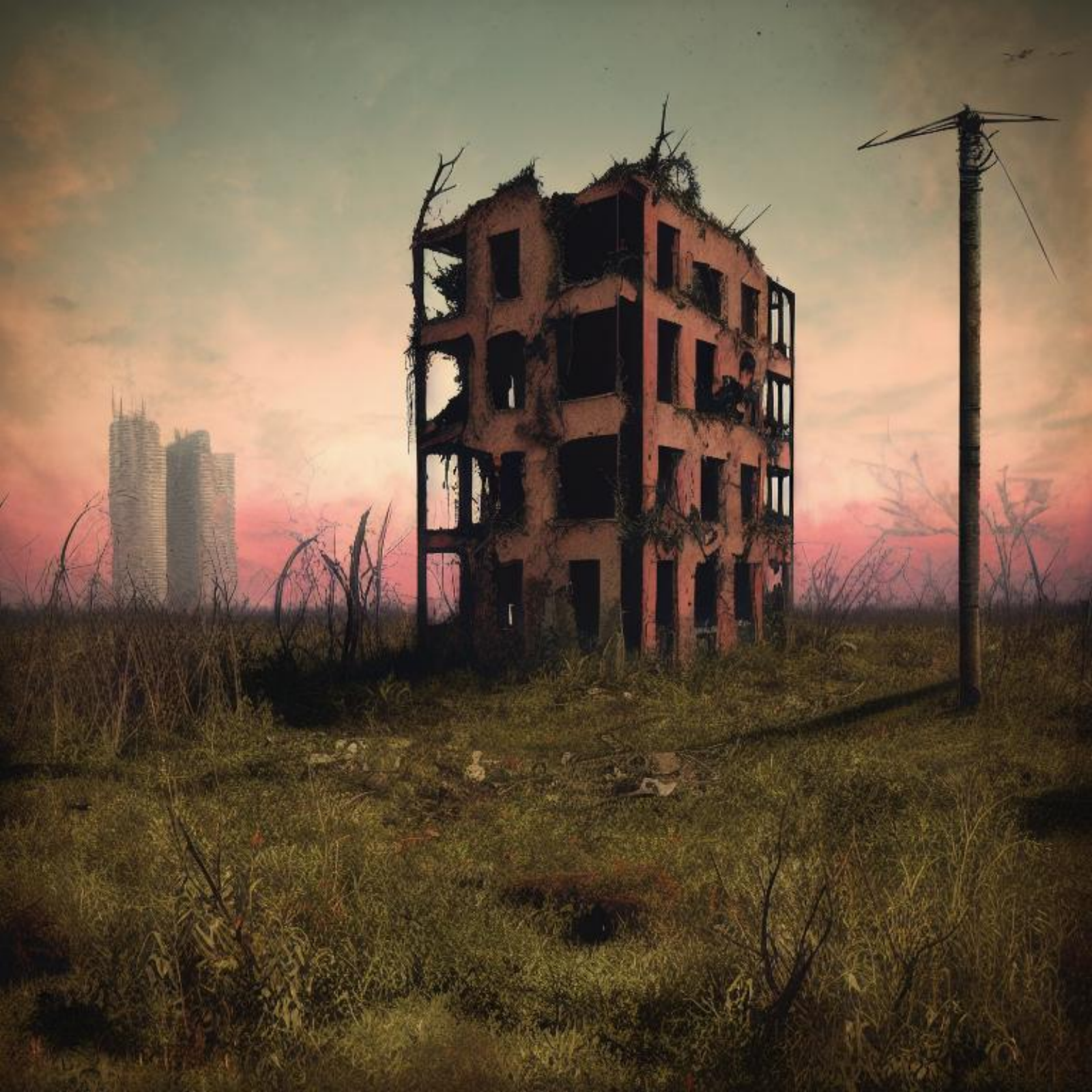}} & 
        \noindent\parbox[c]{0.14\columnwidth}{\includegraphics[width=0.14\columnwidth]{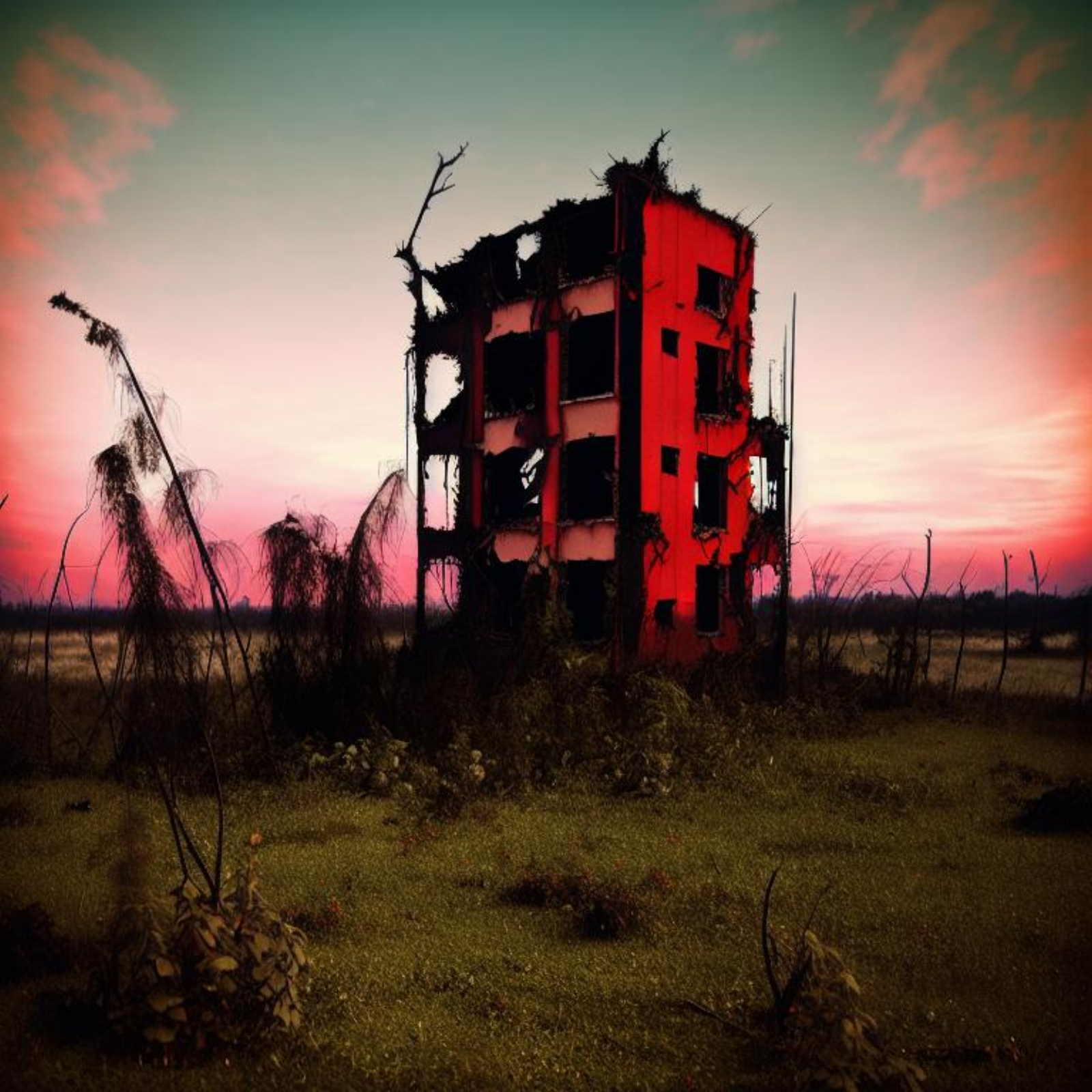}} & 
        \noindent\parbox[c]{0.14\columnwidth}{\includegraphics[width=0.14\columnwidth]{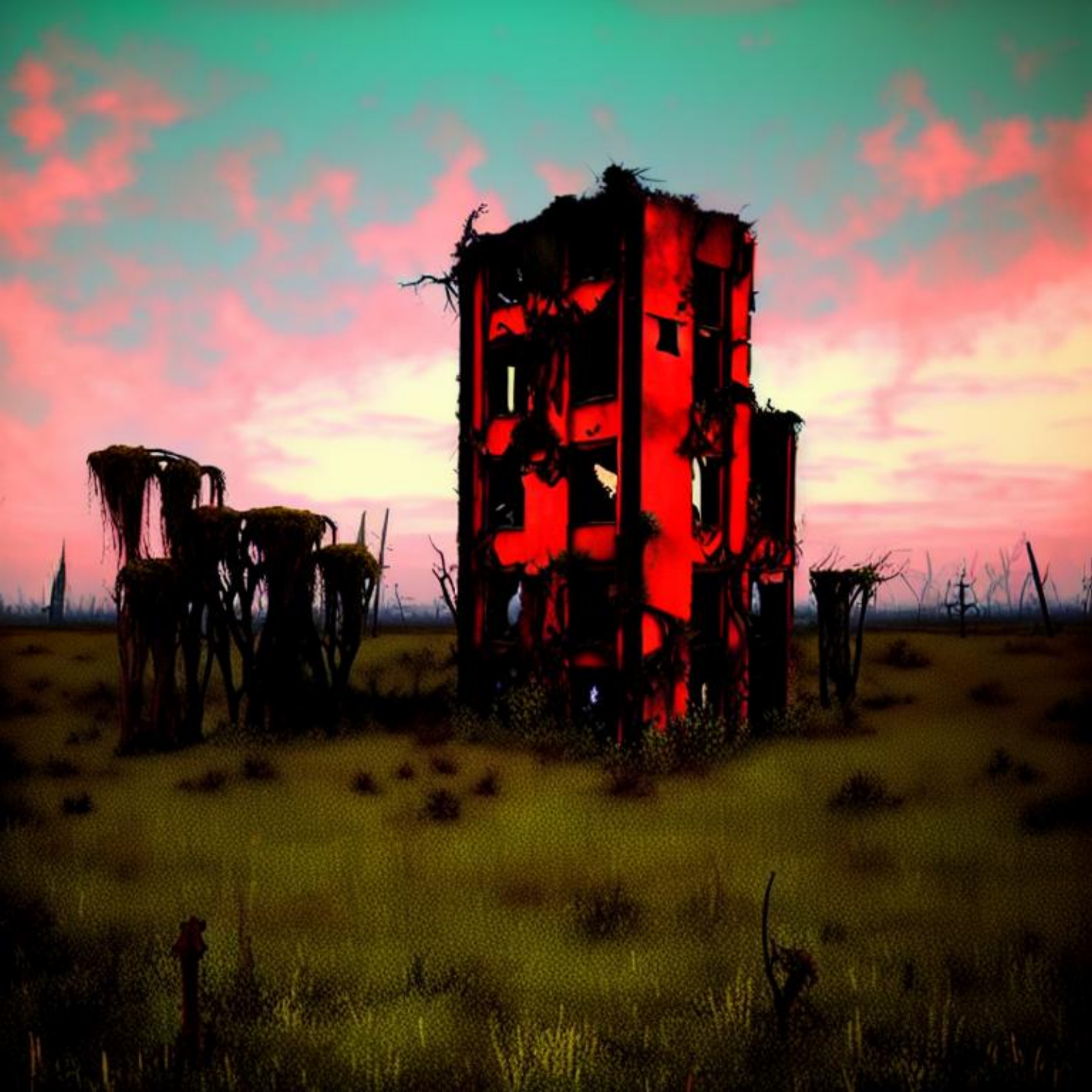}} & 
        \noindent\parbox[c]{0.14\columnwidth}{\includegraphics[width=0.14\columnwidth]{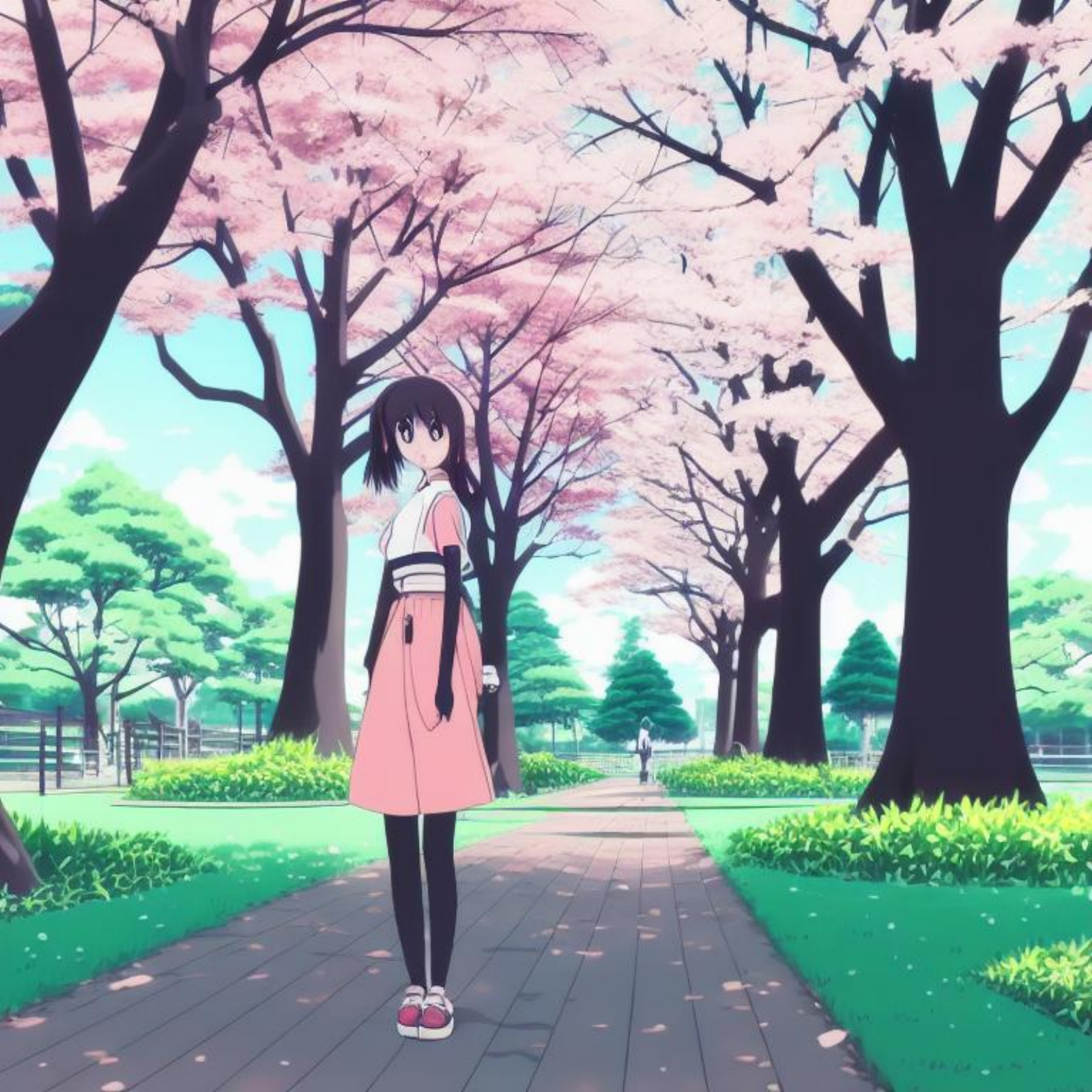}} & 
        \noindent\parbox[c]{0.14\columnwidth}{\includegraphics[width=0.14\columnwidth]{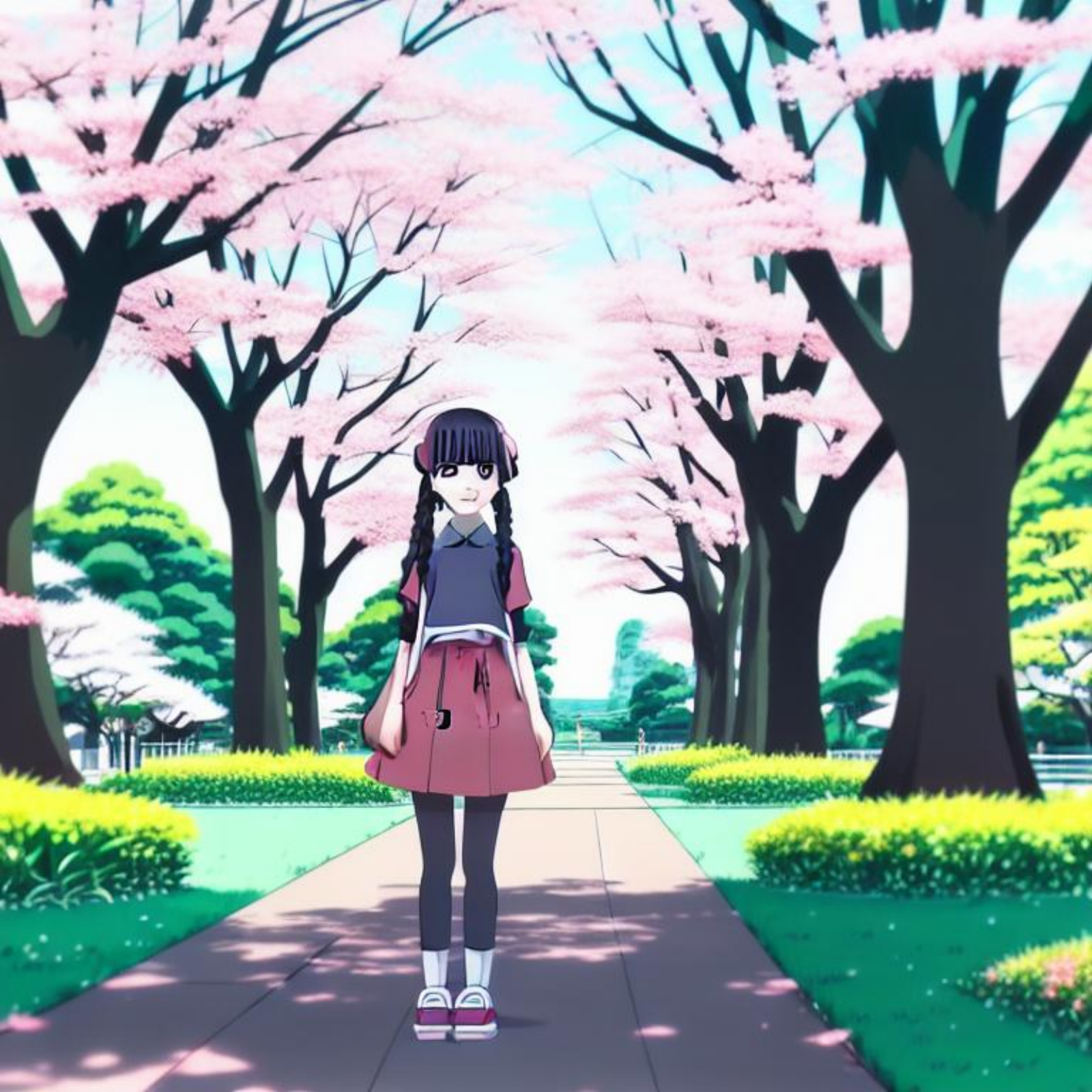}} & 
        \noindent\parbox[c]{0.14\columnwidth}{\includegraphics[width=0.14\columnwidth]{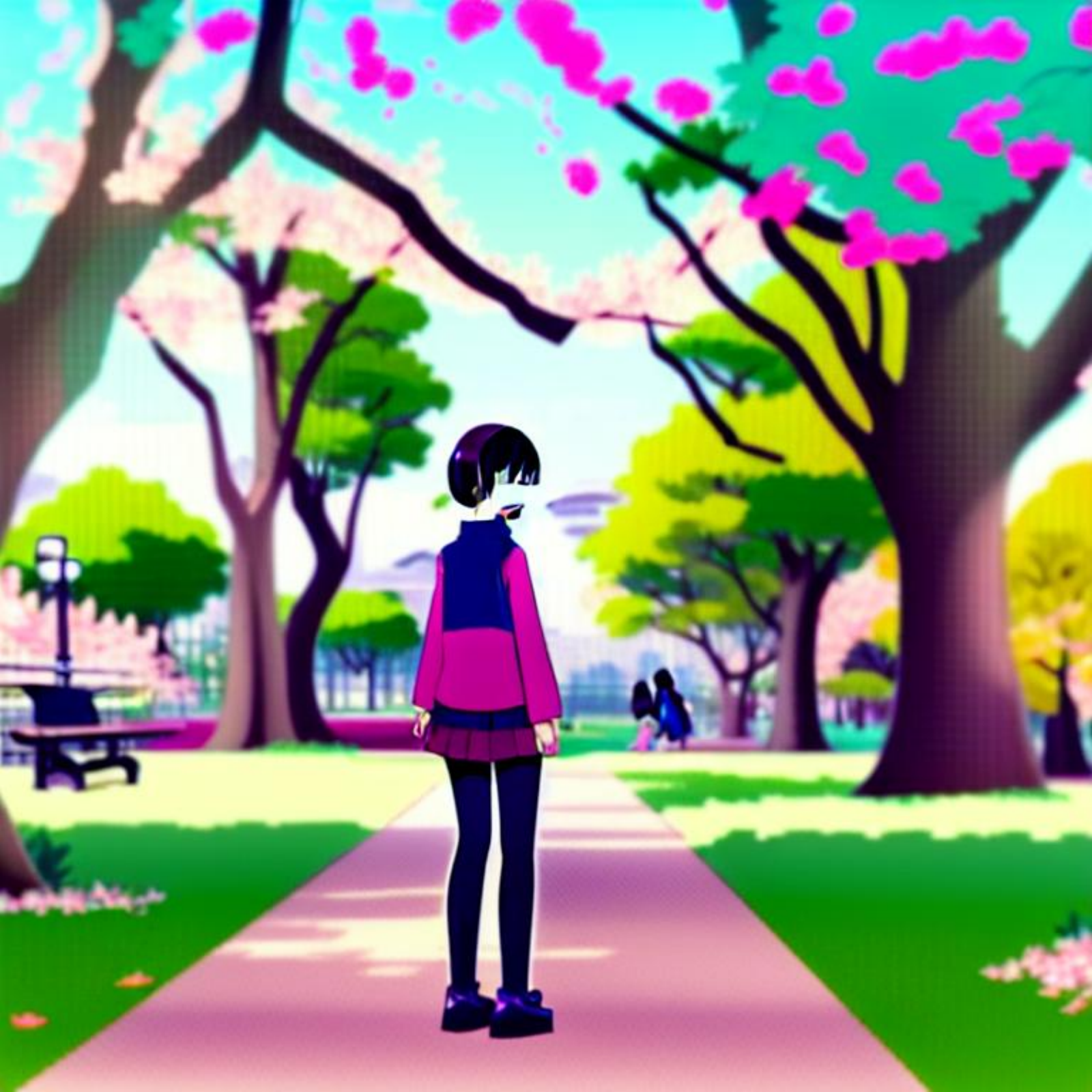}} \\

        \shortstack[l]{\tiny 20 steps} &
        \noindent\parbox[c]{0.14\columnwidth}{\includegraphics[width=0.14\columnwidth]{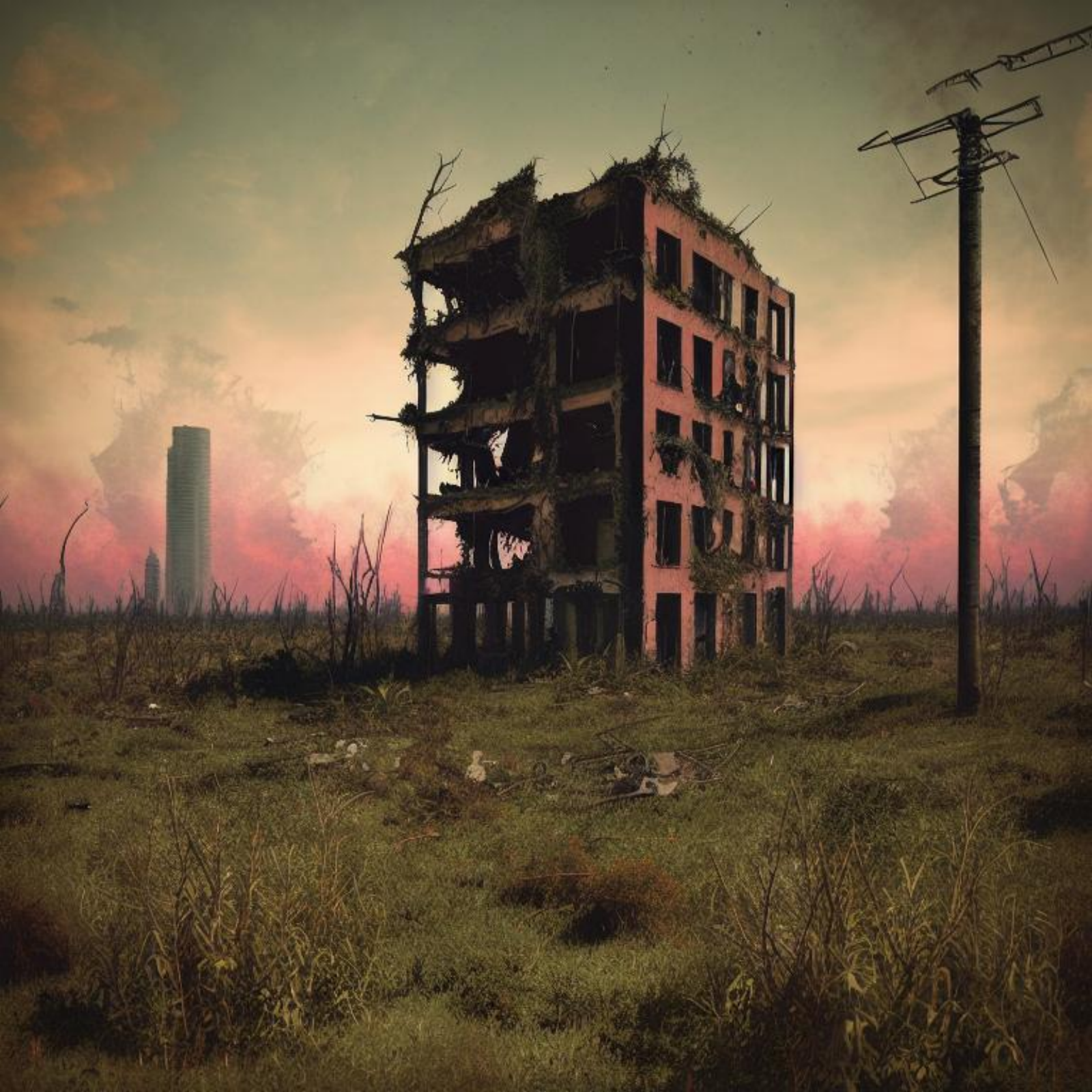}} & 
        \noindent\parbox[c]{0.14\columnwidth}{\includegraphics[width=0.14\columnwidth]{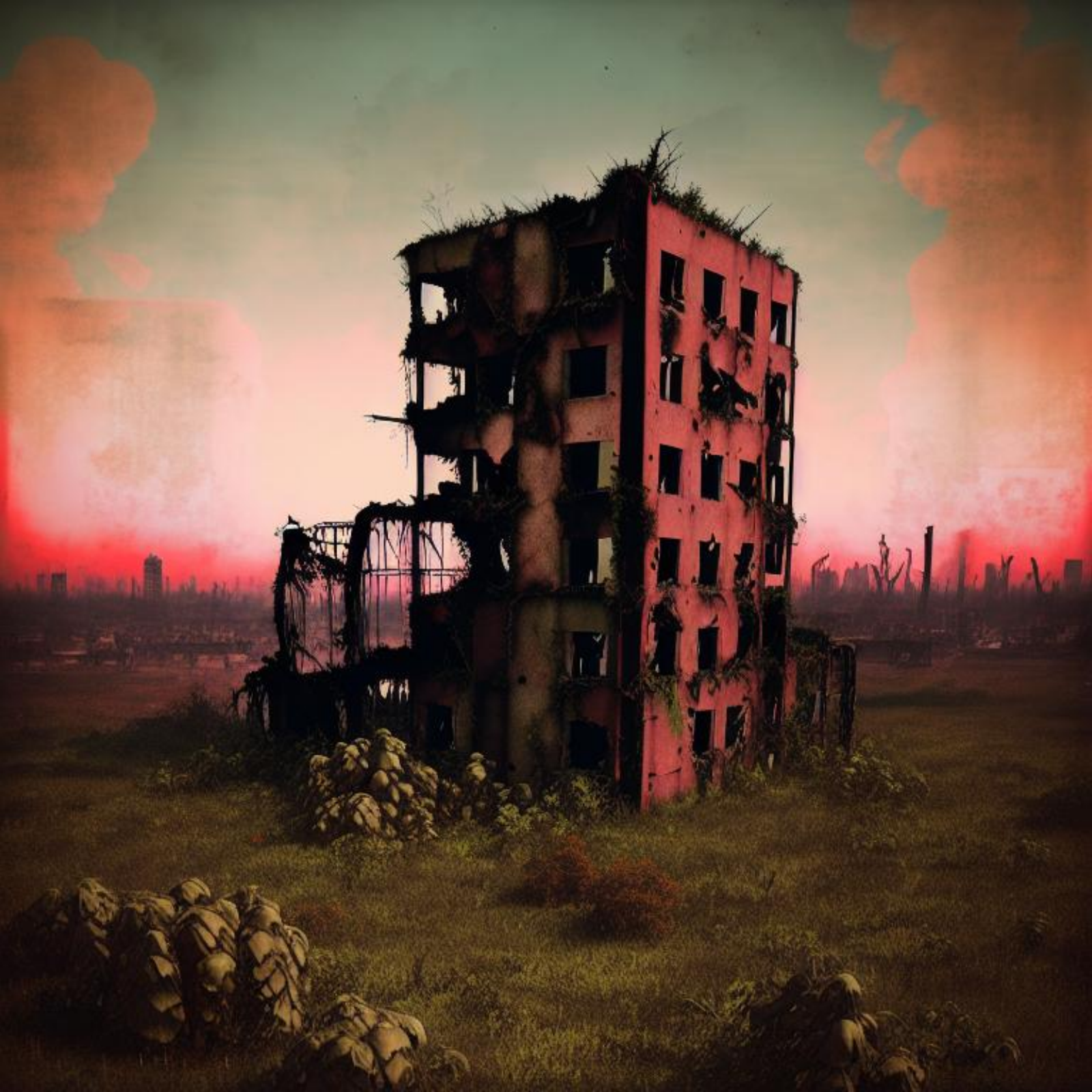}} & 
        \noindent\parbox[c]{0.14\columnwidth}{\includegraphics[width=0.14\columnwidth]{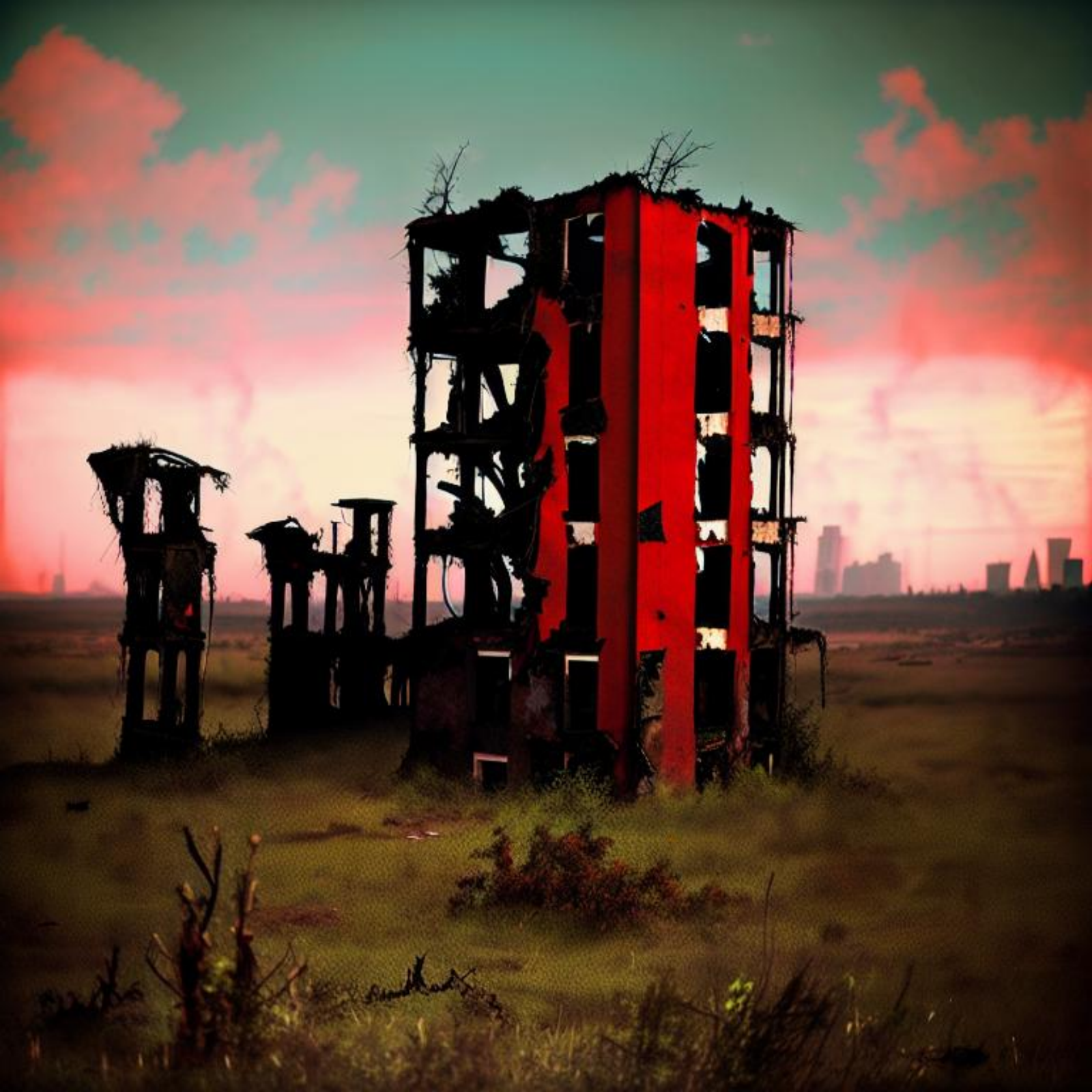}} & 
        \noindent\parbox[c]{0.14\columnwidth}{\includegraphics[width=0.14\columnwidth]{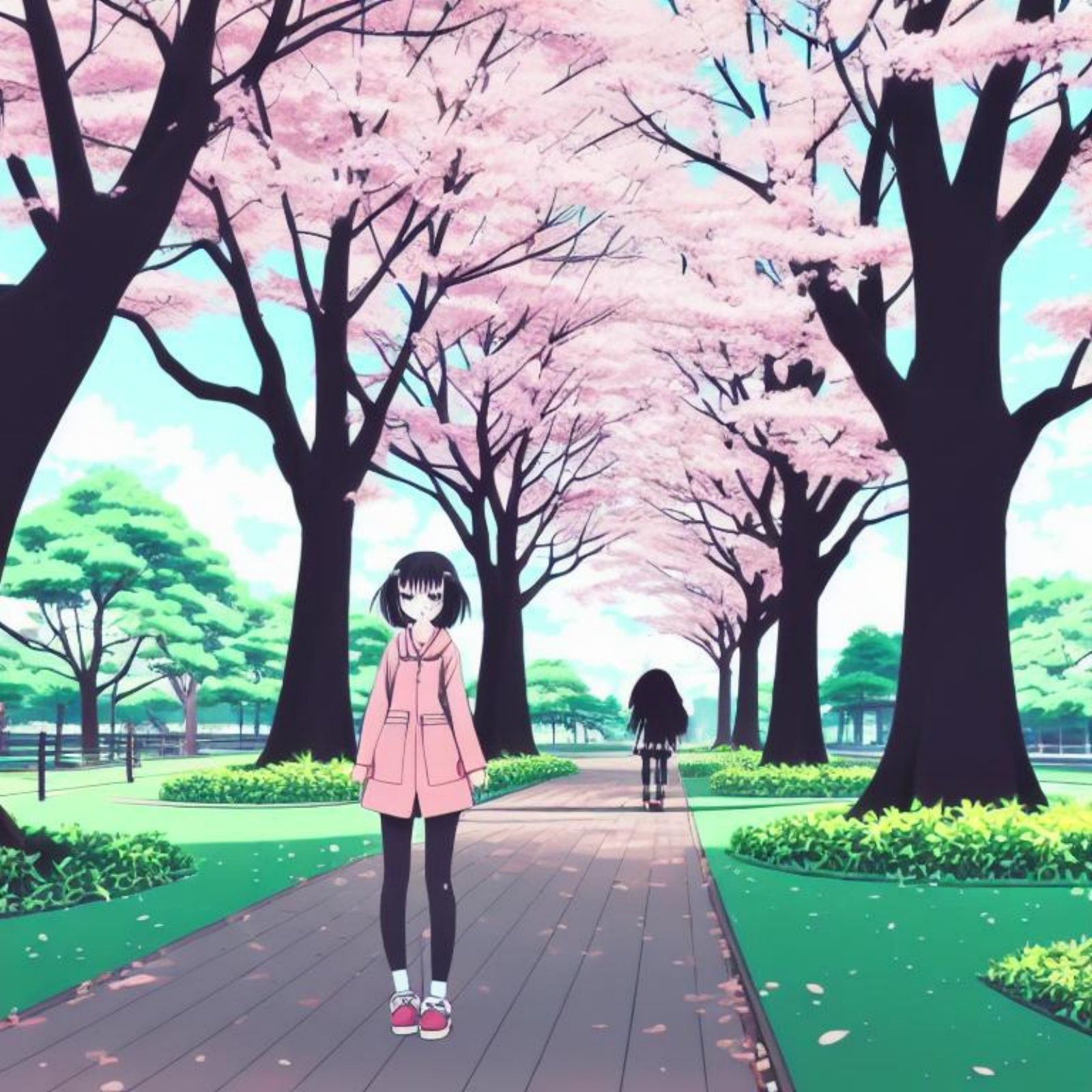}} & 
        \noindent\parbox[c]{0.14\columnwidth}{\includegraphics[width=0.14\columnwidth]{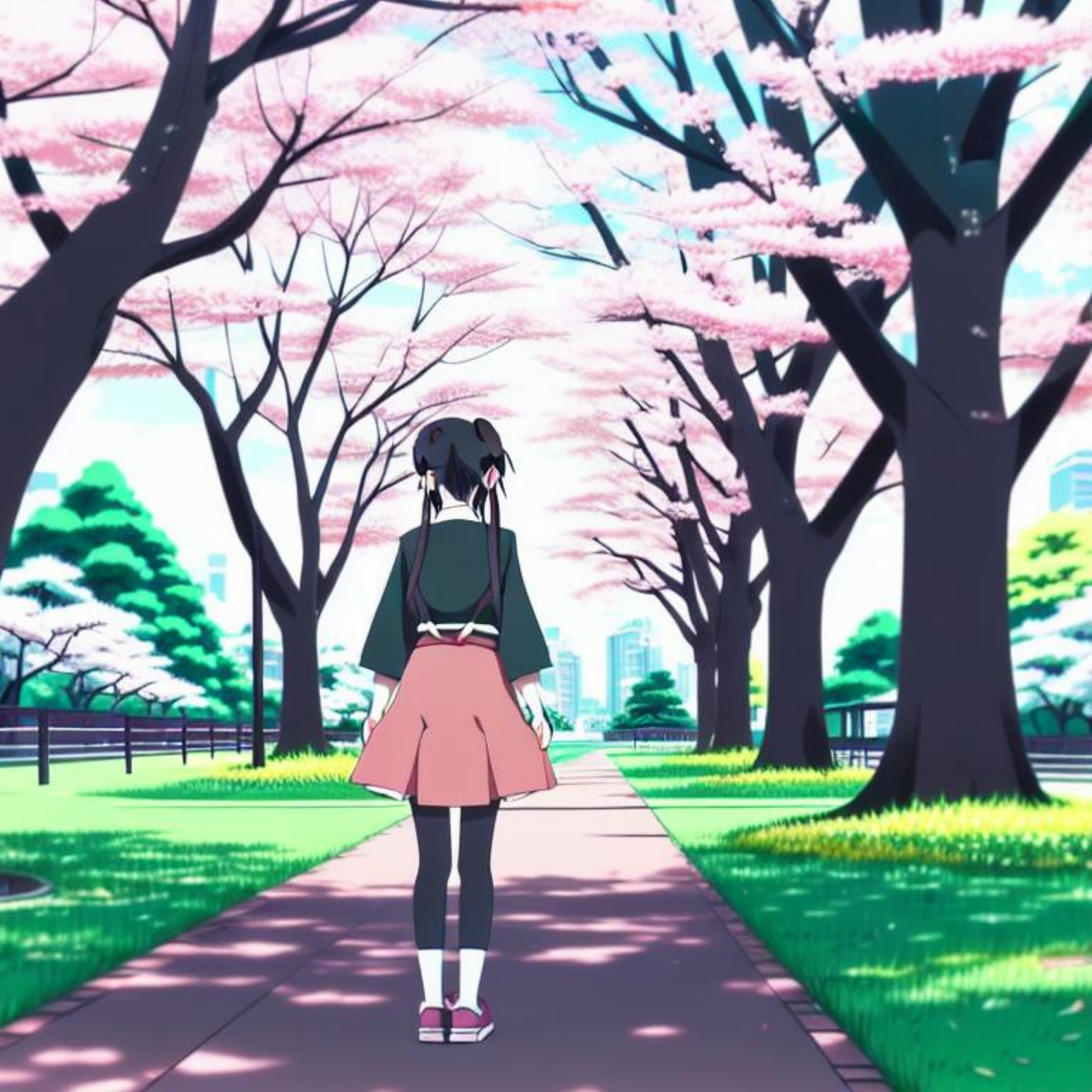}} & 
        \noindent\parbox[c]{0.14\columnwidth}{\includegraphics[width=0.14\columnwidth]{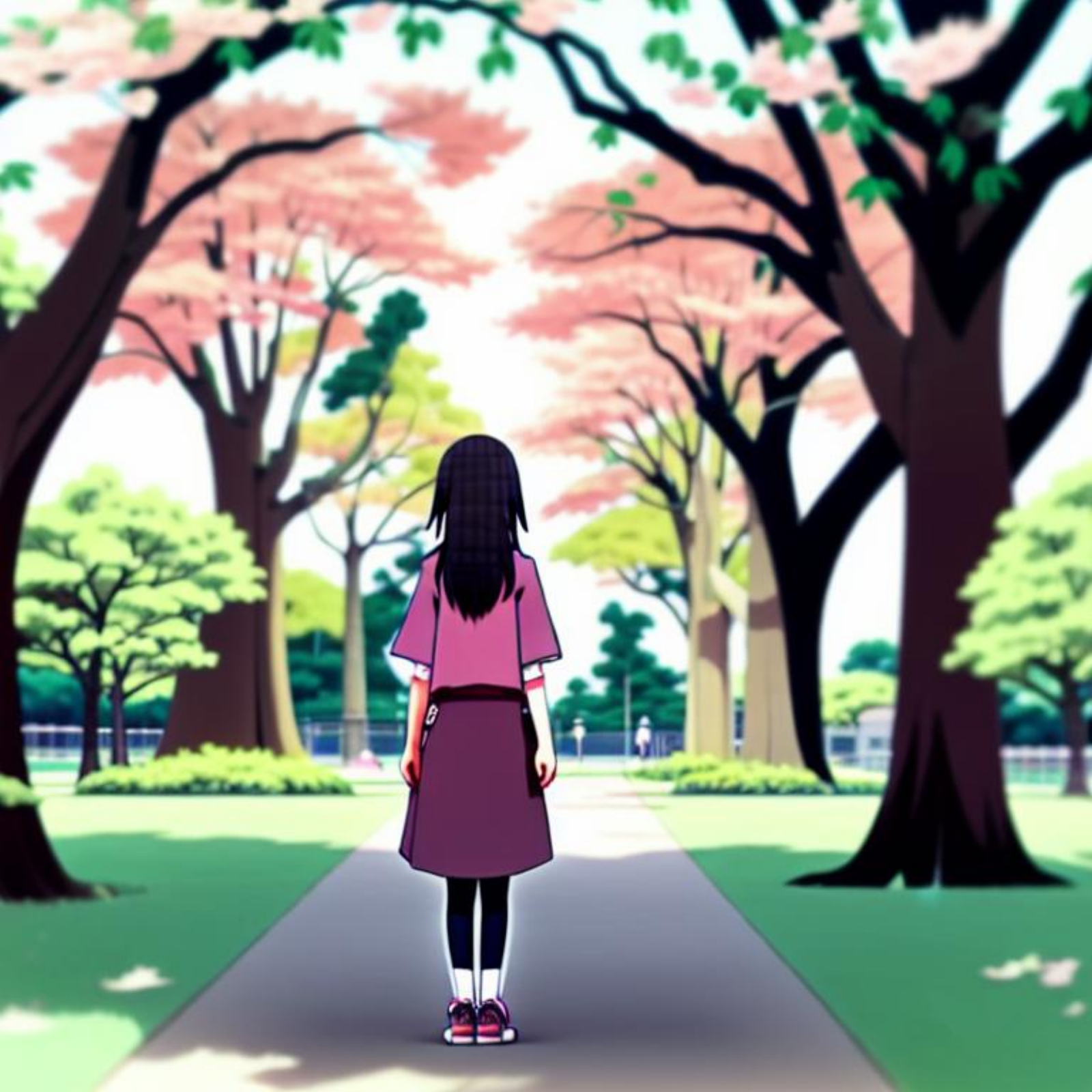}} \\

        \shortstack[l]{\tiny 40 steps} &
        \noindent\parbox[c]{0.14\columnwidth}{\includegraphics[width=0.14\columnwidth]{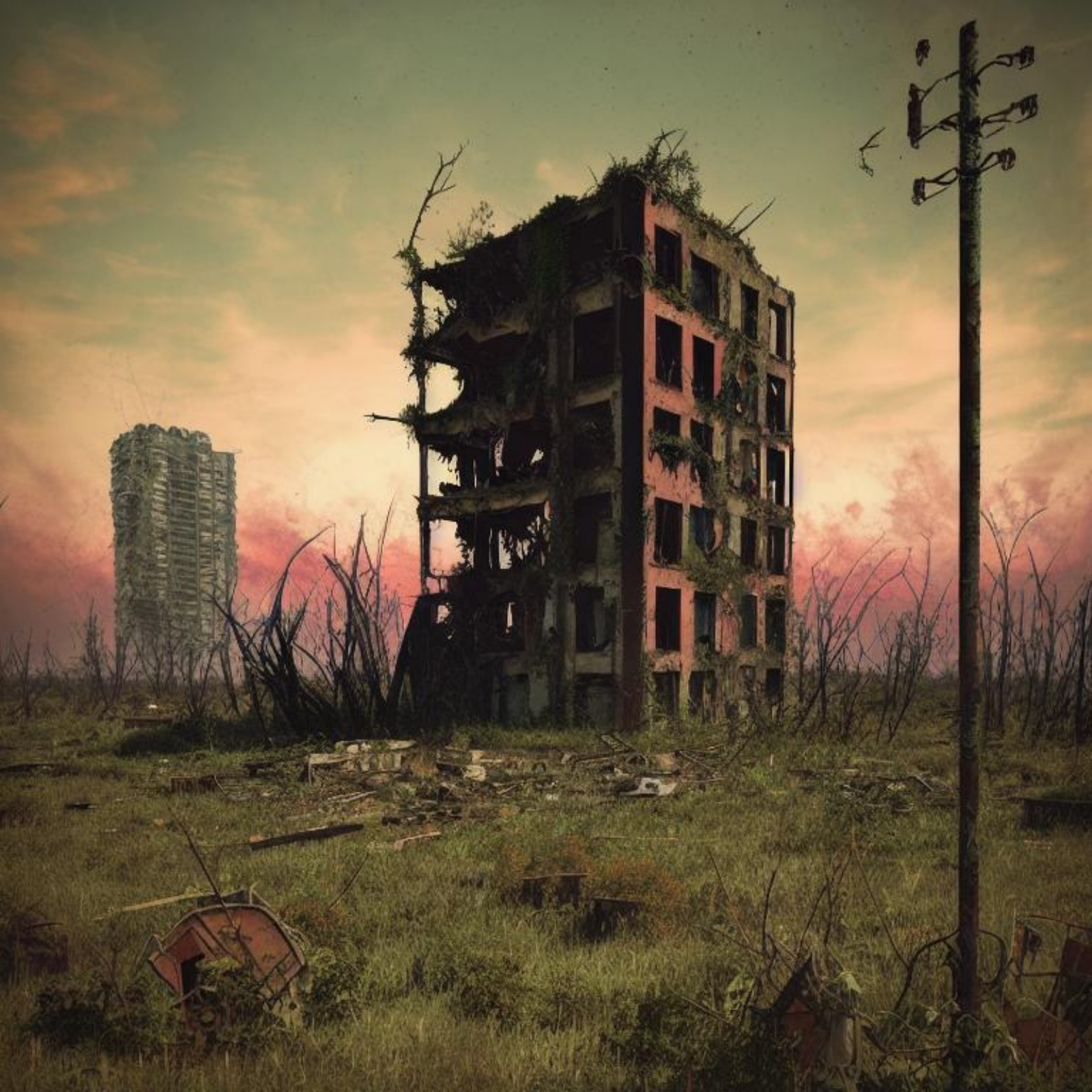}} & 
        \noindent\parbox[c]{0.14\columnwidth}{\includegraphics[width=0.14\columnwidth]{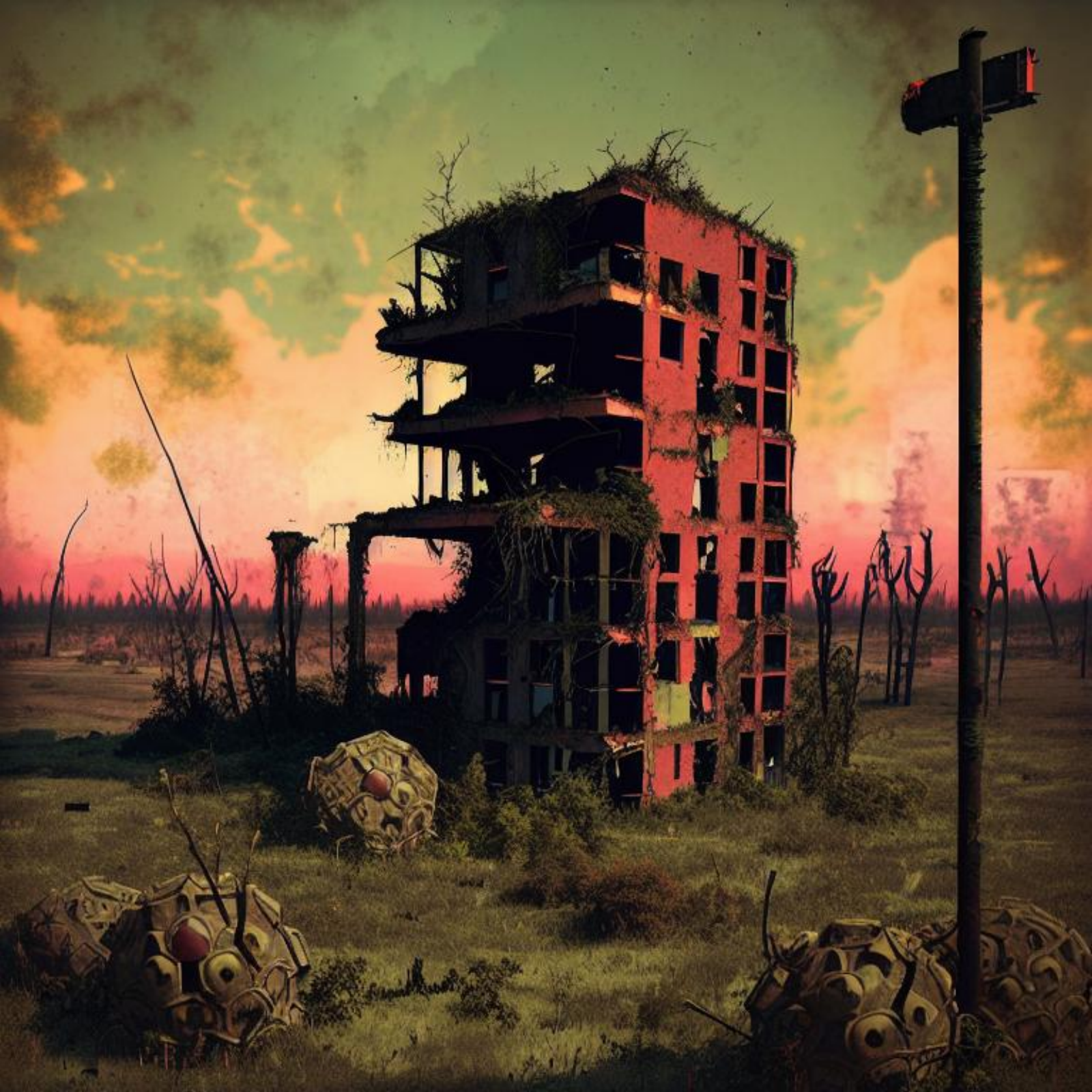}} & 
        \noindent\parbox[c]{0.14\columnwidth}{\includegraphics[width=0.14\columnwidth]{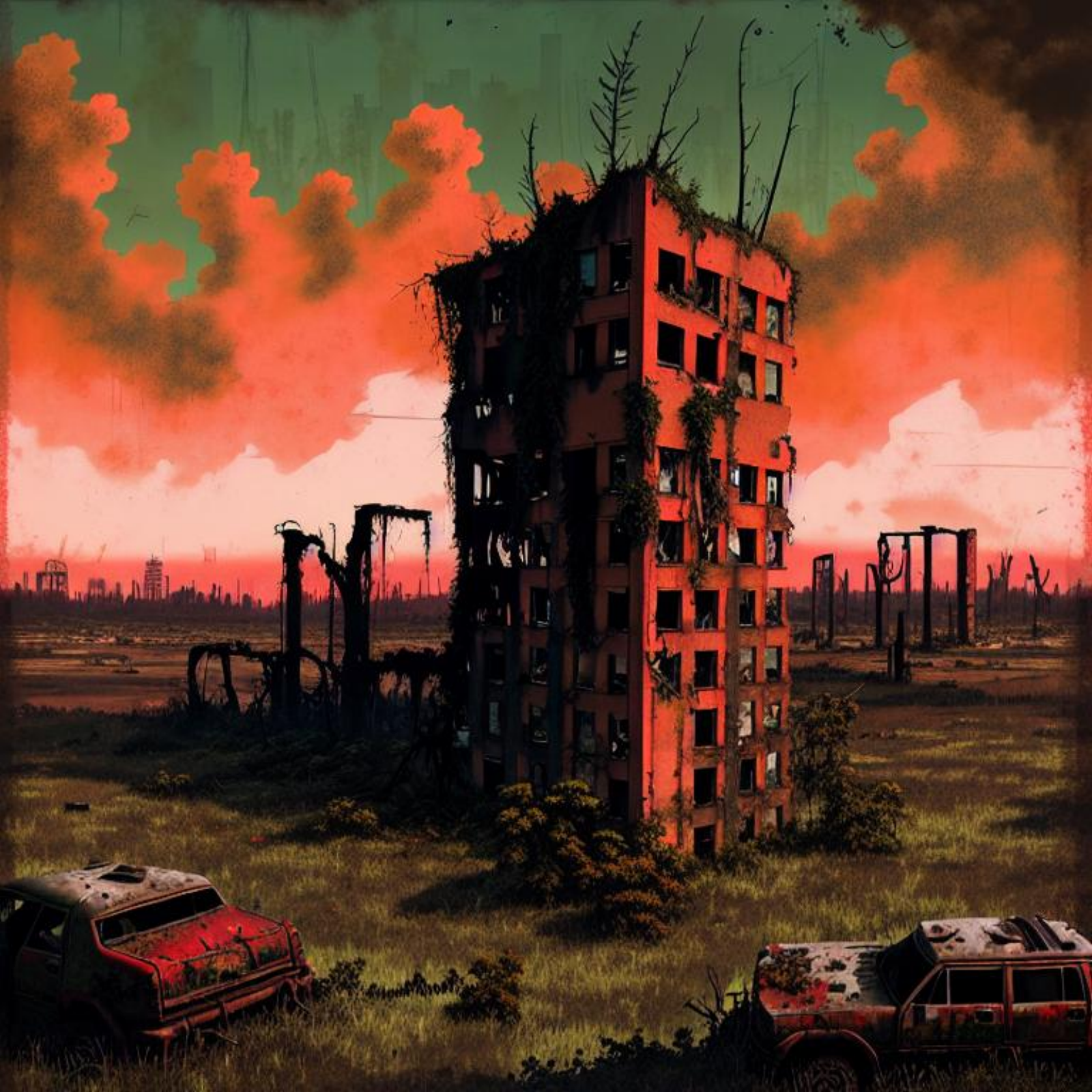}} & 
        \noindent\parbox[c]{0.14\columnwidth}{\includegraphics[width=0.14\columnwidth]{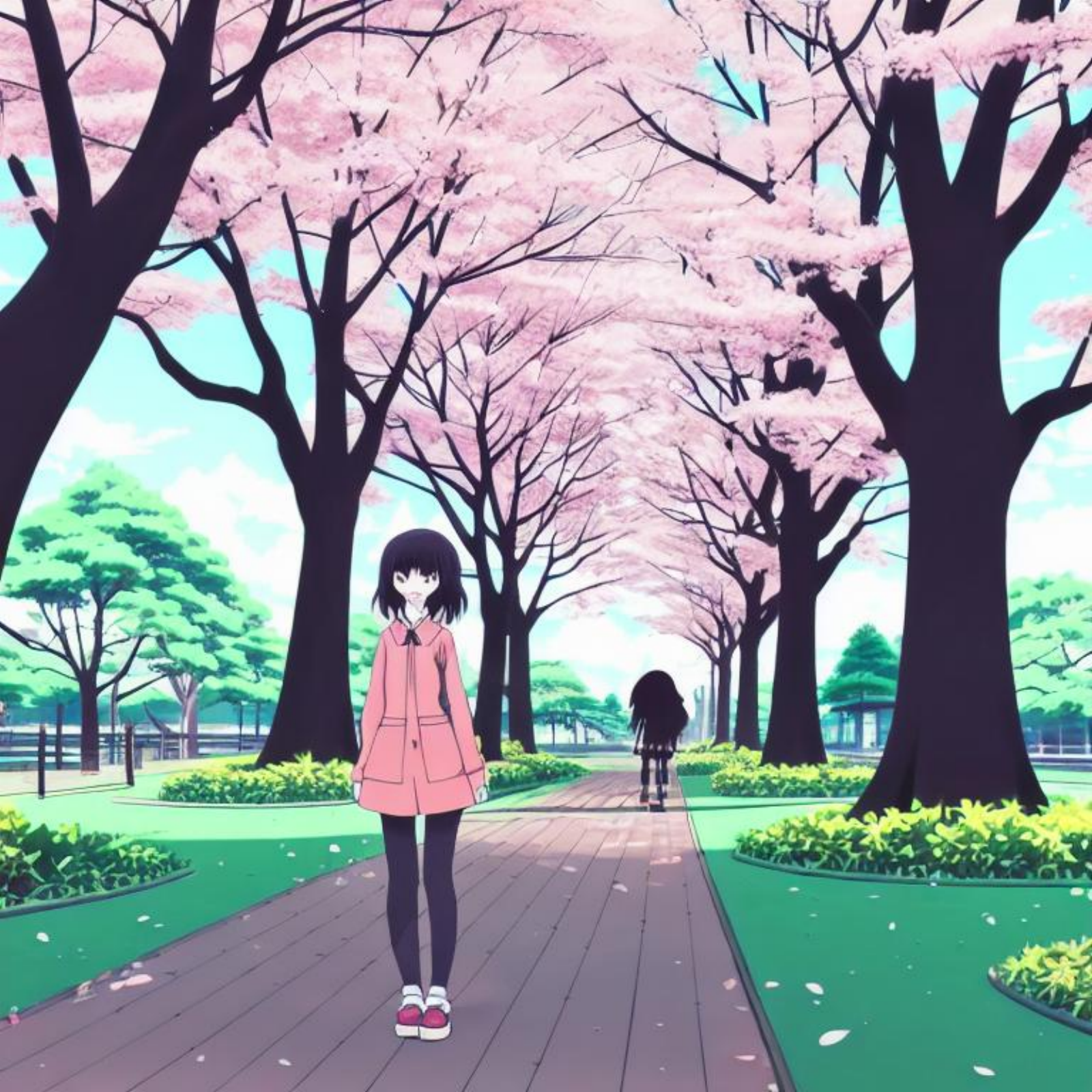}} & 
        \noindent\parbox[c]{0.14\columnwidth}{\includegraphics[width=0.14\columnwidth]{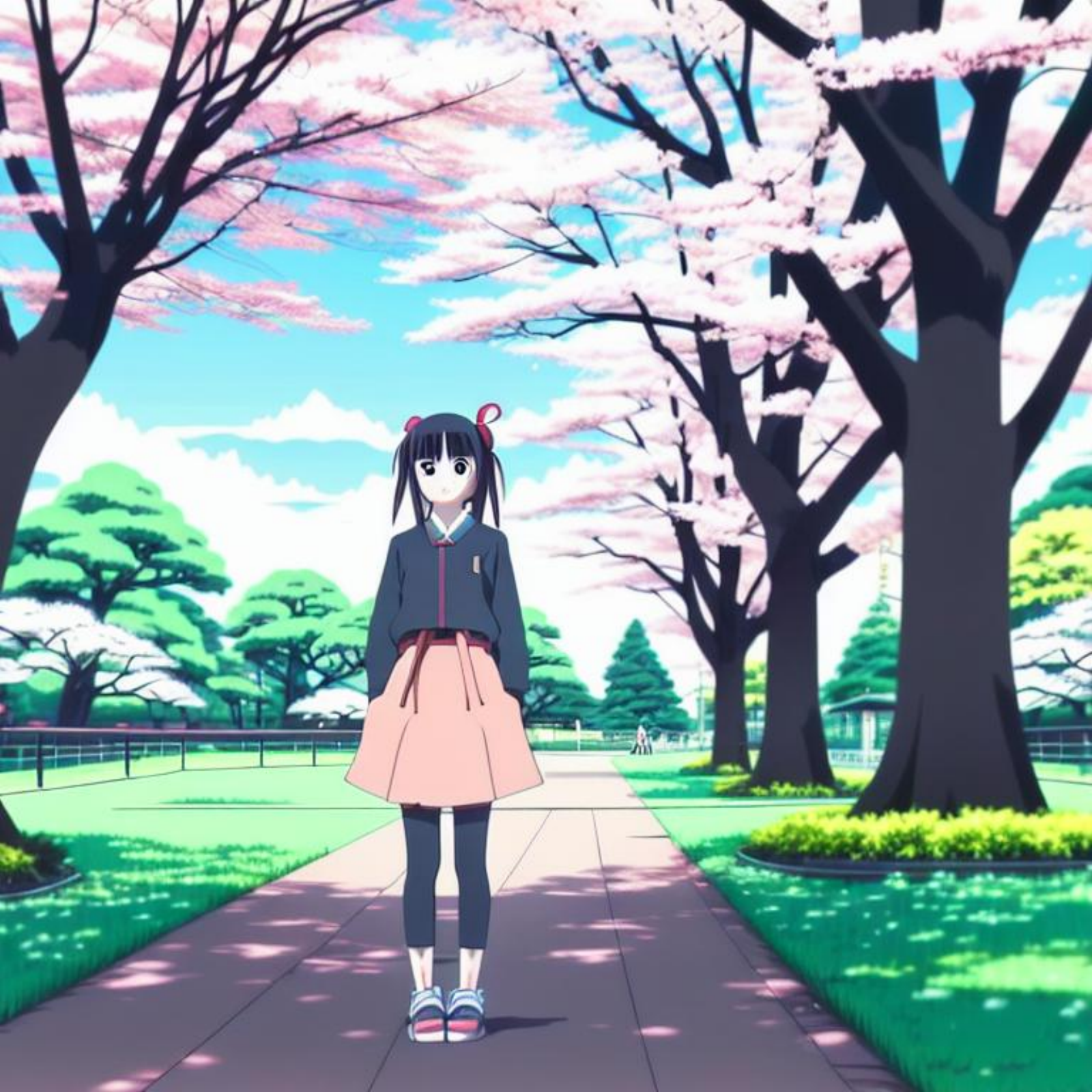}} & 
        \noindent\parbox[c]{0.14\columnwidth}{\includegraphics[width=0.14\columnwidth]{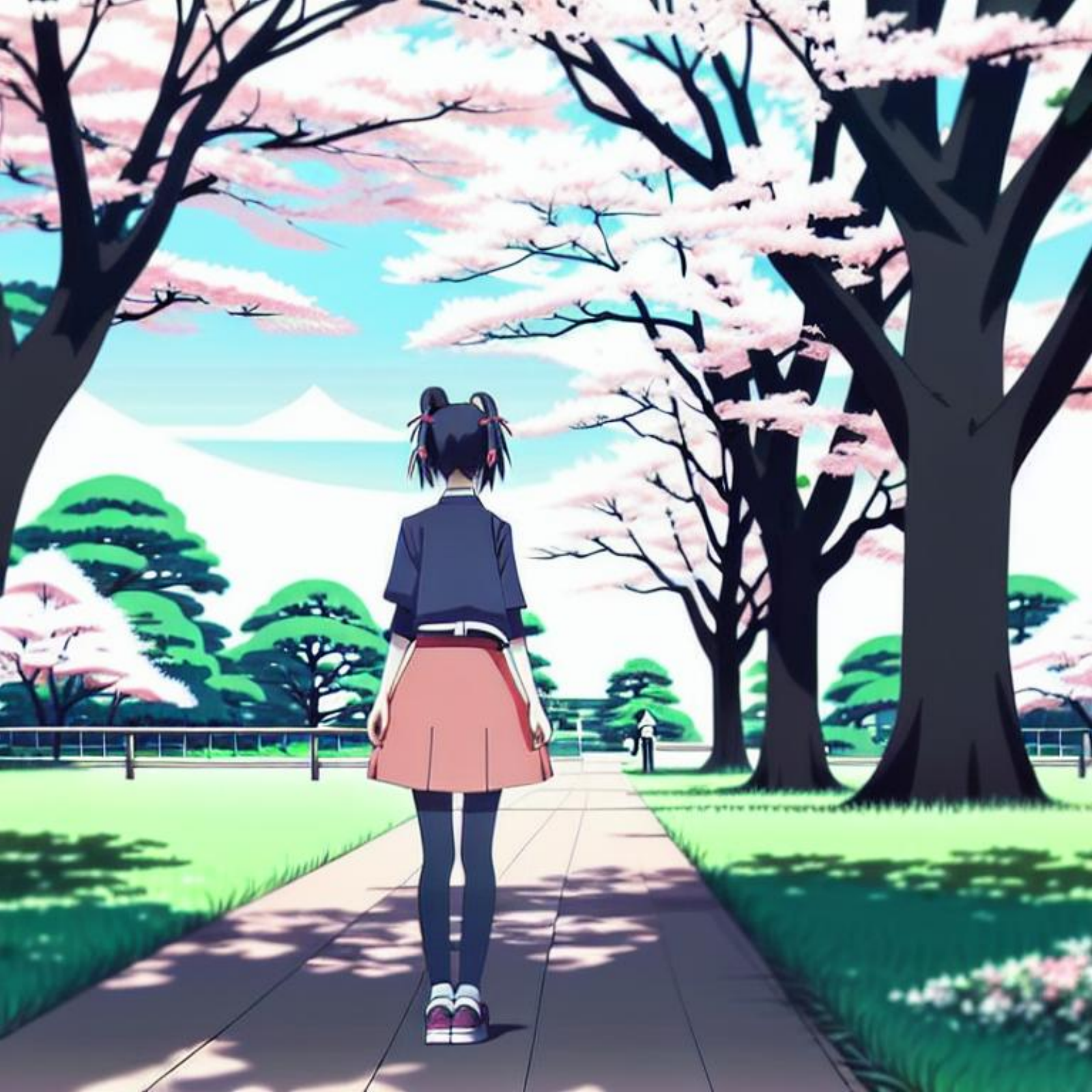}} \\
    \end{tabu}
    \caption{Comparison of samples generated from Dreamlike Photoreal V2.0 using PLMS4 with HB $\beta$ under various sampling steps and guidance scale $s$. Specifically, we employ $\beta = 0.7$ for $s = 7.5$, $\beta = 0.6$ for $s = 15$, and $\beta = 0.6$ for $s = 22.5$ to account for the varying degrees of artifact manifestation associated with each guidance scale.}
    \label{fig:scale_step_dreamlike_hb}
\end{figure}




\tabulinesep=1pt
\begin{figure}
    \centering
    \begin{tabu} to \textwidth {@{}l@{\hspace{5pt}}c@{\hspace{2pt}}c@{\hspace{2pt}}c@{\hspace{4pt}}c@{\hspace{2pt}}c@{\hspace{2pt}}c@{}}
        & \multicolumn{3}{c}{\shortstack{\scriptsize "A post-apocalyptic world with ruined \\ \scriptsize buildings, overgrown vegetation, and a red sky"}}
        & \multicolumn{3}{c}{\shortstack{\scriptsize "A girl standing in a park in \\ \scriptsize Japanese animation style"}} \\

        & \multicolumn{1}{c}{\shortstack{\scriptsize $s = 7.5$}}
        & \multicolumn{1}{c}{\shortstack{\scriptsize $s = 15$}}
        & \multicolumn{1}{c}{\shortstack{\scriptsize $s = 22.5$}}
        & \multicolumn{1}{c}{\shortstack{\scriptsize $s = 7.5$}}
        & \multicolumn{1}{c}{\shortstack{\scriptsize $s = 15$}}
        & \multicolumn{1}{c}{\shortstack{\scriptsize $s = 22.5$}}
        \\
        
        \shortstack[l]{\tiny 10 steps} &
        \noindent\parbox[c]{0.14\columnwidth}{\includegraphics[width=0.14\columnwidth]{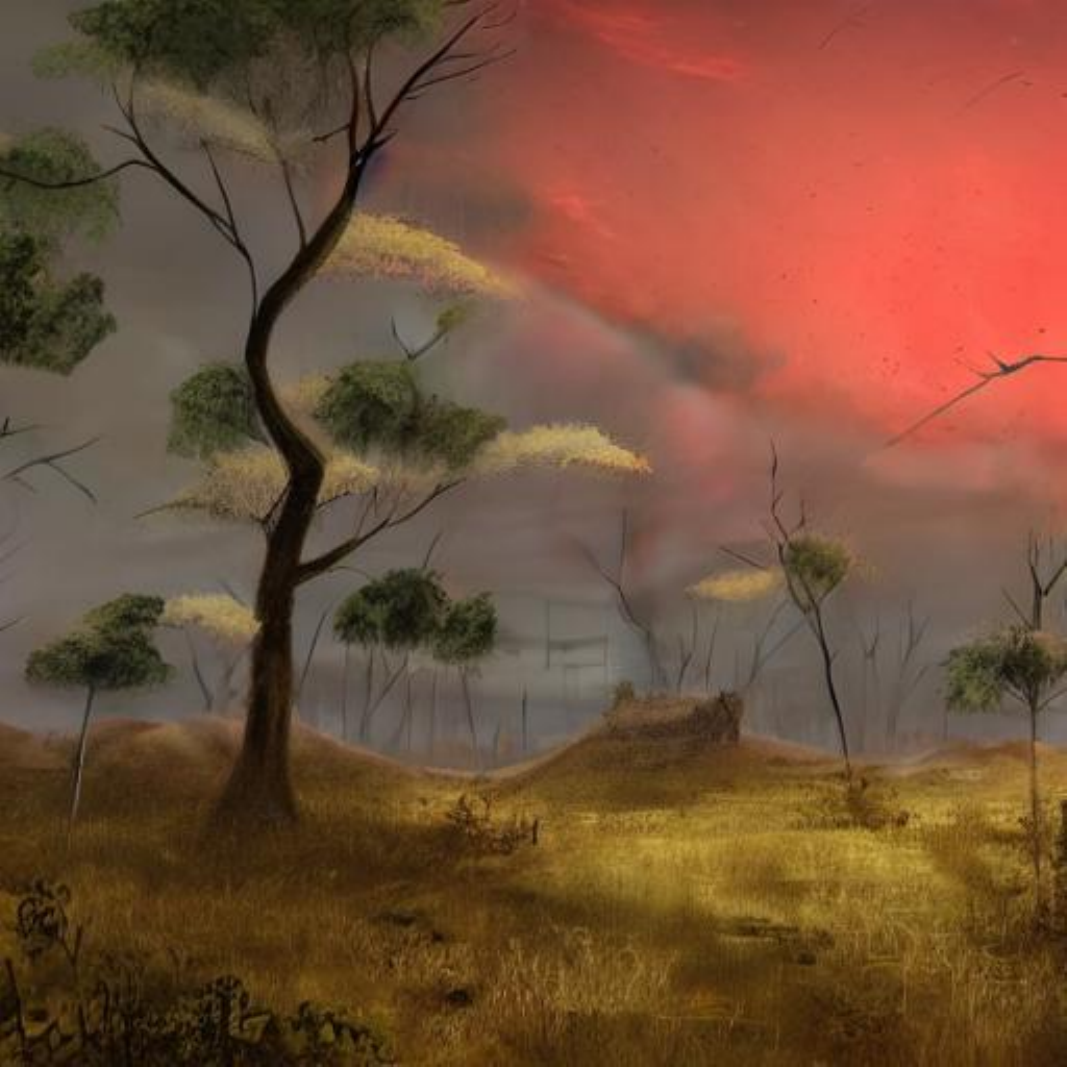}} & 
        \noindent\parbox[c]{0.14\columnwidth}{\includegraphics[width=0.14\columnwidth]{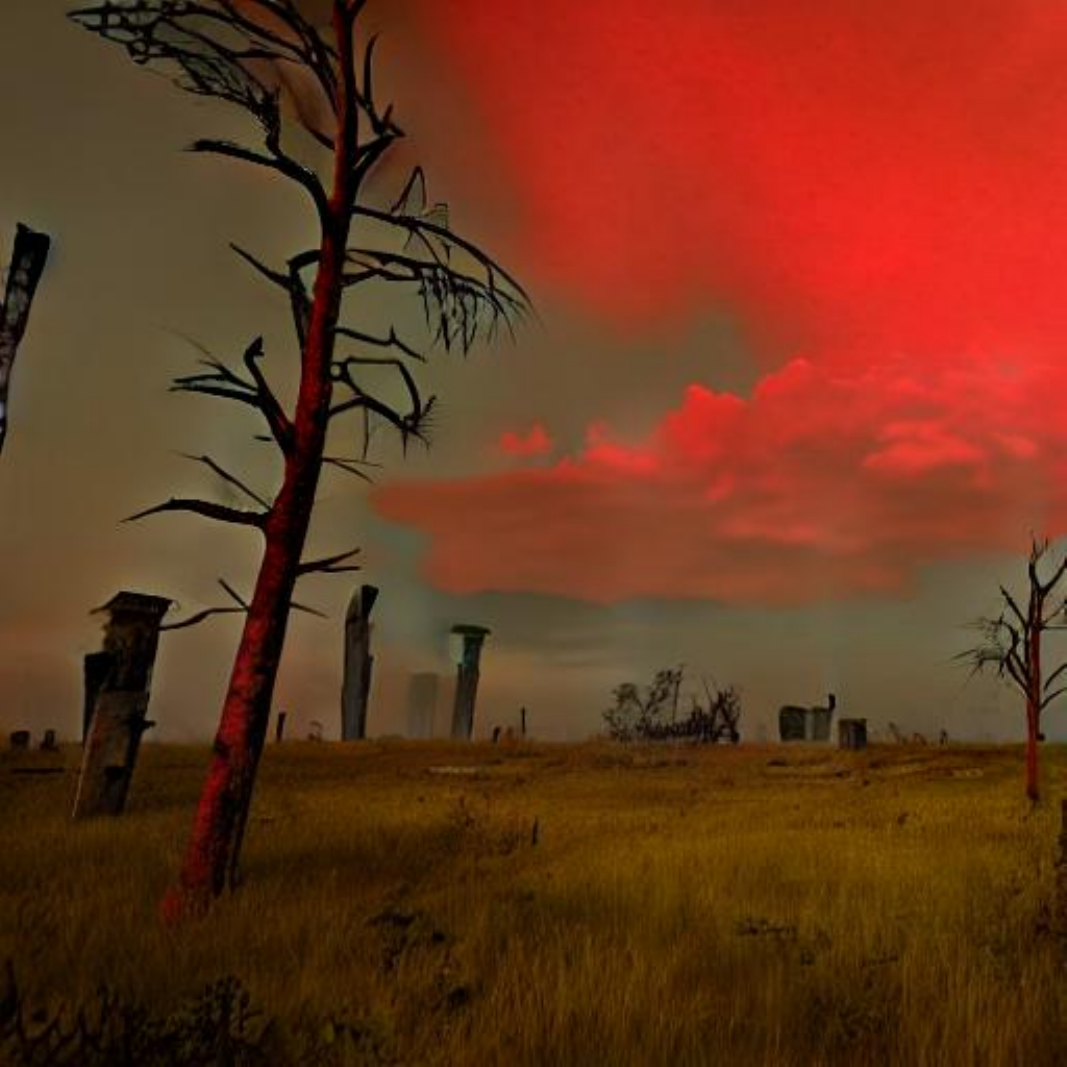}} & 
        \noindent\parbox[c]{0.14\columnwidth}{\includegraphics[width=0.14\columnwidth]{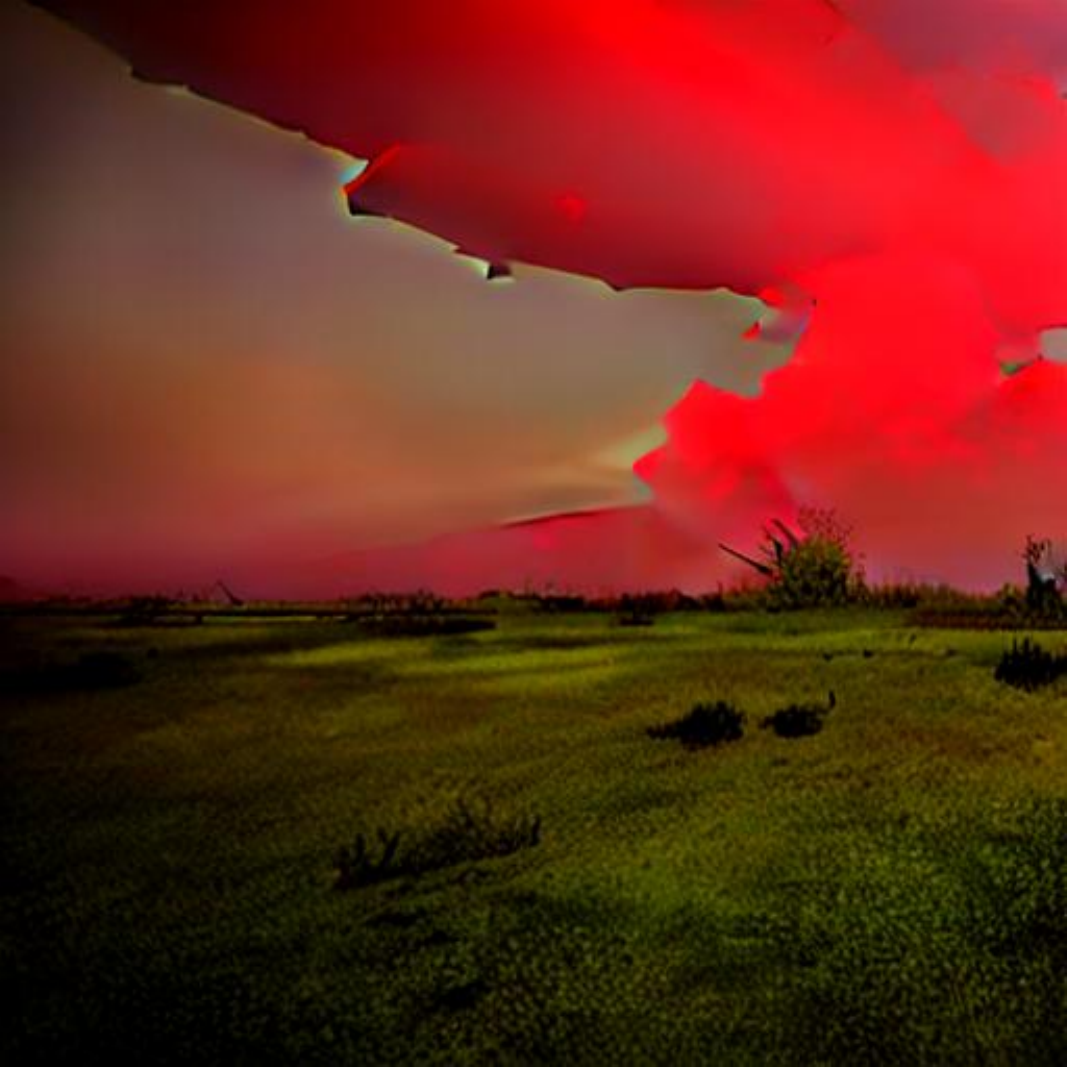}} & 
        \noindent\parbox[c]{0.14\columnwidth}{\includegraphics[width=0.14\columnwidth]{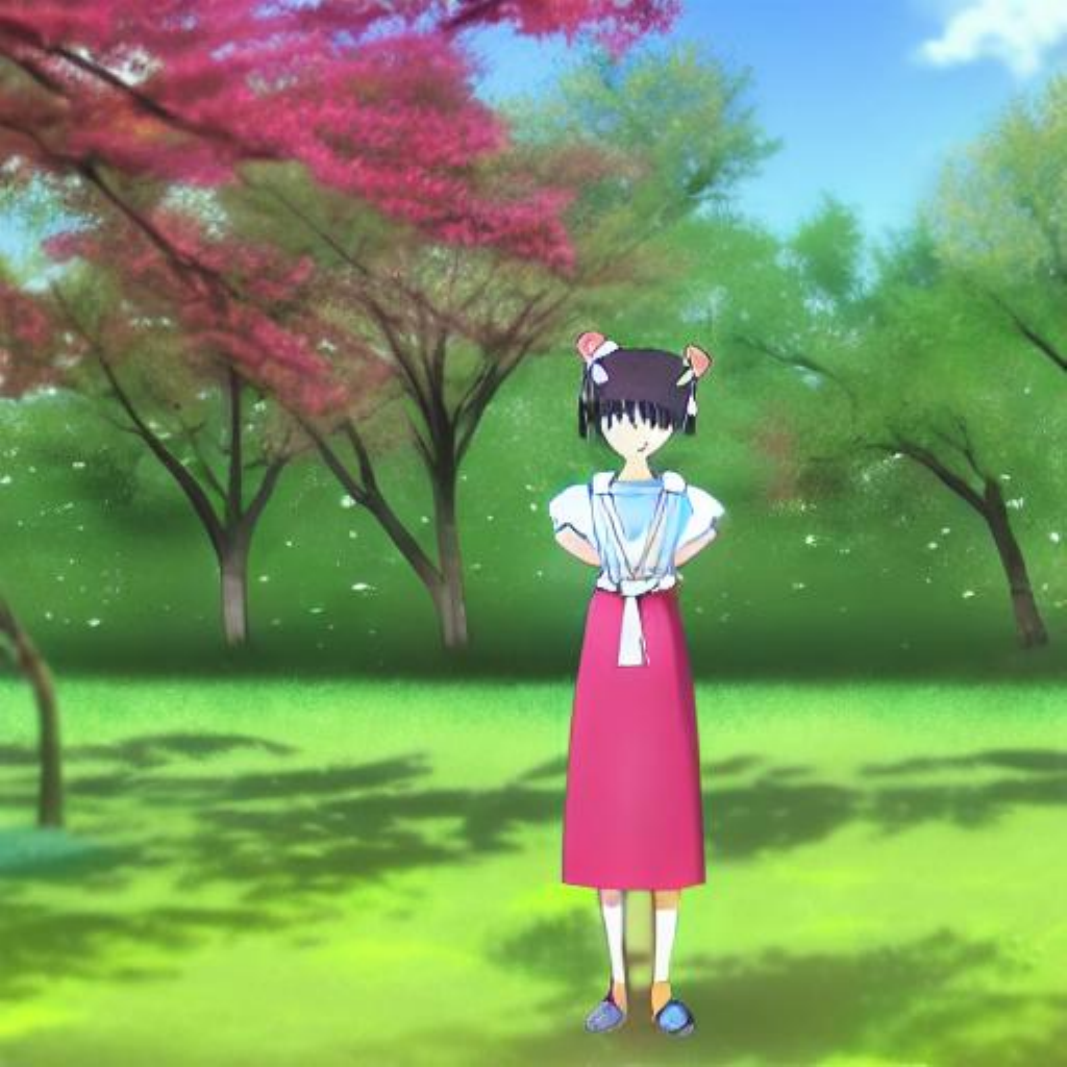}} & 
        \noindent\parbox[c]{0.14\columnwidth}{\includegraphics[width=0.14\columnwidth]{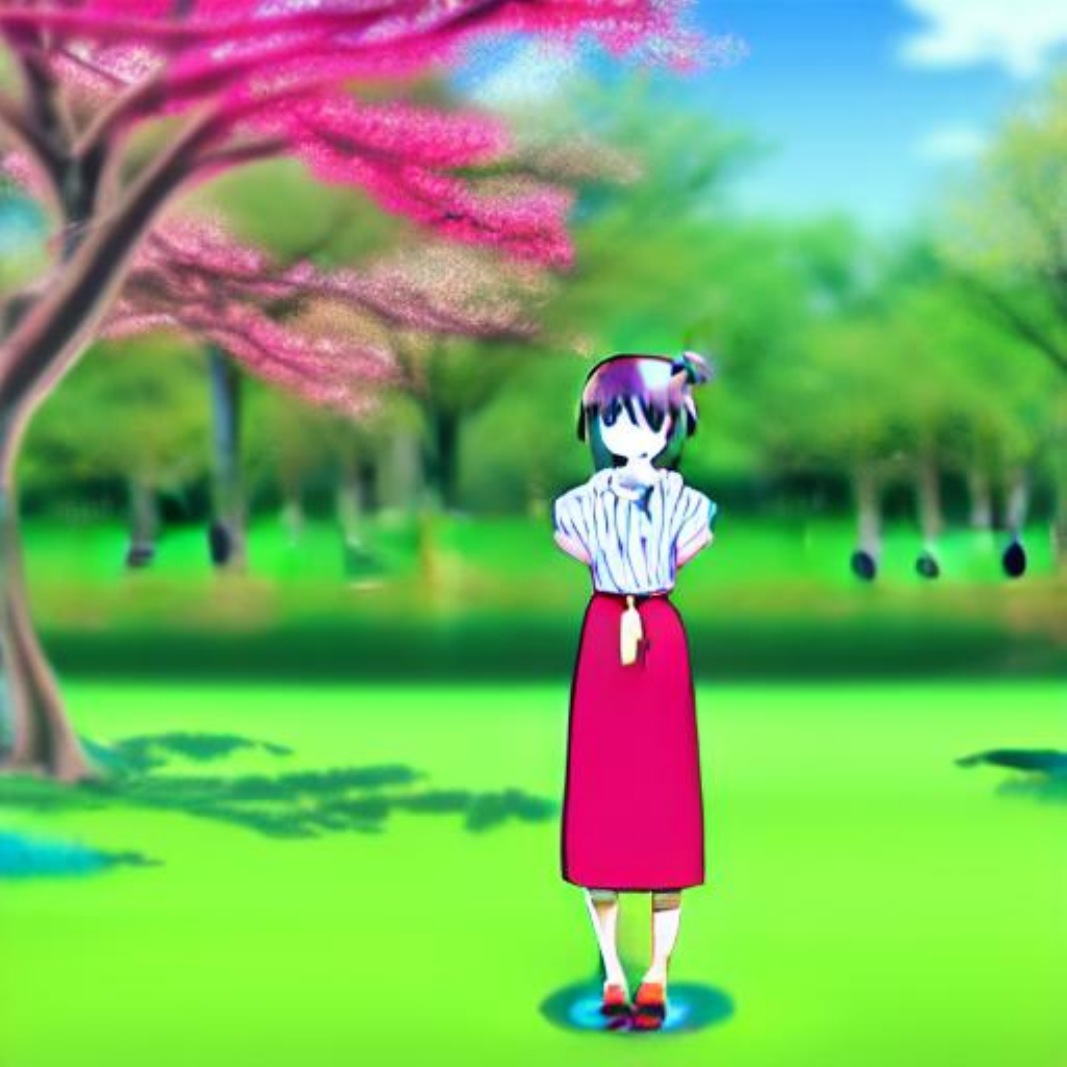}} & 
        \noindent\parbox[c]{0.14\columnwidth}{\includegraphics[width=0.14\columnwidth]{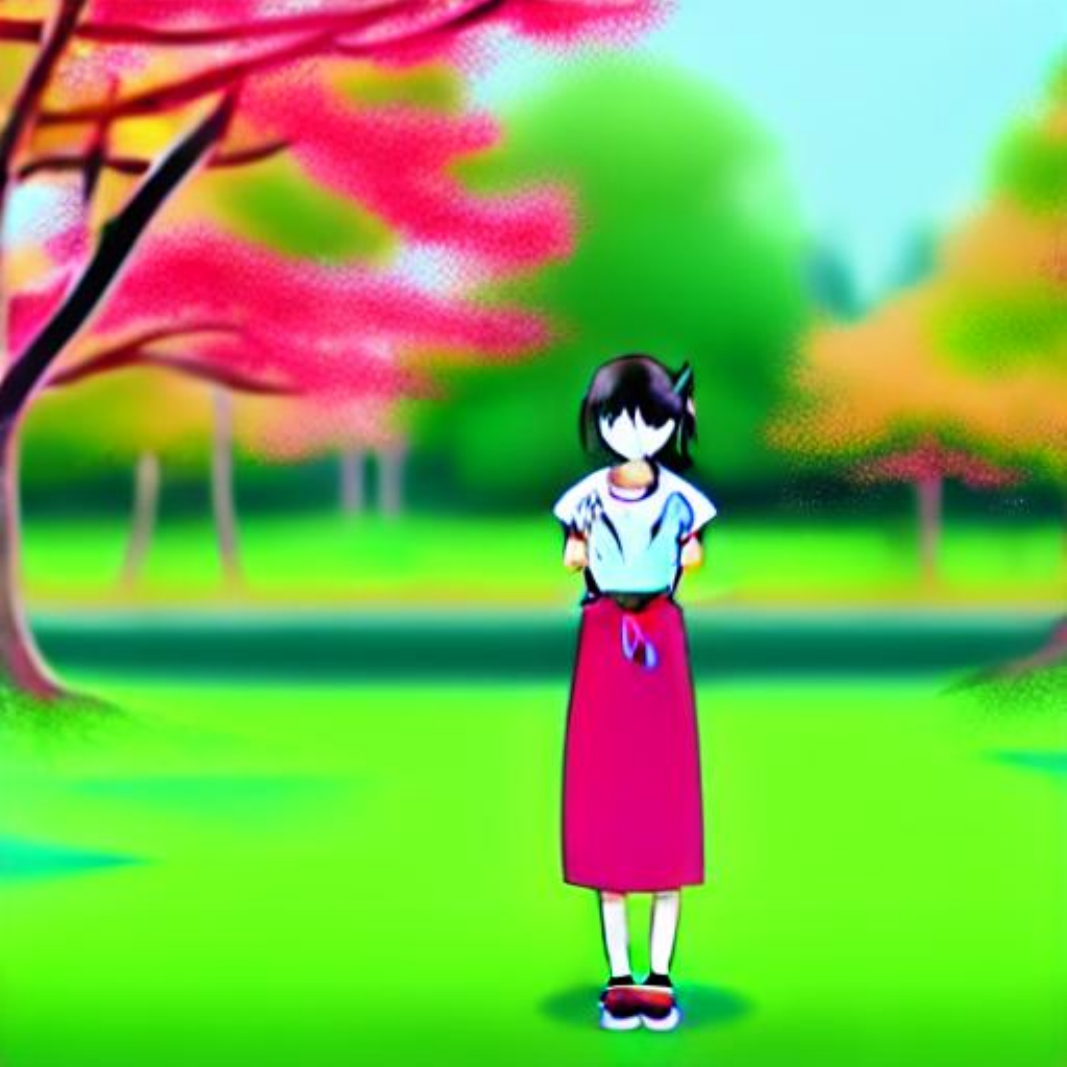}} \\

        \shortstack[l]{\tiny 15 steps} &
        \noindent\parbox[c]{0.14\columnwidth}{\includegraphics[width=0.14\columnwidth]{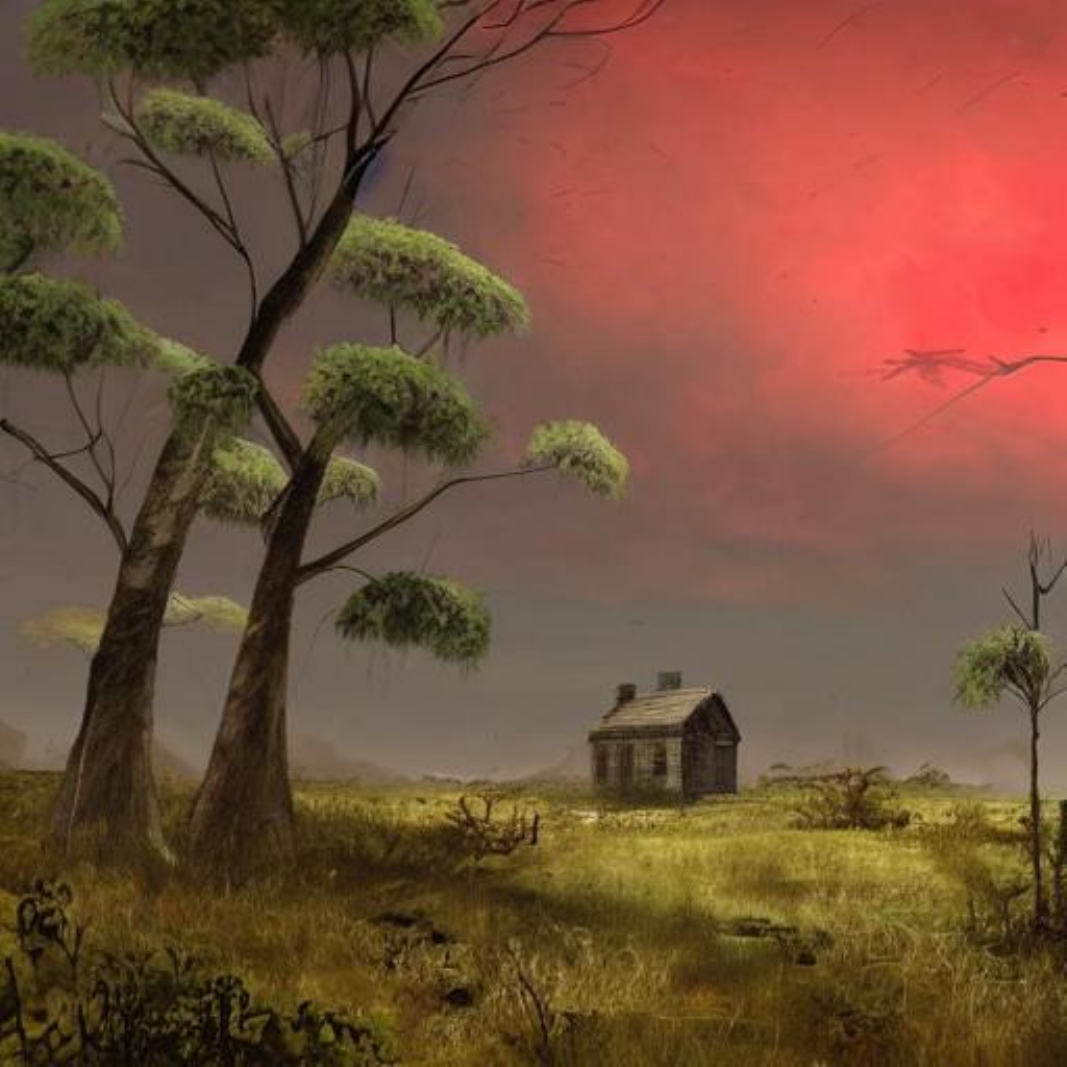}} & 
        \noindent\parbox[c]{0.14\columnwidth}{\includegraphics[width=0.14\columnwidth]{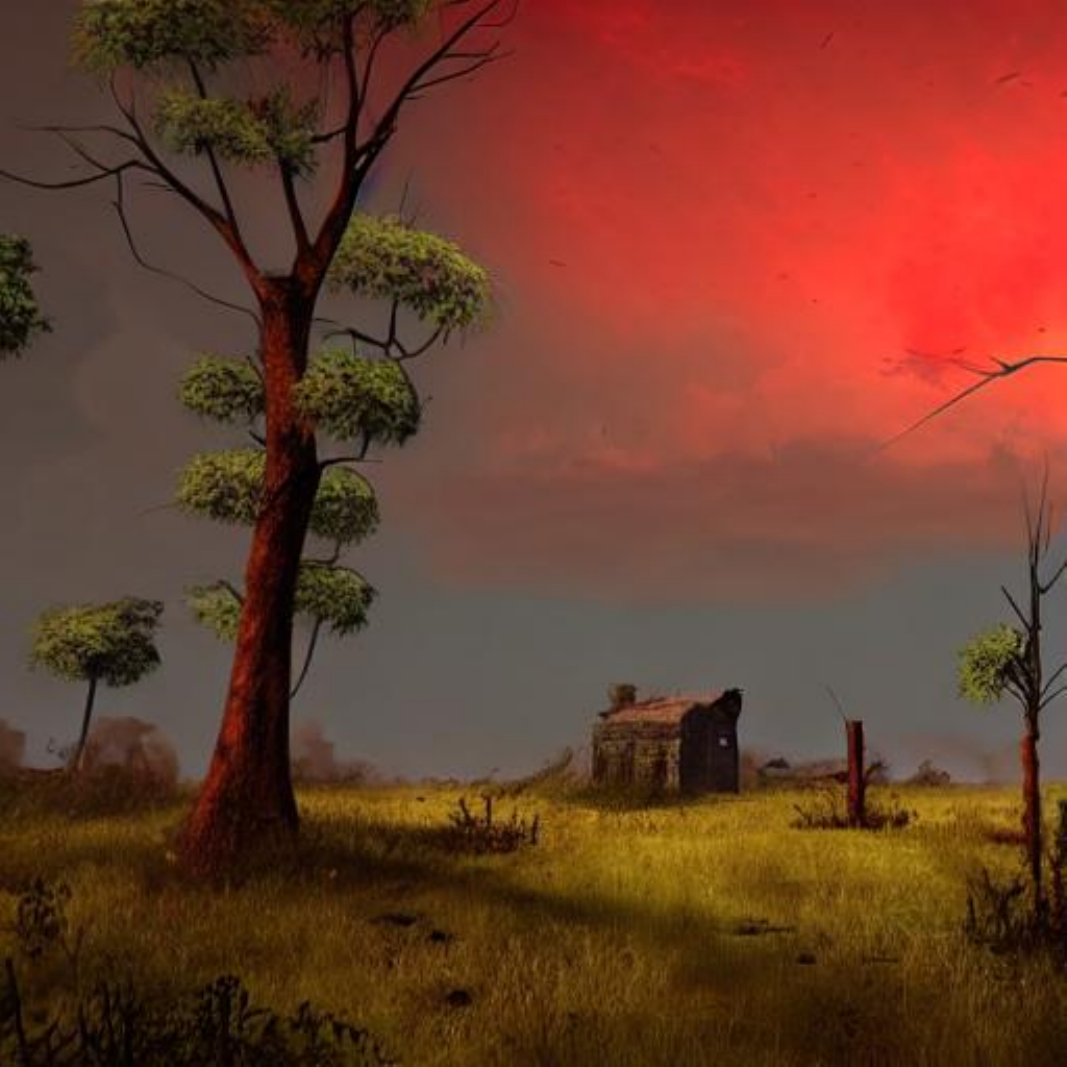}} & 
        \noindent\parbox[c]{0.14\columnwidth}{\includegraphics[width=0.14\columnwidth]{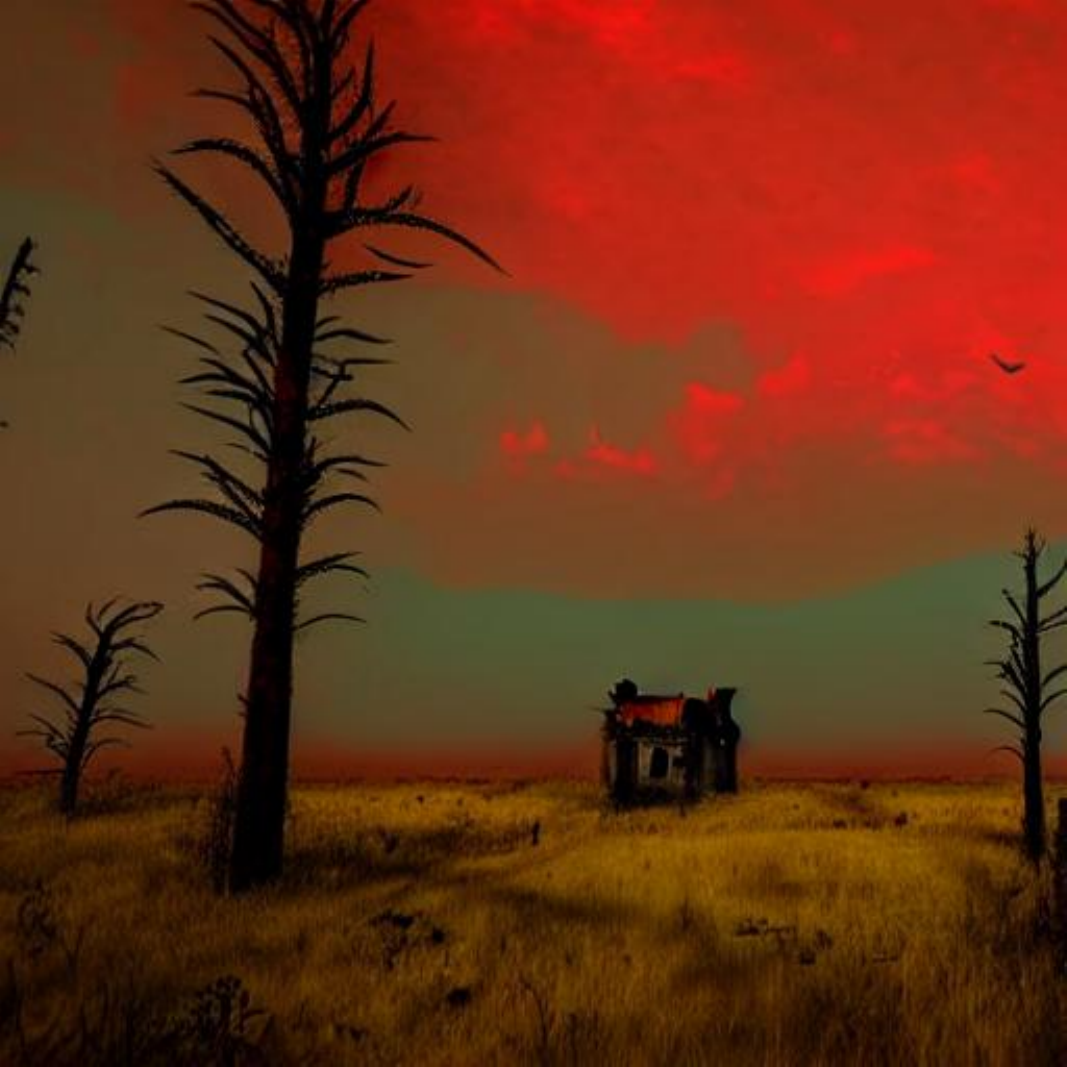}} & 
        \noindent\parbox[c]{0.14\columnwidth}{\includegraphics[width=0.14\columnwidth]{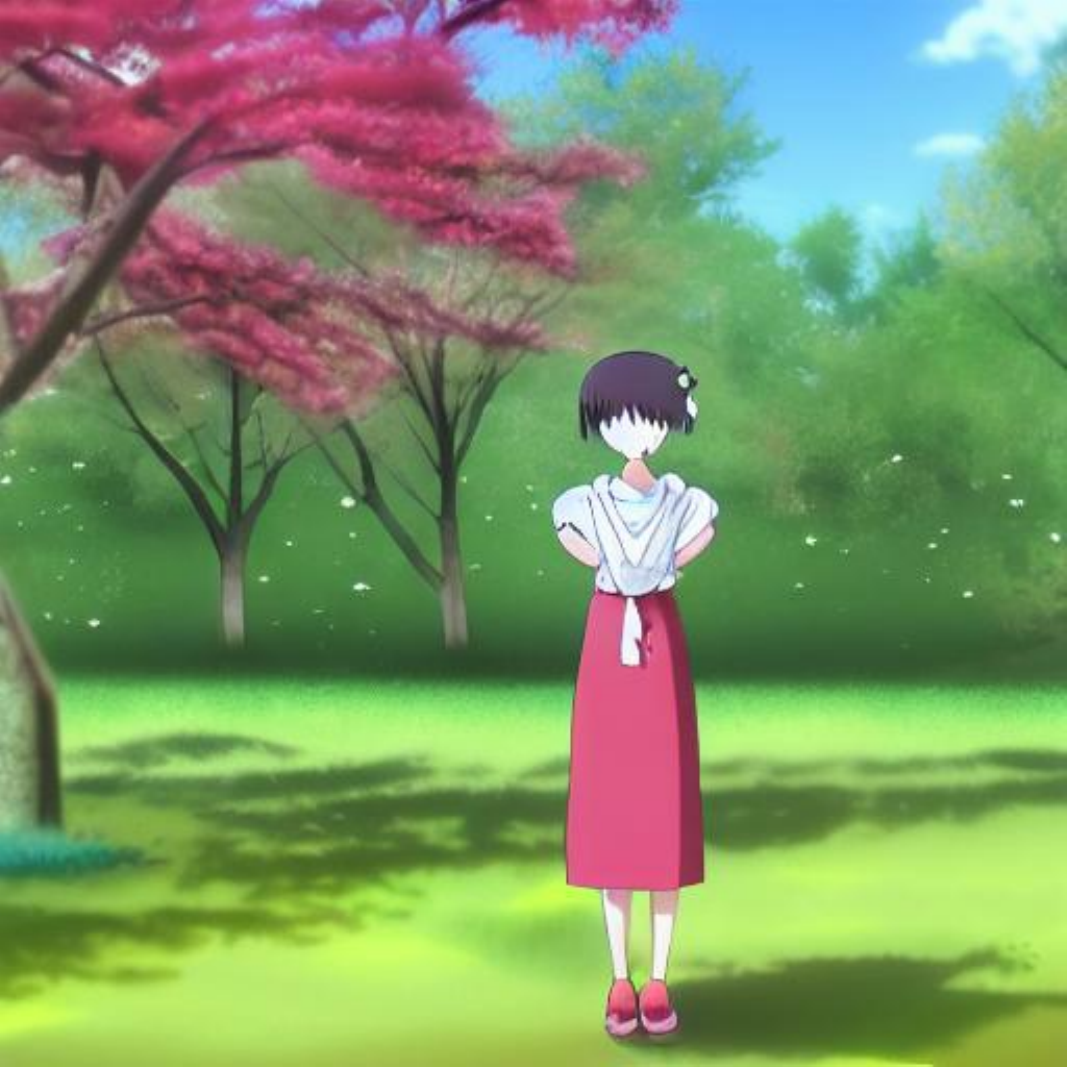}} & 
        \noindent\parbox[c]{0.14\columnwidth}{\includegraphics[width=0.14\columnwidth]{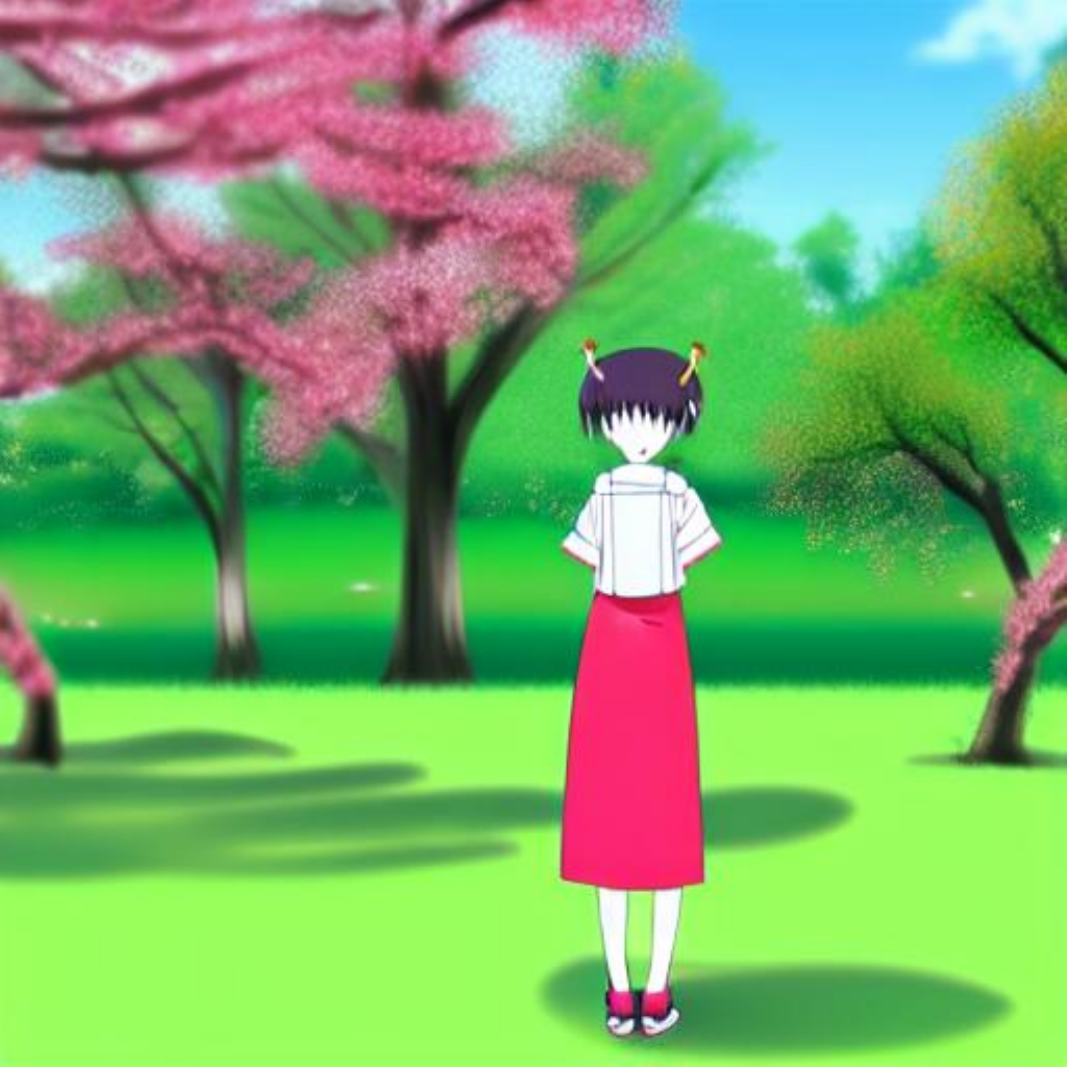}} & 
        \noindent\parbox[c]{0.14\columnwidth}{\includegraphics[width=0.14\columnwidth]{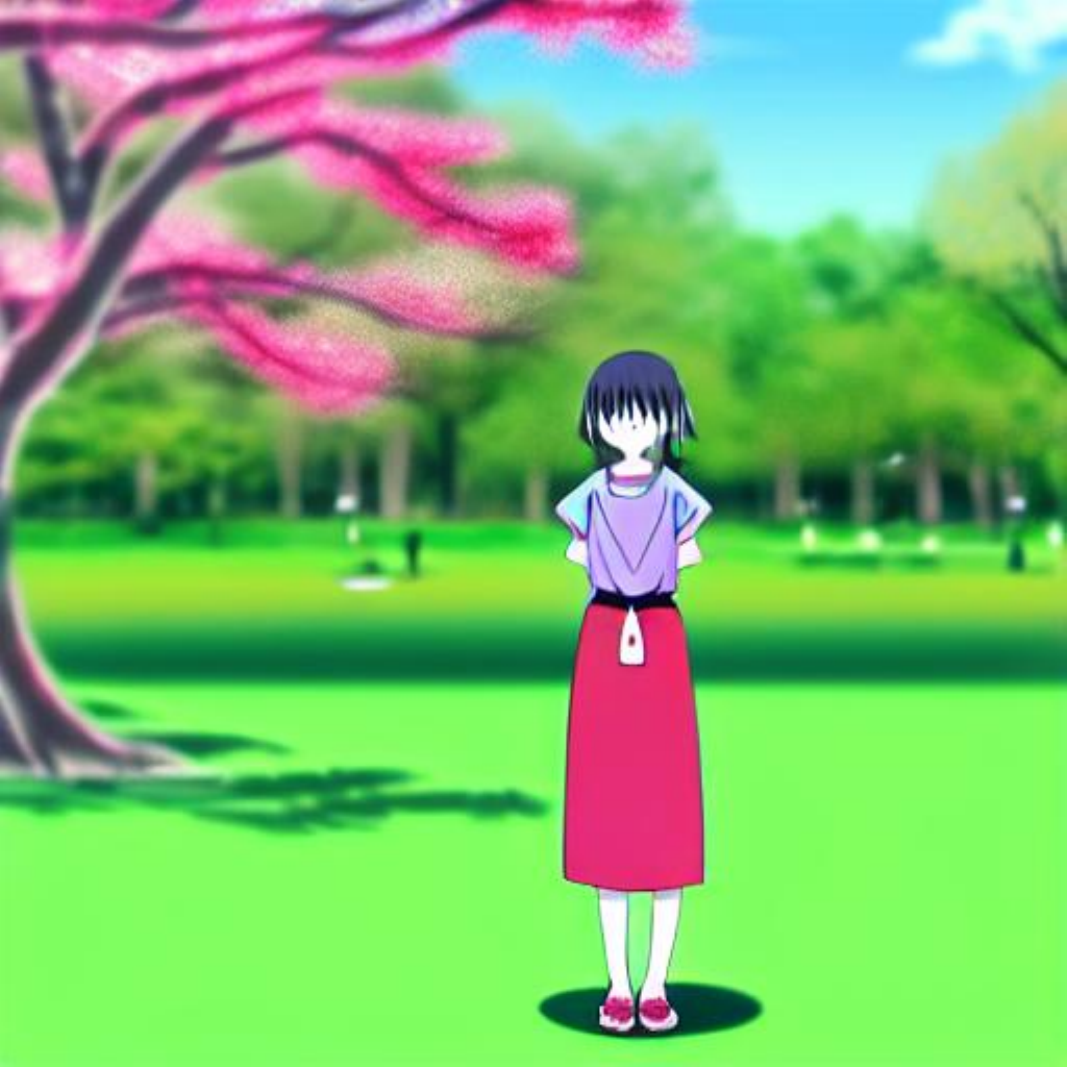}} \\

        \shortstack[l]{\tiny 20 steps} &
        \noindent\parbox[c]{0.14\columnwidth}{\includegraphics[width=0.14\columnwidth]{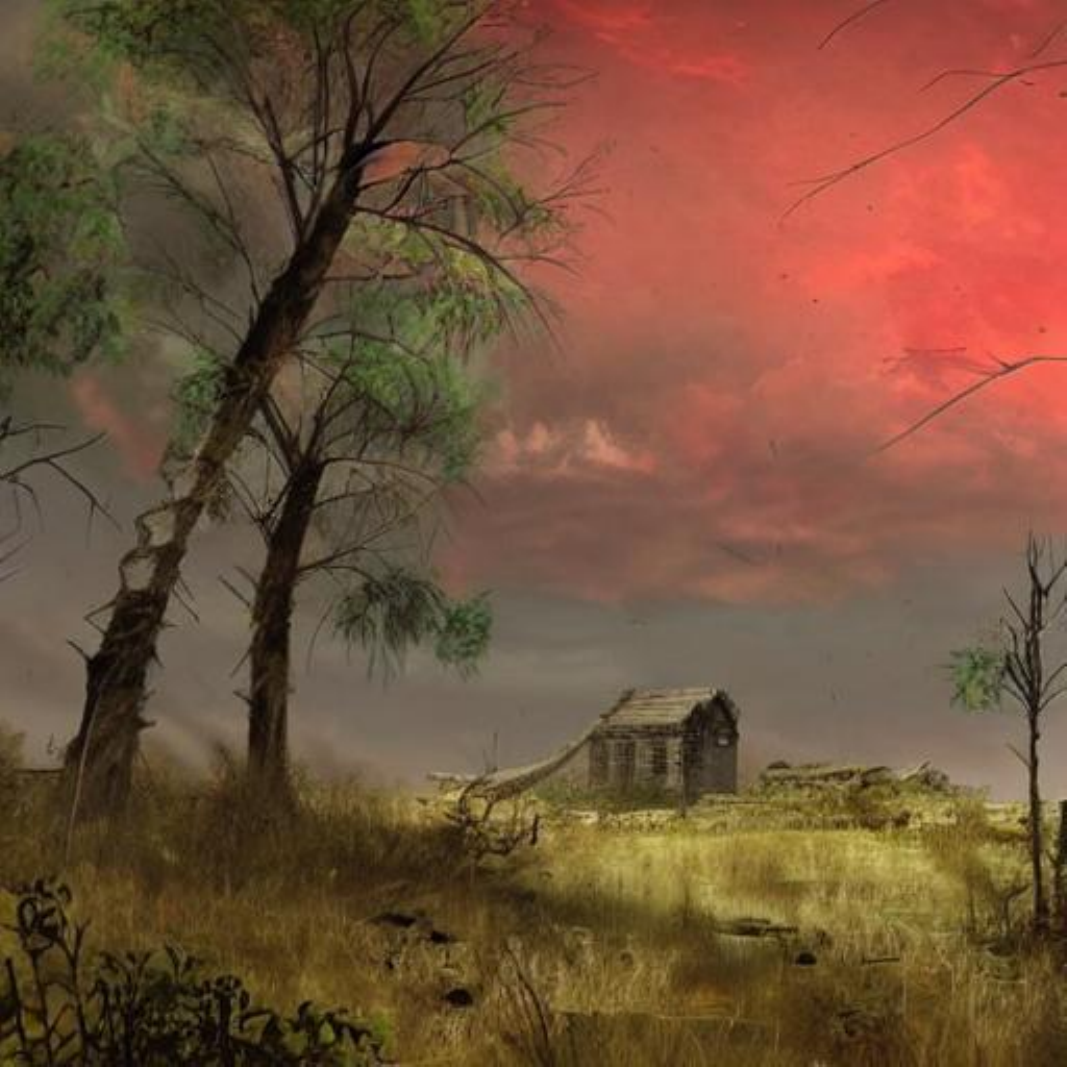}} & 
        \noindent\parbox[c]{0.14\columnwidth}{\includegraphics[width=0.14\columnwidth]{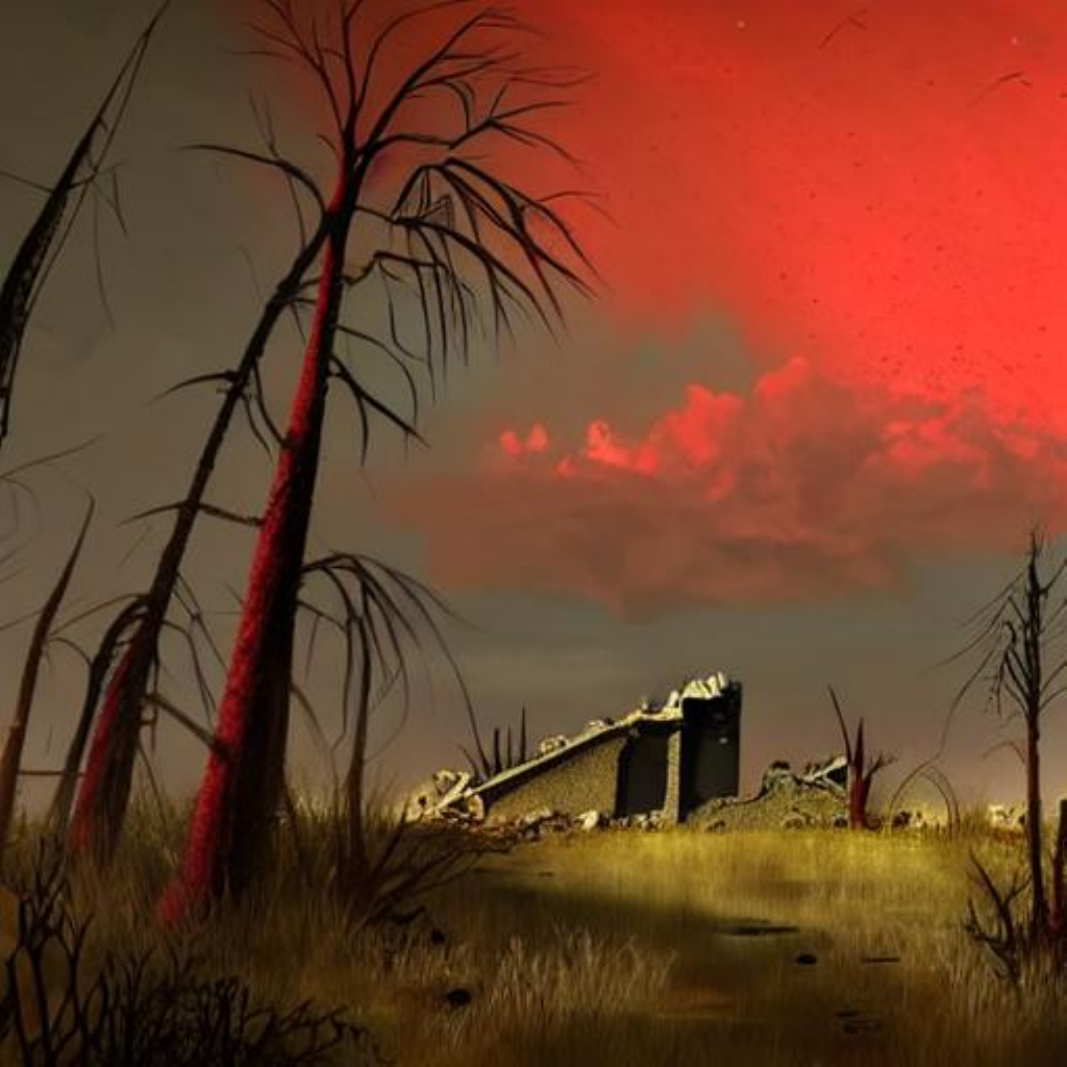}} & 
        \noindent\parbox[c]{0.14\columnwidth}{\includegraphics[width=0.14\columnwidth]{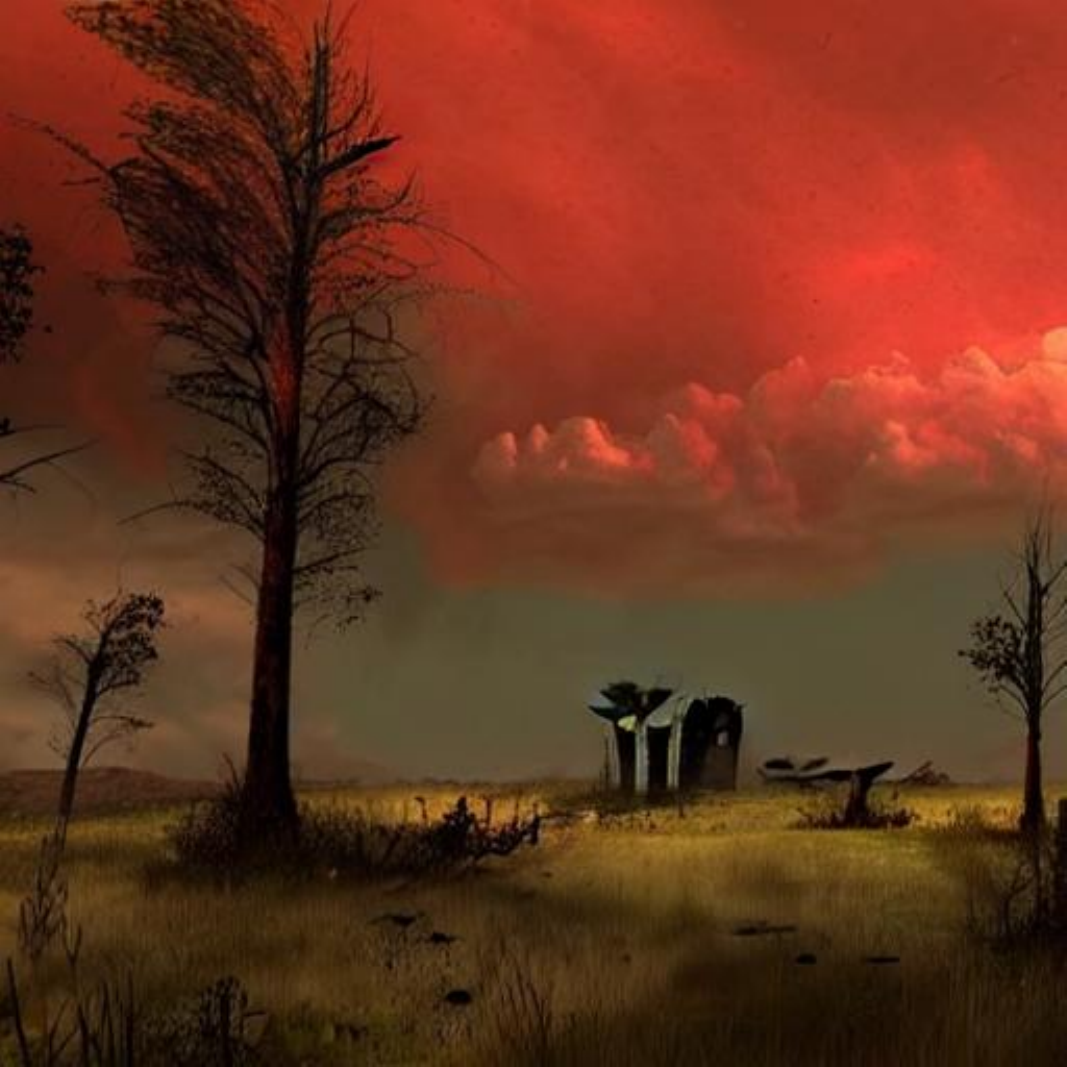}} & 
        \noindent\parbox[c]{0.14\columnwidth}{\includegraphics[width=0.14\columnwidth]{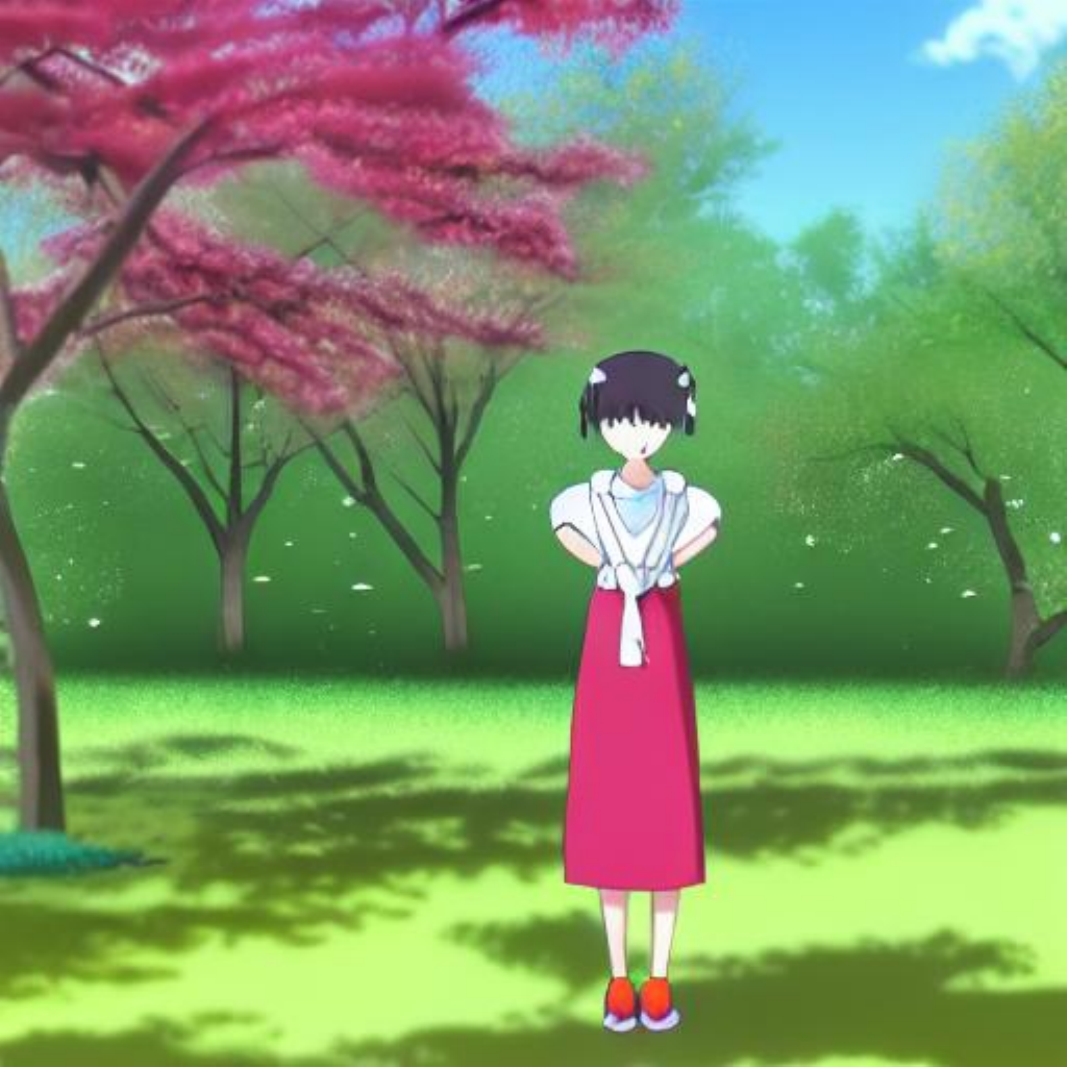}} & 
        \noindent\parbox[c]{0.14\columnwidth}{\includegraphics[width=0.14\columnwidth]{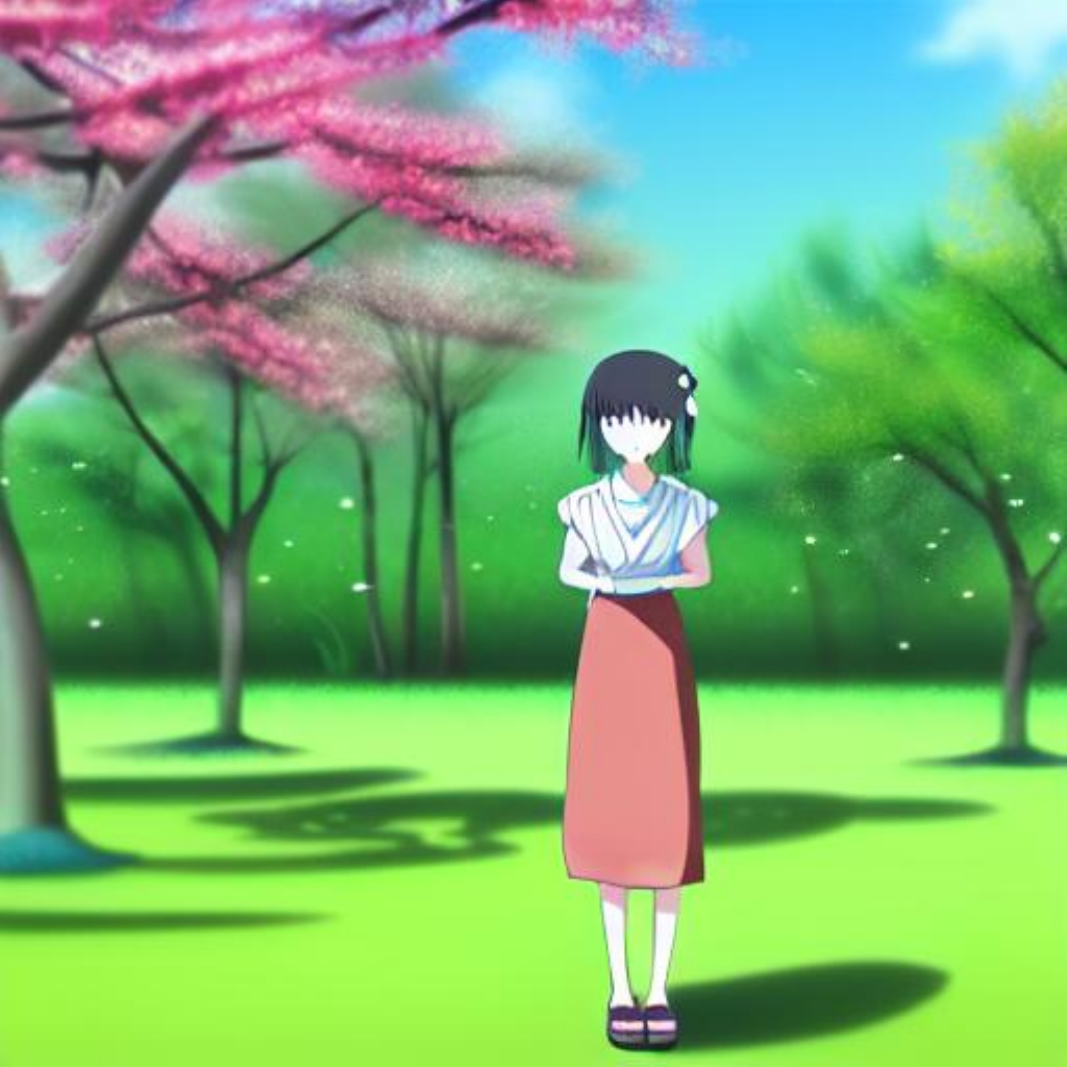}} & 
        \noindent\parbox[c]{0.14\columnwidth}{\includegraphics[width=0.14\columnwidth]{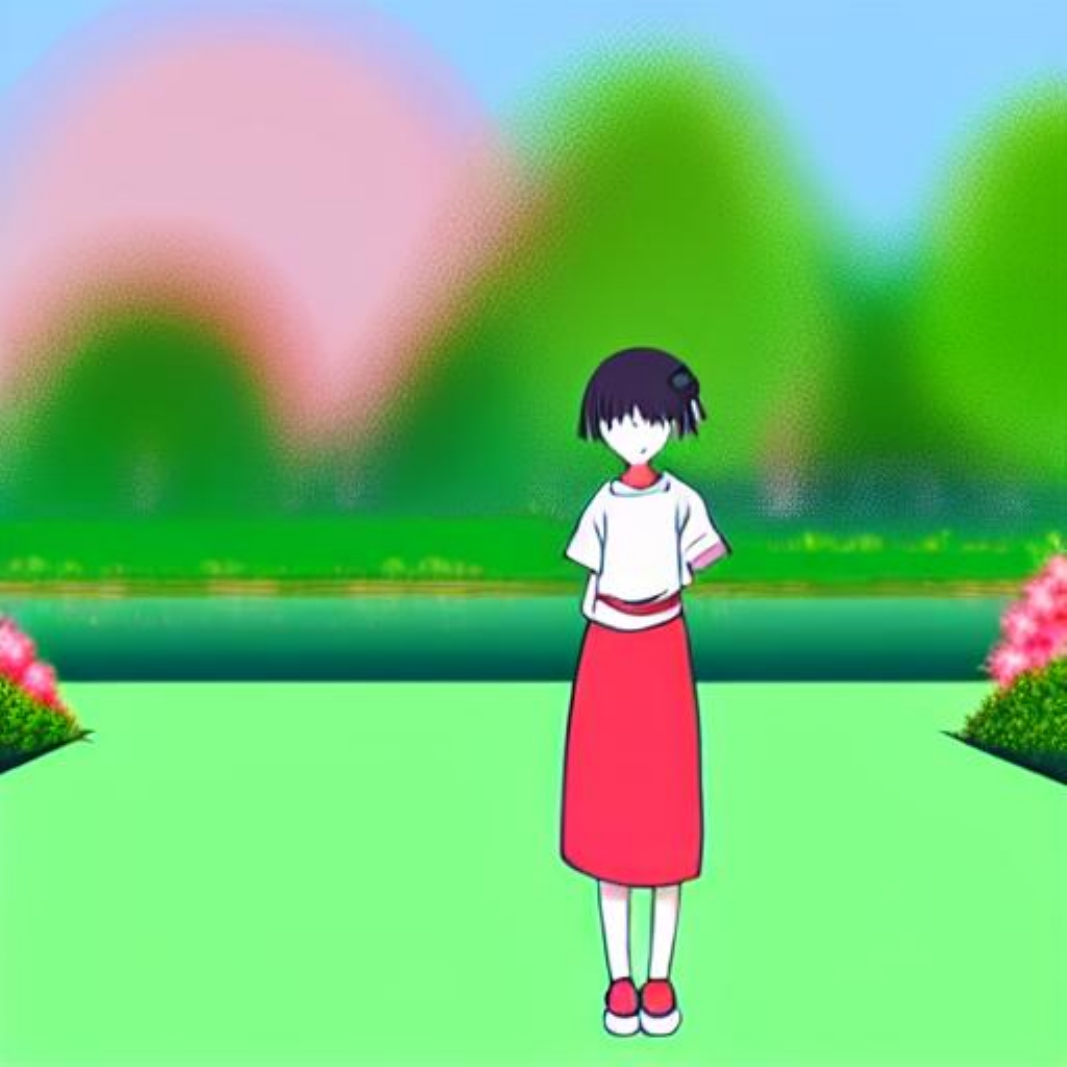}} \\

        \shortstack[l]{\tiny 40 steps} &
        \noindent\parbox[c]{0.14\columnwidth}{\includegraphics[width=0.14\columnwidth]{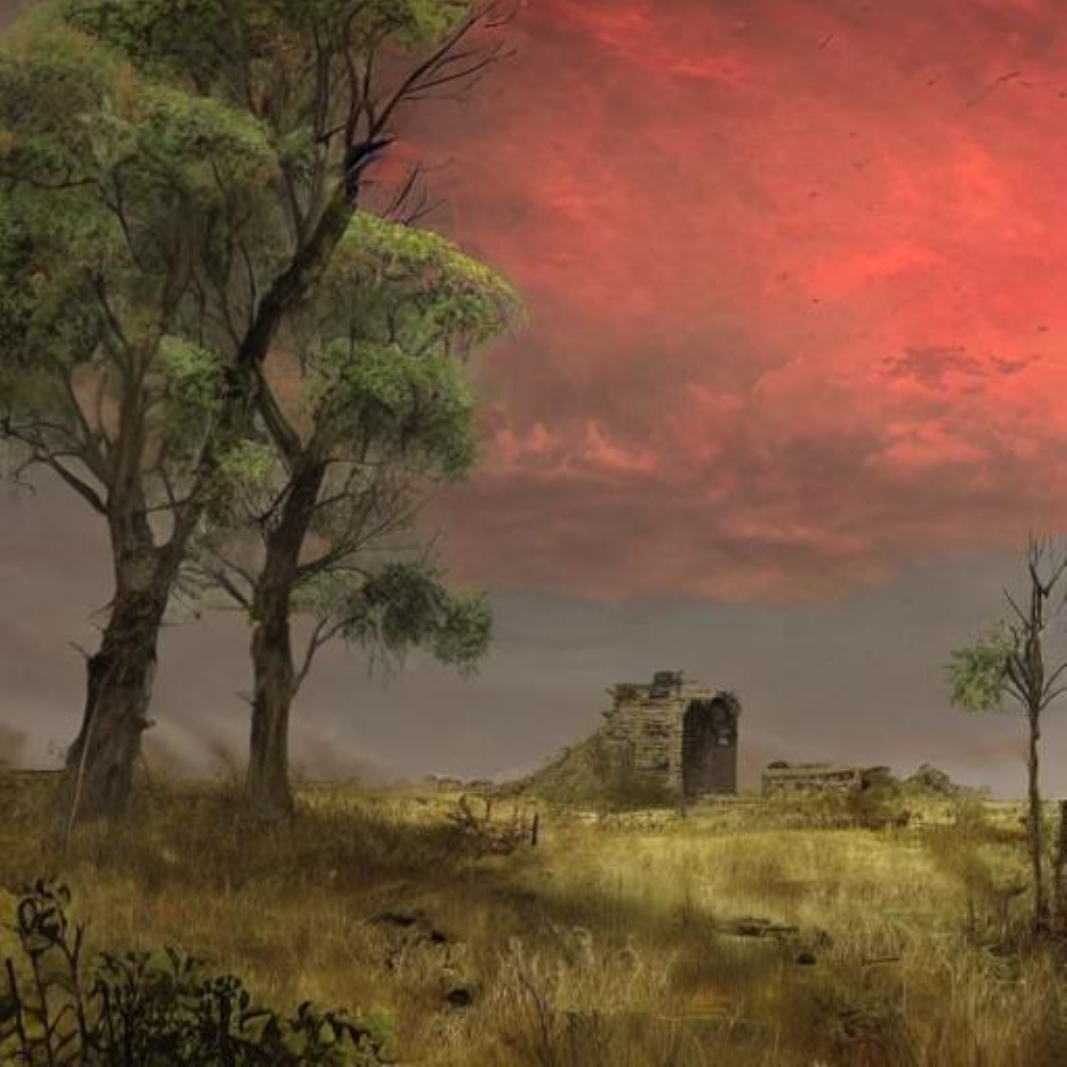}} & 
        \noindent\parbox[c]{0.14\columnwidth}{\includegraphics[width=0.14\columnwidth]{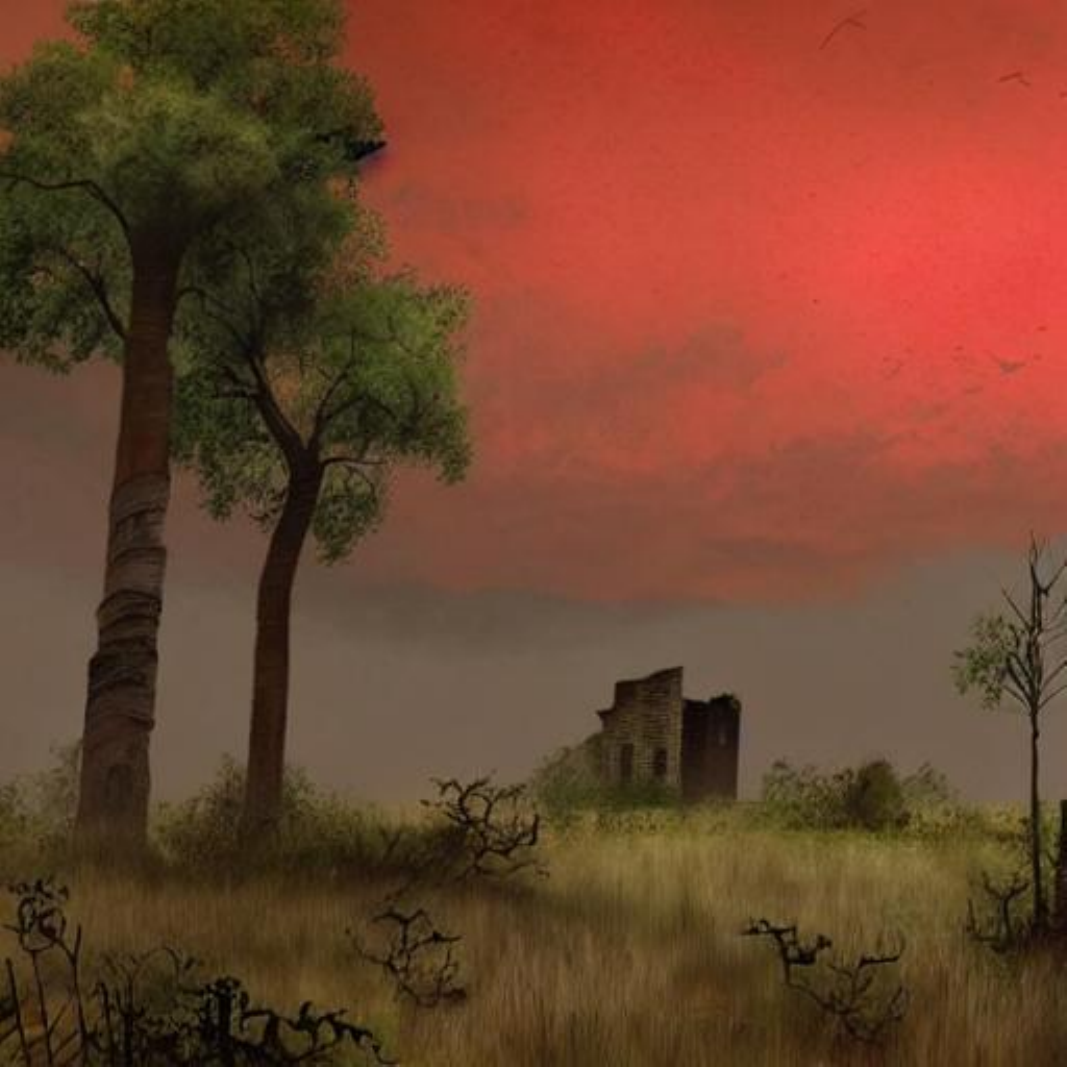}} & 
        \noindent\parbox[c]{0.14\columnwidth}{\includegraphics[width=0.14\columnwidth]{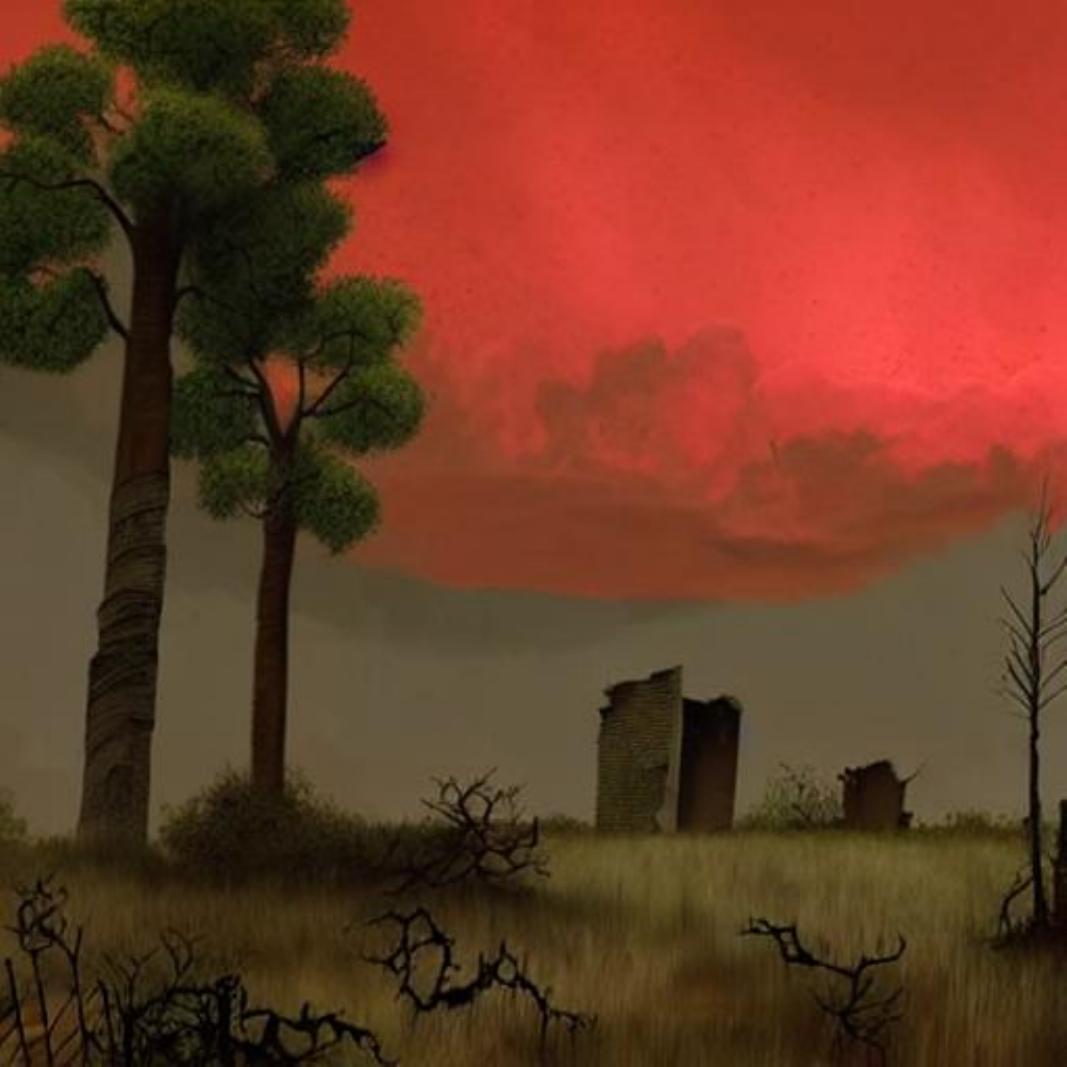}} & 
        \noindent\parbox[c]{0.14\columnwidth}{\includegraphics[width=0.14\columnwidth]{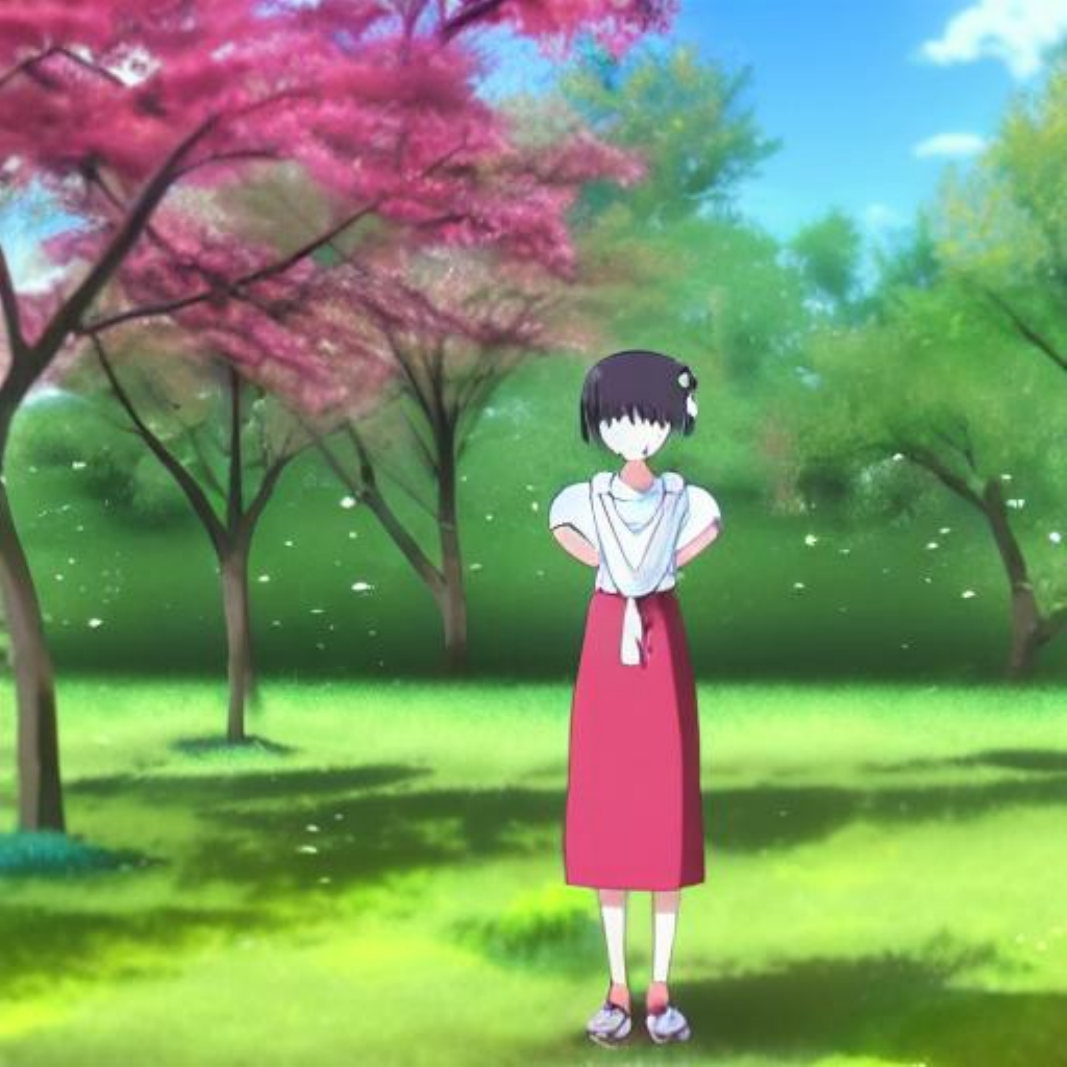}} & 
        \noindent\parbox[c]{0.14\columnwidth}{\includegraphics[width=0.14\columnwidth]{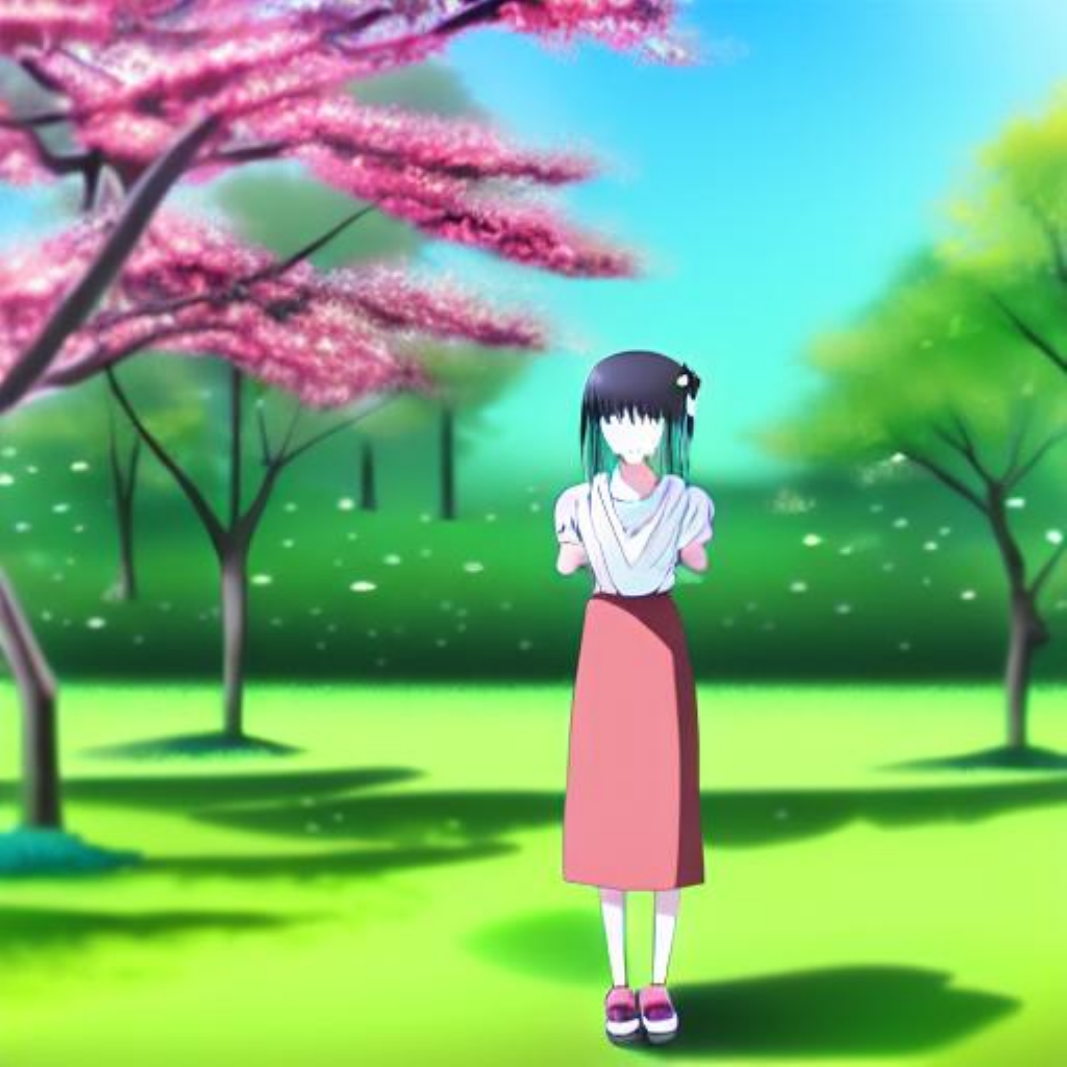}} & 
        \noindent\parbox[c]{0.14\columnwidth}{\includegraphics[width=0.14\columnwidth]{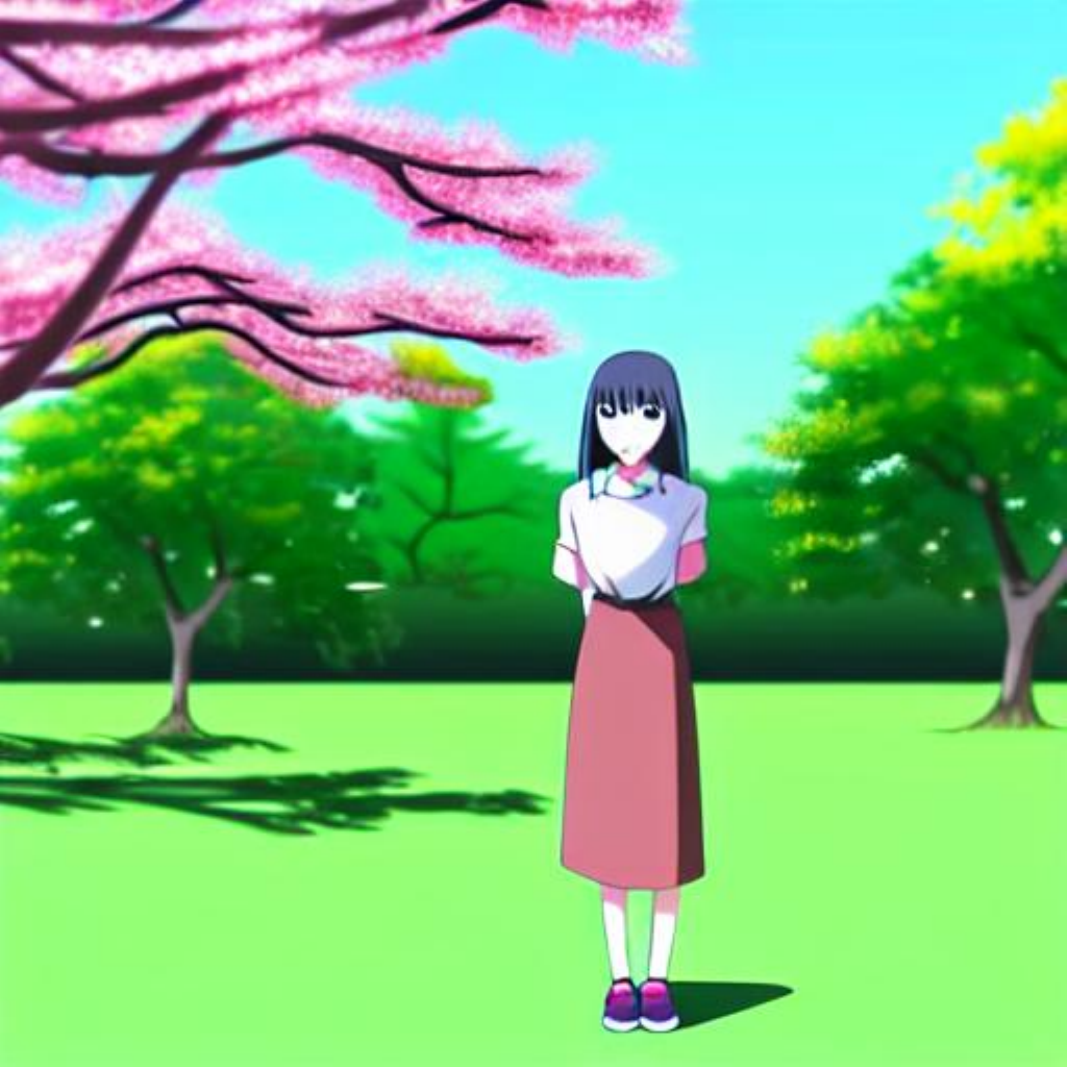}} \\
    \end{tabu}
    \caption{Comparison of samples generated from Stable Diffusion 1.5 using GHVB($3.0 + \beta$) under various sampling steps and guidance scale $s$. Specifically, we employ $\beta = 0.5$ for $s = 7.5$, $\beta = 0.2$ for $s = 15$, and $\beta = 0.1$ for $s = 22.5$ to account for the varying degrees of artifact manifestation associated with each guidance scale.}
    \label{fig:scale_step_sd15_ghvb}
\end{figure}


\section{Elaboration on the Order of Convergence Approximation} \label{apx:order}

In Appendix \ref{apx:conv}, we explored the theoretical aspects of the order of convergence for numerical methods. In this section, we will delve into the estimation of the order of convergence specifically for GHVB in Section \ref{sec:ghvb}.

To assess the order of convergence, we focus on the error $e$, referred to as the global truncation error. This error is quantified by measuring the absolute difference in the latent space between the numerical solution and an approximate exact solution obtained through 1,000-step PLMS4 sampling. The order of convergence for a numerical method is defined as $q$, where the error $e$ follows the relationship $e=\mathcal{O}(\delta^q)$, with $\delta$ representing the step size.

To estimate the order of convergence practically, we adopt a straightforward approach. It involves selecting two distinct step sizes, denoted as $\delta_{\text{new}}$ and $\delta_{\text{old}}$, and computing the corresponding errors $e_{\text{new}}$ and $e_{\text{old}}$. These errors can be approximated using the following formulas:

\begin{align}
e_{\text{new}} \approx C_{\text{new}} (\delta_{\text{new}})^q, \quad \quad e_{\text{old}} \approx C_{\text{old}} (\delta_{\text{old}})^q
\end{align}

Here, we make the assumption that $C_{\text{new}}$ is approximately equal to $C_{\text{old}}$. By taking the ratio of $e_{\text{new}}$ to $e_{\text{old}}$, we obtain:

\begin{align}
\frac{e_{\text{new}}}{e_{\text{old}}} \approx \left(\frac{\delta_{\text{new}}}{\delta_{\text{old}}}\right)^q
\end{align}

Consequently, we can estimate the order of convergence, denoted as $q$, by evaluating the logarithmic ratio of errors and step sizes:

\begin{align}
q \approx \frac{\log(e_{\text{new}}/e_{\text{old}})}{\log(\delta_{\text{new}}/\delta_{\text{old}})}
\end{align}

In our investigation of GHVB in Section \ref{sec:ghvb}, we conducted sampling experiments using 20, 40, 80, 160, 320, and 640 steps. This choice of an exponential sequence for the number of steps was intentional, as it allowed us to approximate $\delta_{\text{new}}/\delta_{\text{old}} \approx 1/2$. By doing so, we facilitated the estimation process. The results, representing the approximated order of convergence for GHVB, are visually depicted in Figure \ref{fig:ablation_on_ghvb}. 

\section{Ablation Study on HB Momentum} \label{apx:hb}

Incorporating Polyak's Heavy Ball (HB) momentum directly into existing diffusion sampling methods is a more straightforward approach to mitigating divergence artifacts than GHVB. This can be achieved by modifying a few lines of code. In this section, we conduct a comprehensive analysis of the convergence speed of this approach.

To evaluate its effectiveness, we generate target results using the 1,000-step PLMS4 method. 
We compare the target results with those obtained from several methods with and without HB momentum, using LPIPS in the image space and L2 in the latent space. We then estimate their orders of convergence, as explained in Appendix \ref{apx:order}. The results of this analysis are visually presented in Figure \ref{fig:abalation_hb}.

In contrast to the interpolation-like behavior observed in Figure \ref{fig:ablation_on_ghvb} for GHVB, we observe that the use of HB momentum leads to an increase in both the LPIPS score and the L2 distance when selecting values of $\beta$ that are less than 1. This is even worse than the 1\ts{st}-order method DDIM when $\beta$ is below 0.7. These findings indicate a deviation from the desired convergence behavior, highlighting a potential decrease in solution accuracy, even though HB momentum has been shown to successfully mitigate divergence artifacts.

Additionally, we find that the numerical orders of convergence also tend to approach the same value. These observations align with our analysis in Theorem \ref{thm:hb} of Appendix \ref{apx:conv}, indicating that when $\beta$ deviates from 1, the employed approach exhibits 1\ts{st}-order convergence and is unable to achieve high-order convergence. These conclusions emphasize the importance of carefully considering the choice of $\beta$ in order to strike a balance between convergence speed and solution quality. Further details and insights into the performance of the HB momentum approach can be obtained from Figure \ref{fig:abalation_hb}, enhancing our understanding of its behavior within the context of the studied problem.

\begin{figure}
    \centering
    \includegraphics[width=0.3\textwidth]{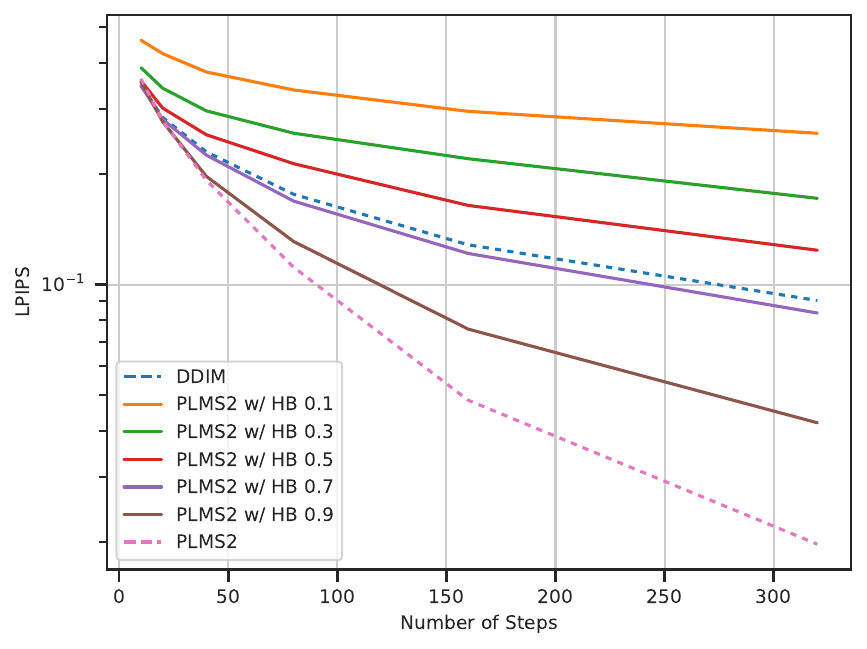}
    \includegraphics[width=0.3\textwidth]{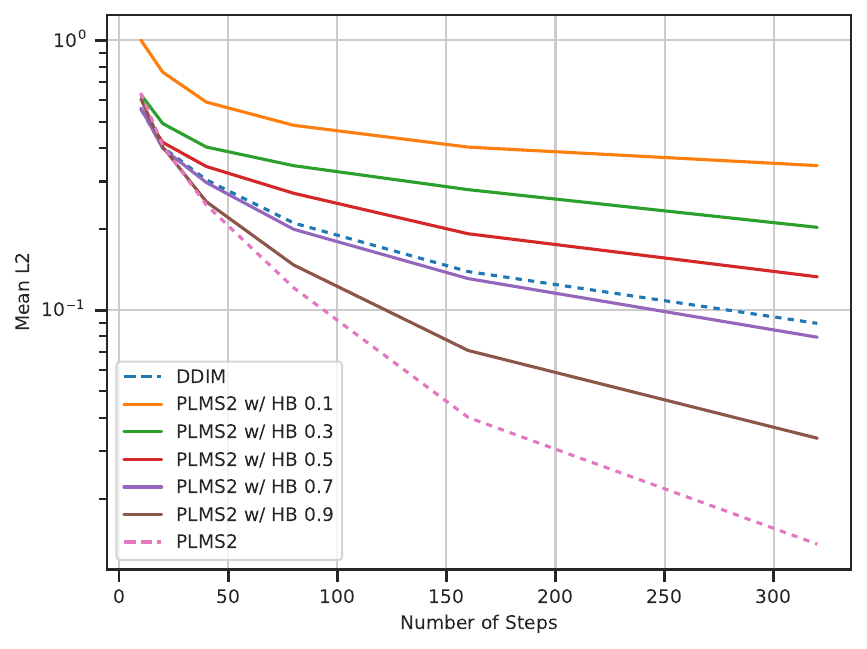}
    \includegraphics[width=0.3\textwidth]{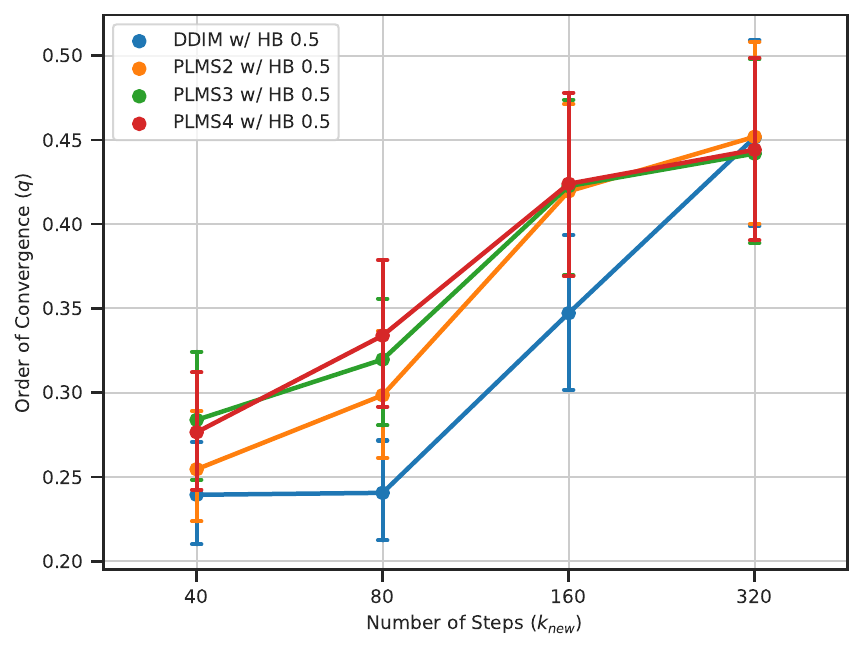}
    \caption{Comparison of LPIPS, mean L2 distance, and order of convergence of HB when using different damping coefficients. Statistical means are averaged from 160 initial latent codes.}
    \label{fig:abalation_hb}
\end{figure}

\section{Ablation Study on Nesterov Momentum}\label{apx:nest}

In Appendix \ref{apx:moment}, we investigated the potential of incorporating different types of momentum, such as Nesterov's momentum, into existing diffusion sampling methods to mitigate divergence artifacts. Similar to the analysis conducted in Section \ref{sec:exp_ghvb} and Appendix \ref{apx:hb}, the primary objective of this section is to explore the convergence speed of Nesterov's momentum by comparing two key metrics: LPIPS in the image space and L2 in the latent space.

Figure \ref{fig:abalation_nest} presents the results, which reveal intriguing parallels with the behavior of HB momentum observed in Figure \ref{fig:abalation_hb}. When Nesterov's momentum is applied to the PLMS2 method, the accuracy of the model progressively diminishes as the value of $\beta$ deviates from 1, as indicated by the corresponding increase in both LPIPS and L2 metrics. Notably, the model's accuracy drops below that of the DDIM when $\beta$ falls below 0.5.

Furthermore, our analysis of the order of convergence demonstrates that Nesterov's momentum does not achieve a high order of convergence, similar to HB momentum. These findings emphasize the importance of carefully considering the choice of momentum method, along with the specific values assigned to $\beta$, in order to strike an optimal balance between convergence speed and solution quality.

\begin{figure}
    \centering
    \includegraphics[width=0.3\textwidth]{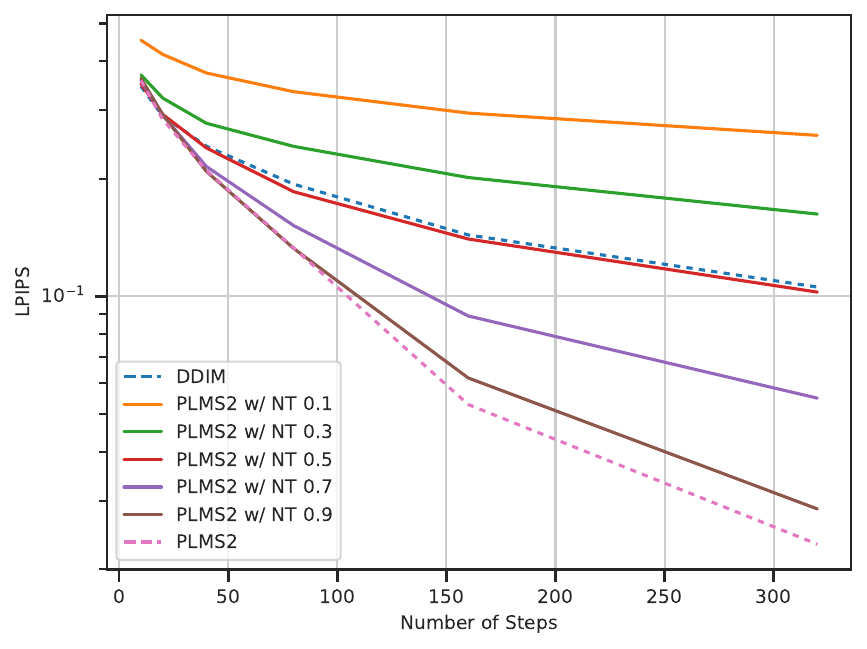}
    \includegraphics[width=0.3\textwidth]{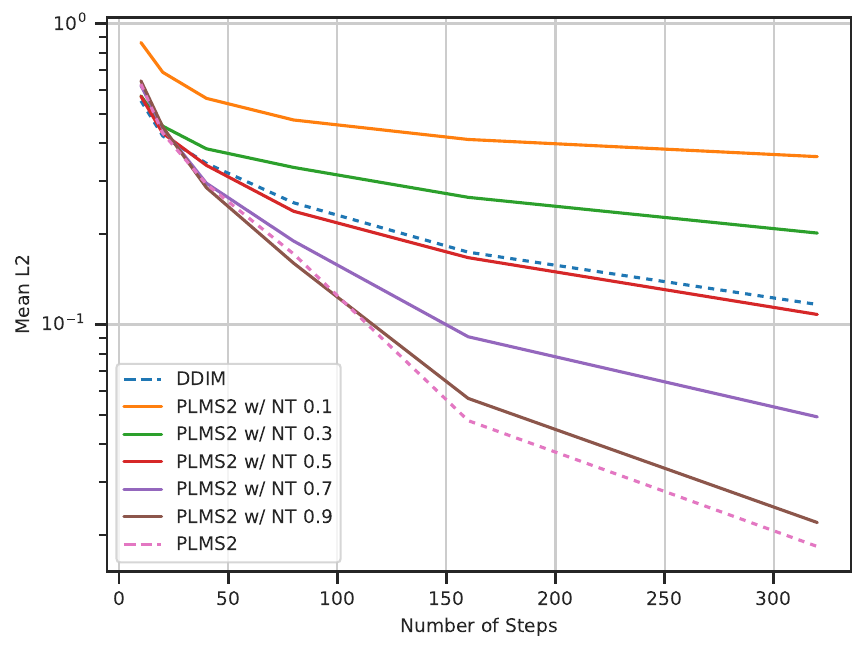}
    \includegraphics[width=0.3\textwidth]{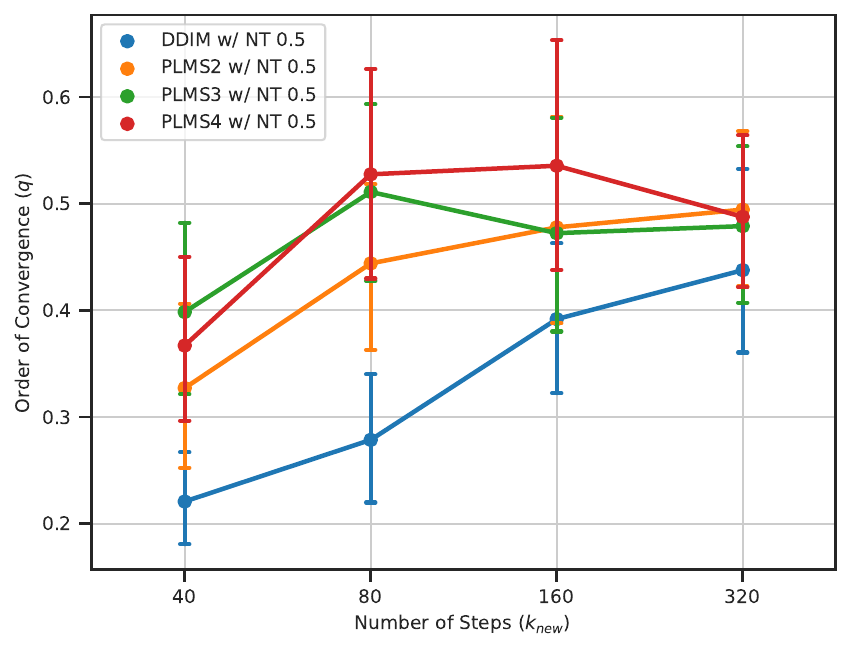}
    \caption{Comparison of LPIPS, mean L2 distance, and order of convergence of Nesterov's momentum when using different damping coefficients. Statistical means are averaged from 160 initial latent codes.}
    \label{fig:abalation_nest}
\end{figure}

\tabulinesep=1pt
\begin{figure}[ht]
    \centering
    \subcaptionbox{Prompt: "A beautiful illustration of a schoolgirl riding her bicycle to school in a small village"}{
        \begin{tabu} to \textwidth {
            @{}
            l@{\hspace{6pt}}
            c@{\hspace{2pt}}
            c@{\hspace{2pt}}
            c@{\hspace{2pt}}
            c@{}c@{\hspace{5pt}} | @{\hspace{5pt}}
            c@{}
        }
            & \multicolumn{1}{c}{\shortstack{\scriptsize $\beta = 0.2$}}
            & \multicolumn{1}{c}{\shortstack{\scriptsize $\beta = 0.4$}}
            & \multicolumn{1}{c}{\shortstack{\scriptsize $\beta = 0.6$}}
            & \multicolumn{1}{c}{\shortstack{\scriptsize $\beta = 0.8$}} &
            & \multicolumn{1}{c}{\shortstack{\scriptsize $\beta = 1.0$}} \\
    
            \shortstack[l]{\scriptsize (a) w/ HB} &
            \noindent\parbox[c]{0.17\columnwidth}{\includegraphics[width=0.17\columnwidth]{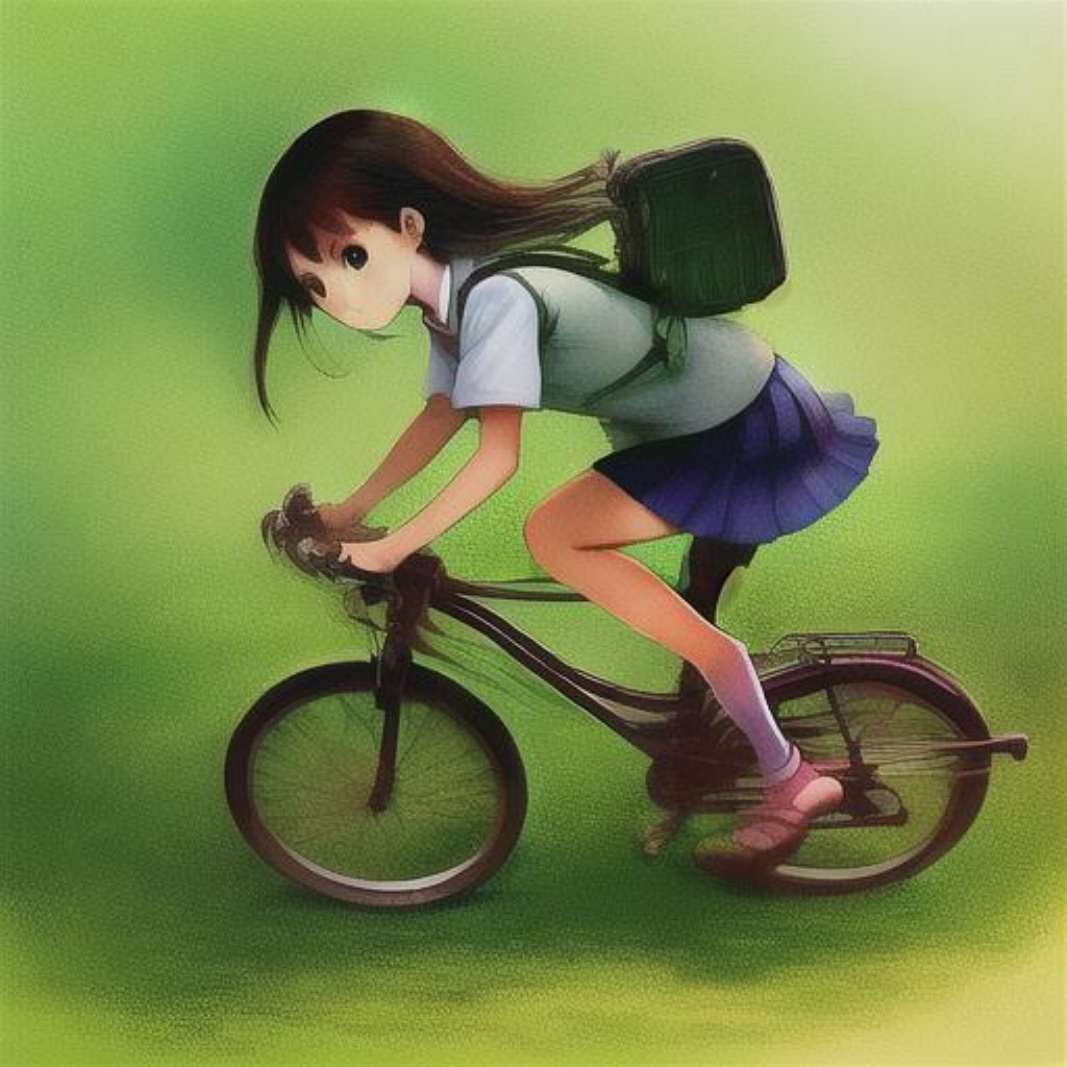}} & 
            \noindent\parbox[c]{0.17\columnwidth}{\includegraphics[width=0.17\columnwidth]{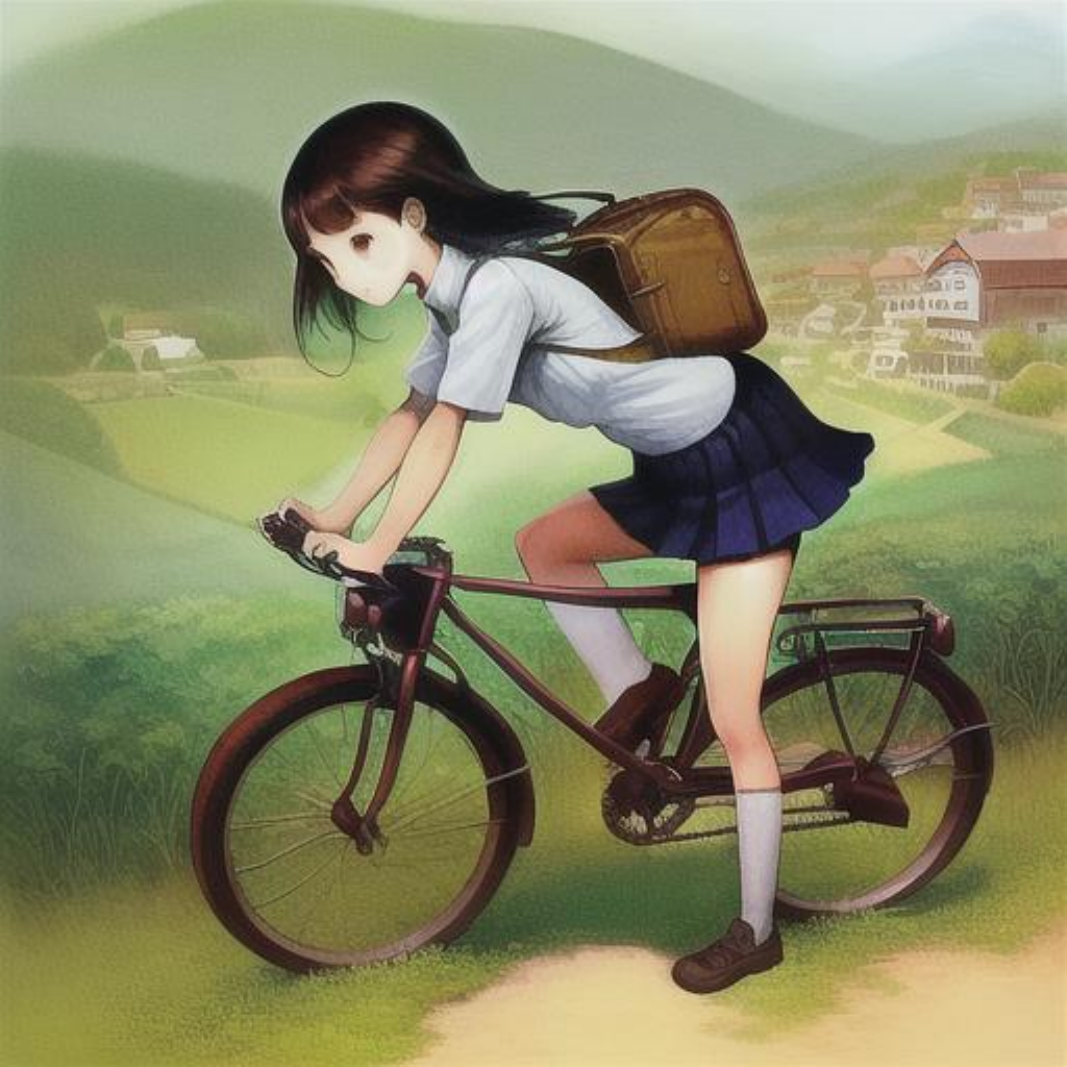}} & 
            \noindent\parbox[c]{0.17\columnwidth}{\includegraphics[width=0.17\columnwidth]{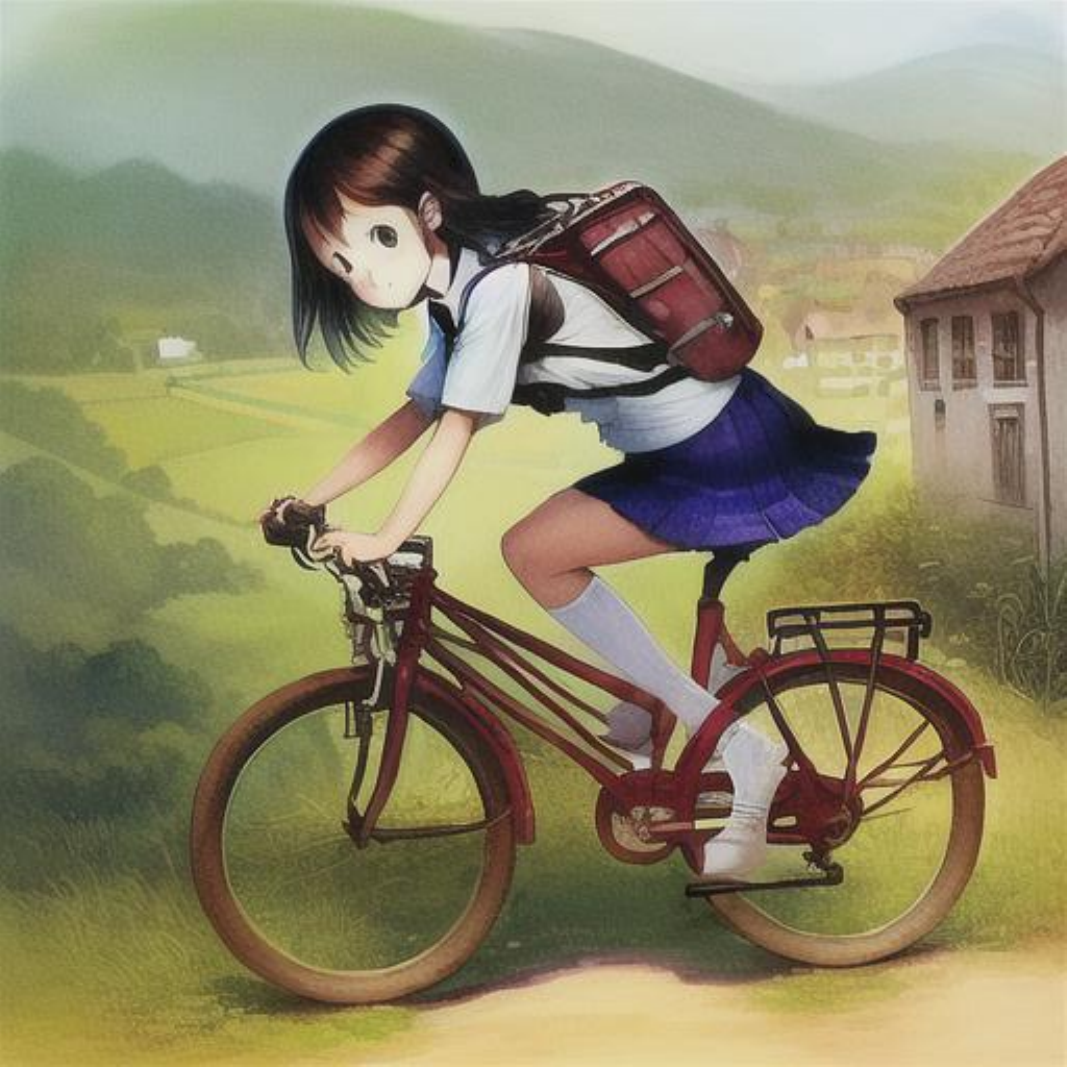}} & 
            \noindent\parbox[c]{0.17\columnwidth}{\includegraphics[width=0.17\columnwidth]{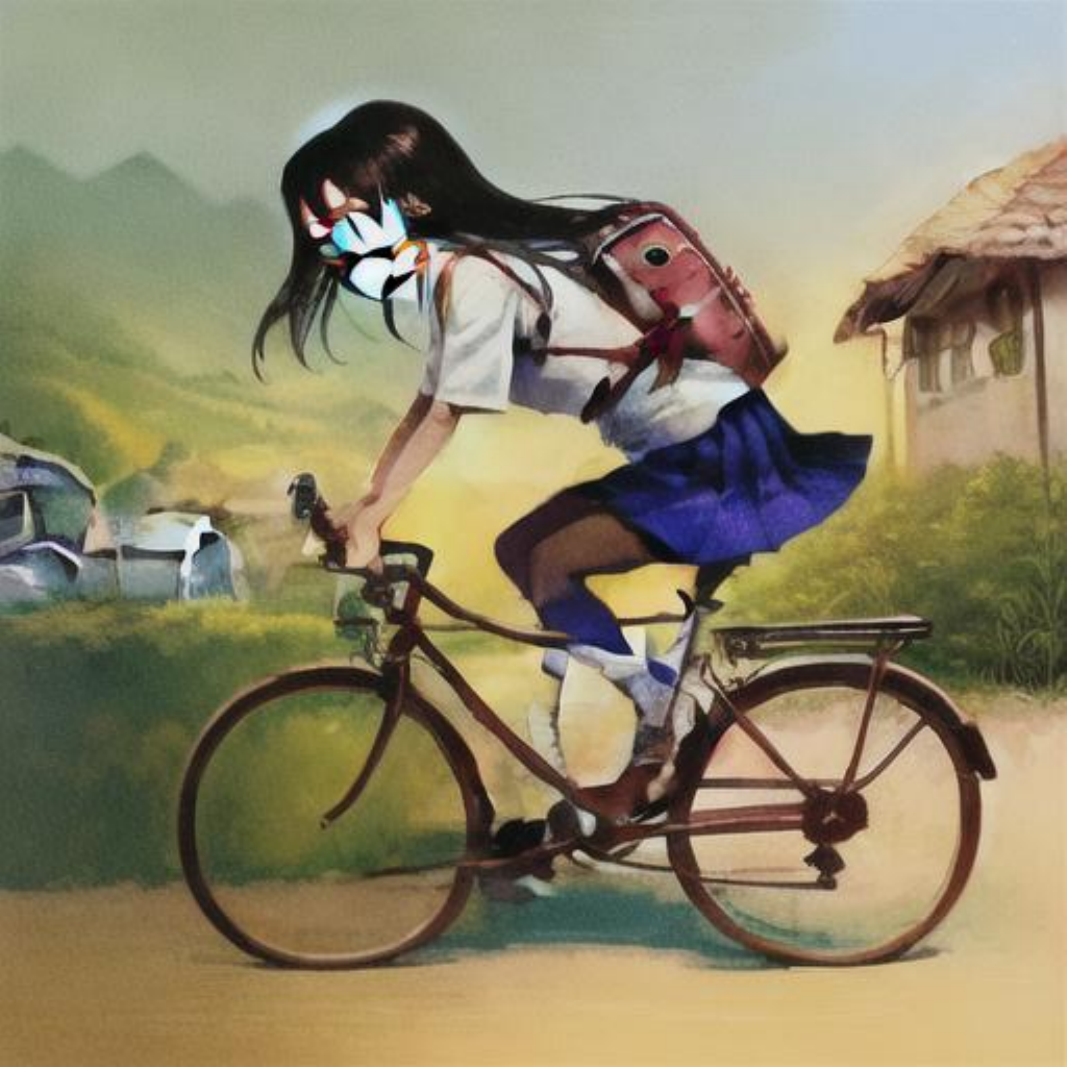}} & &
            \noindent\parbox[c]{0.17\columnwidth}{\includegraphics[width=0.17\columnwidth]{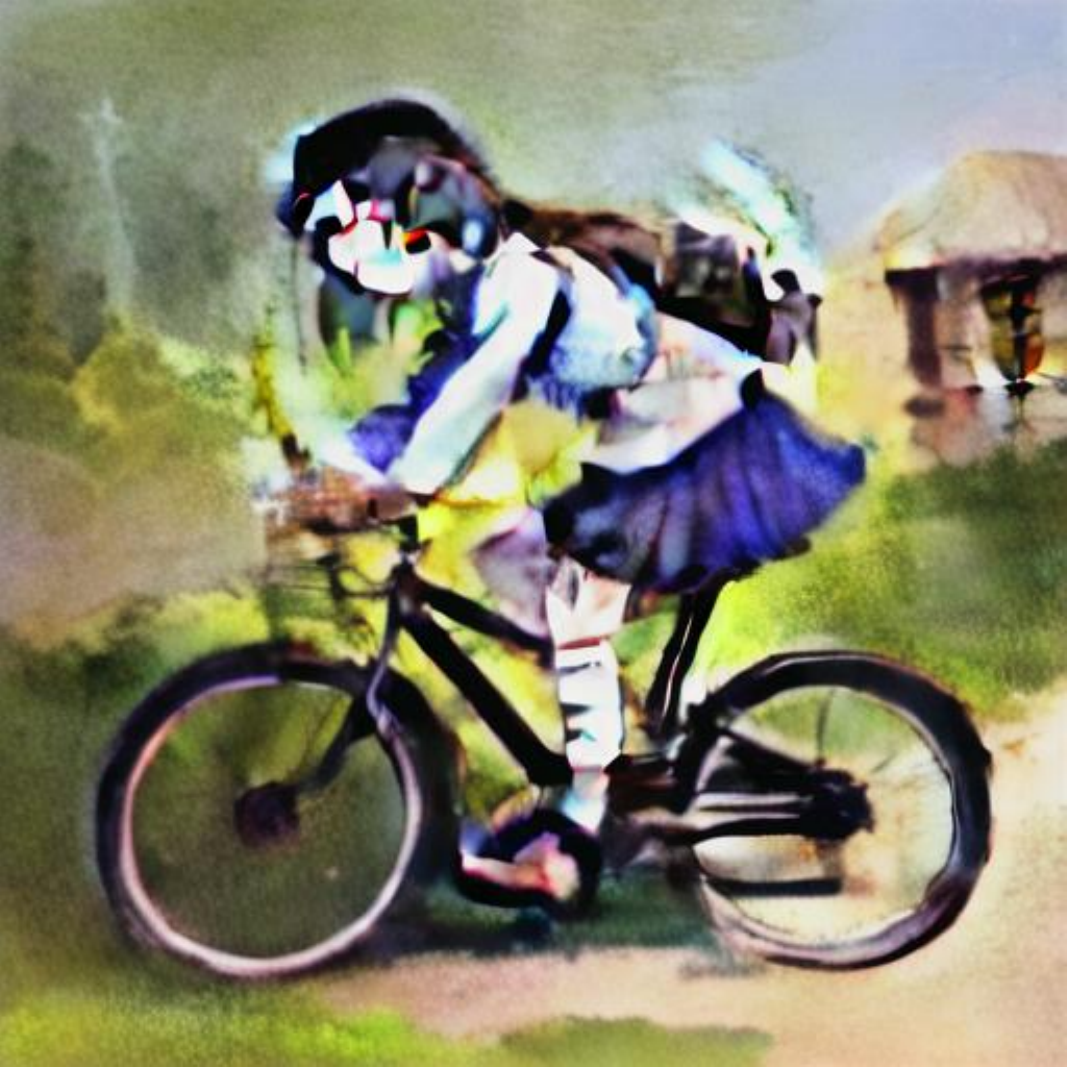}} \\
    
            & \multicolumn{1}{c}{\shortstack{\scriptsize $\beta = 0.2$}}
            & \multicolumn{1}{c}{\shortstack{\scriptsize $\beta = 0.4$}}
            & \multicolumn{1}{c}{\shortstack{\scriptsize $\beta = 0.6$}}
            & \multicolumn{1}{c}{\shortstack{\scriptsize $\beta = 0.8$}} &
            & \\
    
            \shortstack[l]{\scriptsize (b) w/ NT} &
            \noindent\parbox[c]{0.17\columnwidth}{\includegraphics[width=0.17\columnwidth]{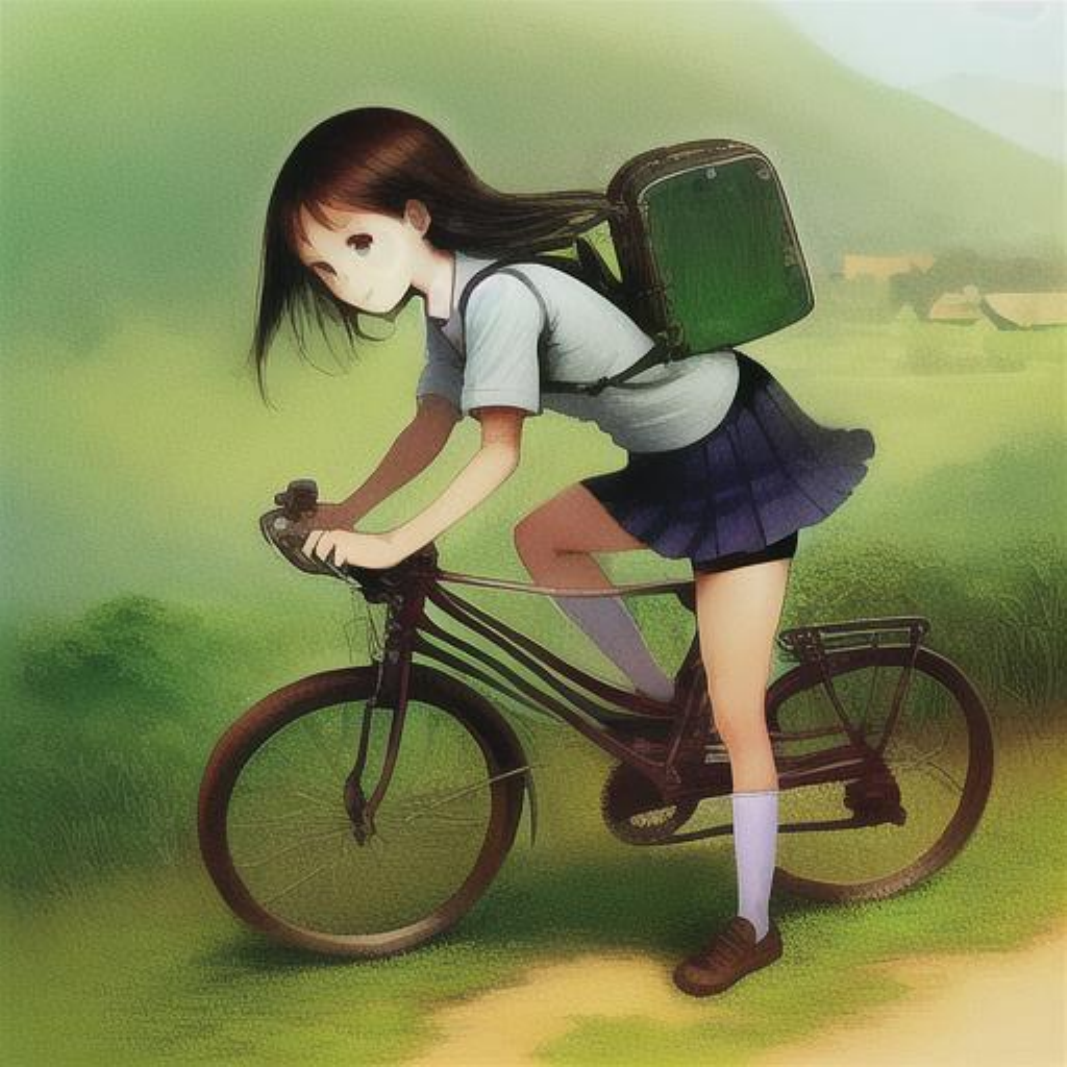}} & 
            \noindent\parbox[c]{0.17\columnwidth}{\includegraphics[width=0.17\columnwidth]{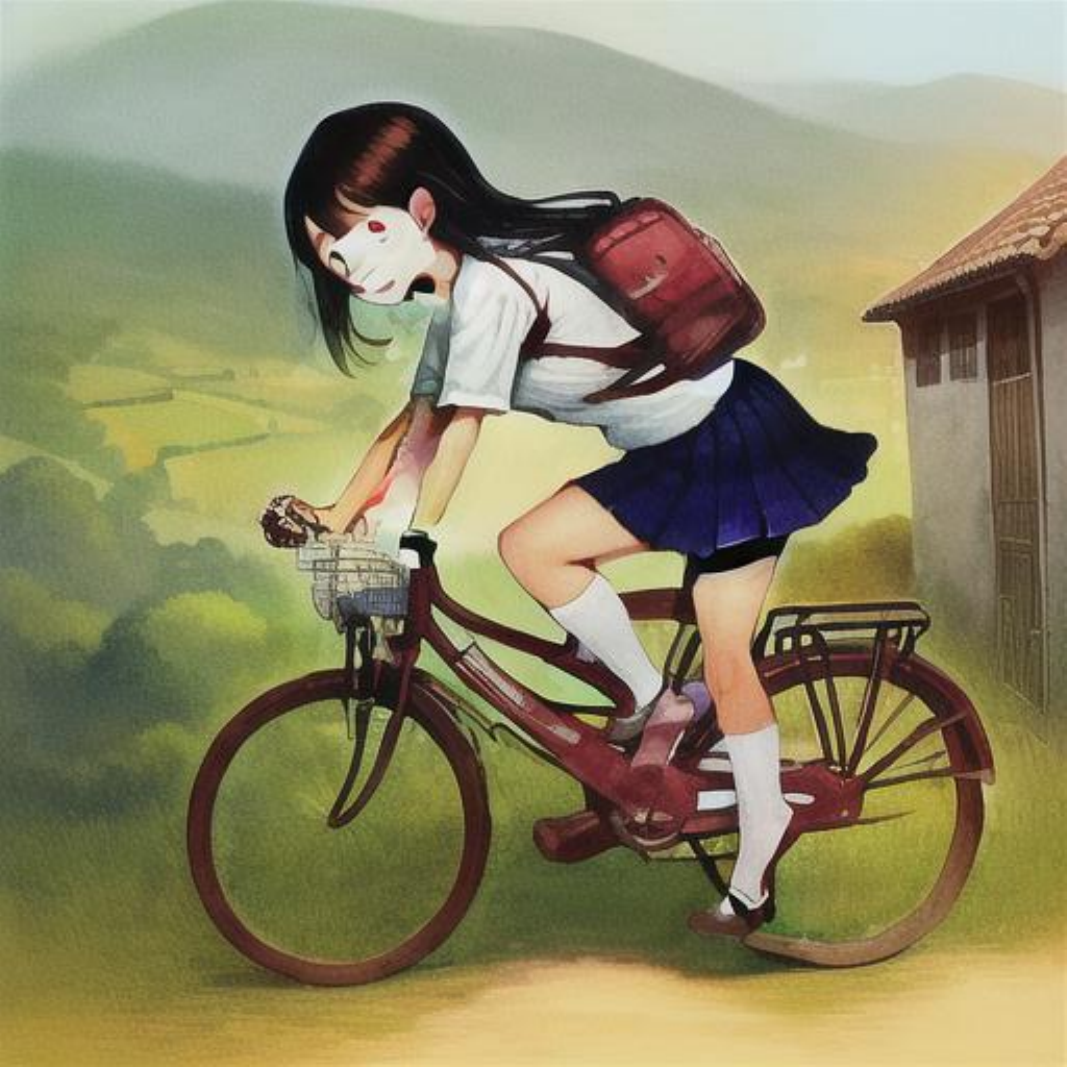}} & 
            \noindent\parbox[c]{0.17\columnwidth}{\includegraphics[width=0.17\columnwidth]{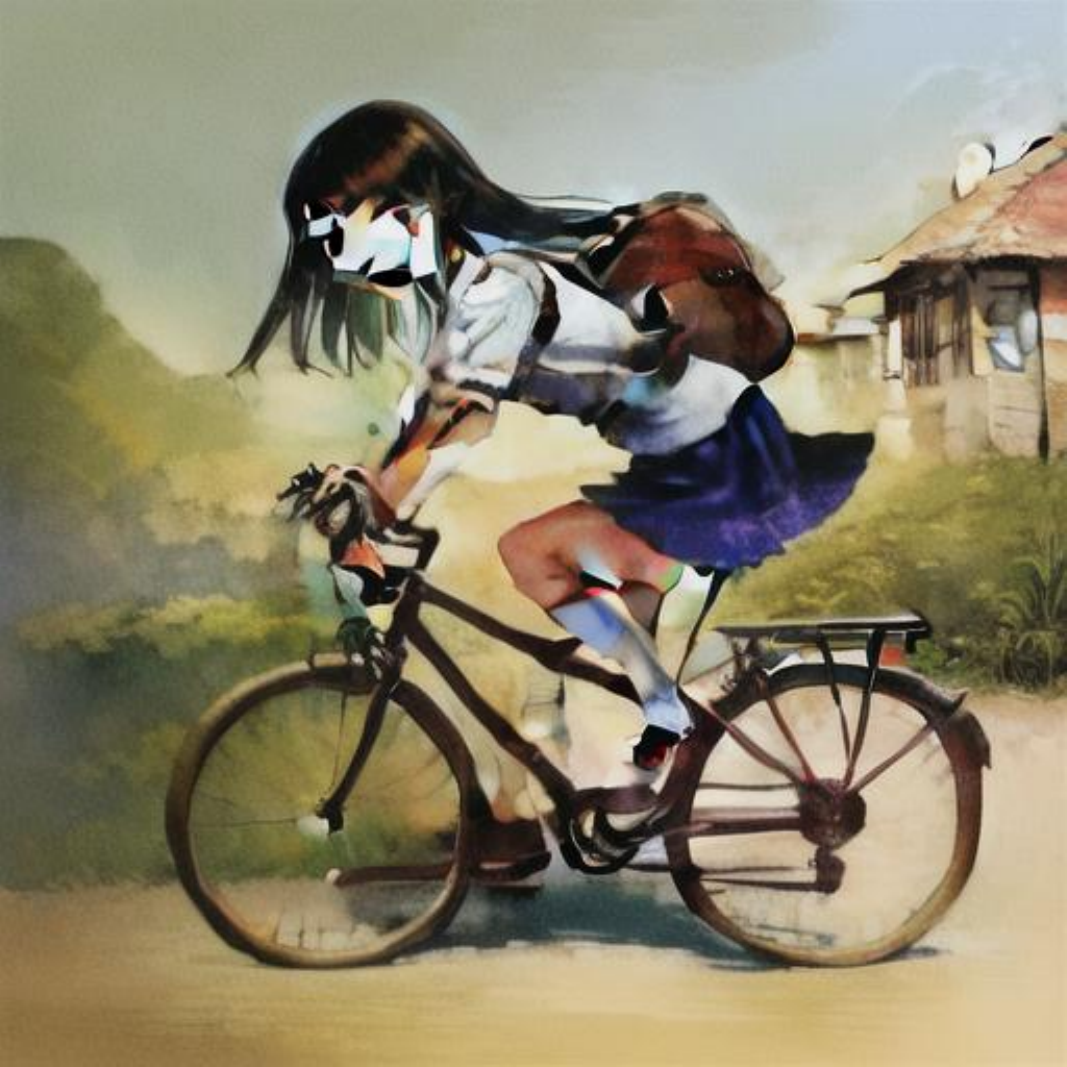}} & 
            \noindent\parbox[c]{0.17\columnwidth}{\includegraphics[width=0.17\columnwidth]{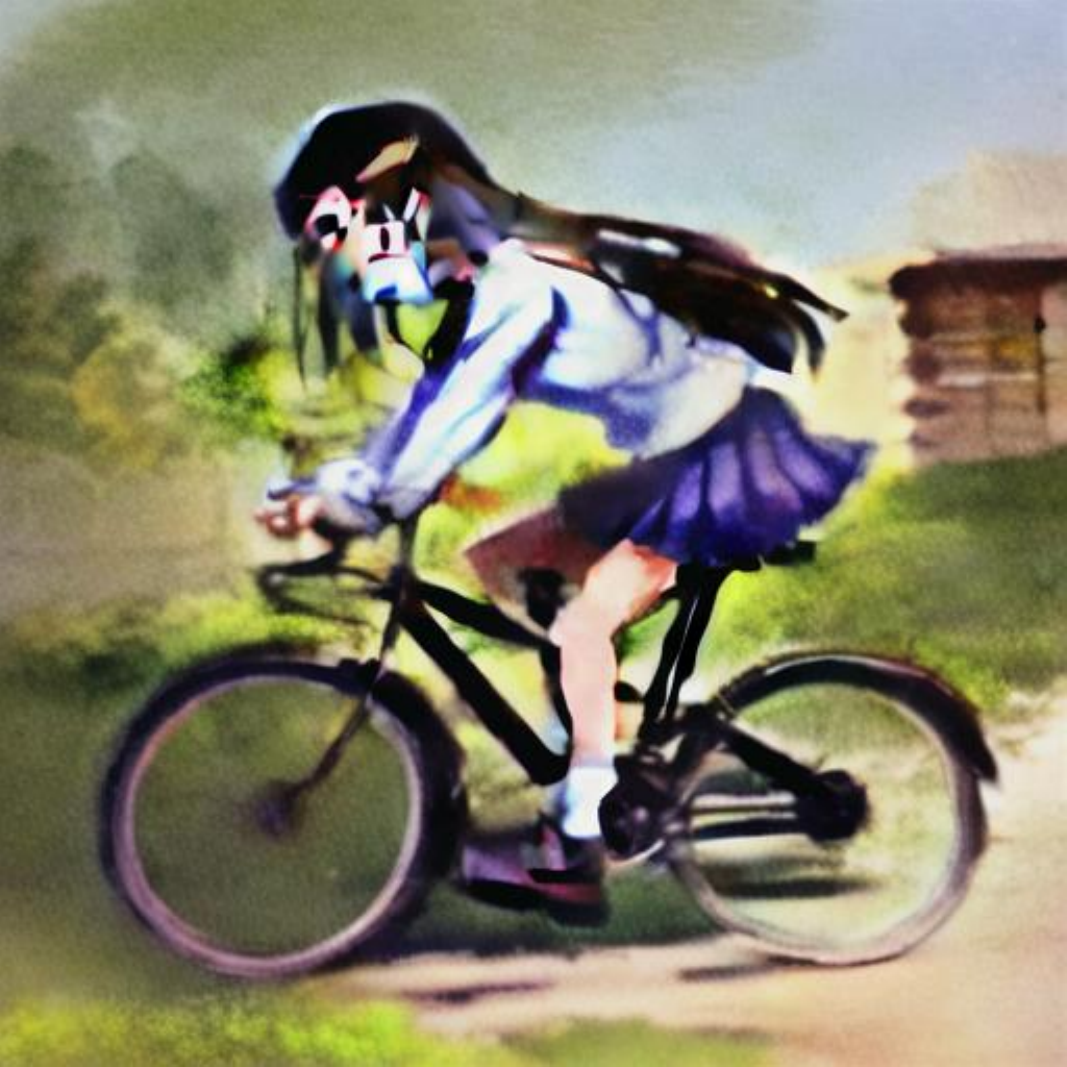}} & &
            \\
    
        \end{tabu}
    }

    \subcaptionbox{Prompt: "A beautiful illustration of a cozy cafe in a futuristic city"}{
        \begin{tabu} to \textwidth {
            @{}
            l@{\hspace{6pt}}
            c@{\hspace{2pt}}
            c@{\hspace{2pt}}
            c@{\hspace{2pt}}
            c@{}c@{\hspace{5pt}} | @{\hspace{5pt}}
            c@{}
        }
            & \multicolumn{1}{c}{\shortstack{\scriptsize $\beta = 0.2$}}
            & \multicolumn{1}{c}{\shortstack{\scriptsize $\beta = 0.4$}}
            & \multicolumn{1}{c}{\shortstack{\scriptsize $\beta = 0.6$}}
            & \multicolumn{1}{c}{\shortstack{\scriptsize $\beta = 0.8$}} &
            & \multicolumn{1}{c}{\shortstack{\scriptsize $\beta = 1.0$}} \\
    
            \shortstack[l]{\scriptsize (a) w/ HB} &
            \noindent\parbox[c]{0.17\columnwidth}{\includegraphics[width=0.17\columnwidth]{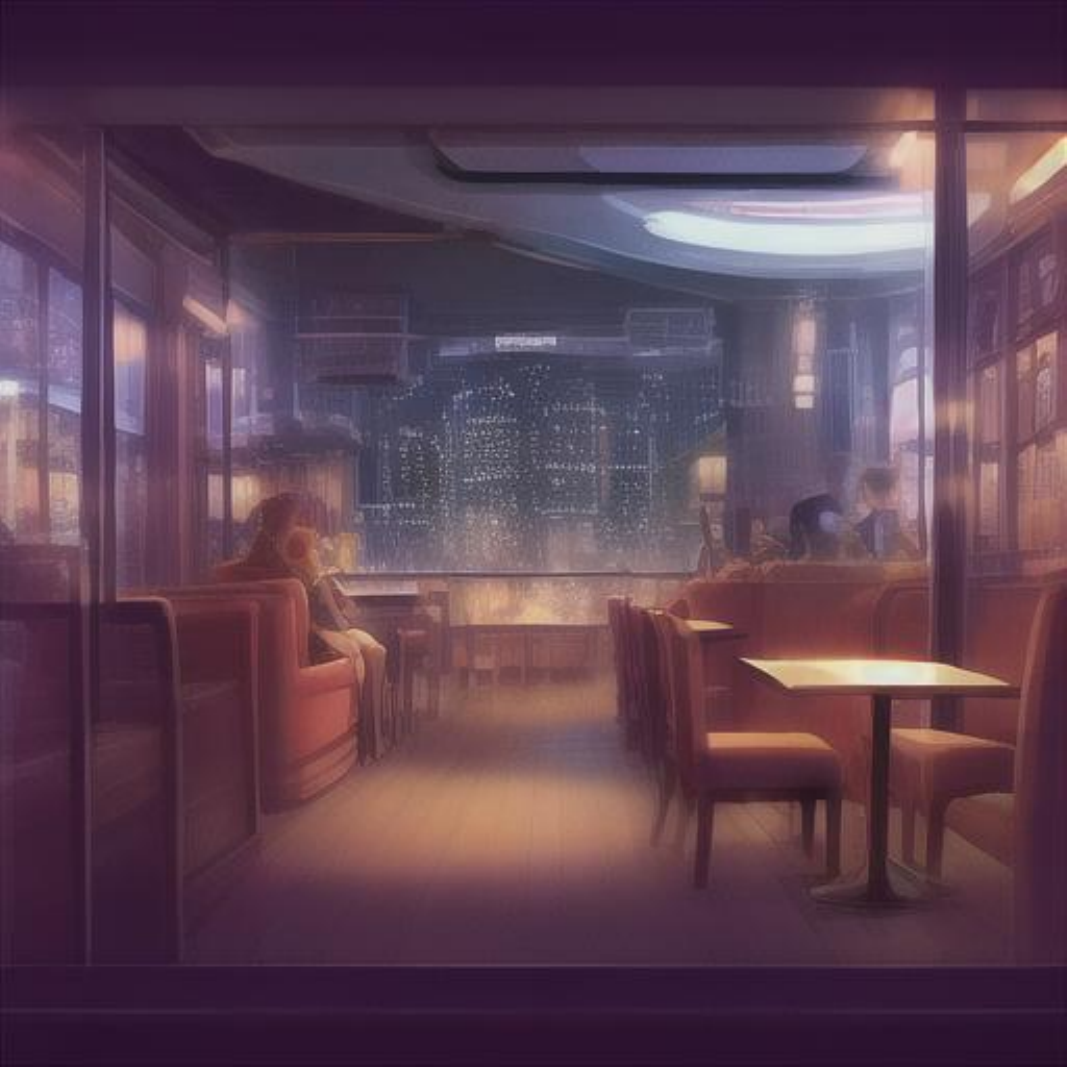}} & 
            \noindent\parbox[c]{0.17\columnwidth}{\includegraphics[width=0.17\columnwidth]{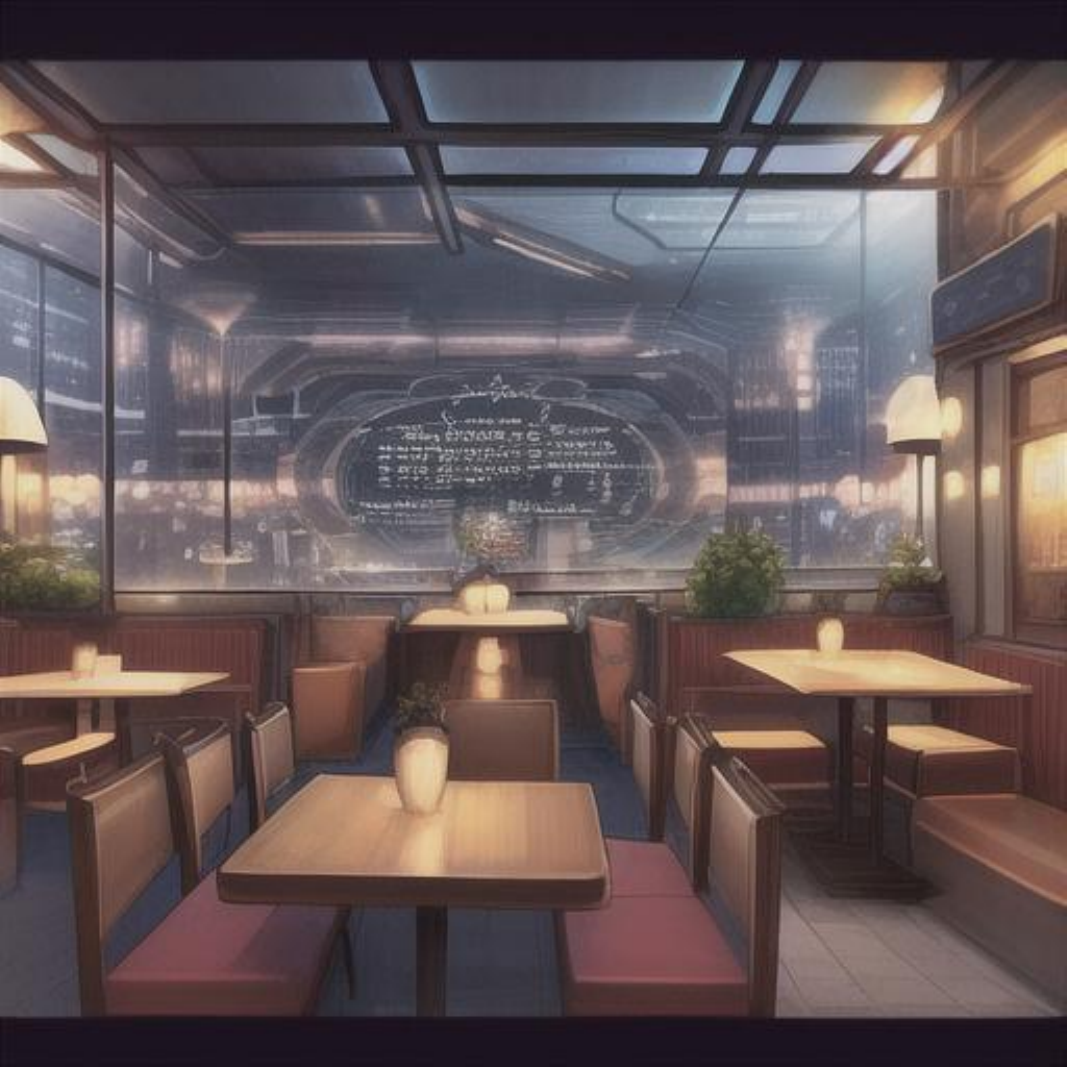}} & 
            \noindent\parbox[c]{0.17\columnwidth}{\includegraphics[width=0.17\columnwidth]{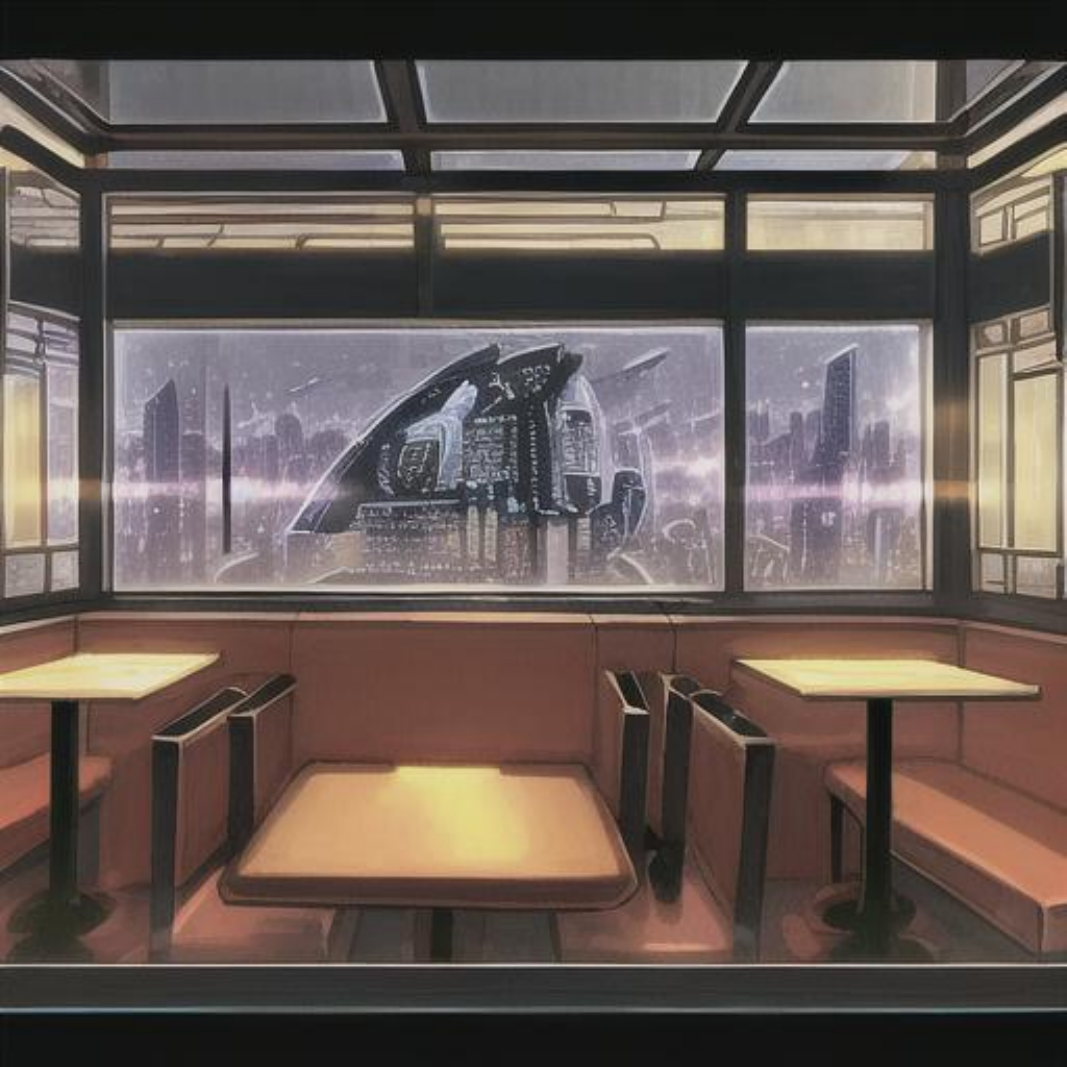}} & 
            \noindent\parbox[c]{0.17\columnwidth}{\includegraphics[width=0.17\columnwidth]{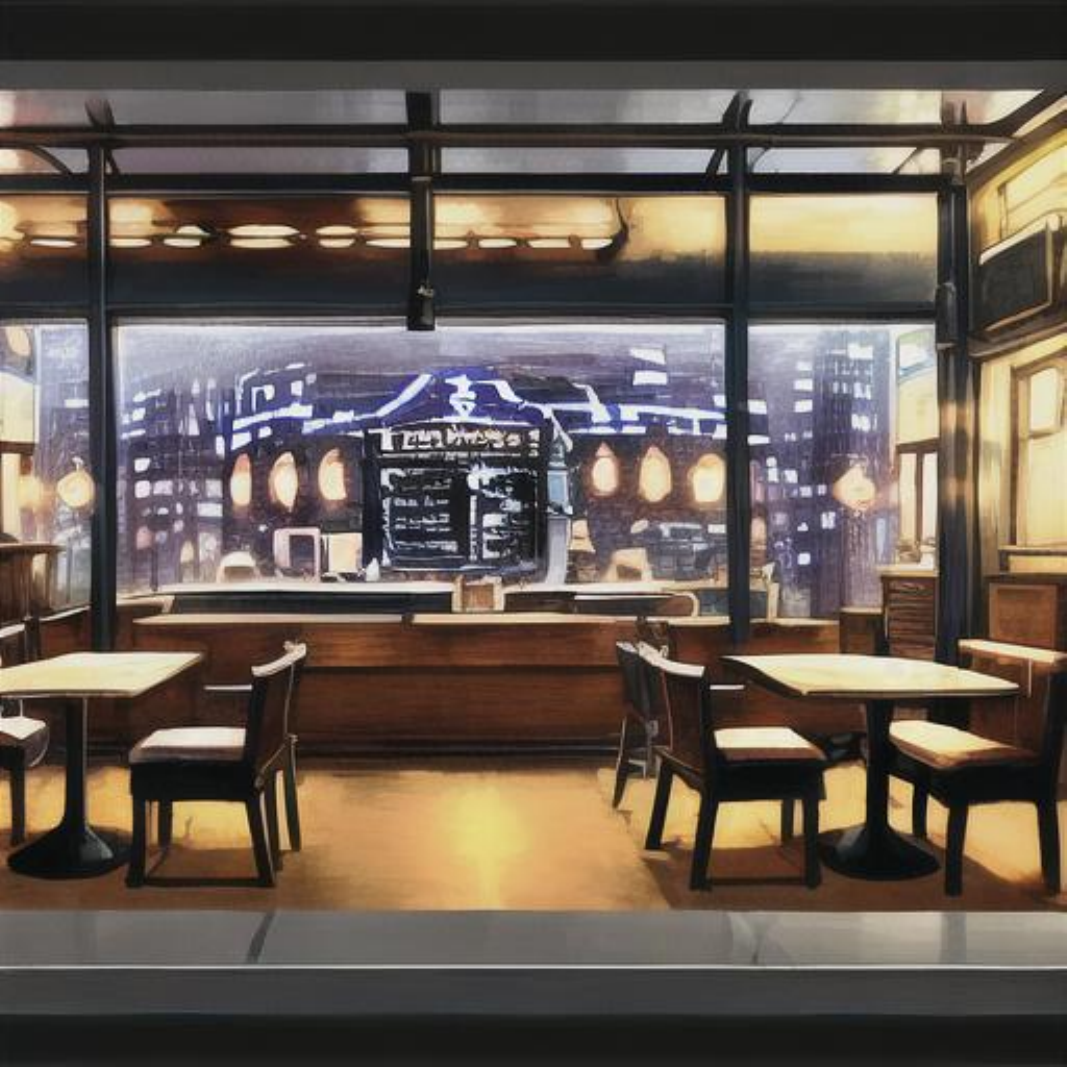}} & &
            \noindent\parbox[c]{0.17\columnwidth}{\includegraphics[width=0.17\columnwidth]{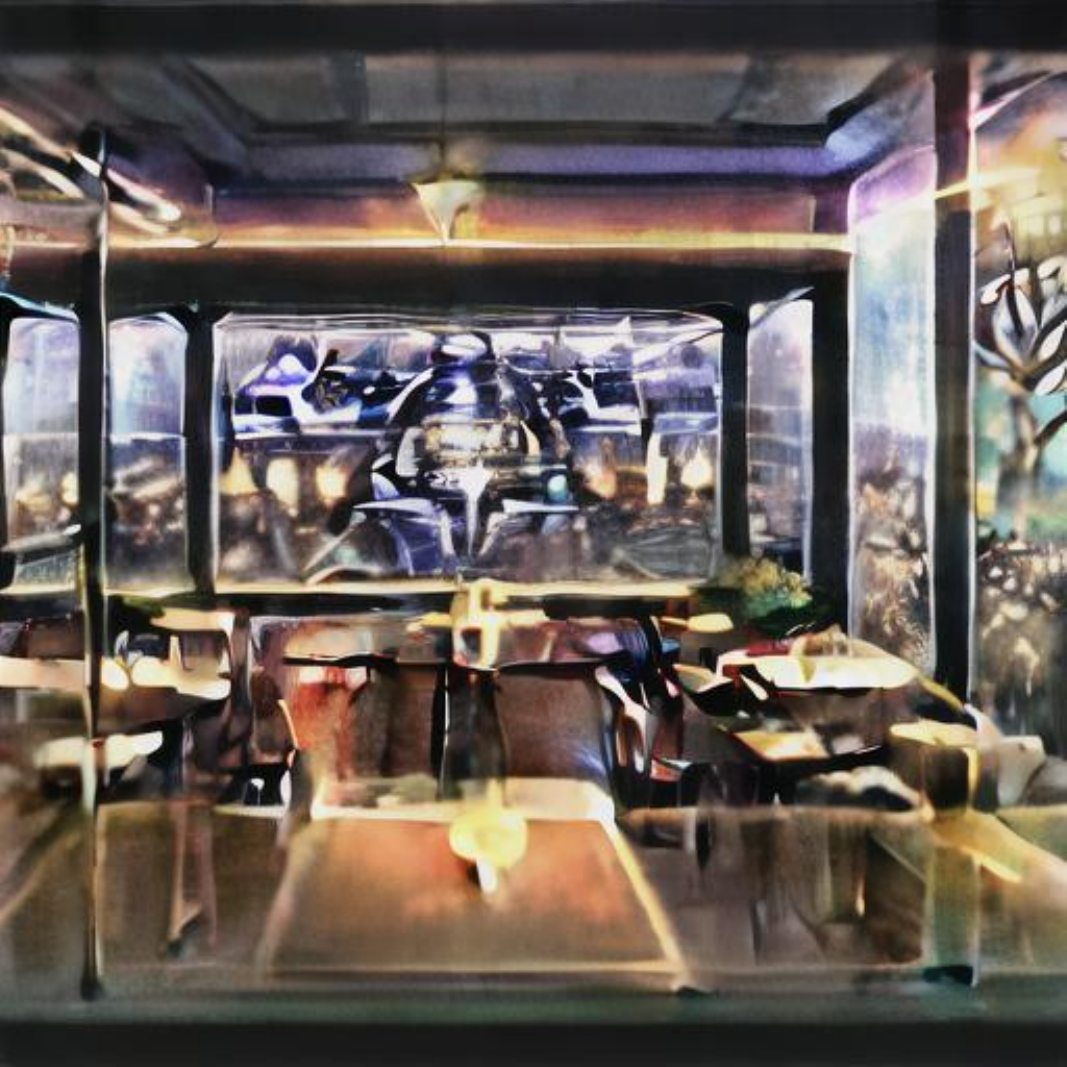}} \\
    
            & \multicolumn{1}{c}{\shortstack{\scriptsize $\beta = 0.2$}}
            & \multicolumn{1}{c}{\shortstack{\scriptsize $\beta = 0.4$}}
            & \multicolumn{1}{c}{\shortstack{\scriptsize $\beta = 0.6$}}
            & \multicolumn{1}{c}{\shortstack{\scriptsize $\beta = 0.8$}} &
            & \\
    
            \shortstack[l]{\scriptsize (b) w/ NT} &
            \noindent\parbox[c]{0.17\columnwidth}{\includegraphics[width=0.17\columnwidth]{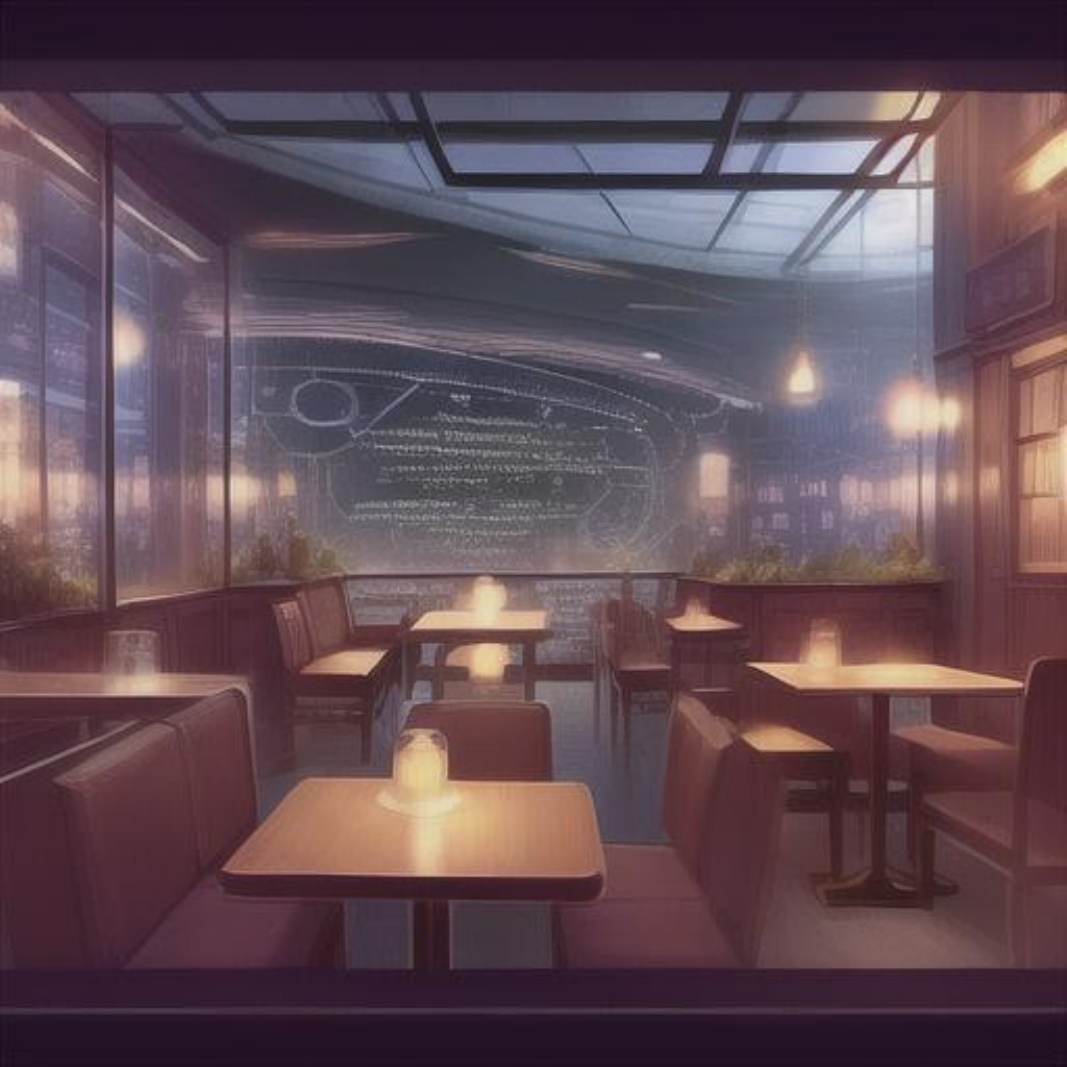}} & 
            \noindent\parbox[c]{0.17\columnwidth}{\includegraphics[width=0.17\columnwidth]{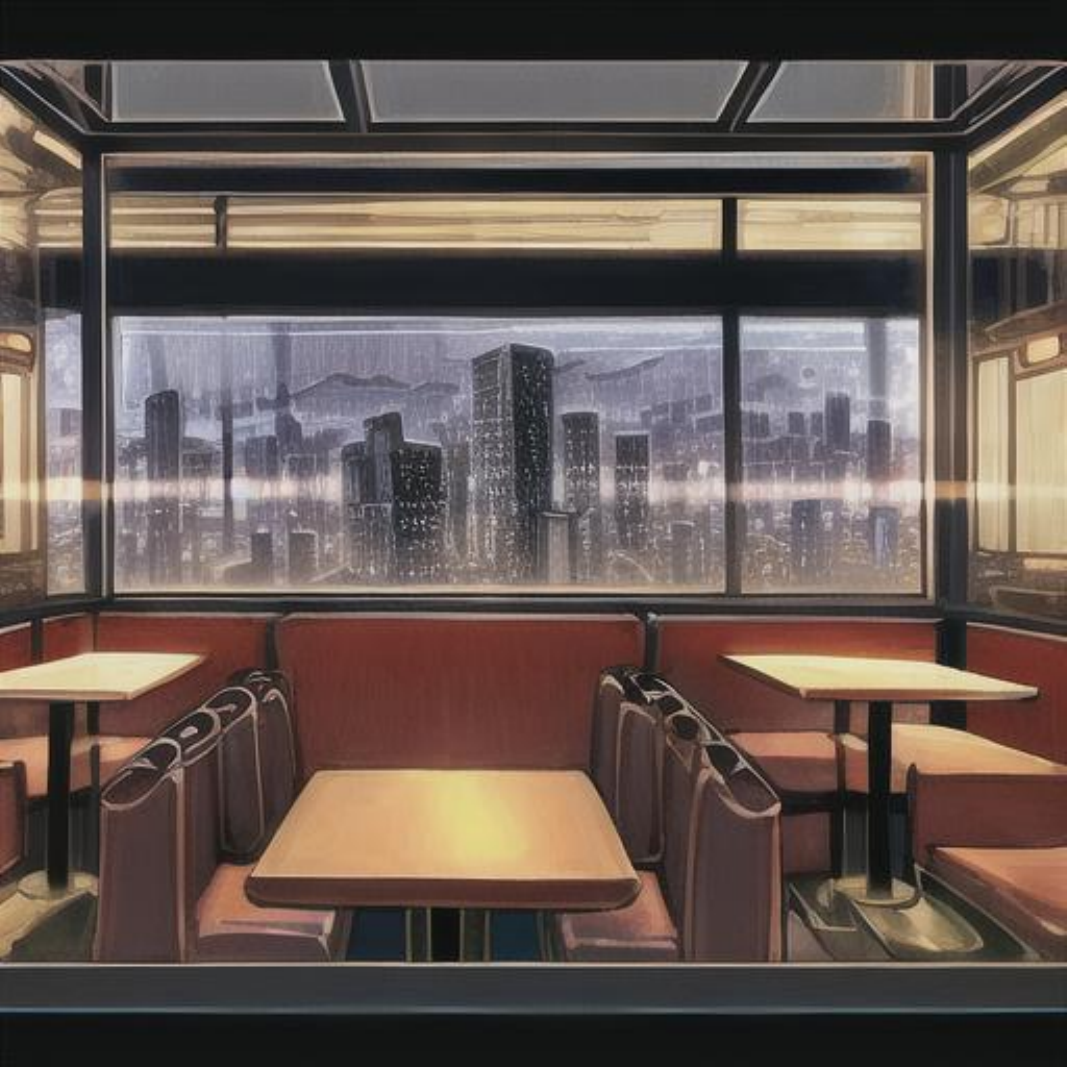}} & 
            \noindent\parbox[c]{0.17\columnwidth}{\includegraphics[width=0.17\columnwidth]{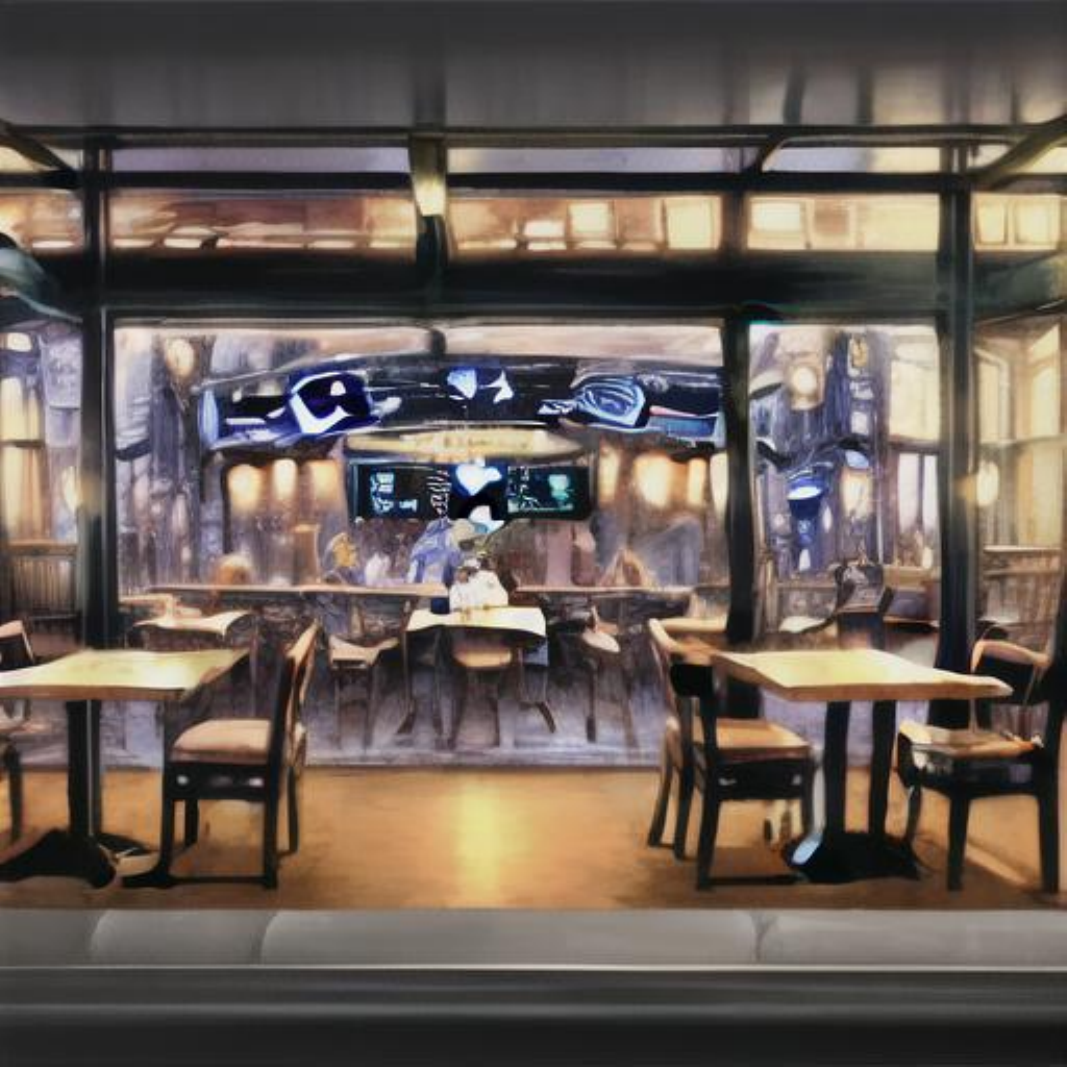}} & 
            \noindent\parbox[c]{0.17\columnwidth}{\includegraphics[width=0.17\columnwidth]{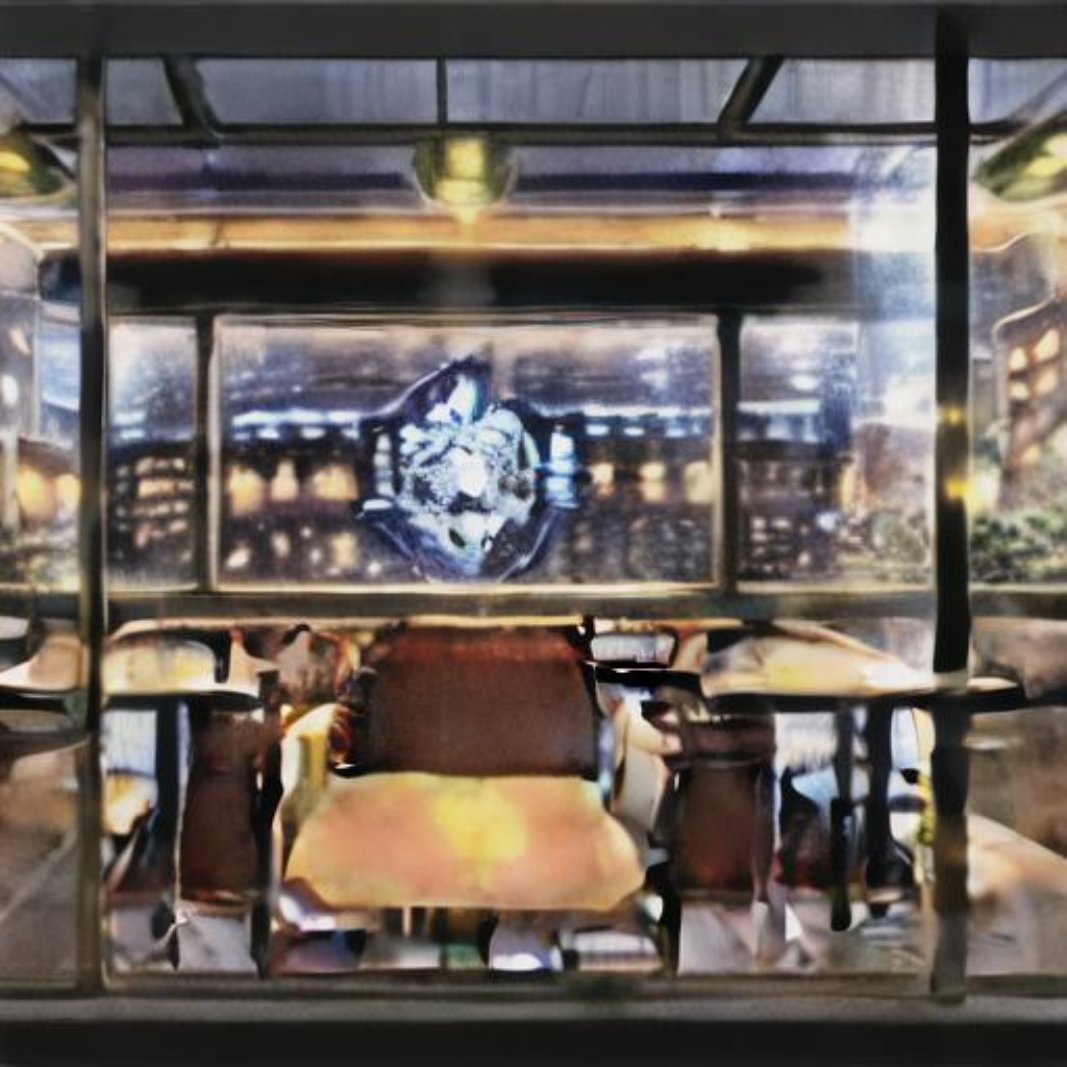}} & &
            \\
    
        \end{tabu}
    }

    \subcaptionbox{Prompt: "A painting of an old European castle in a deep forest with a Blood Moon in the background"}{
        \begin{tabu} to \textwidth {
            @{}
            l@{\hspace{6pt}}
            c@{\hspace{2pt}}
            c@{\hspace{2pt}}
            c@{\hspace{2pt}}
            c@{}c@{\hspace{5pt}} | @{\hspace{5pt}}
            c@{}
        }
            & \multicolumn{1}{c}{\shortstack{\scriptsize $\beta = 0.2$}}
            & \multicolumn{1}{c}{\shortstack{\scriptsize $\beta = 0.4$}}
            & \multicolumn{1}{c}{\shortstack{\scriptsize $\beta = 0.6$}}
            & \multicolumn{1}{c}{\shortstack{\scriptsize $\beta = 0.8$}} &
            & \multicolumn{1}{c}{\shortstack{\scriptsize $\beta = 1.0$}} \\
    
            \shortstack[l]{\scriptsize (a) w/ HB} &
            \noindent\parbox[c]{0.17\columnwidth}{\includegraphics[width=0.17\columnwidth]{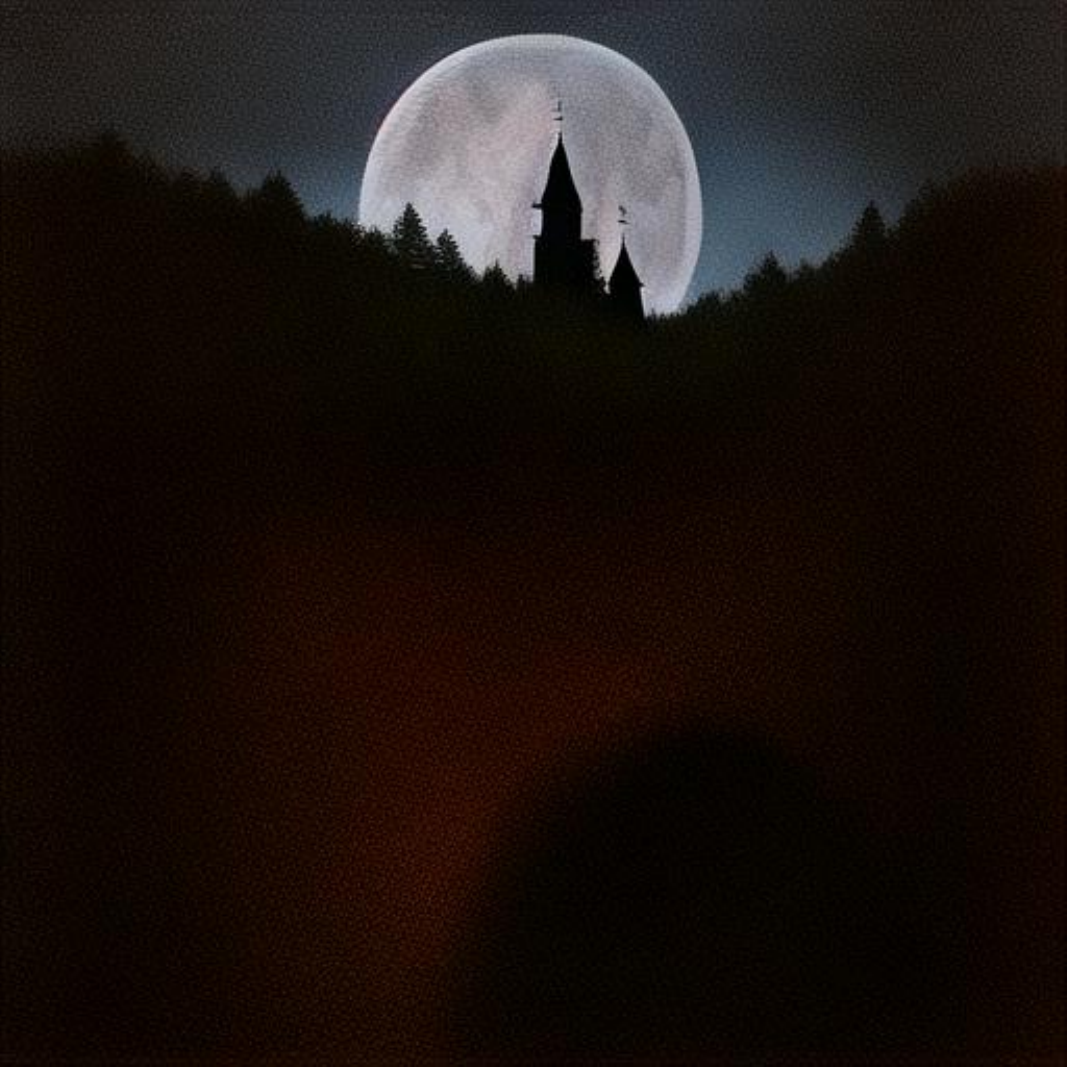}} & 
            \noindent\parbox[c]{0.17\columnwidth}{\includegraphics[width=0.17\columnwidth]{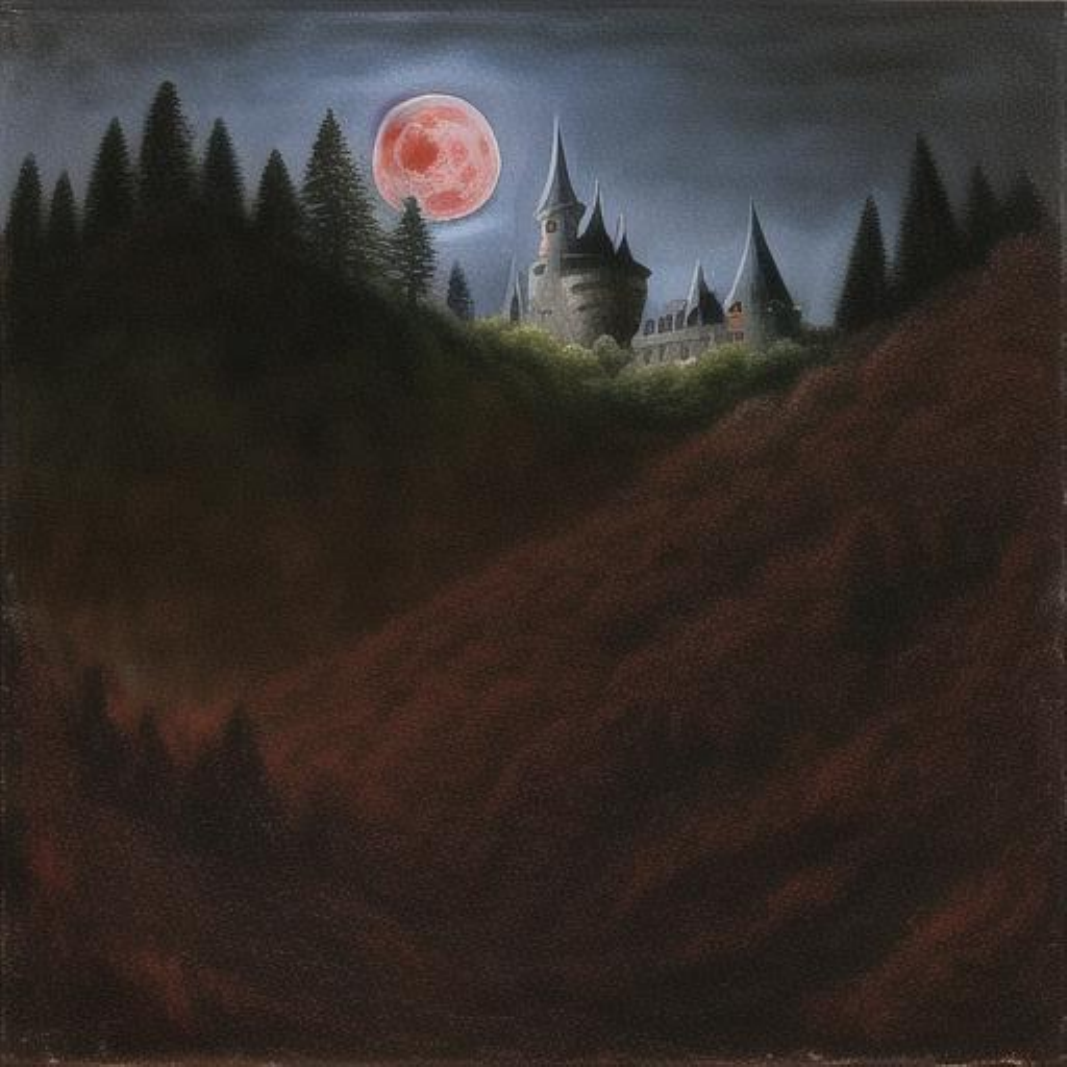}} & 
            \noindent\parbox[c]{0.17\columnwidth}{\includegraphics[width=0.17\columnwidth]{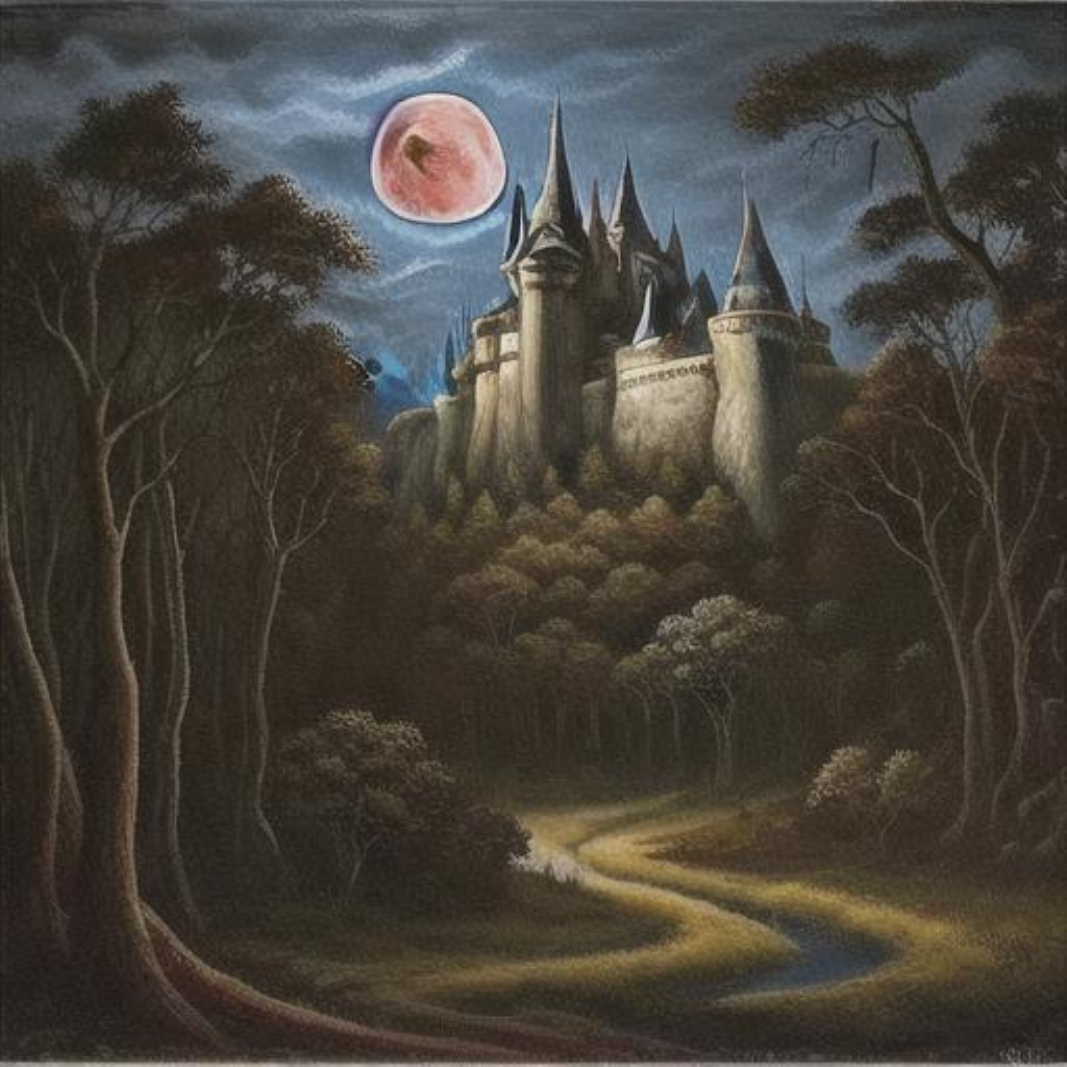}} & 
            \noindent\parbox[c]{0.17\columnwidth}{\includegraphics[width=0.17\columnwidth]{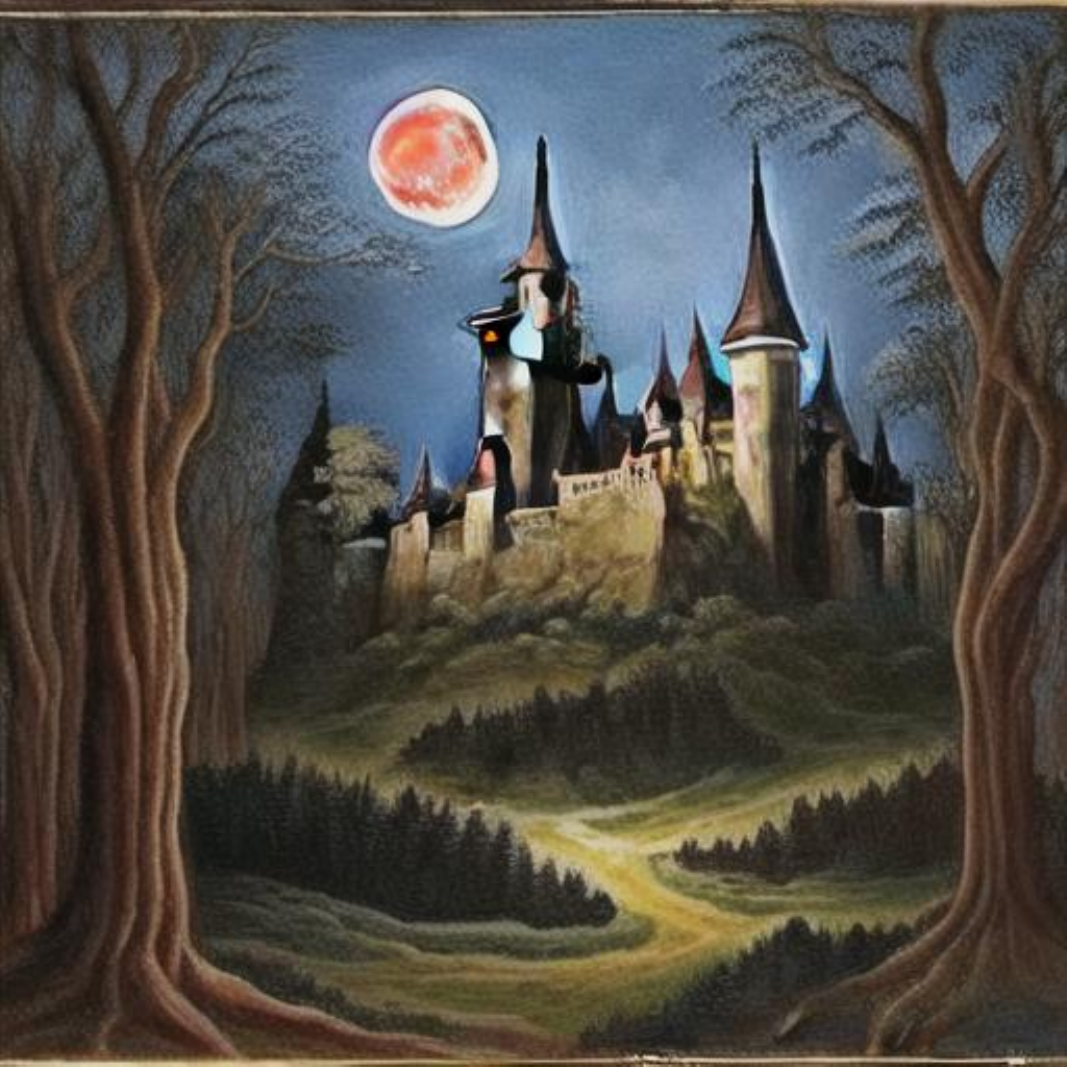}} & &
            \noindent\parbox[c]{0.17\columnwidth}{\includegraphics[width=0.17\columnwidth]{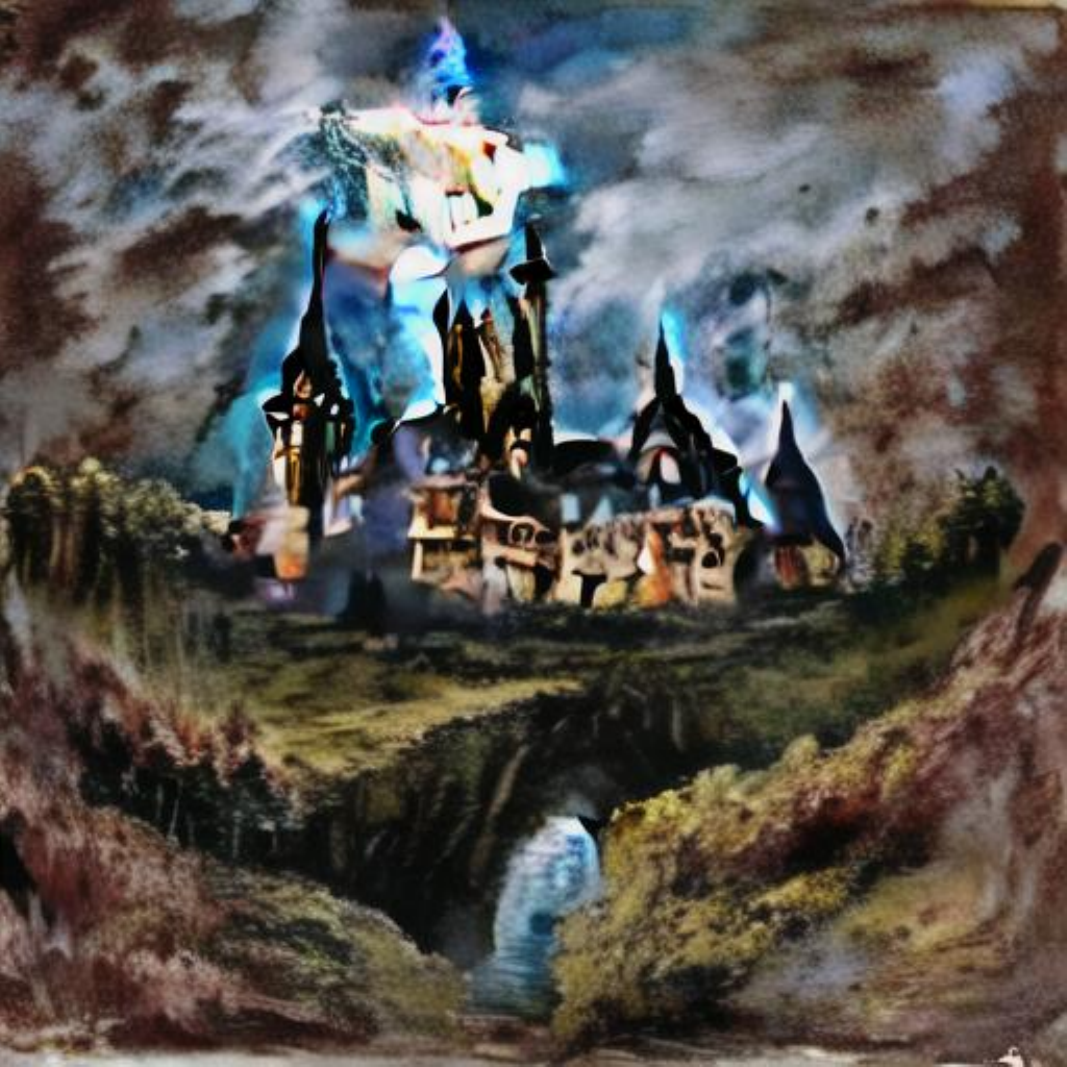}} \\
    
            & \multicolumn{1}{c}{\shortstack{\scriptsize $\beta = 0.2$}}
            & \multicolumn{1}{c}{\shortstack{\scriptsize $\beta = 0.4$}}
            & \multicolumn{1}{c}{\shortstack{\scriptsize $\beta = 0.6$}}
            & \multicolumn{1}{c}{\shortstack{\scriptsize $\beta = 0.8$}} &
            & \\
    
            \shortstack[l]{\scriptsize (b) w/ NT} &
            \noindent\parbox[c]{0.17\columnwidth}{\includegraphics[width=0.17\columnwidth]{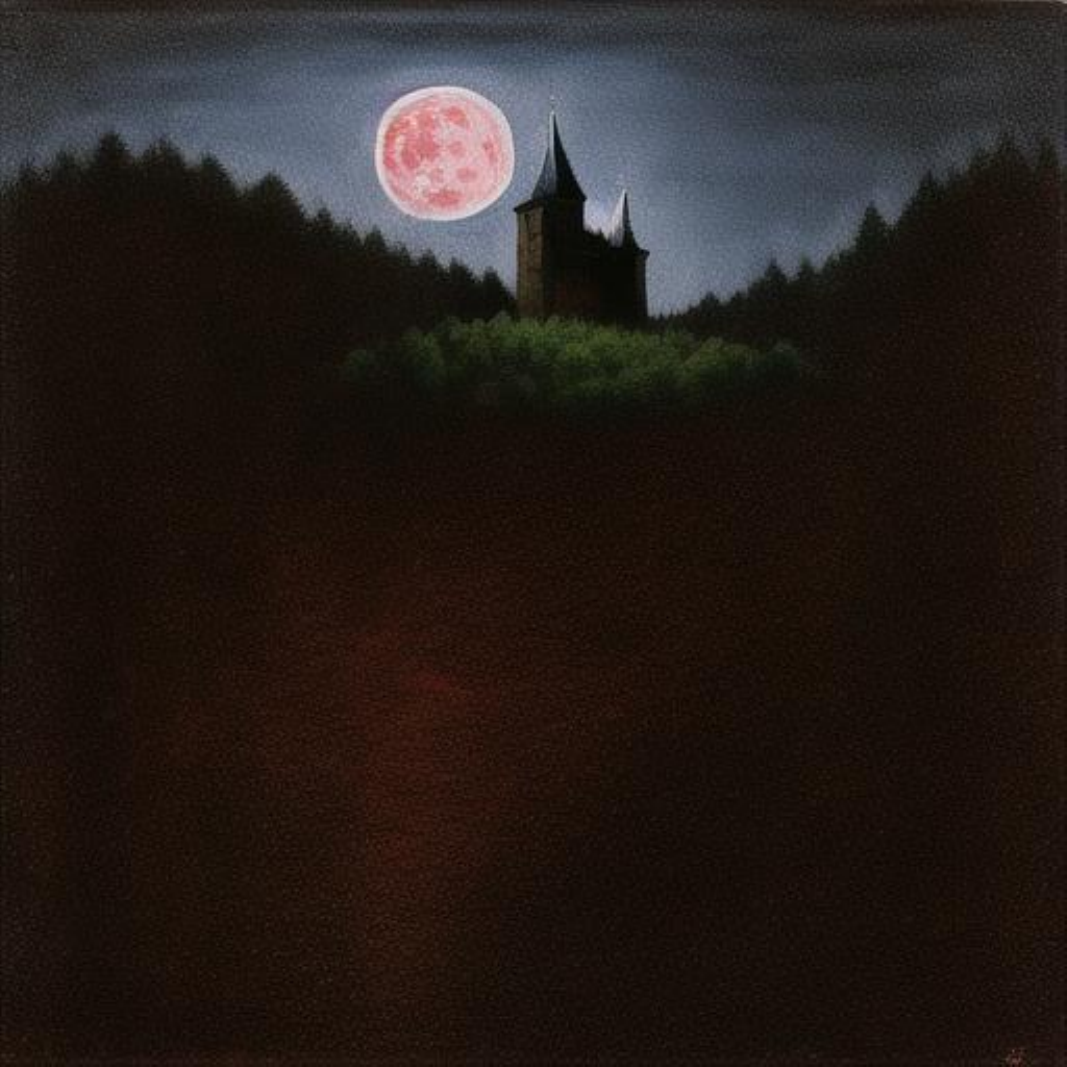}} & 
            \noindent\parbox[c]{0.17\columnwidth}{\includegraphics[width=0.17\columnwidth]{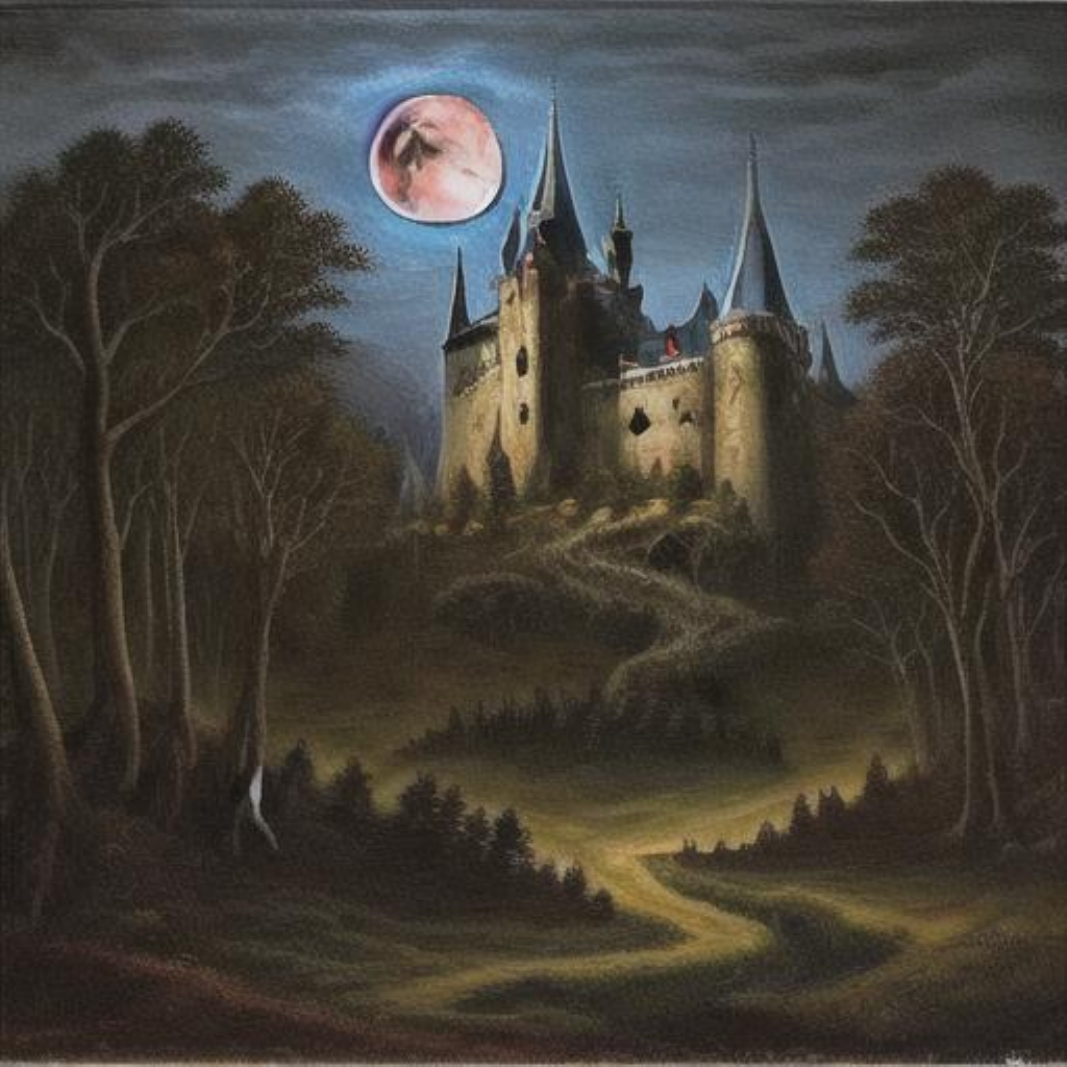}} & 
            \noindent\parbox[c]{0.17\columnwidth}{\includegraphics[width=0.17\columnwidth]{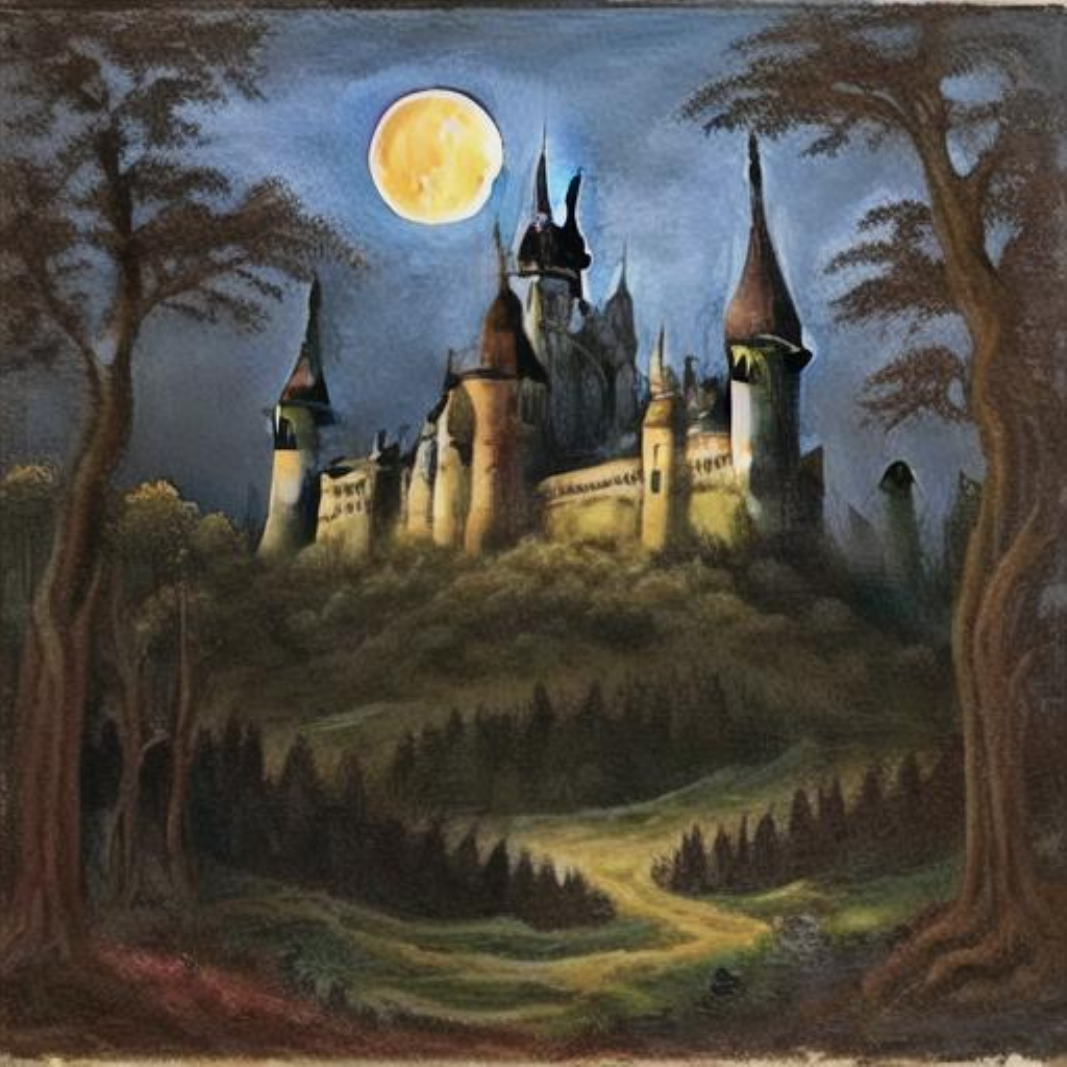}} & 
            \noindent\parbox[c]{0.17\columnwidth}{\includegraphics[width=0.17\columnwidth]{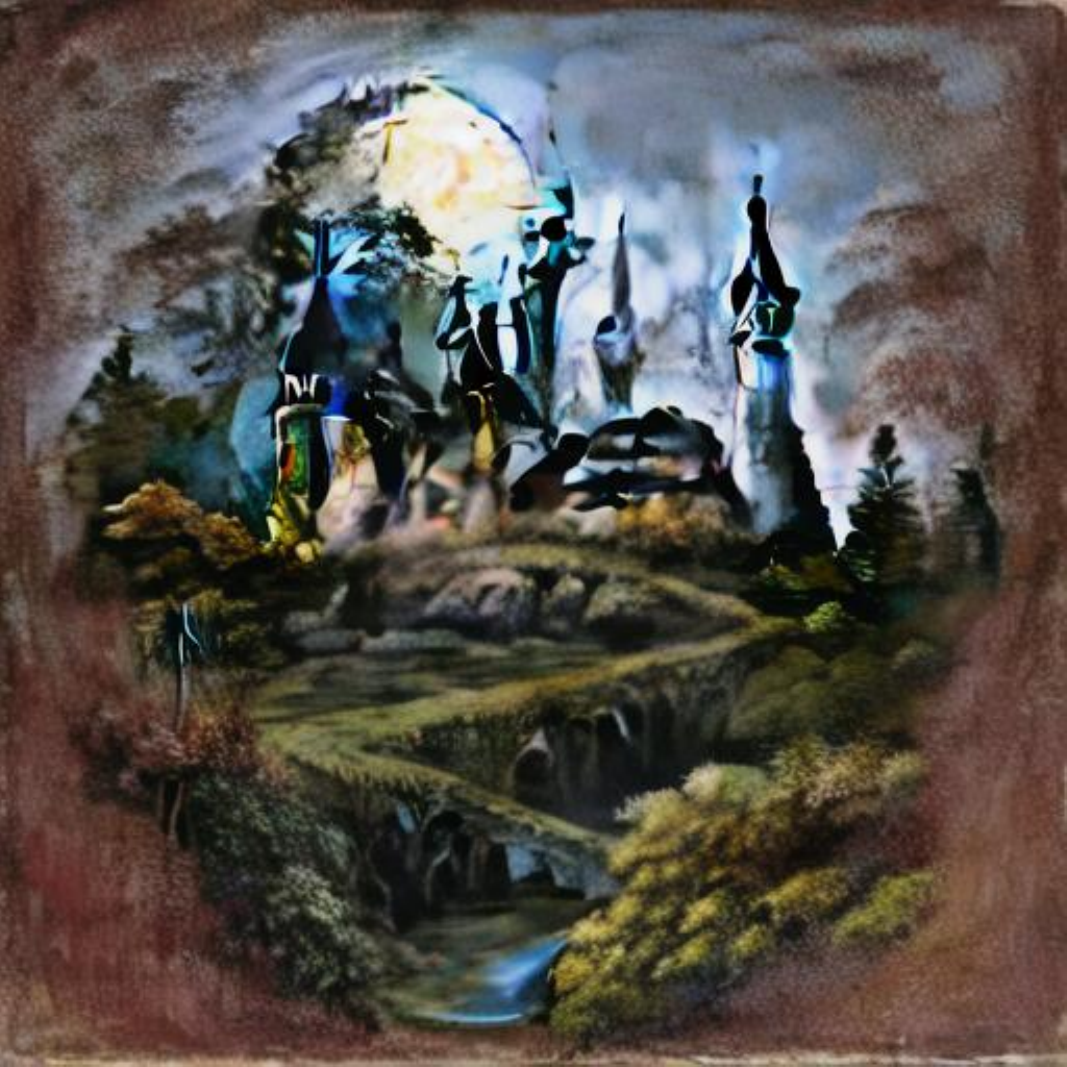}} & &
            \\
    
        \end{tabu}
    }
    \caption{Comparison between two variations of momentum: (a) Polyak's Heavy Ball (HB) and (b) Nesterov (NT). These momentum variations are applied to PLMS4 \cite{liu2022pseudo} on a fine-tuned Stable Diffusion model called Anything V4 \cite{Anything} with 15 sampling steps and a guidance scale of 15. Both variations effectively reduce artifacts. However, the choice of the effectiveness parameter $\beta$ might differ due to the distinct shapes of their respective stability regions.}
\end{figure}



\section{Statistical Reports}

In this section, we present detailed statistical reports for the experiments conducted in Section \ref{sec:exp_sd} and Section \ref{sec:ghvb}. These reports provide detailed information, including mean values and their corresponding 95\% confidence intervals, to offer a thorough understanding of the experimental results.

Firstly, we focus on the experiment related to mitigating the magnitude score in Section \ref{sec:exp_sd}. The results depicted in Figure \ref{fig:ablation_magnitude1} are presented in Table \ref{tab:ablation_magnitude1}. Additionally, the outcomes illustrated in Figure \ref{fig:ablation_magnitude2} are reported in Table \ref{tab:ablation_magnitude2}. For the ablation study of GHVB in Section \ref{sec:ghvb}, we provide the results shown in Figure \ref{fig:ablation_on_beta} in Table \ref{tab:ablation_on_beta}. Similarly, the findings presented in Figure \ref{fig:ablation_on_ghvb} are reported in Table \ref{tab:ablation_on_ghvb}.

Furthermore, we include a runtime comparison of each sampling method in Table \ref{tab:sampling_time}, detailing the wall clock time required for each method. The results indicate that all the methods exhibit similar sampling times, ensuring a fair comparison across the different approaches.

\begin{table}
\centering
\small
\begin{tabular}{lcccccc}
\toprule
\multicolumn{1}{c}{} & \multicolumn{6}{c}{Number of steps}                                                     \\
 Method & \multicolumn{1}{c}{10} & \multicolumn{1}{c}{15} & \multicolumn{1}{c}{20} & \multicolumn{1}{c}{25} & \multicolumn{1}{c}{30} & \multicolumn{1}{c}{60} \\ \midrule
DPM                  & 1.113 $\pm$ .090 & 1.369 $\pm$ .102 & 1.087 $\pm$ .083 & 0.972 $\pm$ .075 & 0.919 $\pm$ .068 & 0.869 $\pm$ .067 \\
DPM w/ HB 0.8        & 0.974 $\pm$ .078 & 1.057 $\pm$ .079 & 0.916 $\pm$ .072 & 0.857 $\pm$ .070 & 0.831 $\pm$ .067 & 0.834 $\pm$ .065 \\
DPM w/ HB 0.9        & 1.043 $\pm$ .082 & 1.186 $\pm$ .088 & 0.986 $\pm$ .075 & 0.921 $\pm$ .073 & 0.867 $\pm$ .068 & 0.844 $\pm$ .065 \\
PLMS4 w/ HB 0.8      & 1.958 $\pm$ .116 & 1.469 $\pm$ .105 & 1.213 $\pm$ .097 & 1.060 $\pm$ .087 & 0.963 $\pm$ .076 & 0.838 $\pm$ .063 \\
PLMS4 w/ HB 0.9      & 2.499 $\pm$ .118 & 1.888 $\pm$ .112 & 1.534 $\pm$ .112 & 1.270 $\pm$ .104 & 1.116 $\pm$ .091 & 0.887 $\pm$ .066 \\
PLMS4                & 3.149 $\pm$ .116 & 2.460 $\pm$ .116 & 1.911 $\pm$ .115 & 1.597 $\pm$ .116 & 1.372 $\pm$ .106 & 0.957 $\pm$ .075 \\ \bottomrule
\end{tabular}
\caption{95\% confidence intervals for the magnitude scores of HB (Figure \ref{fig:ablation_magnitude1})} \label{tab:ablation_magnitude1}
\end{table}
\begin{table}
\centering
\small
\begin{tabular}{lcccccc}
\toprule
\multicolumn{1}{c}{} & \multicolumn{6}{c}{Number of steps}                                                     \\
 Method & \multicolumn{1}{c}{10} & \multicolumn{1}{c}{15} & \multicolumn{1}{c}{20} & \multicolumn{1}{c}{25} & \multicolumn{1}{c}{30} & \multicolumn{1}{c}{60} \\ 
 \midrule
DDIM                 & 0.844 $\pm$ .076 & 0.765 $\pm$ .064 & 0.728 $\pm$ .060 & 0.744 $\pm$ .063 & 0.761 $\pm$ .062 & 0.778 $\pm$ .062 \\
GHVB2.1              & 1.238 $\pm$ .097 & 0.924 $\pm$ .072 & 0.832 $\pm$ .067 & 0.820 $\pm$ .066 & 0.825 $\pm$ .066 & 0.829 $\pm$ .064 \\
GHVB2.3              & 1.291 $\pm$ .102 & 0.952 $\pm$ .072 & 0.872 $\pm$ .070 & 0.836 $\pm$ .064 & 0.831 $\pm$ .066 & 0.828 $\pm$ .062 \\
GHVB2.5              & 1.392 $\pm$ .103 & 1.016 $\pm$ .081 & 0.907 $\pm$ .071 & 0.851 $\pm$ .069 & 0.842 $\pm$ .065 & 0.845 $\pm$ .063 \\
GHVB2.7              & 1.514 $\pm$ .105 & 1.095 $\pm$ .085 & 0.968 $\pm$ .077 & 0.877 $\pm$ .068 & 0.864 $\pm$ .067 & 0.826 $\pm$ .063 \\
GHVB2.9              & 1.673 $\pm$ .107 & 1.203 $\pm$ .090 & 1.023 $\pm$ .079 & 0.934 $\pm$ .075 & 0.901 $\pm$ .071 & 0.835 $\pm$ .063 \\
PLMS4                & 3.149 $\pm$ .116 & 2.460 $\pm$ .116 & 1.911 $\pm$ .115 & 1.597 $\pm$ .116 & 1.372 $\pm$ .106 & 0.957 $\pm$ .075 \\ 
\bottomrule
\end{tabular}
\caption{95\% confidence intervals for the magnitude scores of GHVB (Figure \ref{fig:ablation_magnitude2})} \label{tab:ablation_magnitude2}
\end{table}
\begin{table}
\centering
\small
\begin{tabular}{lccccccc}
\toprule
\multicolumn{1}{c}{} & \multicolumn{7}{c}{Number of steps}                                                                    \\
Method & 10           & 20           & 40           & 80           & 160          & 320          & 640          \\ 
\midrule
DDIM                 & 0.584 $\pm$ .034 & 0.409 $\pm$ .029 & 0.304 $\pm$ .029 & 0.210 $\pm$ .026 & 0.139 $\pm$ .022 & 0.085 $\pm$ .014 & 0.048 $\pm$ .009 \\
GHVB1.1              & 0.592 $\pm$ .034 & 0.406 $\pm$ .029 & 0.295 $\pm$ .030 & 0.189 $\pm$ .026 & 0.113 $\pm$ .020 & 0.054 $\pm$ .010 & 0.019 $\pm$ .005 \\
GHVB1.3              & 0.609 $\pm$ .035 & 0.410 $\pm$ .029 & 0.276 $\pm$ .029 & 0.158 $\pm$ .023 & 0.086 $\pm$ .017 & 0.030 $\pm$ .007 & 0.009 $\pm$ .003 \\
GHVB1.5              & 0.624 $\pm$ .036 & 0.409 $\pm$ .029 & 0.261 $\pm$ .029 & 0.145 $\pm$ .023 & 0.067 $\pm$ .014 & 0.021 $\pm$ .005 & 0.006 $\pm$ .002 \\
GHVB1.7              & 0.645 $\pm$ .037 & 0.411 $\pm$ .030 & 0.254 $\pm$ .028 & 0.133 $\pm$ .023 & 0.053 $\pm$ .011 & 0.016 $\pm$ .005 & 0.004 $\pm$ .002 \\
GHVB1.9              & 0.663 $\pm$ .037 & 0.414 $\pm$ .030 & 0.246 $\pm$ .028 & 0.123 $\pm$ .021 & 0.044 $\pm$ .009 & 0.013 $\pm$ .004 & 0.003 $\pm$ .001 \\
PLMS2                & 0.676 $\pm$ .038 & 0.418 $\pm$ .030 & 0.246 $\pm$ .028 & 0.119 $\pm$ .021 & 0.041 $\pm$ .009 & 0.011 $\pm$ .003 & 0.003 $\pm$ .001 \\ 
\bottomrule
\end{tabular}
\caption{95\% confidence intervals for L2 norm of GHVB (Figrue \ref{fig:ablation_on_beta})} \label{tab:ablation_on_beta}
\end{table}
\begin{table}
\centering
\small
\begin{tabular}{lccccc}
\toprule
\multicolumn{1}{c}{} & \multicolumn{5}{c}{Number of steps ($k_{\text{new}}$)}                            \\
 Method & 40           & 80           & 160          & 320          & 640          \\ 
 \midrule
GHVB0.5              & 0.247 $\pm$ .030 & 0.235 $\pm$ .029 & 0.351 $\pm$ .045 & 0.450 $\pm$ .057 & 0.474 $\pm$ .057 \\
GHVB1.5              & 0.550 $\pm$ .072 & 0.717 $\pm$ .086 & 0.922 $\pm$ .089 & 1.337 $\pm$ .102 & 1.519 $\pm$ .102 \\
GHVB2.5              & 0.624 $\pm$ .077 & 1.121 $\pm$ .115 & 1.546 $\pm$ .132 & 1.906 $\pm$ .153 & 1.846 $\pm$ .147 \\
GHVB3.5              & 0.459 $\pm$ .063 & 0.920 $\pm$ .107 & 1.877 $\pm$ .170 & 1.960 $\pm$ .147 & 1.779 $\pm$ .163 \\ 
\bottomrule
\end{tabular}
\caption{95\% confidence intervals for the numerical orders of convergence of GHVB (Figure \ref{fig:ablation_on_ghvb})} \label{tab:ablation_on_ghvb}
\end{table}

\begin{table}[ht]
\centering
\small
\begin{tabular}{lccc}
\toprule
\multicolumn{1}{c}{}   & \multicolumn{3}{c}{Number of steps}  \\
Method    & 15         & 30         & 60         \\ \midrule
DPM-Solver++           & 2.49       & 4.84      & 9.54 \\
DPM-Solver++ w/ HB 0.9 & 2.49       & 4.84      & 9.54 \\
PLMS4                  & 2.49       & 4.84      & 9.54 \\
PLMS4 w/ HB 0.9        & 2.46       & 4.79      & 9.43 \\
PLMS4 w/ NT 0.9        & 2.53       & 4.93      & 9.70 \\
GHVB3.9                & 2.50       & 4.84      & 9.54 \\ \bottomrule
\end{tabular}
\caption{Comparison of the average sampling time per image (in seconds) when using different numbers of steps in Stable Diffusion 1.5 on an NVIDIA GeForce RTX 3080. The time differences are marginal.} \label{tab:sampling_time}
\end{table}

\section{Ablation on Magnitude Score}

In this section, our objective is to provide further verification and justification of the experiment conducted in Section \ref{sec:exp_sd} by exploring various parameter settings for the magnitude score and assessing their effects on the selected model.

\subsection{Results with Alternative Parameter Settings}

To gain deeper insights into the integration of momentum into sampling methods, we analyze the results of the magnitude scores depicted in Figure \ref{fig:ablation_magnitude1} (Section \ref{sec:exp_sd}). This analysis involves varying the threshold $\tau$ and the kernel size $k$ for max-pooling in the calculation of the magnitude score. By investigating different parameter settings, we aim to validate the outcomes of the experiment and uncover the scaling impact of the magnitude score. The results, shown in Figure \ref{fig:magnitude_tau_kernelsize}, highlight the influence of threshold $\tau$ and kernel size $k$ on the magnitude score. It is important to note that while extreme values of $\tau$ or $k$ may introduce ambiguity in interpreting the outcomes, the overall observed trends remain consistent.

\tabulinesep=1pt
\begin{figure}
    \centering
    \begin{tabu} to \textwidth {
        @{}l
        @{\hspace{10pt}}c
        @{\hspace{5pt}}c
        @{\hspace{5pt}}c
        @{\hspace{5pt}}c
        @{}
    }

        \multicolumn{1}{c}{\shortstack[l]{\scriptsize Threshold}}
        & \multicolumn{1}{c}{\scriptsize $\tau = 0$}
        & \multicolumn{1}{c}{\scriptsize $\tau = 1.5$}
        & \multicolumn{1}{c}{\scriptsize $\tau = 3.0$}
        & \multicolumn{1}{c}{\scriptsize $\tau = 10.0$}
        \\

        \multicolumn{1}{c}{\shortstack[l]{\scriptsize kernel size \\ \scriptsize $k = 1$}} &
        \noindent\parbox[c]{0.21\columnwidth}{\includegraphics[width=0.21\columnwidth]{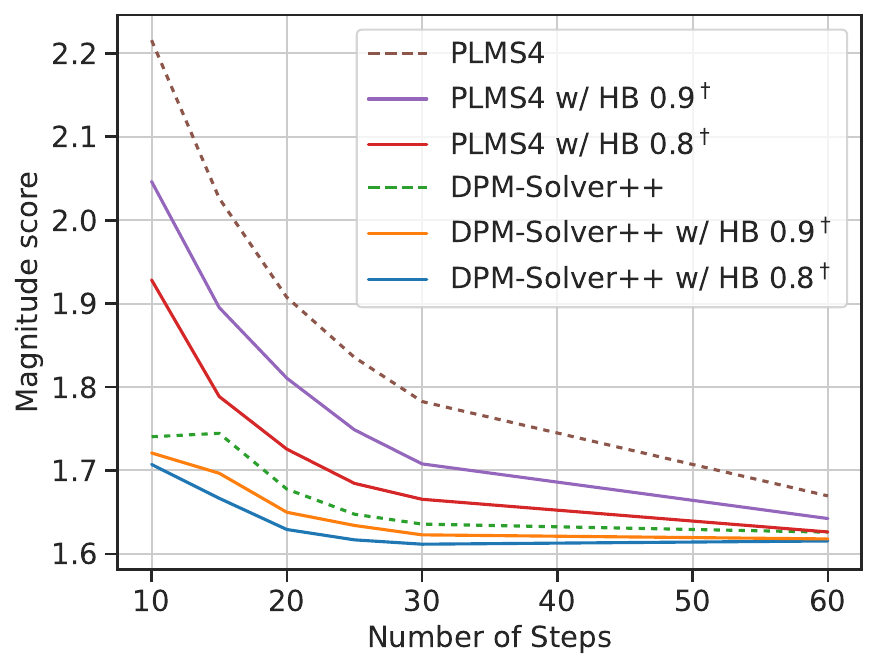}} & 
        \noindent\parbox[c]{0.21\columnwidth}{\includegraphics[width=0.21\columnwidth]{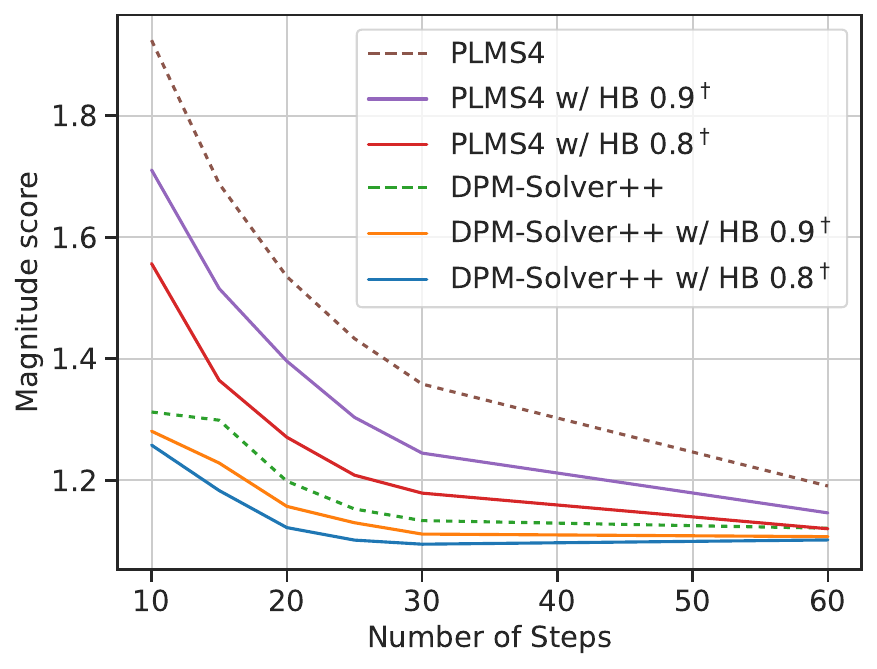}} & 
        \noindent\parbox[c]{0.21\columnwidth}{\includegraphics[width=0.21\columnwidth]{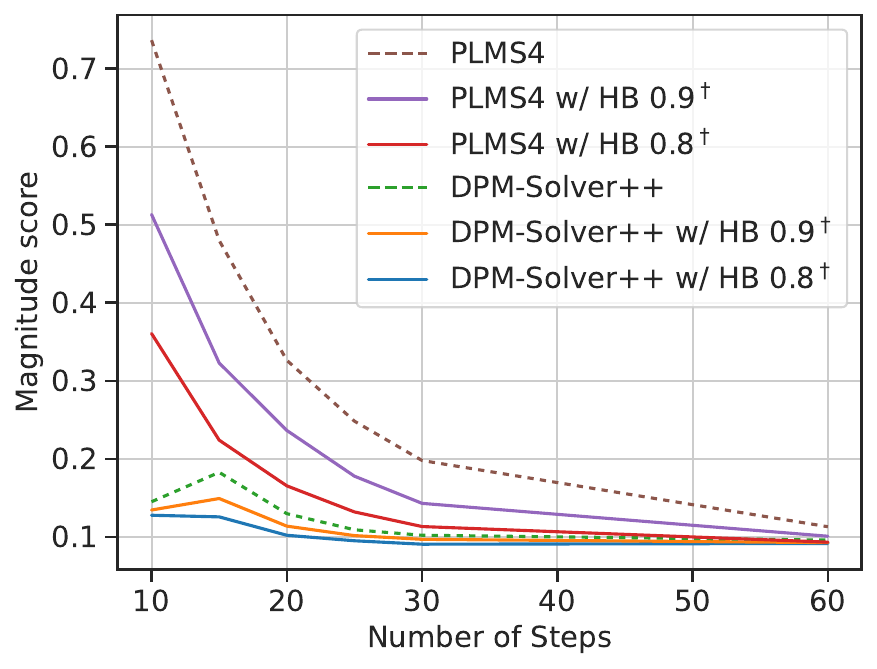}} &
        \noindent\parbox[c]{0.21\columnwidth}{\includegraphics[width=0.21\columnwidth]{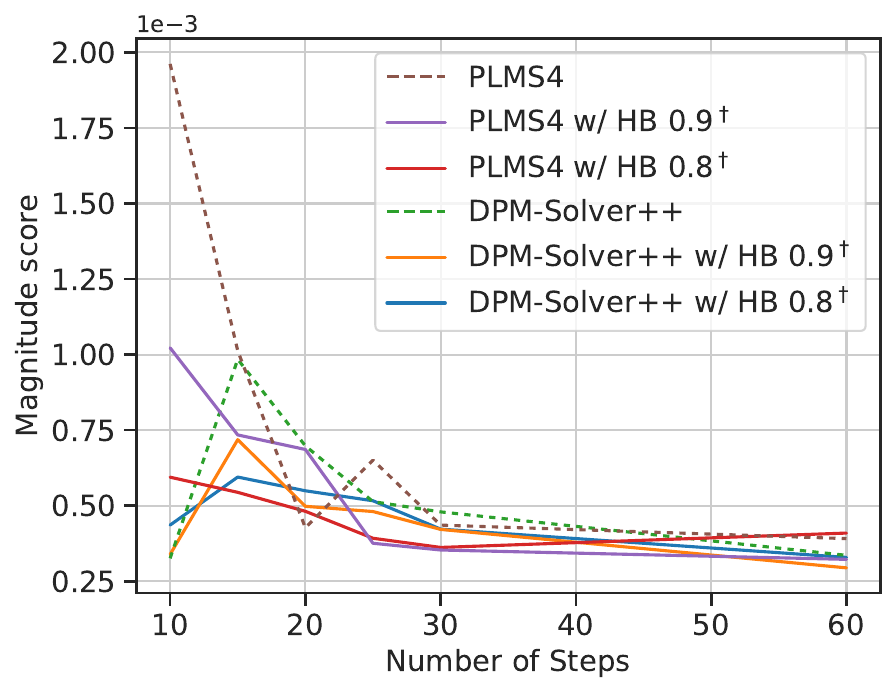}} \\

        \multicolumn{1}{c}{\shortstack[l]{\scriptsize kernel size \\ \scriptsize $k = 4$}} &
        \noindent\parbox[c]{0.21\columnwidth}{\includegraphics[width=0.21\columnwidth]{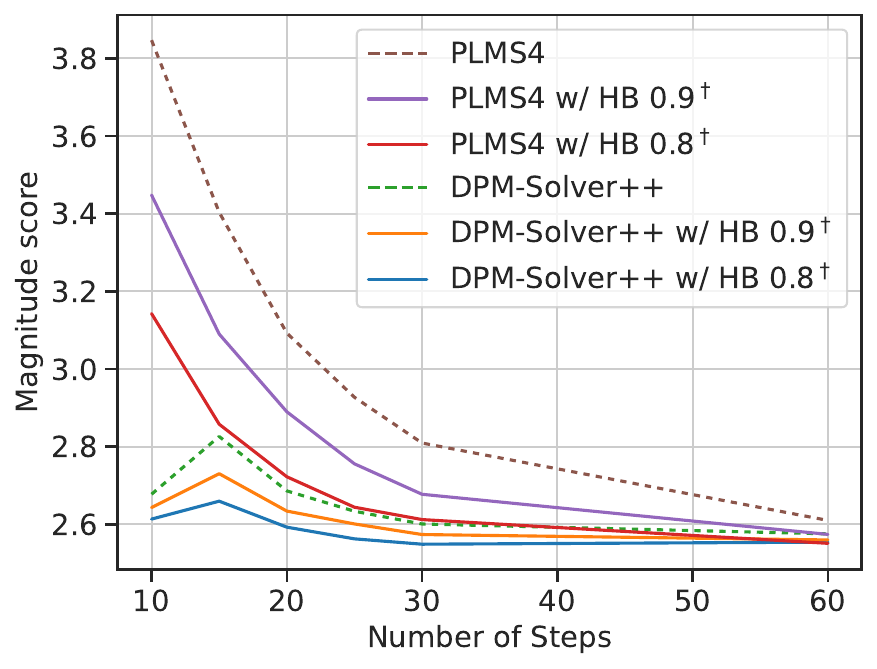}} & 
        \noindent\parbox[c]{0.21\columnwidth}{\includegraphics[width=0.21\columnwidth]{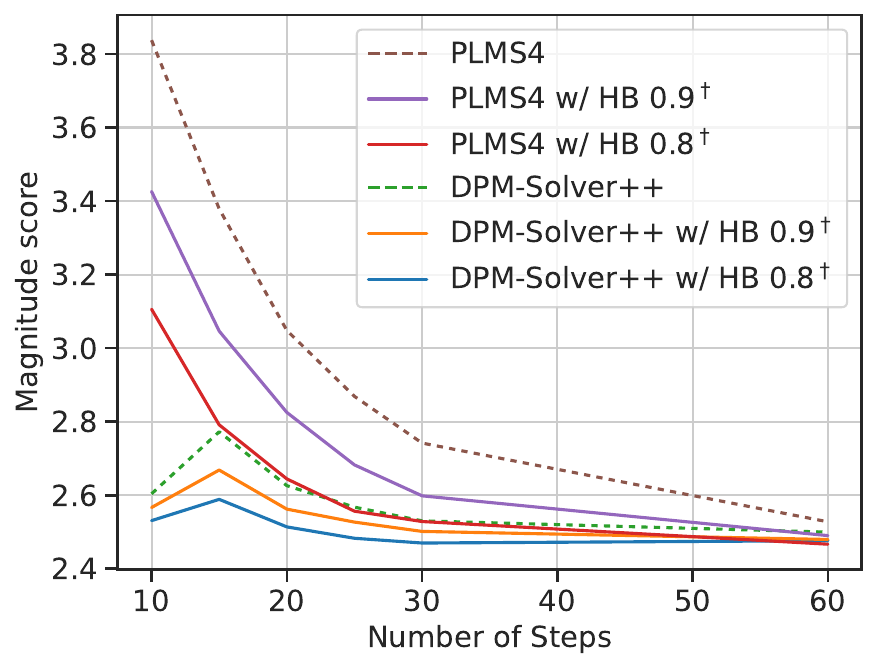}} & 
        \noindent\parbox[c]{0.21\columnwidth}{\includegraphics[width=0.21\columnwidth]{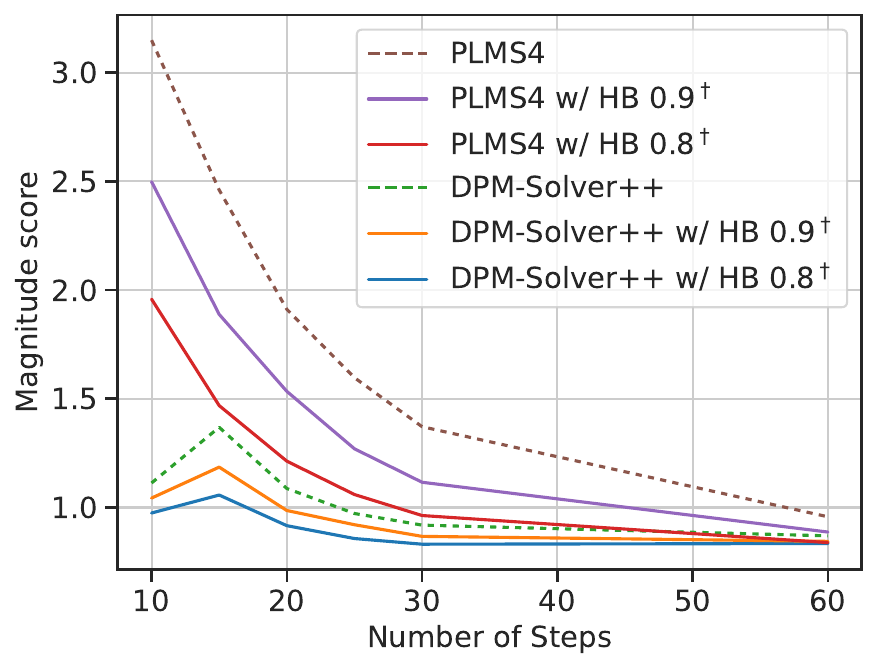}} &
        \noindent\parbox[c]{0.21\columnwidth}{\includegraphics[width=0.21\columnwidth]{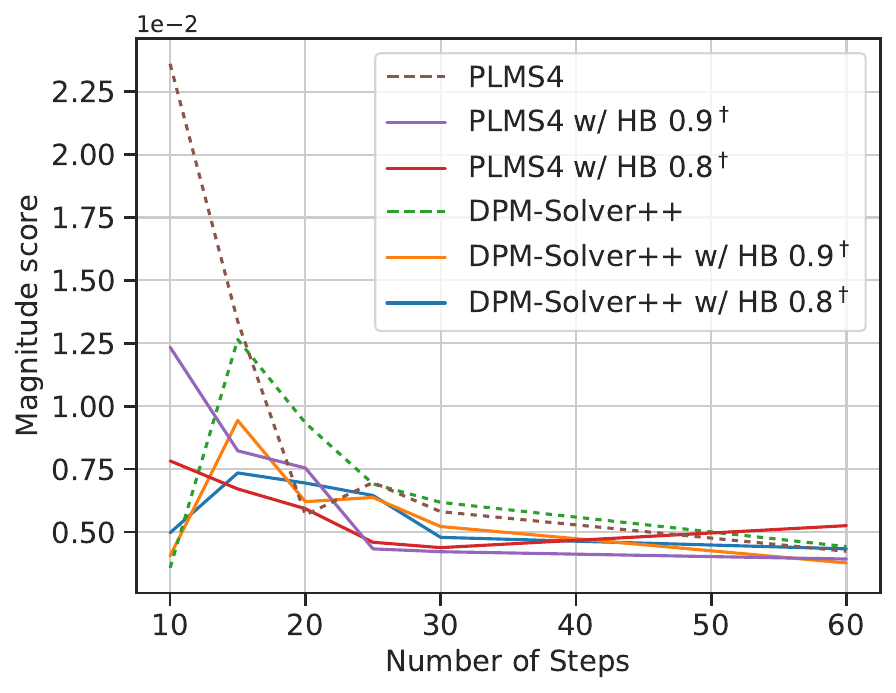}} \\

        \multicolumn{1}{c}{\shortstack[l]{\scriptsize kernel size \\ \scriptsize $k = 64$}} &
        \noindent\parbox[c]{0.21\columnwidth}{\includegraphics[width=0.21\columnwidth]{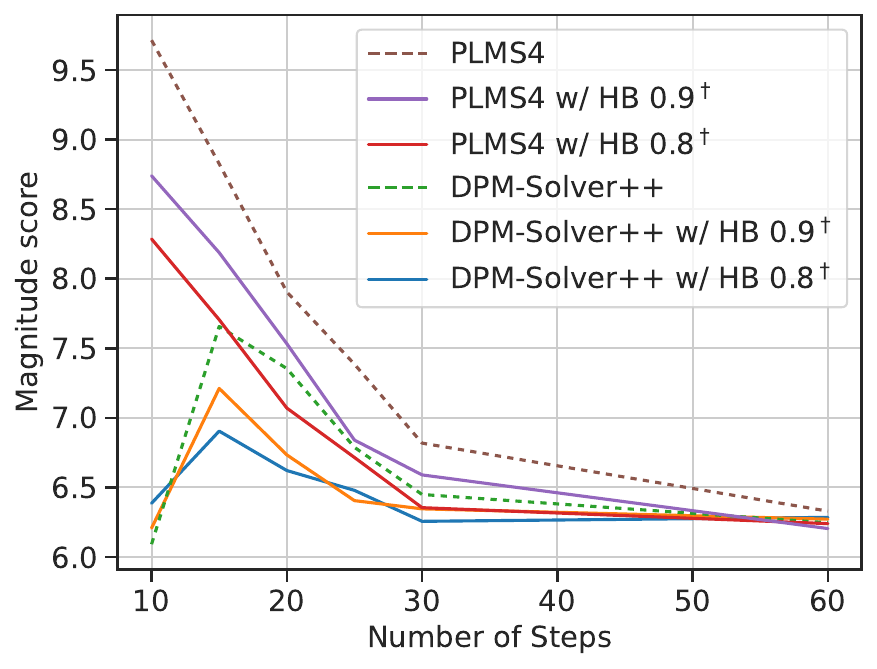}} & 
        \noindent\parbox[c]{0.21\columnwidth}{\includegraphics[width=0.21\columnwidth]{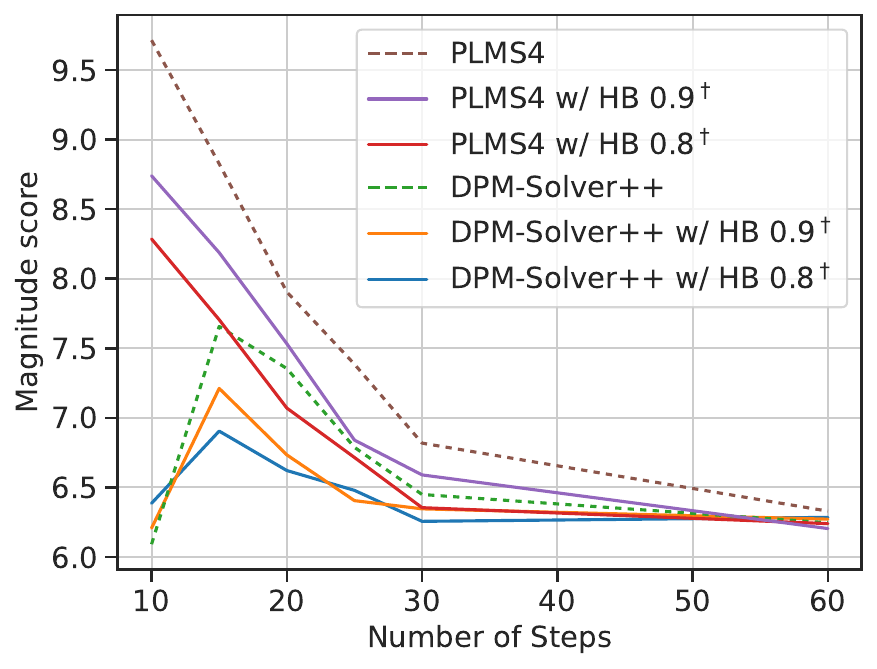}} & 
        \noindent\parbox[c]{0.21\columnwidth}{\includegraphics[width=0.21\columnwidth]{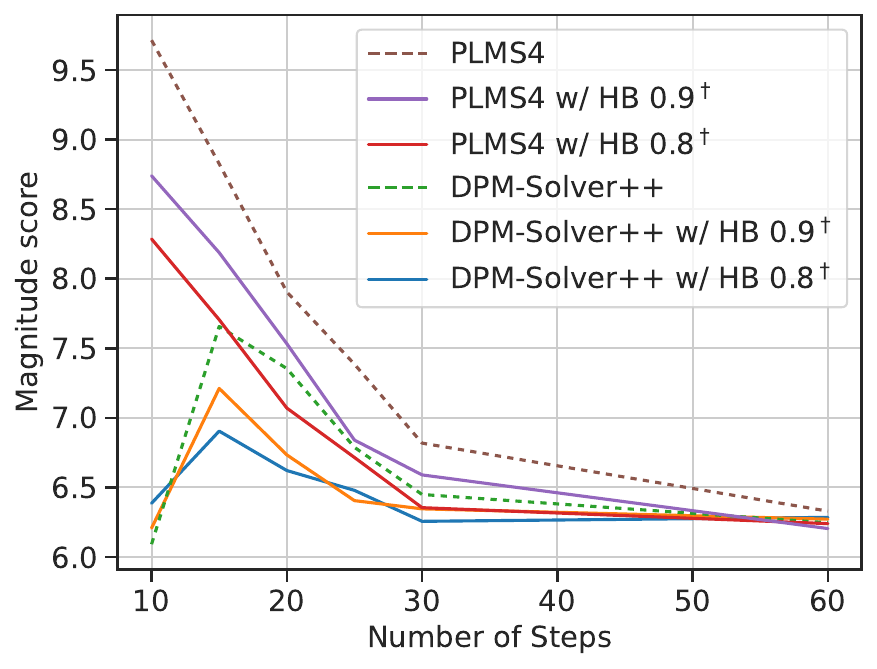}} &
        \noindent\parbox[c]{0.21\columnwidth}{\includegraphics[width=0.21\columnwidth]{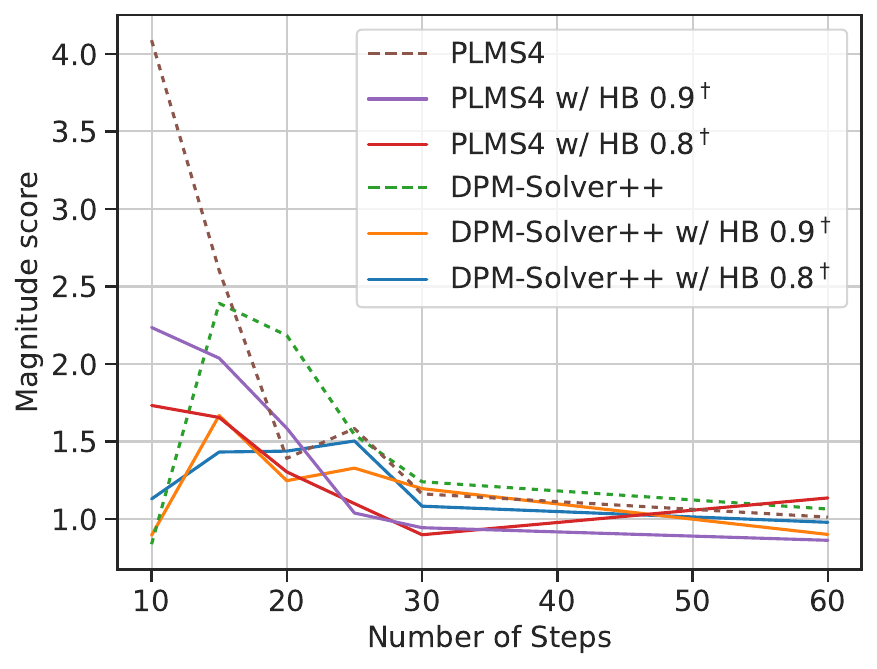}} \\

    \end{tabu}
    \caption{Comparison of magnitude scores on Anything V4 on different combinations of threshold $\tau$ and kernel size $k$ used in max-pooling. \#samples = 160}
    \label{fig:magnitude_tau_kernelsize}
\end{figure}

\subsection{Results on Alternative Models}

In this section, we present the findings from our analysis conducted on alternative diffusion models, namely Stable Diffusion 1.5, Waifu Diffusion V1.4, and Dreamlike Photoreal 2.0. The primary aim of this investigation is to assess the impact of different models on the magnitude score and determine whether the trends identified in Section \ref{sec:exp_sd} hold across diverse model architectures.

For this analysis, we employed the same magnitude score parameters as in Section \ref{sec:exp_sd}. The results of our examination are illustrated in Figure \ref{fig:magnitude_model}, which showcases the magnitude scores for each model.
One important observation is that the change in model architecture only affects the scale of the magnitude score, while the overall trend remains consistent across all models. 

\tabulinesep=1pt
\begin{figure}
    \centering
    \begin{tabu} to \textwidth {
        @{}c
        @{\hspace{10pt}}c
        @{\hspace{10pt}}c
        @{}
    }

        \multicolumn{1}{c}{\scriptsize Stable Diffusion 1.5}
        & \multicolumn{1}{c}{\scriptsize Waifu Diffusion V1.4}
        & \multicolumn{1}{c}{\scriptsize Dreamlike Photoreal 2.0}
        \\
        
        \noindent\parbox[c]{0.30\columnwidth}{\includegraphics[width=0.30\columnwidth]{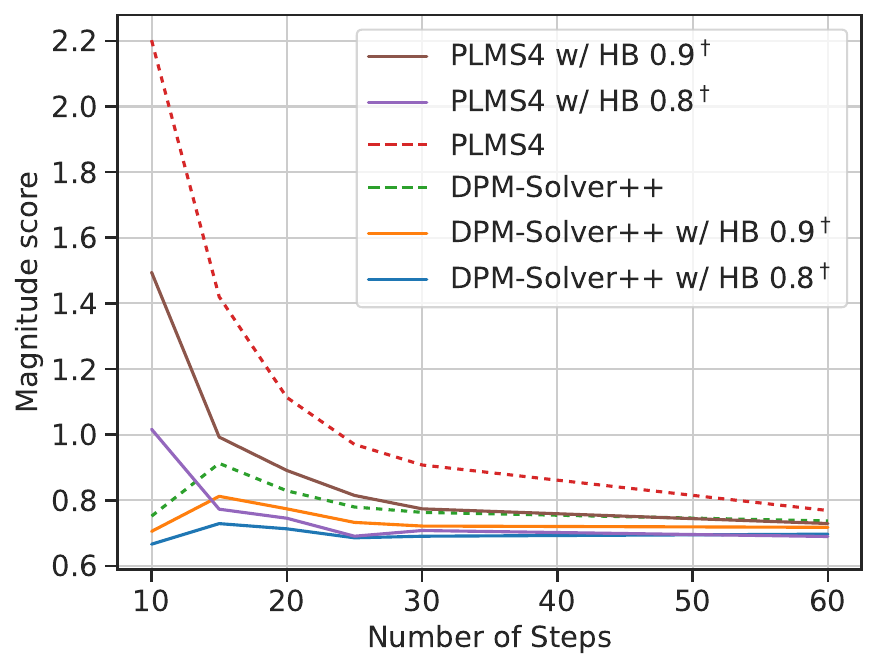}} & 
        \noindent\parbox[c]{0.30\columnwidth}{\includegraphics[width=0.30\columnwidth]{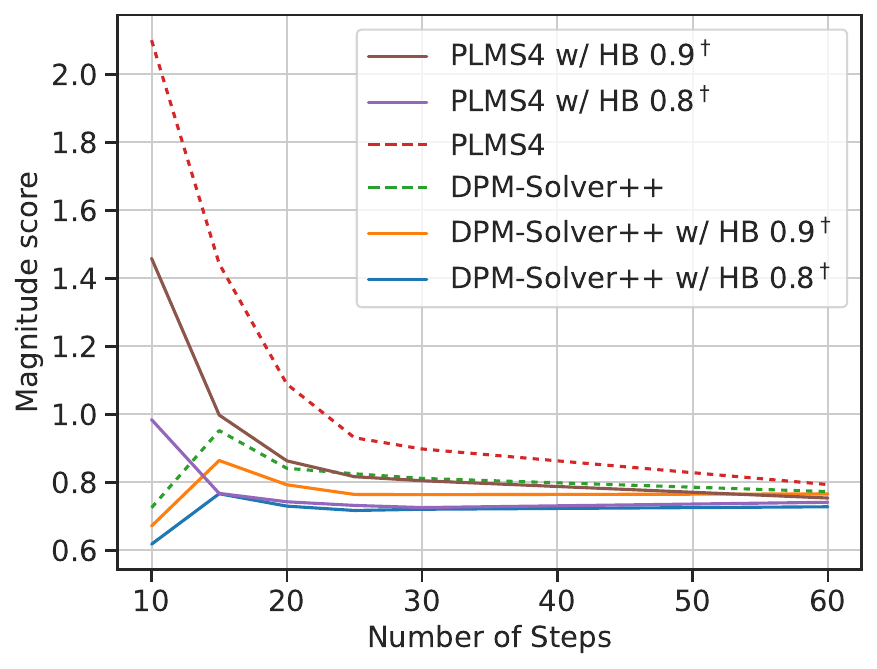}} & 
        \noindent\parbox[c]{0.30\columnwidth}{\includegraphics[width=0.30\columnwidth]{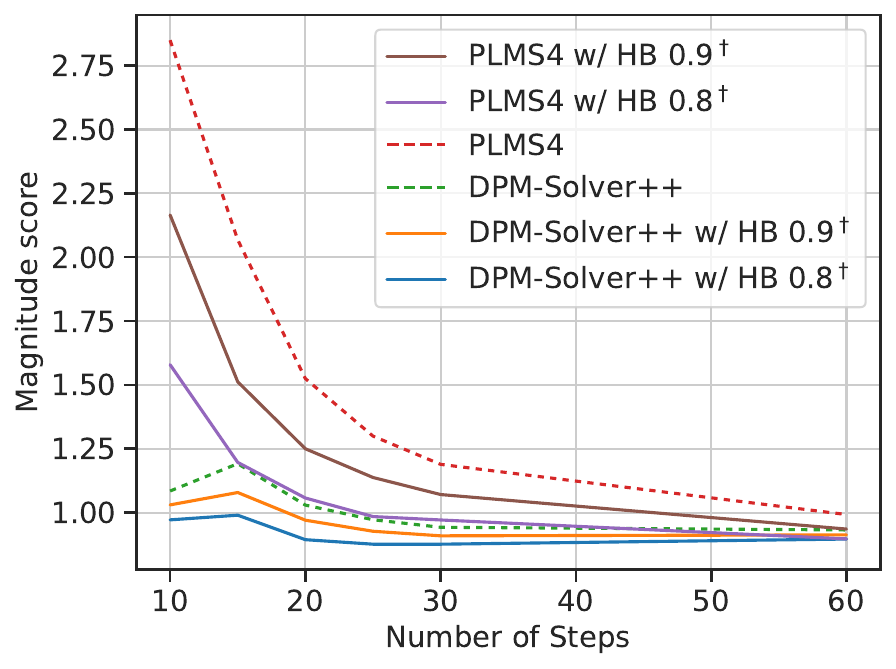}} \\

        \noindent\parbox[c]{0.30\columnwidth}{\includegraphics[width=0.30\columnwidth]{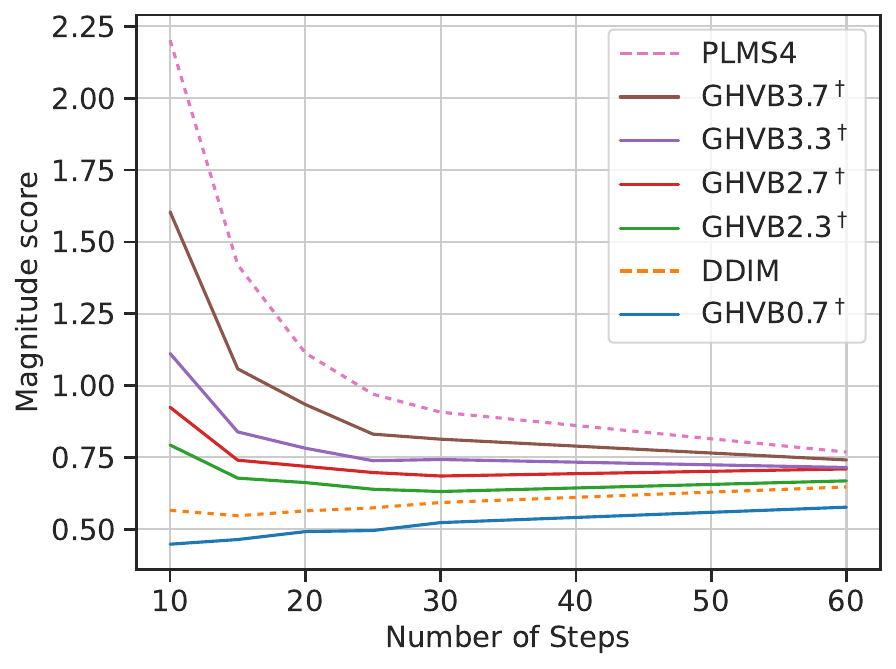}} & 
        \noindent\parbox[c]{0.30\columnwidth}{\includegraphics[width=0.30\columnwidth]{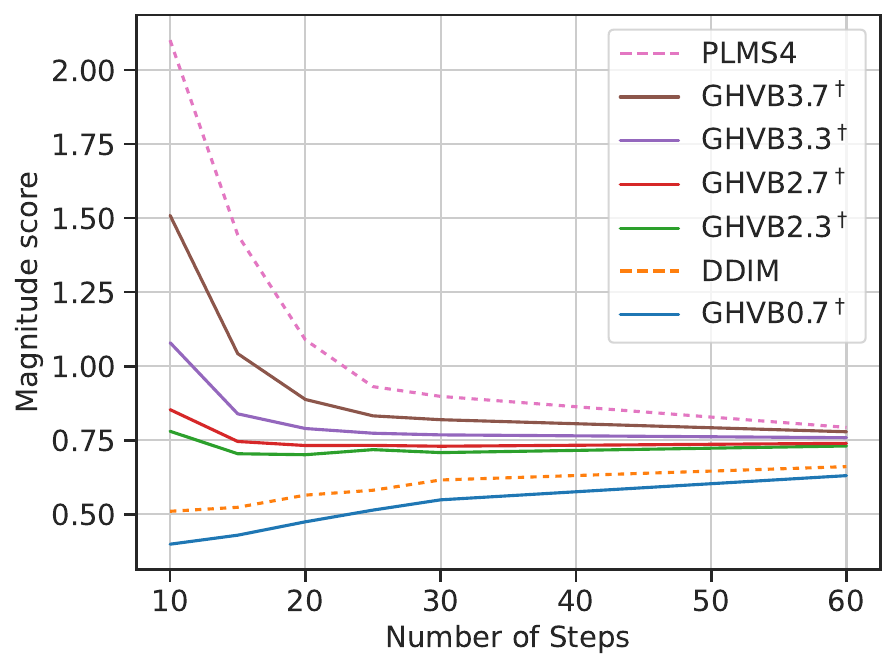}} & 
        \noindent\parbox[c]{0.30\columnwidth}{\includegraphics[width=0.30\columnwidth]{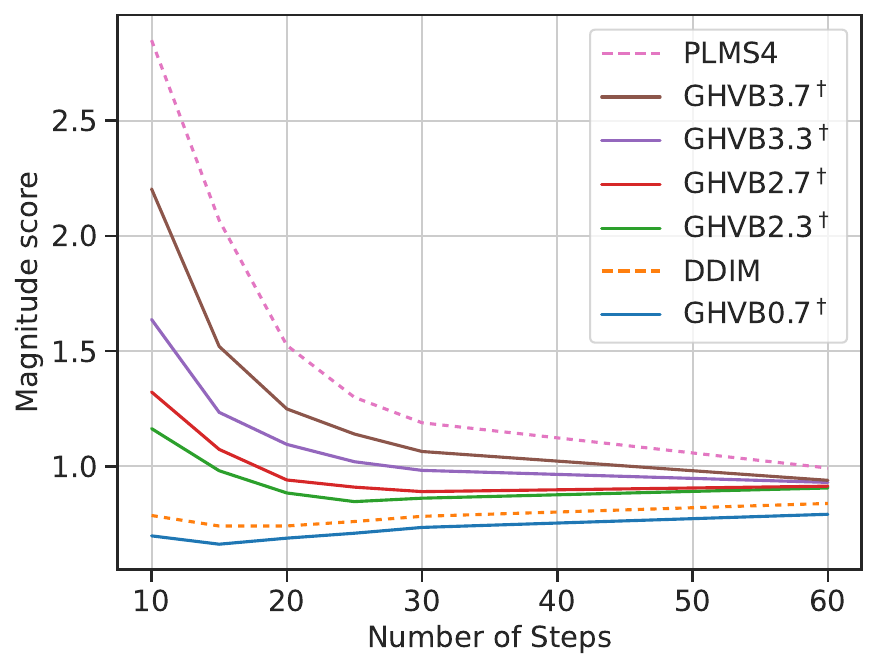}} \\

    \end{tabu}
    \caption{Comparison of magnitude scores in different diffusion models. \#samples = 160}
    \label{fig:magnitude_model}
\end{figure}


\section{Frequently Asked Questions} \label{apx:faqs}

\textbf{Q: What does the term ``divergence artifacts'' refer to?}

In this paper, the term ``divergence artifacts'' is used to describe visual anomalies that occur when the numerical solution diverges, resulting in unusually large magnitudes of the results. In the context of latent-based diffusion, we specifically define divergence artifacts as visual artifacts caused by latent codes with magnitudes that exceed the usual range. These artifacts commonly arise when the stability region of the numerical method fails to handle all eigenvalues of the system, leading to a divergent numerical solution. To visually demonstrate the presence of divergence artifacts, we have included Figure \ref{fig:latent_scaling}. This illustration showcases the process of starting with diffusion results and subsequently scaling the latent code within a $4\times4$ square located at the center of the latent image. This scaling is achieved by multiplying the latent code with a constant factor. As a result of this manipulation, divergence artifacts become distinctly visible, particularly at the center of the resulting image. This illustration provides a clear representation of the impact that scaling the latent code can have on the occurrence of divergence artifacts.


\textbf{Q: Can we directly interpolate two existing numerical methods instead of using the GHVB method?}

Indeed, this is possible. However, the order of the resulting method will be the lowest order of the two methods. To illustrate this point, let us consider a direct interpolation between the 1\ts{st}-order Euler method (AB1) and the 2\ts{nd}-order Adams-Bashford method (AB2), expressed as follows:

\begin{align}
x_{n+1} = x_n + \delta \left((1-\beta) f(x_n) + \beta\frac{3}{2}f(x_n) - \beta\frac{1}{2}f(x_{n-1})\right)
\end{align}

As outlined in Appendix \ref{apx:conv}, despite the orders of the interpolated methods, the resulting method is a 1\ts{st}-order method.

\begin{figure}
    \centering
    \begin{tabu} to \textwidth {
        @{}
        l@{\hspace{5pt}}
        c@{\hspace{5pt}}
        c@{\hspace{2pt}}
        c
        @{}
    }  
        &
        & \multicolumn{2}{c}{ Scaling factor} \\
        
        &  Original (no scaling) &  $\times 3.0$ &  $\times 6.0$  \\
        
        \shortstack[l]{Stable   Diffusion 1.5} &
        \noindent\parbox[c]{0.19\columnwidth}{\includegraphics[width=0.19\columnwidth]{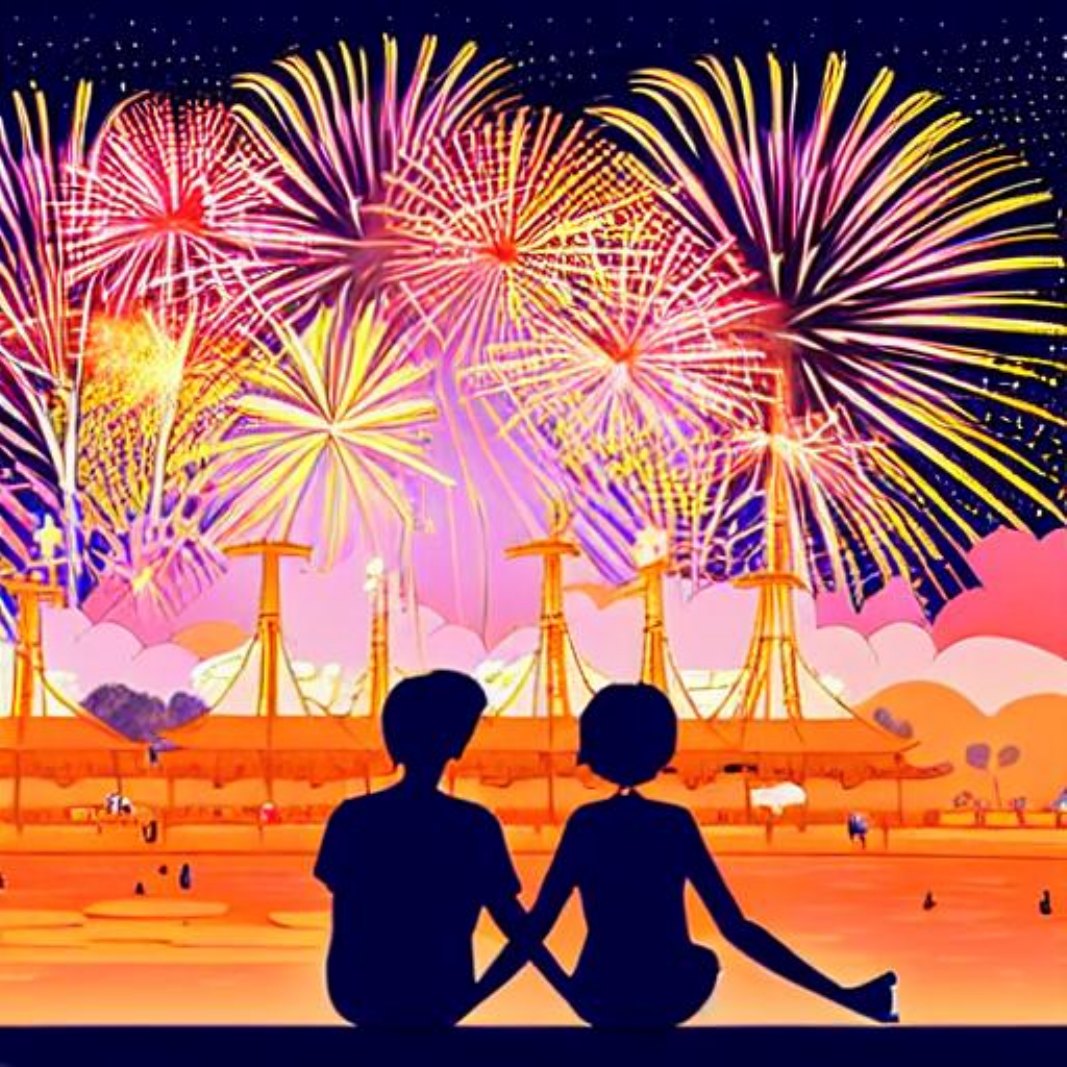}} & 
        \noindent\parbox[c]{0.19\columnwidth}{\includegraphics[width=0.19\columnwidth]{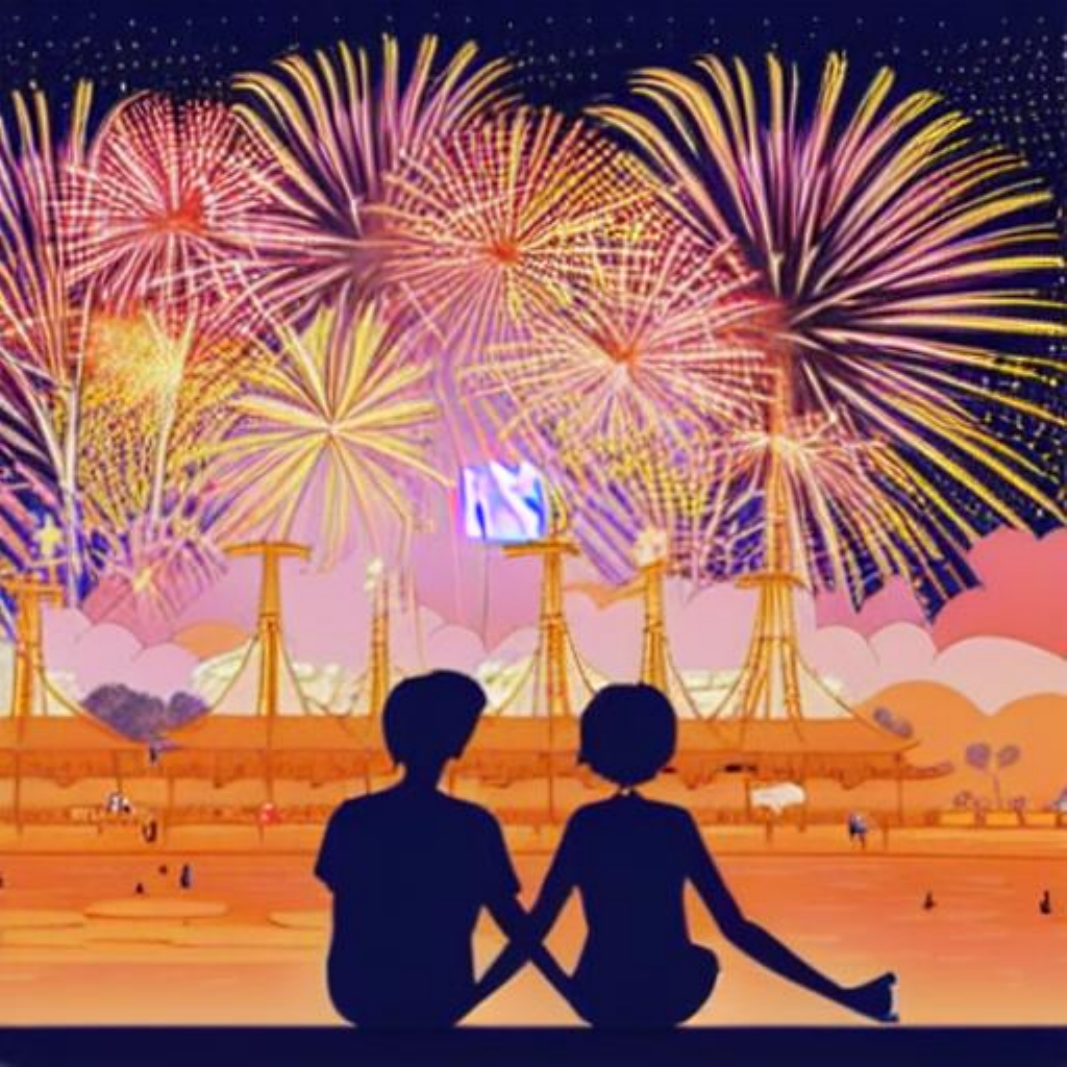}} & 
        \noindent\parbox[c]{0.19\columnwidth}{\includegraphics[width=0.19\columnwidth]{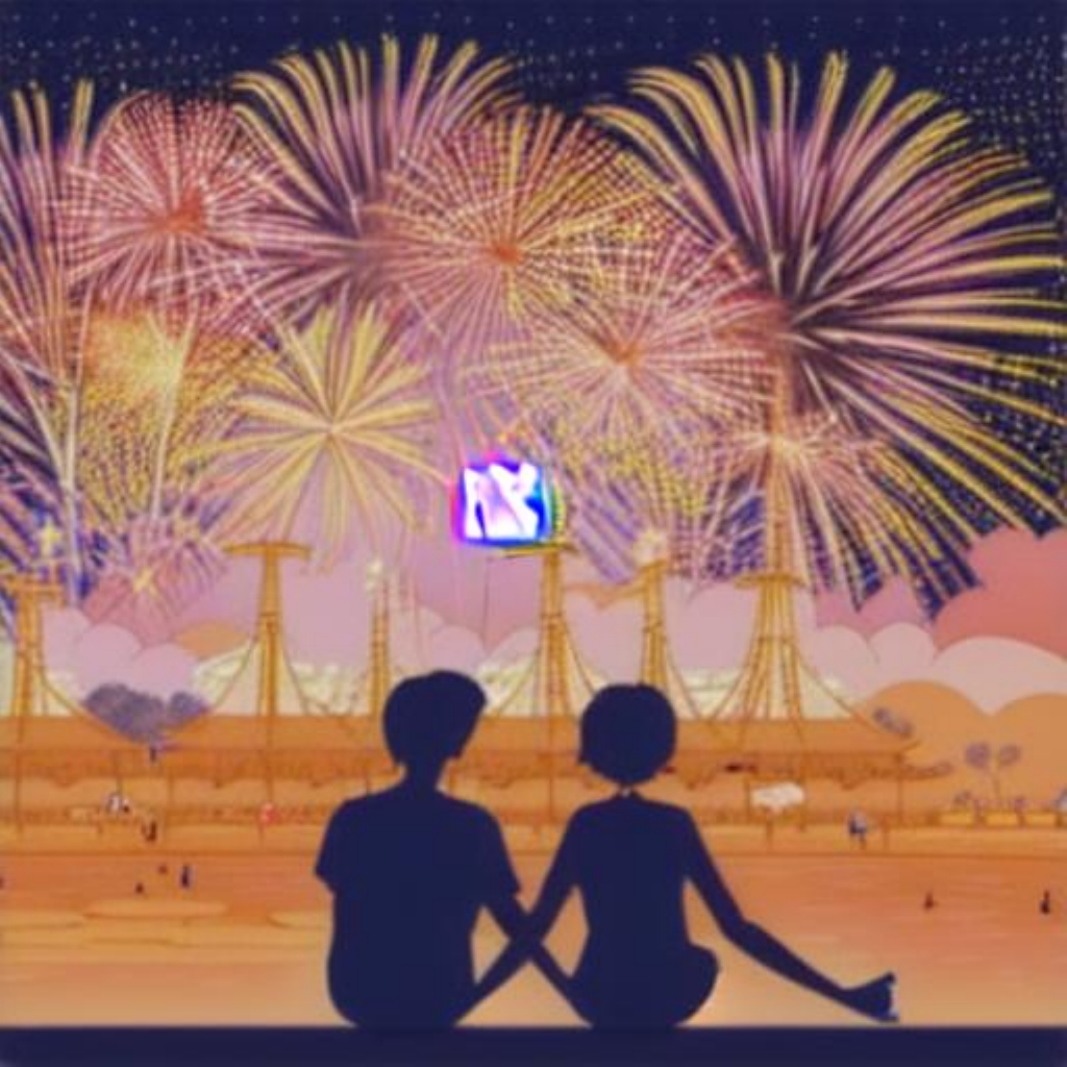}} \\

        \shortstack[l]{ Anything  Diffusion V4.0} &
        \noindent\parbox[c]{0.19\columnwidth}{\includegraphics[width=0.19\columnwidth]{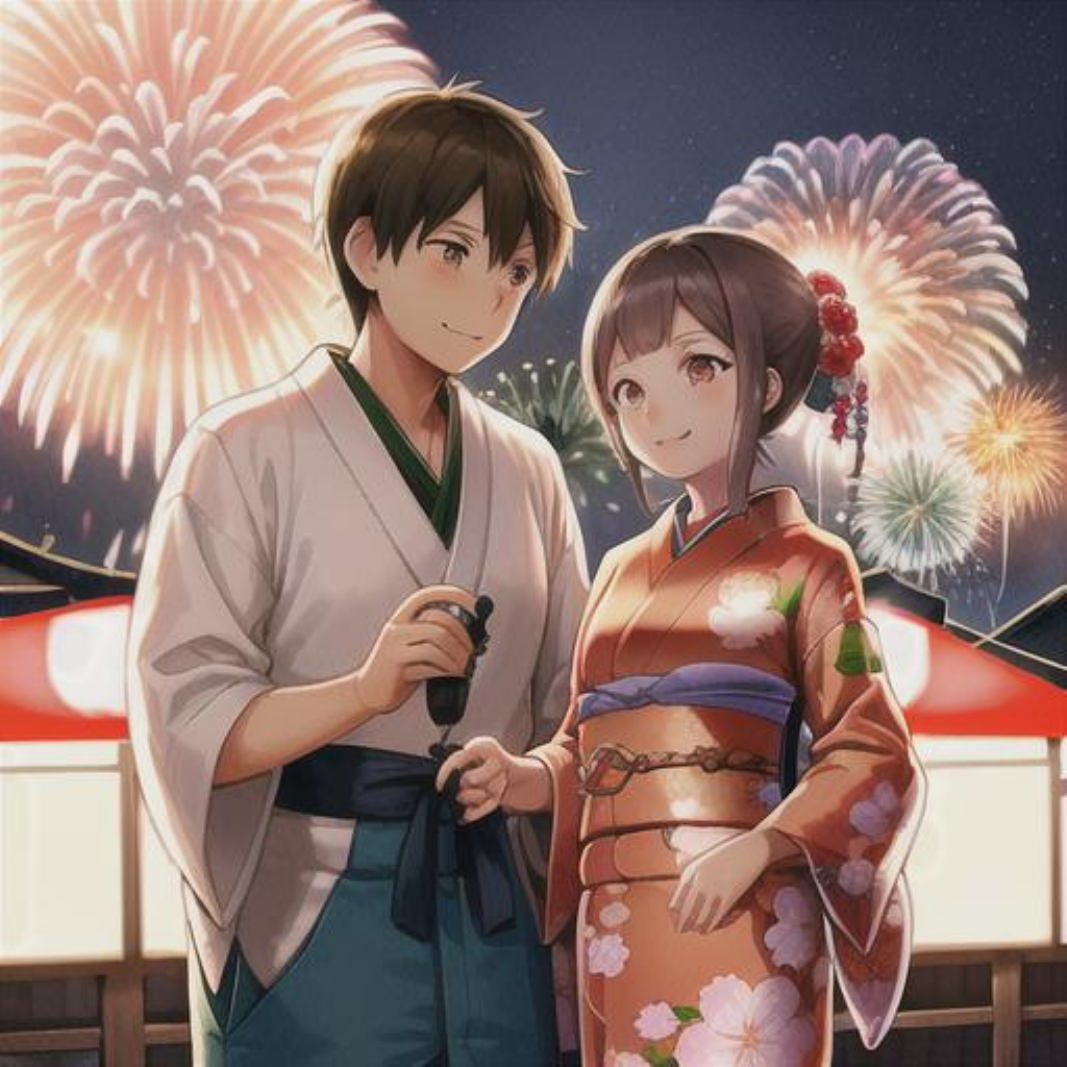}} & 
        \noindent\parbox[c]{0.19\columnwidth}{\includegraphics[width=0.19\columnwidth]{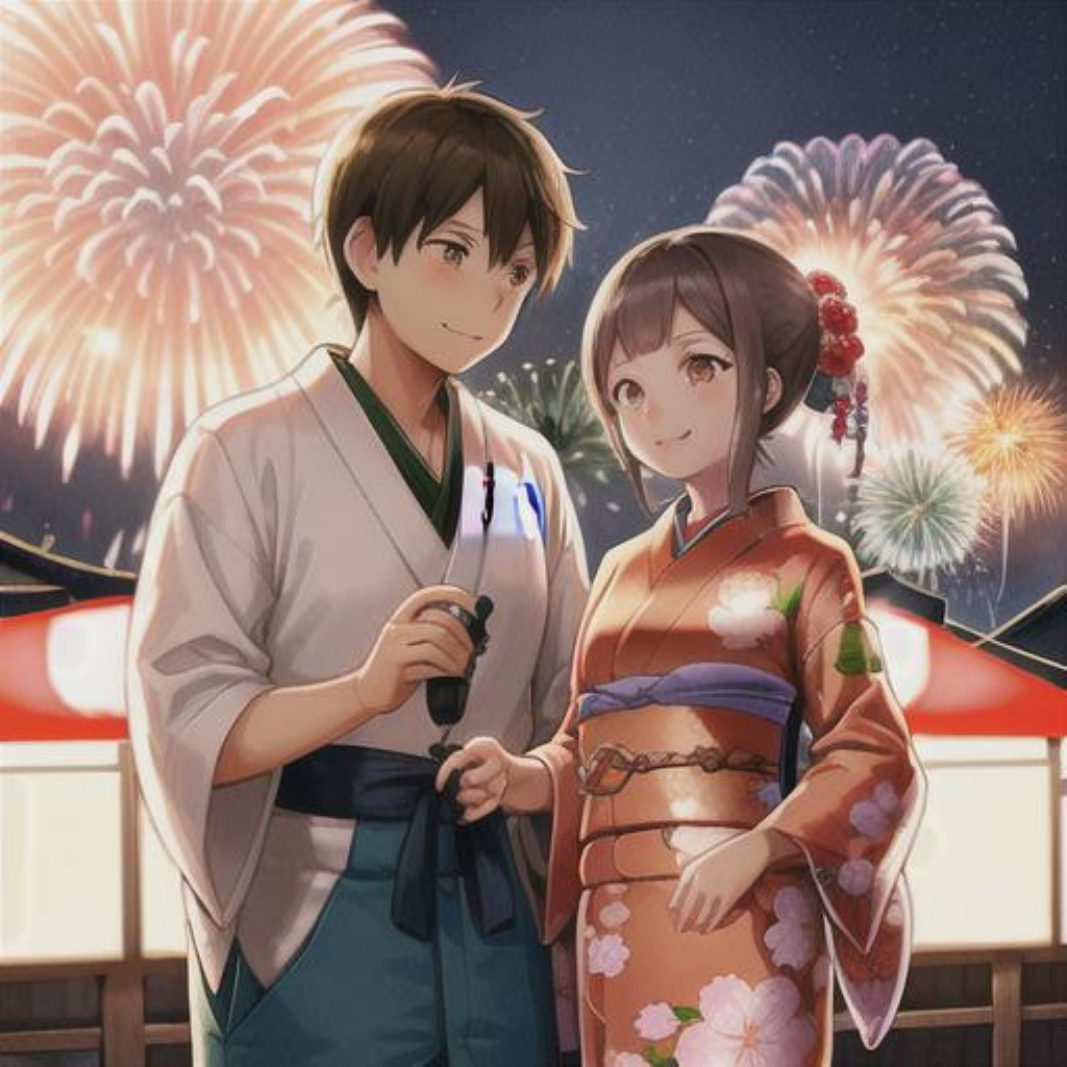}} & 
        \noindent\parbox[c]{0.19\columnwidth}{\includegraphics[width=0.19\columnwidth]{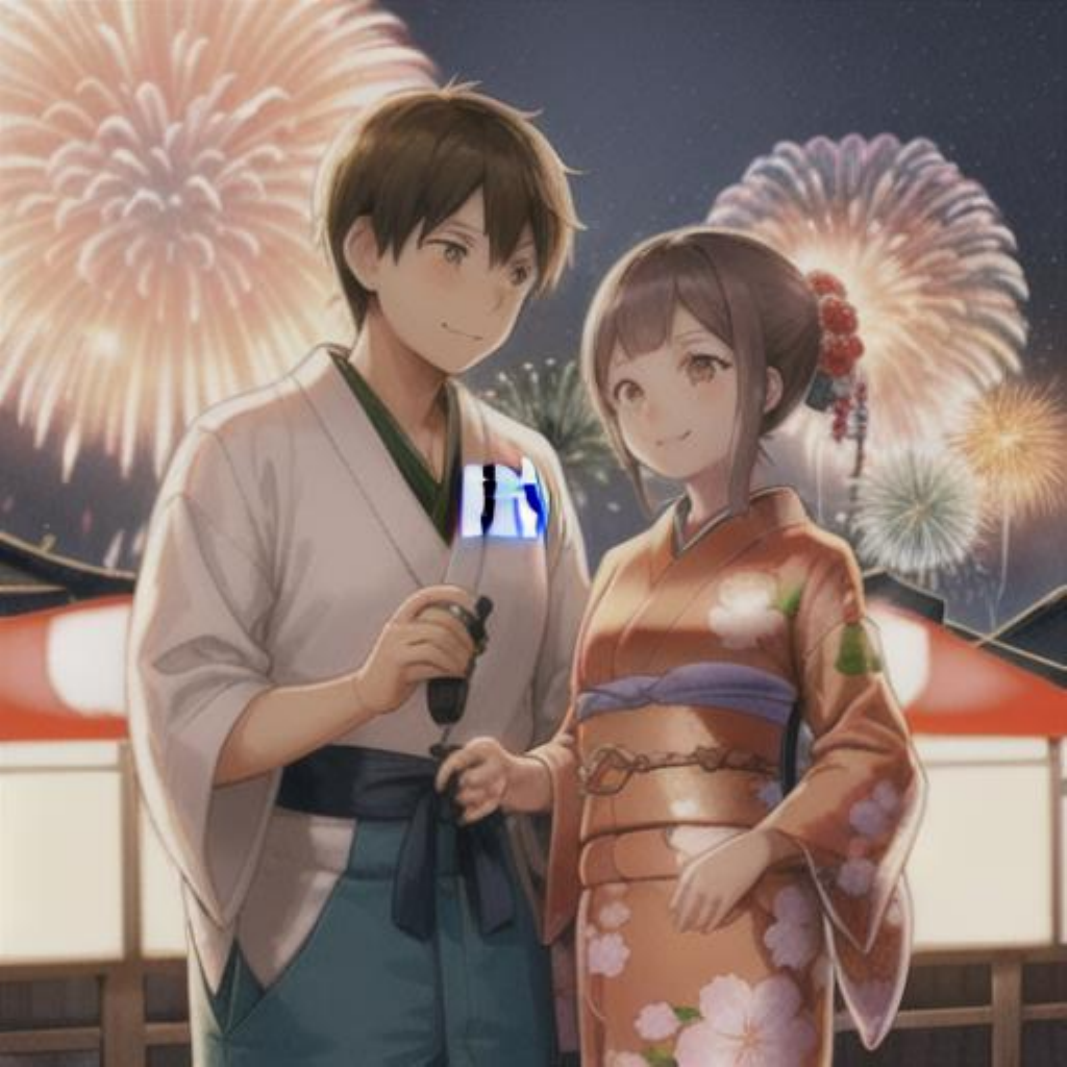}} \\

        \shortstack[l]{ Pastel-mix  Diffusion V4.0} &
        \noindent\parbox[c]{0.19\columnwidth}{\includegraphics[width=0.19\columnwidth]{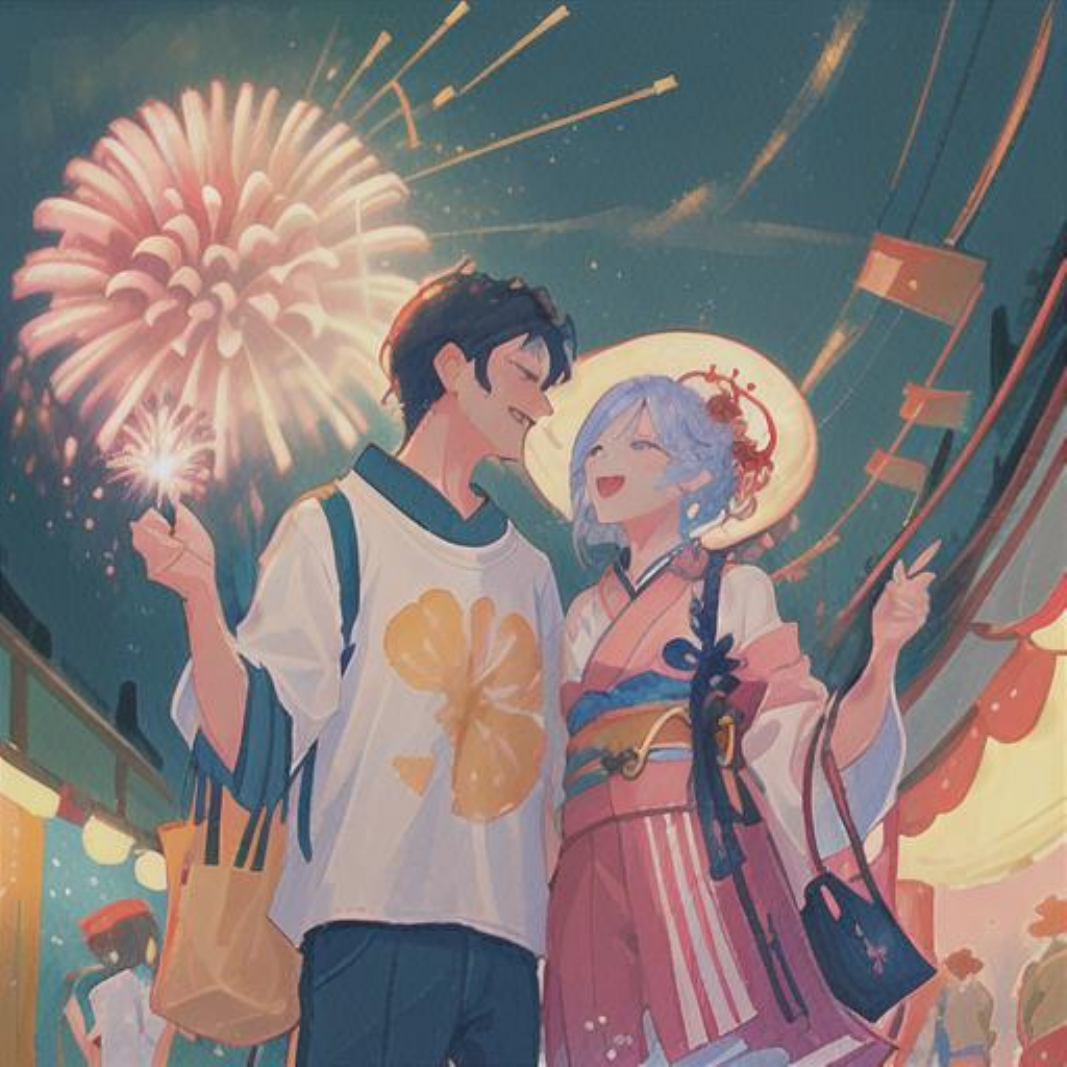}} & 
        \noindent\parbox[c]{0.19\columnwidth}{\includegraphics[width=0.19\columnwidth]{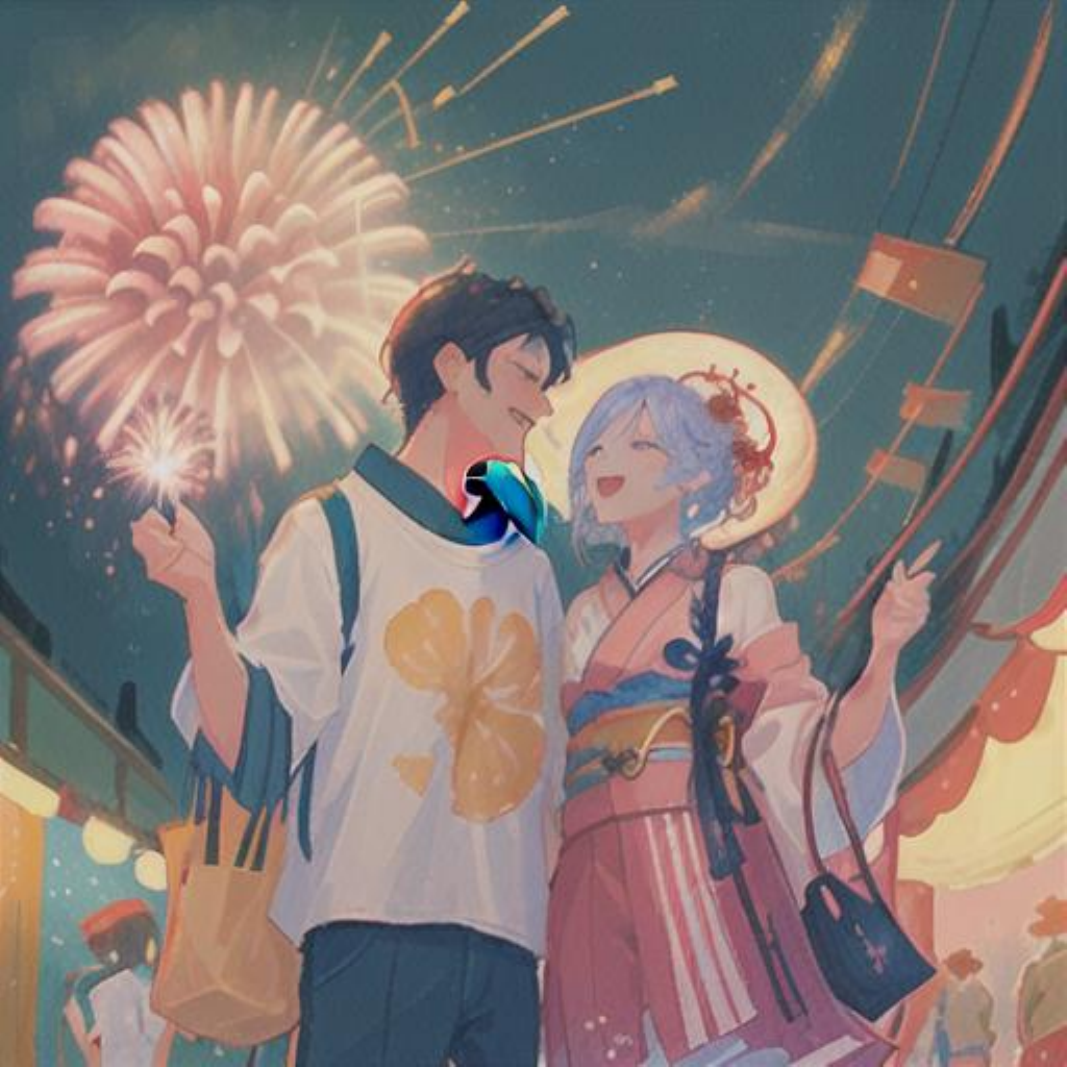}} & 
        \noindent\parbox[c]{0.19\columnwidth}{\includegraphics[width=0.19\columnwidth]{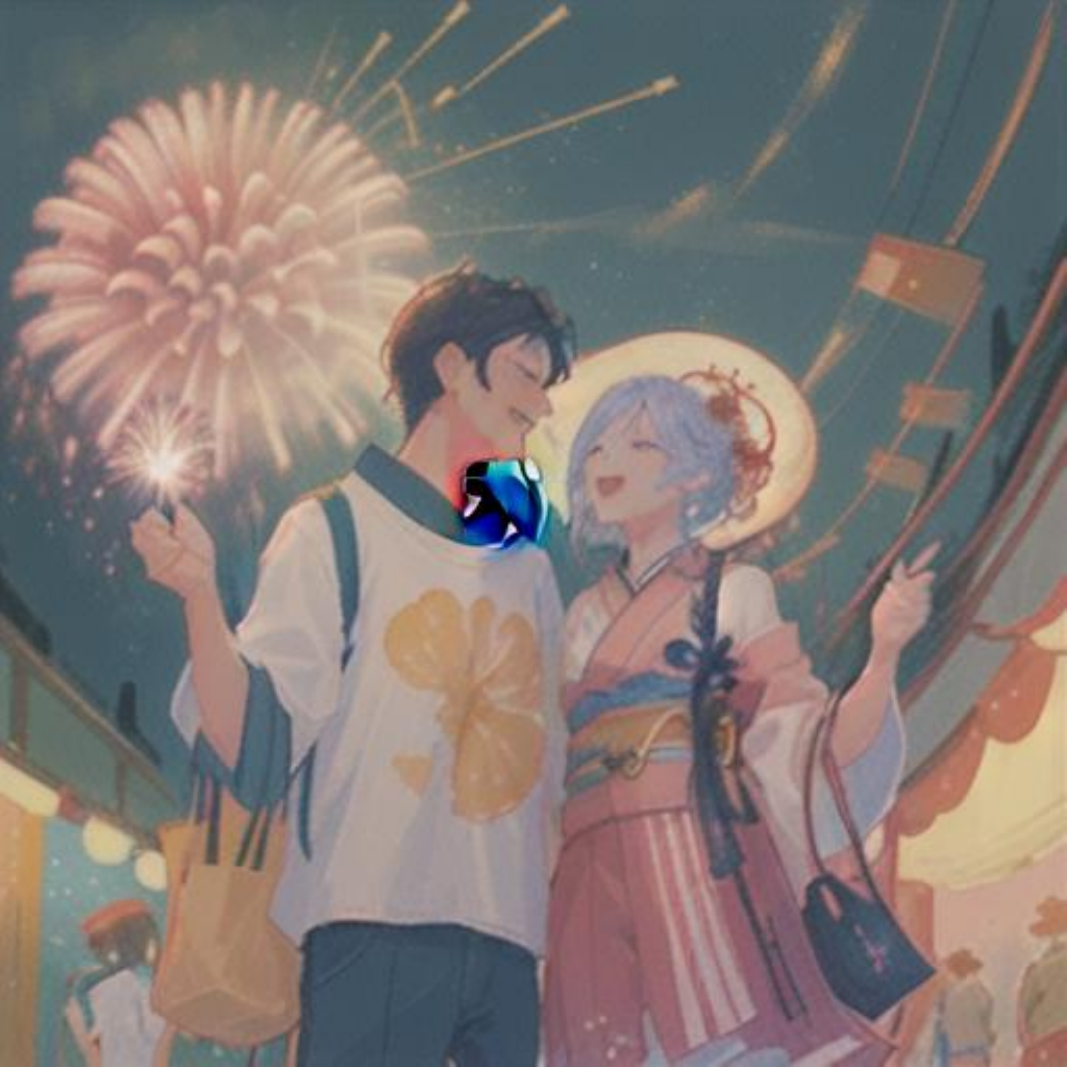}} \\
        
    \end{tabu}
    \caption{Visualization of divergence artifacts through latent code scaling. The figure illustrates the diffusion results obtained by scaling the latent code within a $4\times4$ square positioned at the center of the latent image. This scaling process involves multiplying the latent code with a constant factor. As a consequence, the resulting image showcases the emergence of divergence artifacts, which are particularly prominent at the center. The samples presented in this figure were generated using PLMS4 \cite{liu2022pseudo} with a guidance scale of 15 and 250 sampling steps. Prompt: "A beautiful illustration of a couple looking at fireworks in a summer festival in Japan"}
    
    \label{fig:latent_scaling}
\end{figure}

\end{document}